\documentclass[a4paper,10pt,twoside]{StyleThese}

\def\myauthor{Titouan Vayer}
\def\mytitle{A contribution to Optimal Transport for structured data and incomparable spaces}

\usepackage{amsmath,amssymb}             \usepackage[utf8]{inputenc}
\usepackage[T1]{fontenc}
\usepackage{lmodern}
\usepackage{array}
\usepackage{enumerate}
\usepackage[shortlabels]{enumitem}
\usepackage[framemethod=TikZ]{mdframed}

\usepackage{longtable}
\usepackage{etoolbox}
\usepackage{caption}
\usepackage{subcaption}
\usepackage{graphbox,graphicx}
\usepackage{xr}
\usepackage{epigraph}
\usepackage{pdfpages}
\usepackage{nameref}
\usepackage{cite}
\usepackage{todonotes}
\usepackage[]{amsthm}

\definecolor{ForestGreen}{rgb}{0.0, 0.27, 0.13}
\usepackage{listings}
\definecolor{backcolour}{rgb}{0.95,0.95,0.92}
\lstdefinestyle{mystyle}{
    backgroundcolor=\color{backcolour},
    basicstyle=\ttfamily\footnotesize,
    breakatwhitespace=false,
    breaklines=true,
    captionpos=b,
    keepspaces=true,
    numbersep=5pt,
    showspaces=false,
    showstringspaces=false,
    showtabs=false,
    tabsize=2,
    keywordstyle=\color{red!70!black},
    commentstyle=\color{ForestGreen!80!black}
}
\newcommand{\todoi}[1]{\todo[inline]{#1}}
\newcommand{\captionfonts}{\small}

\makeatletter  \long\def\@makecaption#1#2{  \vskip\abovecaptionskip
  \sbox\@tempboxa{{\captionfonts #1: #2}}  \ifdim \wd\@tempboxa >\hsize
    {\captionfonts #1: #2\par}
  \else
    \hbox to\hsize{\hfil\box\@tempboxa\hfil}  \fi
  \vskip\belowcaptionskip}
\makeatother   
\usepackage{tikz}
\usetikzlibrary{arrows,automata}

\usepackage[left=2.5cm,right=2.5cm,top=2.5cm,bottom=2.5cm,includefoot,includehead,headheight=13.6pt]{geometry}

\usepackage{aecompl}

\usepackage[intoc]{nomencl}

 \makenomenclature
 \renewcommand{\nomgroup}[1]{\ifthenelse{\equal{#1}{N}}{\item[\Large\sffamily\textbf{Notations}]}{\ifthenelse{\equal{#1}{X}}{\item[\Large\sffamily\textbf{Acronyms}]}{}}}

\usepackage{makeidx}
 \makeindex

\makenomenclature

\usepackage{ifpdf}

\ifpdf
  \usepackage{graphicx}
  \DeclareGraphicsExtensions{.jpg,.pdf,.png}
  \usepackage[pagebackref,hyperindex=true]{hyperref}
\else
  \usepackage{graphicx}
  \DeclareGraphicsExtensions{.ps,.eps}
  \usepackage[dvipdfm,pagebackref,hyperindex=true]{hyperref}
\fi

\usepackage{cleveref}[2012/02/15]\crefformat{footnote}{#2\footnotemark[#1]#3}

\graphicspath{{.}{imgs/}}

\renewcommand*{\backref}[1]{}
\renewcommand*{\backrefalt}[4]{\ifcase #1 (Not cited)\or
(Cited on page~#2.)\else
(Cited on pages~#2.)\fi}

\usepackage{color}
\definecolor{linkcol}{rgb}{0,0,0.4}
\definecolor{citecol}{rgb}{0.5,0,0}

\hypersetup
{
bookmarksopen=true,
pdftitle=\mytitle,
pdfauthor=\myauthor, pdfsubject="", pdfmenubar=true, pdfhighlight=/O, colorlinks=true, pdfpagemode=None, pdfpagelayout=SinglePage, pdffitwindow=true, linkcolor=linkcol, citecolor=citecol, urlcolor=linkcol }

\usepackage{pdfpages}
\usepackage{appendix}

\def\abs{\operatorname{abs}}
\def\argmax{\operatornamewithlimits{arg\,max}}
\def\argmin{\operatornamewithlimits{arg\,min}}
\def\diag{\operatorname{Diag}}

\usepackage{multicol}
\usepackage{rotating}                                                                                                      \usepackage{fancyhdr}

\let\headruleORIG\headrule
\renewcommand{\headrule}{\color{black} \headruleORIG}

\usepackage{colortbl}
\arrayrulecolor{black}

\fancypagestyle{plain}{
  \fancyhead{}
  \fancyfoot{}
  
}

\makeatletter

\def\cleardoublepage{\clearpage\if@twoside \ifodd\c@page\else  \hbox{}  \thispagestyle{empty}  \newpage  \if@twocolumn\hbox{}\newpage\fi\fi\fi}

\makeatother

\usepackage{multirow}

\newtheorem{definition}{Definition}[section]

\renewcommand{\epsilon}{\varepsilon}

\newcommand{\indentStd}{\noindent\hspace{\lenA}}

\newcommand{\lieu}[1]{\textsl{#1}}

\newcommand{\bleu}[1]{{\color{blue}{#1}}}

\newcommand{\pasfini}[1]{\bleu{}\todoi{Pas fini! je te dirai quand
    c'est fait}}

    \usepackage{minitoc}
\mtcsetfeature{minitoc}{open}{\vspace{1.5mm}}
\mtcsetfeature{minitoc}{close}{\vspace{1.5mm}}

\let\minitocORIG\minitoc
\renewcommand{\minitoc}{\minitocORIG \vspace{1.5em}}

\def\alphab{\boldsymbol\alpha}
\def\betab{\boldsymbol\beta}
\def\epsilonb{\boldsymbol\epsilon}
\def\Sigmab{\boldsymbol\Sigma}

\def\thetab{\boldsymbol\theta}

\def\G{\pi}
\def\GG{\boldsymbol\G}
\def\Gs{\pi^s}
\def\GGs{\boldsymbol\pi^s}
\def\Gv{\pi^v}
\def\GGv{\boldsymbol\pi^v}
\def\a{{\bf a}}
\def\b{{\bf b}}

\def\h{{\bf h}}
\def\g{{\bf g}}
\def\Gbf{{\bf G}}

\def\e{{\bf e}}
\def\w{{\bf w}}
\def\v{{\bf v}}
\def\x{{\bf x}}
\def\y{{\bf y}}
\def\V{{\bf V}}
\def\Dbf{{\bf D}}
\def\Bbf{{\bf B}}
\def\Mbf{{\bf M}}

\def\X{{\bf X}}
\def\Y{{\bf Y}}
\def\L{{\bf L}}
\def\R{{\mathbb{R}}}
\def\Abf{{\mathbf{A}}}
\def\Bbf{{\mathbf{B}}}
\def\U{{\mathbf{U}}}

\def\vec{{\text{vec}}}
\def\tr{{\text{tr}}}
\def\one{{\mathbf{1}}}

\newcommand{\MovieL}{\textsc{MovieLens}}
\newcommand{\COOT}{\text{COOT}}
\newcommand{\bz}{\mathbf{z}}
\newcommand{\bw}{\mathbf{w}}
\newcommand{\xbf}{\mathbf{x}}
\newcommand{\ebf}{\mathbf{e}}
\newcommand{\ybf}{\mathbf{y}}
\newcommand{\zbf}{\mathbf{z}}
\newcommand{\sbf}{\mathbf{s}}

\newcommand{\simplex}{\Sigma}
\newcommand{\ie}{\textit{i.e.}}

\newcommand{\couplingset}{\Pi}
\newcommand{\Pm}{\mathcal{P}}
\newcommand{\E}{\mathbb{E}}
\newcommand{\C}{\mathbf{C}}

\newcommand{\Sp}{\mathbb{S}}
\newcommand{\gw}{GW}
\newcommand{\wass}{W}
\newcommand{\sgw}{SGW}
\newcommand{\risgw}{RISGW}
\newcommand{\gm}{GM}
\newcommand{\D}{\Delta}
\newcommand{\Sn}{S_{n}}
\newcommand{\insided}{c}
\newcommand{\supp}{\text{supp}}
\newcommand{\Stief}{\mathbb{V}_{p}(\R^{q})}
\newcommand{\lebsm}{\mathcal{L}}

\usepackage{bm}

\newtheorem{prop}{Proposition}[section]
\newtheorem*{prop*}{Proposition}
\newtheorem*{theo*}{Theorem}
\newtheorem*{lemma*}{Lemma}

\newmdtheoremenv[innertopmargin=0pt]{theo}{Theorem}[section]
\newtheorem{corr}{Corollary}[section]
\newtheorem{lemma}{Lemma}[section]
\newtheorem{prob}{Problem}

\newtheorem{Example}{Example}[section]
\newtheorem{Remark}{Remark}[section]

\newmdtheoremenv[backgroundcolor=black!10,rightline=false,leftline=false,bottomline=false,innertopmargin=2pt]{memo}{Memo}[section]

\newcounter{Memo}[section]

\usepackage{microtype}
\usepackage{graphicx}
\usepackage{xcolor}
\usepackage{caption}
\usepackage{booktabs} \usepackage{amsfonts}
\usepackage{amsmath}
\usepackage{bbm}
\usepackage{amsthm}
\usepackage{multirow}
\usepackage{bm}
\usepackage{soul} \usepackage[normalem]{ulem} \setstcolor{blue}
\usepackage{hyperref}
\usepackage{wrapfig}
\usepackage{caption}
\usepackage{tabularx}
\usepackage{algorithm}
\usepackage{algpseudocode}

\usepackage{tikz}
\usepackage[mathscr]{euscript}
\usepackage{esint}

\newcommand{\fgwdistance}{FGW}
\newcommand{\Pcal}{\mathcal{P}}
\newcommand{\lossfgw}{E_{p,q,\alpha}}
\usepackage{mathtools}
\usepackage{nicefrac}
\usepackage{enumitem}
\usepackage{colortbl}

\newcommand{\kron}{\otimes_{K}}

\newcommand{\gwloss}{J}
\newcommand{\dom}{\text{dom}}

\newcommand{\F}{\mathcal{F}}
\newcommand{\Pbf}{\mathbf{P}}
\newcommand{\pbf}{\mathbf{P}}
\newcommand{\Obf}{\mathbf{O}}
\newcommand{\Qbf}{\mathbf{Q}}

\newcommand{\ubf}{\mathbf{u}}
\newcommand{\vbf}{\mathbf{v}}
\newcommand{\qbf}{\mathbf{q}}

\newcommand{\rg}{\text{rank}}

\newcommand{\integ}[1]{{[\![#1]\!]}}
\def\P{{\mathcal{P}}}
\def\Xcal{{\mathcal{X}}}
\def\Ycal{{\mathcal{Y}}}
\def\Zcal{{\mathcal{Z}}}
\def\Tcal{{\mathcal{T}}}
\def\ot{{\Tcal}}
\def\mongeot{{M}}

\newcommand{\froeb}[2]{{\langle {#1},{#2} \rangle_{\F}}}
\newcommand{\scalar}[3]{{\langle {#1},{#2} \rangle_{{#3}}}}
\newcommand{\dr}{\mathrm{d}}

\newcommand{\red}[1]{{\color{red} #1}}

\newcommand{\blue}[1]{{\color{blue} #1}}
\newcommand{\orange}[1]{{\color{orange} #1}}
\newcommand{\purple}[1]{{\color{purple} #1}}

\usepackage{multicol}

\newcommand\summaryname{Summary of the contributions}
\newenvironment{Abstract}    {\small\begin{center}    \bfseries{\summaryname} \end{center}}

\newcommand\tv[1]{#1}

\begin{document}
\renewcommand{\bibname}{{\sffamily Bibliography}}
\newcommand{\lc}[1]{\textcolor{magenta}{(#1)}}
\newcommand{\rt}[1]{\textcolor{green}{(#1)}}
\begin{titlepage}

\includepdf[pages={1},pagecommand={\pagestyle{fancy}}]{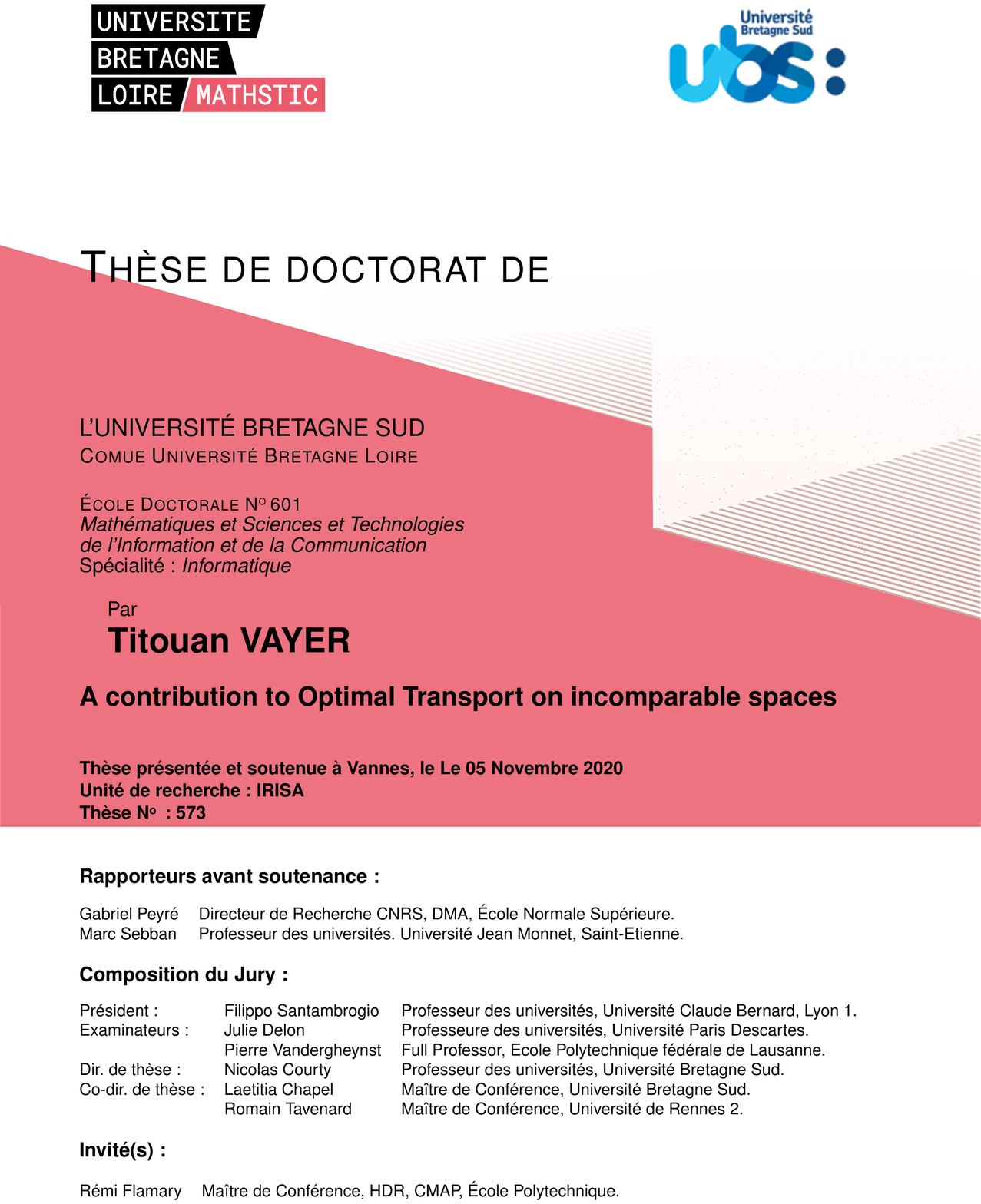}

\end{titlepage}
\sloppy

\titlepage

\chapter*{Acknowledgements}
\minitoc
\label{cha:ack}

Puisqu'il est ainsi permis de plumer plus librement, de mettre rigueur de côté, pour des remerciements :
\\ \\
Merci à mes encadrant.e.s: Laetita, Nicolas, Romain; les officiel.le.s. Vous m'avez permis de jouer aux mathématiques et à l'ordinateur avec une rare liberté, je pense. C'était excitant, gratifiant, inspirant et je garderai de votre confiance un souvenir certain. 
Merci Rémi, l'officieux mais tout autant présent. La seule chose égalant de fréquence tes enthousiasmes sont tes conjectures, souvent justes, je crois. 
Je m'excuse par retard pour mes nombreuses exothermies avortées, qui ont parsemé ces trois années. Au-delà de tout ce qui a permis ce manuscrit, merci aux moments d'ailleurs et de pas loin, aux discussions, aux rires. Cela me manquera. Merci aussi à Nicolas Klutchnikoff, tu m'as sauvé la mise dans des situations plutôt importantes; nos discussions de 100 ans avec Romain devant un tableau à Villejean resteront. Merci Ievgen, j'espère que l'on pourra continuer à jouer au Transport Optimal en dessous de la Loire. Merci à Obelix, ce fut un plaisir de passer ces trois années à vos côtés. 
\\ \\
Je remercie également les rapporteurs, Gabriel Peyré et Marc Sebban, ainsi
que les membres du jury, Julie Delon, Filippo Santambrogio, Pierre Vandergheynst, chercheur.e.s dont j’admire les travaux. Merci d'avoir pris le temps de vous intéresser aux miens.
\\ \\
Thank you Caglayan, we followed each other through space and time, and I am so happy that you were on this trip. I am confident that we will share a şerefe in Turkey one day. Until then I believe we can still have some Yec'hed mat in Rennes. 
\\ \\
Merci Mathilde, Maëlle, Thibault, Hadrien. Simple, sans vous il n'y aurait pas plus de trois lignes mal réglées sur ces feuilles. Vous m'avez rendu la vie légère, ce cocon rennais tourbillonnera sans doute longtemps là-haut, tissé de joyeux souvenirs. 
\\ \\
Merci Étienne, qui n'a pas de frère que le sang. Merci les Marie, centres du triangle triangle triangle. Des breton.ne.s se perdent: rennais.es, briochain.e.s, nantais.es (hé oui), vous êtes mille, ce sera toujours un grand plaisir de faire des saltos et de rire comme des baleineaux en votre compagnie. Merci la Lilla : j'aime à penser que votre rencontre marque le début de tout le tintouin qui risque de suivre. Merci Ludo et Max, sans vous je serais très probablement dans l'Empire. 
\\ \\
Merci à mes parents, à ma famille, sans doute n'ai-je pas réussi à rendre tout ça vraiment clair, mais avec le temps, qui sait... Vos enthousiasmes me touchent. Enfin merci Myriam, être à tes côtés m'apporte énormément, j'y ai sans aucun doute appris plus que dans toutes mes pérégrinations mathématiques.

\chapter*{R\'esum\'e (Français)}
\minitoc
\label{cha:resume_fr}

\section*{Introduction}

``Ainsi, l'on voit dans les Sciences, tantôt des théories brillantes, mais longtemps inutiles, devenir tout à coup le fondement des applications les plus importantes, et tantôt des applications très simples en apparence, faire naître l'idée de théories abstraites dont on n'avait pas encore le besoin, diriger vers les théories des travaux des Géomètres, et leur ouvrir une carrière nouvelle.'' C'est ainsi que Nicolas De Condorcet \cite{condorcet} au 18ème siècle a introduit les travaux de Gaspard Monge \cite{monge_81} qui sont au coeur de la théorie du transport optimal. Comment déplacer des masses d'un endroit à un autre de sorte à minimiser l'effort global de déplacement ? Condorcet avait raison : cette ``idée simple'' a maturé au fil des ans en une élégante théorie au croisement des mathématiques et de l'optimisation et est aujourd'hui au centre de nombreuses applications en machine learning.

D'une manière générale, l'intérêt du transport optimal réside à la fois dans sa capacité à fournir des relations, des correspondances, entre des ensembles de points et dans le fait qu'il induise une notion géométrique de distance entre des distributions de probabilité (voir Figure \ref{fig:ot_illus_intro}). Ces deux propriétés se sont révélées très utiles pour un large éventail de tâches qui sont, pour n'en citer que quelques-unes, le recalage d'images \cite{Haker:2001}, la recherche d'image par contenu \cite{rubner98}, l'adaptation de domaine \cite{courty2017optimal}, le traitement du signal \cite{kolouri_2017}, l'apprentissage non supervisé \cite{arjovsky17a,pmlr-v84-genevay18a}, l'apprentissage supervisé et semi-supervisé \cite{Frogner_2015,solomon14}, le traitement automatique du langage \cite{word_emb_doc_dist_2015}, l'équité en machine learning \cite{pmlr-v97-gordaliza19a}, ou encore en biologie \cite{biology_ot} ou en astrophysique \cite{Frisch}. 

\begin{figure}[t]
  \begin{center}
  \includegraphics[align=t,width=0.44\linewidth]{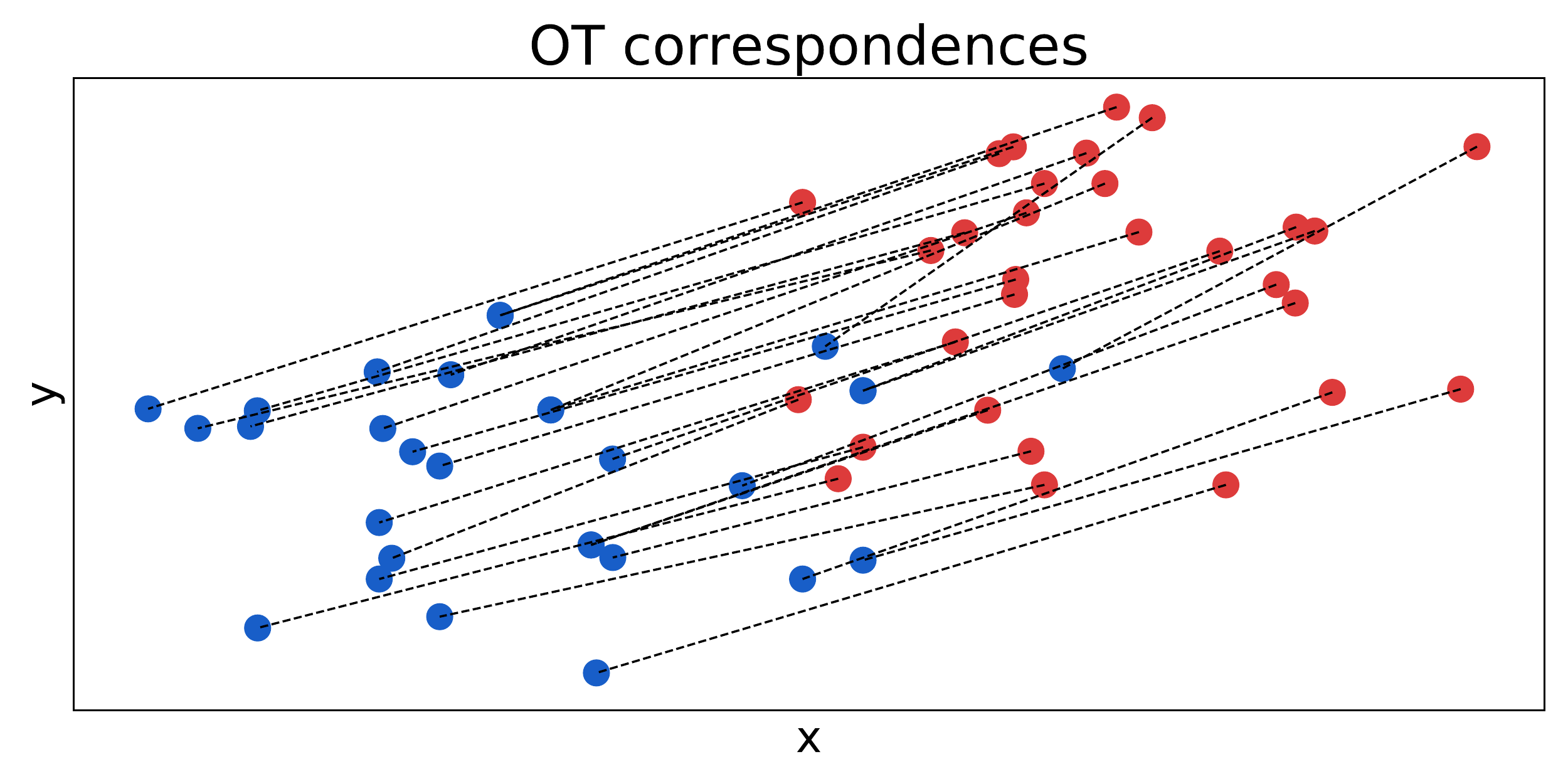}
  \includegraphics[align=t,width=0.55\linewidth]{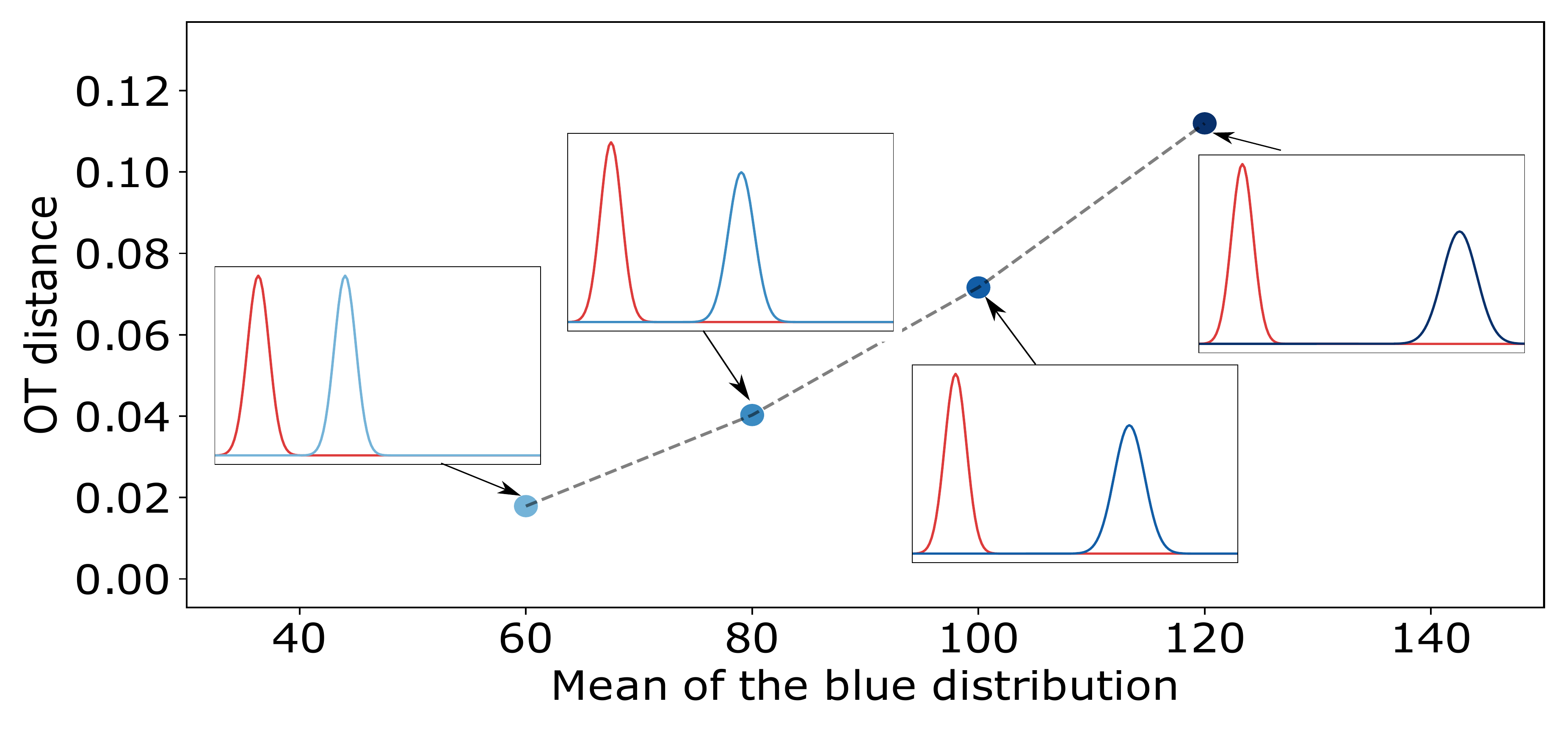}
  \end{center}
  \caption{Le Transport Optimal fournit des outils pour trouver des correspondances, relations, entre des ensembles de points et donne une notion géométrique de distance entre des distributions de probabilité. {\bf (à gauche)} Deux ensembles de points en 2D. Les correspondances trouvées par l'OT sont illustrées en pointillées. {\bf (à droite)} Les distances définies par l'OT entre deux distributions de probabilité bleue et rouge. Plus la moyenne de la distribution bleue est éloignée de la rouge, plus la distance est grande.}
  \label{fig:ot_illus_intro}
\end{figure}

Malgré ses nombreuses propriétés, le problème de transport optimal reste difficile à résoudre en pratique et il est connu pour souffrir de problèmes de scalabilité qui empêchent son utilisation sur des données volumineuses, omniprésentes en machine learning. L'émergence du transport optimal a été grandement favorisée par de récentes avancées dans le domaine de l'optimisation \cite{cuturi2013sinkhorn,altschuler2017near,genevay_stochastic}.  

De plus, dans sa formulation originelle, le transport optimal reste assez limité aux applications où il existe un ``moyen direct'' pour comparer les points, appelés \textit{samples}, provenant des distributions. Il est donc souvent limité aux cas où les samples font partie du même espace métrique, qui est la plupart du temps un espace euclidien. Cette limitation empêche notamment l'utilisation du transport optimal pour une variété de tâches dans lesquelles il existe une \emph{information de structure} supplémentaire sur les données, qui ne peut généralement pas être décrite par des espaces euclidiens. On peut citer par exemple le cas où les samples sont décrits par des graphes, des arbres ou des séries temporelles. Cette limitation empêche également son utilisation lorsque les samples se trouvent dans des espaces métriques différents, potentiellement non liés, ou lorsqu'une notion de distance entre les samples ne peut pas être facilement définie. Tous ces cas font partie de ce qu'on appellera par la suite le cas de figure \emph{incomparable}. Une solution intéressante se trouve dans l'élégante théorie de la distance de Gromov-Wasserstein \cite{memoli_gw} qui ne nécessite pas la comparaison des échantillons \emph{entre} les distributions. Cependant, cette distance est connue pour être encore plus difficile à résoudre que le transport optimal classique. L'objectif de cette thèse est d'aider à surmonter ces différents obstacles en:

\begin{enumerate}[label=(\roman*)]
\item Définissant de nouveaux problèmes de transport optimal sur des espaces incomparables et en particulier pour des données structurées.
\item Réduisant l'écart entre la compréhension théorique du transport optimal classique et celle du problème de Gromov-Wasserstein.
\end{enumerate}

\paragraph{Données structurées et espaces incomparables en machine learning}

Avant d'entrer dans le détail des contributions de cette thèse il est important de rendre explicite la notion de \emph{structure} et \emph{d'espaces incomparables}. Une première approche est de considérer une \emph{information structurelle} comme étant l'élément d'information qui encode les relations spécifiques qui existent entre les composants d'un objet. Cette définition peut être mise en relation avec le concept de \emph{relationnal reasoning} \cite{relationnalreasoning} dans lequel des \emph{entités} (ou des éléments ayant des attributs tels que l'intensité d'un signal) coexistent avec certaines \emph{relations} ou propriétés entre eux. 

Des tels cas de données structurées apparaissent naturellement lorsque la structure est \emph{explicite}. Par exemple, dans le contexte des graphes, les arêtes sont représentatives de la structure, de sorte que chaque attribut du graphe (généralement un vecteur de $\mathbb{R}^{d}$) peut être lié à d'autres par les arêtes entre les noeuds. Ces objets apparaissent notamment pour modéliser des composés chimiques ou des molécules~\cite{DBLP:journals/corr/KriegeGW16}, la connectivité du cerveau~\cite{ktena2017distance} ou encore les réseaux sociaux~\cite{Yanardag15}. Cette famille générique de données structurées comprend également les arbres~\cite{day1985optimal} ou encore les séries temporelles dont les valeurs sont corrélées dans le temps, de sorte leur comparaison nécessite de prendre en compte dans la modélisation la structure temporelle, la direction du temps.

Les informations de structure des données en machine learning peuvent apparaître de manière plus subtile voire même \emph{implicitement}. Elles peuvent se matérialiser lorsqu'on construit une \emph{prior} ou un \emph{biais} structurel sur la représentation des données. Par exemple, dans le contexte de l'apprentissage profond, de nombreuses architectures exploitent \emph{l'équivariance} à une transformation pour améliorer la généralisation des modèles. Les réseaux de neurones convolutifs (CNN) sont un excellent exemple de cette prior structurelle: un CNN est équivariant aux translation, \ie\ si nous translatons l'entrée du réseau de neurone, la sortie des convolutions sera également translatée. Ce biais structurel est connu pour révéler une certaine hiérarchie spatiale sur les pixels utile dans de nombreuses applications \cite{Wang2018NonlocalNN,chen_iterative_2018}. D'autres travaux ont étudié la conception de couches avec des équivariances à d'autres transformations telles que les permutations, rotations ou réflexions \cite{kondor2018generalization,pmlr-v48-cohenc16,NIPS2014_5424}. 

Contrairement aux méthodes ``end-to-end'' telles que les réseaux de neurones, d'autres approches plus ``hand-engineering'', basées sur la segmentation des images, permettent de dévoiler une certaine structure utile sur les images \cite{bachgraphkernel,Jianbo_cuts}. Les structures implicites sont également au cœur de nombreux outils de traitement du langage naturel (NLP) utilisés pour trouver de bonnes représentations vectorielles des mots \cite{mikolov_2013,word2vec,2014-glove}, pour la reconnaissance vocale \cite{hinton_2012} ou plus généralement pour l'apprentissage de séquences \cite{seq_to_seq}. Dans ces cas, les biais structurels passent soit par l'utilisation de variables latentes ou celle de probabilités conditionnelles. Lorsqu'ils sont disponibles, les labels ou les classes induisent également une structure implicite sur les features des données. Par exemple, en adaptation de domaines, on peut souhaiter que les samples du domaine source ayant le même label soient appariés de manière cohérente dans la même région de l'espace cible, et ainsi éviter qu'ils ne soient divisés en des emplacements trop éloignés \cite{courty2017optimal,AlvarezMelis2018Structured}. La tendance récente dans la communauté de machine learning des réseaux de neurones pour les graphes (GNN) \cite{Gnn_survey} est l'un des nombreux exemples soulignant l'importance des données structurées de nos jours. 

Alors que les exemples précédents considèrent les données structurées comme des \emph{entrées} du processus d'apprentissage, la prédominance des données structurées en machine learning se manifeste également dans de nombreux travaux où les données structurées sont des \emph{sorties}. On peut citer par exemple le domaine de la prédiction structurée dans lequel on veut apprendre à produire des objets structurés tels que des séquences, des arbres ou des assignements \cite{structure_svm,crf_2001,10.3115/1118693.1118694,JMLR:v21:19-021,pmlr-v80-mensch18a,korba_2018}. 

En bref, la notion de \emph{structure} en machine learning est omniprésente et apparaît aussi souvent qu'il y a une information supplémentaire sur les objets qui va au-delà de leurs représentations caractéristiques, de leurs features. Comme le montrent de nombreux contextes de machine learning tels que les modèles graphiques \cite{Pearl:1986:FPS:9075.9076,Pearl:2009:CMR:1642718}, l'apprentissage par renforcement \cite{Dzeroski2001} ou les modèles bayésiens non-paramétriques  \cite{hjort10}, considérer les objets comme une composition complexe d'entités avec certaines interactions est particulièrement utile, afin d'apprendre à partir de petites quantités de données. 

La notion précédente de données structurées peut se voir comme étant un cas particulier de données définies sur des espaces incomparables. Dans cette situation, chaque sample possède une ``caractéristique'' qui lui est propre et qui peut ne pas être partagée avec les autres samples. Par exemple, lorsque l'on considère un ensemble de données composé de plusieurs graphes, la structure d'un graphe n'est généralement pas partagée avec les autres graphes. Cette notion, volontairement large, englobe également le cas où les données proviennent de sources hétérogènes. Un exemple particulier de ce problème est l'adaptation de domaines hétérogènes \cite{yeh2014heterogeneous,pmlr-v33-zhou14,10.5555/2283516.2283652} qui vise à exploiter les connaissances provenant de domaines sources hétérogènes pour améliorer les performances d'apprentissage dans un domaine cible avec, potentiellement, différents espaces pour les features entre les domaines source et cible. Le cas des datasets MNIST/USPS \cite{lecun-mnisthandwrittendigit-2010,usps_dataset} illustre assez bien cette situation : sur la base de la connaissance d'images de chiffre de taille $28 \times 28$ (\ie\ des vecteurs de $\R^{784}$) de MNIST, comment construire \textit{e.g.} un classifieur qui fonctionne sur des images de chiffres $16 \times 16$  (\ie\ des vecteurs de $\R^{256}$) de USPS ? Il va sans dire que ce problème se pose souvent dans tous les domaines du machine learning : il est courant que les données soient recueillies à partir de sources diverses et hétérogènes et les méthodes qui s'appuient sur cette diversité présentent souvent un grand intérêt.

\begin{figure}[t]
   \begin{center}
        \includegraphics[width=0.9\linewidth]{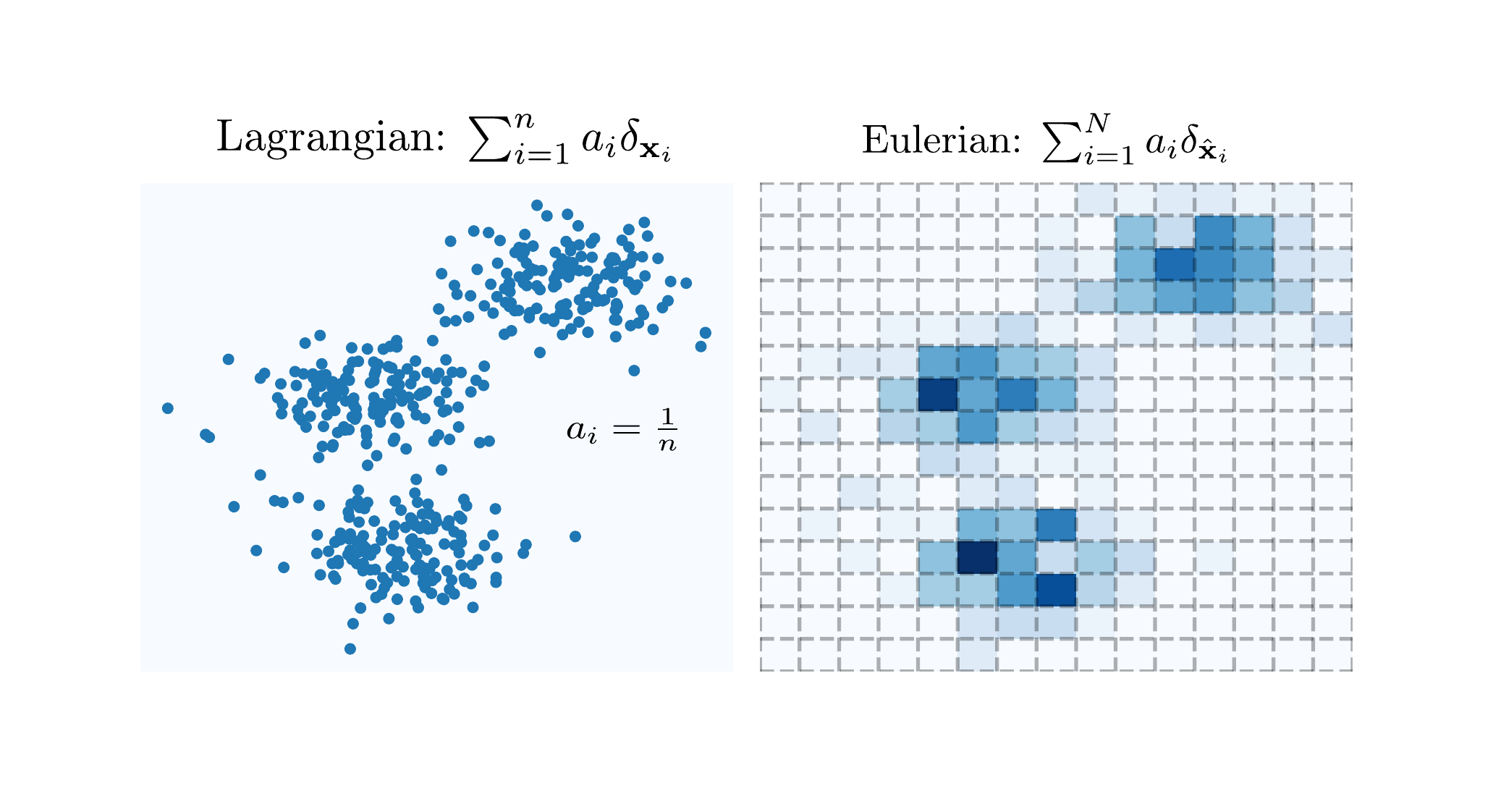}
            \caption{\label{fig:euler_res} On associe à un dataset $(\xbf_i)_{i \in \integ{n}}$ une distribution de probabilité le décrivant entièrement. Dans cet exemple $\xbf_i \in \R^{2}$. {\bf (à gauche)} Formulation Lagragienne (ou nuage de points): $\sum_{i=1}^{n} a_i \delta_{\xbf_i}$ est une mesure de probabilité discrète où $(a_i)_{i\in \integ{n}}$ est un vecteur de probabilité, \ie\ $a_i \geq 0$ et $\sum_{i=1}^{n}a_i=1$ et $\delta$ est la mesure de dirac \ie\ $\delta_{\xbf_i}(\xbf)=1$ si $\xbf=\xbf_i$ et $0$ sinon. {\bf (à droite)} Formulation Euleriennne (histogrammes): on associe une mesure de probabilité $\sum_{i=1}^{N} a_i \delta_{\hat{\xbf}_i}$ où $(\hat{\xbf}_i)_{i \in \integ{N}}$ est une grille régulière qui discrétise l'espace $\R^{2}$ et $(a_i)_{i \in \integ{N}}$ est un vecteur de probabilité. }
     \end{center}
\end{figure}

\paragraph{Le transport optimal pour le machine learning}  La question centrale qui se pose souvent en machine learning est la suivante : comment représenter les données et comment les comparer ? Le cadre des distributions de probabilités apporte certaines réponses à cette question en associant une mesure de probabilité à une collection de samples qui forment un dataset $(\xbf_i)_{i \in \integ{n}}$. La représentation \emph{Lagragienne} du dataset résulte en une mesure de probabilité discrète $\sum_{i=1}^{n} a_i \delta_{\xbf_i}$ dans laquelle on associe à chaque point $\xbf_i$ un dirac $\delta_{\xbf_i}(\xbf)=1$ si $\xbf=\xbf_i$ sinon $0$ ainsi qu'un poids $a_i \geq 0$ tel que $(a_i)_{i \in \integ{n}}$ est un vecteur de probabilité qui satisfait $\sum_{i=1}^{n} a_i=1$. Lorsqu'aucune information concernant l'importance relative des samples  dans le dataset n'est disponible les poids peuvent être choisis uniformes de sorte que $a_i=\frac{1}{n}$. De même une représentation \textit{Eulerienne} peut être construite \textit{via} la distribution de probabilité $\sum_{i=1}^{N} a_i \delta_{\hat{\xbf}_i}$ dans laquelle $(\hat{\xbf}_i)_{i \in \integ{N}}$ est une grille régulière qui discrétise l'espace. Cette formulation est équivalente à construire un histogramme sur nos données (voir Figure \ref{fig:euler_res}). Ces points de vue sur les données préconisent de trouver un moyen approprié de comparer leur représentation sous forme de distributions de probabilité et, à ce titre, la question de trouver des outils pour les comparer est au cœur de nombreux algorithmes de machine learning. 

Bien qu'il existe diverses divergences telles que $\phi$-divergences \cite{phidiv} ou Maximum Mean Discrepancies (MMD) \cite{greton_2007}, la richesse du transport optimal réside dans sa capacité à intégrer la géométrie de l'espace sous-jacent dans sa formulation et à prêter attention aux relations, aux correspondances des échantillons au sein de leurs représentations respectives. Pour souligner en bref l'avantage de représenter les données par des histogrammes/distributions de probabilités couplé à l'utilisation du transport optimal, nous pouvons citer par exemple \cite{rubner_earth_2000} pour la recherche d'images par contenu ou \cite{word_emb_doc_dist_2015} pour le traitement du langage naturel. À ce stade, des questions naturelles se posent : cette représentation théorique des données est-elle utilisable lorsque leur nature est intrinsèquement structurée, ou lorsque les samples se trouvent dans des espaces incomparables ? Dans ce cas, comment pouvons-nous représenter les données sous forme de distributions de probabilités ? Dans quelle mesure cette représentation est-elle valable ? Le transport optimal est-il toujours applicable et, si ce n'est pas le cas, comment comparer ces distributions de probabilités ? L'objectif de cette thèse, entre autres, est de fournir des réponses à ces questions.

\section*{Contributions}  

Cette thèse couvre la majeure partie des travaux de l'auteur et se concentre sur un seul axe de recherche qui est \emph{Le transport optimal sur des espaces incomparables}. Des travaux supplémentaires \cite{vayer2020time} sur les séries temporelles sur des espaces incomparables, qui ne sont pas basés sur le transport optimal, ne sont pas inclus dans cette thèse. Le lecteur intéressé peut cependant trouver les détails dans la bibliographie. 

\subsection*{Chapitre \ref{cha:ot_general}} 

Ce chapitre présente les résultats fondamentaux de la théorie classique du transport optimal et résume/illustre ses différentes formulations ainsi que quelques solveurs numériques connus. La philosophie de ce chapitre est de fournir un aperçu de haut niveau du transport optimal, tant en théorie qu'en pratique. Ce chapitre se conclut sur la théorie de Gromov-Wasserstein qui est au cœur de la thèse. Un lecteur familier avec les concepts de base du transport optimal peut passer cette partie bien qu'elle contienne des concepts et des notations essentiels qui seront abordés tout au long de la thèse.

\subsection*{Chapitre \ref{cha:fgw}} 

Ce chapitre est consacré au transport optimal pour les données structurées, et notamment dans le contexte des graphes. Il est basé sur les travaux des articles \cite{vay_struc} et \cite{Vayer_2020} et fournit des réponses à la question de la définition d'un cadre mathématique pour le transport optimal dans le cas de données structurées. Nous fournissons un cadre général, basé sur la notion de Fused Gromov-Wasserstein qui définit une distance de transport optimal entre des objets structurés tels que des graphes labelés non dirigés.

En résumé, nous considérons les graphes labelés non orientés comme des tuples de la forme $\mathcal{G}=(\mathcal{V},\mathcal{E},\ell_f,\ell_s)$ où $(\mathcal{V},\mathcal{E})$ est l'ensemble des noeuds et des arêtes du graphe.
$\ell_f : \mathcal{V} \rightarrow \Omega_f$ est une fonction qui associe à chaque noeud $v_{i} \in \mathcal{V}$ un feature $a_{i}\stackrel{\text{def}}{=}\ell_f(v_{i})$ dans un espace métrique $(\Omega_f,d)$. Nous appelons \emph{information de feature} l'ensemble de tous les features $(a_{i})_i$ du graphe. De même, $\ell_s : \mathcal{V} \rightarrow \Omega_s$ associe à un nœud $v_i$ du graphe un point $x_{i}\stackrel{\text{def}}{=}\ell_s(v_{i})$ appartenant à un espace métrique $(\Omega_s,C)$ spécifique à chaque graphe. $C : \Omega_s \times \Omega_s \rightarrow \mathbb{R}$ est une application symétrique qui vise à mesurer la similarité entre les nœuds du graphe. Cependant, contrairement à l'espace des features, $\Omega_s$ est implicite et, en pratique, il suffit de connaître la mesure de similarité $C$. Avec un léger abus de notation, $C$ est utilisé pour désigner à la fois la mesure de similarité de la structure et la matrice qui encode cette similarité entre les paires de nœuds du graphe $(\C(i,k) = C(x_i, x_k))_{i,k}$. Selon le contexte, $\C$ peut soit encoder les informations de voisinage des nœuds, les informations des arrêtes du graphe, ou plus généralement, il peut modéliser une distance entre les nœuds telle que la distance shortest-path. Nous désignons par \emph{information de structure} l'ensemble de tous les points de structure $(x_{i})_i$ du graphe.

\begin{figure}[t]
    \centering
        \includegraphics[width=0.9\columnwidth]{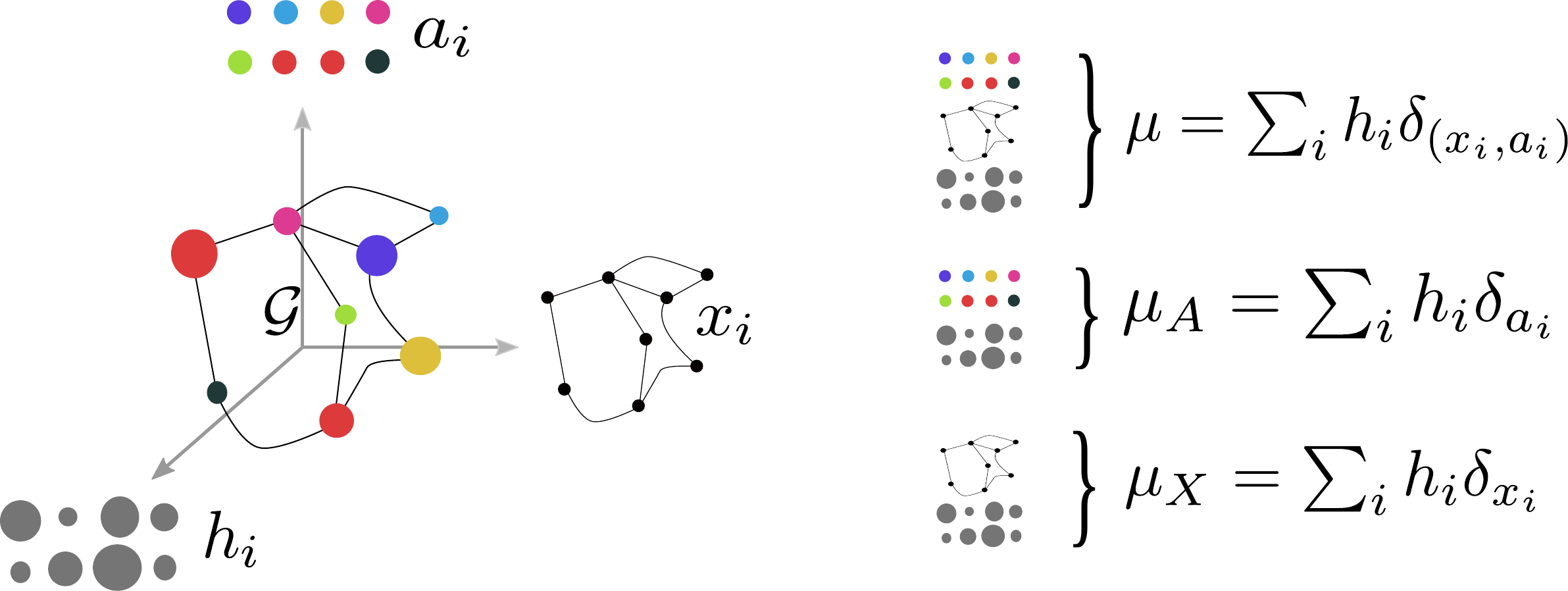}
    \caption{{\bf (à gauche)} Graphe labelé avec $(a_{i})_{i}$ l'information de feature, $(x_{i})_{i}$ l'information de structure et $(h_{i})_{i}$ un vecteur de poids qui mesure l'importance relative des noeuds. {\bf (à droite)} Les données structurées sont entièrement décrites par une mesure de probabilité $\mu$ sur l'espace produit des features et de la structure, avec respectivement des marginales $\mu_{X}$ et $\mu_{A}$ sur la structure et les features. \label{graphex_res}}
\end{figure}

Nous proposons d'enrichir le graphe décrit précédemment avec un vecteur de poids qui a pour but d'encoder l'importance relative des noeuds du graphe. Pour ce faire, si nous supposons que le graphe a $n$ noeuds, nous associons aux noeuds des poids $(h_{i})_{i} \in \Sigma_{n}$. 
Par cette procédure, nous obtenons la notion de \emph{données structurées} comme un tuple $\mathcal{S}=(\mathcal{G},h_{\mathcal{G}})$ où $\mathcal{G}$ est un graphe labelé et $h_{\mathcal{G}}$ est une fonction qui associe un poids à chaque noeud. Cette définition permet au graphe d'être représenté par une mesure de probabilité $\mu= \sum_{i=1}^{n} h_{i} \delta_{(x_{i},a_{i})}$ sur l'espace produit feature/structure qui décrit l'ensemble de la donnée structurée (voir Figure \ref{graphex_res}). 

Considérons à présent deux données structurées $\mu= \sum_{i=1}^{n} h_{i} \delta_{(x_{i},a_{i})}$ et $\nu= \sum_{i=1}^{m} g_{j} \delta_{(y_{j},b_{j})}$, où  $\h \in \Sigma_{n}$ et $\g \in \Sigma_{m}$ sont des histogrammes. Nous notons $\Mbf_{\Abf\Bbf}=(d(a_{i},b_{j}))_{i,j}$ la matrice $n \times m$  de distance entre les features et $\C_{1},\C_{2}$, les matrices de structure des graphes. 

Nous définissons une nouvelle distance de transport optimal appelée la distance de Fused Gromov-Wasserstein. Elle est donnée pour un paramètre $\alpha \in [0,1]$ par :
\begin{equation}
\label{discretefgw_res}
\fgwdistance(\C_1,\C_2,\h,\g)= \underset{\GG \in \couplingset(\h,\g)}{\min}\sum_{i,j,k,l} (1-\alpha) d(a_{i},b_{j})^{q}+\alpha |C_{1}(i,k)-C_{2}(j,l)|^{q} \pi_{i,j}\pi_{k,l}  \\
\end{equation}
Nous prouvons que cette fonction définit bien une distance entre les graphes labelés (Theorem \ref{metrictheodiscrete}) et qu'elle permet aussi de donner naissance à une notion de barycentre, de moyenne, de graphes basée sur la moyenne de Fréchet (Section \ref{sec:bary}). Nous établissons un algorithme pour calculer ces différents objets (Section \ref{sec:solved_fgw}) qui nous permet en particulier de résoudre aussi le problème de Gromov-Wasserstein classique, et nous prouvons qu'elle s'avère utile dans de nombreux scénarii de machine learning sur les graphes comme la classification, la simplification de graphes ou encore le clustering de graphes (Section \ref{sec:expeFGW}). Nous concluons cette partie en élargissant la définition précédente au cas de données structurées continues (Section \ref{sec:continuousfgw}) où nous montrons que FGW possède des propriétés de distance similaires et définit de plus une géodésique sur l'espace des données structurées.

\subsection*{Chapitre \ref{cha:gw_euclidean}} 
Ce Chapitre vise à combler l'écart entre la théorie classique du transport optimal et la théorie de Gromov-Wasserstein. Il s'ouvre sur le cas particulier des distributions 1D dont l'étude est basée sur les travaux de l'article \cite{vay_sliced_gromov_2019}. Nous établissons la première closed-form de Gromov-Wasserstein dans le cas des mesures discrètes sur la droite réelle et du coût euclidien au carré. Nous prouvons que celle-ci peut être calculée rapidement avec une complexité en $O(n\log(n))$. En particulier nous montrons qu'un couplage optimal pour Gromov-Wasserstein est trouvé en considérant soit le couplage diagonal ou le couplage anti-diagonal lorsque les points sont triés. Nous proposons, en utilisant cette closed-form, une nouvelle divergence appelée Sliced Gromov-Wasserstein, à la manière de Sliced Wasserstein. Nous établissons ses propriétés (Theorem \ref{propertiessgw}) et l'utilisons dans des scénarii tels que la comparaison de meshes 3D et les réseaux de neurones génératifs.

Une deuxième partie plus prospective se concentre sur la théorie de Gromov-Wasserstein pour les espaces euclidiens en la liant à la théorie classique du transport optimal. Nous abordons notamment la question de la régularité des plans de transport de Gromov-Wasserstein. Nous donnons des conditions nécessaires sous lesquelles nous pouvons prouver que le couplage optimal de Gromov-Wasserstein est supporté par une fonction déterministe, de la même manière que le couplage optimal pour le cas de Wasserstein avec un coût quadratique entre des mesures régulières (théorème de Brenier \cite{brenier_1991}). Pour cela nous considérons les cas où les mesures de distance ou similarité dans chaque espace $c_{\Xcal},c_{\Ycal}$ sont définies par les produits scalaires ou par des coûts quadratiques. Nous montrons que résoudre $\gw$ équivaut à résoudre conjointement un problème de transport linéaire et un problème d'alignement. Ainsi, la régularité des plans optimaux de $\gw$ peut être étudiée au travers de formulations équivalentes, plus simples à analyser. Cela nous permet également de construire des solutions algorithmiques pour $\gw$ sur des espaces euclidiens. En résumé :

\begin{enumerate}[label=(\roman*)]
\item Dans la Section \ref{sec:inner_product_case} nous considérons le cas où $c_{\Xcal},c_{\Ycal}$ sont définis par des produits scalaires dans chaque espace. A condition que la mesure de probabilité source soit régulière par rapport à la mesure de Lebesgue, nous donnons une condition suffisante pour l'existence d'un plan de transport optimal \emph{déterministe}, \ie\ supporté par une fonction déterministe $T$. Nous montrons que cette fonction est de la forme $\nabla u \circ \Pbf$ où $u$ est une fonction convexe et $\Pbf$ est une application linéaire qui peut être considérée comme une transformation globale ``recalant'' les mesures de probabilité dans le même espace (Theorem \ref{maintheorem}). Nous utilisons cette formulation pour montrer que la distance $\gw$ entre les mesures de probabilité 1D admet une solution closed-form. Plus précisément, nous montrons que le couplage optimal est déterminé par les fonctions de distribution cumulative et anticumulative de la distribution source (Theorem \ref{theo:1D1D1D}). 
\item Dans la Section \ref{sec:sq_eclidean} nous considérons $c_{\Xcal},c_{\Ycal}$ définis comme étant le carré des distances euclidiennes dans chaque espace. Nous montrons que le problème équivaut à une maximisation d'une fonction convexe sur $\couplingset(\mu,\nu)$. Nous utilisons la dualité Fenchel-Legendre dans l'espace des mesures pour en déduire un problème équivalent à celui de Gromov-Wasserstein (Theorem \ref{maintheo2}). Nous l'analysons plus en profondeur et montrons que la régularité des plans de transport optimaux est plus compliquée à établir que dans le cas précédent.
\item Dans la Section \ref{sec:numerical} nous utilisons les formulations précédentes pour obtenir des solutions numériques efficaces pour le problème $\gw$ basées sur des BCD. Nous montrons que ces algorithmes se comparent favorablement par rapport aux solveurs standards tels que le gradient conditionnel ou avec la régularisation entropique.
\item Nous concluons par la Section \ref{sec:gm} en considérant le problème \emph{Gromov-Monge} dans les espaces euclidiens, qui est l'équivalent du problème de Monge du transport linéaire dans le contexte de Gromov-Wasserstein. Nous discutons du cas particulier de Gromov-Monge entre les mesures gaussiennes et nous montrons que ce problème admet une closed-form quand on se limite aux push-forward linéaires (Theorem \ref{maintheo3}). Nous donnons des interprétations géométriques de ce résultat et nous comparons le push-forward optimal avec la celui de la théorie du transport optimal classique dans le cas des mesures gaussiennes.
\end{enumerate}

\subsection*{Chapitre \ref{cha:coot}} 

\begin{figure*}[t]
    \centering
    \includegraphics[width=1\linewidth]{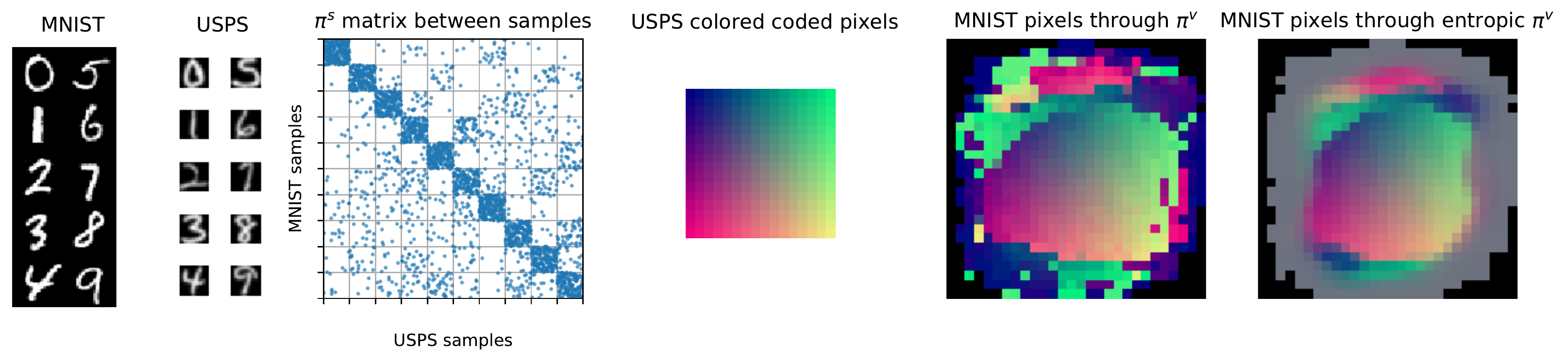}
    \caption{Illustration de \COOT \ entre les datasets MNIST et USPS. \textbf{(gauche)} samples des datasets MNIST et USPS ; \textbf{(centre gauche)} Matrice de transport $\Gs$ entre les samples, triés par classe ; \textbf{(centre)} Image USPS avec les pixels colorés en respectant leur position 2D ; \textbf{(centre droite)} couleurs transportées sur image MNIST en utilisant $\Gv$, les pixels noirs correspondent aux pixels MNIST non informatifs (toujours à 0); \textbf{(droite)} couleurs transportées sur image MNIST en utilisant $\Gv$ avec régularisation entropique.}
    \label{fig:mnist_usps_res}
\end{figure*}

Ce chapitre présente un nouveau cadre théorique pour comparer des mesures de probabilité sur des espaces incomparables, à savoir le problème de co-transport optimal. Contrairement à l'approche de Gromov-Wasserstein, cette approche permet d'optimiser simultanément deux couplages entre les samples et les features des données. Ce chapitre fournit une analyse théorique approfondie de ce cadre et, du point de vue des applications, il aborde le problème de l'adaptation des domaines hétérogènes et du co-clustering. Ce chapitre est basé sur l'article \cite{redko2020cooptimal}. 

Plus précisément nous considérons deux datasets quelconques $\X=[\x_1,\dots,\x_n]^T\in\mathbb{R}^{n\times d}$ et $\X'=[\xbf'_1,\dots,\xbf'_{n'}]^T\in\mathbb{R}^{n'\times d'}$, avec en général $n\neq n'$ et $d\neq d'$. Les lignes sont appelées \emph{samples} et les colonnes \emph{features}. Nous associons aux samples $(\xbf_i)_{i \in \integ{n}}$ et $(\xbf'_i)_{i \in \integ{n'}}$ des poids $\w=[w_1,\dots,w_n]^\top \in\Sigma_n$ et $\w'=[w_1',\dots,w_{n'}']^\top \in\Sigma_{n'}$. De la même manière on associe aux features des poids $\v\in\Sigma_d$ et $\v'\in\Sigma_{d'}$.
Le problème de Co-transport optimal est défini par:
\begin{equation}
 \label{eq:co-optimal-transport_res}
COOT(\X,\X',\w,\w',\v,\v')=\min_{\begin{smallmatrix}\GGs \in\Pi(\w,\w') \\ \GGv\in\Pi(\v,\v')\end{smallmatrix}} \quad \sum_{i,j,k,l} L(X_{i,k},X'_{j,l})\GGs_{i,j}\GGv_{k,l} 
\end{equation}

Pour illustrer cette définition, nous résolvons le problème d'optimisation entre deux datasets classiques~: MNIST et USPS. Ils contiennent des images de différentes résolution (les images de USPS sont de taille $16 \times16$ et MNIST $28 \times28$) qui appartiennent aux mêmes classes (chiffres entre 0 et 9). En outre, les chiffres sont également différemment centrés comme l'illustrent les exemples de la partie gauche de la Figure \ref{fig:mnist_usps_res}. Cela signifie que sans prétraitement, les images ne se trouvent pas dans le même espace topologique et donc ne peuvent pas être comparées directement à l'aide des distances conventionnelles. Les images représentent les \emph{samples} tandis que chaque pixel agit comme un \emph{feature} conduisant à $256$ et $784$ features pour USPS et MNIST

Le résultat de la résolution du problème est reporté sur la Figure \ref{fig:mnist_usps_res}. Dans la partie centrale gauche, nous fournissons le couplage optimal $\GGs$ entre les samples, {\textit{i.e.} les différentes images}, triés par classe. Le couplage $\GGv$, à son tour, décrit les relations entre les features, \textit{i.e.} les pixels, dans les deux domaines.  Pour le visualiser, nous codons en couleur les pixels de l'image USPS source et utilisons $\GGv$ pour transporter les couleurs sur une image MNIST cible de sorte que ses pixels soient définis comme des combinaisons convexes de couleurs de la première avec des coefficients donnés par $\GGv$.  Les résultats correspondants sont présentés dans la partie droite de la Figure~\ref{fig:mnist_usps_res}.

Nous montrons que cette nouvelle formulation inclut Gromov-Wasserstein comme cas particulier (Proposition \ref{prop:cot_equal_gw}), et qu'elle présente l'avantage de travailler directement sur les données brutes sans avoir à calculer, stocker et choisir les mesures de similarité. De plus, \COOT \ fournit deux couplages interprétables entre les features et les samples. Nous montrons que \COOT\ définit une notion de distance entre les datasets (Proposition \ref{sec:metric_properties}) et nous établissons une procédure d'optimisation basée sur la résolution de transports linéaires (Section \ref{sec:optimsec}). Sur le plan pratique, nous apportons la preuve de l'utilité de \COOT\ pour le machine learning notamment en adaptation de domaines hétérogènes (Section \ref{sec:hda}) et pour le co-clustering (Section \ref{sec:co_clustering}).

\dominitoc

\pagenumbering{roman}

\tableofcontents \addcontentsline{toc}{chapter}{Table of contents} \mtcaddchapter

\newpage
\begin{tabular}{p{3cm}p{11cm}p{1cm}}
\textbf{Linear algebra} & \\
$\mathbf{a}, \mathbf{A}$ & all vectors in $\R^{d}$ and matrices are written in bold. The coordinates will be written $A_{i,j}$ or $A(i,j)$ for matrices and $a_i$ for vectors without bold.\\
$\integ{n}$ & the subset $\{1,\cdots,n\}$ of $\mathbb{N}$ \\
$\mathbf{I_n},\mathbf{J_n}$ & the $n\times n$ identity matrix and anti-identity matrix \\
$\|.\|, \langle,\rangle$ & a norm (depends on the context) and an inner product \\
$\ell_{p}$ & Denotes the standard $\|.\|_{p}$ norm \\
$\kron$ & is the Kronecker product of matrices, \ie\ for two matrices $\Abf \in \R^{n\times m},\Bbf^{p\times q}$, $\Abf \kron \Bbf \in \R^{np \times mq}$ defined by $(A_{i,j}\Bbf)_{i,j}$ \\
$\otimes$ & the the tensor-matrix multiplication, \emph{i.e.} for a tensor $\L=(L_{i,j,k,l})$ and a matrix $\Abf$, $\L \otimes \Abf$ is the matrix $\left(\sum_{k,l} L_{i,j,k,l}A_{k,l}\right)_{i,j}$. \\
$\tr,\text{det}$ & is the trace operator for matrices \ie\ $\tr(\Pbf)=\sum_{i,j}P_{i,j}$ and the determinant operator\\
$\|.\|_{\F}$ & is the Frobenius norm for matrices \ie\ $\|\Pbf\|_{\F}=\sqrt{\tr(\Pbf^{T} \Pbf)}$ \\
$\froeb{.}{.}$ & is the inner product for matrices \ie\ $\froeb{\Pbf}{\Qbf}=\tr(\Qbf^{T} \Pbf)$ \\
$\rg(\mathbf{A}),\ker(\mathbf{A})$ & The rank and the kernel of a matrix $\mathbf{A}$ \\
$DS$ & is the set of doubly-stochastic matrices, \ie\ $DS=\{\X \in \R^{n \times n}| \X \one_n=\one_n, \X^{T} \one_n=\one_n, \X \geq 0 \}$ \\
$\Pi_n$ & the set of permutation matrices, \ie\ $\X=(x_{ij})_{(i,j) \in \integ{n}^{2}} \in \Pi_n$ if $x_{ij} \in \{0,1\}$ and $\sum_{i=1}^{n}x_{ij}=\sum_{j=1}^{n}x_{ij}=1$. \\
$\mathcal{O}(p)$ & the subset of $\R^{p\times p}$ of all orthogonal matrices. \\
$\Stief$ & is the Stiefel manifold, \ie\ the set of all orthonormal $p$-frames in $\R^{q}$ or equivalently $\Stief=\{\Bbf \in \R^{q\times p} | \Bbf^{T}\Bbf=\mathbf{I_{p}} \}$. \\
$\Sn$ & the set of all permutations of $\integ{n}$. \\
$\oslash $ & Element-wise division operator for two vectors, \ie\ $\mathbf{u} \oslash \mathbf{v}=(\frac{u_{i}}{v_{i}})_i$ \\
$\simplex_n $ & the set of probability vectors in $\R_{+}^{n}$ (or histograms with $n$ bins), \ie\ $\simplex_n=\{\a \in \R_{+}^{n}\ | \sum_{i=1}^{n} a_i=1\}$ \\
 &  \\
\textbf{Measure theory} & \\
$\P(\Xcal)$ & the set of probability measures on a space $\Xcal$ \\
$\mathcal{M}(\Xcal)$ & the set of Borel finite signed measures on a space $\Xcal$ \\
$\lebsm$ & the Lebesgue measure on $\R$ or $\R^{n}$ depending on the context \\
$\supp(\mu)$ & the support of $\mu$ (see definition \ref{def:support})\\
$\delta_x$ & the dirac measure on $x$, \ie\  $\delta_x(y)=0$ if $y\neq x$ else $1$\\
$\#$ & the push forward operator. \\
$\mu \otimes \nu$ & the product measure of two probability measures $\mu,\nu$, \ie, $\mu \otimes \nu(A\times B)=\mu(A)\nu(B) $ \\
$(id \times T) \# \mu$ & Is the measure $\mu \otimes (T \# \mu) $ \\
$\couplingset(\mu,\nu)$ & the set of couplings of two probability measures $\mu,\nu$ (see definition \ref{def:coupling}) \\
$\mongeot_{c}(\mu,\nu)$ &The Monge cost between two probability measures $\mu,\nu$ (see definition \ref{MongeOt})\\
$\ot_{c}(\mu,\nu)$ & The Kantorovitch cost between two probability measures $\mu,\nu$ (see definition \ref{linear_ot})\\
$\wass_{p}(\mu,\nu)$ & The $p$-Wasserstein distance between two probability measures $\mu,\nu$ \\
$\gw_{p}(\mu,\nu)$ & The $p$-Gromov-Wasserstein distance between two probability measures $\mu,\nu$ \\
$\mathcal{F}_{\mu}$& the Fourier transform of the probability measure $\mu$ \ie\ $\mathcal{F}_{\mu}(s)=\int e^{-2i \pi \langle s,x \rangle } \dr\mu(x)$ \\
$\mathcal{N}(\mathbf{m}, \Sigmab)$ & multivariate Gaussian (normal) distribution with mean $\mathbf{m}$ and covariance $\Sigmab$ \\
\end{tabular}

\begin{tabular}{p{3cm}p{11cm}p{1cm}}
& \\
\textbf{Functions} & \\
$DT$ & the Jacobian of a function $T$ \\
$C(\Xcal),C(\Xcal,\Ycal)$ & the set of continuous functions form $\Xcal$ to $\R$ (\textit{resp.} from $\Xcal$ to $\Ycal$)  \\
$C_{b}(\Xcal),C_{b}(\Xcal,\Ycal)$ & the set of continuous and bounded  functions form $\Xcal$ to $\R$ (\textit{resp.} from $\Xcal$ to $\Ycal$) \\
$C^{p}(\Xcal)$ & the set of function of class $C^{p}$. \\ 
$Lip_{k}(\Xcal)$ & the set of Lipschitz functions on $\Xcal$ with Lipschitz constant $k$ \\
$L^{p}(\mu)$ & The set of $p$-integrable functions with respect to a measure $\mu$, \ie\ $f\in L^{p}(\mu)$ if $\int |f|^{p}\dr \mu<+\infty$\\
$\nabla$ & the gradient operator \\
$d_{\Xcal} \oplus d_{\Ycal}$ & The ``sum'' distance on the cartesian product $(\Xcal,d_{\Xcal}) \times (\Ycal,d_{\Ycal})$. More precisley it is defined by $d_{\Xcal} \oplus d_{\Ycal}((x,x'),(y,y'))=d_{\Xcal}(x,x')+d_{\Ycal}(y,y')$.  \\
$\mathbb{S}^{d-1}$ & The $d$-dimensional hypersphere, \ie\ $\mathbb{S}^{d-1}=\left\{\thetab \in \mathbb{R}^{d} \ | \ \|\thetab\|_2=1\right\}$ \\
& \\
\textbf{Acronyms} & \\
OT & Is the acronym for Optimal Transport \\
W,SW,GW, SGW& Stands respectively for Wasserstein, Sliced Wasserstein, Gromov-Wasserstein and Sliced Gromov-Wasserstein\\
QAP,QP,BAP,BP & Stands respectively for Quadratic Assignment Problem, Quadratic Program, Bilinear Assignment Problem, Bilinear Program \\
\end{tabular}

\mainmatter

\chapter{Introduction}
\minitoc
\label{cha:introduction}

\section{Moving probability measures: a least effort problem}

``Thus, we see in science, sometimes brilliant but for a long time useless theories, suddenly becoming the basis of the most important applications, and sometimes seemingly very simple applications, giving rise to the idea of abstract theories that are not yet needed, directing towards the theories of the Surveyors' work, and opening up a new career for them'' \footnote{Original quote: ``Ainsi, l'on voit dans les Sciences, tantôt des théories brillantes, mais longtemps inutiles, devenir tout à coup le fondement des applications les plus importantes, et tantôt des applications très simples en apparence, faire naître l'idée de théories abstraites dont on n'avait pas encore le besoin, diriger vers les théories des travaux des Géomètres, et leur ouvrir une carrière nouvelle.''}. That is how Nicolas De Condorcet \cite{condorcet} in the 18th century introduced the work of Gaspard Monge \cite{monge_81} which is at the core of the Optimal Transport (OT) theory. How to move some masses from one location to another so as to minimize the overall effort? Condorcet was right: this ``simple idea'' has evolved over the years to become a theory at the crossroads of mathematics/optimization and is today at the center of many machine learning applications.

Broadly speaking the interest of Optimal Transport lies in both its ability to provide correspondences between sets of points and its ability to induce a geometric notion of distance between probability distributions (see Figure \ref{fig:ot_illus}) Both have proved to be very useful for a wide range of tasks that are, to name a few, image registration \cite{Haker:2001}, image retrieval \cite{rubner98}, domain adaptation \cite{courty2017optimal}, signal processing \cite{kolouri_2017}, unsupervised learning \cite{arjovsky17a,pmlr-v84-genevay18a}, supervised and semi-supervised learning \cite{Frogner_2015,solomon14}, natural language processing \cite{word_emb_doc_dist_2015}, fairness \cite{pmlr-v97-gordaliza19a}, in biology \cite{biology_ot} or in astrophysics \cite{Frisch}. 

\begin{figure}[t]
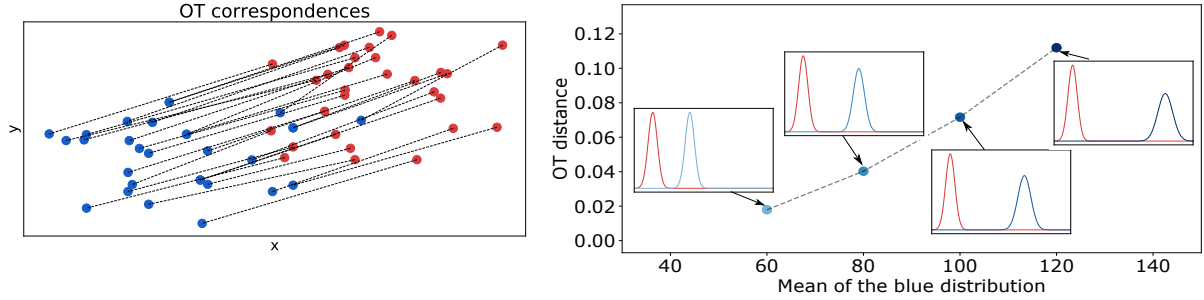

  \begin{center}
  \includegraphics[align=t,width=0.44\linewidth]{corres.pdf}
  \includegraphics[align=t,width=0.55\linewidth]{illus_ot.pdf}
  \end{center}
  \caption{Optimal Transport provides tools for finding correspondences between set of point and gives a geometric notion of distance between probability distributions. {\bf (left)} Two sets of points in 2D. The correspondences found by OT is depicted in dashed line. {\bf (right)} OT distance between two probability distributions in red and blue. The further away the mean of the red probability distribution is from the blue one the higher is the OT distance. }
  \label{fig:ot_illus}
\end{figure}

Despite its many properties, the optimal transport problem remains difficult to solve in practice and is known to suffer from scalability problems that prevent its use on large, ubiquitous data in machine learning. The emergence of optimal transport in the machine learning community has been greatly favoured by recent achievements on the optimization side \cite{cuturi2013sinkhorn,altschuler2017near,genevay_stochastic} which tend to circumvent the heavy computational complexity of solving OT problems. 

Moreover optimal transport, in its early formulation, is quite restricted to applications where there exists a \textit{direct way} of comparing the samples of the data. Its applicability is thus often limited to the case where the samples are part of a common ground metric space, that is most of the time Euclidean. This limitation prevents its use for a variety of machine learning tasks where there is an additional \emph{structural information} on the data which can not usually be described in the Euclidean setting \textit{e.g.} when the samples are described by graphs, trees or time series. It also prevents the use of optimal transport when the samples lie in different, seemingly not related, metric spaces or when a meaningful notion of distance between the samples can not be easily defined. All of these instances can be framed into the \emph{incomparable} setting, that is when the samples lie on \emph{incomparable spaces}. An interesting remedy in this situation can be found in the theory of the Gromov-Wasserstein distance \cite{memoli_gw} which does not require the comparison of the samples \emph{across} the distributions. However, it is well-known to be arduous to solve and to suffer from tedious scalability issues. The purpose of this thesis is to help overcoming these obstacles by:

\begin{enumerate}[label=(\roman*)]
\item Defining new Optimal Transport frameworks for incomparable spaces and especially for structured data.
\item Reducing the theoretical gap between classical Optimal Transport and the Gromov-Wasserstein theory in Euclidean spaces and in particular derive scalable and applicable formulations.
\end{enumerate}

\section{Structured Data and Incomparable Spaces in Machine Learning} 

Before going into the details of this thesis, it is important to make explicit what is behind the notion of \emph{structure} through the manuscript. As a starter we can see the \emph{structural information} as the piece of information which encodes the specific relationships that exist among the components of an object. This definition can be related with the concept of \emph{relational reasoning} \cite{relationnalreasoning} where some \emph{entities} (or elements with attributes such as the intensity of a signal) coexist with some \emph{relations} or properties between them. 

Natural instances of such structured data arise when the structure is \emph{explicit}. For example, in a graph context, edges are representative of this notion so that each attribute of the graph (typically  $\mathbb{R}^{d}$ vectors) may be linked to some others through the edges between the nodes. Notable examples are found in chemical compounds or molecules modeling~\cite{DBLP:journals/corr/KriegeGW16}, brain connectivity~\cite{ktena2017distance}, or social networks~\cite{Yanardag15}. This generic family of structured data also encompasses trees~\cite{day1985optimal} or time series where the signals' values are correlated through time so that comparing times series requires one to take the direction of time into account.

Structure information of data in machine learning can also be more subtle or even \emph{implicit}. For example it can appear where one build structural priors, or \emph{inductive bias} \cite{relationnalreasoning}, on the objects representation. For instance in the context of deep learning, many successful architectures exploit the equivariance to a symmetry transformation to improve generalization. The convolutional neural network (CNN) is a prime example of this ``built-in'' inductive bias which satisfies translation equivariance, \ie\ if we translate the input, the output
of the convolutions will also be translated. This inductive bias is known to reveal useful spatial hierarchy, or structure, on the pixels \cite{Wang2018NonlocalNN,chen_iterative_2018} and other works have studied designing layers with equivariances to other transformations such as permutation, rotation, reflection \cite{kondor2018generalization,pmlr-v48-cohenc16,NIPS2014_5424}. Unlike end-to-end methods other more ``hand-engineering'' approaches based \textit{e.g.} on images segmentation can leverage some structure in images that can further be usefully exploited \cite{bachgraphkernel,Jianbo_cuts}. Implicit structure is also at the core of many natural language processing (NLP) tools used to find good word representations \cite{mikolov_2013,word2vec,2014-glove}. It is also the main ingredient of speech recognition \cite{hinton_2012} or sequence learning \cite{seq_to_seq}. In these cases structural assumptions on the sequences of words are made whether by the mean of latent variables or \textit{via} conditional probabilities. When available, labels or classes also induce an implicit structure on the feature space of the data. For instance in Domain Adaptation one may desire the source samples with the same label to be matched consistently within the same region of the target space preventing them from being split into disjointed far locations \cite{courty2017optimal,AlvarezMelis2018Structured}. The recent trend of graph neural networks (GNN) \cite{Gnn_survey} in the machine learning community is one of the many examples emphasizing that structured data remains an important and challenging setting nowadays. While previous instances consider structured data as inputs of the learning process the prevalence of this notion in machine learning also arises in another line of works where it is an output. This setting is \textit{e.g.} considered in the structured prediction approach where one wants to learn to produce structured object such as sequences, trees or assignments \cite{structure_svm,crf_2001,10.3115/1118693.1118694,JMLR:v21:19-021,pmlr-v80-mensch18a,korba_2018}. 

In short, the notion of \emph{structure} in machine learning is omnipresent and appears as often as there is an additional information about the objects that goes beyond their feature representations. As shown in many contexts in machine learning such as graphical models \cite{Pearl:1986:FPS:9075.9076,Pearl:2009:CMR:1642718}, relational reinforcement learning \cite{Dzeroski2001} or Bayesian nonparametrics \cite{hjort10}, considering objects as a complex composition of entities together with some interactions is crucial in order to learn from small amounts of data. 

The previous notion of structure data can be seen as a special case of data defined on incomparable spaces. Informally in this situation, each sample has its own ``characteristic'' that may not be shared with the other samples. For instance when considering a dataset of multiple graphs (each graph being a data point) the structure of one graph is usually not shared among the other graphs. This notion, purposely broad, encompasses also the case where data come from heterogeneous sources. As a particular instance of this problem Heterogeneous Domain Adaptation \cite{yeh2014heterogeneous,pmlr-v33-zhou14,10.5555/2283516.2283652} aims to exploit knowledge from heterogeneous source domains to improve the learning performance in a target domain with, potentially, different feature spaces between the source and target domains. The MNIST/USPS \cite{lecun-mnisthandwrittendigit-2010,usps_dataset} case is a prime example of this situation: based on the knowledge of $28\times 28$ digit images (\ie\ $\R^{784}$ vectors) from MNIST how to build \textit{e.g.} a classifier that works well on $16 \times 16$ digit images (\ie\ $\R^{256}$ vectors) of USPS? Needless to say this problem often arises in all the fields of machine learning: it is common that the data are gathered from heterogeneous sources in practice and methods that build upon this diversity are often of high interest.   
\begin{figure}[t]
   \begin{center}
        \includegraphics[width=0.9\linewidth]{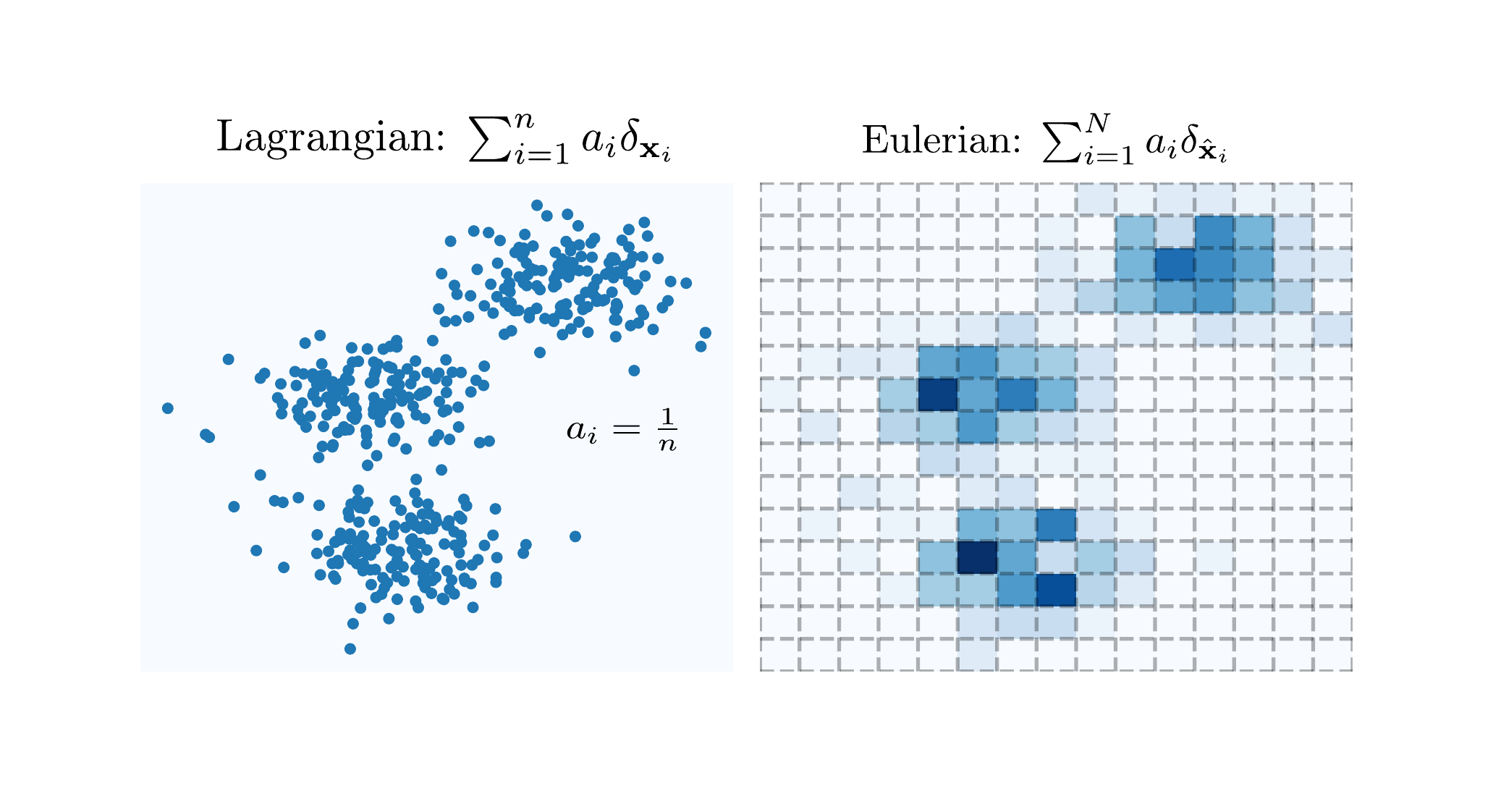}
            \caption{\label{fig:euler} We associate to a dataset $(\xbf_i)_{i \in \integ{n}}$ a probability measure describing the dataset. In this example $\xbf_i \in \R^{2}$. {\bf (left)} Lagragian formulation (points clouds): $\sum_{i=1}^{n} a_i \delta_{\xbf_i}$ is a discrete probability measure where $(a_i)_{i\in \integ{n}}$ is a probability vector, \ie\ $a_i \geq 0$ and $\sum_{i=1}^{n}a_i=1$ and $\delta$ is the dirac measure $\delta_{\xbf_i}(\xbf)=1$ if $\xbf=\xbf_i$ else $0$. {\bf (right)} Eulerian formulation (histograms): $ \sum_{i=1}^{N} a_i \delta_{\hat{\xbf}_i}$ where $(\hat{\xbf}_i)_{i \in \integ{N}}$ is a regular grid on $\R^{2}$ and $(a_i)_{i \in \integ{N}}$ is a probability vector. }
     \end{center}
\end{figure}
 
\section{Motivating Optimal Transport} 

The central question that often arises in machine learning is: how to represent data and how to compare them? The framework of probability distributions provides an answer to this query by associating a probability measure to a collection of samples that forms a dataset $(\xbf_i)_{i \in \integ{n}}$. The \emph{Lagragian} representation of the dataset results in a discrete probability measure $\sum_{i=1}^{n} a_i \delta_{\xbf_i}$ in which one associates to each point $\xbf_i$ a dirac $\delta_{\xbf_i}(\xbf)=1$ if $\xbf=\xbf_i$ otherwise $0$ as well as a weight $a_i \geq 0$ such that $(a_i)_{i \in \integ{n}}$ is a probability vector which satisfies $\sum_{i=1}^{n} a_i=1$. When no information about the relative importance of the samples in the dataset is available the weights can be chosen as uniform so that $a_i=\frac{1}{n}$. Similarly a \textit{Eulerian} representation can be constructed \textit{via} the probability distribution $\sum_{i=1}^{N} a_i \delta_{\hat{\xbf}_i}$ in which $(\hat{\xbf}_i)_{i \in \integ{N}}$ is a regular grid on the space. This formulation produces an histogram of our data (see Figure \ref{fig:euler_res}). This point of view on the data advocates finding an appropriate way of comparing their representation as probability distributions and, as such, the question of finding adequate measures of ``how far'' are two probability distributions is at the core of many machine learning algorithms. Although various divergences exist such as $\phi$-divergences \cite{phidiv} or Maximum Mean Discrepancies (MMD) \cite{greton_2007}, the richness of optimal transport lie in its ability to incorporate the geometry of the underlying space in its formulation and to pay attention to the relations, the correspondences of the samples within their respective representations. To highlight in short the benefit of representing data through probability distributions coupled with OT we can cite \cite{rubner_earth_2000} for image retrieval or \cite{word_emb_doc_dist_2015} for natural language processing. At this point natural questions arise: is this framework applicable when the nature of the data is inherently structured or when the different data points lie in incomparable spaces? In this case how can we represent data as probability distributions? To what extent this representation is valuable? Is the Optimal transport framework still applicable, and, if not, how do we compare these probability distributions? The purpose of this thesis, \textit{inter alia}, is to give some answers to these questions.

\section{Outline of the Thesis}  

This thesis covers mostly all the author's work conducted and focus on a single line of research that is \emph{Optimal Transport on Incomparable Spaces}. Additional line of works \cite{vayer2020time} on time series on incomparable spaces, which is not based on optimal transport, is not included in this thesis but the interested reader can find the details in the bibliography. The rest of the thesis is hinged so that all chapters can be read separately in any order except for Chapter \ref{cha:ot_general} that provides all the mathematical background and tools used in the other chapters.  

\paragraph{Chapter \ref{cha:ot_general}} sets up the mathematical and numerical background of optimal transport. It presents the fundamental results of classical optimal transport theory and summarizes/illustrates its different formulations as well as some well-known solvers. The philosophy of this chapter is to provide a high-level overview of OT both in theory and practice. This chapter concludes with the Gromov-Wasserstein theory which is at the core of thesis. A reader familiar with the basic concepts of optimal transport may skip this part although it contains crucial concepts and notations that will be discussed throughout the thesis.

\paragraph{Chapter \ref{cha:fgw}} is dedicated to optimal transport for structured data, and especially in the context of graphs. It is based on the works of the articles \cite{vay_struc} and \cite{Vayer_2020} and gives some answers to the question of defining a mathematical framework for optimal transport in the case of structured data. A general framework for this setting is given, based on the Fused Gromov-Wasserstein distance that defines a OT distance between structured data such as undirected graphs, and applied on real world graph data applications. 

\paragraph{Chapter \ref{cha:gw_euclidean}} aims at bridging the gap between the classical optimal transport theory and the Gromov-Wasserstein theory. The chapter opens with the special case of 1D distributions which results in the Sliced Gromov-Wasserstein formulation based on the works in \cite{vay_sliced_gromov_2019}. A second more prospective part focuses on the Gromov-Wasserstein theory for Euclidean spaces and connects with the classical optimal transport theory by questioning the regularity of Gromov-Wasserstein optimal transport plans. 

\paragraph{Chapter \ref{cha:coot}} presents a new framework for comparing probability measures on incomparable spaces, namely the CO-Optimal transportation problem. Contrary to the Gromov-Wasserstein approach this approach simultaneously optimizes two transport maps between both samples and features of the data. This chapter provides a thorough theoretical analysis of this framework and, from an application side, this work tackles the problem of Heterogeneous Domain Adaptation and co-clustering/data summarization. This chapter is based on the article \cite{redko2020cooptimal}.

\chapter{Generality about optimal transport}
\minitoc
\label{cha:ot_general}

Optimal transport is a long-standing mathematical problem whose theory has matured over the years. A good gateway for this theory can be found in \cite{San15a}. A more mathematical oriented overview can be found in \cite{Villani} while the most complete document about numerical aspects of OT can be found in \cite{cot_peyre_cutu}. The objective of this chapter to present in short the main results of the ``classical'' OT theory both mathematically and numerically. We will discuss in the last part of this chapter the theory related to the Gromov-Wasserstein transportation problem for which we refer the reader to \cite{memoli_gw,Sturm2012} for its foundation.

\section{Linear Optimal Transport theory}

\paragraph{The Monge problem} The OT problem has been historically introduced by Gaspard Monge \cite{monge_81} and can be described as the following ``least effort problem'': given two probability distributions $\mu$ and $\nu$ how do we \emph{transfer} all the probability mass of $\mu$ onto $\nu$ so that the overall \emph{effort} of transferring this mass is minimized? Originally the idea was to move dirt (déblais) to one place to another (remblais) in the most efficient way.

For properly defining this problem we need to define the notions of \emph{transfer} and \emph{effort}: the former can be expressed through the notion of \emph{push-forward} and the latter through the notion of \emph{cost}. More precisely and given two Polish spaces $\Xcal,\Ycal$\footnote{A Polish space $\Xcal$ is a separable completely metrizable topological space.} and two probability measures $\mu \in \P(\Xcal),\nu \in \P(\Ycal)$ a cost is a function $c:\Xcal \times \Ycal \rightarrow \R_{+} \cup \{+\infty\}$ which values $c(x,y)$ aim at measuring how far is $x \in \Xcal$ from $y \in \Ycal$ and quantifies somehow the ``effort'' of moving $x$ forward to $y$. The \emph{push-forward}\footnote{$T$ is often called a map or Monge map in the OT literature} of a probability measure $\mu$ through a function $T: \Xcal \rightarrow \Ycal$ is defined as the probability measure $T\#\mu \in \P(\Ycal)$ which satisfies equivalently one of the two following conditions:
\begin{enumerate}[label=(\roman*)]
\label{def:pushfor}
\item  $T\#\mu(A)=\mu(T^{-1}(A))=\mu(\{x \in \Xcal \ | \ T(x)\in A \})$ for every measurable set $A$  
\item  $\int_{\Ycal} \phi(y) \dr(T\#\mu)(y)=\int_{\Xcal} \phi \circ T(x)\dr \mu(x)$ for every measurable function $\phi$.
\end{enumerate}

These conditions simply state that we transform, or \emph{push}, the probability measure $\mu$ thanks to $T$ so as to create another probability measure on $\Ycal$. When we consider a discrete probability distribution $\mu=\sum_{i=1}^{n} a_i \delta_{x_i}$ the push-forward measure $T\#\mu$ is simply defined by $T\#\mu=\sum_{i=1}^{n} a_i \delta_{T(x_i)}$.

As described at the beginning, one wants to move the \emph{source} distribution $\mu$ forward to the \emph{target} distribution $\nu$. This translates mathematically as finding a map $T$ which satisfies $T\#\mu=\nu$. When we consider the Euclidean setting and when the probability measures have densities $f,g$ with respect to the Lebesgue measure by the change of variable formula the push-forward condition writes:
\begin{equation}
\label{monge_ampere0}
g(T(\xbf))\det(DT(\xbf))=f(\xbf)
\end{equation}
where $DT$ stands for the Jacobian of $T$. Among all these possible push-forwards, OT aims at finding the map $T$ which minimizes the total cost of having moved $\mu$ forward to $\nu$ that is $\int_{\Xcal} c(x,T(x)) \dr \mu(x)$. Overall the problem of Monge \eqref{MongeOt} can be formulated as the following non-convex optimization problem:

\begin{equation}
\label{MongeOt}
\tag{MP}
\mongeot_{c}(\mu,\nu)=\inf_{T\#\mu=\nu} \int_{\Xcal} c(x,T(x)) \dr \mu(x).
\end{equation}

In general finding such optimal map $T$ of the \eqref{MongeOt} problem is quite difficult to solve since the solution may not be unique and may not even exists (see Figure \ref{fig:monge}). Even in the regular setting where $\mu,\nu$ have densities, equation \eqref{monge_ampere0} is highly non-linear in $T$ which is one of the major difficulty preventing from an easy analysis of the Monge Problem. As such the Monge problem remained an open question for many years and results about the existence and unicity of the optimal Monge map were limited to special cases until the works of Brenier \cite{brenier_1991} which implications will be detailed after. 

\begin{figure}[t]
   \begin{center}
        \includegraphics[width=0.55\linewidth]{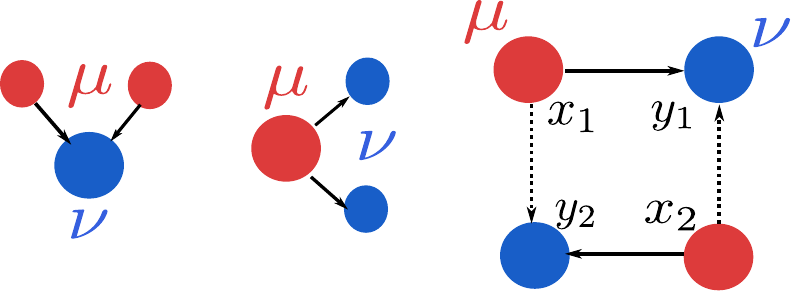}
    \caption{Push-forward between two discrete probability measures $\mu$ and $\nu$ in three \textit{scenarii}.  \textbf{(left)} $\mu$ is supported on $x_1$,$x_2$ with corresponding weights $\frac{1}{2},\frac{1}{2}$ and $\nu$ on $y_1$ with weight $1$. The only possible push-forward $T$ is $T(x_1)=y_1$ and $T(x_2)=y_1$. \textbf{(center)} In this situation there is no push-forward of $\mu$ onto $\nu$ because no function can satisfies $T(x_1)=y_1$ and $T(x_1)=y_2$ when $y_1 \neq y_2$. The problem of Monge admits no solution in this case.  \textbf{(right)} All points are equidistant form each other that is $c(x_i,y_j)=1$ for $i,j\in \integ{2}^{2}$. In this case the solution of problem \ref{MongeOt} is not unique it may associate $x_1$ whether with $y_1$ or $y_2$ with the same overall cost (same with $x_2$). \label{fig:monge} }
     \end{center}
\end{figure}

\paragraph{Kantorovitch formulation} Major breakthroughs in the OT theory were made possible thanks to Kantorovitch \cite{Kantorovich42} who proposes a relaxation of the \eqref{MongeOt} problem. The key idea is to consider a \emph{probabilistic mapping} instead of a deterministic map $T$ to push the source measure forward to the target one. In the Kantorovitch formulation it is allowed to \emph{split} the mass of the probability measures into pieces and transport them towards several targets points. This translates mathematically by replacing the push-forward of a measure by a probabilistic \emph{coupling}:

\begin{definition}[Couplings]
\label{def:coupling}
Let $\mu \in \P(\Xcal),\nu \in \P(\Ycal)$. A coupling $\pi$ of $\mu$ and $\nu$ is a probability distribution on $\Xcal \times \Ycal$ such that both marginals of $\pi$ are respectively $\mu$ and $\nu$. More precisely $\pi$ is part of the following set:
\begin{equation}
\couplingset(\mu,\nu)=\{\pi \in \P(\Xcal \times \Ycal) \ | A\subset \Xcal \ , B\subset \Ycal \ \text{measurable} \ \pi(A \times \Ycal)=\mu(A)\ ; \pi(\Xcal \times B)=\nu(B)\}
\end{equation}
\end{definition}

\begin{Remark}
Unlike the set of push-forwards of $\mu$ forward to $\nu$ the set of couplings of two probability measures is always non-empty as the product measure $\mu \otimes \nu$ is in $\couplingset(\mu,\nu)$.
\end{Remark}

An important illustration of the former definition is when $\mu$ and $\nu$ are discrete probability measures. This situation will be omnipresent thorough the manuscript we propose to detail the notations in the following example: 

\begin{figure*}[!b]
\begin{memo}[Lower semi-continuity]
\label{memo:lsc}
On a metric space $(\Xcal,d)$ a function $f: \Xcal \rightarrow \R\cup\{+\infty\}$ is said to be lower semi-continuous (l.s.c.) if for every sequence $x_n\rightarrow x$ we have $f(x)\leq \lim \inf f(x_n)$. Such functions have the following properties:
\begin{itemize}
\item If $(f_{k})_{k}$ is a sequence of l.s.c. functions on $\Xcal$ then $f=\sup_{k} f_{k}$ is l.s.c.
\item If $f$ is l.s.c. and bounded from below then there exists a sequence of continuous and bounded functions $(f_{k})_{k\in \mathbb{N}}$ converging increasingly to f. We can also suppose that each $f_{k}$ is $k$-Lipsichtz.
\end{itemize}
\end{memo}
\end{figure*}

\begin{Example}[The case of discrete probability measures]
 Let $\mu=\sum_{i=1}^{n} a_{i} \delta_{x_{i}}, \nu=\sum_{i=1}^{m} b_{j} \delta_{y_{j}}$ be discrete probability measures where $x_i \in \Xcal$, $y_j \in \Ycal$ and $\a=(a_i)_{i \in \integ{n}}\in \simplex_n, \b=(b_j)_{j \in \integ{m}} \in \simplex_m$ are probability vectors which belong to the following probability simplex:
\begin{equation}
\simplex_n\stackrel{def}{=}\{\a \in \R_{+}^{n}\ | \sum_{i=1}^{n} a_i=1\}.
\end{equation}
We will use interchangeably the term histogram or probability vector for an element $\a \in \simplex_n$. In this case a coupling $\GG$ is a matrix of the following set:
\begin{equation}
\begin{split}
\couplingset(\a,\b)&=\{\GG \in \R_{+}^{n\times m} \ | \GG \one_m=\a \ ; \GG^{T} \one_n=\b \}\\
&=\{\GG \in \R_{+}^{n\times m} \ | \forall (i,j) \in \integ{n}\times \integ{m},\ \sum_{j=1}^{m} \pi_{ij}=a_i \ ; \sum_{i=1}^{n}\pi_{ij}=b_j \}
\end{split}
\end{equation}
\end{Example}

Using the coupling instead of a deterministic map allows defining the OT problem for a very large class of probability measures under very mild assumptions. More precisely let $\Xcal,\Ycal$ be Polish spaces and $\mu \in \P(\Xcal),\nu \in \P(\Ycal)$. Given a cost $c:\Xcal \times \Ycal \rightarrow \R_{+} \cup \{+\infty\}$ lower semi-continuous (see Memo \ref{memo:lsc}), then Kantorovitch problem aims at finding:

\begin{equation}
\label{linear_ot}
\tag{KP}
\ot_{c}(\mu,\nu)=\inf_{\pi \in \couplingset(\mu,\nu)} \int_{\Xcal\times \Ycal} c(x,y) \dr\pi(x,y).
\end{equation}

The resulting cost, potentially infinite without further assumptions, corresponds to the minimal cost of moving $\mu$ forward to $\nu$ by splitting their masses and transporting forward the pieces according to the transport plan $\pi$. The good news about this formulation is that the infimum is always well defined providing that the cost is positive and lower semi-continuous (actually it suffices that $c$ is bounded from below see \cite[Theorem 1.7]{San15a}). The problem \ref{linear_ot} is defined regardless of the nature of the probability distributions: they can be both discrete or continuous (see Figure \ref{coulings}). The problem appears to be \emph{linear} in $\pi$ so that we will denote \eqref{linear_ot} as the \emph{linear transportation problem}.

Relying on the Kantorovitch formulation \eqref{linear_ot} appears to be very useful in order to find a solution of the Monge problem \eqref{MongeOt}. Indeed a push forward $T\#\mu=\nu$ induces a coupling $\pi=(id\times T)\#\mu$\footnote{$(id\times T)\#\mu$ is the measure $\mu\otimes T\#\mu$ or equivalently $\dr (id\times T)\#\mu(x,y)=\dr \mu(x) \dr (T\#\mu(y))$} then it is easy to verify that $\ot_{c}(\mu,\nu)\leq\mongeot_{c}(\mu,\nu)$. To find the converse inequality it suffices to find an optimal solution $\pi^{*}$ of \eqref{linear_ot} which is of the form $\pi^{*}=(id\times T^{*})\#\mu$ where $T^{*}\#\mu=\nu$. In this case we would have proven that both problems are equal and that $T^{*}$ is optimal for \eqref{MongeOt}. In other words if there is an optimal coupling supported on a \emph{deterministic} function then both \eqref{linear_ot} and \eqref{MongeOt} are equivalent. We will see in next sections cases where we can ensure that necessarily the optimal coupling is of this form.

\begin{figure}[t]
   \begin{center}
        \includegraphics[width=0.7\linewidth]{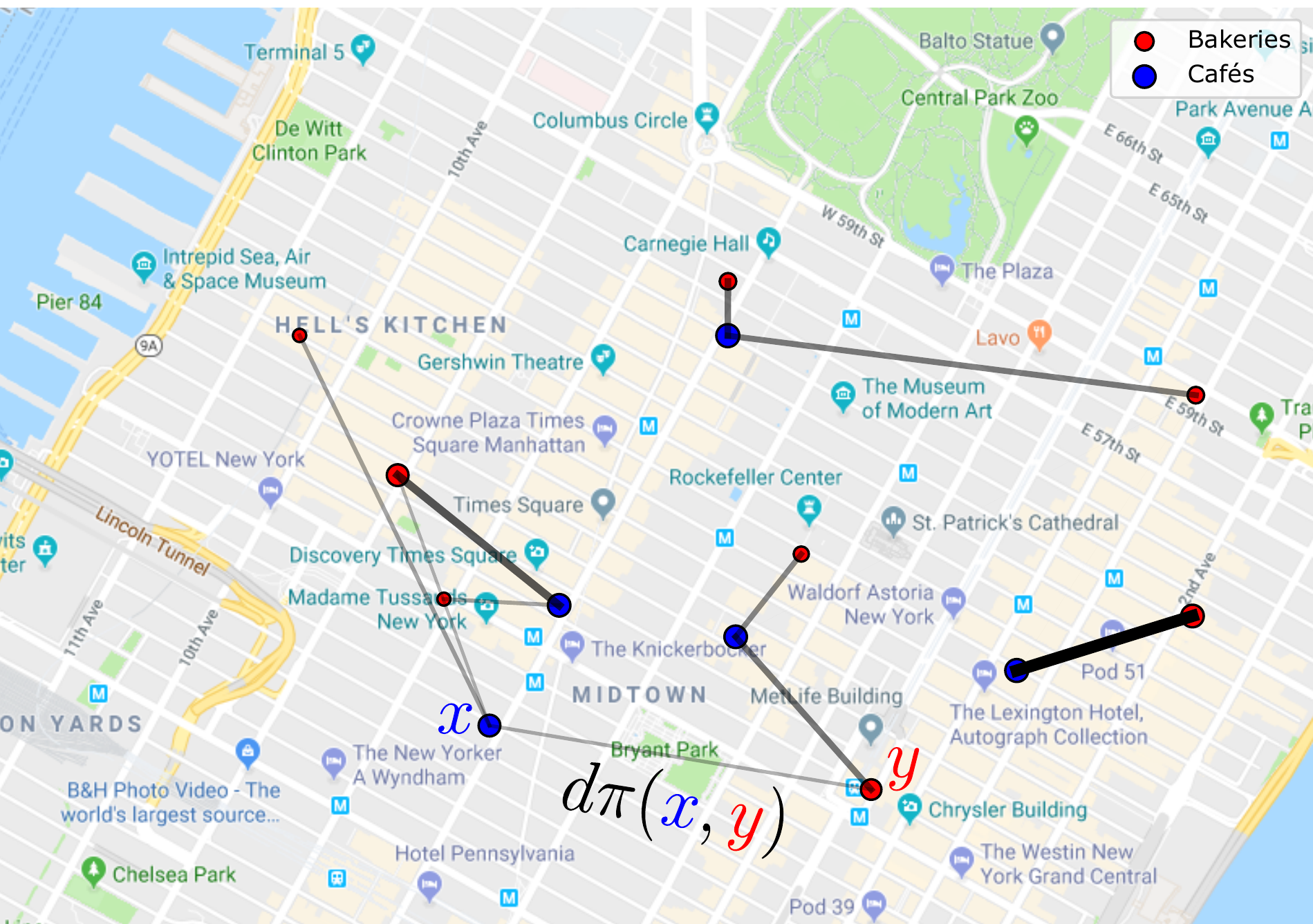}
    \caption{The bakery/cafés problem in Manhattan based on example \ref{ex:manhattan}. The bakeries (\textit{resp} cafés) are in blue (\textit{resp} red) and the dots' size denote the amount of bread disposable (\textit{resp} needed). Lines represent the amount of bread transferred $\dr\pi(x,y)$: the wider the larger. The optimal transport plan is depicted in the Figure and represents the cheapest transport plan so that all breads are transferred from bakeries to cafés. \label{fig:bakery}}
     \end{center}
\end{figure}

\begin{Example} 
\label{ex:manhattan}
The famous bakery analogy of Villani's book \cite{topics_ot} provides a simple illustration of the linear OT problem. Suppose that someone is in charge of the distribution of bread from bakeries to cafés in Manhattan. The bakeries are located at some points $y$ and the cafés at some points $x$ distant from each other by $c(x,y)$. At 8 a.m sharp all the bread from the bakeries has to be transferred to the cafés in order for the citizens of Manhattan to have a good day. The company in charge of the distribution wants to route the breads from bakeries to cafés the cheapest way possible. This problem can be recast into a linear OT problem. Considering two distributions $\mu$="all available breads in bakeries" and $\nu$="all the demands in bread of cafés", the company seeks for a transport plan $\pi$ such that $\dr\pi(x,y)$ is the amount of breads transferred from bakery $y$ to café $x$. The best transport plan minimizes the overall cost of moving all the breads from bakeries to cafés which is $\int c(x,y) \dr\pi(x,y)$ (see Figure \ref{fig:bakery})
\end{Example}

\begin{figure}[t]
   \begin{center}
        \includegraphics[width=0.7\linewidth]{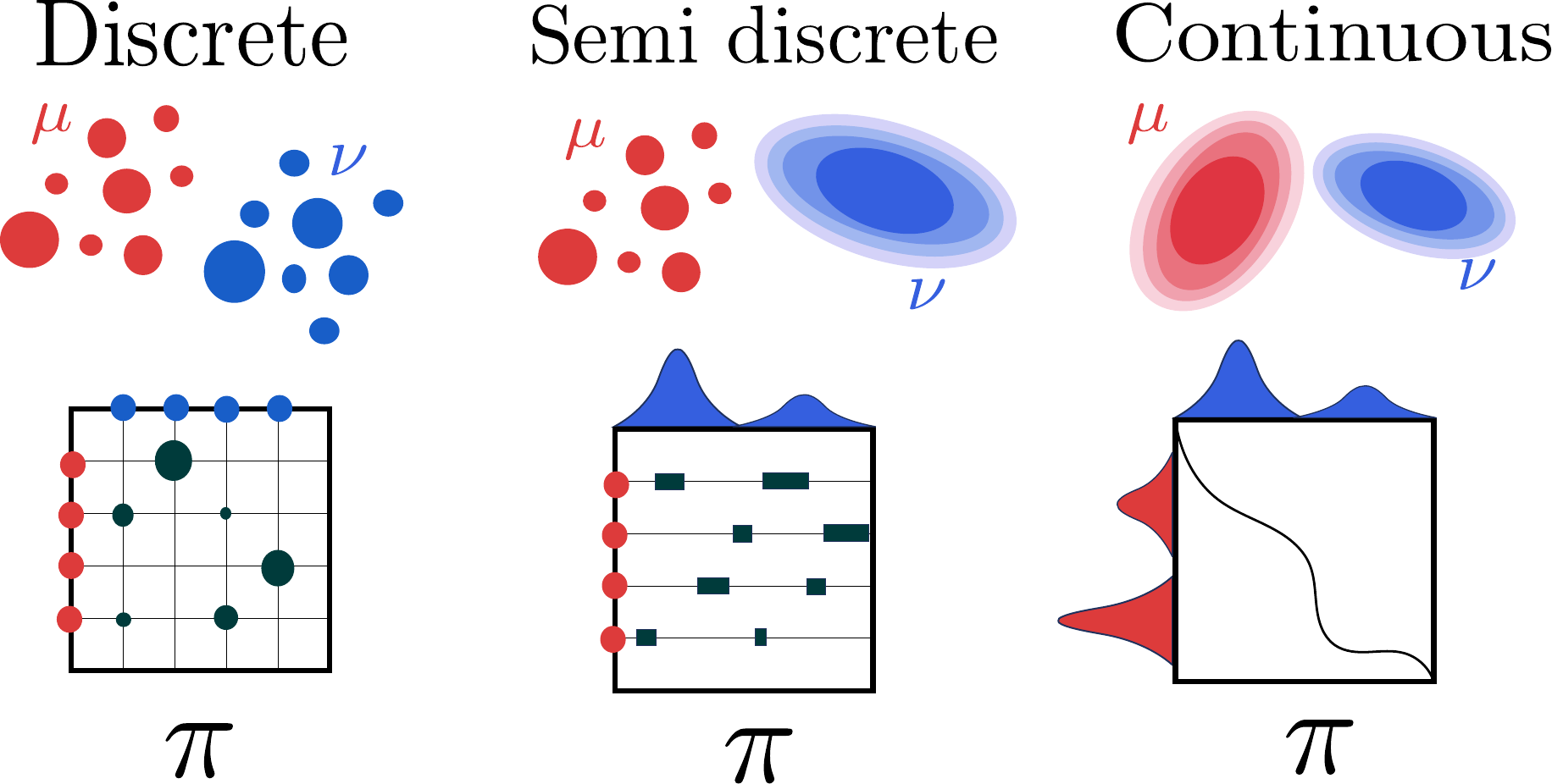}
    \caption{Different OT problems between {\bf (left)} two discrete probability distributions, the coupling $\pi$ is a matrix {\bf (center)} one continuous and one discrete probability distribution, this is called the \emph{semi-discrete} OT problem {\bf (right)} two continuous distributions. \label{coulings} \footnotesize{Inspired from \cite{cot_peyre_cutu}}}
     \end{center}
\end{figure}

\paragraph{Wasserstein distance} The most notable scenario in many OT applications is when $\Xcal=\Ycal=\Omega$ where $(\Omega,d)$ is a Polish space, \textit{e.g.} an Euclidean space. In this case there is a natural way of defining the cost $c$ since the space $\Omega$ is already endowed with a notion of distance between the points. In this situation we can define the so-called \emph{p-Wasserstein distance} for $p\in [1,+\infty[$ as $\wass_{p}(\mu,\nu)=(\ot_{d^{p}}(\mu,\nu))^\frac{1}{p}$ or precisely:
\begin{equation}
\wass_{p}(\mu,\nu)=\left(\inf_{\pi \in \couplingset(\mu,\nu)} \int_{\Omega\times \Omega} d^{p}(x,y) \dr\pi(x,y)\right)^{\frac{1}{p}}
\end{equation}

The name is not misleading: this function satisfies all the axioms of distance on the space of probability distributions with bounded $p$-moments as stated in the next theorem (see \cite[Definition 6.4]{Villani}):

\begin{theo}[The Wasserstein distance is a distance]
Let $(\Omega,d)$ be a Polish space, $p\in [1,+\infty[$ and: 
\begin{equation}
\P_{p}(\Omega)\stackrel{def}{=}\{\mu \in \P(\Omega) \ | \int_{\Omega} d(x_0,x)^{p} \dr \mu(x)<+\infty\}
\end{equation}
with $x_0\in \Omega$ arbitrary.
Let $\mu,\nu \in \P_{p}(\Omega)$. Then $\wass_{p}(\mu,\nu)<+\infty$. Moreover:
\begin{enumerate}[label=(\roman*)]
\item $\wass_{p}(\mu,\nu)=\wass_{p}(\nu,\mu)$ (symmetry)
\item $\wass_{p}(\mu,\nu)=0 \iff \mu=\nu$ (identity of indiscernibles)
\item Let $\zeta \in \P_{p}(\Omega)^{2}$ then $\wass_{p}(\mu,\nu) \leq \wass_{p}(\mu,\zeta)+ \wass_{p}(\zeta,\nu)$ (triangle inequality)
\end{enumerate}
\end{theo}

\begin{Example} One direct example of Wasserstein distance is between two diracs supported on $x$,$y$. In this case $\wass_p(\delta_x,\delta_y)=d(x,y)$ which is quite the intuitive behavior, that is the more the dirac are located far from each others the larger their Wasserstein distance is. 
\end{Example}

The distance property renders $\wass_p$ a powerful tool for comparing probability measures. Another valuable feature of the former distance is that it gives a characterization of the weak convergence of probability measure. Informally a sequence of probability measures gets as close as possible to a probability measure $\mu$ if the Wasserstein distance tends to zero. The convergence is based on the following definition:

\begin{definition}[Weak-convergence]
Let $(\mu_n)_{n\in \mathbb{N}}$ be a sequence of probability measures on $\Xcal$ a Polish space. We say that $(\mu_n)_{n\in \mathbb{N}}$ converges weakly to $\mu$ in $\Xcal$ if for all continuous and bounded functions $f: \Xcal \rightarrow \R$:
\begin{equation}
\int_{\Xcal} f \dr \mu_n \rightarrow \int_{\Xcal} f \dr \mu
\end{equation}
\end{definition}  

The Wasserstein distance \emph{metrizes the weak convergence} of probability measures, in other words $(\mu_n)_{n\in \mathbb{N}}$ converges weakly to $\mu$ if and only if $\wass_{p}(\mu_n,\mu)\rightarrow 0$ \cite[Theorem 6.9]{topics_ot}. Note that other distances can be proposed to metrize the space of probability measures \textit{e.g.} the L\'evy–Prokhorov distance, but the richness of $\wass$ lies in its ability to incorporate a lot of the geometry of the underlying space through the distance $d$. Consequently Wasserstein spaces $(\P_p(\Omega),\wass_p)$ are very large and many metric spaces can be embed into Wasserstein spaces with low distortion \cite{Bourgain1986TheMI,Andoni,frogner2019learning}.

\subsection{The main theorem of the linear Optimal Transport theory}

A fundamental result of the linear OT theory is the cyclical monotonicity property of its optimal transport plans. Basically it illustrates that an optimal transport plan can not be improved \emph{locally} and more importantly that is also sufficient for being a global optimal transport plan. Consequently it characterizes the set of optimal couplings using the notion of $c$-concave functions which appears to be very useful for defining the notions of duality and for solving the Monge problem \eqref{MongeOt} by relying on the Kantorovitch relaxation \eqref{linear_ot}. As such the cyclical monotonicity is maybe the main ingredient of linear OT. This section aims at presenting in short both this result and its consequences. 

\begin{definition}[Cyclically Monotone Set]
Let $c: \Xcal \times \Ycal \rightarrow ]-\infty, +\infty]$ be a real valued function on arbitrary sets $\Xcal,\Ycal$. A set $\Gamma \subset \Xcal \times \Ycal$ is said to be cyclically monotone if for $(x_{i},y_{i})_{i=1}^{N} \in \Gamma^{N}$ and $\sigma$ a permutation of $[1,..,N]$
\begin{equation}
\sum_{i=1}^{N} c(x_{i},y_{i}) \leq \sum_{i=1}^{N} c(x_{i},y_{\sigma(i)})
\end{equation}
We will call $c$-CM such sets.
\end{definition}

$c$-CM sets is an important notion in OT theory since it characterizes optimal transport plans for well behaved costs. We consider the following definition: 
\begin{definition}[Support]
\label{def:support}
Let $(\Xcal,d)$ be a Polish space and $\mu \in \P(\Xcal)$. The support of $\mu$ is defined as the smallest closed set $F$ such that $\mu(F)=1$ or equivalently:
\begin{equation}
\supp(\mu)=\{x \in \Xcal \ | \forall \epsilon>0, \mu(B(x,\epsilon))>0\}
\end{equation}
\end{definition}  
Informally the support of a distribution is where the distribution ``lives'', \ie\ where it is not zero. In the discrete case $\GG$ is a matrix and the support is found in the indices $(i,j)$ such that $\pi_{ij}>0$. The following theorem states that the support of an optimal coupling is actually a $c$-CM set:

\begin{theo}[Theorem 1.38 in \cite{San15a}]
\label{theo:piisccm}
Let $c: \Xcal \times \Ycal \rightarrow [0,+\infty[$ continuous and $\mu \in \P(\Xcal),\nu \in \P(\Ycal)$ with $\ot_{c}(\mu,\nu)<+\infty$. If a coupling $\pi \in \couplingset(\mu,\nu)$ is optimal for \eqref{linear_ot} then $\supp(\pi)$ is a $c$-CM set.
\end{theo}

Interestingly enough the $c$-CM sets are characterized by specific functions based on the notion of $c$-transforms. This property gives another way of computing optimal transport plans based on deterministic functions.

\begin{definition}[$c$-transforms]
Let $\Xcal,\Ycal$ be Polish spaces and $\psi: \Xcal \rightarrow \overline{\R}$ be a function. We define its $c$-\emph{transform} as the function $\psi^{c}: \Ycal \rightarrow \overline{\R}$:
\begin{equation}
\psi^{c}(y)=\inf_{x \in \Xcal} c(x,y)-\psi(x)
\end{equation}
and the $\bar{c}$-\emph{transform} of a function $\phi: \Ycal \rightarrow \overline{\R}$ as the function $\phi^{\bar{c}}: \Xcal \rightarrow \overline{\R}$:
\begin{equation}
\phi^{\bar{c}}(x)=\inf_{y \in \Ycal} c(x,y)-\phi(y)
\end{equation}
Fonctions that can be written as $\psi^{c}$ or $\phi^{\bar{c}}$ are called respectively $c$-\emph{concave} or $\bar{c}$-\emph{concave} functions.
\end{definition}

\begin{Remark}
The $c$-transform is a a generalization of the Legendre transform that is well-known in convex analysis \cite{rockafellar-1970a}. More precisely for function $u : \R^{d} \rightarrow \R$ its Legendre transform is defined as $u^{*}(\ybf)=\underset{\xbf \in \R^{d}}{\sup} \langle \ybf, \xbf \rangle - u(\xbf)$ (see Memo \ref{memo:memo_convex}). The $c$-transform corresponds to this notion by considering $c(\xbf,\ybf)=\langle \xbf,\ybf \rangle$ (up to the change of sign). Another special case deserves attention that is when $c(\xbf,\ybf)=\frac{1}{2}\|\xbf-\ybf\|^{2}$. Consider $\psi: \R^{d} \rightarrow \R \cup \{+\infty\}$, then the function $\psi$ is $c$-concave if and only if the function $u: \xbf\rightarrow \frac{1}{2} \|\xbf\|^{2}-\psi(\xbf)$ is convex and lower semi-continuous and the Legendre transform of $u$ is the function $\xbf \rightarrow \frac{1}{2} \|\xbf\|^{2}-\psi^{c}(\xbf)$  (see \cite[Proposition 1.21]{San15a}). 
\end{Remark}

\begin{figure*}[!b]
\begin{memo}[Convex analysis]
\label{memo:memo_convex}
For any function $u : \R^{d} \rightarrow \R$ its convex conjugate or Legendre transform is defined by $u^{*}(\ybf)=\underset{\xbf \in \R^{d}}{\sup} \langle \ybf, \xbf \rangle - u(\xbf)$. The subdiffenrential of $u$, denoted as $\partial u$, is defined for $\xbf \in \R^{d}$ as $\partial u(\xbf)=\{\ybf\in \R^{d}  \ | \ \forall \zbf \in \R^{d} \ u(\xbf)-u(\zbf) \geq \langle \ybf,\zbf-\xbf\rangle\}$ which reduces to $\partial u(\xbf)=\{\nabla u(\xbf)\}$ when $u$ is differentiable at $\xbf$. When $u$ is convex differentiable the convex conjugate has the following important properties (see \cite[Theorem 23.5]{rockafellar-1970a}): 
\begin{enumerate}[label=\roman*]
\item $u(\xbf)+u^{*}(\ybf)\geq \langle \xbf,\ybf \rangle $ (Fenchel-Young inequality)
\item $u(\xbf)+u^{*}(\ybf) = \langle \xbf,\ybf \rangle \iff \ybf \in \partial u(\xbf)$
\item In particular $u(\xbf)+u^{*}(\nabla u (\xbf))=\langle \xbf,\nabla u (\xbf)\rangle$
\end{enumerate}
\end{memo}
\end{figure*}

When $\Xcal=\Ycal$ and $c$ is symetric both notions are equivalent and in this case we will drop this distinction. One important property about $c$-transform is that it satisfies:
\begin{equation}
\label{eq:c_transf}
\forall x,y \in \Xcal \times \Ycal, \psi(x)+\psi^{c}(y) \leq c(x,y)
\end{equation}
The case of equality of \eqref{eq:c_transf} is attained on special subsets of $\Xcal \times \Ycal$ that are precisely the $c$-CM sets as stated in the next theorem: 

\begin{theo}[Theorem 1.37 in \cite{San15a}]
\label{theo:ccm}
If $\Gamma \neq \emptyset$ is a $c$-CM set in $\Xcal \times \Ycal$ and $c:\Xcal \times \Ycal \rightarrow \R$, then there exists a $c$-concave function $\psi: \Xcal \rightarrow \R \cup\{-\infty\}$ such that:
\begin{equation}
\Gamma \subset\left\{(x, y) \in X \times Y: \psi(x)+\psi^{c}(y)=c(x, y)\right\}
\end{equation}
\end{theo}

To summarize, the support of optimal couplings are necessarily $c$-CM sets and these $c$-CM sets are also characterized by functions using the notion of $c$-transform. The fundamental theorem of optimal transport states that all these results are in fact equivalent: 

\begin{theo}[Fundamental theorem of linear OT]
\label{fundamental_theorem}
Let $\Xcal,\Ycal$ be Polish spaces, $\mu \in \P(\Xcal),\nu \in \P(\Ycal)$ and $c:\Xcal \times \Ycal \rightarrow [0,+\infty[$ lower-semi continuous such that $\ot_{c}(\mu,\nu)<\infty$. Let $\pi\in \couplingset(\mu,\nu)$ then the following conditions are equivalent:
\begin{enumerate}[label=(\roman*)]
\item $\pi$ is optimal for \eqref{linear_ot}
\item The support of $\pi$ is $c$-cyclically monotone
\item There exists a measurable $c$-concave function $\psi$ such that $\psi(x)+\psi^{c}(y)=c(x,y)$ $\pi$ \textit{a.e}.
\end{enumerate}

\end{theo}

\begin{proof}
We will give a sketch of proof, for completeness the reader can refer to Theorem 5.10 in \cite{Villani}. Theorems \ref{theo:piisccm} and \ref{theo:ccm} already proved that $(i) \implies (ii) \implies (iii)$ in the case where $c$ is continuous. To pass from continuity to lower semi-continuity we can consider a sequence $(c_k)_k$ of costs that converges increasingly to $c$ and observe that \cite[Lemma 1.41]{San15a}:
\begin{equation*}
\lim_{k \rightarrow + \infty} \inf_{\pi \in \couplingset(\mu,\nu)} \int c_k(x,y) \dr \pi(x,y)=\inf_{\pi \in \couplingset(\mu,\nu)} \int c(x,y) \dr \pi(x,y)
\end{equation*}
and, using some subtleties, this can prove $(i)\implies (ii) \implies (iii)$ when $c$ is lower semi-continuous. For the converse we can easily prove that $(iii) \implies (i)$. By hypothesis $\int c(x,y)\dr \pi(x,y)= \int \psi(x)\dr \mu(x)+ \int \psi^{c}(y)\dr \nu(y)$. However for any other coupling $\pi'$ we have $\int c(x,y)\dr \pi'(x,y)\geq \int \psi(x)\dr \mu(x)+ \int \psi^{c}(y)\dr \nu(y)$ by relation \eqref{eq:c_transf} and so $\int \int c(x,y)\dr \pi'(x,y) \geq \int c(x,y)\dr \pi(x,y)$ so that $\pi$ is optimal. Technical details are hidden here for proving the measurability and integrability of $\psi,\psi^{c}$. 
\end{proof}

\paragraph{The main theorem of OT for duality} A first implication of the fundamental theorem is related to a duality principle which is a widely used property in linear programming (see Section \ref{sec:numerical_tour} for more details). This property can be extended in full generality in the context of OT as stated in the next theorem:

\begin{theo}[Duality theorem]
Let $\Xcal,\Ycal$ be Polish spaces, $\mu \in \P(\Xcal), \nu \in \P(\Ycal)$ and $c:\Xcal \times \Ycal \rightarrow [0,+\infty]$ be lower semi-continuous (\textit{l.s.c.}) such that $\ot_{c}(\mu,\nu)<+\infty$ then \emph{strong duality holds}. More precisely the \emph{dual problem}:
\begin{equation}
\tag{DKP}
\label{dual_problem_general_ot}
\sup_{\begin{smallmatrix}\phi,\psi \in C_b(\Xcal)\times C_b(\Ycal)\ \\ \forall (x,y)\in \Xcal \times \Ycal, \ \phi(x)+\psi(y)\leq c(x,y)\end{smallmatrix}} \int_{\Xcal} \phi(x)\dr\mu(x)+\int_{\Ycal} \psi(y)\dr\nu(y)
\end{equation}
leads to the same optimum as the \eqref{linear_ot} problem. Equivalently:
\begin{equation}
\label{strong_duality}
\ot_{c}(\mu,\nu)=\sup_{\phi,\psi \in \Phi_{c}} \int_{\Xcal} \phi(x)\dr\mu(x)+\int_{\Ycal} \psi(y)\dr\nu(y)
\end{equation}
where $\Phi_{c}$ is the set of continuous bounded functions which verifies:
\begin{equation}
\label{eq:dualphipsi}
\forall x,y \in \Xcal \times \Ycal, \phi(x)+\psi(y)\leq c(x,y)
\end{equation}
This result holds when $\Phi_{c}$ is replaced by $\Phi_{c}(\mu,\nu)$ the set of integrable functions which satisfies \eqref{eq:dualphipsi}.
\end{theo}

\begin{proof}
For completeness we will give here an idea of the proof, the interested reader can refer to Theorem 5.10 in \cite{Villani} for more details. If $\phi,\psi \in \Phi_{c}(\mu,\nu)$ and $\pi \in \couplingset(\mu,\nu)$ then by hypothesis:
\begin{equation}
\int \phi(x) \dr\mu(x)+\int \psi(y) \dr\nu(y)=\int \phi(x)+\psi(y) \dr \pi(x,y) \leq \int c(x, y) \dr \pi(x,y)
\end{equation}
Which implies that $\sup_{\phi,\psi \in \Phi_{c}(\mu,\nu)} \int_{\Xcal} \phi(x)\dr\mu(x)+\int_{\Ycal} \psi(y)\dr\nu(y)\leq \ot_{c}(\mu,\nu)$. To show the converse inequality we will use the cyclical monotonicity properties of optimal transport plans. Let $\pi^{*}$ be an optimal coupling for the \eqref{linear_ot} problem. Using Theorem \ref{fundamental_theorem} we know that there exists a $c$-concave function $\psi$ such that $\psi(x)+\psi^{c}(y)= c(x,y)$ for all $x,y \in \Xcal \times \Ycal$. In this way:
\begin{equation}
\begin{split}
\int c(x, y) \dr \pi^{*}(x,y)&= \int \psi(x) \dr\mu(x)+ \int \psi^{c}(y) \dr\nu(y) \\
&\leq \sup_{\phi,\psi \in \Phi_{c}} \int_{\Xcal} \phi(x)\dr\mu(x)+\int_{\Ycal} \psi(y)\dr\nu(y)
\end{split}
\end{equation}
Last inequality stems from the property \eqref{eq:c_transf} of $c$-transform since $(\psi,\psi^{c})\in \Phi_{c}$. If $c$ is continuous and bounded then so are $\psi,\psi^{c}$ so last inequality is valid. If $c$ is only \textit{l.s.c.} then we can show that there is a sequence $(c_k)_k$ bounded and $n$-Lipschitz such that $c=\sup_k c_k$. A limit argument suffices to conclude for this case.
\end{proof}

The functions $\phi,\psi$ are usually called \emph{Kantorovitch potentials} and play an important role in OT problems. Given two admissible potentials $\phi,\psi$ \ie\ that satisfy $\phi(x)+\psi(y)\leq c(x,y)$ we can always cook up a pair of ``better'' potentials using the $c$-transform. Indeed, due to \eqref{eq:c_transf}, one can check that the pairs $(\phi,\phi^{c})$, $(\psi,\psi^{\bar{c}})$ are also admissible potentials and improve the objective function. It turns out that after one iteration of this procedure we can not improve the potentials anymore. Based on this remark we can also write the duality as the maximization over one \emph{single} potential which is the \emph{semi-dual} formulation:
\begin{equation}
\label{eq:semi_dual_formulation}
\sup_{\phi \ c-\text{concave}} \int_{\Xcal} \phi(x)\dr\mu(x)+\int_{\Ycal} \phi^{c}(y)\dr\nu(y)
\end{equation}

\begin{Example}
When $c=d$ is a distance on some space $\Omega$ then there is a tight connection between $c$-transform and 1-Lipschitz functions. Indeed suppose that $\phi$ is a 1-Lipschitz function, then for $x,y\in \Omega^{2}$, $\phi(y)\leq \phi(x)+d(x,y)$ so that $\phi(y)=\inf_{x\in \Omega} \phi(x)+d(x,y)=(-\phi)^{d}(y)$ which proves that the $d$-transform of $-\phi$ is $\phi$. The converse is also true so that the semi-dual formulation can be written:
\begin{equation}
\sup_{\phi \in Lip_1(\Omega)} \int_{\Omega} \phi(x)\dr\mu(x)-\int_{\Omega} \phi(y)\dr\nu(y)
\end{equation}
This formulation is very useful in practice in the context of generative modeling \cite{arjovsky17a} (see Section \ref{sec:other_formu}).
\end{Example}

\paragraph{The main theorem of OT for the Monge problem} From a theoretical perspective one fundamental question that arises is the \emph{regularity} of these potentials and, with some assumptions, we can use them to solve the Monge problem \eqref{MongeOt} based on the Kantorovitch relaxation \eqref{linear_ot}. We have the following result: 

\begin{prop}[Proposition 1.15 in \cite{San15a}]
\label{kanto_potential}
Let $\Xcal=\Ycal=\Omega \subset \R^{d}$ and $c\in C^{1}$. In the following $\phi$ denotes a Kantorovitch potential. If $(\xbf_{0},\ybf_{0})\in \supp(\pi^{*})$ then: 
\begin{equation}
\label{nabla_eq}
\nabla \phi(\xbf_{0})=\nabla_{\xbf} c(\xbf_{0},\ybf_{0})
\end{equation}
provided that $\phi$ is differentiable at $\xbf_0$.

\end{prop}

This proposition suggests the following strategy in order to find an optimal coupling: (1) Ensure that $\phi$ is differentiable $\mu$ \textit{a.e}. This can be guaranteed when $\mu$ is absolutely continuous with respect to the Lebesgue measure and when $\phi$ is regular enough, such as Lipschitz. (2) Deduce from previous proposition that $\pi$ is characterized by a deterministic function that is the map associating $y_0$ to each $x_0$. The idea here is to ``inverse'' $\nabla c$ in \eqref{nabla_eq} and deduce from $(x_{0},y_{0})$ that $y_{0}$ is uniquely determined from $x_{0}$. This can be done using some regularity assumptions on $c$ and the spaces $\Xcal,\Ycal$. When conditions (1) and (2) are satisfied we can deduce that the  optimal coupling is unique since it was constructed using $\phi$ and $c$ only.

The step (2) can be verified \textit{e.g.} when we have the following condition:

\begin{definition}[Twist condition]
\label{twist_condition}
For $\Omega \subset \R^{d}$ we say that $c: \Omega \times \Omega \rightarrow \R$ satisfies the Twist condition whenever $c$ is differentiable \textit{w.r.t.} $\xbf$ at every point, and the map $\ybf \rightarrow \nabla_{\xbf} c\left(\xbf_{0}, \ybf\right)$ is injective for every $\xbf_{0}$. 
\end{definition}

When working on Euclidean domains and when the cost $c\in C^{2}$ this condition corresponds to $det \left(\frac{\partial^{2} c}{\partial \ybf_{i} \partial \xbf_{j}}\right) \neq 0$. The squared Euclidean cost is an important example of costs which satisfies the Twist condition and leads to the celebrated Brenier theorem \cite{brenier_1991}:

\begin{theo}[Brenier]
Let $\Omega=\R^{p}$, $c(\xbf,\ybf)=\frac{1}{2}\|\xbf-\ybf\|_{2}^{2}$ and $\mu \in \P(\R^{p})$ absolutely continuous with respect to the Lebesgue measure and $\nu \in \P(\R^{p})$ with $\int \|\xbf\|_{2}^{2}\dr \mu(\xbf)<+\infty$,$\int \|\ybf\|_{2}^{2}\dr \nu(\ybf)<+\infty$.

The optimal transport plan $\pi^{*}$ of \eqref{linear_ot} is unique and supported on the gradient of a convex function. More precisely it can be written as $\pi^{*}=(id\times T)\#\mu$ where $T=\nabla \phi$ and $\phi: \R^{d} \rightarrow \R\cup \{+\infty\}$ is convex and finite almost everywhere. 

Moreover $T$ is the unique solution of \eqref{MongeOt}. If $T'$ is another optimal solution then $T=T'$ $\mu$ \textit{a.e}. 
\end{theo}

This result can be generalized to costs $c(\xbf,\ybf)=h(\ybf-\xbf)$ with $h$ strictly convex and in this case $T$ can be written as $T(\xbf)=\xbf-(\nabla h)^{-1}\nabla \phi(\xbf)$ where $\phi$ is a $c$-concave function (see \textit{e.g.} \cite{gangbo1996}). 

The regularity of potential functions $\phi,\psi$ and its consequences for Optimal Transport problems is a long-standing line of research. Other more general hypothesis on the cost function $c$ than the Twist condition can be built thanks to Ma, Trudinger and Wang who found a key assumption on the cost $c$ that requires fourth-order condition on the cost functions \cite{mtw_2005}. The resulting MTW conditions turned out to be sufficient to prove the regularity of the Kantorovitch potentials. We refer the reader to \cite{figalli_2010} for a survey on this topic.

\subsection{Special cases: 1D transportation and transport between Gaussians}
\label{sec:sepcial_cases_ot}
Two important special cases will be considered in this manuscript, namely the cases where $\mu$ and $\nu$ are probability distributions on $\R$ or when they are Gaussians distributions. Theses cases are well-known in linear OT for having \emph{closed-form solutions} which are given in the next results (respectively \cite[Theorem 2.9]{San15a} and \cite[Remark 2.30]{cot_peyre_cutu}).

\begin{theo}[Closed-form expression on the real-line]
Asssume that $\Omega=\R$, $\mu,\nu \in \P(\R)$. Let $F_{\mu}$ be the cumulative distribution function:
\begin{equation}
\forall t \in \R, \ F_{\mu}(t)=\mu(]-\infty,t])
\end{equation}
and $F_{\mu}^{-1}$ its pseudo inverse, namely:
\begin{equation}
\forall x \in [0,1], \ F_{\mu}^{-1}(x)=\inf\{t\in \R \ | F_{\mu}(t)\geq x\} 
\end{equation}

If $c(x,y)=h(y-x)$ where $h$ is stricly convex then \eqref{linear_ot} has a unique solution given by $\pi^{*}=(F_{\mu}^{-1} \times F_{\nu}^{-1})\#\lebsm_{[0,1]}$ where $\lebsm_{[0,1]}$ is the Lebesgue measure restricted to $[0,1]$. 

Moreover if $\mu$ is atomless $\pi^{*}$ is supported on $T_{mon}(x)=F_{\nu}^{-1}(F_{\mu}(x))$, \ie\ $\pi^{*}=(id\times T_{mon})\#\mu$. If $h$ is only convex then $\pi^{*}$ is still optimal but uniqueness can not be guaranteed. 

\end{theo}

This theorem states that it suffices to sort the support of the distributions in order to recover the optimal coupling (see Figure \ref{sorting_fi}). In the special case where $\mu=\frac{1}{n} \sum_{i=1}^{n} \delta_{x_i}, \nu=\frac{1}{n} \sum_{i=1}^{n} \delta_{y_i}$ this corresponds to sort $x_1\leq \cdots \leq x_n$, $y_1\leq \cdots \leq y_n$ and to associate $x_1$ with $y_1$, $x_2$ with $y_2$ and so on. In the case $\mu=\sum_{i=1}^{n} a_i \delta_{x_i}, \nu= \sum_{i=1}^{m} b_j \delta_{y_j}$ the previous theorem states that, after sorting the points, the optimal mapping is obtained by putting as much mass as possible from $x_1$ to $y_1$ and to add the remaining mass to $y_2$. This procedure is repeated until there is no more mass left. This corresponds to a \emph{monotone rearrangement} $\pi_{ij},\pi_{i'j'}>0$ then $x_i\leq x_{i'}$ implies that $y_j\leq y_{j'}$. Overall the Wasserstein distance in 1D can be solved using simple sorts. This result is the main ingredient of the sliced-Wasserstein distance (see Section \ref{sec:other_formu}). 

\begin{figure}[t]
   \begin{center}
        \includegraphics[width=0.9\linewidth]{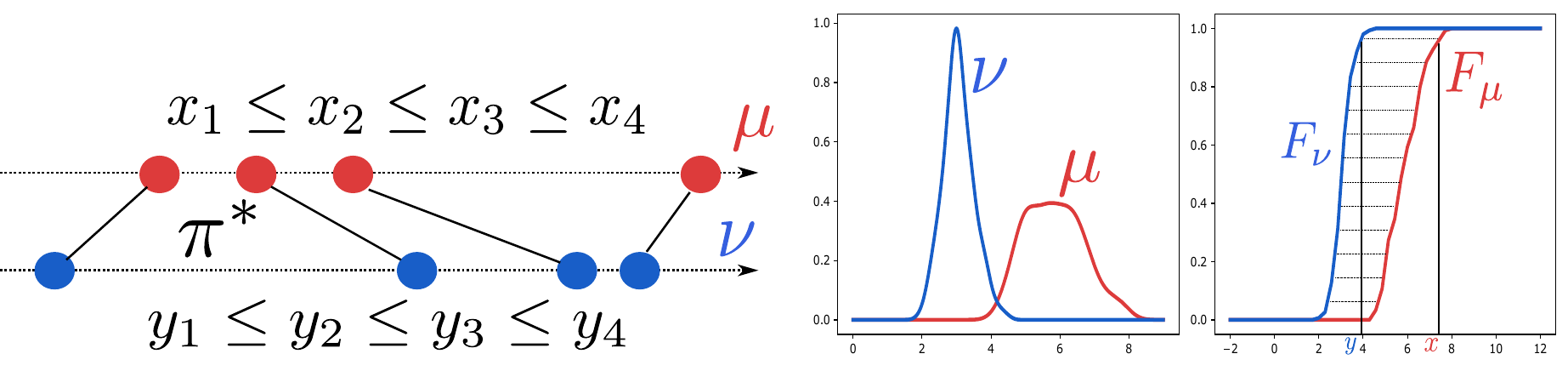}
    \caption{OT for 1D probability measures can be computed using simple sorts. {\bf (left)} Optimal coupling between two discrete probability measures with uniform weights. When the points are sorted it associates first point of the source with the first point of the target and so on. {\bf (right)} Generic case: the optimal coupling associates horizontally the points \textit{w.r.t.} the cumulative distributions of the probability measures. In this example $x$ is associated with $y$. \label{sorting_fi}}
     \end{center}
\end{figure}

Another special case arises when the probability measures are Gaussian. This is a well known result in the literature of OT geometry \cite{givens1984,McCann_1997,takatsu2011} which is recalled in the following theorem:

\begin{theo}[Closed form expression for Gaussians]
\label{prop:closed_gaussian}
Let $\mu=\mathcal{N}(\mathbf{m_\mu},\Sigmab_\mu),\nu=\mathcal{N}(\mathbf{m_\nu},\Sigmab_\nu)$ and suppose that $c(\xbf,\ybf)=h(\ybf-\xbf)$ with $h$ striclty convex. 

Let $T: \xbf\rightarrow \mathbf{m_{\nu}}+\mathbf{A}(\xbf-\mathbf{m_\nu})$ where:
\begin{equation}
\mathbf{A}=\Sigmab_{\mu}^{-1/2}(\Sigmab_{\mu}^{1/2}\Sigmab_{\nu}\Sigmab_{\mu}^{1/2})^{\frac{1}{2}}\Sigmab_{\mu}^{-1/2}
\end{equation}
then $T$ is the unique optimal solution of \eqref{MongeOt} and $\pi^{*}=(id\times T)\#\mu$ is the unique optimal solution of \eqref{linear_ot}. 

In particular when $c(\xbf,\ybf)=\|\xbf-\ybf\|_{2}$ is the Euclidean distance on $\R^{d}$ the $2$-Wasserstein distance is given by:

\begin{equation}
\wass_2^{2}(\mu,\nu)=\|\mathbf{m_\mu}-\mathbf{m_\nu}\|_{2}^{2}+\mathcal{B}(\Sigmab_\mu,\Sigmab_\nu)^{2}
\end{equation}
where $\mathcal{B}(\Sigmab_\mu,\Sigmab_\nu)=\tr\left(\Sigmab_\mu+\Sigmab_\nu-2(\Sigmab_{\mu}^{1/2}\Sigmab_{\nu}\Sigmab_{\mu}^{1/2})^{\frac{1}{2}}\right)$ is the Bures metric \cite{bure}.

\end{theo}

Interestingly enough the problem of computing OT between Gaussian measures draws connections with the general case. Indeed for $\mu,\nu \in \P_{2}(\R^{d})$ and $c(\xbf,\ybf)=\|\xbf-\ybf\|_{2}$ the optimal map $T$ defined in Theorem \ref{prop:closed_gaussian} is actually the optimal Monge map of \eqref{MongeOt} when restricted to the class of linear Monge map \cite[Proposition 1]{flamary2019concentration}. Figure \ref{fig:diplacement} illustrates the behavior of this map for two discrete probability measures on $\R^{2}$.

\begin{figure}[t]
   \begin{center}
        \includegraphics[width=0.5\linewidth]{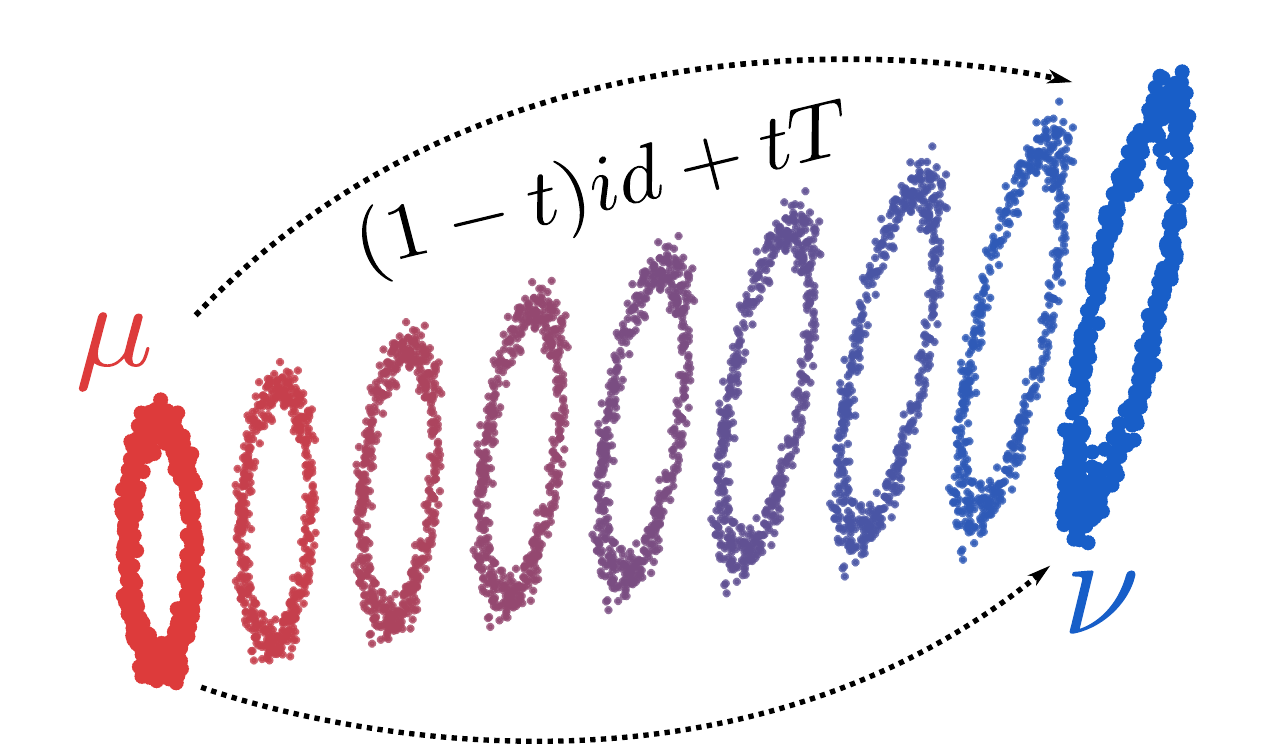}
    \caption{ \label{fig:diplacement} Linear displacement interpolation between two discrete probability measure $\mu,\nu$ on $\R^{2}$ using the linear map $T$ defined in \ref{prop:closed_gaussian}. The figure shows the displacement interpolant which is a probability measure is defined by $((1-t)id+T)\#\mu$ for $t \in [0,1]$ \cite{McCann_1997} (see Section \ref{sec:wass_bary} for more details)}
     \end{center}
\end{figure}

Note that a generalization of the previous result exists for elliptical distributions which are somehow generalizations of Gaussian densities. In this case the $\wass_2$ admits also a closed-form (see \cite{muzellec}).

\subsection{Some statistical aspects of OT \label{sec:sample_complexity} } In most of machine learning applications we do not have access to the true distributions $\mu,\nu$ but only to samples from these distributions. As such a natural question arises: can we infer from this samples good estimates of OT objects such as couplings or OT distances? One particular question is how well can we estimate the Wasserstein distance by relying only on samples of the distribution? If we consider a probability distribution $\mu \in \P(\R^{d})$ and an empirical distribution $\mu_n=\frac{1}{n}\sum_{i=1}^{n} \delta_{x_i}$ where $x_i \sim \mu$ are \textit{iid} samples does $\mu_n$ is a good proxy for $\mu$? Unfortunately the sample complexity of the estimation of the Wasserstein distance is exponential in the dimension of the ambient space. More precisely $\mathbb{E}[\wass_p(\mu_n,\mu)]=O(n^{-\frac{1}{d}})$ so that the Wasserstein distance suffers from the curse of dimensionality \cite{dudley1969,weedbach2017}. It was shown in \cite{weedbach2017} that this result can be refined to $O(n^{-\frac{1}{p}})$ where $p$ is the intrinsic dimension of the data but, generally, this is a major bottleneck for the use of OT in high-dimensional machine learning problems. Previous analysis can be extended to the infinite dimensional setting as analysed in \cite{lei2020convergence}.  The problem of estimating the optimal coupling by relying on small batches of $\mu$ when it is discrete was further analyzed in \cite{pmlr-v108-fatras20a}.  

To circumvent this limitation some robust projection formulations have been proposed \cite{weed_2019_spike,lin2020projection} as well as strategies such as gaussian-smoothing \cite{pmlr-v108-goldfeld20a} or based on wavelet estimator \cite{pmlr-v99-weed19a}. The entropic regularization presented in the next section in also one of the tool that facilitates the estimation of $\wass_p$ for high-dimensional settings.   For more details about statistical aspects of OT we refer the reader to \cite{phdthesis_weed}.

\subsection{A quick numerical tour: solving Optimal Transport}
\label{sec:numerical_tour}

In this section we consider the problem of computing OT between discrete probability measures $\mu=\sum_{i=1}^{n} a_{i} \delta_{x_{i}}, \nu=\sum_{i=1}^{m} b_{j} \delta_{y_{j}}$. The problem can be solved in many ways and we aim here at giving a brief summary these possibilities. We denote by $\C=(c_{ij})_{i \in \integ{n},j\in \integ{m}}$ the matrix of all pair-to-pair costs between the samples $x_i,y_j$, \ie\ $c_{ij}=c(x_i,y_j)$ for all $i \in \integ{n},j\in \integ{m}$. In the discrete case the underlying problem reads:

\begin{equation}
\label{discret_ot}
\min_{\GG \in \couplingset(\a,\b)} \froeb{\C}{\GG}=\min_{\GG \in \couplingset(\a,\b)}  \sum_{ij} c_{ij}\pi_{ij}.
\end{equation}

As described previously the problem is \emph{linear} in $\GG$, in this way the discrete case corresponds to a linear program (LP) \cite{10.5555/248375}. Before discussing potential algorithms for solving equation \eqref{discret_ot} we detail one important special case.

\begin{figure}[t]
   \begin{center}
        \includegraphics[width=0.6\linewidth]{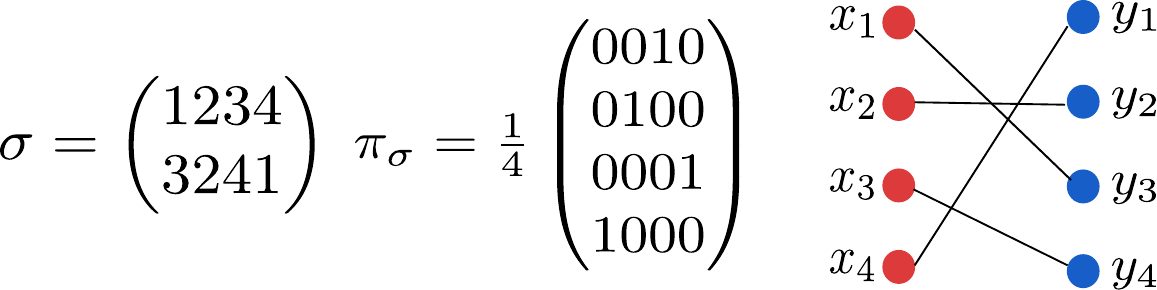}
    \caption{\label{assign}An assignment of $x_1,x_2,x_3,x_4$ to $y_1,y_2,y_3,y_4$ can be described whether by a permutation $\sigma$ or using assignment matrix $\GG_{\sigma}$. In this example both correspond to the permutation $\sigma$: $1 \rightarrow 3$,$2 \rightarrow 2$,$3 \rightarrow 4$, $4 \rightarrow 1$ }
     \end{center}
\end{figure}

\paragraph{Assignment problems} Suppose that $m=n$. In this case we can look for an \emph{assignment} of the points, that is a one-to-one correspondence between the points $x_i,y_i$. This translates mathematically by looking at the permutation $\sigma \in \Sn$ of the points or at the permutation matrix $\GG_{\sigma}=\begin{cases}
    \frac{1}{n},& \text{if } j=\sigma(i)\\
    0,              & \text{otherwise}
\end{cases}$ such that the overall cost is minimized (see Figure \ref{assign}). In this situation one aims at solving:

\begin{equation}
\label{linear_assignment_problem}
\min_{\sigma \in \Sn} \frac{1}{n} \sum_{i=1}^{n} c_{i,\sigma(i)}=\min_{\sigma \in \Sn} \froeb{\C}{\GG_{\sigma}}
\end{equation}

This problem is well known in the literature as the \emph{linear assignment problem} (see \textit{e.g.} \cite{Burkard1999}). It is worth pointing that, in this case, it exactly corresponds to the Monge problem \eqref{MongeOt} in the discrete case. Interestingly enough when the weights of the OT problem \eqref{discret_ot} are set as uniform \ie\ $\a=\b=\frac{1}{n}\one_{n}$ both problems \eqref{discret_ot} and \eqref{linear_assignment_problem} are equivalent. More precisely by combining the fundamental theorem of linear programming \cite{bertsimas-LPbook}, which states that the minimum of a linear program is reached at an extremal point of the polyhedron, and Birkhoff's theorem \cite{birkhoff:1946}, which states that the extremal points of $\couplingset(\frac{1}{n}\one_{n},\frac{1}{n}\one_{n})$ is the set of permutation matrices, we can conclude that the optimal map of \eqref{discret_ot} is reached at $\GG_{\sigma}$ which is optimal for \eqref{linear_assignment_problem}.

\paragraph{Algorithmic solutions} To solve the OT problem \eqref{discret_ot} in general one can rely on classical algorithms for solving (LP) \cite{10.5555/248375}. We make here a brief overview of possible numerical solutions and we refer the reader to Section 3 in \cite{cot_peyre_cutu} for more details.

As seen in Theorem \ref{strong_duality} the OT problem can be solved using duality which reads in the discrete case:

\begin{equation}
\label{eq:duality_discrete}
\max_{\begin{smallmatrix}\alphab\in \R^{n}, \betab \in \R^{m} \\ \forall (i,j)\in \integ{n}\times \integ{m}, \alpha_i+\beta_j\leq c_{ij}\end{smallmatrix}} \alphab^{T}\a+\betab^{T}\b
\end{equation}

where $\alphab,\betab$ denotes the Kantorovitch potentials. Thanks to the fundamental Theorem \ref{fundamental_theorem} an optimal solution $\GG^{*}$ of the primal problem is found when $\alpha_i^{*}+\beta_j^{*}=c_{ij}$ for $\pi^{*}_{ij}>0$ where $\alphab^{*},\betab^{*}$ are solutions of the dual problem. Using this remark we can solve \eqref{discret_ot} by relying on the Network Simplex algorithm which philosophy is to find feasible solutions $(\alphab,\betab)$ such that $\alpha_i+\beta_j=c_{ij}$ whenever $\pi_{ij}>0$ (in this case we say that $\GG$ and $(\alphab,\betab)$ are complementary \textit{w.r.t.} $\mathbf{C}$). The complexity of this algorithm is $O(n^{3} \log(n))$ when $m=n$. The special case of uniform weights for assignment problems can be solved using the Auction algorithm which has a cubic complexity $O(n^3)$. 

\paragraph{Special cases: Monge property} The case where $\mathbf{C}$ has special structure deserves attention. In particular when $\mathbf{C}$ satisfies the following \emph{Monge property} \cite{birkard_perspecitve}: 
\begin{equation}
\label{mongeproperty}
\forall(i,j) \in \integ{n}\times \integ{m}, \ c_{i,j}+c_{i+1,j+1} \leq c_{i+1,j} + c_{i,j+1}
\end{equation}
which can be tested in $O(mn)$ operations. This property has some interesting historical background. It is actually based on the original observation of Monge who states that if quantity must be transported from locations $x_1$,$y_1$ to locations $x_2$,$y_2$ then the route from $x_1$ and the route from $y_1$ must not intersect: better not to cross the paths!\footnote{The original quote by Monge is \cite{monge_81}: "Lorsque le transport du deblai se fait de manière que la somme des produits
des molécules par l’espace parcouru est un minimum, les routes de deux points
quelconques A \& B, ne doivent plus se couper entre leurs extrémités, car la
somme Ab + Ba, des routes qui se coupent, est toujours plus grande que la
somme Aa + Bb, de celles qui ne se coupent pas."} In this case the simple North-West corner rule (see Algorithm \ref{alg:nw}) produces an optimal solution in $O(n+m)$.

\begin{algorithm}[t]
    \caption{\label{alg:nw} North-West corner rule}
    \begin{algorithmic}[1]
    \State $\mathbf{a}, \mathbf{b}, i,j=1$
        \While{$i<=n$,\ $j<=m$}
        \State $\pi_{ij}=\min\{a_i,b_j\}$ // \text{ Send as many units as possible form $i$ to $j$}  
        \State $a_i=a_i-\pi_{ij}$ // \text{Adjust the supply}
        \State $b_j=b_j-\pi_{ij}$ // \text{Adjust the demand}
        \State{\text{If} $a_i=0$, $i=i+1$, \text{if} $b_j=0$, $j=j+1$}
        \EndWhile 
            \end{algorithmic}
\end{algorithm}

\paragraph{Special cases: 1D probability distributions} As seen in Section \ref{sec:sepcial_cases_ot} the case of 1D probability distributions can be solved efficiently using simple sorts when $\C$ is \textit{e.g.} a squared Euclidean distance matrix. The complexity of computing the Wasserstein distance is $O(n\log(n))$ when $n=m$ and weights are uniform and in general it suffices to compute the two cumulative distribution functions which is $O(n\log(n)+m\log(m))$.

\paragraph{Special cases: Gaussian distributions} When $\mu$ and $\nu$ are Gaussian distributions (and with a Euclidean cost) the OT problem is also quite easy to solve. In the discrete case, when relying on samples from $\mu,\nu$, and using the empirical version of the means and covariances, finding the optimal solution has a $O((n+m)d^{2}+d^{3})$ complexity.

Although previous special cases exist solving the OT problem in general remains costly. The next section presents a regularization scheme that tends to lower this computational complexity and was one of the major breakthrough in OT past years.

\subsubsection{Entropic regularization}

 \begin{figure}[t]
   \begin{center}
        \includegraphics[width=1\linewidth]{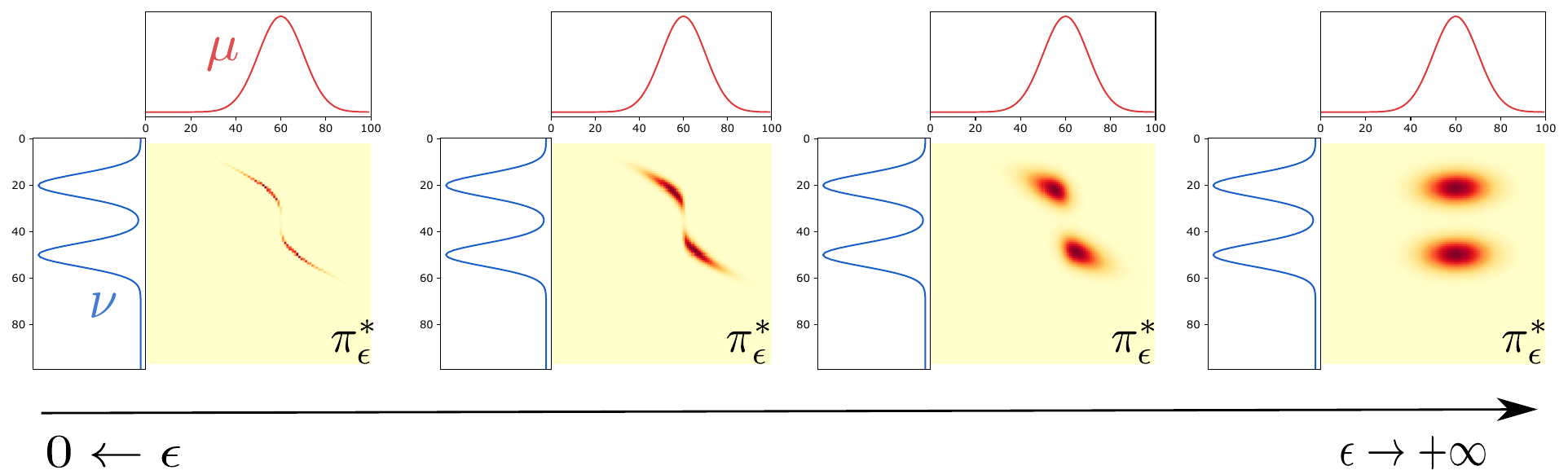}
    \caption{\label{fig:entropic_reg} Effect of the entropic regularization parameter $\epsilon$ on the optimal coupling $\GG^{*}_{\epsilon}$ between two 1D probability distributions. As $\epsilon$ increases the coupling tends to blur and converges to the marginals' product coupling.}
     \end{center}
\end{figure}

The idea of penalizing the entropy of the joint coupling can be traced back to Schrödinger \cite{schro} and its use for linear OT to Wilson \cite{entropy_first}, yet it was made popular quite recently in the OT community \cite{cuturi2013sinkhorn}. The entropic regularization has multiple virtues in practice: 1) it turns the optimal transport problem into a strongly-convex minimization problem which solution is unique 2) solving an entropic regularized OT only involves simple iterations of matrix-vectors products which can be plugged easily  into modern differentiable frameworks 3) it can be accelerated on GPU and can solve in parallel several OT problems 4) it has many desirable properties for high-dimensional problems statistically speaking.

The entropy term for a coupling $\GG$ reads as:
\begin{equation}
\label{entropy_term}
H(\GG)=-\sum_{ij} (\log(\pi_{ij})-1)\pi_{ij}
\end{equation}
 which corresponding entropic regularized OT problem:

 \begin{equation}
\tag{$\epsilon$-KP}
\label{entropic_transport}
\ot_{c}^{\epsilon}(\mu,\nu)=\min_{\GG \in \couplingset(\a,\b)} \froeb{\C}{\GG}-\epsilon H(\GG)
 \end{equation}

 Interestingly enough the optimal cost of \eqref{linear_ot} can be obtained as $\epsilon \rightarrow 0$, \ie\ $\lim_{\epsilon \rightarrow 0} \ot_{c}^{\epsilon}(\mu,\nu)=\ot_{c}(\mu,\nu)$ \cite[Propositon 4.1]{cot_peyre_cutu}. As a side effect, the entropic term tends to blur the optimal coupling so that more points are associated compared to the sparse optimal solution of the original problem. In other words entropy forces the solution to have a spread support. In the limit setting where $\epsilon \rightarrow +\infty$ all points are coupled together such that $\lim_{\epsilon \rightarrow +\infty} \GG^{*}_{\epsilon}=\mathbf{a} \mathbf{b}^{T}$ where $\GG^{*}_{\epsilon}$ denotes the optimal coupling of \eqref{entropic_transport} (see Figure \ref{fig:entropic_reg}). Note that the entropy regularization can also be defined when the probability measures are not discrete and in this case reads $H(\pi)=\int_{\Xcal \times \Ycal}(\log(\frac{\dr \pi(x,y)}{\dr \mu(x) \dr \nu(y)})-1)\dr \pi(x,y)$.

 \paragraph{Sinkhorn-Knopp and Bregman projections} A simple analytic solution of \eqref{entropic_transport} can be found using the lagragian duality as expressed in the following proposition \cite[Proposition 4.3]{cot_peyre_cutu}:

 \begin{prop}
Problem \eqref{entropic_transport} has a unique solution of the form $\GG^{*}=\diag({\bf u})\mathbf{K}\diag({\bf v})$ with $\mathbf{K}=e^{-\frac{\mathbf{C}}{\epsilon}}$ and $\mathbf{u},\mathbf{v} \in \R^{n}_{+}\times \R^{n}_{+}$ .
 \end{prop}

As written in \cite{sinkhorn1967} there is a unique solution of the form $\GG^{*}=\diag({\bf u})\mathbf{K}\diag({\bf v})$ with marginals $\a,\b$ providing that $\mathbf{K}$ is positive definite. Moreover it can be recovered based on the Sinkhorn-Knopp Matrix scaling algorithm that relies on matrices multiplications by alternatively updating $\mathbf{u}$ and $\mathbf{v}$ in order for $\GG^{*}$ to have the prescribed marginals (see in Algorithm \ref{alg:sk}). When $n=m$ and by setting $\tau=\frac{4\log(n)}{\epsilon}$ the Sinkhorn algorithm produces an optimal solution $\GG^{*}$ such that $\froeb{\mathbf{C}}{\GG^{*}}\leq \ot_{c}(\mu,\nu)+\tau$ after $O(\|C\|_{\infty}^{3} \log(n)\tau^{-3})$ iterations \cite{altschuler2017near}. In particular this implies that a $\tau$-approximate solution of the original unregularized problem can be computed in $O(n^{2}\log(n)\tau^{-3})$ time. 

From a practical point of view the Sinkhorn's algorithm suffers from stability issues when $\epsilon \rightarrow 0$ as the kernel $\mathbf{K}$ vanishes rapidly which results in divisions by $0$ during the algorithms' iterations. To avoid such underflows for small value of $\epsilon$ \cite{Schmitzer_stab_sinkhorn} suggest a log-sum-exp stabilization
trick whose iterations turn to be mathematically equivalent to the original iterations.

    \begin{algorithm}[t]
        \caption{\label{alg:sk} Sinkhorn-Knopp Algorithm for entropic transport}
        \begin{algorithmic}[1]
        \State $\mathbf{a}, \mathbf{b}, \mathbf{C}, \epsilon >0, \mathbf{u}^{(0)},\mathbf{v}^{(0)}=\one, \mathbf{K}=\exp (-\frac{\mathbf{C}}{\epsilon})$
            \For{$i=1,\dots,n_{it}$}
            \State $\mathbf{u}^{(i)}=\mathbf{a} \oslash \mathbf{K}^{\top}\mathbf{v}^{(i-1)}$ // \text{ Update left scaling } 
            \State $\mathbf{v}^{(i)}=\mathbf{b} \oslash \mathbf{K} \mathbf{u}^{(i-1)}$ // \text{ Update right scaling } 
            \EndFor 
            \State \Return $\GG^{*}=\diag({\bf u})\mathbf{K}\diag({\bf v})$
        \end{algorithmic}
    \end{algorithm}

This problem is also a special case of a Kullback-Leiber minimization problem where one wants to find a coupling matrix $\GG^{*}$ the closest possible to a kernel $\mathbf{K}$ in the sense of the Kullback-Leiber geometry. More precisely \eqref{entropic_transport} is equal to:

\begin{equation}
\min_{\GG \in \couplingset(\a,\b)} \mathbf{KL}(\GG | \mathbf{K})
\end{equation}

where $\mathbf{K}=e^{-\frac{\mathbf{C}}{\epsilon}}$ and $\mathbf{KL}(\GG,\mathbf{K})=\sum_{ij}\pi_{ij}\log(\frac{\pi_{ij}}{K_{ij}})-\pi_{ij}+K_{ij}$ is the Kullback-Leiber divergence between $\GG$ and $\mathbf{K}$. Reformulating OT problems as a minimization of a Kullback-Leiber divergence allows the use of the machinery of Bregman projections in order to find a solution and to analyse the convergence \cite{benamou:2015}. This formulation is particularly interesting for solving multi-marginals OT problems \cite{phdthesis_Luca}, regularized OT barycenter \cite{benamou:2015,bigot_2019} and the Gromov-Wasserstein problem \cite{peyre2016gromov} in short. 

\paragraph{Sinkhorn divergences} One drawback of entropic regularized OT is that it induces a bias $\ot_{c}^{\epsilon}(\mu,\mu)\neq 0$ which can be problematic for learning using $\ot_{c}^{\epsilon}$. In \cite{pmlr-v84-genevay18a} authors propose to correct this bias by considering the so-called Sinkhorn divergence:
\begin{equation}
\label{eq:sink:div}
SD_{c,\epsilon}(\mu,\nu)=\ot_{c}^{\epsilon}(\mu,\nu)-\frac{1}{2}\ot_{c}^{\epsilon}(\mu,\mu)-\frac{1}{2}\ot_{c}^{\epsilon}(\nu,\nu).
\end{equation} 
This divergence enjoys many valuable properties. First it defines a symmetric positive definite smooth function on the space of probability measures that is convex in both $\mu,\nu$ and that metrizes the weak convergence of probability measures \cite{pmlr-v89-feydy19a}. Second it interpolates, through $\epsilon$, between Wasserstein distance and Kernel norms
(MMD) allowing  finding a trade-off
between both. Finally it is more suited for high-dimensional problems where the estimation of the Wasserstein distance is known to suffer from the curse of dimensionality (the sample complexity if $O(n^{-\frac{1}{d}})$ as explained in Section \ref{sec:sample_complexity}) whereas the sample complexity of $SD_{c,\epsilon}$ is $O(\epsilon^{-\frac{d}{2}} n^{-\frac{1}{2}})$ \cite{genevay:2019}.

\paragraph{Stochastic Optimal Transport: going large scale.} The regularization of linear OT allows deriving stochastic formulations that are useful in practice to handle large scale datasets. This setting was considered in \cite{genevay_stochastic,seguy2018large} where authors rely on the dual formulation \eqref{strong_duality} or the semi-dual formulation \eqref{eq:semi_dual_formulation} in the regularized case. More precisely for $\mu,\nu\in \P(\R^{d})$ and $\epsilon>0$ the regularized dual (\textit{resp.} semi-dual) boils down to solve the following unconstrained maximization problems:
\begin{equation}
\label{eq:smooth_dual}
\tag{s-D}
\sup_{\phi,\psi \in C(\R^{d})\times C(\R^{d})} \underset{\xbf\sim\mu,\ybf\sim\nu}{\E}[F_{\epsilon}(\phi(\xbf),\psi(\ybf))]
\end{equation}
\begin{equation}
\label{eq:smooth_semi_dual}
\tag{s-SD}
\sup_{\psi \in C(\R^{d})} \underset{\xbf\sim\mu}{\E}[H_{\epsilon}(\xbf,\psi)]
\end{equation}
where $F_{\epsilon}(\phi(\xbf),\psi(\ybf))=\phi(\xbf)+\psi(\ybf)-\epsilon e^{\frac{\phi(\xbf)+\psi(\ybf)-c(\xbf,\ybf)}{\epsilon}}$ and $H_{\epsilon}(\xbf,\psi)=\int_{\Ycal} \psi(\ybf)\dr \nu(\ybf)-\epsilon\log(\int e^{\frac{1}{\epsilon}(\psi(\ybf)-c(\xbf,\ybf))} \dr \nu(\ybf))-\epsilon$ when entropic regularization is used. Since the problem is recast in the form of an unconstrained maximization of an expectation, the idea is to use stochastic gradients tools such as Stochastic Gradient Descent (SGD), or Stochastic Averaged Gradient (SAG) to compute a solution of \eqref{eq:smooth_dual},\eqref{eq:smooth_semi_dual}. When both $\mu=\sum_{i=1}^{n} a_i \delta_{\xbf_i}$,$\nu=\sum_{j=1}^{m} b_j \delta_{\ybf_j}$ are discrete (SGD) or (SAG) are directly applicable to maximize the following finite sums:
\begin{equation}
\label{eq:smooth_dual_discrete}
\tag{s-Ddis}
\max_{\alphab,\betab \in \R^{n}\times \R^{m}} \sum_{i=1,j=1}^{n,m} F_{\epsilon}(\alpha_{i},\beta_{j})a_i b_j
\end{equation}
\begin{equation}
\label{eq:smooth_semi_dual_discrete}
\tag{s-SDdis}
\max_{\betab \in \R^{m}} \sum_{i=1}^{n} H_{\epsilon}(\xbf_i,\betab)a_i
\end{equation}
In \cite{genevay_stochastic} authors propose to use (SAG) to compute \eqref{eq:smooth_semi_dual_discrete} which operates at each iteration by sampling a point $\xbf_k$ from $\mu$ then to compute the gradient of $H_{\epsilon}(\xbf_k,\betab)$ corresponding to that sample while keeping in memory
a copy of past gradients. This approach costs $O(n)$ per iteration due to the computation of the gradient and converges to a solution within $O(k^{-1})$ iterations. In contrast in \cite{seguy2018large} propose to solve \eqref{eq:smooth_dual_discrete} by applying an (SGD) on mini-batches of both $\mu,\nu$ which comes with a $O(p^{2})$ cost per iteration where $p$ is the mini-batch size and also converges in $O(k^{-1})$. In the continuous setting the problem is infinite dimensional so that it can not be solved using (SGD) anymore. In \cite{genevay:2019} authors propose to represent the dual variables as kernel expansions while in \cite{seguy2018large} the dual variables are parametrized by a neural network. Another line of works rely on the unregularized problem and on the special case of $\wass_1$. In this case the duality reads $\sup_{\phi}\underset{\xbf\sim\mu,\ybf\sim\nu}{\E}[\phi(\xbf)-\phi(\ybf)]$ where the maximization is done over all $1$-Lipschitz function $\phi$. In \cite{arjovsky17a} authors tackled this problem in the context of generative modelling. They parametrized $\mu,\nu$ using a neural network and used the same (SGD)+mini-batch procedure resulting on a $O(p)$ cost per iteration. Their approach however relies on a weight clipping of the (NN) weights in-between gradient updates to enforce the Lipschitz constraint which lead to optimization difficulties \cite{NIPS2017_7159}. Note all approaches comes at the price of biasing the optimal coupling due to the mini-batch sampling. This effect was further analyzed in \cite{pmlr-v108-fatras20a}. 

\subsection{Other formulations \label{sec:other_formu}}

Apart from entropic-regularized OT there are a lot of other methods for approximating OT. One of them relies on the closed-form expression of OT for probability distributions over the real line resulting on the so-called \emph{Sliced Wasserstein distance} (SW) \cite{rabin2011wasserstein}. Considering $\mu,\nu \in \P(\R^{d})$ the key idea is to randomly select lines in $\R^{d}$, to project the measures into these lines and to compute the resulting 1D-Wasserstein distance which can be done using simple sorts as seen previously. The sliced-Wasserstein distance is the average of all these 1D-Wasserstein distances over all drawn lines. More precisely:

\begin{definition}[Sliced Wasserstein distance]
Let $\lambda_{d-1}$ be the uniform measure on $\mathbb{S}^{d-1}$ . 

For $\thetab \in \mathbb{S}^{d-1}$ we note $P_{\thetab}$ the projection on $\thetab$, \ie\ $P_{\thetab}(\xbf)=\langle \xbf, \thetab \rangle$. Let $\mu,\nu \in \P(\R^{d})^{2}$. The Sliced Wasserstein distance between $\mu$ and $\nu$ is defined as:
 \begin{equation}
 SW^p_p(\mu,\nu)=\int_{\mathbb{S}^{d-1}} W^p_p(P_{\thetab}\#\mu,P_{\thetab}\#\nu) d\lambda_{d-1}(\thetab)
 \label{eq:pSW}
 \end{equation}
 where the Wasserstein distance is defined with the standard Euclidean distance on $\R^{d}$.

 \end{definition}

$SW$ enjoys several interesting properties. First $SW_2$ induces a similar topology than $\wass_2$: it defines a distance on $\P_{p}(\R^{d})$ \cite{bonotte_phd} that metrizes the weak convergence \cite{Nadjahi_sliced_assympt} and which is equivalent to the Wasserstein distance for measures with compact supports \cite{Nadjahi_2020_prop_sliced, bonotte_phd}. Second it defines a positive definite kernel $e^{-\gamma SW^2_2(\mu,\nu)}$ for $\gamma>0$ over the space of probability distributions that can be easily plugged into an SVM \cite{Kolouri_2016_CVPR}. This contrasts with the Wasserstein distance $\wass_2$ which is not \emph{Hilbertian} and consequently does not define a positive definite kernel (see Section 8.3 in \cite{cot_peyre_cutu}). In terms of sample complexity $SW$ is known to be dimension independant \cite{Nadjahi_2020_prop_sliced} such as $O(n^{-1/2})$ when $p=2$ \cite{lin2020projection,Nadjahi_2020_prop_sliced} and better samples complexities can be found by projecting on subspaces of dimension $k>1$ instead of random lines \cite{lin2020projection, subspace_robust_wass_patty_2019}, yet raising tractability issues as the sorting trick is no longer valid. Moreover, as a side effect of its definition, $SW$ is unable to find the correspondences between the samples of the distributions as it does not provide an optimal transport map $\pi$ which is valuable for certain application such as domain adaptation \cite{courty2017optimal}.

From a practical side estimating $SW$ requires the calculation of an integral over the hypersphere which can be done using a simple Monte-Carlo scheme. Hence for discrete probability measures with $n$ atoms the overall complexity of computing $SW$ is $O(Ln\log(n))$ where $L$ is the number of projection directions on $\mathbb{S}^{d-1}$. The quality of the Monte Carlo estimates is impacted by the number of projections as well as
the variance of the evaluations of the Wasserstein distance has pointed out empirically in \cite{kolouri_generalized_sliced,Deshpande_2019_CVPR} and more formally in \cite{Nadjahi_2020_prop_sliced}. The low computational complexity of $SW$ makes it very attractive for a number of \textit{scenarii} such as in deep learning for generative modeling \cite{cvpr_sliced_gan,Deshpande_2019_CVPR}, for barycenter computation \cite{bonneel:hal-00881872} or topological data analysis \cite{carriere_persi} to name a few. This ``projection'' idea was further developed and improved in several works which proposed to project on $k$-dimensional subspaces \cite{subspace_robust_wass_patty_2019}, to use non-linear projection \cite{kolouri_generalized_sliced} or to generalize $SW$ for the unbalanced setting \cite{bonneel_spot}.

Many other interesting formulations can be derived from the original OT formulation. Since they are not considered in this manuscript we just give a brief overview here. In \cite{ferradans2014regularized} author propose to regularize the linear OT with a quadratic term, resulting on a quadratic regularized OT. For regular grids \cite{solomon_convo} define a Wasserstein distance that can be computed efficiently in $O(n^2\log(n))$ using convolutions. Another line of works consider an unbalanced setting where the source probability measure is partially transferred to the target probability measure resulting on the unbalanced formulation \cite{chizat_unbalanced}. A case of particular interest is when  the target probability measure is discrete and the source continuous, namely the semi-discrete OT. It founds many applications in practice and can be tackled using Laguerre cells \cite{LEVY2018135}. Finally the multi-marginal OT aims at solving an linear OT problem where there are many target/source probability measures and one optimal coupling for transporting them all \cite{phdthesis_Luca}.

\subsection{Wasserstein barycenter \label{sec:wass_bary}}

The Wasserstein distance is also an interesting tool in order to compute a notion of \emph{barycenter} of probability distributions. In an Euclidean setting the traditional barycenter of points $(\xbf_i)_{i \in \integ{m}}$ can be computed by solving $\inf_{\xbf \in \R^{d}} \sum_{i=1}^{n} \lambda_i \|\xbf-\xbf_i\|^{2}_{2}$ where $\lambda_i\geq0$ and $\sum_i \lambda_i=1$. The barycenter vector is then given by $\xbf=\sum_{i=1}^{n} \lambda_i \xbf_i$. This can be generalized to arbitrary metric spaces $(\Xcal,d)$ using the so-called Fréchet (or Karcher) mean \cite{karcher2014riemannian}:
\begin{equation}
\label{eq:general_frechet}
\inf_{x \in \Xcal} \sum_{i=1}^{n} \lambda_i d(x,x_i)^{p}
\end{equation} 
for $p\in \mathbb{N}$. The problem \eqref{eq:general_frechet} motivates the use of Wasserstein barycenter by considering the metric space $(\P_p(\Omega),\wass_p)$. Generally \eqref{eq:general_frechet} is non-convex and difficult to solve for arbitrary metric space, however in the case of the Wasserstein distance the situation is somehow easier since it can be formulated as a convex problem for which existence can be proved and efficient numerical solvers exist. For a set of input probability measures $(\nu_i)_{i \in \integ{n}} \in \P(\Omega)$ the Wasserstein barycenter reads as the following variational problem:
\begin{equation}
\label{eq:wass_bary}
\inf_{\mu \in \P(\Omega)} \sum_{i=1}^{n} \lambda_i \ot_{c}(\mu,\nu_i).
\end{equation} 
The barycentric formulation finds many applications in machine learning such in Bayesian inference \cite{pmlr-v38-srivastava15}, fairness \cite{pmlr-v97-gordaliza19a}, in image processing for texture synthesis and mixing \cite{rabin_texture} or in neuroimaging \cite{gramfort_2015} to name a few.
As proven in \cite{agueh2011barycenters} in the context of $\wass_2^{2}$ for $\Omega=\R^{d}$ this problem is convex and when one of the input measure has a density the barycenter is well-defined and unique. Even though there exist special cases (see Section 9.2 in \cite{cot_peyre_cutu}) in practice finding a solution in the general setting is difficult. In the following we detail one solution for the scenario where the input measures are discrete. More formally let $(\nu_i)_{i=1}^{n}$ be discrete probability measures with weights $\b_i \in \simplex_{n_i}$ and that are supported on $\mathbf{Y}_{i}=(\ybf^{i}_q)_{q\in \integ{n_i}} \in \R^{n_{i}\times d}$ for each $i \in \integ{n}$. Instead of looking at all possible discrete probability measures we can search a $k$ atoms probability measure \ie\ of the form $\hat{\mu}=\sum_{p=1}^{k}a_p \delta_{\xbf_{p}}$ where $\X=(\xbf_{p})_{p\in \integ{k}} \in \R^{k\times d}$ and $\a \in \simplex_k$. Overall the resulting problem is:

\begin{equation}
\label{eq:bary_dis}
\begin{split}
\min_{\a \in \simplex_k, \X \in \R^{k\times d}} \sum_{i=1}^{n} \lambda_i \min_{\GG \in \couplingset(\a,\b_i)} \froeb{\GG}{\mathbf{C}_{\X \Y_i}}= \min_{\begin{smallmatrix} \a \in \simplex_k, \X \in \R^{k\times d} \\ \forall i \in \integ{n}, \GG_i \in \couplingset(\a,\b_i) \end{smallmatrix}} \sum_{i=1}^{n} \lambda_i \froeb{\GG_i}{\mathbf{C}_{\X \Y_i}}
\end{split}
\end{equation}

where $\mathbf{C}_{\X \Y_i} \in \R^{k \times n_i}$ is the matrix defined by all pair to pair costs between the points of the barycenter and $\nu_i$, \ie\ $\mathbf{C}_{XY_i}=(c(\xbf_p,\ybf^{i}_q))_{p,q \in \integ{k}\times \integ{n_i}}$. In \cite{pmlr-v32-cuturi14} author propose to solve \eqref{eq:bary_dis} using Block Coordinate Descent (BCD) that alternates between minimizing \textit{w.r.t.} $\a,\X$ and $\GG_i$ while keeping others fixed:
\begin{enumerate}[label=(\roman*)]
\item The minimization \textit{w.r.t.} all $\GG_i$ with $\a, \X$ fixed involves solving $n$ OT problems which can be done using algorithms described in Section \ref{sec:numerical_tour}.
\item The minimization \textit{w.r.t.} $\X$ with $\a, \GG_i$ fixed can be performed in closed-form in the case $\Omega=\R^{d}$ and $c(\xbf,\ybf)=\|\xbf-\ybf\|^{2}_{2}$ \cite[Equation 8]{pmlr-v32-cuturi14}:
\begin{equation}
\X=\diag\left(\frac{1}{\a}\right) \left(\sum_{i=1}^{n} \lambda_{i}  \GG_{i} \mathbf{Y}_{i}\right) 
\end{equation}
\item The minimization \textit{w.r.t.} the weight $\a$ with $\X, \GG_i$ fixed relies on the optimal dual variables of all OT sub-problems of step (i) and applies a projected subgradient minimization \textit{w.r.t.} $\a$ as described in Algortihm 1 in \cite{pmlr-v32-cuturi14}.
\end{enumerate}
These three steps are repeated until convergence of $\X$ and $\a$. The major bottleneck of this approach is its computational complexity which is driven by the calculation of many OT problems. When the support $\X$ is fixed and by denoting $\mathbf{C}_{\X \Y_i}=\mathbf{C}_{i}$ the problem reduces to:
\begin{equation}
\label{eq:easy_bary}
\min_{\begin{smallmatrix} \a \in \simplex_k \\ \forall i \in \integ{n}, \GG_i \in \couplingset(\a,\b_i) \end{smallmatrix}} \sum_{i=1}^{n} \lambda_i \langle \GG_i,\mathbf{C}_{i} \rangle
\end{equation}
which is an (LP) with $kn^2+n$ variables and $2Nn$ constraints. Note that first order methods such as subgradient descent on the dual have been proposed in \cite{carlier:hal-00987292} to solve \eqref{eq:easy_bary} but in general its scale forbids the use generic solvers even for medium scale problems.
These remarks advocate for the use of entropic regularized OT to obtain fast and smooth approximations of the original barycenter problem as given by:
\begin{equation}
\label{eq:easy_bary_epsi}
\min_{\begin{smallmatrix} \a \in \simplex_k \\ \forall i \in \integ{n}, \GG_i \in \couplingset(\a,\b_i) \end{smallmatrix}} \sum_{i=1}^{n} \lambda_i \froeb{\GG_i}{\mathbf{C}_{i}}-\epsilon H(\GG_i)
\end{equation}
The resulting problem is a smooth convex minimization problem, which can be tackled using gradient descent \cite{pmlr-v32-cuturi14} or with descent method on the semi-dual \cite{Cuturi2018SemidualRO}. Another possibility is to rewrite \eqref{eq:easy_bary_epsi} as a the following weighted KL minimization problem \cite{benamou:2015}:
 \begin{equation}
\min_{\begin{smallmatrix}(\GG_i)_i \\ \forall i \in \integ{n}, \GG_i^{T}\one=\b_i \\ \GG_1 \one=\cdots=\GG_n \one \end{smallmatrix}} \sum_{i=1}^{n} \lambda_i \epsilon KL(\GG_i|\mathbf{K}_{i})
\end{equation}
where $\mathbf{K}_{i}=e^{-\frac{\mathbf{C}_i}{\epsilon}}$. In this formulation the barycenter $\a$ is encoded in the row marginals of all the couplings $\GG_i$ such that $\a=\GG_1 \one=\cdots=\GG_n \one$. It is shown in \cite{benamou:2015} that this problem can also be solved using a generalized Sinkhorn algorithm which involves iterative projections. As such the entropic regularization is quite suited for the barycenter problem and was further analyzed for the general case of continuous probability measures in \cite{bigot:hal-01790015,bigot_2019}. Note that other methods have been proposed which rely \textit{e.g.} on the sliced Wasserstein formulation \cite{bonneel:hal-00881872}, unbalanced formulation \cite{chizat_unbalanced} or on convolutions for geometric domains \cite{solomon_convo}.

\paragraph{The case $n=2$: McCann interpolant} One special case deserves attention that is when $n=2$ and in the case $\Omega=\R^{d}$ equipped with $\|.\|_{2}$. This setting corresponds to the so-called \emph{McCann interpolant} \cite{McCann_1997} where one wants to find:
\begin{equation}
\label{eq:mccan_interp}
\inf_{\mu \in \P(\R^{d})} (1-t) \wass_{2}^{2}(\mu,\nu_1)+t \wass_{2}^{2}(\mu,\nu_2)
\end{equation}
with $t\in [0,1]$ and $\nu_1$ is regular with respect to the Lebesgue measure. Using Brenier theorem we know that there exists a unique push-forward such that $T\#\nu_1=\nu_2$. In this case the barycenter is unique and obtained with $\mu_t=((1-t)id+ tT)\#\nu_1$. In practice when the probability measures $\nu_1,\nu_2$ are discrete with respectively $n$ and $m$ atoms this interpolant can be computed by $\mu_{t}=\sum_{i=1,j=1}^{m,n}\pi_{ij}^{*}\delta_{(1-t)\xbf_{i}+t\ybf_{j}}$ where $\pi^{*}$ is an optimal coupling between $\nu_1,\nu_2$.

\section{The Gromov-Wasserstein problem}

\subsection{Problem statement}

Despite its valuable properties the linear OT problem faces the challenging problem of probability measures whose supports lie in incomparable spaces, that is to say when $\Xcal,\Ycal$ are not part of a common ground metric space. For example when $\mu \in \P(\R^{3}), \nu \in \P(\R^{2})$ the definition of a meaningful cost $c : \R^{3} \times \R^{2} \rightarrow \R_{+}$ is not straightforward. In particular in this setting we can not define a distance between $\xbf,\ybf \in \R^{3} \times \R^{2}$ so that the Wasserstein distance can no longer be defined. Moreover the Wasserstein distance is not invariant to important families of invariants, such translations or rotations or more generally \emph{isometries} which is an important flaw of linear OT for certain applications such as shape matching.

The Gromov-Wasserstein (GW) framework is an elegant remedy for this situation. It is built upon a quadratic Optimal Transport problem, as opposed to a linear one for the linear OT problem, and, informally its  optimal value quantifies the metric distortion when transporting points from one space to another. This section aims at presenting the GW problem, its fundamental metric properties as well as numerical solvers. We refer the reader to \cite{Sturm2012,memoli_gw,chowdhury_gromovwasserstein_2019} for further readings. 

\begin{figure}[t]
   \begin{center}
        \includegraphics[width=0.5\linewidth]{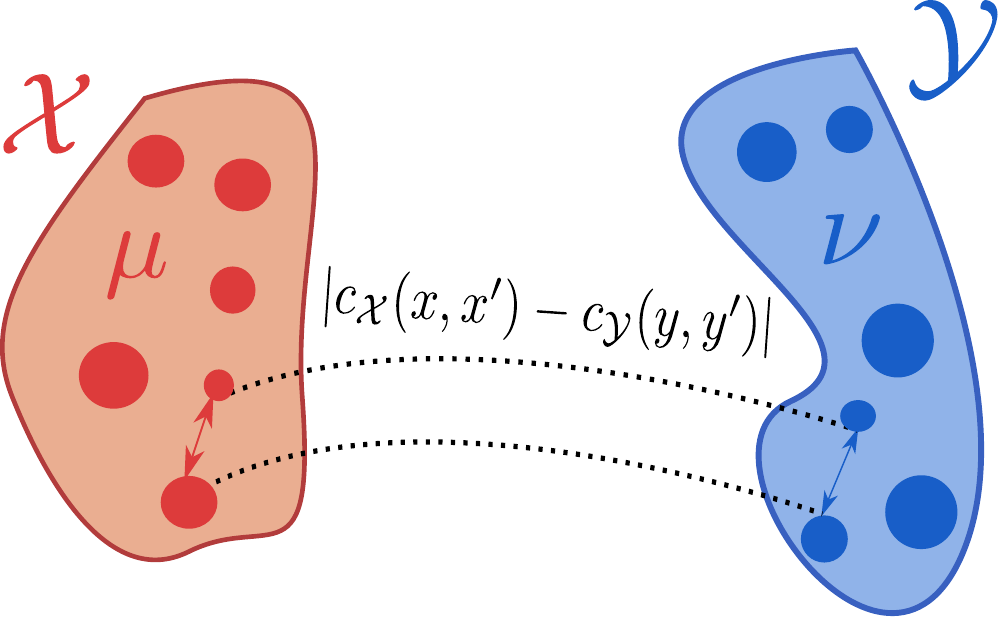}
    \caption{\label{fig:gw_def} The GW problem considers two probability measures $\mu \in \P(\Xcal), \nu \in \P(\Ycal)$ over two spaces that do not necessarily share a common metric. It is built upon the similarities $c_{\Xcal},c_{\Ycal}$ \emph{within} each space and on a measure of the distortion between each pair of points $\big|c_{\Xcal}(x,x')-c_{\Ycal}(y,y') \big|$.}
     \end{center}
\end{figure}

We consider two polish spaces $(\Xcal,d_{\Xcal}),(\Ycal,d_{\Ycal})$. Let $c_{\Xcal}:\Xcal \times \Xcal \rightarrow \R$ and $c_{\Ycal}:\Ycal \times \Ycal \rightarrow \R$ be continuous measurable functions and $\mu \in \P(\Xcal),\nu \in \P(\Ycal)$ be probability measures on $\Xcal,\Ycal$. The Gromov-Wasserstein ($\gw$) problem aims at finding:
\begin{equation}
\label{gw}
\gw_{p}(c_{\Xcal},c_{\Ycal},\mu,\nu)= \underset{\pi \in \couplingset(\mu,\nu)}{\inf} \left(\int_{\Xcal \times \Ycal} \int_{\Xcal \times \Ycal} \big| c_{\Xcal}(x,x')-c_{\Ycal}(y,y') \big|^{p} \dr \pi(x,y)\dr \pi(x',y'))\right)^{\frac{1}{p}}  \\
\end{equation}

for $p\in \mathbb{N}^{*}$ (see Figure \ref{fig:gw_def}). $\gw$ depends on the choice of similarities $c_\Xcal,c_\Ycal$ between points in $\Xcal$ and $\Ycal$. When it is clear form the context we will simply note $\gw_{p}(\mu,\nu)$ instead of $\gw_{p}(c_{\Xcal},c_{\Ycal},\mu,\nu)$. Since $\Xcal,\Ycal$ are already endowed with a natural metric one choice would be to consider $c_{\Xcal}=d_{\Xcal}$ and $c_{\Ycal}=d_{\Ycal}$. This setting brings in light the notion of \emph{metric measure spaces}, as triplets of the form $(\Xcal,d_\Xcal,\mu)$ where $(\Xcal,d_\Xcal)$ is a complete separable metric space and $\mu$ is a Borel probability measure on $\Xcal$. It was studied in depth in \cite{Sturm2012}. Another possibility is to consider triplets $(\Xcal,c_\Xcal,\mu)$ where $c_\Xcal$ is a integrable function, this notion refers to \emph{measure networks} and was studied in \cite{chowdhury_gromovwasserstein_2019}. 

The GW objective is constructed so that if an optimal coupling $\pi$ maps $x$ to $y$ and $x'$ to $y'$ then the couple $(x,x')$ should be ``as similar'' in $\Xcal$ as $(y,y')$ in $\Ycal$. When $c_\Xcal,c_\Ycal$ are distances it implies that $x,x'$ are as close in $\Xcal$ as $y,y'$ in $\Ycal$. In this work we consider a general setting where $c_{\Xcal},c_{\Ycal}$ are continuous and $\Xcal,\Ycal$ are Polish spaces and we will detail the two previous settings. 

As for the linear OT problem the equation \eqref{gw} always admits a solution. To show that we define $L(x,x',y,y')=\big| c_{\Xcal}(x,x')-c_{\Ycal}(y,y') \big|$. If $\Pi(\mu,\nu)$ is compact and the functionnal $\pi \rightarrow \int \int L \dr \pi \dr \pi$ is l.s.c. for the weak-convergence, Weierstrass theorem (see Memo \ref{memo:weir}) proves that the infimum will be attained at some optimal coupling. The first condition is a well-known result in OT theory provided that $\Xcal,\Ycal$ are Polish spaces \cite[Theorem 1.7]{San15a}. For the lower semi-continuity \textit{w.r.t.} the weak-convergence we can show that it suffices that $L$ be itself l.s.c. using the following lemma:

\begin{figure*}[!b]
\begin{memo}[Weierstrass theorem]
\label{memo:weir}
The Weierstrass theorem states that if $f: \Xcal \rightarrow \R\cup{+\infty}$ is l.s.c. and $\Xcal$ is compact then there exists $x^{*}=\inf_{x\in \Xcal} f(x)$ (see box 1.1 in \cite{San15a}). 
\end{memo}
\end{figure*}

\begin{lemma}
\label{lsc_on_measure}
Let $\Omega$ be a Polish space. If $f:\Omega \times \Omega \rightarrow \mathbb{R}_{+} \cup \{+\infty\}$ is lower semi-continuous, then the functional $J:\Pcal(\Omega)\rightarrow \mathbb{R} \cup \{+\infty\}$ with $J(\mu)=\int \int f(w,w') \dr\mu(w) \dr\mu(w')$ is l.s.c. for the weak convergence of measures.
\end{lemma}
\begin{proof}
Since $f$ is l.s.c. and bounded from below by $0$ we can consider $(f_{k})_{k}$ a sequence of continuous and bounded functions converging increasingly to $f$ (see e.g \cite{San15a}). By the monotone convergence theorem $J_{k}(\mu) \rightarrow J(\mu)\stackrel{def}{=}\sup_{k} J_{k}(\mu)=\sup_{k} \int \int f_{k} d\mu d\mu$. Moreover every $J_{k}$ is continuous for the weak convergence. Using theorem 2.8 \cite{billing} on the Polish space $\Omega\times \Omega$ we know that if $\mu_{n}$ converges weakly to $\mu$ then the product measure $\mu_{n}\otimes \mu_{n}$ converges weakly to $\mu\otimes \mu$. In this way $\lim_{n\rightarrow \infty}J_{k}(\mu_n)=J_{k}(\mu)$ since $f_{k}$ are continuous and bounded. In particular every $J_{k}$ is l.s.c. We can conclude that $J$ is l.s.c. as the supremum of l.s.c. functionals on the metric space of $(\Pcal(\Omega),\delta)$ (see \emph{e.g.} \cite{San15a}). Here we equipped $\Pcal(\Omega)$ with a metric $\delta$ as \textit{e.g.} $\delta(\mu,\nu)=\sum_{k=1}^{\infty} 2^{-k}|\int_\Omega f_k d\mu-\int_\Omega f_k d\nu|$ (see remark 5.11 in \cite{ambrosio2005gradient}).
\end{proof}

Overall since $L$ is l.s.c. due to the continuity of $c_\Xcal,c_\Ycal$ we can apply Lemma \ref{lsc_on_measure} with $\Omega=\Xcal \times \Ycal$ and conclude that $\pi \rightarrow \int \int L \dr \pi \dr \pi$ is l.s.c. for the weak-convergence and by the means of Weierstrass theorem equation \eqref{gw} always admits a minimizer.

\begin{Remark} As a consequence of our formulation the resulting $\gw$ cost main be infinite. A simple condition to remedy this possibility would be $\int \int \big| c_{\Xcal}(x,x')-c_{\Ycal}(y,y') \big|^{p} \dr(\mu\otimes\mu)(x,x')\dr(\nu\otimes\nu)(y,y')<\infty$. Since: 
\begin{equation}
\begin{split}
\big| c_{\Xcal}(x,x')-c_{\Ycal}(y,y') \big|^{p}&\leq \left(|c_{\Xcal}(x,x')|+|c_{\Ycal}(y,y')|\right)^{p} \stackrel{*}{\leq} 2^{p-1} (|c_{\Xcal}(x,x')|^{p} + |c_{\Ycal}(y,y')|^{p})
\end{split}
\end{equation} 
if $c_{\Xcal},c_{\Ycal}$ are $p$-integrable functions, \ie\ $c_{\Xcal} \in L^{p}(\mu \otimes \mu)$ and $c_{\Ycal} \in L^{p}(\nu \otimes \nu)$ then the cost is finite (we used Hölder's inequality in (*) see Memo \ref{memo:holder})
\end{Remark}

\begin{figure*}[!b]
\begin{memo}[Hölder's inequality]
\label{memo:holder}
Let $(\Xcal,\mu)$ be a measurable space and $(f,g)\in L^{p}(\mu) \times L^{q}(\mu)$ with $p,q >0$ verifying $\frac{1}{p}+\frac{1}{q}=1$. The Hölder's inequality states:
\begin{equation}
\int |fg|\dr \mu \leq (\int |f|^{p}\dr \mu)^{\frac{1}{p}} (\int |g|^{q}\dr \mu)^{\frac{1}{q}}
\end{equation} 
As a corollary for $q \geq 1$ we have:
\begin{equation}
\label{holder}
\forall x,y \in \mathbb{R}_{+}, \ (x+y)^{q} \leq 2^{q-1}(x^{q} + y^{q}).
\end{equation}
Indeed, if $q>1$:

$(x+y)^{q} = \big((\frac{1}{2^{q-1}})^{\frac{1}{q}} \frac{x}{(\frac{1}{2^{q-1}})^{\frac{1}{q}}} + (\frac{1}{2^{q-1}})^{\frac{1}{q}} \frac{y}{(\frac{1}{2^{q-1}})^{\frac{1}{q}}}\big)^{q} \leq \big[ (\frac{1}{2^{q-1}})^{\frac{1}{q-1}} +(\frac{1}{2^{q-1}})^{\frac{1}{q-1}}\big]^{q-1} \big(\frac{x^{q}}{\frac{1}{2^{q-1}}} + \frac{y^{q}}{\frac{1}{2^{q-1}}}\big) \\
= \frac{x^{q}}{\frac{1}{2^{q-1}}} + \frac{y^{q}}{\frac{1}{2^{q-1}}}.$

Last inequality is a consequence of Hölder's inequality. The result remains valid for $q=1$.

\end{memo}

\end{figure*}

\begin{figure}[t]
   \begin{center}
        \includegraphics[width=0.5\linewidth]{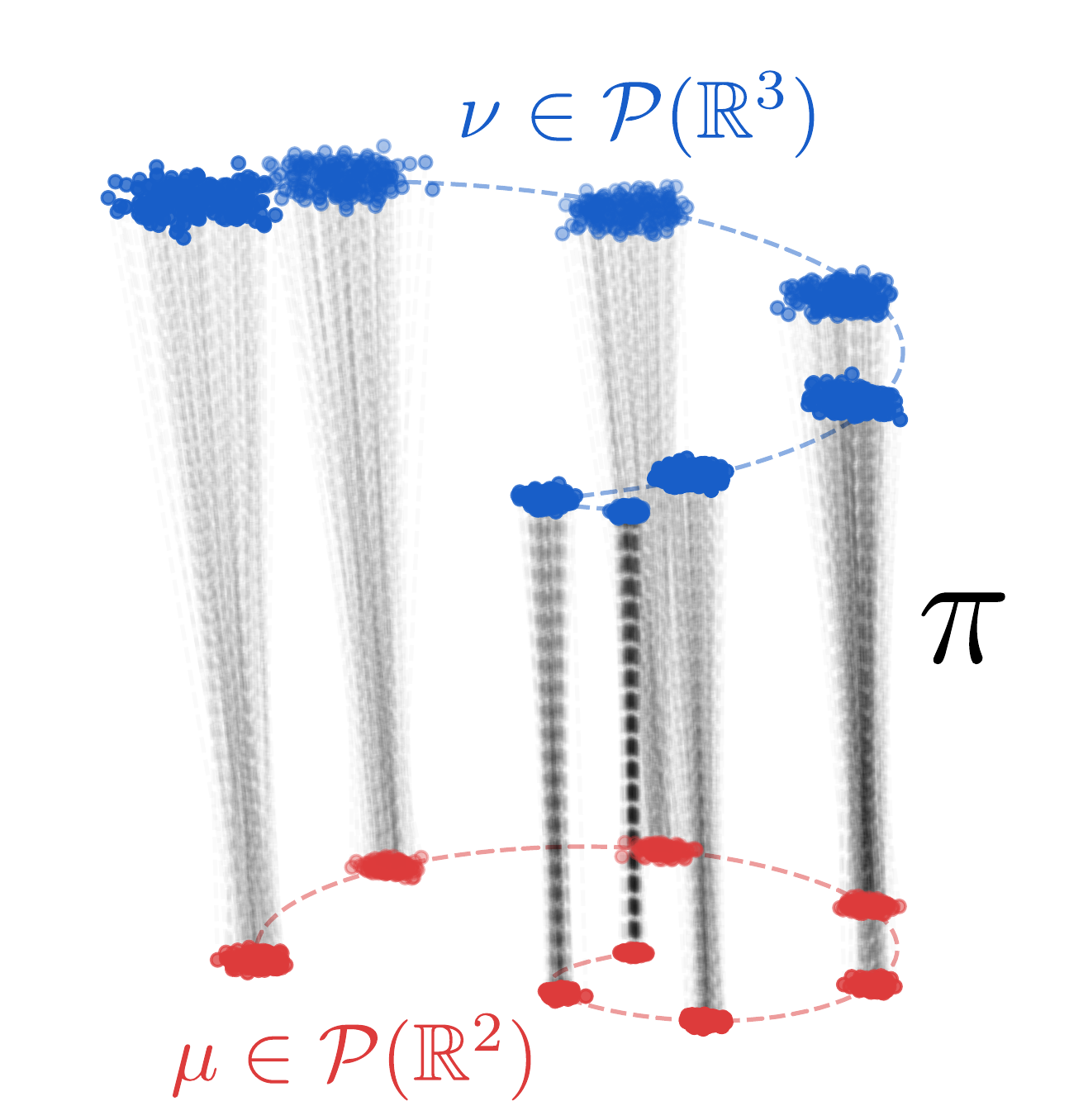}
    \caption{\label{fig:spiral_gw} GW problem between two discrete probability measures $\mu \in \P(\R^{2}),\nu \in \P(\R^{3})$. The optimal coupling $\pi$ is depicted in dashed lines. It associates points so as to minimize the distortion between all pair-to-pair distances \emph{within} the support of each measures. }
     \end{center}
\end{figure}

\paragraph{GW for measures on Euclidean spaces} One special case will be consider in Chapter \ref{cha:gw_euclidean}, that is when $\Xcal=\R^{p}$, $\Ycal=\R^{q}$ and $\mu \in \P(\R^{p}),\nu \in \P(\R^{q})$ with $p$ not necessarily equal to $q$. We can consider the GW problem with the standard Euclidean distances for $c_\Xcal,c_\Ycal$ on respectively $\R^{p},\R^{q}$. This setting illustrates the invariance property of the GW problem \textit{w.r.t.} rotations and translations. More precisely let $\mathbf{O}\in \mathcal{O}(p)$ and $\xbf_0 \in \R^{p}$ associated with $T(\xbf)=\mathbf{O}\xbf + \xbf_0$. Then the GW problem is invariant by $T$ that is $\gw_p^{p}(T\#\mu,\nu)=\gw_p^{p}(\mu,\nu)$ (same applies for $\nu$). To see that we simply used for all $\xbf,\xbf'$ $c_{\Xcal}(T(\xbf),T(\xbf'))=\|\mathbf{O}\xbf+\xbf_0-\mathbf{O}\xbf'-\xbf_0\|_2=\|\mathbf{O}(\xbf-\xbf')\|_2=\|\xbf-\xbf'\|_2$ since $\mathbf{O}\in \mathcal{O}(p)$. This property will be generalized for any metric space by considering the notion of \emph{isometry} and contrasts with the Wasserstein distance which is not invariant neither to translation or rotation of the support of one probability measure.

\begin{Example}
As a first illustration of the GW problem we consider two discrete probability measures $\mu \in \P(\R^{2}),\nu \in \P(\R^{3})$ following respectively a spiral in $\R^{2}$, $\R^{3}$ and composed of a mixture of Gaussian distributions. $c_\Xcal,c_\Ycal$ are defined by the Euclidean distances between the points. We compute an optimal coupling of the GW problem using the FW solver presented in Chapter \ref{cha:fgw} (see Section \ref{sec:solving_gw} for more details). The result is depicted in Figure \ref{fig:spiral_gw}.
\end{Example}

\subsection{Properties of GW \label{sec:prop_of_gw}}

One of the main property of the $\gw$ problem is that it allows for comparing probability measures whose supports dwell in different, potentially non-related, spaces by defining a notion of equivalence of two probability distributions in this case. This is made possible thanks to the concepts of \emph{isometry} and \emph{isomorphism}.

\begin{definition}[Isometry\label{isometrydef}]
Let $(\Xcal,d_{\Xcal})$ and $(\Ycal,d_{\Ycal})$ be two metric spaces. An isometry is a sujective map $\phi : \Xcal \rightarrow \Ycal$ that preserves the distances:
\begin{equation}
\label{isometryproperty2}
\forall x,x' \in X, d_{\Ycal}(\phi(x),\phi(x'))=d_{\Xcal}(x,x').
\end{equation}

\end{definition}

An isometry is necessarily bijective, since for $\phi(x)=\phi(x')$ we have $d_{\Ycal}(\phi(x),\phi(x'))=0=d_{\Xcal}(x,x')$ and hence $x=x'$ (in the same way $\phi^{-1}$ is also a isometry).  When it exists, $\Xcal$ and $\Ycal$ share the same "size" and any statement about $\Xcal$ which can be expressed through its distance is transported to $\Ycal$ by the isometry $\phi$.

\begin{Example}
Let us consider the two following graphs whose discrete metric spaces are obtained as shortest path between the vertices (see corresponding graphs in Figure \ref{isometric_spaces}). $$\small\begin{pmatrix}x_{1}\\x_{2}\\x_{3}\\x_{4}\end{pmatrix},{\underbrace{\begin{pmatrix}0&1&1&1 \\
    1&0&1&2\\
    1&1&0&2\\
    1&2&2&0
    \end{pmatrix}}_{d_\Xcal(x_i,x_j)}} \quad\text{ and }\quad \begin{pmatrix}y_{1}\\y_{2}\\y_{3}\\y_{4}\end{pmatrix},{\small\underbrace{\begin{pmatrix}0&1&1&1 \\
        1&0&2&2\\
        1&2&0&1\\
        1&2&1&0
    \end{pmatrix}}_{d_\Ycal(y_i,y_j)}}.$$  These spaces are isometric since the surjective map $\phi$ such that $\phi(x_{1})=y_{1}$, $\phi(x_{2})=y_{3}$, $\phi(x_{3})=y_{4}$, $\phi(x_{4})=y_{2}$ verifies equation \eqref{isometryproperty2}.
\end{Example}
    \tikzstyle{vertex1}=[circle,fill=black,minimum size=6pt,inner sep=0pt]
\tikzstyle{vertex2}=[circle,fill=black,minimum size=6pt,inner sep=0pt]
\tikzstyle{vertex3}=[circle,fill=black,minimum size=6pt,inner sep=0pt]
\tikzstyle{vertex4}=[circle,fill=black,minimum size=6pt,inner sep=0pt]

    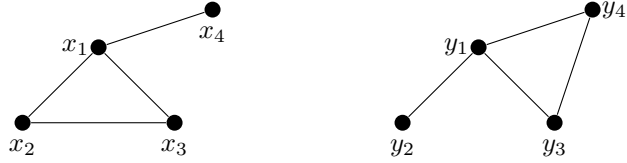
\begin{figure}[t]
    \centering
\tikzstyle{edge} = [draw]
\begin{tikzpicture}[scale=1, auto,swap]
    \foreach \pos/\name in {{(0,0)/a}}
            \node[vertex1] (\name) at \pos {};
    \foreach \pos/\name in {{(2,0)/b}}
            \node[vertex2] (\name) at \pos {};
    \foreach \pos/\name in {{(1,1)/c}}
            \node[vertex3] (\name) at \pos {};
    \foreach \pos/\name in {{(2.5,1.5)/d}}
            \node[vertex4] (\name) at \pos {};
\def\x{5}

    \foreach \pos/\name in {{(0+\x,0)/e}}
            \node[vertex1] (\name) at \pos {};
    \foreach \pos/\name in {{(2+\x,0)/f}}
            \node[vertex2] (\name) at \pos {};
    \foreach \pos/\name in {{(1+\x,1)/g}}
            \node[vertex3] (\name) at \pos {};
    \foreach \pos/\name in {{(2.5+\x,1.5)/h}}
            \node[vertex4] (\name) at \pos {};

    \foreach \source/ \dest in {b/a, c/a, c/b, c/d}
        \path[edge] (\source) -- (\dest);

    \foreach \source/ \dest in {e/g, g/f, g/h, h/f}
        \path[edge] (\source) -- (\dest);

\def\y{0.1}
\foreach \pos/ \name in {{(0,00-\y)/x_2}, {(2,00-\y)/x_3}, {(2.5,1.50-\y)/x_4}}
  \draw \pos node[below, scale = 1]{$\name$};
\draw (1,1) node[left, scale = 1]{$x_1$};

\foreach \pos/ \name in {{(0+\x,0-\y)/y_2}, {(2+\x,00-\y)/y_3}}
  \draw \pos node[below, scale = 1]{$\name$};
\draw (1+\x,1) node[left, scale = 1]{$y_1$};
\draw (2.5+\x,1.5) node[right, scale = 1]{$y_4$};

 \end{tikzpicture}
  \caption{Two isometric metric spaces. Distances between the nodes are given by the shortest path, and the weight of each edge is equal to 1.\label{isometric_spaces}}
     \end{figure}

Another natural and straightforward example is two point clouds rotated from each other. More precisely if we consider $(\xbf_{i})_{i\in \integ{n}}$, $(\ybf_{i})_{i\in \integ{n}}$ where $\xbf_{i},\ybf_{i} \in \R^{p} \times \R^{p}$ equipped with the Euclidean norm $\|.\|_{2}$. Suppose that there exists a orthogonal matrix $\Obf \in \mathcal{O}(p)$ such that $\ybf_i=\Obf \xbf_i$ for all $i \in \integ{n}$ (with a slight abuse of notations we identify the matrix with its linear application). Then for all $(i,j) \in \integ{n}^{2}$ we have: 
\begin{equation}
\|\ybf_i-\ybf_j\|_2=\|\Obf \xbf_i-\Obf \xbf_j\|_2=\|\Obf (\xbf_i-\xbf_j)\|_2=\|\xbf_i-\xbf_j\|_2
\end{equation} 
since $\Obf \in \mathcal{O}(p)$ so that $\Xcal=(\xbf_{i})_{i\in \integ{n}}$ and $\Ycal=(\ybf_{i})_{i\in \integ{n}}$ are isometric.

This notion can be enriched in order to take into account the measures, which results in the notion of \emph{strong isomorphism}:

\begin{definition}[Strong isomorphism]
\label{strong_iso}
Let $(\Xcal,d_{\Xcal}),(\Ycal,d_{\Ycal})$ be Polish spaces and $\mu \in \P(\Xcal),\nu \in \P(\Ycal)$. We say that $(\Xcal,d_{\Xcal},\mu)$ is \emph{ strongly isomorphic} to $(\Ycal,d_{\Ycal},\nu)$ if there exists a bijection $\phi: \text{supp}(\mu) \rightarrow \text{supp}(\nu)$ such that:
\begin{enumerate}[label=\roman*]
\item $\phi$ is an isometry, \textit{i.e.} $d_{\Ycal}(\phi(x),\phi(x'))=d_{\Xcal}(x,x')$ for $x,x' \in \text{supp}(\mu)^{2}$
\item $\phi$ pushes $\mu$ forward to $\nu$,  \textit{i.e.} $\phi \# \mu=\nu$ 
\end{enumerate}
When it is clear from the context we will simply say that $\mu$ is strongly isomorphic to $\nu$ when previous conditions are satisfied.
\end{definition}

\begin{Example}
    Let us consider two mm-spaces $(\Xcal=\{x_1,x_2\},d_{\Xcal}=\{1\},\mu=\{\frac{1}{2},\frac{1}{2}\})$ and $(\Ycal=\{y_1,y_2\},d_{\Ycal}=\{1\},\nu=\{\frac{1}{4},\frac{3}{4}\})$ as depicted in Figure \ref{isometric_not_isomoprhic}. These spaces are isometric but not isomorphic as there exists no measure preserving map which pushes $\mu$ forward to $\nu$

    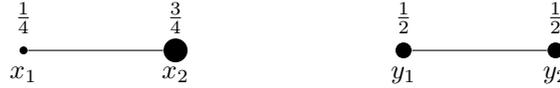
\begin{figure}[t]
    \centering
\tikzstyle{vertex1}=[circle,fill=black,minimum size=6pt,inner sep=0pt]
\tikzstyle{vertex2}=[circle,fill=black,minimum size=3pt,inner sep=0pt]
\tikzstyle{vertex3}=[circle,fill=black,minimum size=9pt,inner sep=0pt]

\tikzstyle{edge} = [draw]
\begin{tikzpicture}[scale=1, auto,swap]
    \foreach \pos/\name in {{(0,0)/a}}
            \node[vertex2] (\name) at \pos {};
    \foreach \pos/\name in {{(2,0)/b}}
            \node[vertex3] (\name) at \pos {};
\def\x{5}

    \foreach \pos/\name in {{(0+\x,0)/e}}
            \node[vertex1] (\name) at \pos {};
    \foreach \pos/\name in {{(2+\x,0)/f}}
            \node[vertex1] (\name) at \pos {};

    \foreach \source/ \dest in {b/a, e/f}
        \path[edge] (\source) -- (\dest);

\def\y{0.1}
\foreach \pos/ \name in {{(0,00-\y)/x_1}, {(2,00-\y)/x_2}, {(0+\x,0-\y)/y_1}, {(2+\x,00-\y)/y_2}}
  \draw \pos node[below, scale = 1]{$\name$};
\def\y{-0.1}
\foreach \pos/ \name in {{(0,00-\y)/\frac{1}{4}}, {(2,00-\y)/\frac{3}{4}}, {(0+\x,0-\y)/\frac{1}{2}}, {(2+\x,00-\y)/\frac{1}{2}}}
  \draw \pos node[above, scale = 1]{$\name$};

 \end{tikzpicture}
    \caption{Two isometric but not isomorphic spaces. \label{isometric_not_isomoprhic}}
    \end{figure}
\end{Example}

Another notion of isomorphism deserves attention especially when $c_\Xcal,c_\Ycal$ are not distances. In this case we will consider the following \emph{weak isomorphism} property:  

\begin{definition}[Weak isomorphism]
\label{weak_iso}
Let $\Xcal,\Ycal$ be Polish spaces and $\mu \in \P(\Xcal),\nu \in \P(\Ycal)$. We say that $(\Xcal,c_{\Xcal},\mu)$ is \emph{weakly isomorphic} to $(\Ycal,c_{\Ycal},\nu)$ if there exists $(\mathcal{Z},c_{\mathcal{Z}},m)$, with $\text{supp}(m)=\mathcal{Z}$ and maps $\phi_{0}: \mathcal{Z} \rightarrow \Xcal$,$\phi_{1}: \mathcal{Z} \rightarrow \Ycal$ such that:
\begin{enumerate}[label=\roman*]
\item $c_{\mathcal{Z}}(z,z')=c_{\Xcal}(\phi_{0}(x),\phi_{0}(x'))=c_{\Ycal}(\phi_{1}(x),\phi_{1}(x'))$ for $z,z' \in \mathcal{Z}^{2}$
\item $\phi_{0} \# m=\mu$ and $\phi_{1} \# m=\nu$ 
\end{enumerate}
When it is clear from the context we will simply say that $\mu$ is weakly isomorphic to $\nu$ when previous conditions are satisfied.
\end{definition}

The weak isomorphism brings to light a kind of ``tripod structure'' in which the isomorphism is defined trough a third space $\mathcal{Z}$. In fact both notions are equivalent when $c_\Xcal,c_\Ycal$ are distances as stated in the next proposition \cite[lemma 1.10]{Sturm2012}:

\begin{prop}
The spaces $(\Xcal,d_{\Xcal},\mu)$ and $(\Ycal,d_{\Ycal},\nu)$ are strongly isomorphic if and only if $(\Xcal,d_{\Xcal},\mu)$ and $(\Ycal,d_{\Ycal},\nu)$ are weakly isomorphic.
\end{prop}

However the weak-isomorphism property has its own interest when working with arbitrary similarity measures $c_\Xcal,c_\Ycal$. The following theorem is fundamental for GW and aims to unify the metric properties of $\gw$ given in \cite{Sturm2012,chowdhury_gromovwasserstein_2019}. It proves that GW defines a metric \textit{w.r.t.} the isomorphism notions:
\begin{theo}[Metric properties of $\gw$]
\label{metric_gw}
In the following $(\Xcal,d_{\Xcal}),(\Ycal,d_{\Ycal})$ are Polish spaces and $\mu \in \P(\Xcal),\nu \in \P(\Ycal)$
\begin{enumerate}[label=\roman*]
\item $\gw_{p}$ is symmetric, positive and satisfies the triangle inequality. More precisely for $(\Xcal,c_{\Xcal},\mu)$, $(\Ycal,c_{\Ycal},\nu)$, $(\Zcal,c_{\Zcal},m)$ we have:
\begin{equation}
\gw_{p}(c_{\Xcal},c_{\Ycal},\mu,\nu)\leq \gw_{p}(c_{\Xcal},c_{\Zcal},\mu,m)+\gw_{p}(c_{\Zcal},c_{\Ycal},m,\nu)
\end{equation}
\item $\gw_{p}(d_{\Xcal},d_{\Ycal},\mu,\nu)=0$ if and only if $(\Xcal,d_{\Xcal},\mu)$ and $(\Ycal,d_{\Ycal},\nu)$ are strongly isomorphic.
\item $\gw_{p}(c_{\Xcal},c_{\Ycal},\mu,\nu)=0$ if and only if $(\Xcal,c_{\Xcal},\mu)$ and $(\Ycal,c_{\Ycal},\nu)$ are weakly isomorphic.
\item More generally, for any $q\geq 1$, $\gw_{p}(d_{\Xcal}^{q},d_{\Ycal}^{q},\mu,\nu)=0$ if and only if $(\Xcal,d_{\Xcal},\mu)$ and $(\Ycal,d_{\Ycal},\nu)$ are strongly isomorphic.
\end{enumerate}
\end{theo}

\begin{proof}
For the first point (i) positiveness is straightforward. For the triangle inequality and symmetry see \cite[Theorem 16]{chowdhury_gromovwasserstein_2019}. For (ii) see \cite[Lemma 1.10]{Sturm2012} and for (iii) see \cite[Theorem 18]{chowdhury_gromovwasserstein_2019}. Note in the proof that for (iii) the result is still valid even if $c_\Xcal,c_\Ycal$ are not $p$-integrable. For (iv) see \cite[Lemma 9.2]{Sturm2012}. 
\end{proof}

This theorem can endow the space of all spaces of the form $(\Xcal,c_\Xcal,\mu)$ with a distance defined by GW which, however, requires the finiteness of $\gw$. More precisely:

\begin{definition}
Let $\Xcal$ be a Polish space, $\mu \in \P(\Xcal)$ and $c_{\Xcal}:\Xcal \rightarrow \R$ be measurable. We define the size $size_{p}$ of $\Xcal$, given $c_{\Xcal}$ and $\mu$ by $\text{size}_{p}(\Xcal,c_{\Xcal},\mu)=(\int c_{\Xcal}(x,x')^{p}\dr \mu(x)\dr \mu(x'))^{\frac{1}{p}}$
\\ \\
We define $\mathbb{X}_{p}$ be the space of all metric measure spaces with finite $L^{p}$-size, \textit{i.e} $\mathbb{X}_{p}=\{(\Xcal,d_{\Xcal},\mu)| \text{size}_{p}(\Xcal,d_{\Xcal},\mu)<+\infty\}$ where $(\Xcal,d_{\Xcal})$ is a Polish space and $\mu \in \P(\Xcal)$.
\\ \\
We define also $\mathbb{N}_{p}$ be the space of all network measure spaces \cite{chowdhury_gromovwasserstein_2019} with finite $L^{p}$-size, \textit{i.e.} $\mathbb{N}_{p}=\{(\Xcal,c_{\Xcal},\mu)| \text{size}_{p}(\Xcal,c_{\Xcal},\mu)<+\infty\}$ where $\Xcal$ is a Polish space, $\mu \in \P(\Xcal)$ and $c_{\Xcal}$ a continuous measurable function.

\end{definition}

The function $size_{p}$ quantifies somehow an average diameter of $\Xcal$ given a probability measure and a function $c_\Xcal$. Using this notion and Theorem \ref{metric_gw} we now state the main theorem about the metric properties of $\gw$:

\begin{theo}[$\gw$ is a distance]
$\gw_{p}$ is a distance on $\mathbb{X}_{p}$ quotiented by the strong isomorphisms. $\gw_{p}$ is a distance on $\mathbb{N}_{p}$ quotiented by the weak isomorphisms.
\end{theo}

This theorem has a lot of implications. It endows the space of all metric (network) measure spaces with a topology, a geometric structure, induced by Gromov-Wasserstein and, as such, allows the use of a wide family of geometric tools and a notion of convergence of metric measure spaces. Moreover it indicates that GW is well suited for comparing objects with respect to a large class of invariants that are for instance rotations, translations or permutations. This property is important \textit{e.g.} for shape comparison where the orientation of a shape does not define its nature or for graphs where any permutation of the nodes result in the same graph. It is sometimes valuable to have a notion of distance which is insensitive to these transformations so as to focus properly on what matters rather than to encode the invariance (see Remark \ref{rk:invariants}). Finally if GW vanishes it implies necessarily that the objects are isomorphic which is interesting for detecting such cases.  

Note also that GW is deeply connected to the Gromov-Hausdorff distance \cite{gromov_metric_1999} that aims at measuring how far are $(\Xcal,d_\Xcal)$ and $(\Ycal,d_\Ycal)$ from being isometric and can be used for studying the convergence of metric spaces \cite{burago_course_2001}. However computing this distance results in a highly non-convex optimization problem whose global solution is untractable. As shown in \cite{memoli_gw} the introduction of measures turns out to ``smooth'' the definition of the Gromov-Hausdorff distance and results in the Gromov-Wasserstein distance.

\begin{Remark}[Implicit or explicit encoding of the invariances]
\label{rk:invariants}
Let $\mu,\nu \in \Pcal(\R^{d})\times \Pcal(\R^{d})$ and $c: \R^{d} \times \R^{d} \rightarrow \R$ a cost. In \cite{alavarez:2019} authors propose the following OT problem:
\begin{equation}
InvOT(\mu,\nu)=\min_{\pi \in \Pi(\mu,\nu)} \min_{f \in F} \int c(\xbf,f(\ybf)) \dr \pi(\xbf,\ybf)
\end{equation}
where $F$ is a class of functions from $\R^{d}$ to $\R^{d}$ aiming at encoding a global transformation of the features. For example $F$ can be defined as $\mathcal{O}(d)$ the set of orthogonal transformations or any linear transformation with bounded Schatten norm (see \cite{alavarez:2019} or Chapter \ref{cha:gw_euclidean} for more details). When considering $F=\mathcal{O}(d)$ the resulting $InvOT$ becomes invariant by rotation of the support of the target measure. It could be interesting in a setting where one wants to match two distributions modulo a rotation as \textit{e.g.} in unsupervised word translation where word embedding algorithms are known to produce vectors intrinsically invariant to angle \cite{alavarez:2019,Grave2018UnsupervisedAO}. This approach can be put into perspective with the $GW$ distance as, in the $\gw$ case, one makes the implicit assumption, or prior, that the invariances are the isometric transformations of the data, whereas in the $InvOT$ approach one makes the prior that we know somehow which class of invariances is of interest for the problem and we encode it into the loss. Both approach are relevant and related (see Chapter \ref{cha:gw_euclidean}): if we have another prior than isometric transformations it is maybe more suited to encode it directly \textit{via} $F$ and if we know that isometries are relevant we can directly build upon $\gw$.
\end{Remark}

\paragraph{Geodesics and GW interpolation} The space of all mm-spaces $\mathbb{X}_{p}$ endow with GW has also a nice geodesic structure which is important in order to derive dynamic formulations and gradient flows \cite{ambrosio2005gradient}. Informally in  $\mathbb{X}_{p}$ we can connect any two points (that are mm-spaces) with a curve that somehow represents the shortest path connecting these points. More precisely given two mm-spaces $(\Xcal,d_\Xcal,\mu),(\Ycal,d_\Ycal,\nu)$ the curve $t\in [0,1] \rightarrow (\Xcal \times \Ycal, d_t, \pi^{*})$ where $\pi^{*}$ is an optimal coupling of the GW problem between $\mu$ and $\nu$ and:
\begin{equation}
\forall (x,y), (x',y')\in (\Xcal \times \Ycal)^{2}, \ d_t((x,y), (x',y'))=(1-t)d_\Xcal(x,x')+td_\Ycal(y,y')
\end{equation}
is a geodesic \cite{Sturm2012}. However computing this geodesic is often intractable in practice since it implies the calculation of the cartesian product $\Xcal \times \Ycal$. One can rely instead on the barycenter formulation defined in Section \ref{sec:bary_gw}.

\subsection{Solving GW \label{sec:solving_gw}}
In this section we describe some numerical solutions to the GW problem. In the following $\mu=\sum_{i=1}^{n} a_i \delta_{x_i} \in \P(\Xcal)$, $\nu=\sum_{j=1}^{m} b_j \delta_{y_j} \in \P(\Ycal)$ are discrete probability measures over respectively $(\Xcal,d_{\Xcal})$, $(\Ycal,d_{\Ycal})$. We note also $\C_1,\C_2$ the matrices of pair-to-pair distances inside each space, \ie\ $\forall (i,k) \in \integ{n}^{2}, \ C_{1}(i,k)=d_{\Xcal}(x_i,x_k)$ and $\forall (j,l) \in \integ{m}^{2}, \ C_{2}(j,l)=d_{\Ycal}(y_j,y_l)$. The GW problem aims at solving:

\begin{equation}
\label{discrete_gw}
\gw_p^{p}(\C_1,\C_2,\a,\b)=\underset{\GG \in \couplingset(\a,\b)}{\min}\sum_{i,j,k,l} |C_{1}(i,k)-C_{2}(j,l)|^{p} \pi_{i,j}\pi_{k,l}= \underset{\GG \in \couplingset(\a,\b)}{\min}\froeb{\L(\C_{1},\C_{2})^{p} \otimes \GG}{\GG} \\
\end{equation}
where we define $L(\C_{1},\C_{2})$ as the tensor $\L(\C_{1},\C_{2})=(|C_{1}(i,k)-C_{2}(j,l)|)_{i,j,k,l}$ and $\otimes$ is the the tensor-matrix multiplication, \emph{i.e.} for a tensor $\L=(L_{i,j,k,l})$, $\L \otimes \GG$ is the matrix $\left(\sum_{k,l} L_{i,j,k,l}\pi_{k,l}\right)_{i,j}$. 

The optimization problem \eqref{discrete_gw} is a non-convex Quadratic Program (QP) which is NP-hard in general \cite{Loiola_survey_qap} and notoriously hard to approximate. When $p=2$, \ie\ when $\L=|.|^{2}$ equation \eqref{discrete_gw} can be recast as:

\begin{equation}
\underset{\GG \in \couplingset(\a,\b)}{\min} \tr(\mathbf{c}_{\C_1,\C_2}\GG^{T})-2\tr(\C_1\GG \C_2\GG^{T})
\end{equation}

where $\mathbf{c}_{\C_1,\C_2}=(\C_1)^{2}\a \one_{m}^{T}+\one_{n}\b^{T}(\C_2)^{2}$ (see Proposition 1 in \cite{peyre2016gromov}). In standard QP form this problem reads also:
\begin{equation}
\label{eq:graph_matching_gw}
\underset{\GG \in \couplingset(\a,\b)}{\min} \mathbf{c}^{T}\xbf(\GG) + \frac{1}{2} \xbf(\GG)^{T} \mathbf{Q} \xbf(\GG)
\end{equation}
where $\xbf(\GG)=\vec(\GG),\mathbf{c}=\vec(\mathbf{c}_{\C_1,\C_2})$ and $\mathbf{Q}=-4\C_1 \otimes_{K} \C_2 $ with $\kron$ the Kronecker product of matrices defined for two arbitrary matrices $\Abf \in \R^{n\times m},\Bbf^{p\times q}$ as $\Abf \kron \Bbf \in \R^{np \times mq}$ with $\Abf \kron \Bbf=(A_{i,j}\Bbf)_{i,j}$ (see Memo \ref{memo:kron}). Equation \eqref{eq:graph_matching_gw} is a non-convex QP as the Hessian $\mathbf{Q}$ is not positive semi-definite in general (its eigenvalues are the products of the eigeinvalues of $\C_1,\C_2$). 

\paragraph{Relations with Quadratic Assignment Problem and Graph Matching} The $\gw$ problem is very related to the so called \emph{Quadratic Assignment Problem} (QAP). This problem was first introduced by Koopmans and Beckmann \cite{koopmans} to model a plant location problem and plays today many roles in optimization. Given two matrices $\Abf=(a_{i,j})_{(i,j) \in \integ{n}^{2}}$ and $\Bbf=(b_{i,j})_{(i,j) \in \integ{n}^{2}}$ the standard form for the QAP reads:
\begin{equation}
\label{eq:the_standard_qap}
\min_{\sigma \in S_n} \sum_{i=1,j=1}^{n} a_{\sigma(i)\sigma(j)} b_{ij}
\end{equation}
The QAP can be understood as a \emph{facility location} problem: given $n$ facilities and $n$ locations, one wants to assign each facility to a location with $a_{ij}$ the flow of material moving from facility $i$ to facility $j$ and $b_{ij}$ the distance from facility $i$ to facility $j$. In this context the cost of \emph{simultaneously} locating facility $\sigma(i)$ to location $i$ and facility $\sigma(j)$ to location $j$ is $ a_{\sigma(i)\sigma(j)} b_{ij}$. In this model one wants to find the assignment that minimizes the overall cost of locating each facility. This problem was considered for example in \cite{hospital} for locating hospital departments so as to minimize the total distance traveled by patients but it also covers a large variety of applications such as scheduling \cite{schedule}, parallel and distributed computing \cite{computer_Bokhari} or balancing of turbine runners \cite{hydro1}. For a comprehensive survey on this topic we refer to \cite{cela_all,Loiola_survey_qap}. Unfortunately the QAP is NP-Hard in general and only few special cases are known to be solvable in polynomial time. That is the case for example when matrices $\Abf$ and $\Bbf$ have simple known structures, such as a diagonal structure and Toeplitz or separability properties such as $a_{i,j}=\alpha_{i}\alpha_{j}$ \cite{cela_new_2016,Cela-Schmuck-Wimer-Woeginger:2011,cela-2013}). The question of finding a polynomial time algorithm that solves the QAP when $\Abf$ and $\Bbf$ satisfy the Monge and the anti-Monge properties are, to the best of our knowledge, still open \cite{cela_all}. 

The QAP is intrinsically linked with the \emph{graph matching} problem whose literature is also extensive (see \textit{e.g.} \cite{berg,graph_matching_relax,Zaslavskiy_2009,NIPS2018_7323,caetano}). The graph matching problem refers to optimization problems where the goal is to match edge affinities of two graphs that are represented by symmetric matrices $\Abf=(a_{i,j})_{(i,j) \in \integ{n}^{2}}$ and $\Bbf=(b_{i,j})_{(i,j) \in \integ{n}^{2}}$. A common approach for these types of problem is to attempt to solve:
\begin{equation}
\label{eq:graph1}
\min_{\X \in \Pi_n} \|\Abf \X - \X \Bbf\|^{2}_{\F}
\end{equation} 
where $\Pi_n$ is the set of permutation matrices, \ie\ $\X=(x_{ij})_{(i,j) \in \integ{n}^{2}} \in \Pi_n$ if $x_{ij} \in \{0,1\}$ and $\sum_{i=1}^{n}x_{ij}=\sum_{j=1}^{n}x_{ij}=1$. By noticing that $\|\Abf \X - \X \Bbf\|^{2}_{\F}=\|\Abf \X\|_{\F}^{2}+\|\X \Bbf\|^{2}_{\F}-2\tr(\Abf \X \Bbf \X^{T})$ and that $\|\Abf \X\|_{\F}^{2}=\tr(\X^{T}\Abf^{T}\Abf \X)=\tr(\X \X^{T}\Abf^{T}\Abf)=\tr(\Abf^{T}\Abf)=\|\Abf\|_{\F}^{2}$ the problem \eqref{eq:graph1} is equivalent to:
\begin{equation}
\min_{\X \in \Pi_n} -\tr(\Abf \X \Bbf \X^{T})
\end{equation}
As such the graph matching problem is a QAP and consequently is NP-Hard in general. A way of finding a approximate solution is to consider a relaxation of the constraints by replacing $\Pi_n$ by its convex-hull, namely the set of doubly stochastic matrices $DS=\{\X \in \R^{n \times n}| \X \one_n=\one_n, \X^{T} \one_n=\one_n, \X \geq 0 \}$ (see \cite{dsplusplus,Bernard_tighter,Schellewald}). 

The previous discussion can relate the $\gw$ problem with the graph matching one. Indeed when we consider two discrete probability measures with the same number of atoms and with uniform weights (\ie\ $\a=\b=\one_n/n$) then $\gw$ is equivalent to the relaxation of the graph matching problem with affinity matrices $\C_1,\C_2$. To see this it suffices to notice that, in this case, $\couplingset(\a,\b)$ is the set of doubly stochastic matrices (modulo a factor $n$ which has no impact). We will see in Chapter \ref{cha:gw_euclidean} that the QAP and graph matching point of views can be quite enlightening for deriving properties of the GW distance.

\begin{figure*}[!b]
\begin{memo}[Kronecker product and vec operator]
\label{memo:kron}
Let $\Abf \in \R^{n\times m},\Bbf^{p\times q}$ be two arbitrary matrices. The $\vec$ operator converts the matrix into a column vector. It is defined as $\vec(\Abf) \in \R^{nm\times 1}=(A_{1,1},\dots,A_{n,1},A_{1,2},\dots,A_{n,2},\dots,A_{n,m})^{T}$. The Kornecker product of two matrices result in a block matrix $\Abf \kron \Bbf \in \R^{np \times mq}$ defined by $\Abf \kron \Bbf=(A_{i,j}\Bbf)_{i,j}$.
These two operators satisfy the following properties (see \cite{matrix_cookbook}):
\begin{itemize}
\item $\vec(\Abf\Bbf)=(\mathbf{I}\kron \Abf)\vec(\Bbf)=(\Bbf^{T}\kron \mathbf{I})\vec(\Abf)$
\item $ (\Abf \kron \Bbf)^{T}=(\Abf^{T}\kron \Bbf^{T})$
\item $(\Abf \kron \Bbf)(\Bbf \kron \Abf)=(\Abf\Bbf \kron \Bbf\Abf)$
\item $\tr(\Abf^{T}\Bbf)=\vec(\Abf)^{T}\vec(\Bbf)$
\end{itemize}
\end{memo}
\end{figure*}

\paragraph{Entropic regularization} In \cite{peyre2016gromov,solomon_entropic_2016} authors propose to solve \eqref{discrete_gw} using entropic regularization which results in the following optimization problem:

\begin{equation}
\label{entropic_discrete_gw}
\underset{\GG \in \couplingset(\a,\b)}{\min}\froeb{\L(\C_{1},\C_{2})^{p} \otimes \GG}{\GG}-\epsilon H(\GG). \\
\end{equation}

\begin{algorithm}[t]
    \caption{\label{alg:entropic_reg_gw} Solving Entropic-regularized GW}
    \begin{algorithmic}[1]
    \State $\mathbf{a}, \mathbf{b}, \C_1,\C_2, \epsilon >0$
    \State{ Initalize $\GG=\a\b^{T}, \mathbf{u}^{(0)},\mathbf{v}^{(0)}=\one$}
    \For{$i=1,\dots,n^{1}_{it}$}
    \State{Compute the gradient of the GW loss $\C=2\L(\C_{1},\C_{2})^{p} \otimes \GG$}
      \State{Set $\mathbf{K}=\exp (-\frac{\mathbf{C}}{\epsilon})$} 
      \State \text{ // Do Sinkhorn-Knopp Algorithm \ref{alg:sk}} 
        \For{$i=1,\dots,n^{2}_{it}$} 
        \State $\mathbf{u}^{(i)}=\mathbf{a} \oslash \mathbf{K}^{\top}\mathbf{v}^{(i-1)}$ // \text{ Update left scaling } 
        \State $\mathbf{v}^{(i)}=\mathbf{b} \oslash \mathbf{K} \mathbf{u}^{(i-1)}$ // \text{ Update right scaling } 
        \EndFor 
    \EndFor
    \State \Return $\GG^{*}=\diag({\bf u})\mathbf{K}\diag({\bf v})$
    \end{algorithmic}
\end{algorithm}

This is a non-convex optimization which was tackled using projected gradient descent using the geometry of the $KL$ divergence for both the gradient step and the projection step \cite{peyre2016gromov}. More precisely by denoting $E_{\epsilon}(\GG)$, the loss in \eqref{entropic_discrete_gw} the iterations of this algorithm read:
\begin{equation}
\label{eq:iter_projected}
\GG \leftarrow \text{Proj}_{\couplingset(\a,\b)}^{KL}(\GG * e^{-\tau \nabla E_{\epsilon}(\GG)}) 
\end{equation}
where $\tau>0$ is the step size of the gradient descent and $*$ denotes elementwise (Hadamard) matrix multiplication. The projection operator is defined as the result of the minimization problem:
\begin{equation}
\text{Proj}_{\couplingset(\a,\b)}^{KL}(\mathbf{K})= \argmin_{\GG \in \couplingset(\a,\b)}KL(\GG|\mathbf{K})
\end{equation}
As shown in Section \ref{sec:numerical_tour} the projection can be solved using the efficient Sinkhorn-Knopp Algorithm (see Algorithm \ref{alg:sk}). The gradient $\nabla E_{\epsilon}(\GG)$ can be calculated as $2 \L(\C_{1},\C_{2})^{p}\otimes \GG+\epsilon \log(\GG)$, and, as noted in \cite{peyre2016gromov}, the special case where the step size $\tau$ is defined as $\tau=\frac{1}{\epsilon}$ the iterations \eqref{eq:iter_projected} boil down to solving an entropic regularized linear OT problem with ground cost $2\L(\C_{1},\C_{2})^{p} \otimes \GG$. Overall the procedure to solve GW with entropic regularization is a projected gradient procedure where the projection step can be solved using an entropic linear OT as described in Algorithm \ref{alg:entropic_reg_gw}. Note that, as for the linear OT case, the resulting optimal coupling is not sparse since the entropy term tends to blur the solution. Moreover one usually wants to have a gradient step $\tau$ which is not ``too big'' in order to ensure the convergence of the algorithm. Since the gradient step $\tau$ is inversely proportional to the regularization parameter $\epsilon$ this comes at the price of blurring the resulting optimal solution. In practice, there is a trade-off between regularization and convergence of the algorithm (see \cite{peyre2016gromov} and experiments of Chapter \ref{cha:gw_euclidean}).

When $p=2$ the previous problem reduces to the \emph{softassign quadratic assignement} problem. In the special case where the problem is convex the convergence of the previous scheme was analyzed in \cite{convergence_soft_assign,convergence_soft_assign2} but, with arbitrary matrices $\C_1,\C_2$ there is no known results to the best of our knowledge. 

\paragraph{Computing a lower-bound} Originally \eqref{discrete_gw} was tackled by computing a lower bound (called the \textbf{TLB}) in \cite{memoli_gw}. More precisely, authors propose to solve the following problem:
\begin{equation}
\label{eq:lower_bound_discrete_gw_1}
\underset{\GG^1 \in \couplingset(\a,\b)}{\min}\sum_{k,l} \left(\underset{\GG^2 \in \couplingset(\a,\b)}{\min}\sum_{i,j}|C_{1}(i,k)-C_{2}(j,l)|^{p}\pi^{2}_{i,j}\right) \pi^{1}_{k,l}
\end{equation}
This problem is actually a ``Wasserstein of Wasserstein distances'': it is equivalent to solve an OT problem which ground cost results itself on OT problems between the 1D empirical distributions of the lines of $\C_1,\C_2$. More precisely by considering $\mu_k=\sum_i a_i \delta_{C_{1}(i,k)} \in \P(\R)$ and $\nu_l=\sum_j b_j \delta_{C_{2}(j,l)} \in \P(\R)$ \eqref{eq:lower_bound_discrete_gw_1} is equivalent to:
\begin{equation}
\label{eq:lower_bound_discrete_gw}
\underset{\GG \in \couplingset(\a,\b)}{\min}\sum_{k,l} \wass_{p}^{p}(\mu_k,\nu_l)  \pi_{k,l}
\end{equation}
where the Wasserstein distance is computed using $|.|$ as ground cost. The advantage of this formulation is that $\eqref{eq:lower_bound_discrete_gw}$ only involves linear OT problems that can be solved using tools presented in Section \ref{sec:numerical_tour}. This idea is based on the \emph{local distribution of distances} and was also successfully applied in computer graphics for 3D shape comparison \cite{SGP:SGP05:197-206,memoli_gw}.

\paragraph{Computational complexities} One major bottleneck for computing GW is first the calculation of the big tensor $(|C_{1}(i,k)-C_{2}(j,l)|^{p})_{i,j,k,l}$ which is $O(n^{2}m^{2})$ is general. By noticing that it suffices instead to compute $\L(\C_1,\C_2)^{p}\otimes \GG$ authors in \cite{peyre2016gromov} show that, in the case $p=2$, one can rely on the separability of $\L$ which results in a $O(n^{2}m+m^{2}n)$ complexity. A second bottleneck is the complexity of finding an optimal solution which is driven by the algorithmic method. The convergence of Algorithm \ref{alg:entropic_reg_gw} for the entropic-regularized GW problem is still not well-understood and slow in practice as shown for example in Chapter \ref{cha:gw_euclidean}. We will present in Chapter \ref{cha:fgw} an algorithm based on Frank-Wolfe (FW) to find a sparse local optimal solution with a cubic complexity.  Using the FW properties we will show that this algorithm also converges to a local stationary point with a $O(\frac{1}{\sqrt{t}})$ rate. Regarding the lower bound computation one can rely on the sorting strategies for 1D distributions to compute an inner Wasserstein distance $\wass_p^{p}(\mu_i,\nu_j)$ in $O(n \log(n)+m\log(m))$ hence a $O(n^2m \log(n)+m^2n\log(m)))$ complexity for all pairs. Then finding a solution $\GG$ has the same complexity as computing a linear OT problem and one can rely on entropic regularization to reach a quadratic time complexity. More recently, in \cite{sato2020fast} authors propose to fix the outer optimal transport plan of the lower bound to $\a \b^{T}$ which results in a divergence that can be computed in $O((n^{2}+m^{2})\log(nm))$ using a sweep line strategy.

\paragraph{Illustration of previous solvers} In order to give a simple illustration of the different solvers presented above we consider two unlabeled graphs with the same number of nodes ($n=m=20$) and with $4$ communities. Each node of the graph lies in the implicit metric space defined by the shortest-path distance inside the graph so that $\C_1$ is the shortest path distance matrix between each node (same for $\C_2$). We consider uniform weights, \ie\ $\a=\b=\frac{1}{n}\one_n$ and $p=2$ for GW. We solve the GW problem by relying on (1) The FW algorithm defined in Chapter \ref{cha:fgw}, (2) The entropic regularized GW problem with $\epsilon=8e-4$, (3) The lower bound approach \textbf{TLB}. The behavior of the three approaches is depicted in Figure \ref{fig:gw_example}

\begin{figure}[t]
   \begin{center}
        \includegraphics[width=1\linewidth]{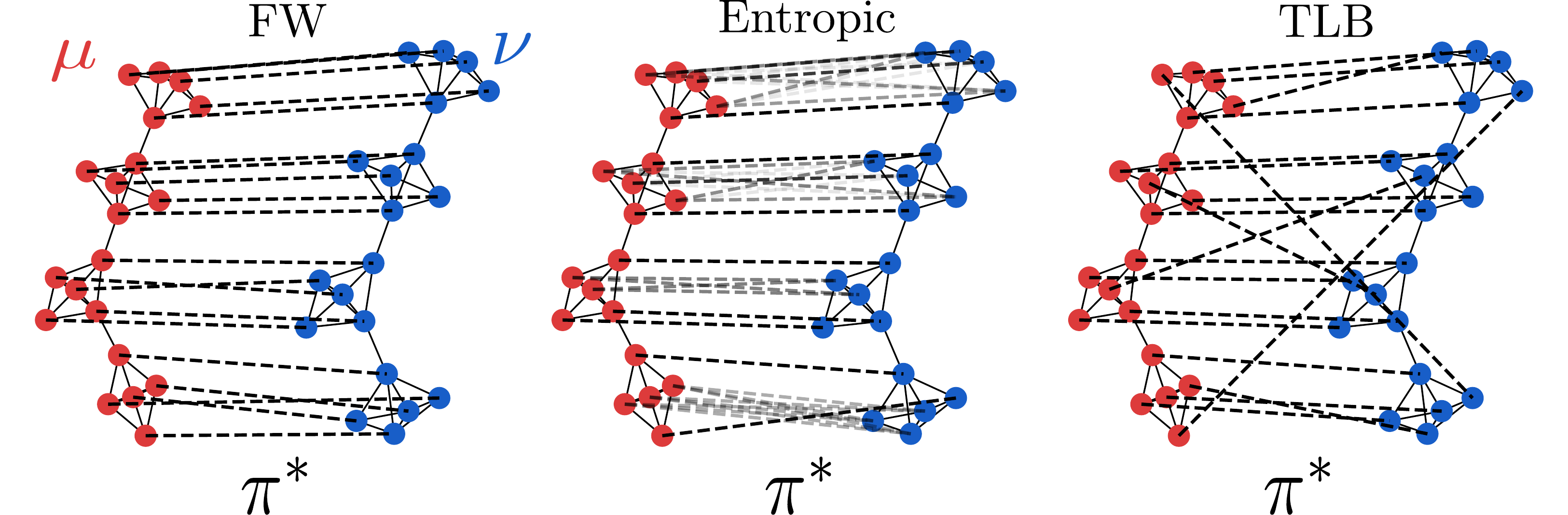}
    \caption{\label{fig:gw_example} Illustration of the optimal coupling for the GW problem between two graphs described by their shortest-path. {\bf (left)} FW solver of Chapter 3 {\bf (middle)} Entropic regularized GW (Algorithm \ref{alg:entropic_reg_gw}) {\bf(right)} Lower bound approach \textbf{TLB}. The optimal coupling $\GG^{*}$ is depicted in dashed lines. The darker the stronger. }
     \end{center}
\end{figure}

\subsection{Gromov-Wasserstein Barycenter \label{sec:bary_gw}}

In the same vein as the Wasserstein barycenter we can build upon the Gromov-Wasserstein distance a notion of barycenter. This setting was first tackled in \cite{peyre2016gromov} where authors consider the problem of computing the barycenter of a family of discrete probability measures over different metric spaces \textit{w.r.t.} the Gromov-Wasserstein geometry. More formally let $(\C_i,\b_i)_{i \in \integ{n}}$ be this family where $\C_i$ is a arbitrary matrix and $\b_i$ is a probability vector. $\C_i$ can be chosen to be a distance matrix or more generally any similarity matrix, such as a kernel matrix encoding a notion of similarity between the points inside each distribution as in \cite{chowdhury_gromovwasserstein_2019}. When the weights of the barycenter are given and fixed to $\a \in \simplex_k$ with $k\in \mathbb{N}^{*}$ the GW barycenter problem aims at finding:

\begin{equation}
\label{eq:bary_gw}
\begin{split}
\min_{\C \in \R^{k\times k}} \sum_{i=1}^{n} \lambda_i \gw_p^{p}(\C,\C_i,\a,\b_i)= \min_{\begin{smallmatrix} \C \in \R^{k\times k} \\ \forall i \in \integ{n}, \GG_i \in \couplingset(\a,\b_i) \end{smallmatrix}} \sum_{i=1}^{n} \lambda_i \froeb{\L(\C,\C_{i})^{p} \otimes \GG}{\GG}
\end{split}
\end{equation}
with $\lambda_i\geq0$ and $\sum_{i=1}^{n} \lambda_i=1$. In \cite{peyre2016gromov} authors consider this problem based on the entropic regularized version of GW. They propose to solve \eqref{eq:bary_gw} by relying on a BCD procedure which alternates between solving $n$ GW problems with $\C$ fixed and by finding $\C$ with all $\GG_i$ fixed. The latter is given in closed-form when $p=2$ by \cite[Equation 14 ]{peyre2016gromov}:
\begin{equation}
\C=\frac{1}{\a\a^{T}} \sum_{i=1}^{n} \lambda_i \GG_i^{T} \C_i \GG_i
\end{equation}
where the division is made element-wise. Solving the $n$ GW problems can be done using Algorithm \ref{alg:entropic_reg_gw} as in \cite{peyre2016gromov} or with the GW or the lower bound approach as presented in the previous section. This approach was further generalized in \cite{chowdhury2019gromovwasserstein} where authors leverage the Riemannian geometry of Gromov-Wasserstein space to treat \textit{e.g.} the case where $\C_i$ are not necessarily symmetric.

\paragraph{Illustration} To illustrate the GW barycenter we consider a simple dataset consisting in 4 2D-shapes from the apple class of the MPEG-7 computer vision database. Each shape is associated with a discrete probability measure with uniform weights. The matrices $\C_i$ are simply the Euclidean distances between the points in each shape. We compute the barycenter using the previous discussion with both the FW solver for solving the GW problems and using the entropic regularized GW ($\epsilon=1e-3$) with Algorithm \ref{alg:entropic_reg_gw}. We compute a MDS on the $\C$ matrix obtained by the barycenter procedure to recover 2D points. Results are depicted in Figure \ref{fig:gw_barycenter_apple}. 

\begin{figure}[t]
   \begin{center}
        \includegraphics[width=0.9\linewidth]{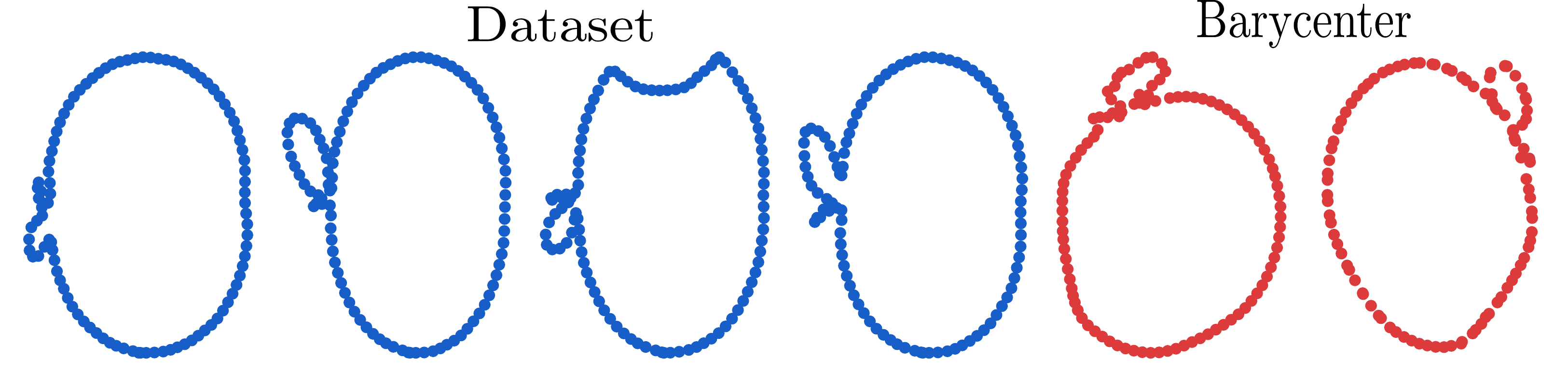}
    \caption{\label{fig:gw_barycenter_apple} GW barycenter of 4 apple shapes in blue. The barycenters are depicted in the right side of the Figure in red. The first barycenter uses the FW algorithm for solving the GW problems and the second is computed using entropic regularization solved with Algorithm \ref{alg:entropic_reg_gw}. Note that the barycenter are arbitrarily rotated due to the MDS procedure.}
     \end{center}
\end{figure}

\paragraph{Applications of GW} The GW problem is well suited for comparing heterogeneous data while being invariant to the isometries of the data. As such it has first received attention for shapes comparison \cite{memoli_gw,solomon_entropic_2016} or in computer vision \cite{schmitzer_modelling_2013} where it is often valuable to compare objects without any assumption on their orientation. GW was further exploited to handle unstructured geometric data such as point clouds or meshes in \cite{gwcnn} where authors use GW to learn regular 2D grids that faithfully represent the 3D meshes while being applicable for standards CNN architectures. More recently GW has been the subject of much attention in the graph community as a graph matching tool \cite{pmlr-v97-xu19b,xu_scalable_2019,Fey2020Deep} or for graphs representation \cite{Kwon_2020} (see Chapter 4 for more details). It also proves its usefulness in cellular biology thanks to its ability for aligning heterogeneous types of single-cell measurements \cite{Demetci_gw_cell}. Closer to the machine learning community GW has been applied in Domain Adaptation (DA) in the complex settings of Unsupervised DA \cite{Xia_2020_CVPR} and Heterogeneous (DA) \cite{ijcai2018-412} (see Chapter 5 for more details), in generative modeling on incomparable spaces \ie\ when the data generated do not share the same Euclidean space as the source data \cite{bunne_gan}, or for cross-lingual correspondences of word embeddings \cite{alvarez-melis_gromov-wasserstein_2018}. 
    
\section{Conclusion}

The optimal transport framework is a powerful tool for comparing probability distributions by relying on both Wasserstein and Gromov-Wasserstein distances. The linear OT theory is a well-studied problem both from the theoretical and numerical side. In contrast a lot of questions remain unanswered for the Gromov-Wasserstein theory and interesting connections with the Wasserstein theory can be made. From a theoretical perspective there is no known result yet about the regularity of the optimal transport plans. As for the practical side the GW problem remains very costly to solve and difficult to approximate. The purpose of the next chapters is, \textit{inter alia}, to work towards these directions and to address the following questions:
\begin{enumerate}[label=\roman*]
\item Is there any favorable cases where Optimal Transport plans of the GW are supported on a Monge map as in the Brenier theorem? (see Chapter \ref{cha:gw_euclidean})
\item Can we find some special cases where GW admits a closed-form solution such as the 1D or the Gaussian cases? If so, can we derive scalable and useful formulations from these cases? (see Chapter \ref{cha:gw_euclidean})
\item Is the GW framework suited for the structured data setting? How does it behave for concrete structured data problems such as graph applications? (see Chapter \ref{cha:fgw})
\item More generally can we derive other formulations than GW that are maybe more useful for data on incomparable spaces? (see Chapter \ref{cha:coot}) 
\end{enumerate}

\chapter{Optimal Transport for structured data}
\epigraph{\itshape Perhaps as you went along you did learn something.}{-- Ernest Hemingway, \textit{The Sun Also Rises}}
\minitoc
\label{cha:fgw}

\newpage

\begin{Abstract}
This chapter is based on the papers \cite{vay_struc,Vayer_2020} and considers the problem of computing distances between structured
objects such as undirected graphs, seen as probability distributions in a
specific metric space. We consider a new transportation distance ({\em
i.e.} that minimizes a total cost of transporting probability masses) that unveils
the geometric nature of the structured objects space. Unlike Wasserstein or
Gromov-Wasserstein metrics that focus solely and respectively on features (by
considering a metric in the feature space) or structure (by seeing structure as
a metric space), our new distance exploits jointly both information, and is
consequently called Fused Gromov-Wasserstein  (FGW). After discussing its
properties and computational aspects, we show results on a graph classification
task, where our method outperforms both graph kernels and
deep graph convolutional networks. Exploiting further on the metric properties
of FGW, interesting geometric objects such as Fr{\'e}chet means or barycenters
of graphs are illustrated and discussed in a clustering context. We provide in a second part the mathematical framework for this distance in the continuous setting, prove its metric, geodesic and interpolation properties and provide a concentration result for the convergence of finite samples.
\end{Abstract}

\section{Introduction}

There is a longstanding line of research on learning from structured data, {\em i.e.} objects that are a combination of a feature and structural information (see for example \cite{Bakir:2007:PSD:1296180,relationnalreasoning}). As immediate instances, graph data are usually ensembles of nodes with attributes (typically  $\mathbb{R}^{d}$ vectors) linked by some specific relation. Notable examples are found in chemical compounds or molecules modeling~\cite{DBLP:journals/corr/KriegeGW16}, brain connectivity~\cite{ktena2017distance}, or social networks~\cite{Yanardag15}. This generic family of objects also encompasses time series \cite{pmlr-v70-cuturi17a}, trees~\cite{day1985optimal} or even images~\cite{bachgraphkernel}.

Being able to leverage on both feature and structural information in a learning task is a tedious task, that requires the association in some ways of those two pieces of information in order to capture the similarity between the 
structured data. Several kernels have been designed to perform this task~\cite{Shervashidze:2011:WGK:1953048.2078187,wlkernel}. As a good representative of those methods, 
the Weisfeiler-Lehman kernel~\cite{wlkernel} captures in each node a notion of vicinity by aggregating, in the sense of the topology of the graph, the surrounding features. Recent advances in graph convolutional networks~\cite{DBLP:journals/corr/BronsteinBLSV16, Kipf2016SemiSupervisedCW,NIPS2016_6081,Gnn_survey} allows learning end-to-end the best combination of features by relying on parametric convolutions on the graph, {\em i.e.} learnable linear combinations of features. In the end, and in order to compare two graphs that might have different number of nodes and connections, those two categories of methods build a new representation for every graph that shares the same {space}, and that is amenable to classification.

Contrasting with those previous methods, we suggest in this chapter to see graphs as probability distributions, embedded in a specific metric space. We propose to define a specific notion of distance between these probability distributions, that can be used in most of the classical machine learning approaches. Beyond its mathematical properties, disposing of a distance between structured data, provided it is meaningful, is desirable in many ways: {\em i)} it can then be plugged into distance-based machine learning algorithms such as $k$-nn or t-SNE {\em ii)} its quality is not dependent on the learning set size, and {\em iii)} it allows considering interesting quantities such as geodesic interpolation or barycenters. 

Yet, defining this distance is not a trivial task. While features can always be compared using a standard metric, such as Euclidean distances, comparing structures requires a notion of similarity which can be found \textit{via} the notion of {\em isometry}, since the graph nodes are not ordered (we define later on which cases two graphs are considered identical). We use the notion of transportation distance to compare two graphs represented as probability distributions. Optimal transport have inspired a number of recent breakthroughs in machine learning (\emph{e.g.} \cite{huang2016,courty2017optimal,arjovsky17a}) because of its capacity to compare empirical distributions, and also the recent advances in solving the underlying problem~\cite{cot_peyre_cutu}. Yet, the natural formulation of OT cannot leverage the structural information of objects since it only relies on a cost function that compares their feature representations. 

However, some modifications over OT formulation have been proposed in order to compare structural information of objects. Following the pioneering work by M\'emoli \cite{memoli_gw}, Peyr\'e {\em et al.} \cite{peyre2016gromov} propose a way of comparing two distance matrices that can be seen as representations of some objects' structures. They use the Gromov-Wasserstein distance (see Chapter \ref{cha:ot_general}) capable of comparing two distributions even if they do not lie in the same ground space and apply it to compute barycenter of molecular shapes. Even though this approach has wide applications, it only encodes the intrinsic structural information in the transportation problem. To the best of our knowledge, the problem of including both structural and feature information in a unified OT formulation remains largely under-addressed.

\paragraph{OT distances that include both features and structures.}
Recent approaches tend to incorporate some structure information as a regularization of the OT problem.
For example in \cite{AlvarezMelis2018Structured} and \cite{courty2017optimal},  authors constrain transport maps to favor some assignments in certain groups. These approaches require a known and simple structure such as class clusters to work but do not generalize well to more general structural information.
In their work \cite{Thorpe2017}, propose an OT distance that combines both a Lagrangian formulation of a signal and its temporal structural information. They define a metric, called Transportation $L^{p}$ distance, that can be seen as a distance over the coupled space of time and feature. They apply it for signal analysis and show that combining both structure and feature tends to better capture the signal information. Yet, for their approach to work, the structure and feature information should lie in the same ambiant space, which is not a valid assumption for more general problems such as similarity between graphs. In \cite{DBLP:conf/aaai/NikolentzosMV17}, authors propose a graph similarity measure for discrete labeled graph with OT. Using the eigenvector decomposition of the adjency matrix, which captures graph connectivities, nodes of a graph are first embedded in a new space, then a ground metric based on the distance in both this embedding and the labels is used to compute a Wasserstein distance serving as a graph similarity measure.

\paragraph{Contributions.}

After defining structured data as discrete probability measures (Section~\ref{sec:back}), we propose a new framework, namely $FGW$, capable of taking into account both structure and feature information into the optimal transport problem. The framework can compare any usual structured machine learning data even if the feature and structure information dwell in spaces of different dimensions, allowing the comparison of undirected labeled graphs. It is based on a distance that embeds a trade-off parameter which allows balancing the importance of the features and the structure. We provide a conditional-gradient algorithm for computing $FGW$ (Section~\ref{sec:fgw}), and we evaluate it (Section~\ref{sec:expeFGW}) on both synthetic and real-world graph datasets on various tasks. We show that $FGW$ is particularly useful for both supervised and unsupervised learning on graphs.  

Among the contributions of this Chapter the numerical solution presented in Section~\ref{sec:fgw} can also be used to compute the Gromov-Wasserstein distance. To the best of our knowledge this is the first optimization scheme for GW that does not require entropic regularization and which results in a sparse optimal solution. 

We also define and illustrate a notion of labeled graph barycenters using $FGW$ (Section~\ref{sec:bary}), based on the Fréchet mean, and apply it for clustering and coarsening of graphs problems. 

In a last part (Section~\ref{sec:continuous_fgw}) we generalize the definition of structured data to compact metric spaces. We present the theoretical foundations of our framework in this general setting and states the mathematical properties of $FGW$. Notably, we show that it is a metric in the space of structured objects with respect to an intuitive equivalence relation between structured objects, we give a concentration result for the convergence of finite samples, and we study its interpolation and geodesic properties. 

\section{Structured data as probability measures}
\label{sec:back}

In this chapter, we focus on comparing structured data which combine a feature \textbf{and} a structure information. In order to give a good intuition about the method we first consider the discrete setting which corresponds to labeled graphs.
More formally, we consider undirected labeled graphs as tuples of the form $\mathcal{G}=(\mathcal{V},\mathcal{E},\ell_f,\ell_s)$ where $(\mathcal{V},\mathcal{E})$ are the set of vertices and edges of the graph.
$\ell_f: \mathcal{V} \rightarrow \Omega_f$ is a labelling function which associates each vertex $v_{i} \in \mathcal{V}$ with a feature $a_{i}\stackrel{\text{def}}{=}\ell_f(v_{i})$ in some feature metric space $(\Omega_f,d)$. 
We will denote by \emph{feature information} the set of all the features $(a_{i})_i$ of the graph.
Similarly, $\ell_s: \mathcal{V} \rightarrow \Omega_s$ maps a vertex $v_i$ from the graph to its structure representation $x_{i}\stackrel{\text{def}}{=}\ell_s(v_{i})$ in some structure space $(\Omega_s,C)$ specific to each graph. $C : \Omega_s \times \Omega_s \rightarrow \mathbb{R}$ is a symmetric application which aims at measuring the similarity between the nodes in the graph. Unlike the feature space however, $\Omega_s$ is implicit and in practice, knowing the similarity measure $C$ will be sufficient. With a slight abuse of notation, $C$ will be used in the following to denote both the structure similarity measure and the matrix that encodes this similarity between pairs of nodes in the graph $(\C(i,k) = C(x_i, x_k))_{i,k}$.
Depending on the context, $\C$ can either encode the neighborhood information of the nodes, the edge information of the graph or more generally it can model a distance between the nodes such as the shortest path distance or the harmonic distance \cite{NIPS2017_6614}. When $C$ is a metric, such as the shortest-path distance, we naturally endow the structure with the metric space $(\Omega_s,C)$.
We will denote by \emph{structure information} the set of all the structure embeddings $(x_{i})_i$ of the graph. 

\begin{figure}[t]
    \centering
        \includegraphics[width=0.9\columnwidth]{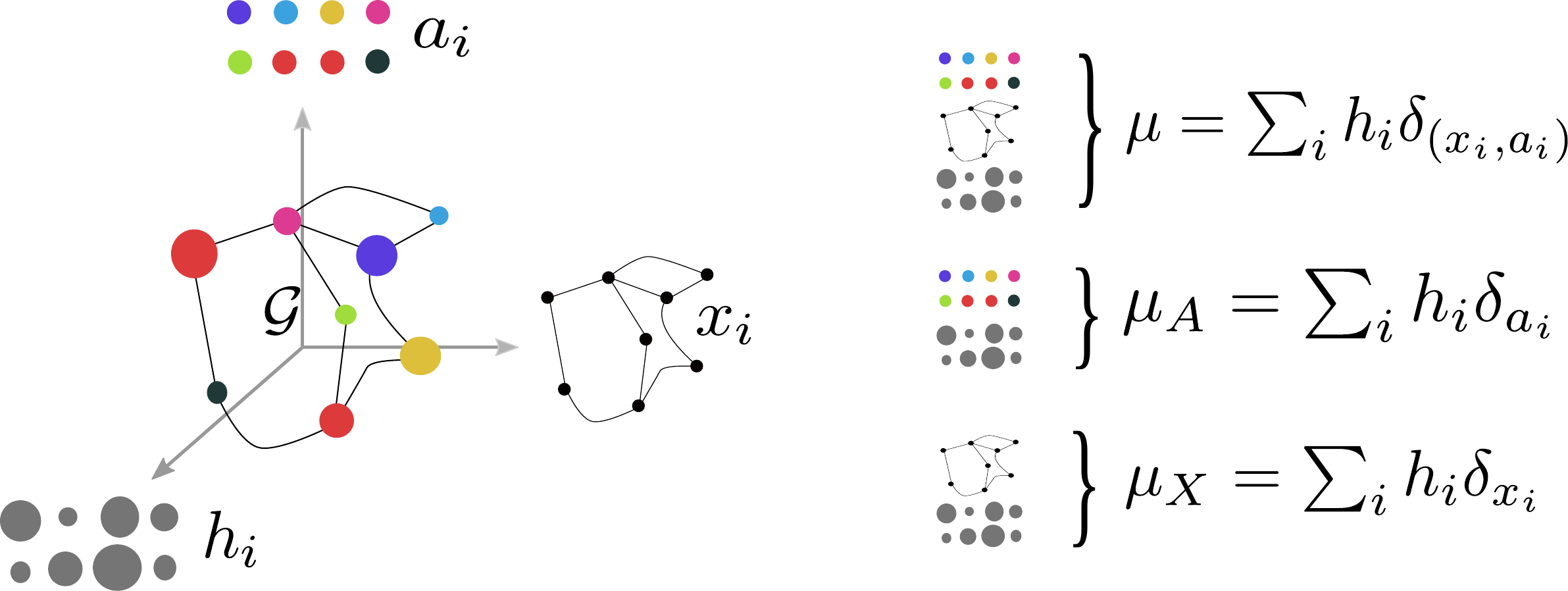}
    \caption{{\bf (left)} Labeled graph with $(a_{i})_{i}$ its feature information, $(x_{i})_{i}$ its structure information and a probability vector $(h_{i})_{i}$ that measures the relative importance of the vertices. {\bf (right)} Associated structured data which is entirely described by a fully supported probability measure $\mu$ over the product space of feature and structure, with marginals $\mu_{X}$ and $\mu_{A}$ on the structure and the features respectively.\label{graphex}}
\end{figure}

We propose to enrich the previously described graph with a probability vector which serves the purpose of signaling the relative importance of the vertices in the graph. To do so, if we assume that the graph has $n$ vertices, we equip those vertices with weights $(h_{i})_{i} \in \Sigma_{n}$. 
Through this procedure, we derive the notion of \emph{structured data} as a tuple $\mathcal{S}=(\mathcal{G},h_{\mathcal{G}})$ where $\mathcal{G}$ is a graph as described previously and $h_{\mathcal{G}}$ is a function that associates a weight to each vertex. This definition allows the graph to be represented by a fully supported probability measure over the product space feature/structure $\mu= \sum_{i=1}^{n} h_{i} \delta_{(x_{i},a_{i})}$ which describes the entire structured data (see Figure \ref{graphex}). When all the weights are  equal (\textit{i.e.} $h_{i}=\frac{1}{n}$), so all vertices have the same relative importance, the structured data holds the exact same information as its graph. However, weights can be used to encode some \textit{a priori} information. For instance on segmented images, one can construct a graph using the spatial neighborhood of the segmented zones, the features can be taken as the average color in the zone, and the weights as the ratio of image pixels in the zone.

\section{Fused Gromov-Wasserstein approach for structured data}
\label{sec:fgw}
We aim at defining a distance between two graphs $\mathcal{G}_1$ and  $\mathcal{G}_2$, described respectively by their probability measure $\mu= \sum_{i=1}^{n} h_{i} \delta_{(x_{i},a_{i})}$ and $\nu= \sum_{i=1}^{m} g_{j} \delta_{(y_{j},b_{j})}$, where  $\h \in \Sigma_{n}$ and $\g \in \Sigma_{m}$ are probability vectors. Without loss of generality we suppose $(x_{i},a_{i})\neq (x_{j},a_{j})$ for $i\neq j$ (same for $y_{j}$ and $b_{j}$). We recall that $\couplingset(\h,\g)$ the set of all admissible couplings between $\h$ and $\g$. To that extent, the matrix $\GG \in \couplingset(\h,\g)$ describes a probabilistic matching of the nodes of the two graphs. We note $\Mbf_{\Abf\Bbf}=(d(a_{i},b_{j}))_{i,j}$ the $n \times m$ matrix standing for the distance between the features.
The structure matrices are denoted $\C_{1}$ and $\C_{2}$, and $\mu_{X}$ and $\mu_{A}$ (resp. $\nu_{Y}$ and $\nu_{B}$) are representative of the marginals of $\mu$ (resp. $\nu$) \textit{w.r.t.} the structure and feature respectively (see Figure \ref{graphex}).
We also define the similarity between the structures by measuring the similarity between all pairwise distances within each graph thanks to the 4-dimensional tensor $\L(\C_{1},\C_{2})$:
\begin{equation*}
    \L(\C_{1},\C_{2})=(L_{i,j,k,l}(\C_{1},\C_{2}))_{i,j,k,l}=(|C_{1}(i,k)-C_{2}(j,l)|)_{i,j,k,l} .
\end{equation*}

\subsection{Fused Gromov-Wasserstein distance}
\begin{figure}
    \centering
    \includegraphics[width=0.7\linewidth]{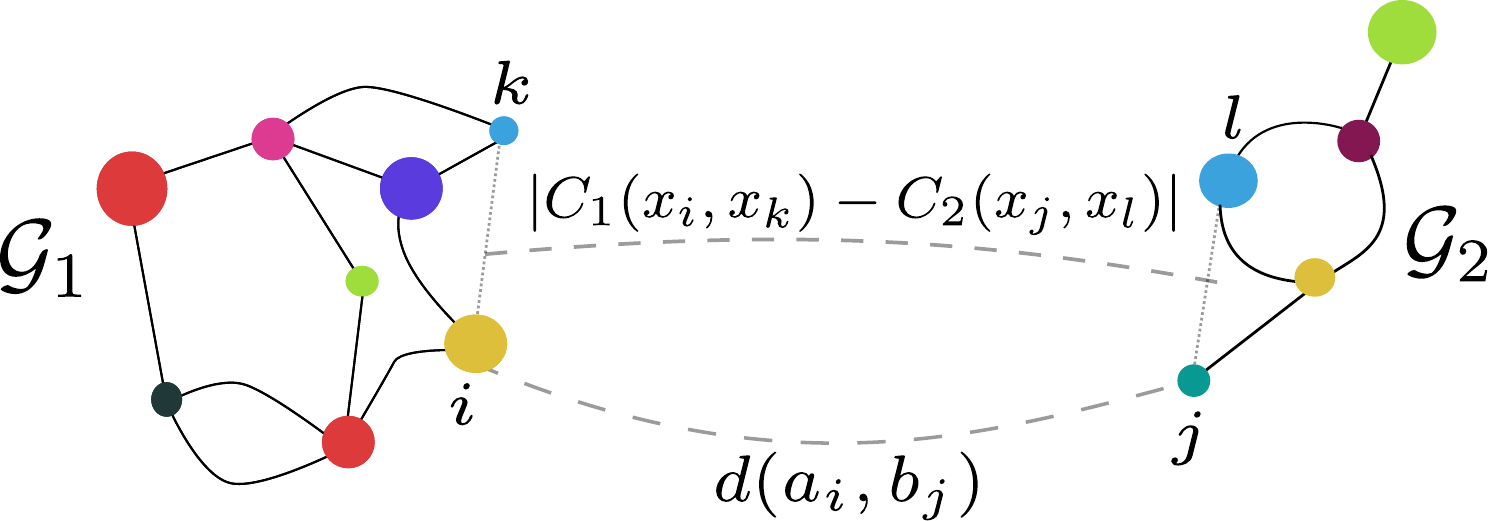}
    \caption{\label{def_fgw}    $FGW$ loss $E_{q}$ for a coupling $\pi$ depends on both a similarity between each feature of each node of each graph $(d(a_{i},b_{j}))_{i,j}$ and between all intra-graph structure similarities $(\left|C_{1}(x_i, x_k)-C_{2}(x_j, x_l)\right|)_{i,j,k,l}$.}
 \end{figure}

We define a novel Optimal Transport discrepancy called the Fused Gromov-Wasserstein distance. It is defined for a trade-off parameter  $\alpha \in [0,1]$ as
\begin{equation}
\label{discretefgw}
\fgwdistance(\Mbf_{\Abf\Bbf},\C_1,\C_2,\h,\g)= \underset{\GG \in \couplingset(\h,\g)}{\min}E_{q}(\Mbf_{\Abf\Bbf},\C_{1},\C_{2},\GG)  \\
\end{equation}
where
\begin{equation}
\begin{split}
E_{q}(\Mbf_{\Abf\Bbf},\C_{1},\C_{2},\GG) &= \froeb{(1-\alpha) \Mbf_{\Abf \Bbf}^{q} + \alpha \L(\C_{1},\C_{2})^{q} \otimes \GG}{\GG} \\
&=\sum_{i,j,k,l} (1-\alpha) d(a_{i},b_{j})^{q}+\alpha |C_{1}(i,k)-C_{2}(j,l)|^{q} \pi_{i,j}\pi_{k,l}
\end{split}
\end{equation}
The $FGW$ distance looks for the coupling $\GG$ between the vertices of the graph
that minimizes the cost $E_{q}$ which is a linear combinaison of a cost
$d(a_{i},b_{j})$ of transporting one feature $a_{i}$ to a feature $b_{j}$ and a
cost $|C_{1}(i,k)-C_{2}(j,l)|$ of transporting pairs of nodes in each structure
(see Figure \ref{def_fgw}). As such, the optimal coupling tends to associate pairs of feature and
structure points with similar distances within each structure pair and with
similar features. $\alpha$ acts as a trade-off parameter between the cost of the structures represented by $\L(\C_{1},\C_{2})$ and the cost on the features $\Mbf_{\Abf\Bbf}$. In this way, the convex combination of both terms leads to the use of both information in one formalism resulting on a single map $\GG$ which ``moves'' the mass from one joint probability measure forawrd to the other.
As an important feature of $FGW$, by relying on a sum of (inter- and intra-)vertex-to-vertex distances, it can handle structured data with continuous attributed or discrete labeled nodes (thanks to the definition of $d$) and can also be computed even if the graphs have different number of nodes. 
 
This new distance is called the $FGW$ distance as it acts as a generalization of the Wasserstein and Gromov-Wasserstein distances. Indeed as $\alpha$ tends to zero, the $FGW$ distance recovers the Wasserstein distance between the features ${W}_{q}(\mu_{A},\nu_{B})^{q}$ and as $\alpha$ tends to one, we recover the Gromov-Wasserstein distance $\gw_{q}(\mu_{X},\nu_{Y})^{q}$ between the structures (see Proposition \ref{interpolationtheorem} of Section \ref{sec:interpo_section}).

More importantly $FGW$ enjoys metric properties on labeled graphs as stated in the following theorem:

\begin{theo}[${FGW}$ defines a metric for $q=1$ and a semi-metric for $q >1$]
\label{metrictheodiscrete}
If $q=1$, and if $\C_{1}$, $\C_{2}$ are distance matrices such as shortest-path matrices then ${FGW}$ defines a metric over the space of structured data quotiented by the measure preserving isometries that are also feature preserving. More precisely, ${FGW}$ satisfies the triangle inequality and vanishes \textit{iff} $n=m$ and there exists a permutation $\sigma \in \Sn$ such that: 
\begin{equation}
\label{weightpreserve}
\forall i \in \integ{n}, \ h_{i}=g_{\sigma(i)}
\end{equation}
\begin{equation}
\label{featurepreserve}
\forall i \in \integ{n}, \ a_{i}=b_{\sigma(i)} 
\end{equation}
\begin{equation}
\label{structurepreserve}
\forall i,k \in  \integ{n}^2,  \ C_{1}(i,k)=C_{2}(\sigma(i),\sigma(k))
\end{equation}
If $q>1$,  the triangle inequality is relaxed by a factor $2^{q-1}$ such that ${FGW}$ defines a semi-metric.
\end{theo}

This results is a direct consequence of Theorem \ref{metrictheo} in Section \ref{Topology} where $FGW$ is defined for general metric spaces. The resulting permutation $\sigma$ preserves the weight of each node (equation \eqref{weightpreserve}), the features (equation \eqref{featurepreserve}) and the pairwise structure relation between the nodes (equation \eqref{structurepreserve}). For example, comparing two graphs with uniform weights on the vertices and with shortest-path structure matrices, the ${FGW}$ distance vanishes \textit{iff} the graphs have the same number of vertices and there exists a one-to-one mapping between the vertices of the graphs which preserves both the shortest-paths and the features. More informally, in this case graphs have vertices with the same labels connected by the same edges, and thus $FGW$ can be used to determine if two graphs are isomorphic \cite{west_introduction_2000}.

The metric ${FGW}$ is fully unsupervised and can be used in a wide set of applications such as $k$-nearest-neighbors, distance-substitution kernels, pseudo-Euclidean embeddings, or representative-set methods. Arguably, such a distance also allows for a fine interpretation of the similarity (through the optimal mapping $\GG$), contrary to end-to-end learning machines such as neural networks.

\subsection{Fused Gromov-Wasserstein barycenter}
\label{sec:bary}
OT barycenters have many desirable properties and applications (see Chapter \ref{cha:ot_general} for more details), yet no formulation can leverage both structural and feature information in the barycenter computation. 
In this section, we consider the ${FGW}$ distance to define a barycenter of a set of structured data as a Fr\'echet mean. 
We look for the structured data $\mu$ that minimizes the sum of (weighted) ${FGW}$ distances within a given set of structured data $(\mu_{k})_{k \in \integ{K}}$ associated with structure matrices $(\C_{k})_{k \in \integ{K}}$, features $(\Bbf_{k})_{k \in \integ{K}}$ and base histograms $(\h_{k})_{k \in \integ{K}}$. 
 For simplicity, we assume that the histogram $\h$ associated to the barycenter is known and fixed; in other words, we set the number of vertices $N$ and the weight associated to each of them.

In this context, for a fixed $N \in \mathbb{N}$ and $(\lambda_{k})_{k}$ such that $\sum_{k=1}^{K} \lambda_{k}=1$ , we aim to find the set of features $\Abf = (a_i)_i$ and the structure matrix $\C$ of the barycenter that minimize the following equation:
\begin{equation}
\label{eq:frechet_mean_fgw}
\underset{\mu}{\text{min}} \sum_{k=1}^{K} \lambda_{k} {FGW}_{q,\alpha}(\mu,\mu_{k}) =\underset{\begin{smallmatrix} \C \in \mathbb{R}^{N \times N},\ \Abf \in \mathbb{R}^{N \times d} \\ \forall k \in \integ{K}, \GG_{k} \in \couplingset(\h,\h_{k}) \end{smallmatrix}}{\text{min}} \sum_{k=1}^{K} \lambda_{k} E_{q}(\Mbf_{\Abf \Bbf_{k}},\C,\C_{k},\GG_{k})
\end{equation}
Note that this problem is jointly convex \textit{w.r.t.} $\C$ and $\Abf$ but not \textit{w.r.t.} $\GG_{k}$. We discuss the proposed algorithm to solve this problem in the next section.  Interestingly enough, one can derive several variants of this problem, where the features or the structure matrices of the barycenter can be fixed. Solving the related simpler optimization problem extends straightforwardly. We give examples of such barycenters both in the experimental section where we solve a graph based $k$-means problem.

\subsection{Optimization and algorithmic solution}
\label{sec:solved_fgw}
In this section we discuss the numerical optimization problem for computing the ${FGW}$ distance between discrete distributions.

\paragraph{Solving the Quadratic Optimization problem.}
Equation \eqref{discretefgw} is clearly a quadratic {problem} \emph{w.r.t.} $\GG$ which is NP-hard in general \cite{Loiola_survey_qap}. However finding a solution in practice can be done quite efficiently. We propose here a method based on the Frank-Wolfe algorithm \cite{pmlr-v28-jaggi13} (\emph{aka} Conditional Gradient).
When considering $q=2$  the
${FGW}$ computation problem can be re-written as finding $\GG^{*}$ such
that:
\begin{equation}
 \GG^{*}=\underset{\GG \in \couplingset(\h,\g)}{\arg\min}\quad\text{vec}(\GG)^{T} \Qbf(\alpha) \text{vec}(\GG)+ \text{vec}(\Dbf(\alpha))^{T}  \text{vec}(\pi)
 \label{l2eq}
 \end{equation}
where $\Qbf=-2 \alpha  \C_{2} \otimes_{K} \C_{1}$ and $\Dbf(\alpha)=(1-\alpha)\Mbf_{\Abf \Bbf}$.
$\otimes_{K}$  denotes the Kronecker product of two matrices, $\text{vec}$ the
column-stacking operator. With such form, the resulting optimal map can be seen
as a quadratic regularized map from initial Wasserstein \cite{ferradans2014regularized,flamary2014optlaplace}. However, unlike these approaches, we have a quadratic
but provably non convex term.
The gradient $\Gbf$ that arises from equation \eqref{discretefgw} can be expressed with the following partial derivative \emph{w.r.t.} $\GG$:
\begin{equation}
  \Gbf(\GG)=(1-\alpha) \Mbf_{\Abf \Bbf}^{q}+2\alpha \L(\C_{1},\C_{2})^{q}\otimes \GG
    \label{eq:gradfgw}
\end{equation}
Note that despite the apparent $O(m^2n^2)$ complexity of computing the
tensor product $\L(\C_{1},\C_{2})^{q}\otimes \GG$ given $\GG$, one can simplify the sum to complexity $O(mn^2+m^2n)$
\cite{peyre2016gromov} operations when $q=2$.

\begin{algorithm}[t]
    \caption{\label{alg:cg}
     Conditional Gradient (CG) for $FGW$}
            \begin{algorithmic}[1]
            \State $\pi^{(0)}\leftarrow \h\g^\top$
            \For {$i=1,\dots,$}
            \State $\Gbf\leftarrow$ Gradient from equation \eqref{eq:gradfgw} \emph{w.r.t.} $\GG^{(i-1)}$
            \State $\tilde\GG^{(i)}\leftarrow $ Solve OT with ground loss $\Gbf$
            \State $\tau^{(i)}\leftarrow$ Line-search for loss \eqref{discretefgw} with $\tau\in(0,1)$ using Alg. \ref{alg:line}
            \State $\GG^{(i)}\leftarrow (1-\tau^{(i)})\GG^{(i-1)}+\tau^{(i)}\tilde\GG^{(i)} $
            \EndFor
        \end{algorithmic}
          \end{algorithm}
Solving a large scale QP with a classical solver can be computationally expensive. In \cite{ferradans2014regularized}, authors propose a solver for a graph regularized optimal transport problem whose resulting optimization problem is also a QP. We can then directly use their conditional gradient scheme to solve our optimization problem as presented in Algorithm \ref{alg:cg}. It only needs at each iteration to compute the gradient in equation \eqref{eq:gradfgw} and to solve a linear OT problem with classical solvers (see Chapter \ref{cha:ot_general} for more details). The line-search part is a constrained minimization of a second degree polynomial function and, as such, admits a closed form expression written in Algorithm \ref{alg:line}. While the problem is non convex, conditional gradient is known to converge to a local stationary point with a $O(\frac{1}{\sqrt{t}})$ rate~\cite{lacoste2016convergence}. More precisely we note $g^{(i)}$ the Frank-Wolfe gap at iteration $i$ defined by: 
\begin{equation}
g^{(i)}=\max_{\GG \in \couplingset(\h,\g)} \froeb{\GG-\GG^{(i)}}{-\Gbf^{i}}
\end{equation}
We note also $|||.|||$ the dual norm for tensors: $|||\L|||=\sup_{\|\Abf\|_{\F}=1} \|\L \otimes \Abf\|_{\F}$ where $\L$ is a 4-dimensional tensor and $\Abf$ a matrix. Then using this dual norm we have that the gradient $\Gbf(.)$ is $|||2 \alpha \L(\C_{1},\C_{2})^{q}|||$-Lipschitz: 
\begin{equation}
\begin{split}
\forall (\GG_1,\GG_2) \in \couplingset(\h,\g)^{2}, \ \|\Gbf(\GG_1)-\Gbf(\GG_2)\|_{\F}&=  \|2\alpha \L(\C_{1},\C_{2})^{q} \otimes (\GG_1-\GG_2)\|_{\F} \\
&\leq 4 \alpha^{2} |||\L(\C_{1},\C_{2})^{q}||| \ \|\GG_1-\GG_2\|_{\F}
\end{split}
\end{equation}
We note also $\text{diam}_{\|.\|_{\F}}(\couplingset(\h,\g))^{2}$ the diameter of $\couplingset(\h,\g)$ (see \cite{pmlr-v28-jaggi13}):
\begin{equation}
\text{diam}_{\|.\|_{\F}}(\couplingset(\h,\g))^{2}=\max_{(\GG_1,\GG_2) \in \couplingset(\h,\g)^{2}}\|\GG_1-\GG_2\|^{2}_{\F}
\end{equation}
Then by using Theorem 1 in \cite{lacoste2016convergence} then minimal gap encountered by the iterates during the algorithm after $t$ iterations satisfies:
\begin{equation}
\min_{0 \leq i \leq t} g^{(i)} \leq \frac{\max\{ 2 h_{0}, C\}}{\sqrt{t+1}}
\end{equation}
where $C=4 \alpha^{2} |||\L(\C_{1},\C_{2})^{q}||| \ \text{diam}_{\|.\|_{\F}}(\couplingset(\h,\g))^{2}$ and $h_{0}$ is is the initial global suboptimality.
        
        \begin{algorithm}[t]
        \caption{\label{alg:line}
         Line-search for CG ($q=2$)}
                    \begin{algorithmic}[1]

        		\State $\mathbf{c}_{\C_{1},\C_{2}}$ from equation (6) in \cite{peyre2016gromov}
                \State $a=-2 \alpha \langle \C_{1} \tilde\GG^{(i)} \C_{2},\tilde\GG^{(i)} \rangle$ 
                \State $b{=}\froeb{(1-\alpha) \Mbf_{\Abf \Bbf}+\alpha \mathbf{c}_{\C_{1},\C_{2}}}{\tilde\GG^{(i)}}$
                $\, -2\alpha \big( \froeb{\C_{1} \tilde\GG^{(i)} \C_{2}}{\GG^{(i-1)}} + \froeb{\C_{1} \GG^{(i-1)} \C_{2}}{\tilde\GG^{(i)}} \big)$
                \If {$a>0$} 
                \State $\tau^{(i)} \leftarrow \text{min}(1,\text{max}(0,\frac{-b}{2a}))$
                \Else{\State $\tau^{(i)} \leftarrow 1$ 
                if $a+b<0$ else $\tau^{(i)} \leftarrow 0$} 
                \EndIf
            \end{algorithmic}
                  \end{algorithm}

\paragraph{Solving the barycenter problem with Block Coordinate Descent (BCD).}
We propose to minimize equation \eqref{eq:frechet_mean_fgw} using a BCD algorithm, \textit{i.e.} iteratively minimizing with respect to the couplings $\GG_{k}$, to the metric $\C$ and the feature vector $\Abf$.
The minimization of this problem \emph{w.r.t.} $(\GG_{k})_{k \in \integ{K}}$ is equivalent to compute $S$ independent Fused Gromov-Wasserstein distances as discussed above. We suppose that the feature space is $\Omega_f=(\mathbb{R}^{d},\|.\|_{2}^{2})$ and we consider $q=2$.
Minimization \emph{w.r.t.} $\C$ in this case has a closed form (see Proposition 4 in
\cite{peyre2016gromov} and Chapter \ref{cha:ot_general}) :
\begin{equation}
\label{eq:cmatrix}
    \C \leftarrow \frac{1}{\h\h^{T}} \sum_{k=1}^{K} \lambda_{k} \GG_{k}^{T} \C_{k} \GG_{k}
\end{equation}
where the division is computed elementwise and $\h$ is the histogram of the barycenter as discussed in section \ref{sec:bary}.
Minimization \emph{w.r.t.} $\Abf$ can be computed with \cite[Equation 8]{pmlr-v32-cuturi14}:
\begin{equation}
\Abf \leftarrow \sum_{k=1}^{K} \lambda_{k} \Bbf_{k} \GG_{k}^{T} \text{diag}(\frac{1}{\h})
\end{equation}

\tikzstyle{vertex}=[circle,fill=black,minimum size=4pt,inner sep=0pt]
\tikzstyle{edge} = [draw]
\tikzstyle{transp} = [draw, dashed]
\tikzstyle{feuille1}=[rectangle,draw,fill=blue,text=blue,minimum size=4pt,inner sep=0pt]
\tikzstyle{feuille2}=[circle,draw,fill=red,text=blue,minimum size=4pt,inner sep=0pt]
\tikzstyle{entour}=[ellipse,draw,text=blue]

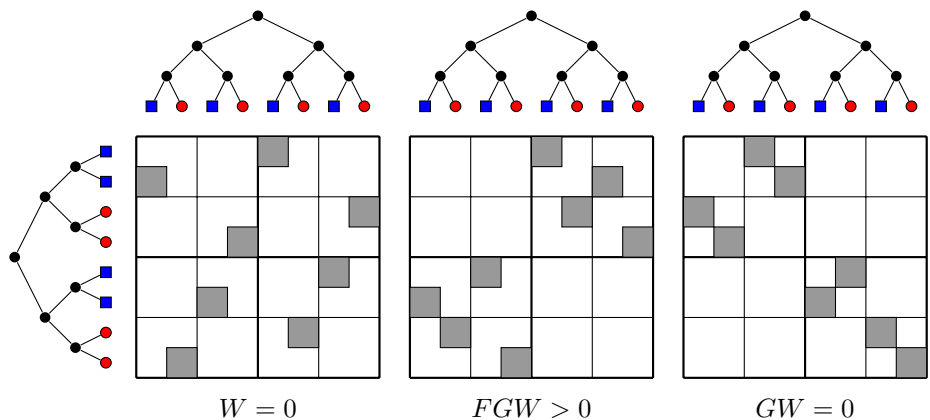
\begin{figure}[t]
\centering
\begin{tikzpicture}[scale=0.4, auto,swap]
    \foreach \pos/\name in {{(0,4)/a}, {(1,2)/c}, {(1,6)/b},
                             {(2,7)/d}, {(2,5)/e}, {(2,3)/f}, {(2,1)/g}}
        		\node[vertex] (\name) at \pos {};
	    \foreach \pos/\name in {
                             {(3,7.5)/h},  {(3,6.5)/i},  {(3,3.5)/l}, {(3,2.5)/m}}
        		\node[feuille1] (\name) at \pos {};
      \foreach \pos/\name in {
                             {(3,5.5)/j},  {(3,4.5)/k},  {(3,0.5)/n}, {(3,1.5)/o}}
       		\node[feuille2] (\name) at \pos {};
                     \foreach \source/ \dest in {b/a, c/a, d/b,e/b, f/c, g/c, d/h, d/i, e/j, e/k, f/l, f/m, g/n, g/o}
        \path[edge] (\source) -- (\dest);

   \draw[black] (4,0) grid[step=2](12,8);
  \draw[black, thick] (4,0) grid[step=4](12,8);

\draw[fill=black!40] (8,7) rectangle(9,8);
\draw[fill=black!40] (4,6) rectangle(5,7);
\draw[fill=black!40] (11,5) rectangle(12,6);
\draw[fill=black!40] (7,4) rectangle(8,5);
\draw[fill=black!40] (10,3) rectangle(11,4);
\draw[fill=black!40] (6,2) rectangle(7,3);
\draw[fill=black!40] (9,1) rectangle(10,2);
\draw[fill=black!40] (5,0) rectangle(6,1);

    \foreach \pos/\name in {{(8,12)/a}, {(6,11)/b}, {(10,11)/c},
                             {(5,10)/d}, {(7,10)/e}, {(9,10)/f}, {(11,10)/g}}
        		\node[vertex] (\name) at \pos {};
	    \foreach \pos/\name in {
                             {(4.5,9)/h},  {(6.5,9)/j},  {(8.5,9)/l}, {(10.5,9)/n}}
        		\node[feuille1] (\name) at \pos {};
      \foreach \pos/\name in {
                             {(5.5,9)/i},  {(7.5,9)/k},  {(9.5,9)/m}, {(11.5,9)/o}}
       		\node[feuille2] (\name) at \pos {};
    \foreach \source/ \dest in {b/a, c/a, d/b,e/b, f/c, g/c, d/h, d/i, e/j, e/k, f/l, f/m, g/n, g/o}
        \path[edge] (\source) -- (\dest);

\draw  (8,-1) node {${W}=0$};

   \draw[black] (12,0) grid[step=2, ,xshift=1cm](20,8);
  \draw[step=4,black, thick, xshift=1cm] (12,0) grid (20,8);
\draw  (17,-1) node {${FGW}>0$};

    \foreach \pos/\name in {{(8+9,12)/a}, {(6+9,11)/b}, {(10+9,11)/c},
                             {(5+9,10)/d}, {(7+9,10)/e}, {(9+9,10)/f}, {(11+9,10)/g}}
        		\node[vertex] (\name) at \pos {};
	    \foreach \pos/\name in {
                             {(4.5+9,9)/h},  {(6.5+9,9)/j},  {(8.5+9,9)/l}, {(10.5+9,9)/n}}
        		\node[feuille1] (\name) at \pos {};
      \foreach \pos/\name in {
                             {(5.5+9,9)/i},  {(7.5+9,9)/k},  {(9.5+9,9)/m}, {(11.5+9,9)/o}}
       		\node[feuille2] (\name) at \pos {};
    \foreach \source/ \dest in {b/a, c/a, d/b,e/b, f/c, g/c, d/h, d/i, e/j, e/k, f/l, f/m, g/n, g/o}
        \path[edge] (\source) -- (\dest);

   \draw[black] (24,0) grid[step=2, ,xshift=-2cm](32,8);
  \draw[step=4,black, thick, xshift=-2cm] (24,0) grid (32,8);

\draw[fill=black!40] (17,7) rectangle(18,8);
\draw[fill=black!40] (19,6) rectangle(20,7);
\draw[fill=black!40] (18,5) rectangle(19,6);
\draw[fill=black!40] (20,4) rectangle(21,5);
\draw[fill=black!40] (15,3) rectangle(16,4);
\draw[fill=black!40] (13,2) rectangle(14,3);
\draw[fill=black!40] (14,1) rectangle(15,2);
\draw[fill=black!40] (16,0) rectangle(17,1);

    \foreach \pos/\name in {{(8+18,12)/a}, {(6+18,11)/b}, {(10+18,11)/c},
                             {(5+18,10)/d}, {(7+18,10)/e}, {(9+18,10)/f}, {(11+18,10)/g}}
        		\node[vertex] (\name) at \pos {};
	    \foreach \pos/\name in {
                             {(4.5+18,9)/h},  {(6.5+18,9)/j},  {(8.5+18,9)/l}, {(10.5+18,9)/n}}
        		\node[feuille1] (\name) at \pos {};
      \foreach \pos/\name in {
                             {(5.5+18,9)/i},  {(7.5+18,9)/k},  {(9.5+18,9)/m}, {(11.5+18,9)/o}}
       		\node[feuille2] (\name) at \pos {};
    \foreach \source/ \dest in {b/a, c/a, d/b,e/b, f/c, g/c, d/h, d/i, e/j, e/k, f/l, f/m, g/n, g/o}
        \path[edge] (\source) -- (\dest);
\draw  (26,-1) node {${GW}=0$};

\draw[fill=black!40] (24,7) rectangle(25,8);
\draw[fill=black!40] (25,6) rectangle(26,7);
\draw[fill=black!40] (22,5) rectangle(23,6);
\draw[fill=black!40] (23,4) rectangle(24,5);
\draw[fill=black!40] (27,3) rectangle(28,4);
\draw[fill=black!40] (26,2) rectangle(27,3);
\draw[fill=black!40] (28,1) rectangle(29,2);
\draw[fill=black!40] (29,0) rectangle(30,1);

\end{tikzpicture}
\vspace{-2mm}

\caption{Example of $FGW$, $GW$ and $W$ on synthetic trees. Dark grey color represents a non null $\pi_{i,j}$ value between two nodes $i$ and $j$. {\bf (left)} the $W$ distance between the features with $\alpha=0$, {\bf (middle)}  $FGW$  {\bf (right)} the $GW$ between the structures $\alpha=1$. \label{mapstoy}}
\end{figure}

\section{Experimental results}
\label{sec:expeFGW}

We illustrate in this section the behavior of our method on synthetic and real datasets. The algorithm presented in the previous section have been implemented in the Python Optimal Transport toolbox~\cite{flamary2017pot}.

\subsection{Illustration of FGW on trees}

  We construct two trees {as illustrated in Figure \ref{mapstoy}, where the 1D node features are shown with colors (in red, features belong to $[0,1]$ and in blue in $[9,10]$).}
 The structure similarity matrices $\C_{1}$ and $\C_{2}$ are the shortest-paths between the nodes. Both trees have the same individual structure and the same features up to a permutation. However when combining both informations the trees are not the same, as they do not have the same labels at the same place. Figure \ref{mapstoy} illustrates the behavior of the ${FGW}$ distance when the trade-off parameter $\alpha$ changes. The left part recovers the Wasserstein distance between the features ($\alpha=0$):  red nodes are coupled to red ones and the blue nodes to the blue ones. For an alpha close to $1$ (right), we recover the Gromov-Wasserstein distance between the structures of the trees: all couples of points are coupled to another couple of points, without taking into account the features. Both approaches fail in discriminating the two trees. Finally, for an intermediate $\alpha$ in $FGW$ (center), the bottom and first level structure are preserved as well as the feature matching (red on red and blue on blue), resulting on a positive distance. Note that $FGW$ preserves also the substructures of the trees through its coupling.

\subsection{Illustration of FGW on 1D distributions}
Figure \ref{fig:illus_emp} illustrates the differences between Wasserstein, Gromov-Wasserstein and Fused Gromov-Wasserstein couplings $\GG^{*}$ on 1D distributions. In this example both the feature and structure are 1-dimensional, that is $(\xbf_i,\ybf_j) \in \R^{2}$ and $(\a_i,\b_j) \in \R^{2}$ (Figure \ref{fig:illus_emp} left). The feature space (vertical axis) denotes two clusters among the elements of both objects illustrated in the OT matrix $\Mbf_{\Abf \Bbf}$, the structure space (horizontal axis) denotes a noisy temporal sequence along the indexes illustrated in the matrices $\C_1$ and $\C_2$ (Figure \ref{fig:illus_emp} center). Wasserstein respects the clustering but forgets the temporal structure, Gromov-Wasserstein respects the structure but do not take the clustering into account. Only FGW retrieves a transport matrix respecting both feature and structure.

\begin{figure*}[t]
    \centering
  \includegraphics[height=5.1cm]{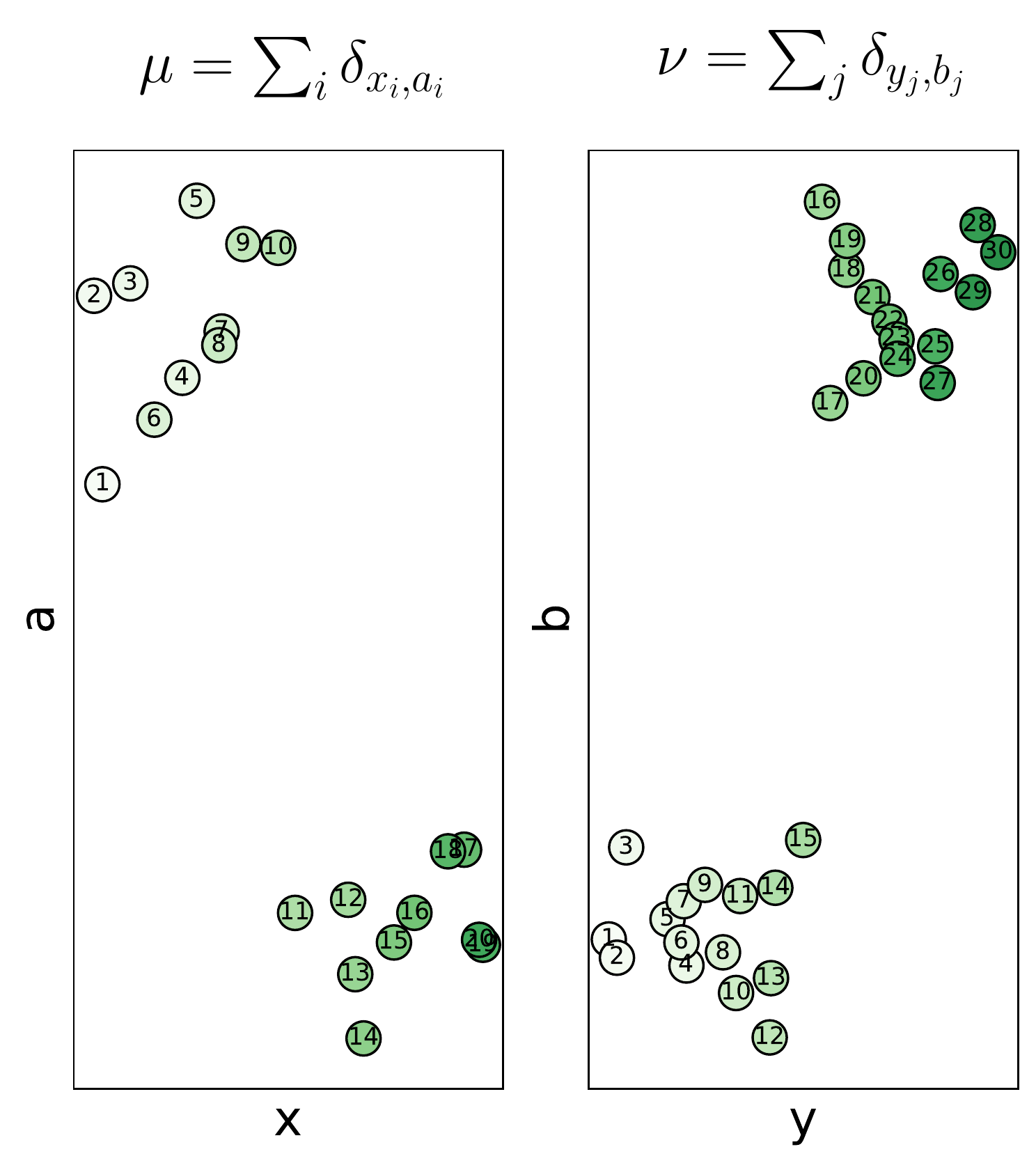}
  \hspace{-2mm}
  \includegraphics[height=4cm]{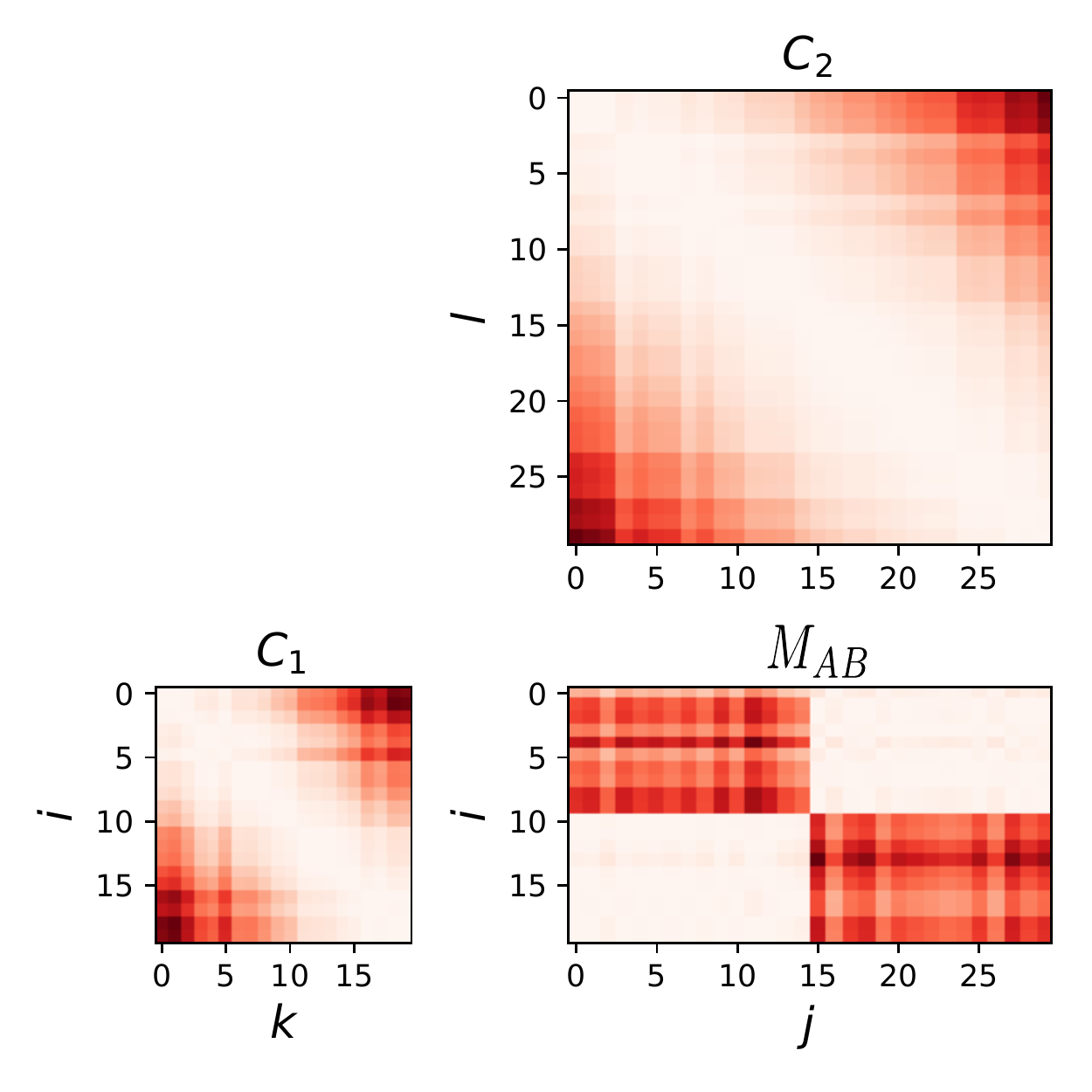}
  \hspace{-2mm}
      \includegraphics[height=5.3cm]{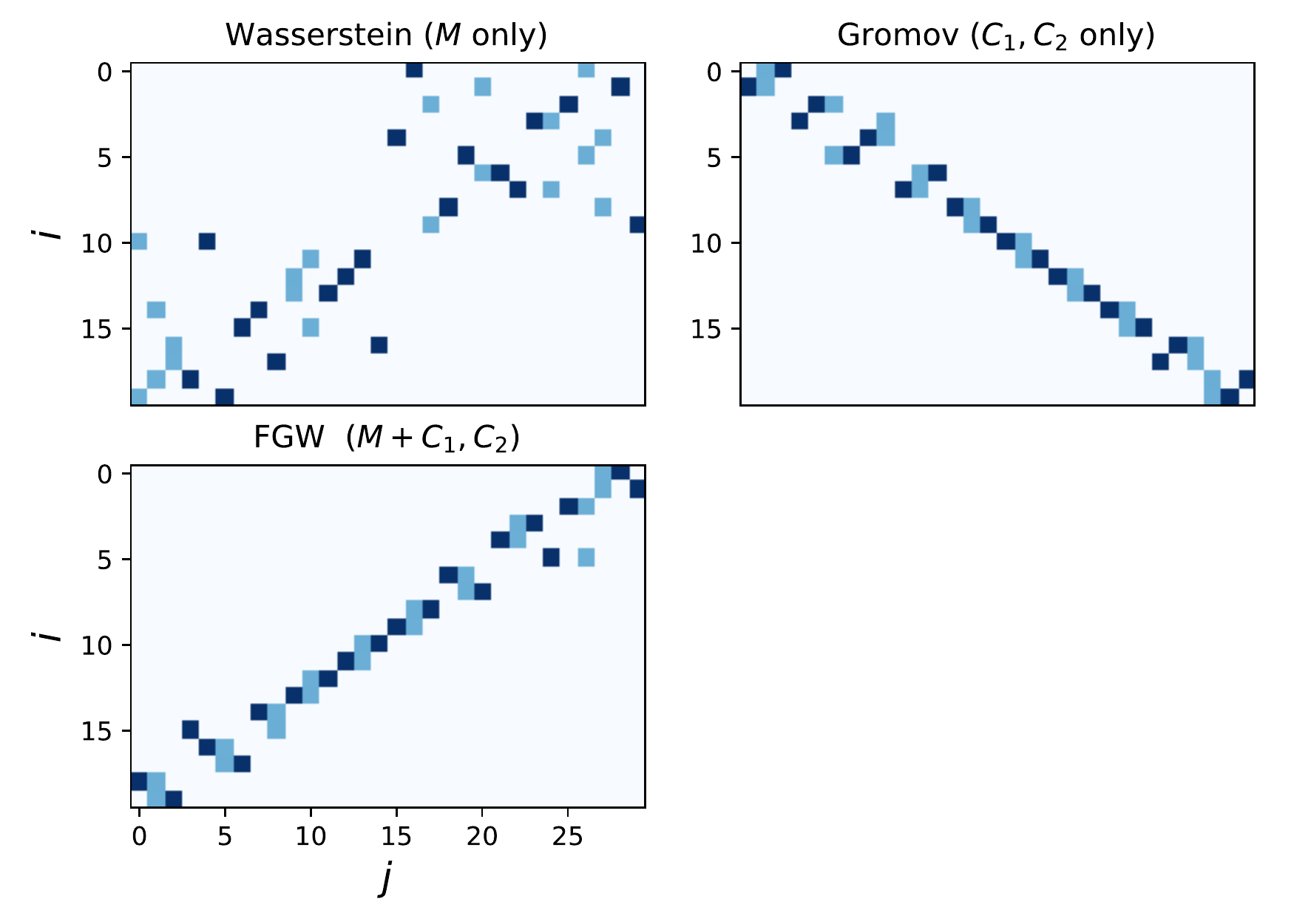}
    \caption{Illustration of the difference between $W$, $GW$ and $FGW$ couplings. {\bf (left)} empirical distributions $\mu$ with 20 samples and $\nu$ with 30 samples which color is proportional to their index. {\bf (middle)} Cost matrices in the feature ($\Mbf_{\Abf \Bbf}$) and structure domains ($\C_1,\C_2$) with similar samples in white. {\bf (right)} Solution for all methods. Dark blue indicates a non zero coefficient of the transportation map. Feature distances are large between points laying on the diagonal of $\Mbf_{\Abf \Bbf}$ such that Wasserstein maps is anti-diagonal but unstructured. Fused Gromov-Wasserstein incorporates both feature and structure maps in a single transport map.}
    \label{fig:illus_emp}
\end{figure*}

\subsection{Illustration of FGW on simple images}

We extract a $28\times28$ image from the MNIST dataset and generate a second one through translation or mirroring of the digit in the original image. We use pixel gray levels as the features, and the structure is defined as the city-block distance on the pixel coordinate grid. We use equal weights for all the pixels in the image. Figure \ref{fig:illus_digit} shows the different couplings obtained when considering either the features only, the structure only or both information. $FGW$ aligns the pixels of the digits, recovering the correct order of the pixels, while both Wassertein and Gromov-Wasserstein distances fail at providing a meaningful transportation map. Note that in the Wasserstein and Gromov-Wasserstein case, the distances are equal to 0, whereas $FGW$ manages to spot that the two images are different. Also note that, in the $FGW$ sense, the original digit and its mirrored version are also equivalent as there exists an isometry between their structure spaces, making $FGW$ invariant to rotations or flips in the structure space in this case.
\begin{figure*}[t]
    \centering
        \includegraphics[height=3cm]{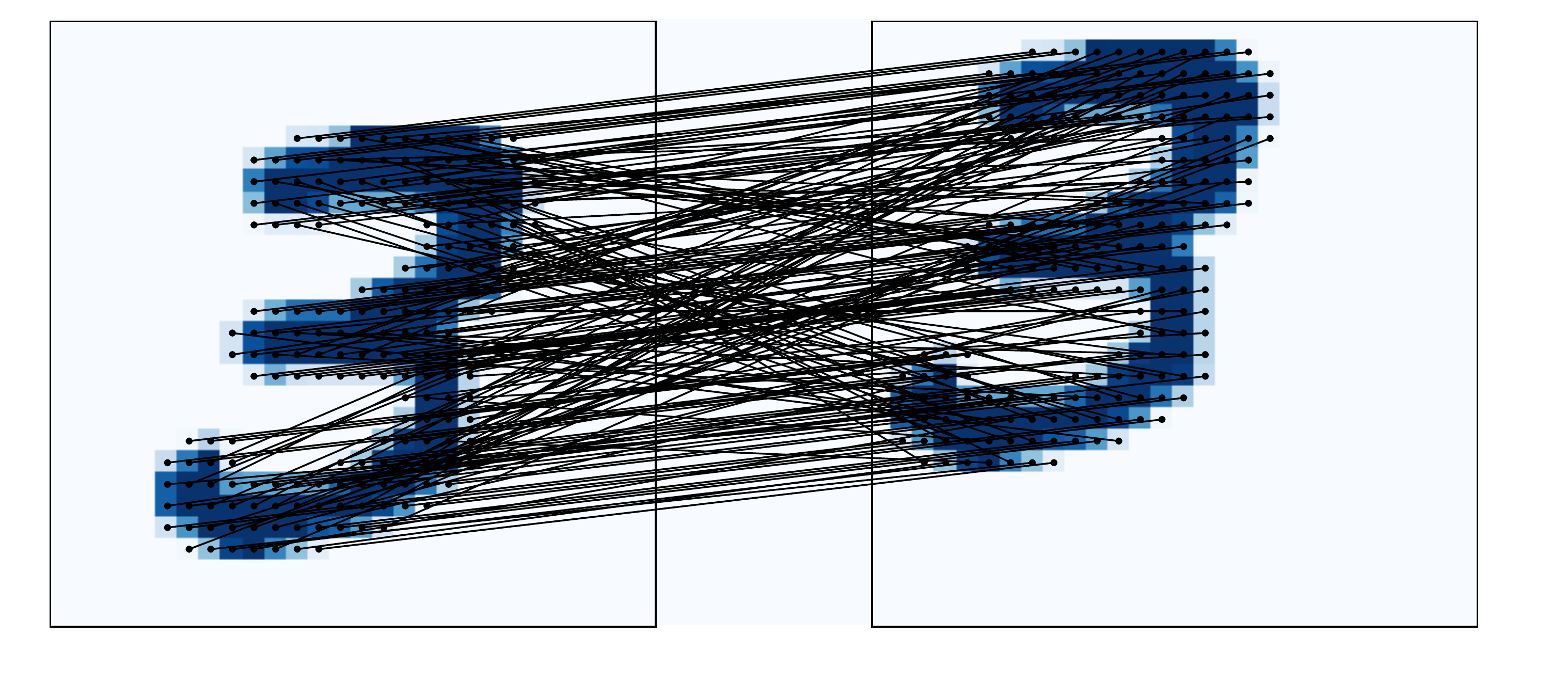}
  \hspace{-2mm}
  \includegraphics[height=3cm]{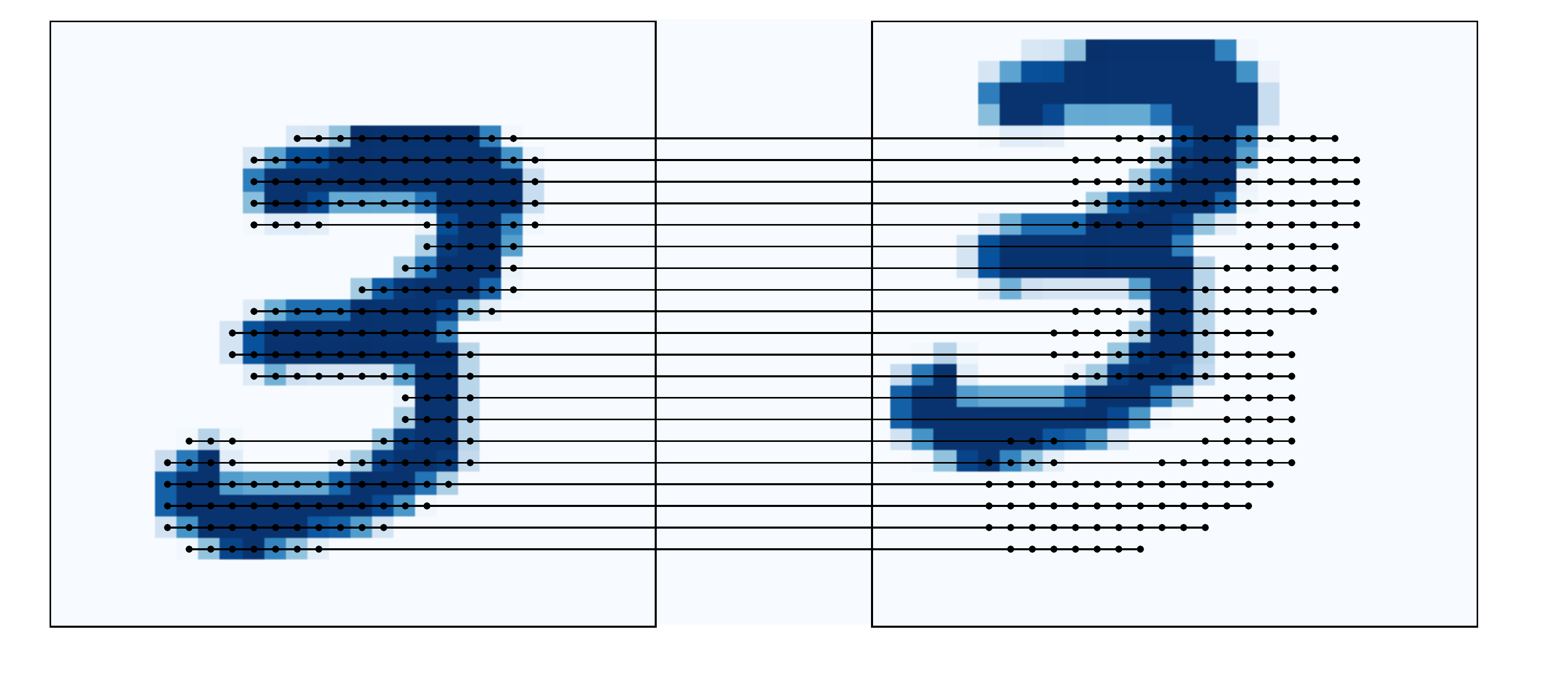}
  \hspace{-2mm}
      \includegraphics[height=3cm]{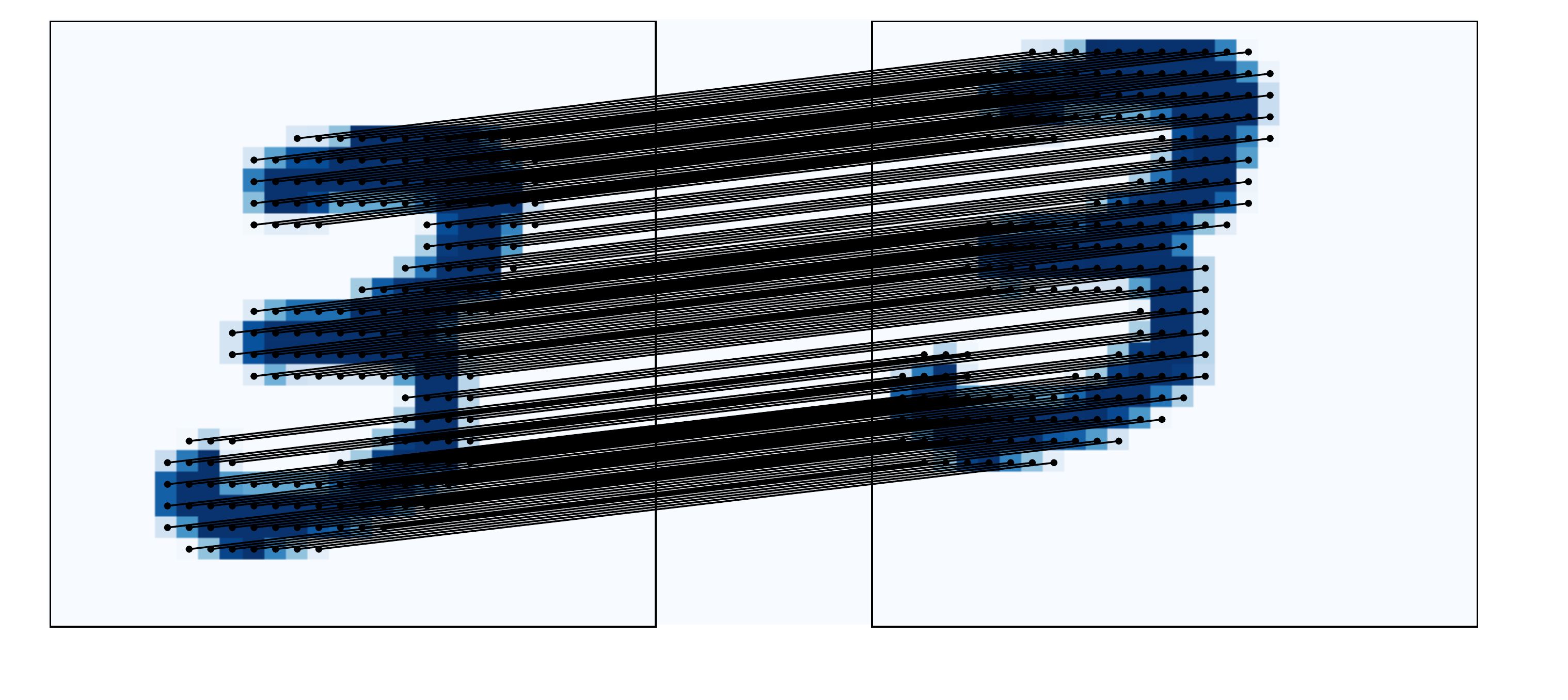}
     \hspace{-2mm}
   \includegraphics[height=3cm]{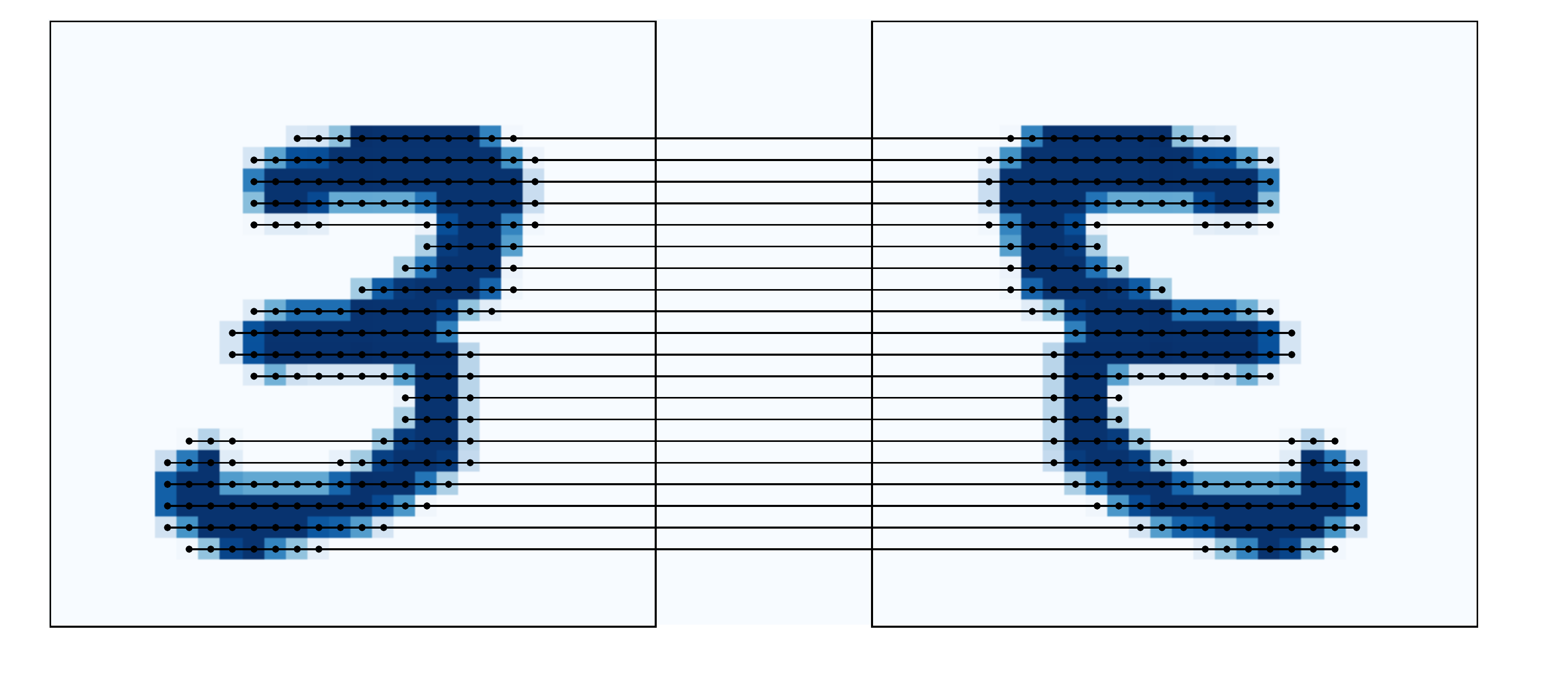}
    \caption{Couplings obtained when considering {\bf (top left)} the features only, where we have $\wass=0$ {\bf (top right)} the structure only, with  $\gw=0$ {\bf (bottom left and right)} both the features and the structure, with $FGW_{0.1,2}$. For readibility issues, only the couplings starting from non white pixels on the left picture are depicted. }
    \label{fig:illus_digit}
\end{figure*}

\subsection{Illustration of FGW on time series data}
One of the main assets of $FGW$ is that it can be used on a wide class of structured data such as graphs and also time series. We consider here 25 monodimensional time series composed of two humps in $[0,1]$ with random uniform height between 0 and 1.  Signals are distributed according to two classes translated from each other with a fixed gap. The $FGW$ distance is computed by considering $\Mbf_{\Abf \Bbf}$ as the Euclidean distance between the features of the signals (here the value of the signal in each point) and $\C_1$ and $\C_2$ as the euclidean distance between timestamps.

 A 2D embedding is computed from a $FGW$ distance matrix between a number of examples in this dataset with multidimensional scaling (MDS) in Figure \ref{mds} (top). One can clearly see that the representation with a reasonable $\alpha$ value in the center is the most discriminant one. This can be better understood by looking as the OT matrices between the classes.
Figure \ref{mds} (bottom) illustrates the behavior of $FGW$ on one pair of examples when going from  Wasserstein to Gromov-Wasserstein. The black line depicts the matching provided by the transport matrix and one can clearly see that while Wasserstein on the left assigns samples completely independently to their temporal position, the Gromov-Wasserstein on the right tends to align perfectly the samples (note that it could have reversed exactly the alignment with the same loss) but discards the values in the signal. Only the true $FGW$ in the center finds a transport matrix that both respects the time sequences and aligns similar values in the signals.

\begin{figure*}[t]
\centering
\includegraphics[width=1\linewidth]{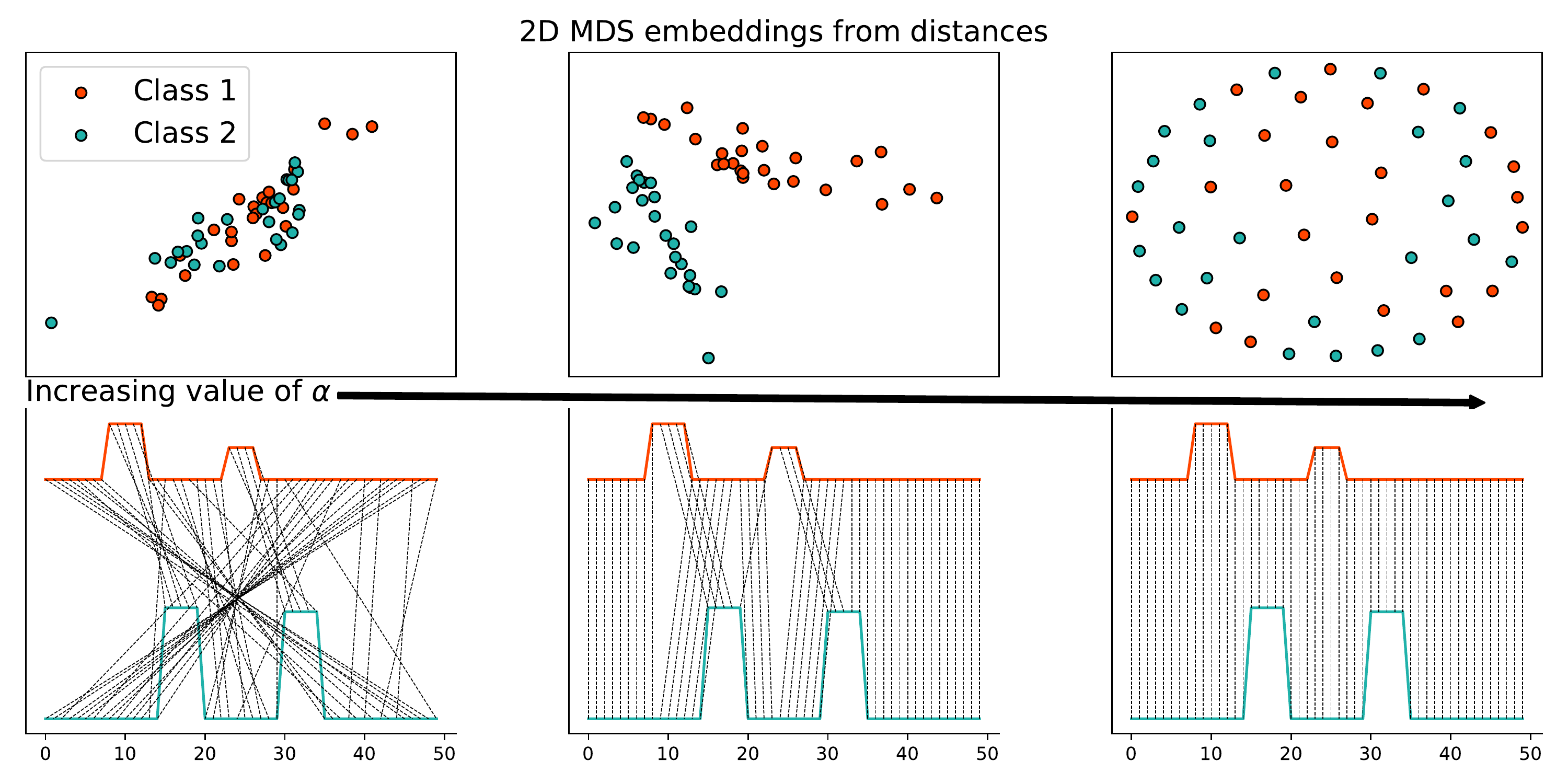}
\caption{Behavior of trade-off parameter $\alpha$ on a toy time series classification problem. $\alpha$ is increasing from left ($\alpha=0$ : Wasserstein distance) to right ($\alpha=1$ : Gromov-Wasserstein distance). {\bf (top row)} 2D-embedding is computed from the set of pairwise distances between samples with MDS {\bf (bottom row)} illustration of couplings between two sample time series from opposite classes.}
\label{mds}
\end{figure*}

\subsection{Graph-structured data classification}

We now use $FGW$ on real-world dataset where we study its behavior on a graph classification task. More precisely we address the question of training a classifier for graph data and evaluate the FGW distance used in a kernel with SVM.

\paragraph{Datasets}  {We consider 12 widely used benchmark datasets divided into 3 groups. BZR, COX2 \cite{cox2bzr},  PROTEINS, ENZYMES  \cite{enzymes}, CUNEIFORM \cite{kriege.cuneiform} and SYNTHETIC \cite{NIPS2013_5155} are vector attributed graphs. MUTAG  \cite{doi:10.1021/jm00106a046}, PTC-MR \cite{DBLP:journals/corr/KriegeGW16} and NCI1 \cite{nci1} contain graphs with discrete attributes derived from small molecules.   IMDB-B, IMDB-M \cite{Yanardag15} contain unlabeled graphs derived from social networks. All datasets are available in \cite{KKMMN2016}. }

\paragraph{Experimental setup} {Regarding the feature distance matrix $\Mbf_{\Abf \Bbf}$
between node features, when dealing with real valued vector attributed graphs, we consider the $\ell_{2}$
 distance between the labels of the vertices. 
In the case of graphs with discrete attributes, we consider two settings: in the first one, we keep the original labels (denoted as \textsc{raw}); we also consider a Weisfeiler-Lehman labeling (denoted as \textsc{wl}) by concatenating the  labels of the neighbors. A vector of size \textsc{h} is created by repeating this procedure \textsc{h} times ~\cite{wlkernel,DBLP:journals/corr/KriegeGW16}. In both cases, we compute the feature distance matrix by using $d(a_{i},b_{j})=\sum_{k=0}^{H} \delta(\tau(a_{i}^{k}),\tau(b_{j}^{k}))$ where $\delta(x,y)=1$ if $x\neq y$ else $\delta(x,y)=0$ and $\tau(a_{i}^{k})$ denotes the concatenated label at iteration $k$ (for $k=0$ original labels are used). Regarding the structure distances $C$, they are computed by considering a
shortest path distance between the vertices. 
  
  For the classification task, we run a SVM using the indefinite kernel matrix $e^{-\gamma {FGW}}$ which
 is seen as a noisy observation of the true positive semidefinite kernel
 \cite{Luss:2007:SVM:2981562.2981682}. We compare classification accuracies with the following
 state-of-the-art graph kernel methods:}  (SPK) denotes the shortest path kernel
 \cite{enzymes}, (RWK) the random walk kernel \cite{Gartner03ongraph}, (WLK) the
 Weisfeler Lehman kernel \cite{wlkernel}, (GK) the graphlet count kernel
 \cite{Shervashidze09efficientgraphlet}. For real valued vector attributes, we
consider the HOPPER kernel (HOPPERK) \cite{NIPS2013_5155} and the propagation kernel
 (PROPAK) \cite{Neumann2016}. We build upon the GraKel library \cite{2018arXiv180602193S} to construct the kernels and C-SVM to perform the classification. We
 also compare $FGW$ with the PATCHY-SAN framework for  CNN on graphs  (PSCN)
 \cite{pmlr-v48-niepert16} building on our own implementation of the method\footnote{\url{https://github.com/tvayer/PSCN}}.

{To compare the methods, most papers about graph classification usually perform a nested cross validation (using 9 folds for training, 1 for testing, and reporting the average accuracy of this experiment repeated 10 times) and report accuracies of the other methods taken from the original papers. However, these comparisons are not fair because of the high variance on most datasets \textit{w.r.t.} the folds chosen for training and testing. This is why, in our experiments, the nested cross validation is performed on the same folds for training and testing for \emph{all} methods. In the result tables \ref{tab:vec}, \ref{tab:disc} and \ref{tab:no} we add a (*) when the best score does not yield to a significative improvement (based on a Wilcoxon signed rank test on the test scores) compared to the second best one. Note that, because of their small sizes, we repeat the experiments 50 times for MUTAG and PTC-MR datasets. For all methods using SVM, we cross validate the parameter $C \in \{10^{-7},10^{-6},...,10^{7}\}$. The range of the WL parameter \textsc{h} is $\{0,1...,10\}$, and we also compute this kernel with \textsc{h} fixed at $2,4$. The decay factor $\lambda$ for RWK $\{10^{-6},10^{-5}...,10^{-2}\}$, for the GK kernel we set the graphlet size $\kappa =3$ and cross validate the precision level $\epsilon$ and the confidence $\delta$ as in the original paper \cite{Shervashidze09efficientgraphlet}. The $t_{\text{max}}$ parameter for PROPAK is chosen within $\{1,3,5,8,10,15,20\}$. For PSCN, we choose the normalized betweenness centrality as labeling procedure and cross validate the batch size in $\{10, 15,...,35\}$ and number of epochs in $\{10, 20,..., 100\}$. Finally for ${FGW}$, $\gamma$ is cross validated within $\{2^{-10},2^{-9},...,2^{10}\}$ and $\alpha$ is cross validated \textit{via} a logspace search in $[0,0.5]$ and symmetrically $[0.5,1]$ (15 values are drawn).

\subsubsection{Results and discussion}

\paragraph{Vector attributed graphs.} The average accuracies reported in Table \ref{tab:vec} show that FGW is a clear state-of-the-art method and
performs best on 4 out of 6 datasets with performances in the error bars of the
best methods on the other two datasets. Results for CUNEIFORM are significantly below those from the original paper \cite{kriege.cuneiform} which can be explained by the fact that the method in this paper uses a graph convolutional approach specially designed for this dataset and that the experimental setting is different. In comparison, the other competitive methods are less consistent as they exhibit some good performances on some datasets only.

\begin{table*}[t]
    \caption{Average classification accuracy on the graph datasets with vector attributes.\label{tab:vec}}
    \vspace{1.5mm}
    \begin{center}
\resizebox{1\linewidth}{!}{
\begin{sc}
    \setlength{\tabcolsep}{4pt}
\begin{tabular}{lllllll}
\toprule
{Vector attributes} &                   BZR &          COX2 &     CUNEIFORM &       ENZYMES &        PROTEIN &      SYNTHETIC \\
\midrule
\midrule
\small{FGW}             &  \textbf{85.12$\pm$4.15}* &  77.23$\pm$4.86 &  \textbf{76.67$\pm$7.04} &  71.00$\pm$6.76 &   \textbf{74.55$\pm$2.74} &  \textbf{100.00$\pm$0.00} \\
\midrule
\small{HOPPERK}               &  84.15$\pm$5.26 &  \textbf{79.57$\pm$3.46} &  32.59$\pm$8.73 &  45.33$\pm$4.00 &   71.96$\pm$3.22 &   90.67$\pm$4.67 \\
\small{PROPAK}                 &   79.51$\pm$5.02 &  77.66$\pm$3.95 &  12.59$\pm$6.67 &  \textbf{71.67$\pm$5.63}* &   61.34$\pm$4.38 &   64.67$\pm$6.70 \\\midrule
\small{PSCN k=10}                                 &  80.00$\pm$4.47 &  71.70$\pm$3.57 &  25.19$\pm$7.73 &  26.67$\pm$4.77 &  67.95$\pm$11.28 &  \textbf{100.00$\pm$0.00} \\
\small{PSCN k=5 }                                 &  82.20$\pm$4.23 &  71.91$\pm$3.40 &  24.81$\pm$7.23 &  27.33$\pm$4.16 &   71.79$\pm$3.39 &  \textbf{100.00$\pm$0.00} \\
\bottomrule
\end{tabular}
\end{sc}}
\end{center}
\end{table*}

\paragraph{Discrete labeled graphs.} We first note in Table \ref{tab:disc} that $FGW$ using WL attributes outperforms all competitive methods, including $FGW$ with raw features. Indeed, the WL attributes allow encoding more finely the neighborood of the vertices by stacking their attributes, whereas FGW with raw features only consider the shortest path distance between vertices, not their sequence of labels. This result calls for using meaningful feature and/or structure matrices in the FGW definition, that can be dataset-dependant, in order to enhance the performances. We also note that $FGW$ with WL attributes outperforms the WL kernel method, highlighting the benefit of an optimal transport-based distance over a kernel-based similarity. Surprisingly results of PSCN are significantly lower than those from the original paper. We believe that it comes from the difference between the folds assignment for training and testing, which suggests that PSCN is difficult to tune.

\begin{table}[t]
    \caption{Average classification accuracy on the graph datasets with discrete attributes.\label{tab:disc}}
    \vspace{1.5mm}
\begin{center}
    \resizebox{0.7\linewidth}{!}{
\begin{sc}
    \setlength{\tabcolsep}{4pt}
\begin{tabular}{llll}
\toprule
{Discrete attr.} &          MUTAG &          NCI1 &           PTC-MR \\
\midrule
\midrule
FGW raw               &  83.26$\pm$10.30 &     72.82$\pm$1.46 &  55.71$\pm$6.74 \\
FGW wl h=2       &   86.42$\pm$7.81 &          85.82$\pm$1.16 &  {63.20$\pm$7.68} \\
FGW wl  h=4       &   \textbf{88.42$\pm$5.67} &          \textbf{86.42$\pm$1.63} & \textbf{65.31$\pm$7.90} \\
\midrule
GK k=3                              &   82.42$\pm$8.40 &    60.78$\pm$2.48 &  56.46$\pm$8.03 \\
RWK             &   79.47$\pm$8.17 &         58.63$\pm$2.44 &  55.09$\pm$7.34 \\
SPK             &   82.95$\pm$8.19 &  74.26$\pm$1.53 &          60.05$\pm$7.39 \\
WLK               &   86.21$\pm$8.48 &          85.77$\pm$1.07 &  62.86$\pm$7.23 \\
WLK h=2                              &   86.21$\pm$8.15 &     81.85$\pm$2.28 &  61.60$\pm$8.14 \\
WLK h=4                              &   83.68$\pm$9.13 &     85.13$\pm$1.61 &
62.17$\pm$7.80 \\\midrule
PSCN k=10                           &  83.47$\pm$10.26 &          70.65$\pm$2.58 &  58.34$\pm$7.71 \\
PSCN k=5                            &  83.05$\pm$10.80 &          69.85$\pm$1.79 &  55.37$\pm$8.28 \\
\bottomrule
\end{tabular}
\quad
\end{sc}}
\end{center}
\end{table}

\paragraph{Non-attributed graphs.} The particular case of the GW distance for graph classification is also illustrated on social
datasets, that contain no labels on the vertices. Accuracies reported in Table \ref{tab:no} show that it greatly
outperforms SPK and GK graph kernel methods.

\begin{table}[t]
    \caption{Average classification accuracy on the graph datasets with no attributes.\label{tab:no}}
    \vspace{1.5mm}
    \begin{center}
    \begin{sc}
    \begin{tabular}{lll}
    \toprule
    {Without attribute} &        IMDB-B &       IMDB-M \\
    \midrule
    \midrule
    GW  &  \textbf{63.80$\pm$3.49} &  \textbf{48.00$\pm$3.22} \\
    \midrule
    GK k=3                 &  56.00$\pm$3.61 &  41.13$\pm$4.68 \\
    SPK &  55.80$\pm$2.93 &  38.93$\pm$5.12 \\
    \bottomrule
    \end{tabular}
    \end{sc}
    \end{center}
\end{table}

\paragraph{Comparison between $FGW$, $W$ and $GW$} During
the validation step, the optimal value of $\alpha$ was consistently selected
inside the $]0,1[$ interval, excluding $0$ and $1$, suggesting that both structure and feature pieces of
information are necessary (details are given in Section \ref{sec:compa}).

\subsection{Graph barycenter and compression}

In this section we use the barycentric formulation of $FGW$ using the Fréchet mean formulation in equation \eqref{eq:frechet_mean_fgw} in two settings. The first one is the computation of barycenter of several toy labeled graphs and the second the compression of one graph into a smaller graph, also known as coarsening \cite{loukas_coarsening}.

\begin{figure}[t]
        \centerline{\includegraphics[width=1.2\textwidth]{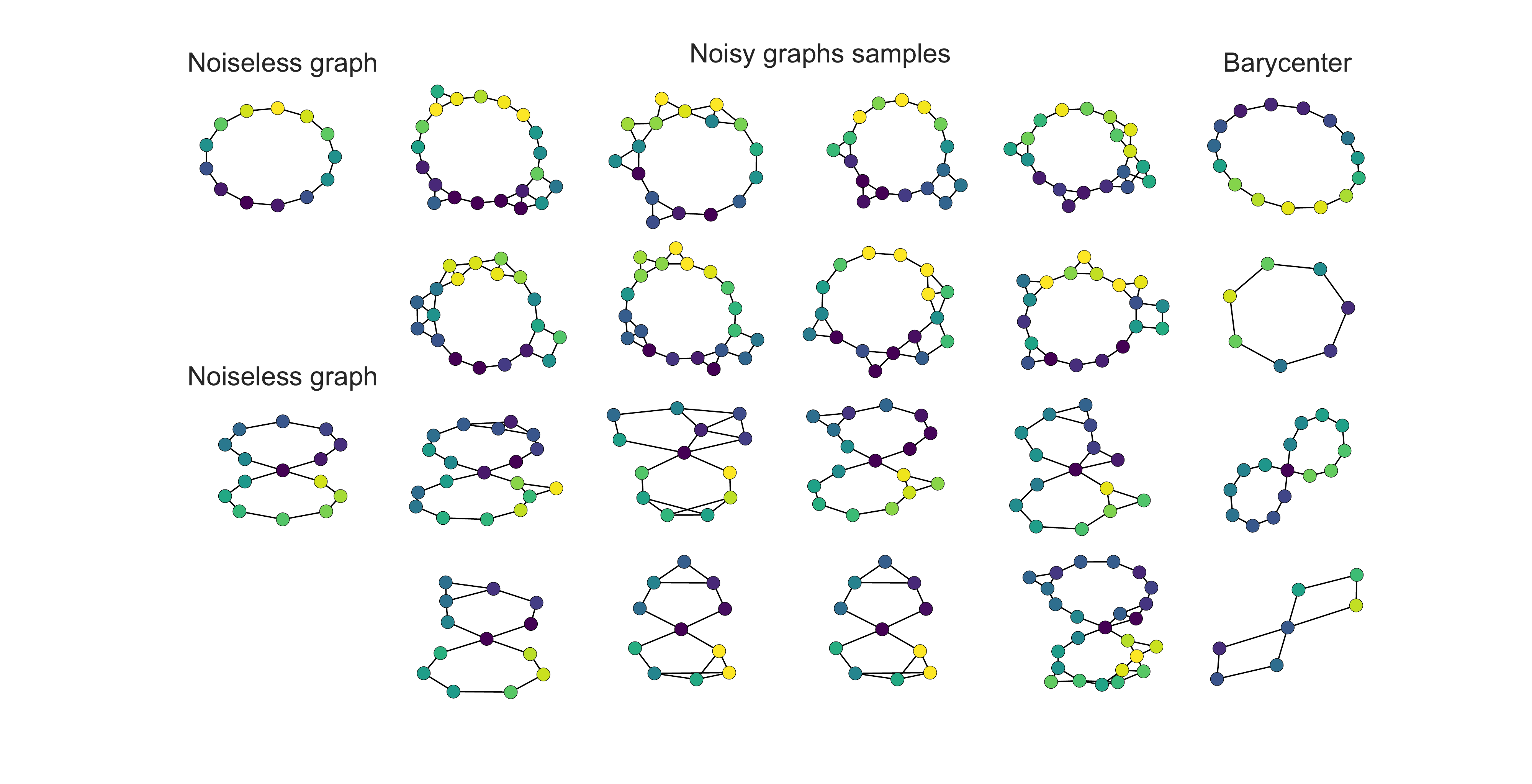}}
    \caption{\label{barygraph}Illustration of $FGW$ graph barycenter. The first column illustrates the original settings with the noiseless graphs, and columns 2 to 5 are the noisy samples that constitute the datasets. Column 6 show the barycenters for each setting, with different number of nodes. Blue nodes indicates a feature value close to $-1$, yellow nodes close to $1$.  
    }
    \end{figure}

\paragraph{Graph barycenter} In this part, we use $FGW$ to compute barycenter of toy graphs. In a first example, we generate graphs following either a circle or $8$ symbol with 1D features following a sine and linear variation respectively.
For each example, the number of nodes is drawn randomly between 10 and 25, Gaussian noise is added to the features and a small noise is applied to the structure (some connections are randomly added). An example graph  with no noise is provided for each class in the first column of Figure \ref{barygraph}. One can see from there that the circle class has a feature varying smoothly (sine) along the graph but the $8$ has a sharp feature change at its center (so that low pass filtering would loose some information). Some examples of the generated graphs are provided in the 2nd-to-7th columns of Figure \ref{barygraph}. We compute the $FGW$ barycenter containing 10 samples using the shortest path distance between the nodes as the structural information and the distance induced by the Euclidean norm $\|.\|_{2}$ for the features. 

Note that the iterations of the barycenter defined in equation \eqref{eq:cmatrix} result in a dense $\C$ matrix and to visualize properly the graph barycenter we need a adjacency matrix. We propose a simple heuristic procedure to recover an adjacency matrix for the graphs' barycenter based on a thresholding of the matrix $\C$. Given a threshold $t$ the matrix $thresh_t(\C)$ is given by $1$ if $C_{ij}<=t$ and $0$ elsewhere. The threshold $t$ is tuned so as to minimize the Frobenius norm between the original $\C$ matrix and the shortest path matrix constructed after thresholding $\C$. More precisely if $SP$ denotes the algorithm which takes as input an adjacency matrix and outputs a shortest-path matrix then the threshold is given:
\begin{equation}
\label{eq:thresh_procedure}
\argmin_{t \in \R_{+}} \|\C-SP(thresh_t(\C))\|_{F}^{2}
\end{equation} 
The idea behind equation \eqref{eq:thresh_procedure} is that $\C$ represents somehow the shortest-path matrix of a graph so that we want the adjacency matrix $thresh_t(\C)$ to give a shortest-path matrix as close as possible to $\C$. Unfortunately equation \eqref{eq:thresh_procedure} is not differentiable with respect to $t$. We use a simple brute force strategy to find a suitable threshold $t$ by looking at $\argmin_{t \in \{t_1,\dots,t_L\}} \|\C-SP(thresh_t(\C))\|_{F}^{2}$ where $t_1,\dots,t_L$ are drawn from $\R_{+}$.

Resulting barycenters are showed in Figure \ref{barygraph} for $n=15$ and $n=7$ nodes. First, one can see that the barycenters are denoised both in the feature space and the structure space. Also note that the sharp change at the center of the $8$ class is conserved in the barycenters which is a nice result compared to other divergences that tend to smooth-out their barycenters ($\ell_2$ for instance). Finally, note that by selecting the number of nodes in the barycenter one can compress the graph or estimate a ``high resolution'' representation from all the samples. To the best of our knowledge, no other method can compute such graph barycenters. Finally, note that $FGW$ is interpretable because the resulting OT matrix provides correspondence between the nodes from the samples and those from the barycenter.

\begin{figure}[t]
    \begin{center}
    \includegraphics[width=.9\linewidth]{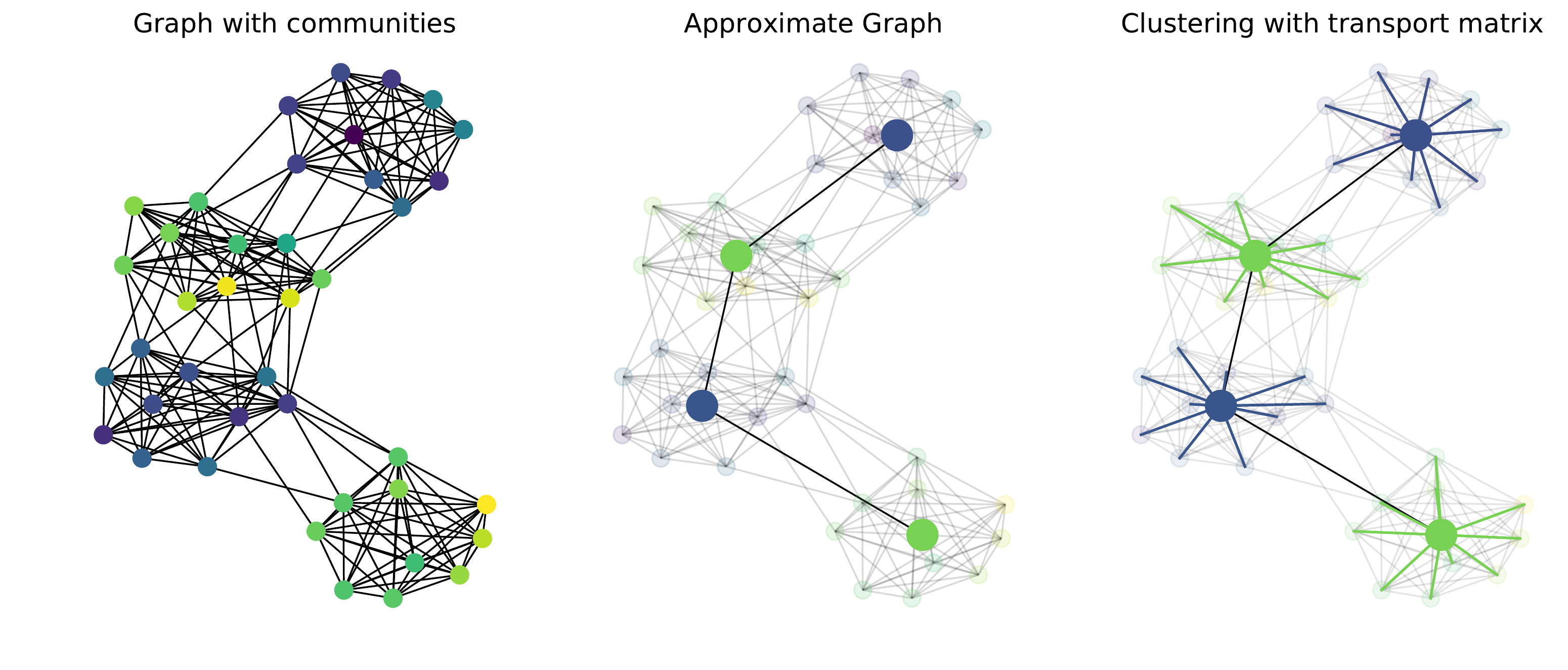}
    \includegraphics[width=.9\linewidth]{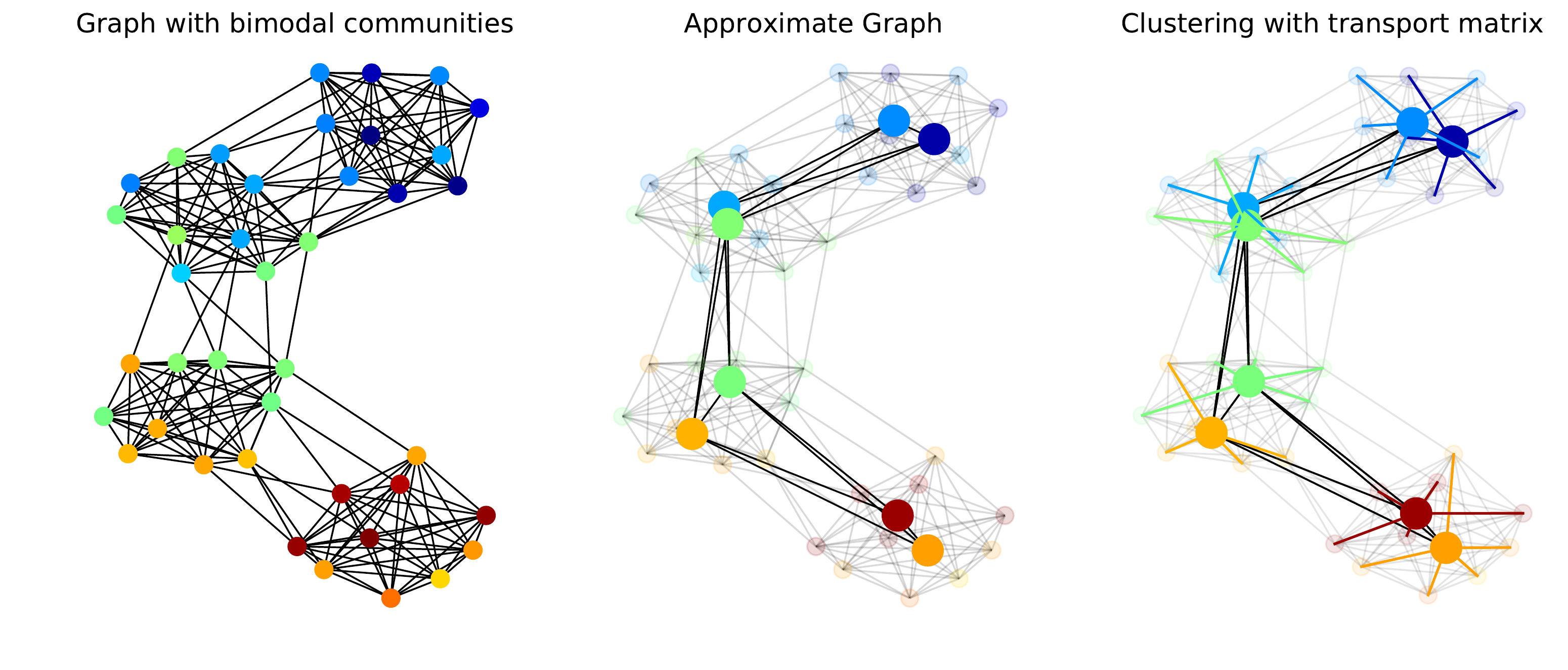}
    \caption{\label{fig:graphclustr} Example of community clustering on graphs using $FGW$. {\bf (top)} Community clustering with 4 communities and uniform features per cluster. {\bf (bottom)}  Community clustering with 4 communities and bimodal features per cluster (and two nodes per cluster in the approximate graph).}
    \end{center}
\end{figure}

\paragraph{Graph compression} In the second experiment, we evaluate the ability of FGW to perform graph approximation and compression on a Stochastic Block Model graph \cite{wang1987stochastic,nowicki2001estimation}. The question is to see if estimating an approximated graph can recover the relation between the blocks and perform simultaneously a community clustering on the original graph (using the coupling matrix $\GG$). We generate two community graphs illustrated in the left column of Figure \ref{fig:graphclustr}. The coarsened graph is obtained by solving the Fréchet mean formulation with $k=1$. More precisely given an original graph $\mu_{0}$ described by its features $\Bbf_{0}$ and its structure $\C_{0}$ we look for a graph of $N$ nodes, with $N$ smaller than the number of nodes of the original graph, which solves:
\begin{equation}
\underset{\mu}{\text{min}} \ FGW_{q,\alpha}(\mu,\mu_{0}) =\underset{\C \in \mathbb{R}^{N \times N},\ \Abf \in \mathbb{R}^{N \times n},\GG}{\text{min}} E_{q}(\Mbf_{\Abf \Bbf_{0}},\C,\C_{0},\GG)
\end{equation}
The results are depicted in Figure \ref{fig:graphclustr}. We can see that the relation between the blocks is sparse and has a ``linear'' structure, the example in the first line has features that follow the blocks (noisy but similar in each block) whereas the example in the second line has two modes per block. The first graph approximation (top line) is done with $N=4$ nodes and we can recover both the blocks in the graph and the average feature on each blocks (colors on the nodes). The second problem is more complex due to the two modes per  block but one can see that when approximating the graph with $N=8$ nodes we recover both the structure between the blocks and the sub-clusters in each block, which illustrates the strength of $FGW$ that encodes both features and structures.

\subsection{Unsupervised learning: graphs clustering}

{In the last experiment, we evaluate the ability of $FGW$ to perform a
clustering of multiple graphs and to retrieve meaningful barycenters of such
clusters. To do so, we generate a dataset of 4 groups of community graphs. Each
graph follows a simple Stochastic Block Model
\cite{wang1987stochastic,nowicki2001estimation} and the groups are defined 
\textit{w.r.t.} the number of communities inside each graph and the distribution
of their labels. The dataset is composed of 40 graphs (10 graphs per group) and
the number of nodes of each graph is drawn randomly from $\{20,30,...,50\}$
as illustrated in Figure \ref{fig:graphclustr2}. We perform a $k$-means clustering using the ${FGW}$ barycenter defined
in equation \eqref{eq:frechet_mean_fgw} as the centroid of the groups and the ${FGW}$
distance for the cluster assignment. We fix the number of nodes of each centroid to 30.
We perform a thresholding on the pairwise similarity matrix $C$ of the centroid at
the end in order to obtain an adjacency matrix for visualization purposes. The threshold value is empirically chosen with the procedure described in the previous section.   
The evolution of the barycenters along the iterations is reported in  Figure~\ref{fig:graphclustr2}. We can see that these centroids recover
community structures and feature distributions that are representative of their
cluster content. On this example, note that the clustering recovers perfectly
the known groups in the dataset.
To the best of our knowledge, there exists no other method able to perform a clustering of
graphs and to retrieve the average graph in each cluster without having to solve a
pre-image problem.

\begin{figure}[t!]
    \begin{center}
        \includegraphics[width=0.46\linewidth]{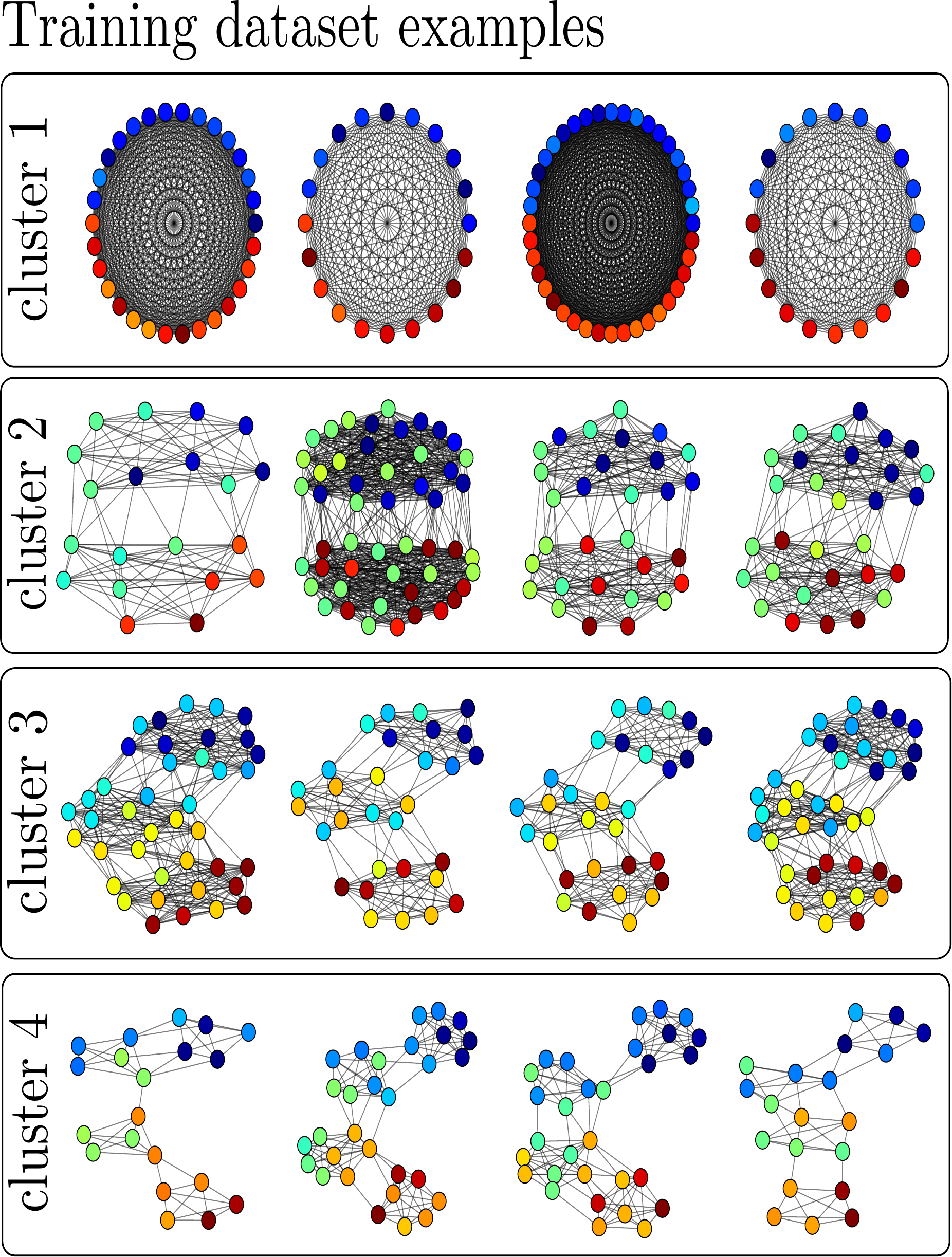}
    	\includegraphics[width=0.53\linewidth]{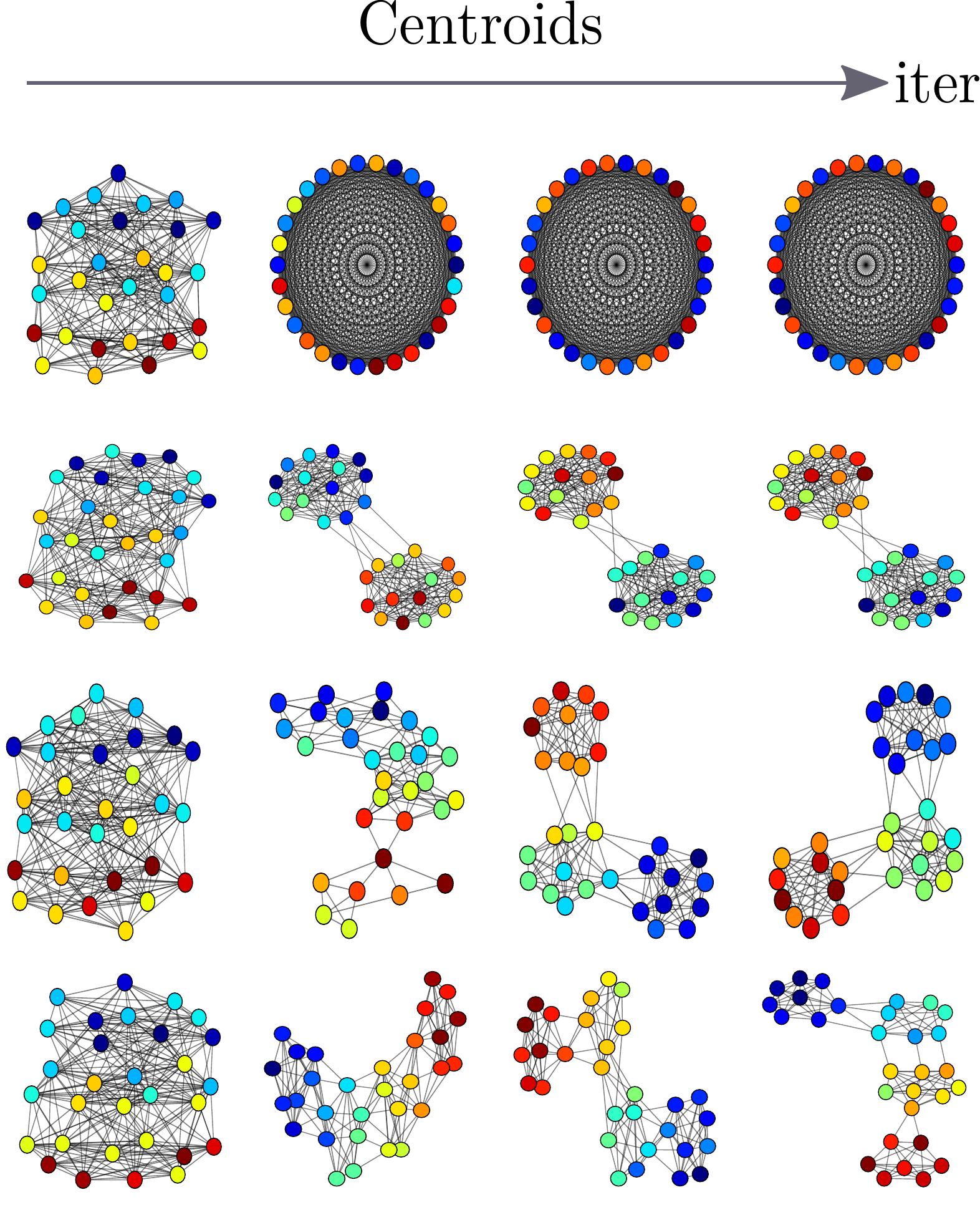}

    \end{center}
\caption{\label{fig:graphclustr2} {\bf (left)} Examples from the clustering dataset, color indicates the labels. {\bf (right)} Evolution of the centroids of each cluster in the $k$-means clustering, from the random initialization until convergence to the final centroid.  }
\end{figure}

\subsection{Other applications of Fused Gromov-Wasserstein}

Since its introduction the Fused Gromov-Wasserstein distance was also used in various contexts where structured data are involved. It has found applications in molecular biology for the analysis of Single-cell RNA (scRNA-seq) sequencing data in \cite{cang_inferring_2020} where authors use $FGW$ as a building block to recover spatial properties of scRNA-sequation In machine learning $FGW$ was used for learning structured autoencoders \cite{xu2020learning}, to study the continuity of Graph Neural Network \cite{bthune2020hierarchical} or on domain adaptation tasks on functional Near-Infrared
Spectroscopy (fNIRS) data \cite{lyu2020domain}. The Fused Gromov-Wasserstein approach was further analysed in \cite{barbe:hal-02795056} where authors propose to improve the $FGW$ distance using a smoothing strategy on the features of the graphs. They propose to incorporate a diffusion kernel on the features which result in a more robust similarity measures of labeled graphs. Authors tackle domain adaptation tasks between labeled graphs where the label information is only available in a target domain. $FGW$ was also used on point clouds in \cite{puy2020flot} for the estimation of scene flows \ie\ of the 3D motion of points at the surface of objects in a scene.

\section{FGW in the continuous setting \label{sec:continuous_fgw}}
\label{sec:continuousfgw}
The previous graph representation for objects with a finite number of points/vertices extends naturally to the continuous setting. The purpose of this section is to generalize $FGW$ to general probability distributions and to state some of its mathematical properties. We consider the following definition of structured objects:

\begin{definition}[\tv{Structured objects}]
\label{def:struct_objects}
A structured object over a metric space $(\Omega,d)$ is a triplet $(\Xcal \times \Omega, d_{\Xcal}, \mu)$, where: $(\Xcal,d_{\Xcal})$ is a metric space and $\mu$ is a \tv{probability measure} over $\Xcal \times \Omega$. $(\Omega,d)$ is denoted as the \emph{feature space}, such that $d:\Omega \times \Omega \rightarrow \mathbb{R}_{+}$ is the distance in the feature space and $(\Xcal,d_{\Xcal})$ the \emph{structure space}, such that $d_{\Xcal}: \Xcal \times \Xcal \rightarrow \mathbb{R}_{+}$ is the distance in the structure space. We will note $\mu_{X}$ and $\mu_{A}$ the structure and feature marginals of $\mu$.
\end{definition}

\begin{definition}[\tv{Space of structured objects}]
We note $\mathbb{X}$ the set of all metric spaces. The space of all structured objects over $(\Omega,d)$ will be written as $\mathbb{S}(\Omega)$ and is defined by all the triplets $(\Xcal \times \Omega, d_{\Xcal}, \mu)$ where $(\Xcal,d_\Xcal) \in \mathbb{X}$ and $\mu \in \Pcal(\Xcal \times \Omega)$.
To avoid finiteness issues we define for $p\in \mathbb{N}^{*}$ the space $\mathbb{S}_{p}(\Omega)\subset \mathbb{S}(\Omega)$, $(\Xcal \times \Omega, d_\Xcal, \mu) \in \mathbb{S}_{p}(\Omega)$ if:
\begin{equation}
\int_{\Omega} d(a,a_0)^{p} \dr \mu_A(a)<+\infty
\end{equation}
(the finiteness of this integral does not depend on the choice of $a_0$)
\begin{equation}
\int_{\Xcal\times \Xcal} d_{\Xcal}(x,x')^{p} \dr \mu_X(x)d \mu_X(x')<+\infty.
\end{equation}
\end{definition}

For the sake of simplicity, and when it is clear from the context, we will sometimes denote only by $\mu$ the whole structured object. In the same way as the discrete case we will note $\mu_{X}$ and $\mu_{A}$ the structure and feature marginals of $\mu$. We recall that those marginals encode very partial information since they focus only on independent feature distributions or only on the structure. \tv{This definition encompasses the discrete setting discussed in above. More precisely let us consider a labeled graph of $n$ nodes with features $A=(a_{i})_{i=1}^{n}$ with $a_i \in \Omega$ and $\Xcal=(x_{i})_{i=1}^{n}$ the structure representation of the nodes. Let $(h_i)_{i=1}^{n}$ be an histogram, then the probability measure $\mu= \sum_{i=1}^{n} h_{i} \delta_{(x_{i},a_{i})}$ defines structured object in the sense of Definition \ref{def:struct_objects} since it lies in $\Pcal(\tv{\Xcal \times \Omega})$.} In this case, an example of $\mu$, $\mu_{X}$ and $\mu_{A}$ is provided in Figure \ref{graphex}. 

\tv{Note that the set of structured objects is quite general and allows also considering discrete probability measures of the form $\mu= \sum_{i,j=1}^{p,q} h_{i,j} \delta_{(x_{i},a_{j})}$ with $p,q$ possibly different than $n$. We propose to focus on a particular type of structured objects, namely the \emph{generalized labeled graphs} as described in the following definition: }

\begin{definition}[\tv{Generalized labeled graph}]
We call \emph{generalized labeled graph} a structured object $(\Xcal \times \Omega, d_{\Xcal}, \mu) \in \mathbb{S}_{p}(\Omega)$ such that $\mu$ can be expressed as $\mu=(id\times \ell_f)\# \mu_X$ where $\ell_f: \Xcal \rightarrow \Omega$ is surjective and pushes $\mu_X$ forward to $\mu_A$, \textit{i.e.} $\ell_f\# \mu_X=\mu_A$. 
\end{definition}
\tv{This definition implies that there exists a function $\ell_f$ which associates a feature $a=\ell_{f}(x)$ to a structure point $x \in \Xcal$ and, since $\ell_f$ is surjective, one structure point can not have two different features. The labeled graph described by $\mu= \sum_{i=1}^{n} h_{i} \delta_{(x_{i},a_{i})}$ is a particular instance of a generalized labeled graph in which $\ell_f$ is defined by $\ell_f(x_i)=a_i$.}

\subsection{Comparing structured objects}

We now aim to define a notion of equivalence between two structured objects $(\tv{\Xcal \times \Omega}, d_{\Xcal}, \mu)$ and $(\tv{\Ycal \times \Omega}, d_{\Ycal},\nu)$. We note in the following $\nu_Y,\nu_B$ the marginals of $\nu$. Intuitively, two structured objects are the same if they share the same feature information, if their structure information are lookalike and if the probability measures are corresponding in some sense. In this section, we present mathematical tools for individual comparison of the elements of structured objects. For completeness we recall here some useful mathematical tools defined in Chapter \ref{cha:ot_general} for comparing the elements of structured objects and we refer the reader to this chapter for more details.  

\begin{definition}[Isometry\label{isometrydef2}]
Let $(\Xcal,d_{\Xcal})$ and $(\Ycal,d_{\Ycal})$ be two metric spaces. An isometry is a sujective map $\phi : \Xcal \rightarrow \Ycal$ that preserves the distances:
\begin{equation}
\label{isometryproperty}
\forall (x,x') \in \Xcal^{2}, d_{\Ycal}(\phi(x),\phi(x'))=d_{\Xcal}(x,x')
\end{equation}

\end{definition}

We refer the reader to section \ref{sec:prop_of_gw} for wider explanations about isometry. The previous map $\phi$ can be used in order to compare the structure information of two structured objects. When the metric spaces are enriched with a probability measure they define a \emph{measurable metric spaces} also called \emph{mm-spaces} (see section \ref{sec:prop_of_gw}). In this case the notion of \emph{strong isomorphism} can be used for comparing mm-spaces:

\begin{definition}[Strong isomorphism]
Two mm-spaces $(\Xcal,d_{\Xcal},\mu_{X}),(\Ycal,d_{\Ycal},\mu_{Y})$ are strongly isomorphic if there exists an isometry $\phi: \supp(\mu_X) \rightarrow \supp(\nu_Y)$ which pushes $\mu_{X}$ forward to $\mu_{Y}$, \ie\ $\phi \# \mu_{X}=\mu_{Y}$. In this case we say that $\phi$ is \emph{measure preserving}.
\end{definition}

All this considered, we can now define a notion of equivalence between structured objects:

\begin{definition}[(II)-Strong isomorphism of structured objects.\label{eqstructobjects}]

\noindent Two structured objects are said to be \emph{(II)-strongly isomorphic} if there exists an isomorphism $I: \supp(\mu_X) \rightarrow \supp(\nu_Y)$ between the structures such that $\phi=(I,id)$ is bijective between $\supp(\mu)$ and $\supp(\nu)$ and measure preserving. More precisely $\phi$ satisfies the following properties: 

\begin{enumerate}[label=\textbf{P.\arabic*}]
\item \label{I} $\phi\#\mu=\nu$.
\item \label{II} The function $\phi$ statisfies: 
\begin{equation*}
\forall (x,a) \in \tv{\supp(\mu)} , \phi(x,a)=(I(x),a).
\end{equation*}
\item \label{III} The function $I: \supp(\mu_X)\rightarrow \supp(\nu_Y)$ is surjective, satisfies $I\#\mu_X=\nu_Y$ and: 
\begin{equation*}
\forall x,x' \in \tv{\supp(\mu_X)}^{2} , d_{\Xcal}(x,x')=d_{\Ycal}(I(x),I(x')).
\end{equation*}

\end{enumerate}

\end{definition}
\begin{Remark} It is easy to check that the (II)-strong isomorphism defines an equivalence relation over $\mathbb{S}_p(\Omega)$. \tv{Moreover the function $\phi$ described in this definition can be seen as a feature, structure and measure preserving function. Indeed from \ref{I} $\phi$ is measure preserving. Moreover $(\Xcal,d_{\Xcal},\mu_X)$ and $(\Ycal,d_{\Ycal},\nu_Y)$ are isomorphic through $I$. Finally using \ref{I} and \ref{II} we have that $\mu_A=\nu_B$ so that the feature information is also preserved. }
\end{Remark}

To illustrate this definition, we consider a simple example of two structured objects in the discrete case:
\begin{Example}
Let two structured objects defined by:
$$\underbrace{\begin{pmatrix}(x_{1},a_{1})\\(x_{2},a_{2})\\(x_{3},a_{3})\\(x_{4},a_{4})\end{pmatrix}}_{x_i,a_i},{\underbrace{\begin{pmatrix}0&1&1&1 \\
    1&0&1&2\\
    1&1&0&2\\
    1&2&2&0
    \end{pmatrix}}_{d_\Xcal(x_i,x_j)}} ,\underbrace{\begin{pmatrix}\nicefrac{1}{4}\\\nicefrac{1}{4}\\\nicefrac{1}{4}\\\nicefrac{1}{4}\end{pmatrix}}_{h_i}\quad \text{and}\quad
    \underbrace{\begin{pmatrix}(y_{1},b_{1})\\(y_{2},b_{2})\\(y_{3},b_{3})\\(y_{4},b_{4})\end{pmatrix}}_{y_{i},b_i},{\small\underbrace{\begin{pmatrix}0&1&1&1 \\
        1&0&2&2\\
        1&2&0&1\\
        1&2&1&0
    \end{pmatrix}}_{d_\Ycal(y_i,y_j)}},\underbrace{\begin{pmatrix}\nicefrac{1}{4}\\\nicefrac{1}{4}\\\nicefrac{1}{4}\\\nicefrac{1}{4}\end{pmatrix}}_{h'_i}
        $$
with for $i$, $a_{i}=b_{i}$ and for $i \neq j$, $a_{i} \neq a_{j}$ (see Figure \ref{equivalent_objects}). The two structured objects have isometric structures and same features individually but they are not (II)-strongly isomorphic. One possible map $\phi=(\phi_{1},\phi_{2}) : \Xcal \times \Omega \rightarrow \Ycal \times \Omega$ such that $\phi_{1}$ leads to an isometry is $\phi(x_{1},a_{1})=(y_{1},b_{1})$, $\phi(x_{2},a_{2})=(y_{3},b_{3})$, $\phi(x_{3},a_{3})=(y_{4},b_{4})$, $\phi(x_{4},a_{4})=(y_{2},b_{2})$. Yet this map does not satisfy $\phi_{2}(x,.)=id$ for any $x$ since $\phi(x_{2},a_{2})=(y_{3},b_{3})$ and $a_{2} \neq b_{3}$. The other possible functions such that $\phi_{1}$ leads to an isometry are simply permutations of this example, yet it is easy to check that none of them verifies \ref{II} (for example with $\phi(x_{2},a_{2})=(y_{4},b_{4})$).

 \begin{figure}[t]
    \centering
\tikzstyle{vertex1}=[circle,fill=red,minimum size=6pt,inner sep=0pt]
\tikzstyle{vertex2}=[circle,fill=green,minimum size=6pt,inner sep=0pt]
\tikzstyle{vertex3}=[circle,fill=blue,minimum size=6pt,inner sep=0pt]
\tikzstyle{vertex4}=[circle,fill=orange,minimum size=6pt,inner sep=0pt]

\tikzstyle{edge} = [draw]
\begin{tikzpicture}[scale=1, auto,swap]
    \foreach \pos/\name in {{(0,0)/a}}
        		\node[vertex1] (\name) at \pos {};
    \foreach \pos/\name in {{(2,0)/b}}
        		\node[vertex2] (\name) at \pos {};
    \foreach \pos/\name in {{(1,1)/c}}
        		\node[vertex3] (\name) at \pos {};
    \foreach \pos/\name in {{(2.5,1.5)/d}}
        		\node[vertex4] (\name) at \pos {};
\def\x{5}

    \foreach \pos/\name in {{(0+\x,0)/e}}
        		\node[vertex1] (\name) at \pos {};
    \foreach \pos/\name in {{(2+\x,0)/f}}
        		\node[vertex2] (\name) at \pos {};
    \foreach \pos/\name in {{(1+\x,1)/g}}
        		\node[vertex3] (\name) at \pos {};
    \foreach \pos/\name in {{(2.5+\x,1.5)/h}}
        		\node[vertex4] (\name) at \pos {};

    \foreach \source/ \dest in {b/a, c/a, c/b, c/d}
        \path[edge] (\source) -- (\dest);

    \foreach \source/ \dest in {e/g, g/f, g/h, h/f}
        \path[edge] (\source) -- (\dest);

\def\y{0.1}
\foreach \pos/ \name  in {{(0,00-\y)/(x_2, \color{red}{a_2}}, {(2,00-\y)/(x_3, \color{green}{a_3}}, {(2.5,1.50-\y)/(x_4, \color{orange}{a_4}}}
  \draw \pos node[below, scale = 1]{$\name$)};
\draw (1,1) node[left, scale = 1]{$(x_1, \color{blue}{a_1}$)};

\foreach \pos/ \name in {{(0+\x,0-\y)/(y_2, \color{red}{b_2}}, {(2+\x,00-\y)/(y_3, \color{green}{b_3}}}
  \draw \pos node[below, scale = 1]{$\name$)};
\draw (1+\x,1) node[left, scale = 1]{$(y_1, \color{blue}{b_1}$)};
\draw (2.5+\x,1.5) node[right, scale = 1]{$(y_4, \color{orange}{b_4}$)};

 \end{tikzpicture}
    \caption{Two structured objects with isometric structures and identical features that are not (II)-strongly isomorphic. The color of the nodes represent the node feature and each edge represents a distance of 1 between the connected nodes. \label{equivalent_objects}}
    \end{figure}
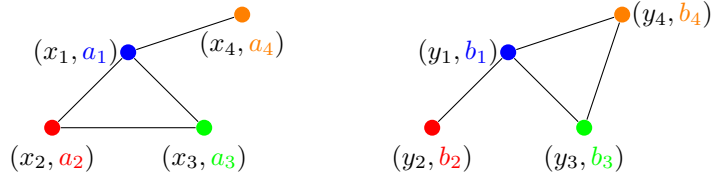
\end{Example}

We generalize the definition \eqref{discretefgw} of Fused Gromov-Wasserstein ($FGW$) to the continuous setting as follows:

\begin{definition}
\label{fgwdef}
 The Fused-Gromov-Wasserstein distance is defined for $\alpha \in [0,1]$ and $p,q \geq 1$ as:
\begin{equation}
\fgwdistance_{\alpha,p,q}(\mu,\nu)=\left( \ \underset{\pi \in \couplingset(\mu,\nu)}{\inf}  \lossfgw(\pi)\ \right)^{\frac{1}{p}}
\label{otfgx}
\end{equation}
where:
\begin{equation*}
\lossfgw(\pi)=\int \int \big((1-\alpha) d(a,b)^{q} +\alpha |d_{\Xcal}(x,x')-d_{\Ycal}(y,y')|^{q} \big)^{p} \dr\pi((x,a),(y,b))\dr\pi((x',a'),(y',b'))
\end{equation*}
We will write in the following $L(x,y,x',y')=|d_{\Xcal}(x,x')-d_{\Ycal}(y,y')|$.
\end{definition}

\begin{figure}[t]
\label{space}
\centering
\includegraphics[width=0.4\textwidth]{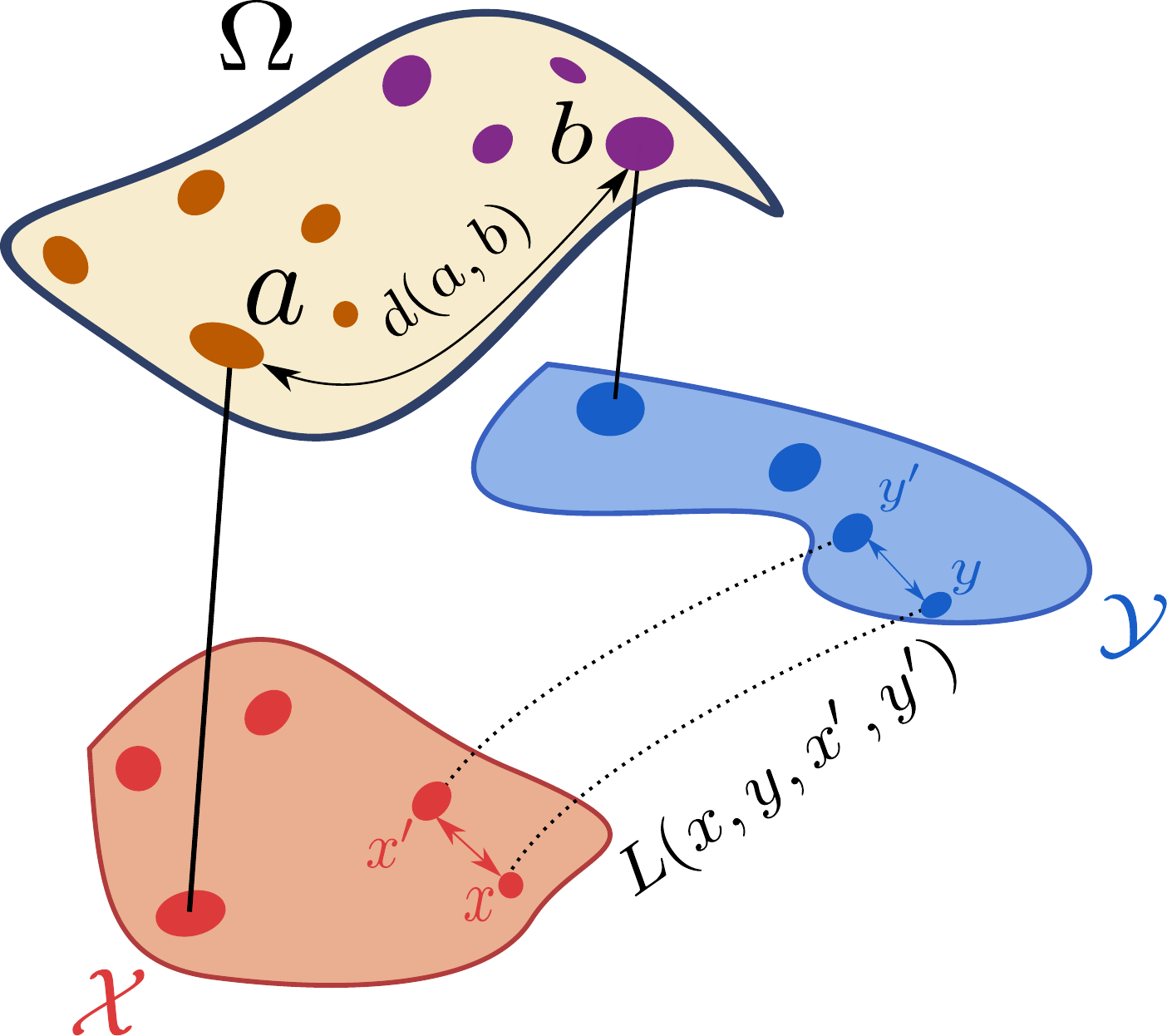}
\caption{Illustration of definition \ref{fgwdef}. The figure shows two structured objects $(\Xcal \times \Omega,d_{\Xcal},\mu)$ and $(\Ycal \times \Omega,d_{\Ycal},\mu)$. The feature space $\Omega$ is the common space for all features. The two metric spaces $(\Xcal,d_{\Xcal})$ and $(\Ycal,d_{\Ycal})$ represent the structures of the two structured objects, the similarity between all pair-to-pair distances of the structure points is measured by $L(x,y,x',y')$. $\mu$ and $\nu$ are the joint measures on the structure space and the feature space. \label{fgwex}}
\end{figure}

Note that this definition is coherent with the definition given in equation \eqref{discretefgw} when $p=1$. For brevity we will simply note $\fgwdistance$ instead for $\fgwdistance_{\alpha,p,q}$ when it is clear from the context. 
Many desirable properties arise from this definition. Among them, one can define a topology over the space of structured objects using the $FGW$ distance to compare structured objects, in the same philosophy as for Wasserstein and Gromov-Wasserstein distances. The definition also implies that $FGW$ acts as a generalization of both Wasserstein and Gromov-Wasserstein distances, with $FGW$ achieving an interpolation between these two distances. More remarkably, $FGW$ distance also realizes geodesic properties over the space of structured objects, allowing the definition of gradient flows. Before reviewing all these properties, we first compare $FGW$ with $GW$ and $W$ (by assuming for now that $FGW$ exists, which will be shown later in Theorem \ref{metrictheo}).

\begin{prop}[Comparison between $FGW$, $GW$ and $W$.]
\label{fgwcomparetogwandw}
With previous notations:
\begin{itemize}
\item The following inequalities hold:
\begin{equation}
\label{wassinequality_main}
\fgwdistance_{\alpha,p,q}(\mu,\nu) \geq (1-\alpha)\wass_{pq}(\mu_{A},\nu_{B})^{q}
\end{equation}
\begin{equation}
\label{gromovinequality_main}
\fgwdistance_{\alpha,p,q}(\mu,\nu) \geq \alpha \gw_{pq}(\mu_{X},\nu_{Y})^{q}
\end{equation}
\item Let us suppose that the structure spaces $(\Xcal,d_{\Xcal})$,$(\Ycal,d_{\Ycal})$ are part of a single ground space $(\mathcal{Z},d_{\mathcal{Z}})$ (\textit{i.e.} $\Xcal,\Ycal \subset \mathcal{Z}$ and $d_{\Xcal}=d_{\Ycal}=d_{\mathcal{Z}}$). We consider the Wasserstein distance between $\mu$ and $\nu$ for the distance on $\mathcal{Z} \times \Omega$ : $\tilde{d}((x,a),(y,b))=(1-\alpha)d(a,b)+\alpha d_{\mathcal{Z}}(x,y)$. Then:
\begin{equation}
\label{samespace_main}
\fgwdistance_{\alpha,p,1}(\mu,\nu)(\mu,\nu) \leq 2\wass_{p}(\mu,\nu).
\end{equation}
\end{itemize}
\end{prop}
Proof of this proposition can be found in Section~\ref{proof:prop1}. In particular, following this proposition, when the $FGW$ distance vanishes then both $GW$ and $W$ distances vanish so that the structure and the feature of the structure object are individually ``the same'' (with respect to their corresponding equivalence relation). However the converse is not necessarily true as shown further in Section \ref{sec:adapt_w_gw}.

In the following we establish some mathematical properties of the $FGW$ distance. The first result relates to the existence of the $FGW$ distance and the topology of the space of structured objects. We prove that the $FGW$ distance is indeed a distance regarding the equivalence relation between structured objects as defined in Defintion \ref{eqstructobjects}, allowing us to derive a topology on $\mathbb{S}(\Omega)$.

\subsection{Topology of the structured object space}
\label{Topology}
The $FGW$ distance has the following properties:

\begin{theo}[\tv{Metric properties}]
\label{metrictheo}

\noindent Let $p,q\geq 1$, $\alpha \in ]0,1[$ and $(\mu,\nu) \in \mathbb{S}_{pq}(\Omega) \times \mathbb{S}_{pq}(\Omega)$. The functional $\pi \rightarrow \lossfgw(\pi)$ always achieves an infimum $\pi^{*}$ in $\Pi(\mu,\nu)$ \emph{s.t.} $\fgwdistance_{\alpha,p,q}(\mu,\nu)=\lossfgw(\pi^{*})<+\infty$. Moreover:
\begin{itemize}
\item[$\bullet$] $\fgwdistance_{\alpha,p,q}$ is symmetric and, for $q=1$, satisfies the triangle inequality. For $q\geq 2$, the triangular inequality is relaxed by a factor $2^{q-1}$.
\item[$\bullet$] For $\alpha \in ]0,1[$, $\fgwdistance_{\alpha,p,q}(\mu,\nu)=0$ if an only if there exists a bijective function $ \phi=(\phi_1,\phi_2): \tv{\supp(\mu)} \rightarrow \tv{\supp(\nu)}$ such that:
\begin{equation}
\label{prese_main}
\phi\#\mu=\nu
\end{equation}
\begin{equation}
\label{id_main}
\forall (x,a) \in \tv{\text{supp}(\mu)} \ , \phi_{2}(x,a)=a
\end{equation}
\begin{equation}
\label{endgame_main}
\forall (x,a),(x',a') \in \tv{\text{supp}(\mu)^{2}}, \ d_{\Xcal}(x,x')=d_{\Ycal}(\phi_{1}(x,a),\phi_{1}(x',a'))
\end{equation}
\item[$\bullet$] If $(\mu,\nu)$ are generalized labeled graphs then $\fgwdistance_{\alpha,p,q}(\mu,\nu)=0$ if and only if $(\Xcal \times \Omega,d_\Xcal,\mu)$ and $(\Ycal \times \Omega,d_\Ycal,\nu)$ are (II)-strongly isomorphic.
\end{itemize}
\end{theo}

Proof of this theorem can be found in Section~\ref{proof:theo1}. \tv{The identity of indiscernibles is the most delicate part to prove and is based on using the Gromov-Wasserstein distance between the spaces $X\times \Omega$ and $\Ycal \times \Omega$.} \tv{The previous theorem states that $FGW$ is a distance over the space of generalized labeled graphs endowed with the strong isomorphism as equivalence relation defined in Definition \ref{eqstructobjects}. More generally for any structured objects the equivalence relation is given by equations \eqref{prese_main}, \eqref{id_main} and \eqref{endgame_main}. Informally, invariants of the $FGW$ are structured objects that have both the same structure and the same features in the same place.} \tv{Despite the fact that $q=1$ leads to a proper metric the case $q=2$ can be computed more efficiently using a separability trick from \cite{peyre2016gromov} as seen in Section \ref{sec:solved_fgw}}. 
\begin{Remark}
Note that the previous theorem actually proves Theorem \ref{metrictheodiscrete}. Indeed if we consider $\mu=\sum_{i=1}^{n} h_i \delta_{(x_i,a_i)}$ and $\nu=\sum_{j=1}^{m} g_j \delta_{(y_j,b_j)}$ describing two labeled graphs as discussed in the previous section. Then $\mu$ and $\nu$ are generalized labeled graph. Using Theorem \ref{metrictheo} $\mu$ and $\nu$ are (II)-strongly isomorphic if and only if there exists a bijection between the supports satisfies \eqref{I}, \eqref{II} and \eqref{III}. Since the supports are discrete this is equivalent to the condition $n=m$ and there exists a permutation $\sigma \in \Sn$ which satisfies the conditions of Theorem \ref{metrictheodiscrete}. The triangle inequality property of Theorem \ref{metrictheodiscrete} derives directly from the triangle inequality of Theorem \ref{metrictheo}.
\end{Remark}
There are some special cases where $W$ and $GW$ can be adapted to structured objects and can be used also to compare them. These cases result in different notions of equivalence as described in the following discussion.

\subsection{Can we adapt W and GW distances for structured objects? \label{sec:adapt_w_gw}} 

Despite the appealing properties of both Wasserstein and Gromov-Wasserstein distances, they fail at comparing structured objects by focusing only on the feature and structure marginals respectively. However, with some hypotheses, one could adapt these distances for structured objects.

\paragraph{Adapting Wasserstein: common structure space} If the structure spaces $(\Xcal,d_\Xcal)$ and $(\Ycal,d_\Ycal)$ are part of a same ground space $(\mathcal{Z},d_{\mathcal{Z}})$ one can build a distance $\hat{d}=d_\mathcal{Z}\oplus d$ between couples $(x,a)$ and $(y,b)$ and apply the Wasserstein distance. In this case, when the Wasserstein distance vanishes then the structured objects are equal in the sense $\mu=\nu$ which implies that $\mu$ and $\nu$ are \textit{de facto} (II)-strongly isomorphic. This approach is very related with the one discussed in \cite{Thorpe2017} where authors define the Transportation $L^{p}$ distance for signal analysis purposes. Their approach can be viewed as a transport between two joint measures: 
\begin{equation}
\mu(\Xcal \times \Omega)=\lebsm(\{(\xbf,f(\xbf)) \ | \ \xbf \in X \subset Z=\mathbb{R}^{d}| \ f(\xbf) \in \Omega \subset \mathbb{R}^{m}\})
\end{equation}
\begin{equation}
\nu(\Ycal\times \Omega)=\lebsm(\{(\ybf,g(\ybf)) \ | \ \ybf \in Y \subset Z=\mathbb{R}^{d}| \ g(\ybf) \in \Omega \subset \mathbb{R}^{m}\})
\end{equation} 
for functions $f,g : Z \rightarrow \mathbb{R}^{m}$ representative of the signal values and $\lebsm$ the Lebesgue measure. The distance for the transport is defined as $\hat{d}((\xbf,f(\xbf)),(\ybf,g(\ybf)))= \frac{1}{\alpha} \|\xbf-\ybf\|_{p}^{p} + \|f(\xbf)-g(\ybf)\|_{p}^{p}$ for $\alpha > 0$ and $\|\cdot\|_{p}$ the $l_{p}$ norm. In this case $f(\xbf)$ and $g(\ybf)$ can be interpreted as encoding the feature information of the signal while $\xbf,\ybf$ encode its structure information. This approach is interesting but cannot be used on structured objects such as graphs that will not share a common structure embedding space.

\paragraph{Adapting Gromov-Wasserstein} The Gromov-Wasserstein distance can also be adapted to structured objects by considering the distances $(1-\beta) d_{\Xcal} \oplus \beta d$ and $(1-\beta) d_{\Ycal} \oplus \beta d$ within each space $\Xcal \times \Omega$ and $\Ycal \times \Omega$ respectively and $\beta \in ]0,1[$. When the resulting $GW$ distance vanishes, structured objects are strongly isomorphic with respect to $(1-\beta) d_{\Xcal} \oplus \beta d$ and $(1-\beta) d_{\Ycal} \oplus \beta d$. However the (II)-strong isomorphism is stronger than this notion since the strong isomorphism allows for ``permuting the labels'' but not the (II)-strong isomorphism. More precisely we have the following lemma:
\begin{lemma}
Let $(\Xcal \times \Omega,d_{\Xcal},\mu),(\Ycal \times \Omega,d_{\Ycal},\nu)$ be two structured objects and $\beta \in ]0,1[$. 

If $(\Xcal \times \Omega,d_{\Xcal},\mu)$ and $(\Ycal \times \Omega,d_{\Ycal},\nu)$ are (II)-strongly isomorphic then $(\Xcal \times \Omega,(1-\beta) d_{\Xcal} \oplus \beta d,\mu)$ and $(\Ycal \times \Omega,(1-\beta) d_{\Ycal} \oplus \beta d,\nu)$ are strongly isomorphic. However the converse is not true in general.
\end{lemma} 

\begin{proof}
To see this, if we consider $\phi$ as defined in Theorem \ref{metrictheo}, then for $(x,a),(x',b) \in \supp(\mu)^{2}$ we have $d_{\Xcal}(x,x')= d_{\Ycal}(I(x),I(x'))$. In this way: 
\begin{equation}
(1-\beta) d_{\Xcal}(x,x')+\beta d(a,b)= (1-\beta) d_{\Ycal}(I(x),I(x'))+\beta d(a,b)
\end{equation} 
which can be rewritten as: 
\begin{equation}
(1-\beta) d\oplus \beta d_{\Xcal}((x,a),(x',b))=(1-\beta) d\oplus \beta d_{\Ycal}(\phi(x,a),\phi(x',b))
\end{equation} 
and so $\phi$ is an isometry with respect to $(1-\beta) d\oplus \beta d_{\Xcal}$ and $(1-\beta) d\oplus \beta d_{\Ycal}$. Since $\phi$ is also measure preserving and surjective $(\Xcal \times \Omega,(1-\beta) d_{\Xcal} \oplus \beta d,\mu)$ and $(\Ycal \times \Omega,(1-\beta) d_{\Ycal} \oplus \beta d ,\nu)$ are strongly isomorphic. 

However the converse is not necessarily true as it is easy to cook up an example with the same structure but with permuted labels so that objects are strongly isomorphic but not (II)-strongly isomorphic. For example in the tree example depicted in Figure \ref{equivalent_objects}, the structures are isometric and the distances between the features within each space are the same between each structured objects so that $(\Xcal \times \Omega,(1-\beta) d_{\Xcal} \oplus \beta d,\mu)$ and $(\Ycal \times \Omega,(1-\beta) d_{\Ycal} \oplus \beta d,\nu)$ are strongly isomorphic, yet not (II)-strongly isomorphic as shown in the example since $FGW >0$. 
\end{proof}

\subsection{Convergence of finite samples}

The metric properties of $FGW$ naturally endow the structured object space with a notion of convergence as described in the next definition:

\begin{definition}[Convergence of structured objects.]

\noindent Let $\big((\Xcal_{n} \times \Omega,d_{\Xcal_{n}}, \mu_{n})\big)_{n \in \mathbb{N}}$ be a sequence of structured objects. It converges to $(\Xcal \times \Omega, d_{\Xcal},\mu)$ in the Fused Gromov-Wasserstein sense if:
\begin{equation}
\lim\limits_{n \rightarrow \infty}\fgwdistance_{\alpha,p,1}(\mu_{n},\mu)=0
\end{equation}

\end{definition}
We consider in this definition only the case $q=1$ as it gives a proper metric (with $q>1$ the triangle inequality is relaxed by a factor $2^{q-1}$). Using Prop. \ref{fgwcomparetogwandw}, it is straightforward to see that if the sequence converges in the $FGW$ sense, both the features and the structure converge respectively in the Wasserstein and Gromov-Wasserstein sense.

An interesting question arises from this definition. If we consider a structured object $(\Xcal \times \Omega,d_{\Xcal}, \mu)$ and if we sample the joint distribution so as to consider $(\{(x_{i},a_{i})\}_{i \in \{1,..,n\} }, d_{\Xcal}, \mu_{n} )_{n \in \mathbb{N}}$ with $\mu_{n}=\frac{1}{n}\sum\limits_{i=1}^{n} \delta_{x_{i},a_{i}}$ where $(x_{i},a_{i}) \in \Xcal \times \Omega$ are sampled from $\mu$. Does this sequence converges to $(\Xcal \times \Omega,d_{\Xcal}, \mu )$ in the $FGW$ sense and how fast is the convergence?

This question can be answered thanks to a notion of ``size'' of a probability measure. For the sake of conciseness we will not present exhaustively the theory but the reader can refer to \cite{weedbach2017} for more details. Given a measure $\mu$ on $\Omega$ we denote as $dim_{p}^{*}(\mu)$ its \emph{upper Wasserstein dimension}. It coincides with the intuitive notion of ``dimension''  when the measure is sufficiently well behaved. For example, for any absolutely continuous measure $\mu$ with respect to the Lebesgue measure on $[0,1]^{d}$, we have $dim_{p}^{*}(\mu)=d$ for any $p \in [1,\frac{d}{2}]$.Using this definition and the results in \cite{weedbach2017}, we can answer the question of convergence of finite sample in the following proposition (proof can be found in Section~\ref{proof:prop2}):

\begin{theo}[Convergence of finite samples and a concentration inequality]
\label{concentration}

With previous notations. Let $p \geq 1$. We have:
\begin{equation}
\lim\limits_{n \rightarrow \infty}\fgwdistance_{\alpha,p,1}(\mu_{n},\mu)=0
\end{equation}
Moreover, suppose that $s > d_{p}^{*}(\mu)$. Then there exists a constant $C$ that does not depend on $n$ such that:
\begin{equation}
\label{epstaudim_main}
\mathbb{E}[\fgwdistance_{\alpha,p,1}(\mu_{n},\mu)] \leq C n^{-\frac{1}{s}}.
\end{equation}
The expectation is taken over the \textit{i.i.d} samples $(x_i,a_i)$. A particular case of this inequality is when $\alpha=1$ so that we can use the result above to derive a concentration result for the Gromov-Wasserstein distance. More precisely, if $\nu_{n}=\frac{1}{n} \sum_{i} \delta_{x_{i}}$ denotes the empirical measure of $\nu \in \Pcal(\Xcal)$ and if $s' > d_{p}^{*}(\nu)$ we have:
\begin{equation}
\label{concentrationgromov_main}
\mathbb{E}[\gw_{p}(\nu_{n},\nu)] \leq C' n^{-\frac{1}{s'}}.
\end{equation}
\end{theo}

\tv{This result is a simple application of the convergence of finite sample properties of the Wasserstein distance, since in this case $\mu_n$ and $\mu$ are part of the same ground space so that equation \eqref{epstaudim_main} derives naturally from equation \eqref{samespace_main} and the properties of Wasserstein}. In contrast to the Wasserstein distance case this inequality is not necessarily sharp and future work will be dedicated to the study of its tightness.

\subsection{Interpolation properties between W and GW distances}
\label{sec:interpo_section}
$FGW$ distance is a generalization of both Wasserstein and Gromov-Wasserstein distances in the sense that it achieves an interpolation between them. More precisely, we have the following theorem:

\begin{prop}[Interpolation properties.]
\label{interpolationtheorem}
As $\alpha$ tends to zero, one recovers the Wasserstein distance between the features information and as $\alpha$ goes to one, one recovers the Gromov-Wasserstein distance between the structure information:
\begin{equation}
\lim\limits_{\alpha \rightarrow 0}\fgwdistance_{\alpha,p,q}(\mu,\nu)=(\wass_{pq}(\mu_{A},\nu_{B}))^{q}
\end{equation}
\begin{equation}
\lim\limits_{\alpha \rightarrow 1}\fgwdistance_{\alpha,p,q}(\mu,\nu)=(\gw_{pq}(\mu_{X},\nu_{Y}))^{q}
\end{equation}
\end{prop}

Proof of this proposition can be found in Section~\ref{proof:theo2}. This result shows that $FGW$ can revert to one of the other distances and thus acts as a generalization of Wasserstein and Gromov-Wasserstein distances, as claimed in the discrete case section. 

\subsection{Geodesic properties}

One desirable property in OT is the underlying geodesics defined by the mass transfer between two probability distributions. These properties are useful in order to define dynamic formulation of OT problems. This dynamic point of view is inspired by fluid dynamics and finds its origin in the Wasserstein context with \cite{Benamou2000}. Various applications in machine learning can be derived from this formulation: interpolation along geodesic paths was used in computer graphics for color or illumination interpolations \cite{Bonneel:2011:DIU:2024156.2024192}. More recently, \cite{chizat_global} used Wasserstein gradient flows in an optimization context, deriving global minima results for non-convex particles gradient descent. In \cite{pmlr-v80-zhang18a} authors used Wasserstein gradient flows in the context of reinforcement learning for policy optimization.

The main idea of this dynamic formulation is to describe the optimal transport problem between two measures as a curve in the space of measures minimizing its total length. We first describe some generality about geodesic spaces and recall classical results for dynamic formulation in both Wasserstein and Gromov-Wasserstein contexts. In a second part, we derive new geodesic properties in the $FGW$ context.

\paragraph{Geodesic spaces} Let $(\Xcal,d_\Xcal)$ be a metric space and $x,y$ two points in $\Xcal$. We say that a curve $w:[0,1] \rightarrow \Xcal$ joining the \textit{endpoints} $x$ and $y$ (\textit{i.e.} with $w(0)=x$ and $w(1)=y$) is a \textit{constant speed geodesic} if it satisfies $d_\Xcal(w(t),w(s)) \leq |t-s| d(w(0),w(1))=|t-s| d_\Xcal(x,y)$ for $t,s \in [0,1]$. Moreover, if $(\Xcal,d_\Xcal)$ is a length space (\textit{i.e.} if the distance between two points of $\Xcal$ is equal to the infimum of the lengths of the curves connecting these two points) then the converse is also true and a constant speed geodesic satisfies  $d_\Xcal(w(t),w(s)) =|t-s| d_\Xcal(x,y)$. It is easy to compute distances along such curve as they are directly embedded into $\R$.

In the Wasserstein context, if the ground space is a complete separable, locally compact length space and if the endpoints of the geodesic are given, then there exists a geodesic curve. Moreover, if the transport between the endpoints is unique then there is a unique displacement interpolation between the endpoints (see Corollary 7.22 and 7.23 in \cite{Villani}). For example, if the ground space is $\mathbb{R}^{d}$ and the distance between the points is measured via the $\|.\|_2$ norm, then the geodesics exist and are uniquely determined (note that this can be generalized to costs of the form $c(\xbf,\ybf)=h(\ybf-\xbf)$ where $h$ is strictly convex). In the Gromov-Wasserstein context, there always exists constant speed geodesics as long as the endpoints are given. These geodesics are unique modulo strong isomorphisms (see \cite{Sturm2012}).

\paragraph{The $FGW$ case} In this paragraph, we suppose that $\Omega=\mathbb{R}^{d}$. We are interested in finding a geodesic curve in the space of structured objects \textit{i.e.} a constant speed curve of structured objects joining two structured objects. As for Wasserstein and Gromov-Wasserstein, the structured object space endowed with the Fused Gromov-Wasserstein distance maintains some geodesic properties. The following result proves the existence of such a geodesic and characterizes it:

\begin{theo}[\tv{Constant speed geodesic.}]
\label{cstespeedtheo}
 Let $p\geq 1$ and $(\Xcal \times \Omega ,d_{\Xcal},\mu_{0})$ and $(\Ycal \times \Omega ,d_{\Ycal},\mu_{1})$ in $\mathbb{S}_p(\mathbb{R}^{d})$. Let $\pi^{*}$ be an optimal coupling for the Fused Gromov-Wasserstein distance between $\mu_{0},\mu_{1}$ and $t \in [0,1]$. We equip $\mathbb{R}^{d}$ with  $\ell_{m}$ norm for $m \geq 1$.

We define $\eta_{t} : \Xcal \times \Omega \times \Ycal \times \Omega \rightarrow \Xcal \times \Ycal \times \Omega$ such that:
\begin{equation}
\forall (x,\a),(y,\b) \in \Xcal \times \Omega \times \Ycal \times \Omega, \ \eta_{t}(x,\a,y,\b)=(x,y,(1-t)\a+t\b)
\end{equation}
Then:
\begin{equation}
\label{geodesic}
\left(\Xcal\times \Ycal \times \Omega,(1-t)d_{\Xcal} \oplus t d_{\Ycal},\mu_{t}=\eta_{t} \# \pi^{*} \right)_{t \in [0,1]}
\end{equation}
is a constant speed geodesic connecting $(\Xcal \times \Omega ,d_{\Xcal},\mu_{0})$ and $(\Ycal \times \Omega ,d_{\Ycal},\mu_{1})$ in the metric space $\left(\mathbb{S}_p(\mathbb{R}^{d}), \fgwdistance_{\alpha,p,1}\right)$.

\end{theo}

Proof of the previous theorem can be found in Section~\ref{proof:theo3}. In a sense this result combines the geodesics in the Wasserstein space and in the space of all mm-spaces since it suffices to interpolate the distances in the structure space and the features to construct a geodesic. The main interest is that it defines the minimum path between two structured objects. For example, considering two discrete structured objects represented by the measures $\mu_{0}=\sum_{i=1}^{n} h_{i}  \delta_{(x_{i},\a_{i})}$ and $\mu_{1}=\sum_{j=1}^{m} g_{j}  \delta_{(y_{j},\b_{j})}$, the interpolation path is given for $t\in [0,1]$ by the measure $\mu_{t}=\sum_{i=1}^{n}\sum_{j=1}^{m} \pi^{*}(i,j) \delta_{(x_{i},y_{j},(1-t)\a_{i} +t \b_{j})}$ where $\pi^{*}$ is an optimal coupling for the $FGW$ distance. However this geodesic is difficult to handle in practice since it requires the computation of the cartesian product $\Xcal_{0}\times \Xcal_{1}$. The Fréchet mean defined in Section \ref{sec:bary} seems to be more suited in practice. The proper definition and properties of velocity fields associated to this geodesic is postponed to further works.

\section{Discussion and conclusion}

Countless problems in machine learning involve structured data, usually stressed in light of the graph formalism. We consider here labeled graphs enriched by an histogram, which naturally leads to represent structured data as probability measures in the joint space of their features and structures. Widely known for their ability to meaningfully compare probability measures, transportation distances  are generalized in this chapter so as to be suited in the context of structured data, motivating the so-called Fused Gromov-Wasserstein distance. We theoretically prove that it defines indeed a distance on structured data, and consequently on graphs of arbitrary sizes. $FGW$ provides a natural framework for {analysis} of labeled graphs as we demonstrate on classification, where it reaches and surpasses most of the time the state-of-the-art performances, and in graph-based $k$-means where we develop a novel approach to represent the clusters centroids  using a {barycentric} formulation of $FGW$. We believe that this metric can have a significant impact on challenging graph signal analysis problems.

{While we considered a unique measure of distance between nodes in the graph structure (shortest path), other choices could be made with respect to the problem at hand, or eventually learned in an end-to-end manner. The same applies to the distance between features.} 
{We also envision a potential use of this distance in deep learning applications where a distance between graph is needed (such as graph auto-encoders). Another line of work will also try to lower the computational complexity of the underlying optimization problem to ensure better scalability to very large graphs.}


\chapter{The Gromov-Wasserstein problem in Euclidean spaces}
\epigraph{\itshape``Is it all right if I go out there?''\\
``Sure,'' Thomas Hudson had told him. ``But it s rugged from now on until spring and spring isn't easy.''\\
``I want it to be rugged,'' Roger had said. ``I am going to start new again.'' \\
``How many time it is now you've started new?''\\
``Too many,'' Roger had said. ``And you don't have to rub it in.''}{-- Ernest Hemingway, \textit{Islands in the Stream}}
\minitoc
\label{cha:gw_euclidean}

\newpage

\begin{Abstract}
This chapter is based on the paper \cite{vay_sliced_gromov_2019} and addresses the problem of GW in Euclidean spaces. Recently used in various machine learning contexts, the Gromov-Wasserstein distance allows for comparing distributions which supports do not necessarily lie in the same metric space. However, this optimal transport distance requires solving a complex non convex quadratic program which is most of the time very costly both in time and memory. Contrary to $\gw$, the Wasserstein distance enjoys several properties ({\em e.g.} duality) that permit large scale optimization. Among those, the solution of $W$ on the real line, that only requires sorting discrete samples in 1D, allows defining the Sliced Wasserstein ($SW$) distance. This first part of this chapter presents a new divergence based on $\gw$ akin to $SW$. More precisely the contributions of are the following:
\begin{itemize}
\item We derive the first closed form solution for $\gw$ when dealing with discrete 1D distributions, based on a new result for the related quadratic assignment problem (Theorem \ref{qap} and Theorem \ref{sovable_gw}). 
\item Based on this result we define a novel OT discrepancy which can deal with large scale distributions via a slicing approach and we show how it relates to the $\gw$ distance while being $O(n\log(n))$ to compute.
\item We illustrate the behavior of this so called Sliced Gromov-Wasserstein ($\sgw$) discrepancy in experiments where we demonstrate its ability to tackle similar problems as $\gw$ while being several order of magnitudes faster to compute. 
\end{itemize}
The second part of this chapter is more prospective and tackle the problem of probability distributions which supports lie on Euclidean spaces with, potentially, different dimensions. This part investigate the regularity of $\gw$ optimal transport plan in the cases of inner product similarities and Euclidean distances. The contributions of this part are, in summary: 
\begin{itemize}
\item We show that the $\gw$ problem in Euclidean spaces is equivalent to jointly solve a linear transportation problem and a ``alignment'' problem (Theorem \ref{maintheorem} and Theorem \ref{maintheo2}).
\item We give necessary conditions under which a $\gw$ optimal transport plan is supported on a deterministic function (Theorem \ref{first_sufficient} and Proposition \ref{sufficientcond2}). This allows to derive a closed-form expression for $\gw$ with inner product similarities between 1D probability distributions (not necessarily discrete, see Theorem \ref{theo:1D1D1D}).
\item We study the Gromov-Monge problem in Euclidean spaces, and in particular the \emph{linear} Gromov-Monge problem for which we exhibit a closed-form expression between Gaussian distributions (Theorem \ref{maintheo3}).
\end{itemize}
\end{Abstract}

\section{Sliced Gromov-Wasserstein}
\label{sec:sliced}

\subsection{Introduction}

As described in Chapter \ref{cha:ot_general} the linear optimal transport problem aims at defining ways to compare probability distributions, through \textit{e.g.} the Wasserstein distance.
It has proved to be very useful for a wide range of machine learning tasks including generative modelling (Wasserstein GANs~\cite{arjovsky17a}), domain adaptation~\cite{courty2017optimal} or supervised embeddings for classification purposes~\cite{huang2016}. However one limitation of this approach is that it implicitly assumes \emph{aligned} distributions, \textit{i.e.} that lie in the same metric space or at least between spaces where a meaningful distance \emph{across} domains can be computed. From another perspective, the Gromov-Wasserstein distance benefits from more flexibility when it comes to the more challenging scenario where heterogeneous distributions are involved, \emph{i.e.} distributions which supports do not necessarily lie on the same metric space. It only requires modelling the topological or relational aspects of the distributions \emph{within} each domain in order to compare them. As such, it has recently received a high interest in the machine learning community, solving learning tasks such as heterogenous domain adaptation~\cite{ijcai2018-412}, deep metric alignment~\cite{gwcnn}, graph classification (see Chapter \ref{cha:fgw} for more details) or generative modelling~\cite{bunne_gan}.

OT is known to be a computationally difficult problem: the Wasserstein distance involves a linear program that most of the time prevents its use to settings with more than a few tens of thousands of points. For medium to large scale problems, some methods relying  \textit{e.g.} on entropic regularization or dual formulation (as seen in Chapter \ref{cha:ot_general}) have been investigated in the past years. Among them, one builds upon the mono-dimensional case where computing the Wasserstein distance can be trivially solved in $O(n \log n)$ by sorting points in order and pairing them from left to right. While this 1D case has a limited interest \emph{per se}, it is one of the main ingredients of the \emph{Sliced} Wasserstein distance~\cite{rabin2011wasserstein}: high-dimensional data are linearly projected into sets of mono-dimensional distributions, the sliced Wasserstein distance being the average of the Wasserstein distances between all projected measures. This framework provides an efficient algorithm that can handle millions of points and has similar properties to the Wasserstein distance \cite{bonotte_phd}. As such, it has attracted attention and has been successfully used in various tasks such as barycenter computation~\cite{bonneel:hal-00881872}, classification~\cite{Kolouri_2016_CVPR} or generative modeling \cite{kolouri2018sliced,cvpr_sliced_gan,sliced_wass_flow_liutkus_2019,Wu_2019_CVPR}.

Regarding $\gw$, the optimization problem is a non-convex quadratic program, with a prohibitive computational cost for problems with more than a few thousands of points: the number of terms grows quadratically with the number of samples and one cannot rely on a dual formulation as for Wasserstein. However several approaches have been proposed to tackle its computation. Initially approximated by a linear lower bound, $\gw$ was thereafter estimated through an entropy regularized version that can be efficiently computed by iterating Sinkhorn projections (see Chapter \ref{cha:ot_general}) or using a conditional gradient scheme relying on linear program OT solvers (see Chapter \ref{cha:fgw}). However, all these methods are still too costly for large scale \textit{scenarii}. In this section, we propose a new formulation related to $\gw$ that lowers its computational cost. To that extent, we derive a novel OT discrepancy called Sliced Gromov-Wasserstein ($\sgw$). It is similar in spirit to the Sliced Wasserstein distance as it relies on the exact computation of 1D $\gw$ distances of distributions projected onto random directions. We notably provide the first 1D closed-form solution of the $\gw$ problem by proving a new result about the Quadratic Assignment Problem for matrices that are squared euclidean distances of real numbers. Computation of $\sgw$ for discrete distributions of $n$ points is $O(L \, n\log(n))$, where $L$ is the number of sampled directions. This complexity is the same as the Sliced-Wasserstein distance and is even lower than computing the value of $\gw$ which is $O(n^3)$ for a known coupling (once the optimization problem solved) in the general case~\cite{peyre2016gromov}. Experimental validation shows that $\sgw$ retains various properties of $\gw$ while being much cheaper to compute, allowing its use in difficult large scale settings such as large mesh matching or generative adversarial networks.

\subsection{From 1D GW to Sliced Gromov-Wasserstein}
\label{sec:sgw}

We first provide and prove a solution for an 1D Quadratic Assignement Problem with a quasilinear time complexity of $O(n\log(n))$. This new special case of the QAP is shown to be equivalent to the \emph{hard assignment} version of $\gw$, called the Gromov-Monge ($GM$) problem, with squared Euclidean cost for distributions lying on the real line. We also show that, in this context, solving $\gm$ is equivalent to solving $\gw$. We derive a new discrepancy named Sliced Gromov-Wasserstein ($\sgw$) that relies on these findings for efficient computation.

\paragraph{Solving a Quadratic Assignment Problem in 1D}

In Koopmans-Beckmann form~\cite{koopmans} a QAP takes as input two $n \times n$ matrices $\Abf=(a_{ij})$, $\Bbf=(b_{ij})$. The goal is to find a permutation $\sigma \in \Sn$, the set of all permutations of $\integ{n}$, which minimizes the objective function {\small$ \sum_{i,j=1}^{n} a_{i,j} b_{\sigma(i),\sigma(j)}$}. In full generality this problem is NP-hard (see Section \ref{sec:solving_gw} for more details). The following theorem is a new result about QAP and states that it can be solved in polynomial time when $\Abf$ and $\Bbf$ are squared Euclidean distance matrices of sorted real numbers: 

\begin{theo}[A new special case for the Quadratic Assignment Problem]
\label{qap}

For real numbers $x_{1} < \dots < x_{n}$ and $y_{1} < \dots < y_{n}$, 
\begin{equation}
\underset{\sigma \in \Sn}{\min} \sum_{i,j}  - (x_{i}-x_{j})^{2}(y_{\sigma(i)}-y_{\sigma(j)})^{2}
\end{equation}
is achieved either by the identity permutation $\sigma(i)=i$ ($Id$) or the anti-identity permutation $\sigma(i)=n+1-i$ ($anti-Id$). In other words:
\begin{equation}
\exists \sigma \in \{Id,anti-Id\}, \ \sigma \in \underset{\sigma \in \Sn}{\argmin} \sum_{i,j}  - (x_{i}-x_{j})^{2}(y_{\sigma(i)}-y_{\sigma(j)})^{2} 
\end{equation}
\end{theo}

To the best of our knowledge, this result is new. It states that if one wants to find the best one-to-one correspondence of real numbers such that their pairwise distances are best conserved, it suffices to sort the points and check whether the identity has a better cost than the anti-identity. 
Proof of this theorem can be found in Section \ref{proof:maintheo}. We postulate that this result also holds for $a_{ij}=|x_{i}-x_{j}|^{k}$ and $b_{ij}=-|y_{i}-y_{j}|^{k}$ with any $k\geq 1$ but leave this study for future works.

\paragraph{Gromov-Wasserstein distance on the real line}

When $n=m$ and $a_{i}=b_{j}=\frac{1}{n}$, one can look for the \emph{hard assignment} version of the $\gw$ distance resulting in the Gromov-Monge problem~\cite{memoli_gromov_monge_2018} associated with the following $\gm$ distance:
\begin{equation}
\gm_{2}(\insided_{\Xcal},\insided_{\Ycal},\mu,\nu) = \min_{\sigma \in \Sn} \frac{1}{n^{2}} \sum_{i,j} \big| \insided_{\Xcal}(x_{i},x_{j})-\insided_{\Ycal}(y_{\sigma(i)},y_{\sigma(j)}) \big|^{2} \label{eq:qapgw}
\end{equation}
where  $\sigma \in \Sn$ is a one-to-one mapping $\integ{n}  \rightarrow
\integ{n}$. Interestingly when the permutation $\sigma$ is known, the computation of the
cost is $O(n^2)$ which is far better than $O(n^3)$ for the general $\gw$ case. It is easy to see that this problem is equivalent to minimizing $ \sum_{i,j=1}^{n} a_{i,j} b_{\sigma(i),\sigma(j)}$ with $a_{ij}=\insided_{\Xcal}(x_{i},x_{j})$ and $b_{ij}=-\insided_{\Ycal}(y_{\sigma(i)},y_{\sigma(j)})$. Indeed we have:
\begin{equation*}
\begin{split}
 \sum_{i,j} \big| \insided_{\Xcal}(x_{i},x_{j})-\insided_{\Ycal}(y_{\sigma(i)},y_{\sigma(j)}) \big|^{2} &=\sum_{i,j} \insided_{\Xcal}(x_{i},x_{j})^{2}+\sum_{i,j} \insided_{\Ycal}(y_{\sigma(i)},y_{\sigma(j)})^{2}-2 \sum_{i,j}\insided_{\Xcal}(x_{i},x_{j}) \insided_{\Ycal}(y_{\sigma(i)},y_{\sigma(j)}) \\
 &=\sum_{i,j} \insided_{\Xcal}(x_{i},x_{j})^{2}+\sum_{i,j} \insided_{\Ycal}(y_{i},y_{j})^{2}-2 \sum_{i,j}\insided_{\Xcal}(x_{i},x_{j}) \insided_{\Ycal}(y_{\sigma(i)},y_{\sigma(j)})
 \end{split} 
\end{equation*}
So that only the term $-2 \sum_{i,j}\insided_{\Xcal}(x_{i},x_{j}) \insided_{\Ycal}(y_{\sigma(i)},y_{\sigma(j)})$ depends on $\sigma$. Thus, when squared Euclidean costs are used for distributions lying on the real line, Theorem~\ref{qap} exactly recovers the solution of the $\gm$ problem defined in equation~\eqref{eq:qapgw}. As matter of consequence, Theorem~\ref{qap} provides an efficient way of solving the Gromov-Monge problem.

Moreover, this theorem also allows finding a closed-form for the $\gw$ distance. 
Indeed, some recent advances in graph matching state that, under some conditions on $\Abf$ and $\Bbf$, the assignment problem is equivalent to its \emph{soft-assignment} counterpart~\cite{NIPS2018_7323}. 
This way, using both Theorem~\ref{qap} and~\cite{NIPS2018_7323}, one can find a solvable case for the $\gw$ distance as stated in the following theorem: 
\begin{theo}[Equivalence between $\gw$ and $\gm$ for discrete measures]
\label{sovable_gw}
Let $\mu \in \P(\R^{p})$, $\nu \in \P(\R^{q})$ be discrete probability measures with same number of atoms and uniform weights, \ie\ $\mu=\frac{1}{n}\sum_{i=1}^{n}\delta_{\xbf_i},\nu=\frac{1}{n}\sum_{i=1}^{n}\delta_{\ybf_i}$ with $\xbf_i \in \R^{p},\ybf_i \in \R^{q}$.  For $\xbf \in \R^{p}$ we note $\|\xbf\|_{2,p}=\sqrt{\sum_{i=1}^{p} |x_{i}|^{2}}$ the $\ell_{2}$ norm on $\R^{p}$ (same for $\R^{q}$). Let $c_{\Xcal}(\xbf,\xbf')=\|\xbf-\xbf'\|_{2,p}^{2}$ , $c_{\Ycal}(\ybf,\ybf')=\|\ybf-\ybf'\|_{2,q}^{2}$. Then: 
\begin{equation}
\gw_{2}(c_{\Xcal},c_{\Ycal},\mu,\nu)=\gm_{2}(c_{\Xcal},c_{\Ycal},\mu,\nu)
\end{equation}
Moreover when $p=q=1$, \ie\ $c_{\Xcal}(x,x')=c_{\Ycal}(x,x')=|x-x'|^{2}$, and $x_{1} < \dots < x_{n}$ and $y_{1} < \dots < y_{n}$ the optimal values are achieved by considering either the identity or the anti-identity permutation.
\end{theo}

A detailed proof is provided in Section \ref{proof:eq_GM_GW}. Note also that, while both possible
solutions for problem \eqref{eq:qapgw} can be computed in $O(n\log(n))$, finding
the best one requires the computation of the cost which seems, at first sight, to have a $O(n^2)$ complexity. However, under the hypotheses of squared Euclidean distances, the cost can be computed in $O(n)$. Indeed, in this case, one can develop the sum in equation \eqref{eq:qapgw} to compute it in $O(n)$ operations using binomial expansion (see details in Section \ref{proof:computing_gw}) so that the overall complexity of finding the best assignment and computing the cost is $O(n\log(n))$ which is the same complexity as the Wasserstein for 1D distributions.

\paragraph{Sliced Gromov-Wasserstein discrepancy} 

Theorem \ref{sovable_gw} can be put in perspective with the Wasserstein distance for 1D distributions which is achieved by the identity permutation when points are sorted~\cite{cot_peyre_cutu}. As explained in Chapter \ref{cha:ot_general}, this result was used to approximate the Wasserstein distance between measures of $\R^{q}$ using the so called Sliced Wasserstein (SW) distance~\cite{bonneel:hal-00881872}. The main idea is to project the points of the measures on lines of $\R^{q}$where computing a Wasserstein distance is easy since it only involves a simple
sort and to average these distances. In the same philosophy we build upon Theorem \ref{sovable_gw} to define a ``sliced'' version of the $\gw$ distance. In the following, we consider $\mu \in \Pcal(\R^{p}),\nu \in \Pcal(\R^{q})$ be probability distributions (not necessarily discrete).

Let $\mathbb{S}^{q-1}=\left\{\thetab \in \mathbb{R}^{q} | \|\thetab\|_{2}=1\right\}$ be
the $q$-dimensional hypersphere and $\lambda_{q-1}$ the uniform measure on
$\mathbb{S}^{q-1}$ . For $\thetab$ we note $P_{\thetab}$ the projection on $\thetab$, \textit{i.e.} $P_{\thetab}(\xbf)=\langle \xbf,
\thetab \rangle$. For a linear map $\D \in \mathbb{R}^{q\times p}$ (identified with slight abuses of notation by its corresponding matrix),
we define the Sliced Gromov-Wasserstein (SGW) discrepancy as follows:

\begin{equation}{}
\label{sgw}
\sgw_{\D}(\mu,\nu)= \underset{\thetab \sim \lambda_{q-1}}{\E}[\gw^{2}_{2}(P_{\thetab}\#\mu_{\D},P_{\thetab}\#\nu)]=  \int_{\mathbb{S}^{q-1}}\gw^{2}_{2}(d^{2},P_{\thetab}\#\mu_{\D},P_{\thetab}\#\nu) \dr\lambda_{q-1}(\thetab)
\end{equation}

 where $\mu_{\D}=\D\#\mu \in \mathcal{P}(\R^{q})$. The function $\D$ acts as a mapping for
 a point in $\R^{p}$ of the measure $\mu$ onto $\R^{q}$. When $p=q$ and when we consider $\D$ as the identity map we simply
 write $\sgw(\mu,\nu)$ instead of $\sgw_{\mathbf{I_{p}}}(\mu,\nu)$. When $p < q$, one straightforward choice is $\D=\D_{pad}$ the "uplifting" operator which pads each point of the measure with zeros: $\D_{pad}(\xbf)=(x_{1},\dots,x_{p},\underbrace{0,\dots,0}_{q-p})$. The procedure is illustrated in Fig \ref{sgw_figure}.

 In general fixing $\D$ implies that some properties of $\gw$, such as the rotational invariance, are lost. Consequently, we also propose a variant of SGW that does not depends on the choice of $\D$ called Rotation Invariant SGW ($\risgw$) and expressed for $p\geq q$ as the following:
 \begin{equation}
  \label{risgw}
  \risgw(\mu,\nu)= \underset{\D \in \Stief}{\min}\sgw_{\D}(\mu,\nu).
  \end{equation}
 We propose to minimize $\sgw_{\D}$ with respect to $\D$ in the Stiefel
 manifold $\Stief$ \cite{absil2009optimization} which is defined as $\Stief=\{\D \in \R^{q\times p} | \D^{T}\D=\mathbf{I_{p}} \}$. It can be seen as finding an optimal projector of the measure $\mu$ \cite{subspace_robust_wass_patty_2019,Deshpande_2019_CVPR}. This formulation comes at the cost of an
additional optimization step but allows recovering one key property of GW.
 When $p=q$ this encompasses for
 \textit{e.g.} all rotations of the space, making $\risgw$ invariant by rotation.

\begin{figure}
\begin{center}
\begin{tikzpicture}
\usetikzlibrary{decorations}
\usetikzlibrary{decorations.pathreplacing}

\draw[->] (0,0,0)--(1,0,0) ;
\draw[->] (0,0,0)--(0,1,0);
\draw[->] (0,0,0)--(0,0,1);

\fill[orange!70] (0.5,0.5,0) circle (0.1cm);
\fill[orange!70] (0.5,0,0.7) circle (0.1cm);
\fill[orange!70] (-0.5,0.2,0.2) circle (0.1cm);
\fill[orange!70] (-0.1,0.6,0.4) circle (0.1cm);
\draw[gray] (0.5,0.5,0) circle (0.1cm);
\draw[gray] (0.5,0,0.7) circle (0.1cm);
\draw[gray] (-0.5,0.2,0.2) circle (0.1cm);
\draw[gray] (-0.1,0.6,0.4) circle (0.1cm);
\draw (0.5,0.5,0) node[right]{$y_1$};
\draw (0.5,0,0.7) node[below]{$y_2$};
\draw (-0.5,0.2,0.2) node[left]{$y_3$};
\draw (-0.1,0.6,0.4) node[above]{$y_4$};

\draw[->] (0,2)--(2,2) ;
\draw[->] (0,2)--(0,3);
\fill[red!70] (0.4,2.8) circle (0.1cm);
\fill[red!70] (0.8,2.8) circle (0.1cm);
\fill[red!70] (0.5,2.3) circle (0.1cm);
\fill[red!70] (1.7,2.2) circle (0.1cm);
\draw[gray] (0.4,2.8) circle (0.1cm);
\draw[gray] (0.8,2.8) circle (0.1cm);
\draw[gray] (0.5,2.3) circle (0.1cm);
\draw[gray] (1.7,2.2) circle (0.1cm);
\draw (0.8,2.8) node[above]{$x_1$};
\draw (0.4,2.8) node[above]{$x_4$};
\draw (0.5,2.3) node[below]{$x_2$};
\draw (1.7,2.2) node[left]{$x_3$};

\draw[->] (2.5,0.5) to[out=340,in=200] (4.5,0.5);
\draw (3.5,0.3) node[below]{$P_{\thetab}\#\nu$};

\draw[->] (2.5,2.5) to[out=20,in=160] (4.5,2.5);
\draw (3.5,2.7) node[above]{$P_{\thetab}\#(\D\#\mu)$};

\draw (3.5,1.5) node{for $\thetab$ $\in$ $\Sp^{q-1}$};
\foreach \x/\y in {2+4/5.6+4,3+4/4.6+4,3.8+4/4.2+4, 5.8+4/1.5+4} \draw[thick,blue!60](\y,2.5) -- (\x,0.5);

\draw[gray,dashed, ->]  (1+4,0.5) -- (6+4,0.5);
\foreach \y in {2+4,3+4,3.8+4, 5.8+4} \fill[orange!70] (\y,0.5) circle (0.1cm);
\foreach \y in {2+4,3+4,3.8+4, 5.8+4} \draw[gray] (\y,0.5) circle (0.1cm);
\foreach \y/\z in {2+4/1,3+4/2,3.8+4/3, 5.8+4/4} \draw(\y,0.5)node[below]{$y^{\thetab}_\z$};

\draw[gray,dashed, ->]   (1+4,2.5) -- (6+4,2.5);
\foreach \x in {1.5+4,4.2+4,4.6+4, 5.6+4} \fill[red!70] (\x,2.5) circle (0.1cm);
\foreach \x in {1.5+4,4.2+4,4.6+4, 5.6+4} \draw[gray] (\x,2.5) circle (0.1cm);
\foreach \x/\z in {1.5+4/1,4.2+4/2,4.6+4/3, 5.6+4/4} \draw(\x,2.5)node[above]{$x^{\thetab}_\z$};
\end{tikzpicture}
\end{center}
\caption{Example in dimension $p = 2$ and $q= 3$ {\bf (left)} that are projected on the line {\bf (right)}. The solution for this projection is the anti-diagonal coupling. \label{sgw_figure}}
\end{figure}
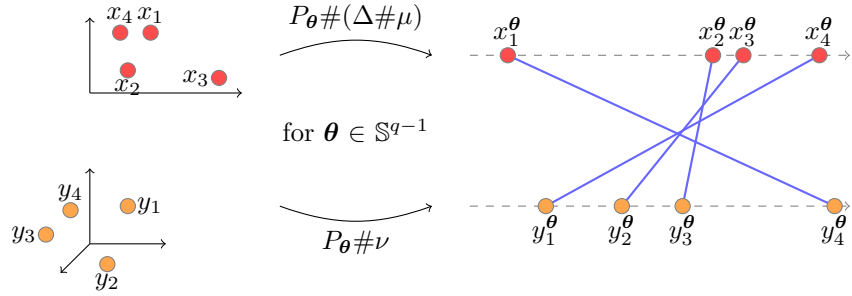

Interestingly enough, $\sgw$ holds various properties of the $\gw$ distance as summarized in the following theorem:
\begin{theo}[Properties of $\sgw$]
\label{propertiessgw}
~
\begin{itemize}
\item For all $\D$, $\sgw_{\D}$ and $\risgw$ are translation invariant. $\risgw$ is also rotational invariant when $p=q$, more precisely if $\Qbf \in \mathcal{O}(p)$ is an orthogonal matrix, $\risgw(\Qbf\#\mu,\nu)=\risgw(\mu,\nu)$ (same for any $\Qbf' \in \mathcal{O}(q)$ applied on $\nu$).
\item $\sgw$ and $\risgw$ are pseudo-distances on $\Pm(\R^{p})$, \textit{i.e.} they are symmetric, satisfy the triangle inequality and $\sgw(\mu,\mu)=\risgw(\mu,\mu)=0$ .
\item Let $\mu,\nu \in \Pm(\R^{p})\times \Pm(\R^{p})$ be probability distribution with \emph{compact supports}. If $\sgw(\mu,\nu)=0$ then $\mu$ and $\nu$ are isomorphic for the distance induced by the $\ell_{1}$ norm on $\R^{p}$, \ie\ $d(\xbf,\xbf')=\sum_{i=1}^{p} |x_{i}-x_{i}'|$ for $(x,x') \in \R^{p} \times \R^{p}$. In particular this implies:
\begin{equation}
\sgw(\mu,\nu)=0 \implies \gw_{2}(d,d,\mu,\nu)=0
\end{equation} 
\end{itemize}
\end{theo}

(with a slight abuse of notation we identify the matrix $\Qbf$ by its linear application). A proof of this theorem can be found in Section \ref{proof:prop_sgw}. This theorem states that if $\sgw$ vanishes then measures must be isomorphic, as it is the case for $\gw$. It states also that $\risgw$ holds most of the properties of $\gw$ in term of invariants. 

\begin{Remark} The $\D$ map can also be used in the context of the Sliced Wasserstein distance so as to define $SW_{\D}(\mu,\nu)$, $RISW(\mu,\nu)$ for $\mu,\nu \in \Pm(\R^{p})\times \Pm(\R^{q})$ with $p\neq q$. Please note that from a purely computational point of view, complexities of these discrepancies are the same as $SGW$ and $RISGW$ when $\mu$ and $\nu$ are discrete measures with the same
number of atoms $n=m$, and uniform weights. Also, unlike $SGW$ and $RISGW$, these discrepancies are not translation invariant. This approach was studied in \cite{robust_sliced} for the case $p=q$ in the context of point cloud registration. More details are given in Section \ref{proof:robust_sw}.
\end{Remark}

\paragraph{Computational aspects} In the following $\mu,\nu$ are \emph{discrete measures} with the \emph{same}
number of atoms $n=m$, and \emph{uniform weights}, \ie\ $\mu=\frac{1}{n}\sum_{i=1}^{n}\delta_{\xbf_i},\nu=\frac{1}{n}\sum_{i=1}^{n}\delta_{\ybf_i}$ with $\xbf_i \in \R^{p},\ybf_i \in \R^{q}$ so that we can apply Theorem \ref{sovable_gw}. Similarly to Sliced Wasserstein, $\sgw$ can be approximated by replacing the integral by a finite sum over randomly drawn directions. In practice we compute $\sgw$ as the average of $GW_2^2$
projected on $L$ directions $\theta$. While the sum in \eqref{sgw} can be implemented with libraries such as Pykeops \cite{charlier2018keops}, Theorem~\ref{sovable_gw} shows that computing~\eqref{sgw} is achieved by an $O(n\log(n))$ sorting of the projected samples and by finding the optimal
permutation which is either the identity or the anti identity. Moreover computing the cost is $O(n)$ for each projection as explained previously. Thus the overall complexity of computing $\sgw$ with $L$ projections is $O(Ln(p+q)+Ln\log(n)+Ln)=O(Ln(p+q+\log(n)))$ when taking into account the cost of projections. The pseudo-code for $\sgw$ is presented in Algorithm \ref{alg:sgw} Note that these computations can be efficiently implemented in parallel on GPUs with modern toolkits such as Pytorch \cite{paszke2017automatic}. 

The complexity of solving $\risgw$ is higher but one can rely on efficient algorithms for optimizing on the Stiefel manifold \cite{absil2009optimization} that have been implemented in several toolboxes \cite{townsend2016pymanopt,meghwanshi2018mctorch}. Note that each iteration in a manifold gradient decent requires the solution of $\sgw$, that can be computed and differentiated efficiently with the frameworks described above. Moreover, the optimization over the Stiefel manifold does not depend on the number of points but only on the dimension $d$ of the problem so that overall complexity is $n_{\text{iter}}(Ln(d+\log(n))+d^{3})$, which is affordable for small $d$. In practice, we observed in the numerical experiments that RISGW converges in few iterations (the order of $10$).

\begin{algorithm}[t]
\caption{Sliced Gromov-Wasserstein for discrete measures \label{alg:sgw}}
\begin{algorithmic}[1]
    \State $p<q$, $\mu= \frac{1}{n} \sum_{i=1}^{n} \delta_{\xbf_{i}} \in \Pm(\R^{p})$ and $\nu=  \frac{1}{n} \sum_{i=1}^{n} \delta_{\ybf_{j}} \in \Pm(\R^{q})$
    \State $ \forall i, \xbf_{i}\leftarrow \D(\xbf_{i}) $, sample uniformly $(\thetab_{l})_{l =1,\dots,L} \in \Sp^{q-1}$ 
    \For {$l=1,\dots,L$}
    \State Sort $(\langle \xbf_{i}, \thetab_{l}\rangle)_{i}$ and $(\langle \ybf_{j}, \thetab_{l}\rangle)_{j}$ in increasing order
    \State Solve \eqref{eq:qapgw} for reals $(\langle \xbf_{i}, \thetab_{l}\rangle)_{i}$ and $(\langle \ybf_{j}, \thetab_{l}\rangle)_{j}$, $\sigma_{\thetab_{l}}$ (Anti-Id or Id is a solution)        
    \EndFor
    \State return {\small$\frac{1}{n^{2}L} \sum\limits_{l=1}^{L} \sum\limits_{i,k=1}^{n} \big(\langle \xbf_{i}{-}\xbf_{k},\thetab_{l} \rangle^{2}{-}\langle \ybf_{\sigma_{\thetab_{l}}(i)}{-}y_{\sigma_{\thetab_{l}}(k)},\thetab_{l} \rangle^{2}\big)^{2}$}
\end{algorithmic}
\end{algorithm}

\subsection{Experimental results}
\label{sec:expe}

{The goal of this section is to validate $\sgw$ and its rotational invariant on both  quantitative (execution time) and qualitative sides. All the experiments were conducted on a standard  computer equipped with a NVIDIA Titan X GPU.}

\begin{figure}[t]
  \centering
  \includegraphics[width=.5\linewidth]{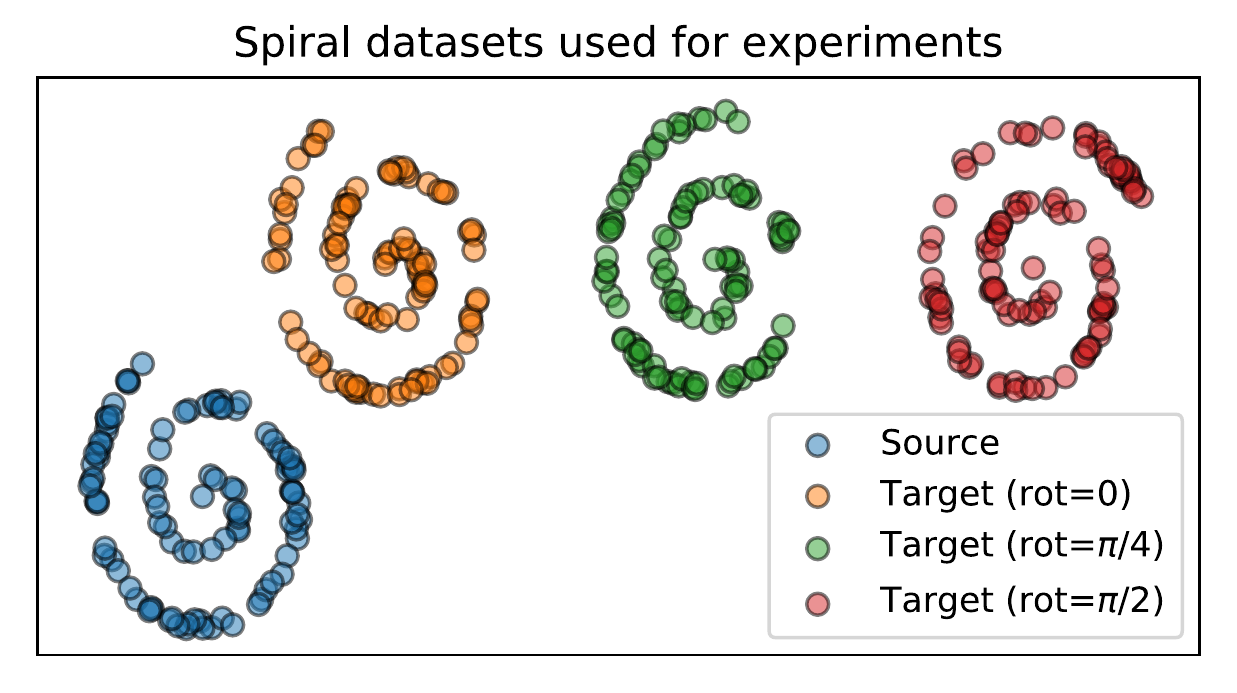}
  \includegraphics[width=.35\linewidth]{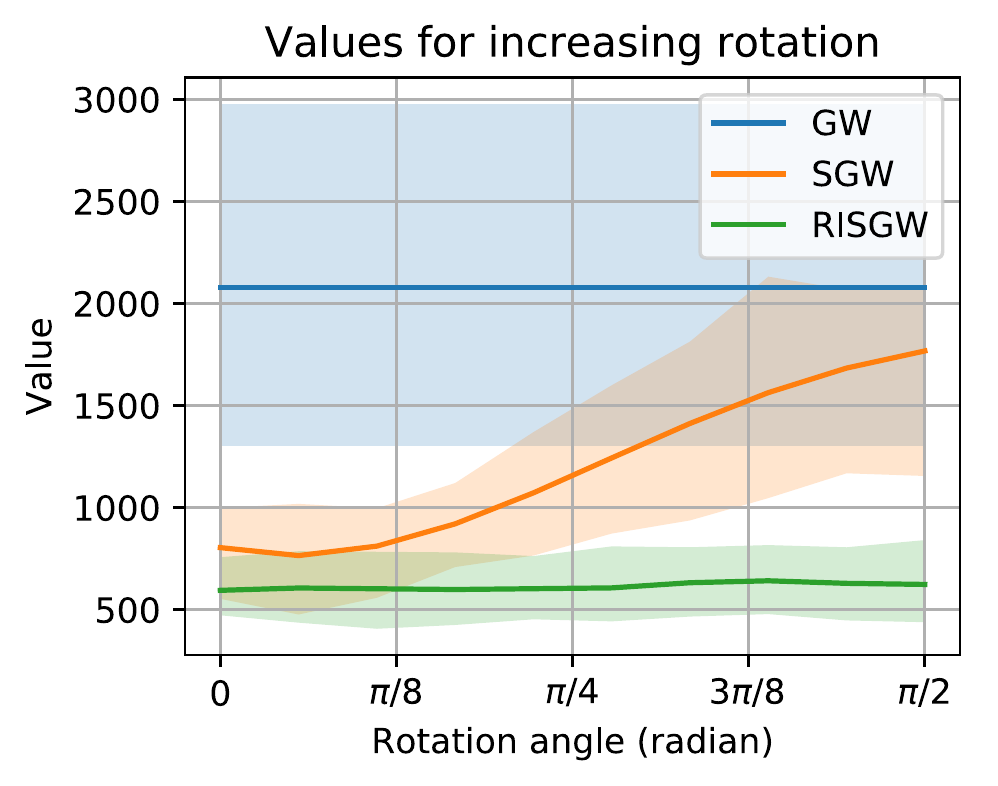}
  \caption{Illustration of $SGW$, $RISGW$ and $GW$ on spiral dataset for varying rotations on discrete 2D spiral dataset. {\bf (left)} Examples of spiral distributions for source and target with different rotations. {\bf (right)} Average value of $SGW$, $GW$ and $RISGW$ with $L=20$ as a function of rotation angle of the target. Colored areas correspond to the 20\% and 80\% percentiles. }
  \label{fig:spiral_example}
\end{figure}

\paragraph{SGW and RISGW on spiral dataset} As a first example, we use the spiral dataset from sklearn toolbox and compute $\gw$, $\sgw$ and $\risgw$ on $n=100$ samples with $L=20$  sampled lines for different rotations of the target distribution. The optimization of $\D$ on the Stiefel manifold is performed using Pymanopt \cite{townsend2016pymanopt} with automatic differentiation with autograd \cite{maclaurin2015autograd}.
Some examples of empirical distributions are available in Figure \ref{fig:spiral_example} (left). The mean value of $\gw$, $\sgw$ and $\risgw$ are reported on Figure \ref{fig:spiral_example} (right) where we can see that $\risgw$ is invariant to rotation as $\gw$ whereas $\sgw$ with $\D=\mathbf{I_p}$ is clearly not.

\paragraph{Runtimes comparison \label{runtimes_comparaison}}
  
We perform a comparison between runtimes of $\sgw$, $\gw$ and its entropic
counterpart \cite{solomon_entropic_2016}. We calculate these
distances between two 2D random measures of $n \in \{1e2,...,1e6\}$ points.
For $\sgw$, the number of projections $L$ is taken from $\{50,200\}$. We use the Python
Optimal Transport (POT) toolbox~\cite{flamary2017pot} to compute $\gw$ distance on CPU. For
entropic-$\gw$ we use the Pytorch GPU implementation from \cite{bunne_gan} that uses the log-stabilized Sinkhorn algorithm
\cite{Schmitzer_stab_sinkhorn} with a 
regularization parameter $\varepsilon=100$.
For $\sgw$, we implemented both a Numpy implementation and a Pytorch
implementation running on GPU.
Figure
\ref{runtimes} illustrates the results. 

\begin{figure}[t]
  \centering
  \includegraphics[width=0.8\linewidth]{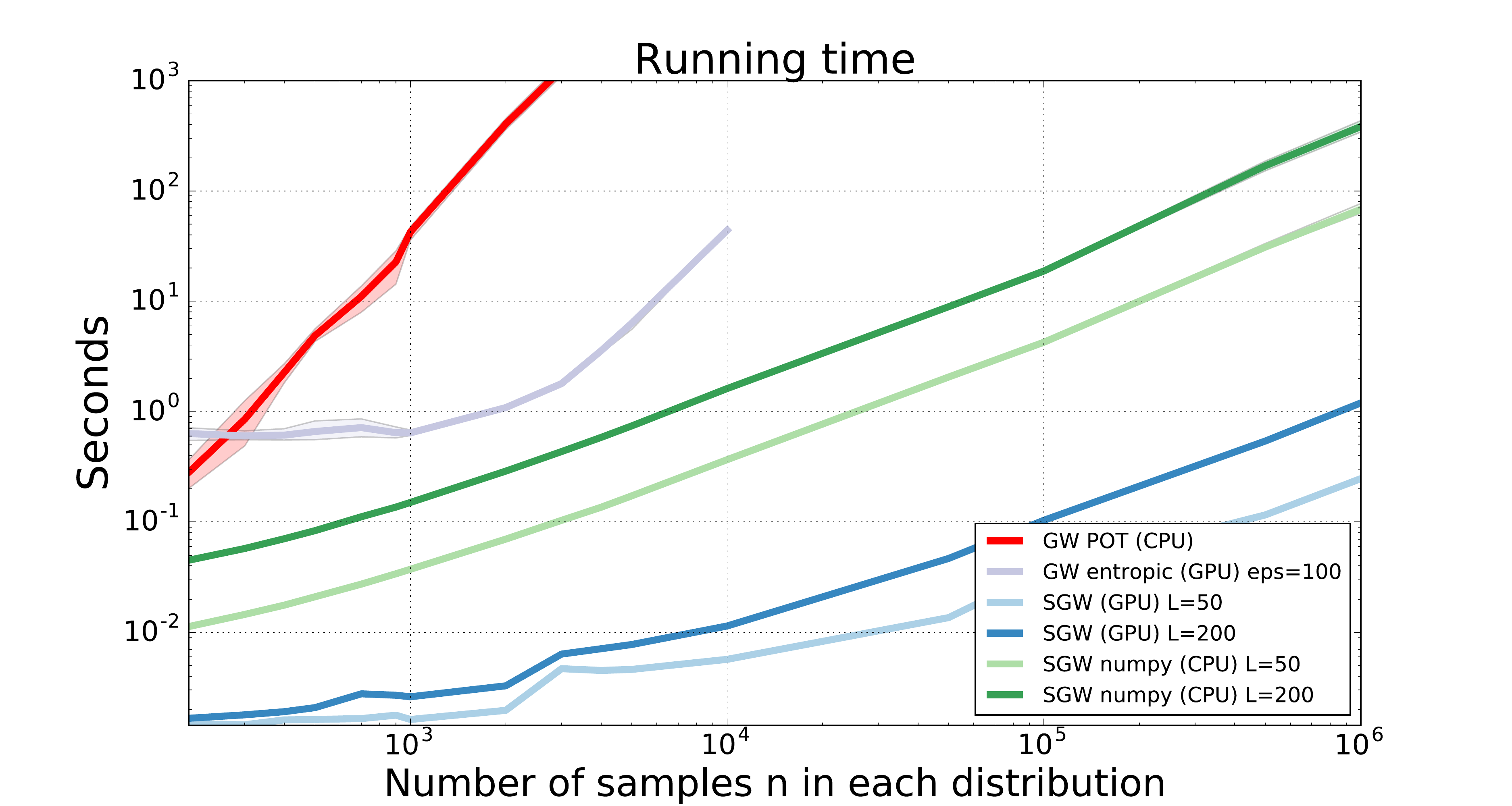}
  \caption{Runtimes comparison between $\sgw$, $\gw$, entropic-$\gw$ between two 2D random distributions with varying number of points from $0$ to $10^6$ in log-log scale. The time includes the calculation of the pair-to-pair distances. \label{runtimes}} 
\end{figure}

$\sgw$ is the only method which scales
\textit{w.r.t.} the number of samples and allows computation for $n>10^4$. 
While entropic-$\gw$
uses GPU, it is still slow because the
gradient step size in the algorithm is inversely proportional to the regularization
parameter \cite{peyre2016gromov} which highly curtails the convergence of the
method.
On CPU, $\sgw$ is two orders of magnitude faster than $\gw$. On GPU, $\sgw$ is five orders of magnitude faster than $\gw$
and four orders of magnitude faster than entropic $\gw$. Still the slope of both $\gw$ implementations are surprisingly good, probably due to their maximum iteration stopping criteria. In this experiment we were able to compute $SGW$ between $10^6$ points in 1s. Finally note that we recover exactly a quasi-linear slope, corresponding to the $O(n\log(n))$ complexity for $\sgw$.

\paragraph{Meshes comparison}

\begin{figure}
\centering
\includegraphics[width=0.6\linewidth]{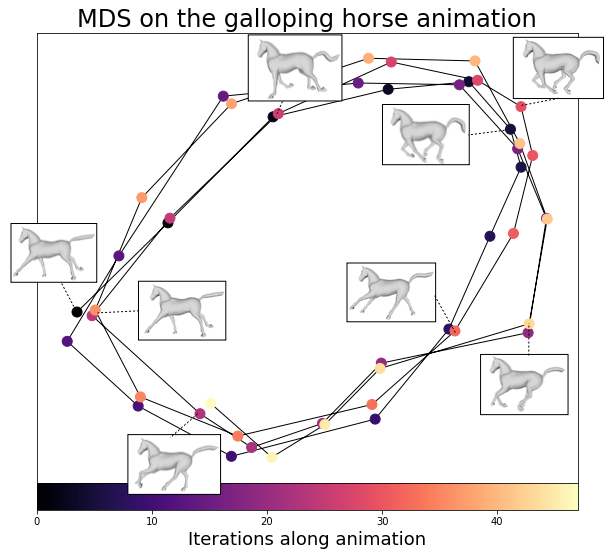}
\caption{Each sample in this Figure corresponds to a mesh and is colored by the corresponding time iteration. One can see that the cyclical nature of the motion is recovered.}\label{fig:gallop}
\end{figure}

In the context of computer graphics, $\gw$ can be used to quantify the correspondances between two meshes. A direct interest is found in shape retrieval, search, exploration or organization of databases. In order to recover experimentally some of the desired properties of the $\gw$ distance, we reproduce an experiment originally conducted in~\cite{rustamov2013map} and presented in~\cite{solomon_entropic_2016} with the use of entropic-$\gw$.

From a given time series of 45 meshes representing a galloping horse, the goal is to conduct a multi-dimensional scaling (MDS) of the pairwise distances, computed with $\sgw$ between the meshes, that allows ploting each mesh as a 2D point. As one can observe in Figure~\ref{fig:gallop}, the cyclical nature of this motion is recovered in this 2D plot, as already illustrated in~\cite{solomon_entropic_2016} with the $\gw$ distance.  Each horse mesh is composed of approximately $9,000$ vertices. 
The average time for computing one distance is around 30 minutes using the POT implementation, which makes the computation of the full pairwise distance matrix impractical (as already mentioned in~\cite{solomon_entropic_2016}). In contrast, our method only requires 25 minutes to compute the full distance matrix, with an average of 1.5s per mesh pair, using our CPU implementation. This clearly highlights the benefits of  our method in this case. 

\paragraph{SGW as a generative adversarial network (GAN)  loss}
In a recent paper~\cite{bunne_gan}, Bunne and colleagues propose a new variant of GAN between incomparable spaces, {\em i.e.} of different dimensions. In contrast with classical divergences such as Wasserstein, they suggest to capture the intrinsic relations between the samples of the target probability distribution by using $\gw$ as a loss for learning. More formally, this translates into the following optimization problem over a desired generator $G$:
\begin{equation}
G^* = \argmin \gw_{2}^{2}(c_{X},c_{G(Z)},\mu,\nu_G),
\end{equation}
where $Z$ is a random noise following a prescribed low-dimensional distribution (typically Gaussian), $G(Z)$ performs the uplifting of $Z$ in the desired dimensional space, and $c_{G(Z)}$ is the corresponding metric. $\mu$  and $\nu_G$ correspond respectively to the target and generated distributions, that we might want to align in the sense of $\gw$. Following the same idea, and the fact that sliced variants of the Wasserstein distance have been successfully used in the context of GAN~\cite{cvpr_sliced_gan}, we propose to use $\sgw$ instead of $\gw$ as a loss for learning $G$.
As a proof of concept, we reproduce the simple toy example of~\cite{bunne_gan}. Those examples consist in generating 2D or 3D distributions from target distributions either in 2D or 3D spaces (Figure~\ref{fig:gan} and Figure~\ref{fig:gan2}). These distributions are formed by $3,000$ samples. We do not use their adversarial metric learning as it might confuse the objectives of this experiment and as it is not required for these low dimensional problems~\cite{bunne_gan}. The generator $G$ is designed as a simple multilayer perceptron with 2 hidden layers of respectively 256 and 128 units with ReLu activation functions, and one final layer with 2 or 3 output neurons (with linear activation) as output, depending on the experiment. The Adam optimizer is used, with a learning rate of $2.10^{-4}$ and $\beta_1=0.5,\beta_2=0.99$. The convergence to a visually acceptable solution takes a few hundred epochs. Contrary to~\cite{bunne_gan}, we directly back-propagate through our loss, without having to explicit a coupling matrix and resorting to the envelope Theorem.  Compared to~\cite{bunne_gan} and the use of entropic-$\gw$ , the time per epoch is more than one order of magnitude faster, as expected from previous experiment.

   \begin{figure}[!t]
   \centering
      \includegraphics[width=0.17\textwidth]{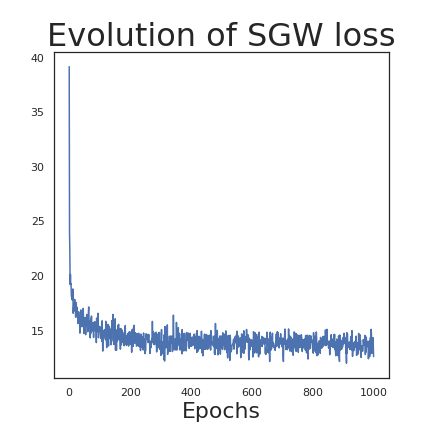}
      \includegraphics[width=0.16\textwidth]{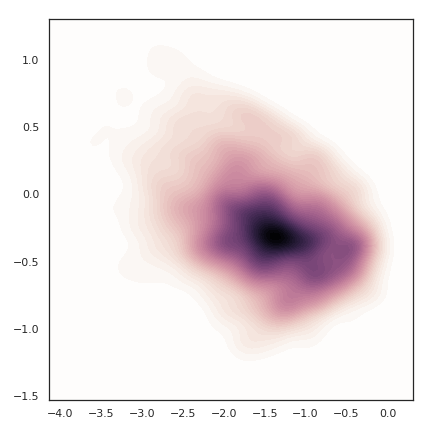}
      \includegraphics[width=0.16\textwidth]{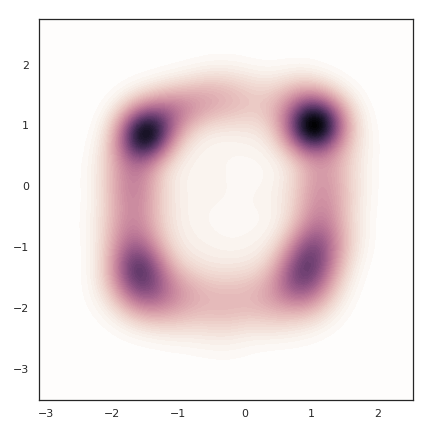}
      \includegraphics[width=0.16\textwidth]{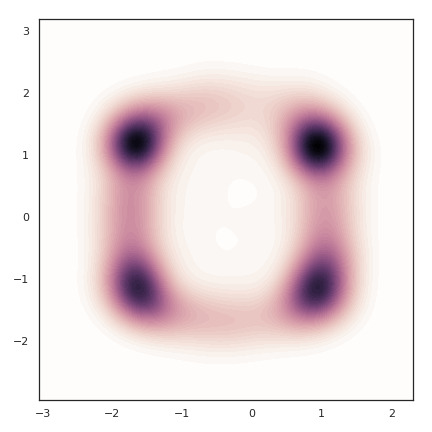}
      \includegraphics[width=0.16\textwidth]{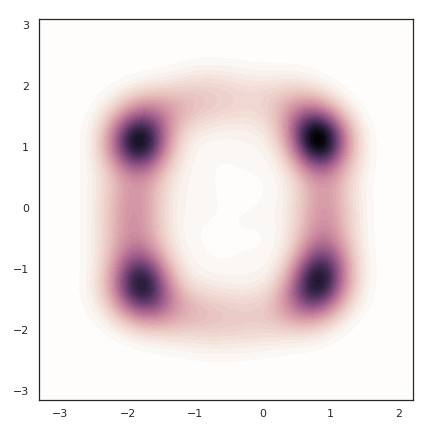}
      \includegraphics[width=0.16\textwidth]{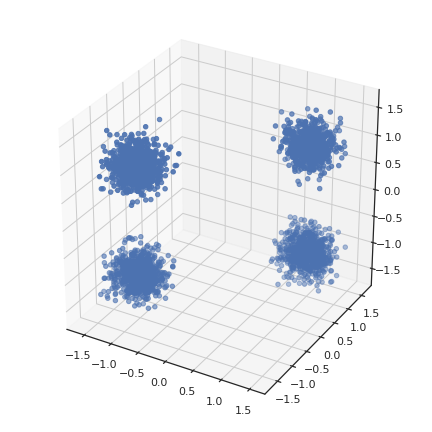}
     \caption{Using $\sgw$ in a GAN loss. First image shows the loss value along epochs. The next 4 images are produced by sampling the generated distribution ($3,000$ samples, plotted as a continuous density map). Last image shows the target 3D distribution.}
     \label{fig:gan}
   \end{figure}

  \begin{figure}[t]
\centering
\resizebox{0.8\textwidth}{!}{
\begin{tabular}{ccc|c}
  \includegraphics[width=0.24\textwidth]{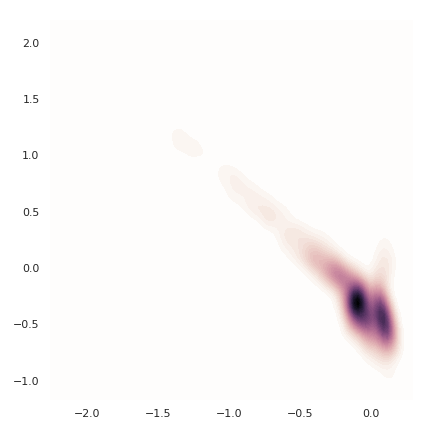}&
  \includegraphics[width=0.24\textwidth]{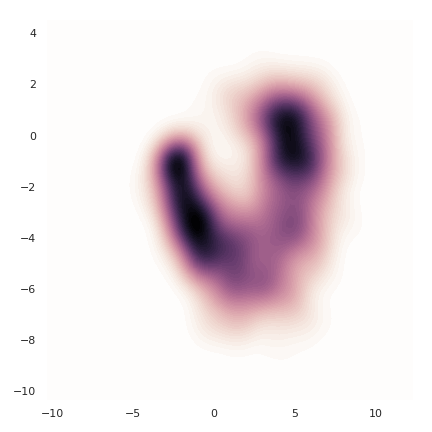}&
  \includegraphics[width=0.24\textwidth]{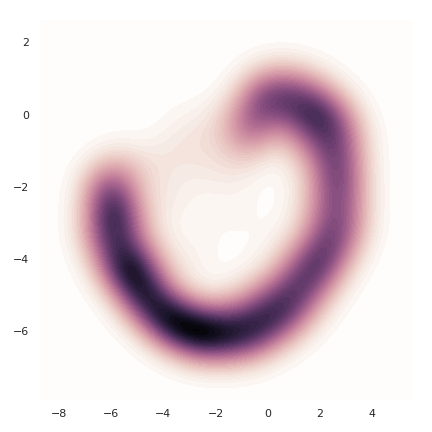}&
  \includegraphics[width=0.24\textwidth]{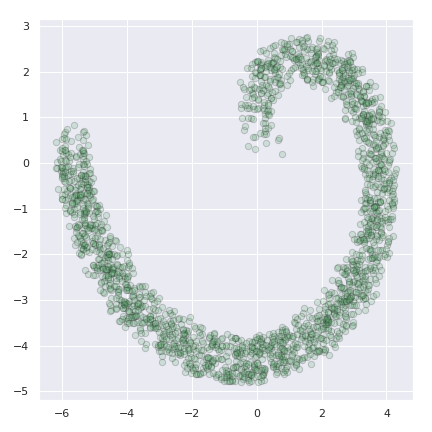}\\
  \includegraphics[width=0.24\textwidth]{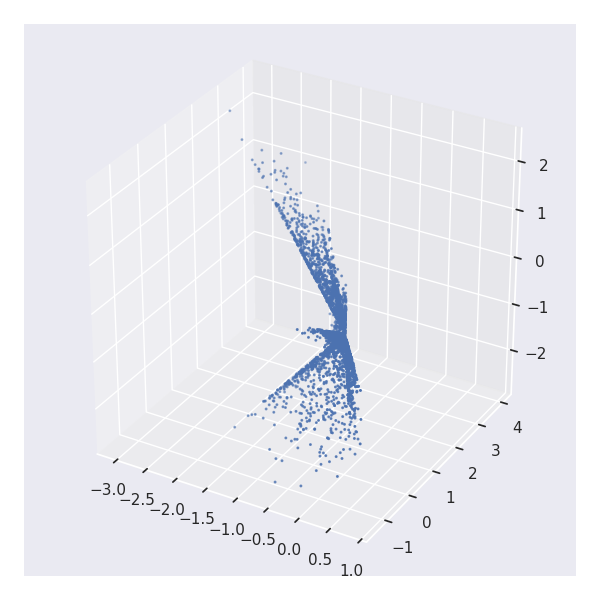}&
  \includegraphics[width=0.24\textwidth]{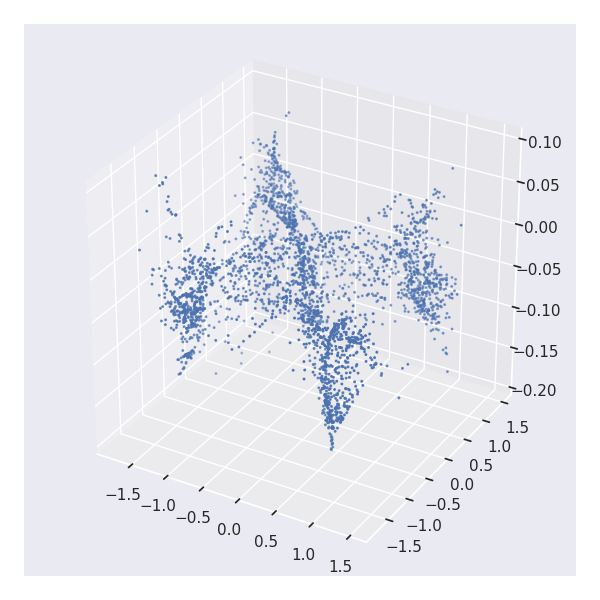}&
  \includegraphics[width=0.24\textwidth]{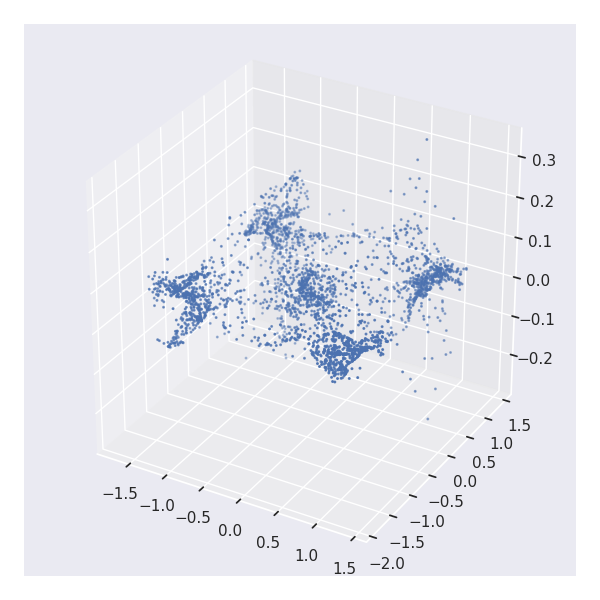}&
  \includegraphics[width=0.24\textwidth]{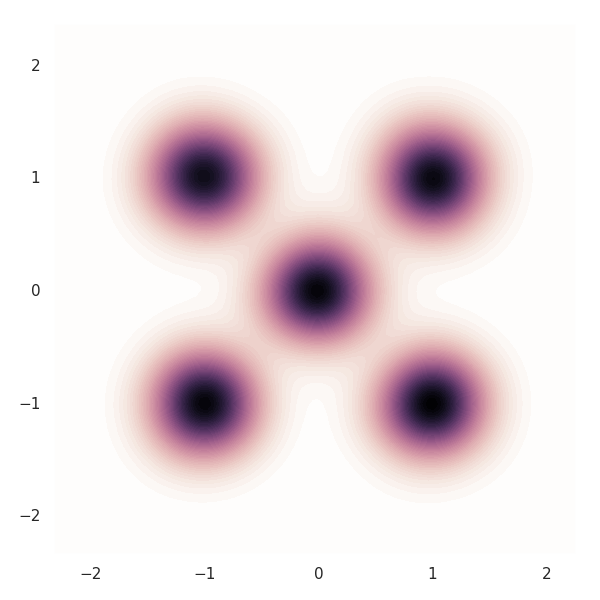}\\
     \end{tabular}}              
 \caption{Using $\sgw$ in a GAN loss. The three rows depicts three different examples. First row is 2D (Generator) to 2D (Target) , Second 3D to 2D. First column is initialization, second one is at $100$ Epochs, third one at $1000$. Last column depicts the target distribution.}
 \label{fig:gan2}
\end{figure}

\subsection{Discussion and conclusion}
In this section we establish a new result about Quadratic Assignment Problem when
matrices are squared euclidean distances on the real line, and use it to state a closed-form
expression for $\gw$ between monodimensional measures. Building upon this result
we define a new similarity measure, called the Sliced Gromov-Wasserstein and a
variant Rotation-invariant $\sgw$ and prove that both conserve various properties of the $\gw$
distance while being cheaper to compute and applicable in a large-scale setting.
Notably $\sgw$ can be computed in 1 second for distributions with 1 million samples each.  This paves the way for novel promising machine learning applications of optimal transport between metric spaces.

Yet, several questions are raised in this work. Notably, our method perfectly
fits the case when the two distributions are given empirically through samples
embedded in 
an Hilbertian space, that allows for projection on the real line.
This is the case in most of the machine learning applications that use the
Gromov-Wasserstein distance.
However, when only distances between samples are available, the projection
operation can not be carried anymore, while the computation of $\gw$ is still
possible. One can argue 
that it is possible to embed either isometrically those distances into a
Hilbertian space, or at least with a low distortion, and then apply the
presented technique. Our future line of work 
considers this option, as well as a possible direct reasoning on the distance
matrix. For example, one should be able to consider geodesic paths (in a graph
for instance) as the equivalent
appropriate geometric object related to the line. This constitutes the direct
follow-up of this work, as well as a better understanding of the accuracy of the
estimated discrepancy with respect to 
the ambiant dimension and the projections number.

\section{Regularity \& formulations of GW problems in Euclidean spaces}

\subsection{Introduction}

In the previous part we built upon the special case of 1D probability discrete measures. We consider in this section general probability measures $\mu \in \P(\R^{p})$ and $\nu \in \P(\R^{q})$ supported on Euclidean spaces $\R^{p}$ and $\R^{q}$ with (possibly) $p \neq q$. The corresponding inner products are denoted by $\scalar{\xbf}{\xbf'}{p}$ (\textit{resp.} $\scalar{\ybf}{\ybf'}{q}$) for vectors in $\R^{p}$ (\textit{resp.} $\R^{q}$) associated with the Euclidean norms which are denoted both by $\|.\|_{2}$ to avoid overloading notations. 

We tackle in this section the problem of the \emph{regularity} of the optimal transport plans of $\gw_2$ in the Euclidean setting. More precisely we consider the following problem:

\begin{prob}
\label{prob_ouvert}
Let $\mu \in \P(\R^{p}),\nu \in \P(\R^{q})$. Can we find an \emph{deterministic} transport map to the $\gw_2$ problem? More precisely does the following statement hold?
\begin{equation*}
\exists T: \R^{p} \rightarrow \R^{q} \text{ such that } T\#\mu=\nu \text{ and } \gamma_{T}=(id \times T) \#\mu \text{ is optimal for the } \gw_2 \text{ problem:}
\end{equation*} 
\begin{equation}
\inf_{\pi \in \couplingset(\mu,\nu)} \int \int |c_\Xcal(\xbf,\xbf')-c_\Ycal(\ybf,\ybf')|^{2} \dr \pi(\xbf,\ybf) \dr \pi(\xbf',\ybf')
\end{equation}
\end{prob}

The Euclidean setting is motivated by the linear transportation theory where the first regularity solution for Optimal Transport was proved by Brenier for probability measures in $\R^{p}$ (see Chapter \ref{cha:ot_general}). In this case we recall that the optimal transport plan $\gamma_{T}$ between two probability measure $(\mu,\nu) \in \P(\R^{p}) \times \P(\R^{p})$, with $\mu$ ``well behaved'' and with cost $c(\xbf,\ybf)=\|\xbf-\ybf\|_2^{2}$, is unique and supported by a map $T$ such that  $\gamma_{T}=(id \times T) \#\mu$. The purpose of this section is to show that the Euclidean setting is also quite suited for the $\gw$ case. The problem of regularity of $\gw$ optimal transport plans was first addressed by Sturm in his seminal work about $\gw$ \cite[Challenges 3.6]{Sturm2012}. More precisely Sturm asks the following question: are there some ``nice'' spaces in which we are able to prove Brenier's like results for $\gw$? We give in this section some partial answers to this query by considering the two cases where $c_{\Xcal},c_{\Ycal}$ are defined by the inner products or by the squared Euclidean distances in each space.

The main results of this section is to derive equivalent formulations of the $\gw$ in these two cases. More precisely we show that solving $\gw$ is equivalent to jointly solve a linear transportation problem and a ``alignment'' problem. As such the regularity of $\gw$ optimal plans can be observed in the light of these ``dual problems''. As another consequence it allows also to derive algorithmic solutions for the $\gw$ problems in Euclidean spaces based on simple Block Coordinate Descent procedures. This section is organized as follow: 

\begin{enumerate}[label=(\roman*)]
\item In Section \ref{sec:inner_product_case} we consider the case where $c_{\Xcal},c_{\Ycal}$ are defined by the inner products in each space. Providing that the source probability measure is regular with respect to the Lebesgue measure we give a sufficient condition for the existence of a \emph{deterministic} optimal transport plan, \ie\ supported on a deterministic function $T$. We show that this function is of the form $\nabla u \circ \Pbf$ where $u$ is a convex function and $\Pbf$ is a linear application which can be seen as a global transformation ``realigning'' the probability measures in the same space. We use this formulation to show that the $\gw$ distance between 1D probability measures admits a closed-form solution. More precisely we show that the optimal coupling is determined by the cumulative and the anti-cumulative distribution functions of the source distribution. We further discuss the difference between the linear OT problem $\wass_2$ and the $\gw$ problem when the target measure is a perturbed version of the source measure. 
\item In Section \ref{sec:sq_eclidean} we consider $c_{\Xcal},c_{\Ycal}$ as the squared Euclidean distances in each space. We show that this setting is equivalent to a maximization of a convex function on $\couplingset(\mu,\nu)$. We use the Fenchel-Legendre duality in the space of measures to derive a problem equivalent to that of Gromov-Wasserstein. We further analyse it and show that the regularity of optimal transport plans is more complicated to state than in the previous case.
\item In Section \ref{sec:numerical} we use the previous formulations to derive efficient numerical solutions for the $\gw$ problem based on Block Coordinate Descent. We show that these procedures compare favourably with respect to standard solvers such as Conditional Gradient (see Chapter \ref{cha:fgw}) or with entropic regularization (see Chapter \ref{cha:ot_general}).
\item We conclude in Section \ref{sec:gm} by considering the \emph{Gromov-Monge} problem in Euclidean spaces, which is the exact counterpart of the Monge problem of linear transportation but in the Gromov-Wasserstein context. We discuss the special case of the Gromov-Monge between Gaussian measures and we show that this problem admits a closed-form solution when restricting to \emph{linear push-forward}. We give geometric interpretations of this result and compare the optimal push-forward with the standard optimal map of linear OT theory in the case of Gaussian measures (see Chapter \ref{cha:ot_general}).
\end{enumerate}

This section is more prospective and somehow opens more doors than it closes. We hope that it will paves the path for further interesting works on this topic and believe that it could help for bridging the gap between the understanding of the linear OT problem and the Gromov-Wasserstein theory. We have chosen to include in the main text some proofs that are reasonably long and that we consider interesting for the overall understanding of the section. The other proofs, which require more space, are postponed to Section \ref{sec:proofs_chap_gw}.

\subsection{The inner product case}
\label{sec:inner_product_case}
To encompass the two cases described in the introduction we consider the following lemma (a proof can be found in Section \ref{sec:proof_reduc}):

\begin{lemma}
\label{calculation_gw}
For a coupling $\pi \in \couplingset(\mu,\nu)$ we note: 
\begin{equation}
\gwloss_2(c_{\Xcal},c_{\Ycal},\pi)\stackrel{def}{=} \int_{\Xcal \times \Xcal} \int_{\Ycal \times \Ycal} |c_{\Xcal}(\xbf,\xbf')-c_{\Ycal}(\ybf,\ybf')|^2 \dr \pi(\xbf,\ybf) \dr \pi(\xbf',\ybf')
\end{equation} 
the $\gw_2$ loss.
Suppose that there exist scalars $a,b,c$ such that $c_{\Xcal}(\xbf,\xbf')=a\|\xbf\|_{2}^{2}+b\|\xbf'\|_{2}^{2}+c\scalar{\xbf}{\xbf'}{p}$ and $c_{\Ycal}(\ybf,\ybf')=a\|\ybf\|_{2}^{2}+b\|\ybf'\|_{2}^{2}+c\scalar{\ybf}{\ybf'}{q}$. Then:
\begin{equation}
\gwloss_2(c_{\Xcal},c_{\Ycal},\pi) = C_{\mu,\nu} -2 Z(\pi)
\end{equation}
where $C_{\mu,\nu}=\int c_{\Xcal}^{2} \dr \mu \dr \mu + \int c_{\Ycal}^{2} \dr \nu \dr \nu-4ab \int \|\xbf\|_{2}^{2} \|\ybf\|_{2}^{2} d\mu(\xbf) d\nu(\ybf)$ and:
\begin{equation}
\begin{split}
Z(\pi) &= (a^{2}+b^{2}) \int \|\xbf\|_{2}^{2} \|\ybf\|_{2}^{2} \dr\pi(\xbf,\ybf)+ c^{2} \|\int \ybf \xbf^{T} \dr\pi(\xbf,\ybf)\|_{\F}^{2} \\
&+(a+b)c\int \big[\|\xbf\|_{2}^{2} \scalar{\E_{Y\sim\nu}[Y]}{\ybf}{q} + \|\ybf\|_{2}^{2}\scalar{\E_{X\sim \mu}[X]}{\xbf}{p} \dr\pi(\xbf,\ybf)\big] \\
\end{split}
\end{equation}

\end{lemma}

In this section we study the $\gw$ problem with $c_{\Xcal}=\scalar{\xbf}{\xbf'}{p}$ and $c_{\Ycal}=\scalar{\ybf}{\ybf'}{q}$. This corresponds to $a,b=0$ and $c=1$ case of Lemma \ref{calculation_gw}. With a small abuse of notation we will denote by $\Pbf \in \R^{q \times p}$ both the linear application $\Pbf : \R^{p} \rightarrow \R^{q}$ and its associated matrix. Moreover we will often make no distinction between a vector $\xbf \in \R^{p}$ and the matrix associated to $\xbf \in \R^{p\times 1}$ such that $\xbf^{T} \in \R^{1 \times p}$.  The next theorem gives a equivalent formulation of $\gw_2$ in this context:

\begin{theo}[Equivalence of GW for the inner product case]
\label{maintheorem}

Let $\mu \in \P(\R^{p}),\nu \in \P(\R^{q})$ with $\int \|\xbf\|_{2}^{4}\dr\mu(x)<+\infty,\int \|\ybf\|_{2}^{4}\dr\nu(\ybf)<+\infty$. Suppose without loss of generality that $p\geq q$ and let: 
\begin{equation}
F_{p,q}\stackrel{def}{=}\{\Pbf \in \R^{q\times p}  | \ \|\Pbf\|_{\F} = \sqrt{p} \} 
\end{equation} 
Then problems:
\begin{equation}
\tag{innerGW}
\label{GWprob}
\underset{\pi \in \Pi(\mu,\nu)}{\inf} \int \int \big( \scalar{\xbf}{\xbf'}{p} - \scalar{\ybf}{\ybf'}{q} \big)^{2} \dr\pi(\xbf,\ybf)\dr\pi(\xbf',\ybf')
\end{equation}
\begin{equation}
\tag{MaxOT}
\label{InvOT}
\underset{\pi \in \Pi(\mu,\nu)}{\sup} \underset{\Pbf \in F_{p,q}}{\sup} \int \scalar{\Pbf \xbf}{\ybf}{q} \ \dr\pi(\xbf,\ybf)
\end{equation}
are equivalent. In other words, $\pi^{*} \in \couplingset(\mu,\nu)$ is an optimal solution of \eqref{GWprob} if and only if $\pi^{*}$ is an optimal solution of \eqref{InvOT}.
\end{theo}

\begin{Remark}
The condition $\int \|\xbf\|_{2}^{4}\dr\mu(\xbf)<+\infty,\int \|\ybf\|_{2}^{4}\dr\nu(\ybf)<+\infty$ suffices to prove that both \eqref{GWprob} and \eqref{InvOT} are finite and that \eqref{InvOT} admits an optimal solution $\pi^{*} \in \Pi(\mu,\nu)$ (we postponed this study to Lemma \ref{lemma:existence_compacity} in Section \ref{sec:finitesness}).
\end{Remark}

 This theorem gives another interesting formulation of the Gromov-Wasserstein problem. It proves that $\gw$ is equivalent to a linear OT problem combined with an ``alignment'' of the measures $\mu$ and $\nu$ on the same space using a linear application $\Pbf$. The set $F_{p,q}$ can be regarded as the set of matrices with fixed Schatten $\ell_2$ norms, that is $\Pbf \in F_{p,q}$ if $\|\sigma(\Pbf)\|_2= \sqrt{p}$ where $\sigma(\Pbf)$ is a vector containing the singular values of $\Pbf$. When $p=q$ any orthogonal matrix $\Obf \in \mathcal{O}(p)$ is in $F_{p,q}$ since $\|\Obf\|_{\F}=\sqrt{\tr(\Obf^{T}\Obf)}=\sqrt{\tr(\mathbf{I_p})}=\sqrt{p}$. More generally when $p>q$ any matrix in the Stiefel manifold $\Delta \in \Stief$ is an element of $F_{p,q}$ since $\Delta^{T} \Delta= \mathbf{I_p}$. Interestingly enough, the problem \eqref{InvOT} can be related to the work of Alvarez and coauthors \cite{alavarez:2019} where they proposed a linear optimal transport problem which takes into account a latent global transformation of the measures. More precisely they consider two probability measures $(\mu,\nu) \in \P(\R^{p}) \times \R^{p}$ (\ie\ $p=q$) and propose two minimize the following problem:
 \begin{equation}
 \label{eq:alvarez}
InvOT(\mu,\nu)=\min_{\pi \in \couplingset(\mu,\nu)} \min_{\|\Pbf\|_{\F}=\sqrt{p}} \int \|\Pbf \xbf- \ybf \|_2^{2} \dr \pi(\xbf,\ybf)
 \end{equation}
 If we note $\Sigmab_\mu=\int \xbf \xbf^{T} \dr \mu(\xbf)$ and we suppose that $\Sigmab_\mu=\mathbf{I_p}$ (which is called the $\mu$-whitened property in \cite{alavarez:2019}) then problem \eqref{eq:alvarez} is equivalent to \eqref{InvOT}. To see this is suffices to develop $\int \|\Pbf \xbf- \ybf \|_2^{2} \dr \pi(\xbf,\ybf)$ as:
 \begin{equation}
 \begin{split}
\int \|\Pbf \xbf- \ybf \|_2^{2} \dr \pi(\xbf,\ybf) &= \int \|\Pbf \xbf\|_2^{2} \dr \mu(\xbf) + \int \|\ybf\|_2^{2} \dr \nu(\ybf) -2 \int \scalar{\Pbf \xbf}{\ybf}{p} \dr \pi(\xbf,\ybf) \\
\end{split}
 \end{equation}
Then we can check that $\int \|\Pbf \xbf\|_2^{2} \dr \mu(\xbf)$ does not depend on $\Pbf$ since:
\begin{equation}
\begin{split}
\int \|\Pbf \xbf\|_2^{2} \dr \mu(\xbf)&=\int \xbf^{T} \Pbf^{T} \Pbf \xbf   \dr \mu(\xbf)\stackrel{*}{=}\int \tr(\xbf^{T} \Pbf^{T} \Pbf \xbf)   \dr \mu(\xbf)\stackrel{**}{=}\int \tr(\Pbf^{T} \Pbf \xbf\xbf^{T} )   \dr \mu(\xbf) \\
&\stackrel{***}{=} \tr(\Pbf^{T} \Pbf \int \xbf  \xbf^{T} \dr \mu(\xbf))=\tr(\Pbf^{T} \Pbf)=\|\Pbf\|_{\F}^{2} = p
\end{split}
\end{equation}
where in (*) we used $\xbf^{T} \Pbf^{T} \Pbf \xbf \in \R$, in (**) that the trace is invariant by cyclical permutation and in (***) the linearity of the trace. Finally we used that $\tr(\Pbf^{T} \Pbf)=\|\Pbf\|^{2}_{\F}=p$ by hypothesis. However note that in general both problems may differ since the $\mu$-whitened property does not hold in general.

Theorem \ref{maintheorem} is based on the following generalization of the Frobenius norm duality to the continuous setting: 

\begin{lemma}
\label{froebnormduality}
For any $\mu \in \P(\R^{p}),\nu \in \P(\R^{q})$ and $\pi \in \Pi(\mu,\nu)$. Then:
\begin{equation}
\underset{\ \|\Pbf\|_{\F}= \sqrt{p}}{\sup} \int \scalar{\Pbf \xbf}{\ybf}{q} \dr\pi(\xbf,\ybf)=\sqrt{p} \|\int \ybf\xbf^{T}\dr\pi(\xbf,\ybf)\|_{\F}
\end{equation}
This supremum is achieved for $\Pbf^{*}=\frac{\sqrt{p}}{\|\int \ybf\xbf^{T}\dr\pi(\xbf,\ybf)\|_{\F}}\int \ybf\xbf^{T}\dr\pi(\xbf,\ybf)$
\end{lemma}

\begin{proof}
We have:
\begin{equation*}
\begin{split}
&\int \langle \Pbf \xbf, \ybf \rangle_{q} \dr\pi(\xbf,\ybf)=\int  \ybf^{T} \Pbf \xbf \dr\pi(\xbf,\ybf) \stackrel{*}{=} \int \tr(\ybf^{T} \Pbf \xbf) \dr\pi(\xbf,\ybf) \stackrel{**}{=} \int \tr(\Pbf \xbf\ybf^{T}) \dr\pi(\xbf,\ybf) \\
&\stackrel{***}{=}  \tr(\Pbf \int \xbf\ybf^{T} \dr\pi(\xbf,\ybf)) = \froeb{\int \ybf\xbf^{T} \dr\pi(\xbf,\ybf)}{\Pbf}
\end{split}
\end{equation*}
where in (*) we used that $\ybf^{T} \Pbf \xbf \in \R$, in (**) we used the cyclical permutation invariance of the trace and in (***) its linearity. Hence $\underset{\ \|\Pbf\|_{\F}= \sqrt{p}}{\sup} \int \scalar{\Pbf \xbf}{\ybf}{q} \ \dr\pi(\xbf,\ybf)= \underset{\ \|\Pbf\|_{\F}= \sqrt{p}}{\sup} \froeb{\Pbf}{\int \ybf\xbf^{T} \dr\pi(\xbf,\ybf)}$. We note $\mathbf{V_{\pi}}=\int \ybf\xbf^{T} \dr\pi(\xbf,\ybf)$. We want to solve:
\begin{equation}
\underset{\ \|\Pbf\|_{\F}= \sqrt{p}}{\sup} \froeb{\Pbf}{\mathbf{V_{\pi}}}
\end{equation}
Let $\Pbf$ such that $\|\Pbf\|_{\F}= \sqrt{p}$. Then by Cauchy-Schwartz (see Memo \ref{memo:ics}) $\froeb{\Pbf}{\mathbf{V_{\pi}}} \leq \|\Pbf\|_{\F}\|\mathbf{V_{\pi}}\|_{\F}=\sqrt{p}\|\mathbf{V_{\pi}}\|_{\F}$. Hence, $\underset{\ \|\Pbf\|_{F}= \sqrt{p}}{\sup} \froeb{\Pbf}{\mathbf{V_{\pi}}} \leq \sqrt{p}\|\mathbf{V_{\pi}}\|_{\F}$. 

Conversely, take $\Pbf^{*}=\sqrt{p}\frac{\mathbf{V_{\pi}}}{\|\mathbf{V_{\pi}}\|_{\F}}$. Then $\|\Pbf^{*}\|_{\F}=\sqrt{p}$ so $\underset{\ \|\Pbf\|_{\F}= \sqrt{p}}{\sup} \froeb{\Pbf}{\mathbf{V_{\pi}}} \geq \froeb{\Pbf^{*}}{\mathbf{V_{\pi}}}= \froeb{\sqrt{p}\frac{\mathbf{V_{\pi}}}{\|\mathbf{V_{\pi}}\|_{F}}}{\mathbf{V_{\pi}}}=\sqrt{p} \|\mathbf{V_{\pi}}\|_{\F}$ which concludes the proof.

\end{proof}

\begin{figure*}[!b]
\begin{memo}[Cauchy-Schwartz inequality]
\label{memo:ics}
Let $\Omega$ be a vector space associated with an inner product $\langle,\rangle$ which defines a norm $\|.\|$ through $\|\xbf\|=\sqrt{\langle\xbf,\xbf\rangle}$ for $\xbf \in \Omega$. The Cauchy-Schwartz inequality reads:
\begin{equation}
\forall \xbf,\ybf \in \Omega^{2}, \ |\langle \xbf, \ybf\rangle|\leq \|\xbf\| \|\ybf\|
\end{equation}
\end{memo}
\end{figure*}

Combining Lemma \ref{calculation_gw} and Lemma \ref{froebnormduality} actually proves Theorem \ref{maintheorem}. Indeed using Lemma \ref{calculation_gw} we see that \eqref{GWprob} is equivalent to maximizing $Z(\pi)= \|\int \ybf\xbf^{T} \dr\pi(\xbf,\ybf)\|_{\F}^{2}$ over the couplings $\pi \in \couplingset(\mu,\nu)$ since the other terms are constant. In this way it is equivalent to maximize $\|\int \ybf\xbf^{T} \dr\pi(\xbf,\ybf)\|_{\F}$ which is equivalent by Lemma \ref{froebnormduality} to maximize $\underset{\Pbf \in F_{p,q}}{\sup} \int \scalar{\Pbf \xbf}{\ybf}{q} \ \dr\pi(\xbf,\ybf)$ \textit{w.r.t.} $\pi$.

\paragraph{Regularity of \eqref{GWprob} OT plans} Theorem \ref{maintheorem} proves that it is equivalent to study the problem \eqref{InvOT} for studying the regularity of $\gw$ optimal transport plans. Interestingly enough the problem \eqref{InvOT} echoes the linear transportation problem $\sup_{\pi \in \couplingset(\mu,\nu)} \int \scalar{\xbf}{\ybf}{p} \dr \pi(\xbf,\ybf)$ when $\mu,\nu \in \P(\R^{p}) \times \P(\R^{p})$ which is widely studied in the literature an can be tackled using tools from convexity analysis such as the Legendre transform. The following result due to McCann is particularly useful in this case:

\begin{theo}[\cite{mccann1995}]
\label{mccantheo}
Let $\mu \in \P(\R^{p}),\nu \in \P(\R^{p})$. Suppose that $\mu$ is absolutely continuous with respect to the Lebesgue measure, then there exists a convex function $u: \R^{p} \rightarrow \R$ whose gradient $\nabla u$ pushes $\mu$ forward to $\nu$, \ie $\nabla u \# \mu =\nu$. Moreover $\nabla u$ is unique $\mu$ a.e. 
\end{theo}
By noticing that, for all $\xbf,\ybf \in \R^{p}\times \R^{p}$, $\scalar{\xbf}{\ybf}{p} \leq u^{*}(\xbf)+u(\ybf)$ where $u^{*}$ is the Legendre transform of the convex function $u$ the result of McCann proves that the map $\nabla u$ defines an optimal coupling $\gamma=(id\times \nabla u)\#\mu$ for the problem $\sup_{\pi \in \couplingset(\mu,\nu)} \int \scalar{\xbf}{\ybf}{p} \dr \pi(\xbf,\ybf)$ between $(\mu,\nu)\in \P(\R^{p}) \times \P(\R^{p})$. Indeed for any coupling $\pi \in \couplingset(\mu,\nu)$: 
\begin{equation}
\label{eq:eq_behind_mccan}
\begin{split}
\int \scalar{\xbf}{\ybf}{p}\dr \pi(\xbf,\ybf) &\leq \int u^{*}(\xbf)+u(\ybf) \dr \pi(\xbf,\ybf) \\
&= \int u^{*}(\xbf) \dr \mu(\xbf) +\int u(\ybf) \dr \nu(\ybf)
\end{split}
\end{equation}
Using that $\nabla u$ pushes $\mu$ forward to $\nu$ implies:
\begin{equation}
\begin{split}
\int \scalar{\xbf}{\ybf}{p}\dr \pi(\xbf,\ybf) &\leq \int u^{*}(\xbf) \dr \mu(x) +\int u(\nabla u (\xbf)) \dr \mu(\xbf)\\
&=\int \scalar{\xbf}{\nabla u (\xbf)}{p}\dr \mu(\xbf)
\end{split}
\end{equation}
since for any convex function $u^{*}(\xbf)+u(\xbf)=\scalar{\xbf}{\nabla u (\xbf)}{}$. Overall $\int \scalar{\xbf}{\ybf}{p}\dr \pi(\xbf,\ybf) \leq \int \scalar{\xbf}{\ybf}{p}\dr \gamma(\xbf,\ybf)$ for any coupling $\pi$ which proves that $\gamma$ is optimal. The idea is to use the same reasoning to find an optimal solution of \eqref{InvOT}. In order to invoke McCann's theorem we will need the regularity of the probability measure $\Pbf\#\mu$ for $\Pbf \in F_{p,q}$:

\begin{prop}
\label{prop:regulariffsurj}
Let $\mu \in \P(\R^{p})$ regular with respect to the Lebesgue measure in $\R^{p}$ and a linear map $l: \R^{p} \rightarrow \R^{q}$ associated with a matrix $\mathbf{L} \in \R^{q \times p}$
\begin{itemize}
\item If $p<q$ (we go from lower to higher dimension), then $l \# \mu \in \R^{q}$ is not regular with respect to the Lebesgue measure on $\R^{q}$.
\item If $p\geq q$ (we go from higher to lower dimension), then $l \# \mu \in \R^{q}$ is regular with respect to the Lebesgue measure on $\R^{q}$ if and only if the linear map $l$ is surjective, that is $\rg(\mathbf{L})=q$.
\end{itemize} 
\end{prop}
\begin{proof}
 We have $l\# \mu(l(\R^{p}))=\mu(l^{-1}(l(\R^{p}))=\mu(\R^{p})=1$ so that $l\# \mu$ gives measure $1$ to the image of $l$. However for the first point the image of $l$ is a strict linear subspace of $\R^{q}$ and therefore has Lebesgue measure zero. Using Radon-Nykodym theorem this implies that $l\# \mu$ can not have a density with respect to the Lebesgue on $\R^{q}$. 
For the second point, suppose that $l$ is surjective. Then $\rg(\mathbf{L})=q$ which implies that $\mathbf{L}\mathbf{L}^{T}$ is invertible so that $\det(\mathbf{L}\mathbf{L}^{T})\neq 0$. We define $J=\sqrt{\det(\mathbf{L}\mathbf{L}^{T})}$. Let $g$ be the density of $\mu$ with respect to the Lebesgue measure in $\R^{p}$. Then by the coarea formula the density of $l\# \mu$ with respect to the Lebesgue measure on $\R^{q}$ is:
\begin{equation}
h(y)=\int_{l^{-1}(y)} \frac{g(\xbf)}{J} dV_{l^{-1}(\ybf)}(\xbf)
\end{equation}
where denotes $dV_{l^{-1}(\ybf)}(\xbf)$ the volume element. Conversely if $l$ is not surjective then $\rg(\mathbf{L})<q$. Then the image of $l$ is a strict linear subspace of $\R^{q}$ with Lebesgue measure zero and therefore $l\# \mu$ can not have a density with respect to the Lebesgue on $\R^{q}$. 
\end{proof}
In the light of previous results we can give the following sufficient condition so that \eqref{GWprob} problem admits an optimal transport plan supported on a deterministic function:
\begin{theo}
\label{first_sufficient}
Let $\mu \in \P(\R^{p}),\nu \in \P(\R^{q})$ with $\int \|\xbf\|_{2}^{4}\dr\mu(\xbf)<+\infty,\int \|\ybf\|_{2}^{4}\dr\nu(\ybf)<+\infty$. Suppose that $p\geq q$ and that $\mu$ is regular with respect to the Lebesgue measure in $\R^{p}$. Suppose that there exists $(\pi^{*},\Pbf^{*})$ an optimal solution of \eqref{InvOT} with $\Pbf^{*}$ surjective. 

Then there exists $u: \R^{q} \rightarrow \R$ convex such that $\nabla u \circ \Pbf^{*}$ pushes $\mu$ forward to $\nu$. Moreover the coupling $\gamma=(id \times \nabla u \circ \Pbf^{*})\#\mu$ is optimal for \eqref{GWprob}. In particular problem \ref{prob_ouvert} holds.
\end{theo}
\begin{proof}
Let $(\pi^{*},\Pbf^{*})$ be maximizers of \eqref{InvOT} with $\Pbf^{*}$ surjective. Using Proposition \ref{prop:regulariffsurj} we know that $\Pbf^{*}\#\mu$ is regular with respect to the Lebesgue measure on $\R^{q}$. Using Theorem \ref{mccantheo} there exists $u: \R^{q} \rightarrow \R$ convex such that $\nabla u \# \Pbf^{*}\#\mu=\nu$ or equivalently $\nabla u \circ \Pbf^{*}$ pushes $\mu$ forward to $\nu$. Moreover we have:
\begin{equation*}
\begin{split}
\int \scalar{\Pbf^{*} \xbf}{ \ybf}{d} \dr \pi^{*}(\xbf, \ybf) & \stackrel{1}{\leq} \int \left(u(\Pbf^{*}x)+u^{*}(\ybf)\right) \dr \pi^{*}(\xbf, \ybf) = \int u(\Pbf^{*} \xbf) \dr \mu(\xbf) + \int u^{*}(\ybf) \dr \nu(\ybf) \\
&\stackrel{2}{=} \int u(\Pbf^{*} \xbf) \dr \mu(\xbf) + \int u^{*}(\ybf) \dr (\nabla u \circ \Pbf^{*}\#\mu)(\ybf) \\
&= \int u(\Pbf^{*} \xbf) \dr \mu(x) + \int u^{*}(\nabla u(\Pbf^{*}\xbf)) \dr \mu(\xbf) \\
&\stackrel{3}{=} \int \scalar{\Pbf^{*} \xbf}{ \nabla u (\Pbf^{*} \xbf)}{q} \dr \mu(\xbf) \\
\end{split}
\end{equation*}
where in (1) we used that $u$ is convex which implies $u(\Pbf \xbf)+u^{*}(\ybf)\geq \scalar{\Pbf \xbf}{\ybf}{q}$ by Fenchel-Young inequality, in (2) we used $\nabla u \circ \Pbf^{*}$ pushes $\mu$ forward to $\nu$, in (3) we used that for any $x$ and convex function $u^{*}(\xbf)+u(\xbf)=\scalar{\xbf}{\nabla u (\xbf)}{}$.

If we define $T=\nabla u \circ \pbf^{*}$ and $\gamma=(id \times T)\#\mu \in \Pi(\mu,\nu)$ we can deduce from (3) that: 
\begin{equation*}
\underset{\pi \in \Pi(\mu,\nu)}{\sup}\underset{\Pbf \in F_{p,q}}{\sup} \int \scalar{\Pbf \xbf}{\ybf}{q} \dr \pi(\xbf, \ybf) \leq \int \scalar{\Pbf^{*} \xbf}{ y}{q} \dr \gamma(\xbf, \ybf)
\end{equation*}
By suboptimality the converse inequality is also true (since $\gamma=(id \times T)\#\mu \in \Pi(\mu,\nu)$). In this way we have:
\begin{equation*}
\underset{\pi \in \Pi(\mu,\nu)}{\sup}\underset{\Pbf \in F_{p,q}}{\sup} \int \scalar{\Pbf \xbf}{\ybf}{q} \dr \pi(\xbf, \ybf = \int \scalar{\Pbf^{*} \xbf}{\ybf}{q} \dr \gamma(\xbf,\ybf)
\end{equation*}
Overall the couple $(\gamma,\Pbf^{*})$ is optimal for the problem \eqref{InvOT} and $\gamma$ is optimal for \eqref{GWprob} using Theorem \ref{maintheorem} which concludes the proof.

\end{proof}

The last result indicates that Monge map of the form $\nabla u \circ \Pbf$ where $\Pbf$ is a linear application may be of interest to study optimal solutions of \eqref{GWprob}. When considering $p=q$ and $\mu,\nu \in \P(\R^{p})\times \P(\R^{p})$ looking at such maps can generate the optimal Monge map $\nabla u$ of linear optimal transport problem $\sup_{\pi \in \couplingset(\mu,\nu)} \int \scalar{\xbf}{\ybf}{p} \dr \pi(\xbf,\ybf)$ when $\Pbf=\mathbf{I_p}$. Note that a couple $(u,\Pbf)$ which satisfies both $\nabla u \circ \Pbf$ pushes $\mu$ onto $\nu$ and $\gamma=(id \times \nabla u \circ \Pbf)$ is optimal for \eqref{GWprob} is not guaranteed to be unique. This difference should be put in perspective with the theory of linear transportation with ground cost $c(\xbf,\ybf)=\|\xbf-\ybf\|_{2}^{2}$. Indeed in this context if one finds a mapping $T=\nabla u$ with any convex function $u$ then Brenier's theorem states that this mapping is optimal since there is a unique Monge map satisfying this property. This unicity result is particularly interesting \textit{e.g.} in order to prove that the linear transport between Gaussian measure admits a closed-form expression \cite{takatsu2011}. However in our context the map $\nabla u \circ \Pbf$ may fail to be unique as shown in the following example:
\begin{Example}
\label{example_non_unicity}
Consider $p=q$, a source measure $\mu \in \P(\R^{p})$ and a target measure $\nu \in \P(\R^{p})$ whose support is invariant by rotation, \ie\ such that $\Obf \# \nu=\nu$ for any $\Obf \in \mathcal{O}(p)$. We can consider \textit{e.g.} any isotropic Gaussian measure $\mathcal{N}(\mathbf{m_{\nu}},\sigma^{2} \mathbf{I_p})$. Since $\Obf$ preserves the angles then the problem \eqref{GWprob} is invariant by $\Obf$ which implies that any optimal map defined in Theorem \ref{first_sufficient} will fail to be unique. More precisely let suppose that $\nabla u \circ \Pbf$ pushes $\mu$ forward to $\nu$ and that $\gamma=(id\times \nabla u \circ \Pbf)$ is optimal for \eqref{GWprob}. Then for any $\Obf \in \mathcal{O}(p)$, $\Obf^{T} \circ \nabla u \circ \Pbf $ also pushes $\mu$ forward to $\nu$ since $\nabla u \circ \Pbf \# \mu=\nu=\Obf \# \nu$ by rotational invariance of the support of $\nu$. This implies that $\Obf^{T} \circ \nabla u \circ \Pbf \# \mu=\nu$. Moreover the map $\gamma'=(id\times  \Obf^{T} \circ \nabla u \circ \Pbf )$ is also optimal with the same cost as $\gamma$. Indeed:
\begin{equation*}
\begin{split}
\int \int \big(\scalar{\xbf}{\xbf'}{p} - \scalar{\ybf}{\ybf'}{p}\big)^{2} \dr\gamma'(\xbf,\ybf)\dr\gamma'(\xbf',\ybf')&=\int \int \big(\scalar{\xbf}{\xbf'}{p} - \scalar{\Obf^{T}\nabla u (\Pbf \xbf)}{\Obf^{T}\nabla u (\Pbf \xbf')}{p}\big)^{2}\dr\mu(\xbf)\dr\mu(\xbf')\\
&\stackrel{*}{=}\int \int \big(\scalar{\xbf}{\xbf'}{p} - \scalar{\nabla u (\Pbf \xbf)}{\nabla u (\Pbf \xbf')}{p}\big)^{2}\dr\mu(\xbf)\dr\mu(\xbf')\\
&=\int \int \big(\scalar{\xbf}{\xbf'}{p} - \scalar{ \ybf}{\ybf'}{p}\big)^{2} \dr\gamma(\xbf,\ybf)\dr\gamma(\xbf',\ybf') \\
\end{split}
\end{equation*}
where in (*) we used that $\Obf \in \mathcal{O}(p)$.
\end{Example}

The condition of Theorem \ref{first_sufficient} seems reasonable and not too restrictive: it suffices that one optimal solution of \eqref{InvOT} where $\Pbf^{*}$ is surjective exists in order to prove that an optimal coupling of \ref{GWprob} is supported by a deterministic map. We believe that some simple assumption can be made on $\mu,\nu$ in order to meet this condition. In particular this is satisfied when $\mu,\nu$ are 1D distributions as detailed below.

\paragraph{Application of Theorem \ref{first_sufficient} for 1D probability distributions} The sufficient condition given in Theorem \ref{first_sufficient} can be used to derived a closed-form expression for the problem \eqref{GWprob} between 1D probability distributions. Let us consider $(\mu,\nu) \in \P(\R) \times \P(\R)$ and $F_{\mu}$ and $F_{\nu}$ be the cumulative distribution functions of $\mu$ and $\nu$ and $F_{\mu}^{-1},F_{\nu}^{-1}$ its pseudo inverses (see Chapter \ref{cha:ot_general}). We suppose that $\mu$ is regular with respect to the Lebesgue measure in $\R$. 

In this case the linear application $\Pbf$ in \eqref{InvOT} reduces to a scalar $p \in \R$: 
\begin{equation}
\label{eq:second_maxot_R}
\sup_{\pi \in \couplingset(\mu,\nu)} \sup_{|p| = 1} \int (px).y \ \dr \pi(x,y)
\end{equation}
In this way an optimal any optimal solution $(\pi^{*},p^{*})$ of \eqref{eq:second_maxot_R} satisfies $p^{*}\in\{-1,1\}$ so that $p$ defines a surjective linear application. Then by applying Theorem \ref{first_sufficient} there exists $u: \R \rightarrow \R$, convex such that $u'\circ p^{*}$ pushes $\mu$ forward to $\nu$ and that $\gamma=(id\times u'\circ p^{*})\#\mu$ is optimal for \eqref{GWprob}. In other words there exists $f=u'$ non-decreasing such that $\gamma=(id\times f\circ p^{*})\#\mu$ is optimal for \eqref{GWprob}. However we known from linear transport theory that there is a unique non-decreasing map $T_{asc}:\R \rightarrow \R$ such that $T_{asc}\# \mu=\nu$ and it is given by $T_{asc}(x)=F_{\nu}^{-1}(F_{\mu}(x))$ (see theorem 2.5 in \cite{San15a}). This proves that if $p^{*}=1$ then $f\circ p^{*}=T_{asc}$ so that $\gamma=(id\times T_{asc})\#\mu$ is optimal for \eqref{GWprob}. If $p^{*}=-1$ then we have to consider the ``anti'' cumulative distribution function. Indeed in this case $f \circ p^{*}$ is non-increasing and pushes $\mu$ forward to $\nu$ which is equivalent to say that $f$ is non-decreasing and pushes $\tilde{\mu}$ forward to $\nu$ where ``$\dr \tilde{\mu}(x)=\dr \mu(-x)$''. This discussion leads to the following result:

\begin{theo}[Closed Form expression for \eqref{GWprob} between 1D distributions]
\label{theo:1D1D1D}
Let $(\mu,\nu) \in \P(\R) \times \P(\R)$ with $\mu$ regular with respect to the Lebesgue measure. Let $F^{\nearrow}_{\mu}(x)=\mu(]-\infty,x])$ be the cumulative distribution function and $F_{\mu}^{\searrow}(x)=\mu(]-x,+\infty])$ be the anti-cumulative distribution function. Let $T_{asc}:\R \rightarrow \R$ defined by $T_{asc}(x)=F_{\nu}^{-1}(F^{\nearrow}_{\mu}(x))$ and $T_{desc}:\R \rightarrow \R$ defined by $T_{desc}(x)=F_{\nu}^{-1}(F_{\mu}^{\searrow}(x))$. 

Then an optimal solution of \eqref{GWprob} is achieved by the map $\gamma=(id\times T_{asc})\#\mu$ or by the map $\gamma=(id\times T_{desc})\#\mu$.
\end{theo}

Theorem \ref{theo:1D1D1D} proves that it suffices to compute the CDF or the anti-CDF of the distribution to recover an optimal coupling. It can be put in light of the results of Section \ref{sec:sliced} where we proved that for discrete probability measures with uniform weights and same number of atoms an optimal coupling for the GW problem with squared Euclidean distances can be found in the diagonal or the anti-diagonal coupling when samples are sorted. As such the previous theorem is stronger since it can be applied for general 1D probability distribution when considering inner product similarities. Can we find other examples of optimal couplings for \eqref{GWprob}, for example when the dimension is larger than $1$? The next discussion gives another example which answers this question.

\paragraph{Construction of an optimal couple $(u,\Pbf)$} As seen in the previous discussion the couples $(u,\Pbf)$ where $u$ is convex so that $\nabla u \circ \Pbf$ pushes $\mu$ forward to $\nu$ may lead to an optimal map \textit{w.r.t.} the Gromov-Wasserstein distance with inner product similarities. In this part we wish to give an example of such couple $(u,\Pbf)$ which leads to an optimal map. Moreover this discussion will also highlight another difference between the linear OT problem and the Gromov-Wasserstein problem. It is known that when $\mu \in \P(\R^{p})$ and when $u :\R^{p} \rightarrow \R$ is convex and differentiable $\mu$ \textit{a.e} then the optimal transport plan for the linear OT problem $\inf_{\pi \in \couplingset(\mu,\nu)}\int \|\xbf-\ybf\|_{2}^{2} \dr \pi(\xbf,\ybf)$ between $\mu$ and $\nu\stackrel{def}{=}\nabla u \#\mu$ is given by $\gamma=(id\times \nabla u)\#\mu$. In other words if one perturbs the source measure $\mu$ with a transformation which is the gradient of a convex function then the cheapest way (in terms of $\wass_2$) of moving the source measure forward to the target is by the means of this transformation. To see this we can do the same reasoning as in \eqref{eq:eq_behind_mccan} by noticing that $\inf_{\pi \in \couplingset(\mu,\nu)}\int \|\xbf-\ybf\|_{2}^{2} \dr \pi(\xbf,\ybf)$ is equivalent to $\sup_{\pi \in \couplingset(\mu,\nu)}\int \scalar{\xbf}{\ybf}{p} \dr \pi(\xbf,\ybf)$ (see also \cite[Theorem 1.48]{San15a}). In our case however the situation is a little bit different as detailed in the next proposition:
\begin{prop}
\label{solving_maxot}
Let $\mu \in \P(\R^{p})$ with $\int \|\xbf\|_2^{4}\dr \mu(\xbf)<+\infty$. Let $u :\R^{p} \rightarrow \R$ be a convex function. We consider:
\begin{equation}
\label{maxProb}
\tag{$E_u$} 
\sup_{\Qbf \in F_{p,q}} \int u(\Qbf \xbf)\dr \mu(\xbf)
\end{equation}
Let $\Pbf \in F_{p,q}$ be a solution to \eqref{maxProb} with $\int \| \nabla u (\Pbf \xbf)\|_2^{4} \dr \mu(\xbf)<+\infty$ then $\gamma=\left(id\times \nabla u \circ \pbf \right)\#\mu$ is an optimal solution of \eqref{GWprob} between $\mu$ and $\nu\stackrel{def}{=}\nabla u \circ \Pbf \# \mu$.
\end{prop}
\begin{proof}
Let $(\pi^{*},\Pbf^{*})$ be maximizers of \eqref{InvOT}. We have:
\begin{equation*}
\begin{split}
\int \scalar{\Pbf^{*} \xbf}{\ybf}{q} \dr \pi^{*}(\xbf, \ybf) & \stackrel{1}{\leq} \int \left(u(\Pbf^{*} \xbf)+u^{*}(\ybf)\right) \dr \pi^{*}(\xbf, \ybf) = \int u(\Pbf^{*} \xbf) \dr \mu(\xbf) + \int u^{*}(\ybf) \dr \nu(\ybf) \\
&\stackrel{2}{=} \int u(\Pbf^{*} \xbf) \dr \mu(\xbf) + \int u^{*}(\ybf) \dr (\nabla u \circ \Pbf\#\mu)(\ybf) \\
&= \int u(\Pbf^{*} \xbf) \dr \mu(\xbf) + \int u^{*}(\nabla u(\Pbf \xbf)) \dr \mu(\xbf) \\
& \leq \sup_{\Qbf \in F_{p,q}} \int u(\Qbf \xbf)\dr \mu(\xbf) +\int u^{*}(\nabla u(\Pbf \xbf)) \dr \mu(\xbf) \\
&\stackrel{3}{=}\int u(\Pbf \xbf) \dr \mu(\xbf) + \int u^{*}(\nabla u(\Pbf \xbf)) \dr \mu(\xbf) \\
&\stackrel{4}{=} \int \scalar{\Pbf \xbf}{ \nabla u (\Pbf \xbf)}{q} \dr \mu(\xbf) \\
\end{split}
\end{equation*}

where in (1) we used that $u$ is convex, in (2) we used $(\nabla u \circ \pbf) \# \mu=\nu$, in(3) we used that $\Pbf$ maximizes $\sup_{\Qbf \in F_{p,q}} \int u(\Qbf \xbf)\dr \mu(\xbf)$ and in (4) we used that for any $x$ and convex function $u^{*}(\xbf)+u(\xbf)=\scalar{\xbf}{\nabla u (\xbf)}{}$.

We can deduce from (4) that: 
\begin{equation*}
\underset{\pi \in \Pi(\mu,\nu)}{\sup}\underset{\Pbf \in F_{p,q}}{\sup} \int \scalar{\Pbf \xbf}{\ybf}{q} \dr \pi(\xbf, \ybf) \leq \int \scalar{\Pbf \xbf}{\ybf}{q} \dr \gamma(\xbf, \ybf)
\end{equation*}

By suboptimality the converse inequality is also true so that $(\gamma,\Pbf)$ is an optimal solution of \eqref{InvOT} and consequently $\gamma$ is an optimal solution of \eqref{GWprob} using Theorem \ref{maintheorem}.

\end{proof}

This results states that if one perturbs a source measure $\mu \in \P(\R^{p})$ with a map $\nabla u \circ \Pbf$ with the condition that $\Pbf:\R^{p}\rightarrow \R^{q}$ achieves $\sup_{\Qbf \in F_{p,q}} \int u(\Qbf \xbf)\dr \mu(\xbf)$ then the cheapest way \textit{w.r.t.} \eqref{GWprob} of moving $\mu$ forward to $\nabla u \circ \Pbf \#\mu$ is by the means of the transformation $\nabla u \circ \Pbf$. To illustrate this proposition we can look at simple convex transformations $u(\xbf)=\frac{1}{2} \xbf^{T} \U \xbf$ where $\U$ is symmetric positive semi-definite (which means that the transformation $\nabla u$ is linear). In this case the following proposition exhibits a couple $(u,\Pbf)$ which leads to an optimal map for \eqref{GWprob}:
\begin{prop}[An optimal couple $(u,\Pbf)$]
Let $\mu \in \P(\R^{p})$. We note $\Sigmab_{\mu}=\int \xbf \xbf^{T} \dr \mu(\xbf)$. Let $u:\R^{p} \rightarrow \R$ convex defined by $u(\xbf)=\frac{1}{2} \xbf^{T} \U \xbf$ where $\U$ is symmetric positive semi-definite. Let $\mathbf{v},\mathbf{w}$ be the eigenvectors corresponding to the largest eigenvalue of respectively $\Sigmab_{\mu}$ and $\U$ and $\Pbf=\sqrt{p}\ \mathbf{v} \mathbf{w}^{T} \in F_{p,p}$ 

The coupling $\gamma=(id\times \nabla u \circ \Pbf)$ is optimal for \eqref{GWprob} between $\mu$ and $\nu\stackrel{def}{=}\nabla u \circ \Pbf\#\mu$
\end{prop}

\begin{proof}
In this case the problem $\sup_{\Qbf \in F_{p,p}} \int u(\Qbf \xbf)\dr \mu(\xbf)$ reduces to $\sup_{\|\Qbf\|_{\F}^{2}=p}= \int \xbf^{T} \Qbf^{T} \U \Qbf \xbf \dr \mu(\xbf)=\sup_{\|\Qbf\|_{\F}^{2}=p} \tr(\Qbf^{T} \U \Qbf \Sigmab_{\mu})$ where $\Sigmab_{\mu}=\int \xbf \xbf^{T} \dr \mu(\xbf)$. By vectorizing the matrix $\Qbf$ it is equivalent to:
\begin{equation}
\max_{\begin{smallmatrix}\qbf\in \R^{p^{2}} \\ \|\qbf\|^{2}_{2}=p \end{smallmatrix}} \qbf^{T} \Mbf_{u,\mu} \qbf
\end{equation}

where $\Mbf_{u}=\U \otimes_{K} \Sigmab_{\mu}$ with $\otimes_{K}$ the Kronecker product of matrices. $\Mbf_{u,\mu}$ is symmetric positive semi-definite. We can rewrite this problem with $\tilde{\qbf}=\sqrt{p} \ \qbf$:
\begin{equation}
\max_{\tilde{\qbf}^{T}\tilde{\qbf}=1} \tilde{\qbf}^{T} \Mbf_{u,\mu} \tilde{\qbf}
\end{equation}
Which is a maximization of a Rayleigh quotient problem. It is well known that a solution of this problem is found at any eigenvector associated to the largest eigenvalue of $\Mbf_{u,\mu}$ (see \textit{e.g.} \cite{qap_anstreicher_1998}). However the eigenvalues of $\Mbf_{u,\mu}$ are given by all the products of the eigenvalues of $\Sigmab_{\mu}$ and $\U$ (see \textit{e.g.} \cite{horn_johnson_1991}). Since they are all positive the largest eigenvalue of $\Mbf_{u,\mu}$ is found at the largest eigenvalue of $\Sigmab_{\mu}$ and $\U$ with corresponding eigenvector $\mathbf{v},\mathbf{w}$ and the optimal $\tilde{\qbf}$ is $\mathbf{v},\mathbf{w}^{T}$. This implies that the optimal $\Qbf$ is $\Qbf=\sqrt{p}\ \mathbf{v}\mathbf{w}^{T}$. We can apply Proposition \ref{solving_maxot} to conclude.
\end{proof}

\subsection{The squared Euclidean case}
\label{sec:sq_eclidean}
The $\gw$ problem is usually considered with distances, as it provides a metric with respect to strong isomorphisms. The goal of this section is to study the case $c_{\Xcal}(\xbf,\xbf')=\|\xbf-\xbf'\|_{2}^{2},c_{\Ycal}(\ybf,\ybf')=\|\ybf-\ybf'\|_{2}^{2}$ with $\mu \in \P(\R^{p}),\nu \in \P(\R^{q})$.
As for the inner product case we can prove that this problem is equivalent to another linear OT problem parametrized by a linear application. More precisely:
\begin{theo}
\label{maintheo2}
Let $\Xcal$ and $\Ycal$ be compact subset of respectively $\R^{p}$ and $\R^{q}$. Let $\mu \in \P(\Xcal),\nu \in \P(\Ycal)$. Assume without loss of generality that $\E_{X\sim \mu}[X]=0$ and $\E_{Y\sim\nu}[Y]=0$. Then problems:
\begin{equation}
\tag{sqGW}
\label{eq:sqGW}
\inf_{\pi \in \couplingset(\mu,\nu)} \int (\|\xbf-\xbf'\|_{2}^{2}-\|\ybf-\ybf'\|_{2}^{2})^{2} \dr \pi(\xbf,\ybf) \dr \pi(\xbf',\ybf')
\end{equation}
and
\begin{equation}
\label{dual_prod_sqeq}
\tag{dual-sqGW}
\sup_{\pi \in \Pi(\mu,\nu)} \sup_{\Pbf \in \R^{q\times p}} \int (\scalar{\Pbf \xbf}{\ybf}{q}+\|\xbf\|_{2}^{2} \|\ybf\|_{2}^{2})\dr \pi(\xbf,\ybf) -\frac{1}{8}\|\Pbf\|^{2}_{\F}
\end{equation}
are equivalent.
\end{theo}
To prove the previous theorem we will rely on the calculus of Lemma \ref{calculation_gw} and the observation that the cost $J_2(c_\Xcal,c_\Ycal,\pi)$ is invariant by translation of the support of the measures so that they can be centered without loss of generaly. This implies that the term $\int \big[\|\xbf\|_{2}^{2} \scalar{\E_{Y\sim\nu}[Y]}{\ybf}{q} + \|\ybf\|_{2}^{2}\scalar{\E_{X\sim \mu}[X]}{\xbf}{p} \dr\pi(\xbf,\ybf)$ in Lemma \ref{calculation_gw} vanishes. More presicely we have the following result: 
\begin{lemma}
\label{lemma:reduce}
Let $\Xcal$ and $\Ycal$ be compact subset of respectively $\R^{p}$ and $\R^{q}$. Let $\mu \in \P(\Xcal),\nu \in \P(\Ycal)$. We can assume without loss of generality that $\E_{X\sim \mu}[X]=0$ and $\E_{Y\sim\nu}[Y]=0$. In this case \eqref{eq:sqGW} is equivalent to:
\begin{equation}
\label{dual_problem}
\sup_{\pi \in \Pi(\mu,\nu)} \int \|\xbf\|_{2}^{2} \|\ybf\|_{2}^{2} \dr\pi(\xbf,\ybf)+2 \|\int \ybf\xbf^{T}\dr\pi(\xbf,\ybf)\|^{2}_{\F}
\end{equation}
\end{lemma}
A proof of this lemma can be found in Section \ref{sec:proof_reduce}. In the following we will note $F(\pi)=\int \|\xbf\|_{2}^{2} \|\ybf\|_{2}^{2} \dr\pi(\xbf,\ybf)+2 \|\int \ybf\xbf^{T}\dr\pi(\xbf,\ybf)\|^{2}_{\F}$. To prove Theorem \ref{maintheo2} the idea is to observe that this problem is a maximization of a convex function of $\pi$. We can use standard convex analysis tools such that the Fenchel-Moreau duality to derive the Fenchel dual of $F(\pi)$ as detailed below.

\paragraph{Duality in the space of measure} We suppose in the following that $\Xcal,\Ycal$ are general compact spaces. In this case the dual space of $C(\Xcal \times \Ycal)$ is $\mathcal{M}(\Xcal\times \Ycal)$ (see \textit{e.g.} Memo 1.3 in \cite{San15a}). We recall the main definitions of the Legendre-Fenchel tansform:
\begin{definition}[Legendre–Fenchel Transformation]
Let $\Xcal,\Ycal$ be compact Hausdorff spaces. For a function $F: \mathcal{M}(\Xcal \times \Ycal) \rightarrow \R \cup \{+\infty\}$ we define its convex conjugate $F^{*}: C(\Xcal \times \Ycal) \rightarrow \R \cup \{+\infty\}$ and its bi-conjugate $F^{**}: \mathcal{M}(\Xcal \times \Ycal) \rightarrow \R \cup \{+\infty\}$ as:
\begin{equation}
\begin{split}
&F^{*}(h)=\sup_{\pi \in \mathcal{M}(\Xcal\times \Ycal)} \int h(x,y)\dr\pi(x,y) -F(\pi) \\
&F^{**}(\pi)=\sup_{h \in \mathcal{C}(\Xcal\times \Ycal)} \int h(x,y)\dr\pi(x,y) -F^{*}(h)
\end{split}
\end{equation}

\end{definition}
One remarkable property of the convex conjugate $F^{*}$ is that it is always l.s.c and convex because $h \rightarrow \int h \dr \pi -F(\pi)$ is an affine function which is always convex and continuous. In this way $F^{*}$ is the pointwise supremum of continuous linear functions which is convex and l.s.c. We denote by $\dom(F)=\{\pi | F(\pi)<+\infty\}$ the domain of $F$. A function is called \emph{proper} if $\dom(F)\neq \emptyset$. A fundamental result in convex analysis is the Fenchel-Moreau theorem states that the bi-conjugate of a proper and l.s.c convex function equals to the original function (see \textit{e.g.} \cite{fenchel_moreau_theo}). In other words when $F$ is convex and well-behaved one can rely on the bi-conjugate to study the original function. In our context this implies that $F(\pi)=F^{**}(\pi)$ for all $\pi \in \mathcal{M}(\Xcal \times \Ycal)$.  To compute $F^{**}(\pi)$ we will need a notion of derivative in the space of measures as described in the next definition: 
\begin{definition}[Fréchet differentiable]
Let $\Xcal,\Ycal$ be compact Hausdorff spaces. A function $F: \mathcal{M}(\Xcal\times \Ycal)\rightarrow \R$ is is Fréchet differentiable at $\pi$ if there exists $\nabla F(\pi) \in \mathcal{C}(\Xcal\times \Ycal)$ such that for any $\epsilon \in \mathcal{M}(\Xcal\times \Ycal)$ as $t\rightarrow 0$:
\begin{equation}
F(\pi+t \epsilon)=F(\pi)+t \int \nabla F(\pi) d \epsilon+o(t)
\end{equation}
\end{definition}

\paragraph{Application to the Gromov-Wasserstein problem}

Let $\Xcal\subset \R^{p},\Ycal \subset \R^{q}$ be compact spaces and $\mu \in \P(\Xcal),\nu \in \P(\Ycal)$. Since $\Xcal,\Ycal$ are compact and the norms are continuous the GW distance is finite and the function $F$ defined in \eqref{dual_problem} is proper, convex and l.s.c. By application of the Fenchel-Moreau theorem we have:
\begin{equation}
\label{moreaumoreau}
\sup_{\pi \in \Pi(\mu,\nu)}F(\pi)=\sup_{\pi \in \Pi(\mu,\nu)}F^{**}(\pi)=\sup_{\pi \in \Pi(\mu,\nu)} \sup_{h \in \mathcal{C}(\Xcal\times \Ycal)} \int h(\xbf,\ybf)\dr\pi(\xbf,\ybf) -F^{*}(h)
\end{equation}
In the following we denote by $\mathbf{V_{\pi}}=\int \ybf\xbf^{T} \dr \pi(\xbf,\ybf)$. We can prove that we can solve the dual problem by parametrizing $h$ by a linear application $\Pbf \in \R^{q\times p}$ as $h(\xbf,\ybf)=\scalar{\Pbf \xbf}{\ybf}{q}+\|\xbf\|_{2}^{2}\|\ybf\|_{2}^{2}$. Using this form the problem becomes much simpler as we only need to optimize on a finite dimensional space instead of maximizing over all continuous function.
\begin{lemma} 
\label{hcanbewritten}
If $\pi^{*}$ is a solution of the primal problem $\sup_{\pi \in \couplingset(\mu,\nu)} F(\pi)$ then there exists $\Pbf \in \R^{q\times p}$ and $h^{*}\in C(\Xcal \times \Ycal)$ of the form $h^{*}(\xbf,\ybf)=\scalar{\Pbf \xbf}{\ybf}{q}+\|\xbf\|_{2}^{2}\|\ybf\|_{2}^{2}$ such that $(\pi^{*},h^{*})$ is a solution of the dual problem \eqref{dual_problem}. 
Moreover when $h^{*}$ is in such form we have $F^{*}(h^{*})=\frac{1}{8}\|\Pbf\|^{2}_{\F}$.
\end{lemma}

A proof can be found in Section \ref{sec:lemma_calculus}. The previous Lemma \ref{hcanbewritten} can be used to prove Theorem \ref{maintheo2}. Indeed with previous notations $h$ can be written in the form $h(\xbf,\ybf)=\scalar{\Pbf \xbf}{\ybf}{q}+\|\xbf\|_{2}^{2} \|\ybf\|_{2}^{2}$. Plugging the calculus of the conjugate $F^{*}(h)=\frac{1}{8}\|\Pbf\|^{2}_{\F}$ into \eqref{moreaumoreau} gives the desired result.

\paragraph{Regularity of \eqref{eq:sqGW} optimal plans} The Fenchel dual problem is more difficult to analyse than the dual problem of the inner product case. Indeed the cost $c: (\xbf,\ybf) \rightarrow \scalar{\Pbf \xbf}{\ybf}{q}+\|\xbf\|_{2}^{2} \|\ybf\|_{2}^{2}$ hardly relates to a ``standard'' linear OT problem. Especially it does not statifies the Twist condition (see Chapter \ref{cha:ot_general}) in general and an approach with convex tools is trickier due to the term $\|\xbf\|_{2}^{2} \|\ybf\|_{2}^{2}$. We will show in the following that if $\nabla u \circ \Pbf$ is ``not too far'' from an isometry then it defines an optimal coupling. We note $\Gamma_{0}(\R,\R)$ the set of derivable convex functions from $\R$ to $\R$ and we define the following set:
\begin{definition} 
We define $\mathcal{H}(\mu,\nu)$ the set of push-forward between $\mu$ and $\nu$ defined by: 
\begin{equation}
\label{quasi_isom_set}
\mathcal{H}(\mu,\nu)=\{T\#\mu=\nu \ | \exists f\in \Gamma_{0}(\R,\R) \ , \|T(\xbf)\|^{2}_{2}=f'(\|\xbf\|^{2}_{2}) \  \mu \ \text{a.e} \}
\end{equation}
\end{definition}
When $p=q$ this set encompasses all the linear push forward of the form $T=c \ \Obf$ where $c>0,\Obf \in \mathcal{O}(p)$. It turns out that this condition is also sufficient when considering only linear push-forward (see Lemma \ref{lemma:push_lin}).

We have the following sufficient condition if we look at Monge map of the form $\nabla u \circ \Pbf$ with $u$ convex and and $\Pbf \in \R^{q\times p}$: 

\begin{prop}
\label{sufficientcond2}
Let $\Xcal$ and $\Ycal$ be compact subset of respectively $\R^{p}$ and $\R^{q}$ with $p\geq q$. Let $\mu \in \P(\Xcal),\nu \in \P(\Ycal)$. Assume that $\mu$ is regular \textit{w.r.t.} the Lebesgue measure on $\R^{p}$ and that $\E_{X\sim \mu}[X]=0$ and $\E_{Y\sim\nu}[Y]=0$ without loss of generality.

Let $(\pi^{*},\Pbf^{*})$ be optimal solution of \eqref{dual_prod_sqeq}. If $\Pbf^{*}$ is surjective there exists $u: \R^{q} \rightarrow \R$ convex such that $\nabla u \circ \Pbf^{*}$ pushes $\mu$ forward to $\nu$.
Moreover if $\nabla u \circ \Pbf^{*} \in \mathcal{H}(\mu,\nu)$ then the coupling $\gamma=(id\times \nabla u \circ \Pbf^{*})\# \mu$ is optimal for \eqref{eq:sqGW}.
\end{prop}
\begin{proof}
Let $(\Pbf^{*},\pi^{*})$ be optimal solution of \eqref{dual_prod_sqeq} with $\Pbf^{*}$ surjective. We have seen previously that $\Pbf^{*} \# \mu$ is regular with respect to the Lebesgue measure on $\R^{q}$ using Proposition \ref{prop:regulariffsurj} such that there exists $u: \R^{q} \rightarrow \R$ convex such that $\nabla u \circ \Pbf^{*}$ pushes $\mu$ forward to $\nu$. Moreover:

\begin{equation*}
\begin{split}
\int \scalar{\Pbf^{*}\xbf}{y}{q} + \|\xbf\|^{2}_{2}\|\ybf\|^{2}_{2}\dr \pi^{*}(\xbf, \ybf) & \stackrel{1}{\leq} \int \left(u(\Pbf^{*}\xbf)+u^{*}(\ybf)\right) \dr \pi^{*}(\xbf, \ybf) + \int \|\xbf\|^{2}_{2}\|\ybf\|^{2}_{2}\dr \pi^{*}(\xbf, \ybf)\\
&= \int u(\Pbf^{*}\xbf) \dr \mu(\xbf) + \int u^{*}(\ybf) \dr \nu(\ybf) + \int \|\xbf\|^{2}_{2}\|\ybf\|^{2}_{2}\dr \pi^{*}(\xbf, \ybf)\\
&\stackrel{2}{=} \int u(\Pbf^{*} \xbf) \dr \mu(\xbf) + \int u^{*}(\ybf) \dr (\nabla u \circ \Pbf^{*}\#\mu)(\ybf) + \int \|\xbf\|^{2}_{2}\|\ybf\|^{2}_{2}\dr \pi^{*}(\xbf, \ybf)\\
&= \int u(\Pbf^{*} \xbf) \dr \mu(\xbf) + \int u^{*}(\nabla u(\Pbf^{*}\xbf)) \dr \mu(\xbf) + \int \|\xbf\|^{2}_{2}\|\ybf\|^{2}_{2}\dr \pi^{*}(\xbf, \ybf)\\
&\stackrel{3}{=} \int \scalar{\Pbf^{*} \xbf}{ \nabla u(\Pbf^{*} \xbf)}{q} \dr \mu(x)+ \int \|\xbf\|^{2}_{2}\|\ybf\|^{2}_{2}\dr \pi^{*}(\xbf, \ybf) \\
\end{split}
\end{equation*}

In (1) we used convexity of $u$, in (2) we used that $\nabla u \circ \Pbf^{*}$ is a push-forward and in (3) we used $u^{*}(\nabla u(\xbf))+u(\xbf)=\scalar{\xbf}{\nabla u(\xbf)}{}$. Now suppose that $\nabla u \circ \Pbf^{*} \in \mathcal{H}(\mu,\nu)$. Then there exists $f\in \Gamma_{0}(\R,\R) \ , \|T(\xbf)\|^{2}_{2}=f'(\|\xbf\|^{2}_{2}) \  \mu \ \text{a.e}$. 

Moreover we have $\int \|\xbf\|^{2}_{2}\|\ybf\|^{2}_{2}\dr \pi^{*}(\xbf, \ybf) \leq \int f(\|\xbf\|^{2}_{2})+ f^{*}(\|\ybf\|^{2}_{2})\dr \pi^{*}(\xbf, \ybf)$ by Young's inequality (see Memo \ref{memo:young}). In this way:
\begin{equation*}
\begin{split}
\int \scalar{\Pbf^{*}\xbf}{\ybf}{q} + \|\xbf\|^{2}_{2}\|\ybf\|^{2}_{2}\dr \pi^{*}(\xbf, \ybf) &\leq \int \scalar{\Pbf^{*} \xbf}{ \nabla u(\Pbf^{*} \xbf)}{q} \dr \mu(\xbf)+ \int f^{*}(\|\xbf\|^{2}_{2})\dr \mu(\xbf)+ \int f(\|\ybf\|^{2}_{2})\dr \nu(\ybf) \\
&\stackrel{5}{=} \int \scalar{\Pbf^{*} \xbf}{ \nabla u(\Pbf^{*} \xbf)}{q} \dr \mu(\xbf)+ \int f^{*}(\|\nabla u (\Pbf^{*} \xbf)\|^{2}_{2})+ f(\|\xbf\|^{2}_{2})\dr \mu(\xbf) \\
&\stackrel{6}{=} \int \scalar{\Pbf^{*} \xbf}{ \nabla u (\Pbf^{*} \xbf)}{q} \dr \mu(\xbf)+ \int \|\nabla u (\Pbf^{*} \xbf)\|^{2}_{2}\|\ybf\|^{2}_{2}\dr \mu(x) \\
&\stackrel{7}{=} \int \scalar{\Pbf^{*} \xbf}{\ybf}{q} + \|\xbf\|^{2}_{2}\|\ybf\|^{2}_{2}\dr \gamma(\xbf, \ybf)
\end{split}
\end{equation*}
In (5) we used that $\nabla u \circ \Pbf^{*}$ is a push-forward. In (6) we used that $f$ satisfies $\|\nabla u (\Pbf^{*} \xbf)\|^{2}_{2}  \in \partial f(\|\xbf\|_{2}^{2})=f'(\|\xbf\|_{2}^{2})$ by definition of $\mathcal{H}(\mu,\nu)$ which implies $f^{*}(\|\nabla u (\Pbf^{*} \xbf)\|^{2}_{2})+ f(\|\xbf\|^{2}_{2})=\|\nabla u (\Pbf^{*} \xbf)\|^{2}_{2}\|\xbf\|^{2}_{2}$ (see Memo \ref{memo:memo_convex}). In (7) $\gamma$ is defined as $\gamma=(id\times \nabla u \circ \Pbf^{*})\# \mu$. Overall $(\Pbf^{*},\gamma)$ is optimal for \eqref{dual_prod_sqeq}. Using Theorem \ref{maintheo2} this proves that $\gamma$ is optimal for \eqref{eq:sqGW} which concludes the proof.

\end{proof}  

\begin{figure*}[!b]
\begin{memo}[Young's Inequality]
\label{memo:young}
Let $a,b \in \R_{+} \times \R_{+}$ and $p,q$ real numbers greater than $1$ with $\frac{1}{p}+\frac{1}{q}=1$. Then:
\begin{equation}
ab\leq \frac{a}{p}+\frac{b}{q}
\end{equation}
More generally if $f$ is a convex function and $f^{*}$ is its Legendre transform then:
\begin{equation}
ab\leq f(a)+f^{*}(b)
\end{equation}
which is a consequence of the celebrated Fenchel-Young inequality (see Memo \ref{memo:memo_convex})
\end{memo}
\end{figure*}  

The condition $\nabla u \circ \Pbf \in \mathcal{H}(\mu,\nu)$ is quite strong compared to the conditions of Theorem \ref{first_sufficient}. For example in the case where $\nabla u$ is linear only orthogonal transformations are admissible (modulo a scaling). We believe that this result can be improved and we leave this study for further works. In general, and without any furter assumption on $\mu$ and $\nu$, we postulate that there might be degenerate cases in which $\mu$ is regular but there is no deterministic optimal couplings.

\subsection{Optimization and numerical experiments}
\label{sec:numerical}

In this section we provide numerical solutions for Gromov-Wasserstein problems in Euclidean spaces. We will rely on the equivalent formulation defined in Theorem \ref{maintheorem} and Theorem \ref{maintheo2}. We consider two discrete probability measures $\mu= \frac{1}{n} \sum_{i=1}^{n} a_{i} \delta_{\xbf_{i}} \in \P(\R^{p})$ and $\nu= \sum_{i=1}^{m} b_{j} \delta_{\ybf_{j}} \in \P(\R^{q})$ with $\a \in \simplex_n,\b \in \simplex_m$. 

\paragraph{Numerical solution for \eqref{GWprob}} We note $\X=(\xbf_{i})_{i=1}^{n} \in \R^{n\times p}$,$\Y=(\ybf_{j})_{j=1}^{m} \in \R^{m\times q}$. As seen in Theorem \ref{maintheorem} computing \eqref{GWprob} can be achieved by solving $\sup_{\GG \in \couplingset(\a,\b)} \sup_{\|\Pbf\|_{\F}=\sqrt{p}} \froeb{\X \Pbf^{T} \Y^{T}}{\GG}$ or equivalently:
\begin{equation}
\label{eq:discreteinner}
\min_{\GG \in \couplingset(\a,\b)} \min_{\|\Pbf\|_{\F}=\sqrt{p}} \froeb{-\X \Pbf^{T} \Y^{T}}{\GG}
\end{equation}
The resulting problem is convex \textit{w.r.t.} $\Pbf$ and $\pi$ (but not jointly convex). We propose to solve equation \eqref{eq:discreteinner} using Block Coordinate Descent (BCD) which alternates between minimizing \textit{w.r.t.} $\Pbf$ and $\GG$. Interestingly enough the minimization of $\GG$ with $\Pbf$ fixed is a linear transportation problem with ground cost $-\X \Pbf^{T} \Y^{T}$ which can be computed using standard solvers (see Chapter \ref{cha:ot_general}). 

The minimization \textit{w.r.t.} $\Pbf$ with $\GG$ fixed reads $\sup_{\|\Pbf\|_{\F}=\sqrt{p}} \froeb{\X \Pbf^{T} \Y^{T}}{\GG}$ and has a closed-form solution based on Lemma \ref{froebnormduality}. More precisely it reads: 
\begin{equation}
\Pbf\leftarrow\frac{\sqrt{d}}{\|\Y^{T} \GG^{T} \X \|_{\F}} \Y^{T} \GG^{T} \X
\end{equation}

This procedure is presented in Algorithm \ref{alg:inner_gw}. The complexity is driven by the linear OT problem which is $O(n^{3}\log(n))$ when $n=m$. The complexity for computing $\Pbf$ at each iteration is $O(mn(q+p)+pq)$ with $O(mn(q+p))$ for the computing $\Y^{T} \GG^{T} \X$ and $O(pq)$ for $\|\Y^{T} \GG^{T} \X \|_{\F}$. The overall complexity when $n=m$ is then $O(n^{3}\log(n)+n^{2}(q+p)+pq)$.

\begin{algorithm}[t]
\caption{Gromov-Wasserstein with inner products \label{alg:inner_gw}}
\begin{algorithmic}[1]
    \State Require $\mu= \frac{1}{n} \sum_{i=1}^{n} a_{i} \delta_{\xbf_{i}} \in \P(\R^{p})$ and $\nu= \sum_{i=1}^{m} b_{j} \delta_{\ybf_{j}} \in \P(\R^{q})$
    \State Set $\X=(\xbf_{i})_{i=1}^{n} \in \R^{n\times p}$,$\Y=(\ybf_{j})_{j=1}^{m} \in \R^{m\times q}$
    \State Initialize $\GG=\a\b^{T}$. 
    \While {not converged}
    \State $\Pbf\leftarrow \frac{\sqrt{d}}{\|\Y^{T} \GG^{T} \X \|_{\F}} \Y^{T} \GG^{T} \X  $ // (maximize \eqref{InvOT} \textit{w.r.t.} $\Pbf$)       
    \State $\GG \leftarrow \argmin_{\GG\in \couplingset(\a,\b)} \froeb{-\X \Pbf^{T} \Y^{T}}{\GG}$ // (maximize \eqref{InvOT} \textit{w.r.t.} $\pi$ $\sim$ linear OT)
    \EndWhile
    \State return $(\pi,\Pbf)$
\end{algorithmic}
\end{algorithm}

\paragraph{Numerical solution for \eqref{eq:sqGW}} We note $\xbf=\diag(\X \X^{T}) \in \R^{n}$, $\ybf=\diag(\Y \Y^{T}) \in \R^{m}$. As seen in Theorem \ref{maintheo2} computing \eqref{eq:sqGW} can be achieved by first substracting the mean of the measures and then by solving:
\begin{equation}
\label{eq:discretesqec}
\min_{\GG \in \couplingset(\a,\b)} \min_{\Pbf \in \R^{q\times p}} \froeb{-(\X \Pbf^{T} \Y^{T} + \xbf \ybf^{T})}{\GG}+\frac{1}{8} \|\Pbf\|_{\F}^{2}
\end{equation}
This problem also is convex \textit{w.r.t.} $\GG$ and $\Pbf$ (but not jointly convex) and we can rely on a BCD procedure to find a local minimum. The minimization \textit{w.r.t.} to $\GG$ with $\Pbf$ fixed is a linear OT problem with ground cost $-(\X \Pbf^{T} \Y^{T} + \xbf \ybf^{T})$ and the minimization \textit{w.r.t.} $\Pbf$ with $\GG$ fixed reads:
\begin{equation}
 \min_{\Pbf \in \R^{q\times p}} \froeb{-\X \Pbf^{T} \Y^{T}}{\GG}+\frac{1}{8} \|\Pbf\|_{\F}^{2} \stackrel{def}{=} \min_{\Pbf \in \R^{q\times p}} G(\Pbf)
 \end{equation}
which is a unconstrained QP which can be solved in closed-form. Indeed the gradient reads $\nabla G(\Pbf)=-\Y^{T} \GG^{T} \X+\frac{1}{4} \Pbf$ and by first order condition a solution can be found at $\nabla G(\Pbf)=0$. In this way the iteration for $\Pbf$ are:
\begin{equation}
\Pbf\leftarrow 4 \Y^{T} \GG^{T} \X
\end{equation}
This procedure is presented in Algorithm \ref{alg:sq_gw}. Computing all the terms $\|\xbf^{(i)}\|_2^{2}\|\ybf^{(j)}\|_2^{2}$ where $\xbf^{(i)}$ is the $i$-th row of $\X$ has a naive complexity of $O(np+mq+mn)$ and one needs also to compute $4 \Y^{T} \GG^{T} \X$ at a $O(mn(q+p))$ price. The overall complexity is then $O(n^{3}\log(n)+(p+q)n^{2}+n^{2}+n(p+q))=O(n^{3}\log(n)+(p+q)n^{2})$ when $n=m$.

\begin{algorithm}[t!]
\caption{Gromov-Wasserstein with squared Euclidean distances \label{alg:sq_gw}}
\begin{algorithmic}[1]
    \State Require $\mu= \frac{1}{n} \sum_{i=1}^{n} a_{i} \delta_{\xbf_{i}} \in \P(\R^{p})$ and $\nu= \sum_{i=1}^{m} b_{j} \delta_{\ybf_{j}} \in \P(\R^{q})$
    \State Subtract the mean: $\xbf_i\leftarrow \xbf_i-\frac{1}{n}\sum_k \xbf_k$, $\ybf_j\leftarrow \ybf_j-\frac{1}{m}\sum_k \ybf_k$.
    \State Set $\X=(\xbf_{i})_{i=1}^{n} \in \R^{n\times p}$,$\Y=(\ybf_{j})_{j=1}^{m} \in \R^{m\times q}$, $\xbf=\diag(\X \X^{T}) \in \R^{n}$, $\ybf=\diag(\Y \Y^{T}) \in \R^{m}$
    \State Initialize $\GG=\a\b^{T}$. 
    \While {not converged}
    \State $\Pbf\leftarrow 4 \Y^{T} \GG^{T} \X  $ // (maximize \eqref{dual_prod_sqeq} \textit{w.r.t.} $\Pbf$)       
    \State $\GG \leftarrow \argmin_{\GG\in \couplingset(\a,\b)} \froeb{-(\X \Pbf^{T} \Y^{T} + \xbf \ybf^{T})}{\GG}$ // (maximize \eqref{dual_prod_sqeq} \textit{w.r.t.} $\pi$ $\sim$ linear OT)
    \EndWhile
    \State return $(\pi,\Pbf)$
\end{algorithmic}
\end{algorithm}

\begin{figure}[t!]
\begin{center}
\includegraphics[width=1.\linewidth]{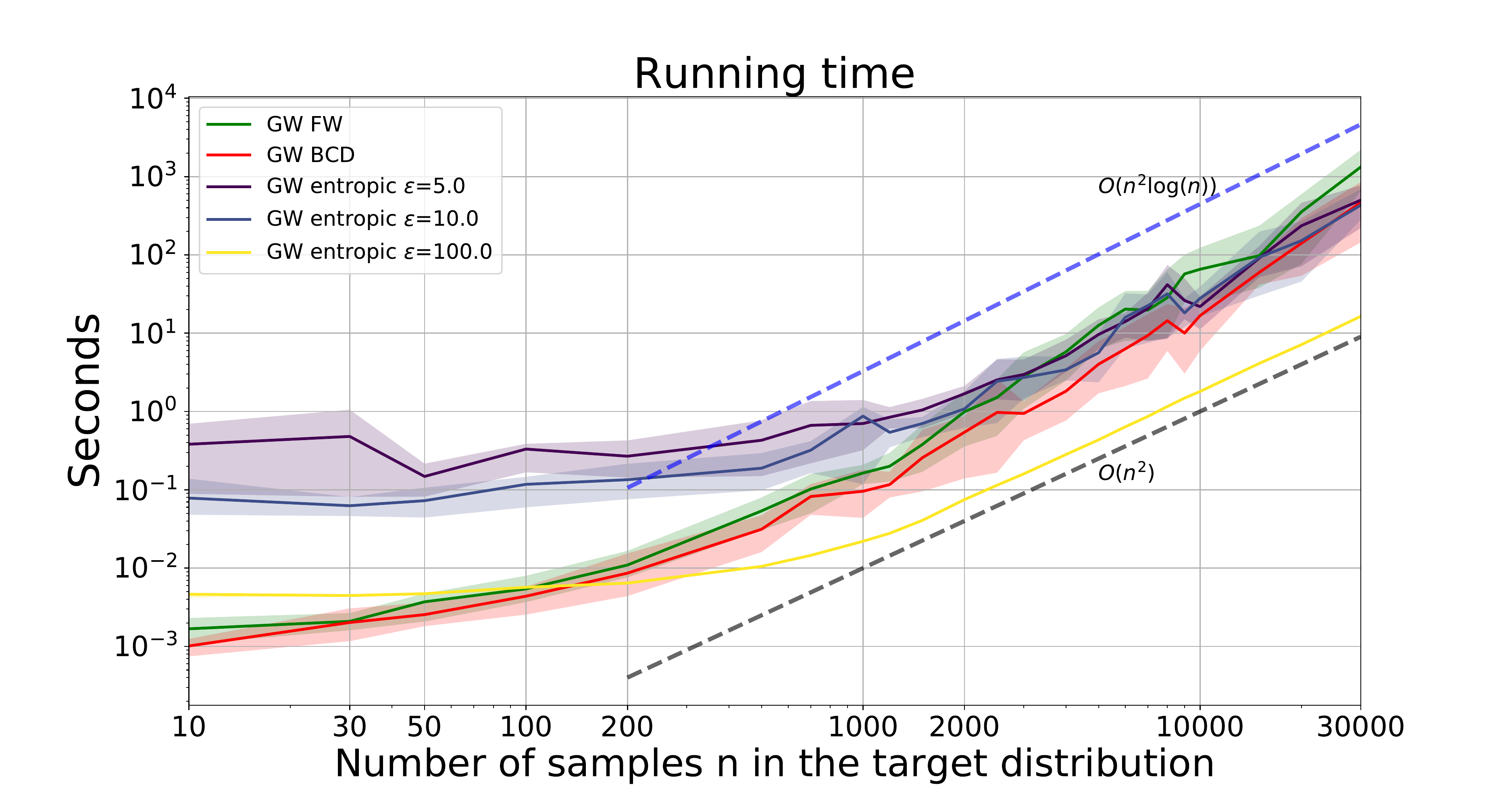}
\end{center}    
\caption{\label{fig:runtime_sq} Runtimes comparison of $\gw$ using the BCD approach Algorithm \ref{alg:sq_gw}. (GW BCD), the Frank-Wolfe approach (GW FW) and entropic-$\gw$ between two random distributions whose source points vary from $10$ to $30 000$ in log-log scale. The time does not include the calculation of the pair-to-pair distances but only the time of the different loops. The variance is computed among the different $10$ realizations. }
\end{figure}

\paragraph{Runtimes comparison} We perform a comparison between runtimes of $\gw$ using different algorithmic solutions. We consider the Gromov-Wasserstein problem between $10$ realizations of one source 2D random probability measures of $n \in \{10,...,30 000\}$ points and one target 3D random measures with $m=100$ points. We compute the Gromov-Wasserstein distance using squared Euclidean distances as $c_\Xcal,c_\Ycal$. We compare the timings of the Frank-Wolfe algorithm (see Chapter \ref{cha:fgw}), the entropic regularized approach with $\epsilon \in \{5,10,100\}$ (see Chapter \ref{cha:ot_general}), and the BCD approach Algorithm \ref{alg:sq_gw} using the same initialization of $\GG=\a\b^{T}$ for all methods. The result is depicted in Figure \ref{fig:runtime_sq}. Please note that the timings are calculated without taking into account the time needed for computing the matrices $\C_1,\C_2$ or $\X,\Y$ but are only based on the loops of the different algorithms. The entropic-$\gw$ could not be computed with $\epsilon\leq 5$ due to overflows in the Sinkhorn algorithm. Another difficulty of computing entropic-$\gw$ is the fact that the regularization parameter $\epsilon$ is inversely proportional to the gradient step \cite{peyre2016gromov}. Thus only high value of $\epsilon$ are computable in reasonable time and for $\epsilon < 5$ we observe that the convergence is very slow even on this simple example. This often implies that only very blurred solutions can be computed which explains the low variance of entropic-$\gw$ with $\epsilon=100$. We can see in Figure \ref{fig:runtime_sq} that the BCD approach is a little bit faster than the FW approach, suggesting that the BCD may converge faster to a local minimum than FW, which is, of course, data dependant. Overall both methods are equally rapid on this example.

\paragraph{Costs comparison} In this experiment we compared the ability of the BCD approach to find a better solution than the FW approach. We consider the Gromov-Wasserstein problem with both inner product similarities and squared Euclidean distances. We compute $200$ distances using both algorithms where each distance is calculated by: (1) We draw $n \in \{500,2000\}$ samples from two 10 dimensional normal distributions (2) We associate to these points random weights $(\a,\b) \in \Sigma_n \times \Sigma_n$ (3) We initialize the algorithms with the same random coupling matrix $\GG$. The initialization is computed by sampling a random matrix with positive entries and by scaling it with the Sinkhorn algorithm in order to have the prescribed marginals $\a,\b$. Results are depicted in Figure \ref{fig:costs}. We plot the $\gw$ cost obtained by BCD approach \textit{vs} the cost obtained by the FW approach after convergence of each algorithm. As seen in Figure \ref{fig:costs} there is no strong differences between FW approach and the BCD approach when squared Euclidean distances are considered (left part of the figure): the $FW$ algorithm finds a better solution in $51\%$ of the cases when $n=2000$ and $56\%$ for $n=500$. Surprisingly both algorithms seem to lead to the same solution when inner product similarities are used (right part of the figure). Indeed in $99\%$ of the cases the costs are identical up to $1e-7$.

\begin{figure}[t!]
\begin{center}
\includegraphics[width=0.9\linewidth]{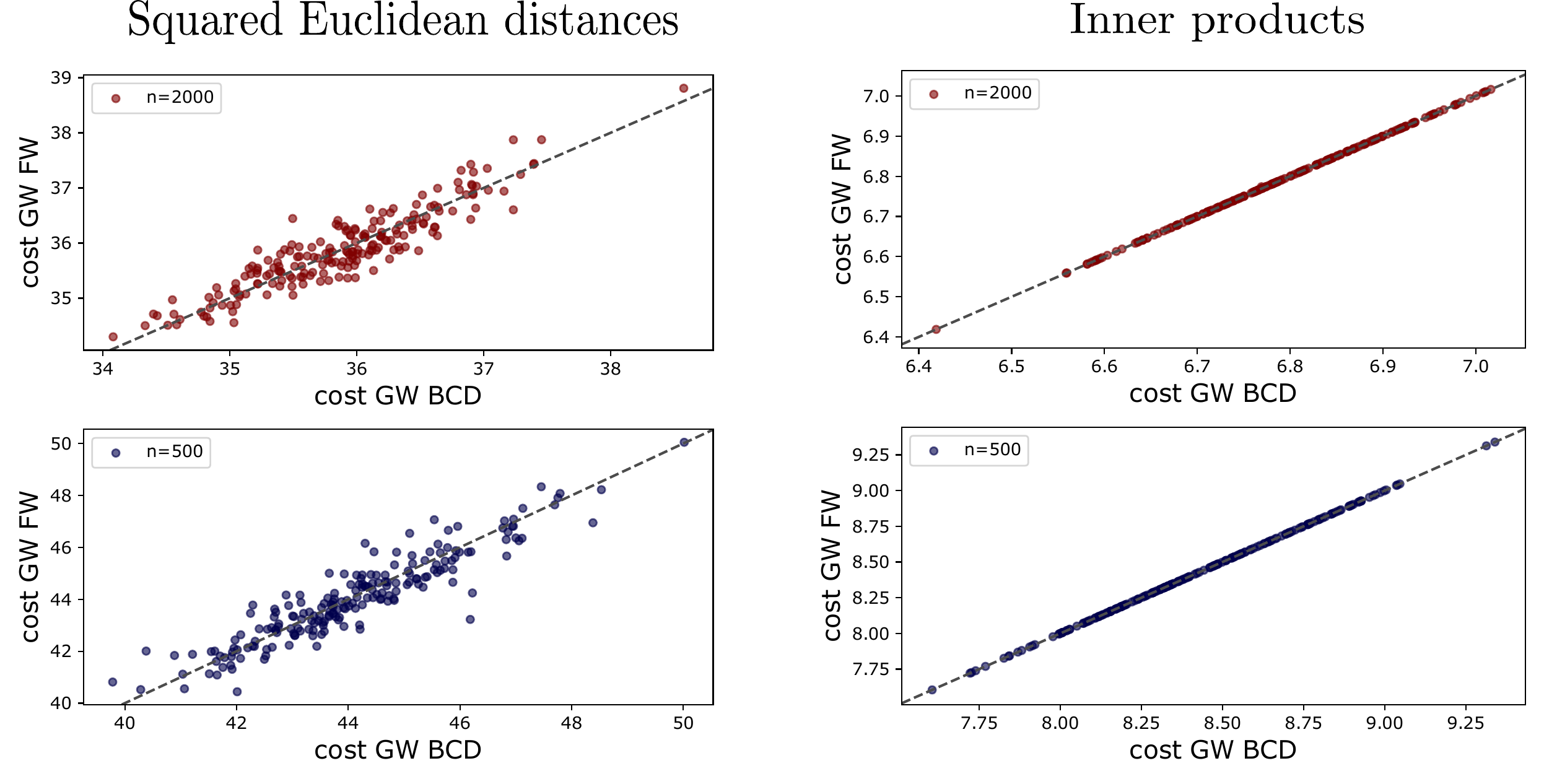}

\end{center}    
\caption{\label{fig:costs} Comparison of the FW solver and the BCD solver for computing the Gromov-Wasserstein distance. {\bf (left)} $c_\Xcal,c_\Ycal$ are squared Euclidean distances. {\bf (right)} $c_\Xcal,c_\Ycal$ are inner products. The $x$-axis is the $\gw$ cost obtained by the BCD approach and the $y$-axis the $\gw$ cost obtained with the FW algorithm for respectively {\bf (top)} $n=2000$ samples {\bf (bottom)} $n=500$ samples in each distribution. The dashed line represents the diagonal $y=x$.}
\end{figure}

\paragraph{Conclusion on the experiments} As seen in these experiments, there are no significant differences between the FW approach and the BCD procedure. It is not surprising for the runtimes comparison since both methods have a theoretical cubic complexity. The cost comparison experiment may also suggest that both BCD procedure and FW are equivalent in the case of inner product similarities, that is the iterations of the BCD are the same than the iterations of the FW. Further studies could be conducted to examine the high-dimensional setting where we postulate that the BCD approach may lead to better results that the FW procedure, \ie\ may converge faster a produce a better solution.

\subsection{The Gromov-Monge problem in Euclidean spaces}
\label{sec:gm}
As seen in Section \ref{sec:sgw} the Gromov-Wasserstein problem is also very related the Gromov-Monge problem \cite{memoli_gromov_monge_2018} which is defined in Euclidean spaces for $\mu \in \P(\R^{p}),\nu \in \P(\R^{q})$ as:
\begin{definition}[Gromov-Monge]
The Gromov-Monge problem aims at finding: 
\begin{equation}
\label{gm_eq}
\tag{GM}
\gm^{2}_{2}(\mu,\nu)= \underset{T\#\mu=\nu}{\inf} J(T) = \underset{T\#\mu=\nu}{\inf} \int \left(\|\xbf-\xbf'\|_{2}^{2}-\|T(\xbf)-T(\xbf')\|_{2}^{2}\right)^{2} \dr \mu(\xbf) \dr \mu(\xbf')
\end{equation}
With the understanding that $\gm_{2}(\mu,\nu)=+\infty$ when $\{T:\R^{p}\rightarrow \R^{q} | T\#\mu=\nu\}=\emptyset$
\end{definition}

 The problem \eqref{gm_eq} is the exact counterpart for the Gromov-Wasserstein distance of the Monge problem for the Wasserstein distance. If we note $\gw_2(\mu,\nu)$ the Gromov-Wasserstein distance with squared Euclidean costs then we have $\gw_2(\mu,\nu)\leq \gm_{2}(\mu,\nu)$ but in general both problems are not equivalent (see \cite{memoli_gromov_monge_2018}). Note that when problem \ref{prob_ouvert} holds, \ie\ when the Gromov-Wasserstein problem admits an optimal transport plan supported on a deterministic map then both \eqref{gm_eq} and \eqref{eq:sqGW} are equivalent. Moreover we have seen previously in Theorem \ref{sovable_gw} that when $\mu$ and $\nu$ are discrete probability measures with the same number of atoms and uniform weights then \eqref{gm_eq} is also equivalent to the Gromov-Wasserstein problem \eqref{eq:sqGW} so that that $\gw_2(\mu,\nu)= \gm_{2}(\mu,\nu)$. 

 We propose to study further the Gromov-Monge problem in this section. Especially we consider the special case of Gaussian measures $\mu=\mathcal{N}(\mathbf{m_{\nu}},\Sigmab_{\nu}),\nu=\mathcal{N}(\mathbf{m_{\mu}},\Sigmab_{\mu})$. It is motivated by the linear OT theory where, when $p=q$, there is a close form solution for $\wass_2$. In this case the optimal Monge map is \emph{linear} and given by \cite{takatsu2011} $T: \xbf\rightarrow \mathbf{m_{\nu}}+\Abf(\xbf-\mathbf{m_\nu})$ where:
\begin{equation}
\Abf=\Sigmab_{\mu}^{-1/2}(\Sigmab_{\mu}^{1/2}\Sigmab_{\nu}\Sigmab_{\mu}^{1/2})^{\frac{1}{2}}\Sigmab_{\mu}^{-1/2}
\end{equation}
Can we derive the same type of result for the Gromov-Monge geometry? We will prove that when restricted to \emph{linear} push-forward and in the special case $p=q$ the problem admits also a close form expression. In this way we consider the following \emph{linear Gromov-Monge} problem:
\begin{equation}
\label{lgm_eq}
\tag{LGM}
L\gm^{2}_{2}(\mu,\nu)= \underset{\begin{smallmatrix} T\#\mu=\nu \\ \text{T is linear} \end{smallmatrix}}{\inf} J(T) = \underset{\begin{smallmatrix} T\#\mu=\nu \\ \text{T is linear} \end{smallmatrix}}{\inf} \int \left(\|\xbf-\xbf'\|_{2}^{2}-\|T(\xbf)-T(\xbf')\|_{2}^{2}\right)^{2} \dr \mu(\xbf) \dr \mu(\xbf')
\end{equation}
We recall that $\Stief$ is the Stiefel manifod defined by $\Stief=\{\Bbf \in \R^{q\times p} | \Bbf^{T}\Bbf=\mathbf{I_{p}} \}$. The main result of this section is the following theorem:
\begin{theo}
\label{maintheo3}
Let $\mu=\mathcal{N}(0,\Sigmab_{\nu}) \in \P(\R^{p}),\nu=\mathcal{N}(0,\Sigmab_{\mu})\in \P(\R^{q})$ centered without loss of generality. Let $\Sigmab_{\mu}=\V_{\mu} \Dbf_{\mu} \V_{\mu}^\top,\Sigmab_{\nu}=\V_{\nu} \Dbf_{\nu} \V_{\nu}^\top$ be the diagonalizations of the covariance matrices such that eigenvalues of $\Dbf_{\mu}$ and $\Dbf_{\nu}$ are ordered nondecreasing.

When $p\neq q$ we have:
\begin{equation}
\label{eq:eqiv2}
\begin{split}
L\gm^{2}_{2}(\mu,\nu)= 4(\tr(\Sigmab_\mu)-\tr(\Sigmab_\nu))^{2}+8(\tr(\Sigmab_\mu \Sigmab_\mu)+\tr(\Sigmab_\nu \Sigmab_\nu)) +16\min_{\Bbf \in \Stief}-\tr(\Dbf_\mu\Bbf^\top  \Dbf_\nu  \Bbf)
\end{split}
\end{equation}

When $p=q$, an optimal linear Monge map is given by $T(\xbf)=\Abf \xbf$ where:
\begin{equation}
\Abf=\V_\nu  \Dbf^{1/2}_\nu \Dbf^{-1/2}_\mu \V_\mu^\top=\Sigmab_{\nu}^{1/2}\V_{\nu}\V_{\mu}^{\top}\Sigmab_{\mu}^{-1/2}
\end{equation}
so that:
\begin{equation}
\label{optimal_cost}
L\gm^{2}_{2}(\mu,\nu)= 4(\tr(\Sigmab_\mu)-\tr(\Sigmab_\nu))^{2}+8(\tr(\Sigmab_\mu \Sigmab_\mu)+\tr(\Sigmab_\nu \Sigmab_\nu))-16\tr(\Dbf_{\mu}\Dbf_{\nu})
\end{equation}
\end{theo}
\begin{proof}[Sketch of proof]
We can show that when considering only linear push-forward the problem \eqref{gm_eq} can be recast as a Orthogonality constrained Quadratic Program (QPOC) which, when $p=q$, admits a close form. This is made possible thanks to the Gaussian assumption which allows to compute the $4$-th order moments of the distributions using Isserlis theorem \cite{isserlis} which prove that $4$-th order moments can be computed using the $2$-nd order ones. We give the full proof in Section \ref{sec:prof_maintheo3}.
\end{proof}

\paragraph{Geometric interpretations of $L\gm$}  Interestingly enough the optimal linear map can be related, \textit{inter alia}, to the optimal linear map of the classical Monge problem. In the following we consider $p=q$ so that we have using Theorem \ref{maintheo3}: 
\begin{equation}
L\gm^{2}_2(\mu,\nu)=4(\tr(\Sigmab_\mu)-\tr(\Sigmab_\nu))^{2}+8(\tr(\Sigmab_\mu \Sigmab_\mu)+\tr(\Sigmab_\nu \Sigmab_\nu))-16\tr(\Dbf_{\mu}\Dbf_{\nu})
\end{equation}
wich corresponds to $\Abf=\V_\nu  \Dbf^{1/2}_\nu \Dbf^{-1/2}_\mu \V_\mu^\top=\Sigmab_{\nu}^{1/2}\V_{\nu}\V_{\mu}^{\top}\Sigmab_{\mu}^{-1/2}$. 

\begin{itemize}

\item[$\bullet$] \textbf{Case $\Sigmab_\mu= \ \Sigmab_\nu$.} When the covariances are equals, \ie\ $\Sigmab_\mu=\Sigmab_\nu$ we can conclude that $L\gm_2(\mu,\nu)=0+8(\tr(\Sigmab_\mu \Sigmab_\mu)+\tr(\Sigmab_\mu \Sigmab_\mu))-16\tr(\Dbf_{\mu}\Dbf_{\mu})=16\tr(\Sigmab_\mu \Sigmab_\mu)-16\tr(\Sigmab_\mu \Sigmab_\mu)=0$ which corresponds to $\Abf=\mathbf{I}$. This implies that an optimal way for transferring the masses so that, on average, the pair-to-pair distances are preserved is by the means of the identity mapping. 
\item[$\bullet$]\textbf{Case $\Sigmab_\mu= k \ \Sigmab_\nu$.} If there is a scaling factor between the covariances $\Sigmab_\mu=k\Sigmab_\nu$ with $k> 0$ then the optimal map scales uniformly the first measure to preserve the distances so that $\Abf=\frac{1}{\sqrt{k}}\mathbf{I}$. With a little calculus one can check that $L\gm_2(\mu,\nu)=4(k-1)^{2}\left(\tr(\Sigmab_{\nu})^{2}+2\tr(\Sigmab_{\nu}\Sigmab_{\nu})\right)$ so that $L\gm_2(\mu,\nu)$ is minimal, equal to zero, when $k=1$ which corresponds to the previous case. When $k$ increases above $1$ the distance increases quadratically and when $k\in]0,1]$ the distances decreases as $k$ goes to $1$. 
\item[$\bullet$] \textbf{Rotation.} Another interesting case is when we rotate the samples of the distribution and we compute the linear Gromov-Monge between the original distribution and its rotated counterpart. This case corresponds to $\Sigmab_{\nu}=\Obf \Sigmab_{\mu} \Obf^{T}$ where $\Obf \in \mathcal{O}(p)$ and we can check easily that $L\gm_{2}(\mu,\nu)=0$ with an optimal map given the rotation. This behavior is intuitive since the Gromov-Monge problem with linear map is invariant by rotations.
\item[$\bullet$] \textbf{Commuting covariances.} Finally when the covariances commute \ie\ $\Sigmab_{\mu} \Sigmab_{\nu}=\Sigmab_{\nu}\Sigmab_{\mu}$ we can relate with the linear transportation. In this situation both are simultaneously diagonalizable and eigenspaces coincide. We recall that the optimal linear map for the Wasserstein distance with $d(\xbf,\ybf)=\|\xbf-\ybf\|_2^{2}$ is $\Abf_{\wass}=\Sigmab_{\mu}^{-1/2}(\Sigmab_{\mu}^{1/2}\Sigmab_{\nu}\Sigmab_{\mu}^{1/2})^{\frac{1}{2}}\Sigmab_{\mu}^{-1/2}$ which reduces to $\Abf_{\wass}=\Sigmab_{\nu}^{1/2}\Sigmab_{\mu}^{-1/2}$ when covariances are commuting. Moreover since matrices share the same eigenspaces then we can take $\V_{\mu}=\V_{\nu}$ for the linear Gromov-Monge. In this way $\Abf$ reduces to $\Abf=\Sigmab_{\nu}^{1/2}\Sigmab_{\mu}^{-1/2}=\Abf_{\wass}$. This proves that when the covariances commute the optimal map of Wasserstein is an optimal map for the linear Gromov-Monge.
\end{itemize}

An illustration of the map $\Abf$ is given in Figure.\ref{fig:gm_fig1} where we compute $LGM$ between two 2D empirical distributions, that we consider Gaussian in first approximation. The optimal map $\Abf$ gives a different behaviour than the Wasserstein map $\Abf_{\wass}$ which seems to better grasp the transformation of the samples. It is somehow natural since the Gromov-Monge cost is less rigid than the Wasserstein one and only forces the samples to be isometrically distributed on average. However note that the target samples are arbitrary rotated for the case of Gromov-Monge since $LGM$ is invariant by rotation of the samples.

\begin{figure}[t!]
\begin{center}
\includegraphics[width=0.95\linewidth]{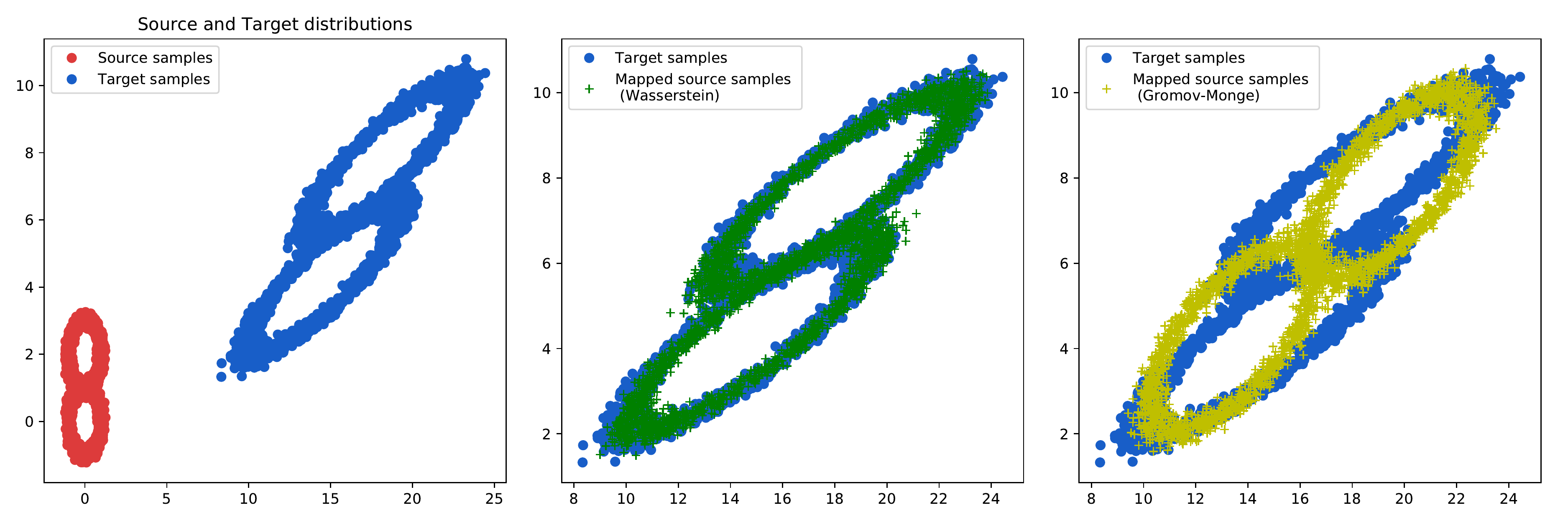}

\end{center}    
\caption{\label{fig:gm_fig1} Example of linear Gromov-Monge mapping estimation between empirical distributions. {\bf (left)} 2D
source and target distributions. {\bf (middle)} resulting linear mapping of Wasserstein $\Abf_{\wass}$. {\bf (right)} resulting linear mapping using the linear Gromov-Monge mapping $\Abf$. Note that in this case the mapped samples are arbitrary rotated. }
\end{figure}

\subsubsection{Solving the problem when $p\neq q$}

As described in Theorem \ref{maintheo3} the situation is more delicate when $p \neq q$, \ie\ when the Gaussian measures are not part of the same ground Euclidean space. In this case there is no close form anymore and one needs to solve the following Quadratic optimization with Orthogonality Constraints (QPOC) problem:
\begin{equation}
\tag{QPOC}
\label{eq:qp_oc}
\min_{\Bbf \in \Stief} -\tr(\Bbf^\top  \Dbf_\nu  \Bbf \Dbf_\mu)\stackrel{def}{=} \min_{\Bbf \in \Stief} F(\Bbf)
\end{equation}
This problem is a particular special case of optimizing a smooth function over the Stiefel manifold, which is non-convex in general but for which various methods have been proposed over the years (see \textit{e.g} \cite{Jiang_2014,Wen_2010,Abrudan_2008,absil2009optimization}). A standard approach for solving these types of problems is to leverage the Riemannian structure of $\Stief$ to use a gradient descent method on the manifold. Informally, the idea is to start with a arbitrary point $\Bbf_{0} \in \Stief$ then iteratively move in a search direction $D(\Bbf_{0})$ defined by a tangent vector while staying on $\Stief$ until a critical point is found. In general this procedure require to compute the so-called \emph{exponential map} which is quite difficult. A cheaper alternative can be found into a \emph{retraction} map which approximates the exponential map (see \textit{e.g.} \cite{pmlr-v48-liue16} and references therein). We propose to solve \eqref{eq:qp_oc} using the Geoopt library \cite{geoopt2020kochurov}. We illustrate the resulting optimal map in Figure \ref{fig:gm_fig2} where we compute the $LGM$ problem between two empirical distributions. The source distribution is 3 dimensional while the target distribution is 2 dimensional. We observe that the optimal Gromov-Monge map manages to capture the overall transformation of the source distribution. 

\begin{figure}[t!]
\begin{center}
\includegraphics[width=0.95\linewidth]{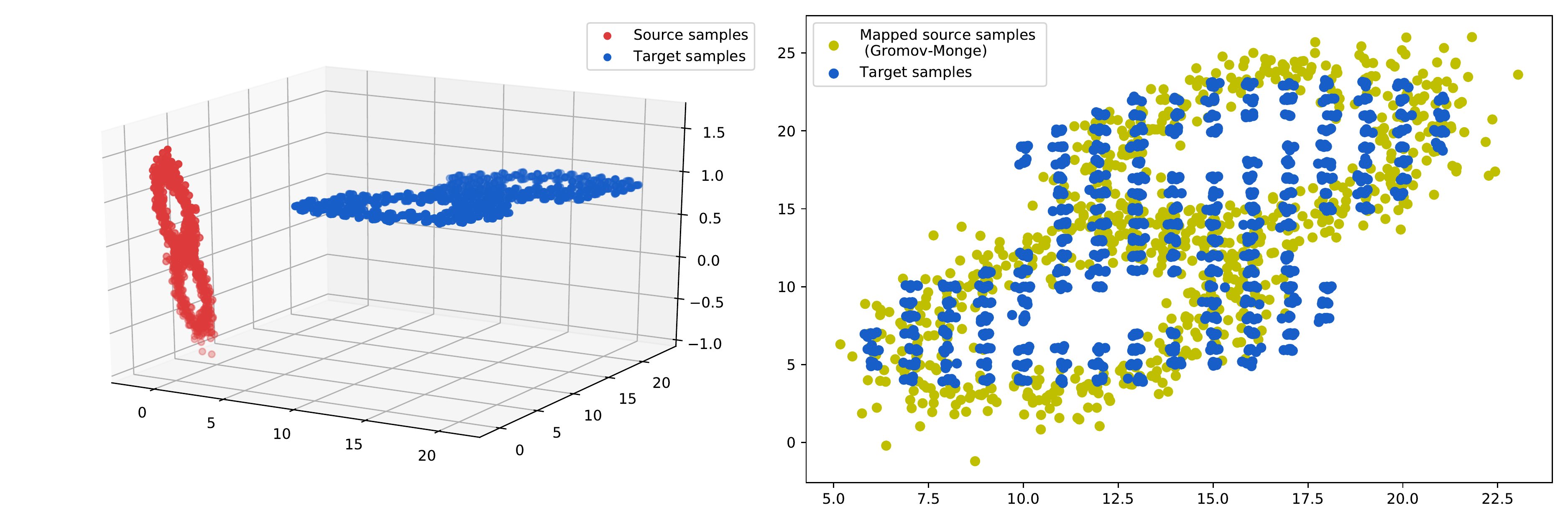}

\end{center}    
\caption{\label{fig:gm_fig2} Example of linear Gromov-Monge mapping estimation between two empirical distributions. {\bf (left)} 3D
source and 2D target distributions. {\bf (right)} resulting linear mapping using the linear Gromov-Monge mapping $\Abf$.}
\end{figure}

\section{Conclusion: perspectives and open questions} 
\label{sec:persperctives_gweucl}
In our opinion the previous results raise interesting questions and pave the path for a deeper understanding of the Gromov-Wasserstein and Gromov-Monge geometry in Euclidean spaces. The works on the regularity of optimal couplings are, to the best of our knowledge, the first work on this subject and we believe that these results could be improved in order to bridge the theoretical gap between Wasserstein and Gromov-Wasserstein.

\paragraph{Gromov-Wasserstein in Euclidean spaces} In the inner product case we have seen that among the set of solutions of \eqref{InvOT} the couples with $\Pbf^{*}$ surjective are of particular interest. For example is it possible under some mild conditions on $\mu,\nu$ to prove that there always exists an optimal couple $(\pi^{*},\Pbf^{*})$ of \eqref{InvOT} such that $\Pbf^{*}$ is surjective? If this result holds then it would imply that, if the source measure is regular, the Gromov-Wasserstein problem with inner products is equivalent to its Monge counterpart:
\begin{equation}
\inf_{T\#\mu=\nu} \int (\scalar{\xbf}{\xbf'}{p}-\scalar{T(\xbf)}{T(\xbf')}{q})^{2} \dr \mu(\xbf) \dr \mu(\xbf')
\end{equation}
The equivalence between Gromov-Monge and Gromov-Wasserstein problems would be a nice generalization of the Brenier's theorem in the case of the Gromov-Wasserstein geometry. Related to this problem the unicity of the couples $(u,\Pbf)$ is also an interesting further work. More precisely can we find suitable conditions on $\mu,\nu$ so that there is a unique $\nabla u \circ \Pbf$ which pushes $\mu$ forward to $\nu$ optimally? When the support of the measures are not ``too symmetric'' it seems reasonable that all the solutions are somehow unique modulo a class of isometries such as rotations. As described in this section the Brenier's theorem states that there is a unique optimal push forward in the case of Wasserstein geometry. This property allows to derive close form expression for the Gaussian case and we believe that similar reasoning could be made if the unicity (up to some rotations) holds for \eqref{InvOT}. Finally can we extend the previous results about the inner case for kernel similarities? More precisely when $c_\Xcal(x,x')=k_\Xcal(x,x')$ defines a kernel (same for $c_\Ycal$) we know from the kernel trick that there exists an inner product space $\mathcal{V}_{\Xcal}$ and a feature map $\phi: \Xcal \rightarrow \mathcal{V}_{\Xcal}$ such that $c_\Xcal(x,x')=\scalar{\phi(x)}{\phi(x')}{\mathcal{V}_{\Xcal}}$. In this case can we find similar results regarding the regularity of the optimal transport plans of the Gromov-Wasserstein distance?  

The squared Euclidean distance case seems to be more delicate to handle as it does not echo to classical OT costs. We believe that the dual problem could be used to find a close form expression for 1D discrete probability measures with possibly different number of atoms and non-uniform weights, as done for the inner product case. This would be an interesting stronger result than the case considered in Section \ref{sec:sgw}. We postulate that, in this case, an optimal coupling is also given by a monotone rearrangement which is increasing or decreasing. For general probability measures assessing the regularity of optimal transport plans seems more complicated. The sufficient condition of Proposition \ref{sufficientcond2} may be improved with a necessary condition in order to characterize optimal couplings under some hypothesis on $\mu,\nu$. Another interesting approach for the squared Euclidean distance case would be to consider optimal couplings when the target measure is a small perturbation of the source measure. More precisely we believe the following holds:

\begin{prob}
Let $\mu \in \P(\R^{p})$. Let $\nu=(id+\epsilon u)\#\mu \in \P(\R^{p})$ with $u: \R^{p} \rightarrow \R^{p}$ be a small perturbation of $\mu$. Then the map $\gamma=(id \times T)\#\mu$ where $T=id+\epsilon u$ is optimal for:
\begin{equation}
\inf_{\pi \in \couplingset(\mu,\nu)} \int \left(\|\xbf-\xbf'\|_{2}^{2}-\|\ybf-\ybf'\|_{2}^{2}\right)^{2} \dr \pi(\xbf,\ybf) \dr \pi(\xbf',\ybf')
\end{equation}
\end{prob}

}

\chapter{CO-Optimal Transport}
\epigraph{``Tu as tort de lire les journaux; ça te congestionne.''}{-- André Gide, \textit{Les Faux-Monnayeurs}}
\minitoc
\label{cha:coot}
\newpage
\begin{Abstract}
This chapter is based on the paper \cite{redko2020cooptimal} and addresses the problem of optimal transport on incomparable spaces. The original formulation of the optimal transport problem relies on the existence of a cost function between the samples of the two distributions, which makes it impractical for comparing data distributions supported on different topological spaces. To circumvent this limitation, we propose a novel OT problem, named \COOT\ for CO-Optimal Transport, that aims to simultaneously optimize two transport maps between both samples and features. This is different from other approaches that either discard the individual features by focusing on pairwise distances (e.g. Gromov-Wasserstein) or need to model explicitly the relations between the features. \COOT\ leads to interpretable correspondences between both samples and feature representations and holds metric properties. We provide a thorough theoretical analysis of our framework and establish rich connections with the Gromov-Wasserstein distance. We demonstrate its versatility with two machine learning applications in heterogeneous domain adaptation and co-clustering/data summarization, where \COOT\ leads to performance improvements over the competing state-of-the-art methods. 
\end{Abstract}

\section{Introduction}

The problem of comparing two sets of samples arises in many fields in machine learning, such as manifold alignment \cite{Cui:2014}, image registration \cite{Haker:2001},
unsupervised word and sentence translation \cite{rapp-1995-identifying} among others. When correspondences between the sets are known \textit{a priori}, one can align them with a global transformation of the features, \textit{e.g}, with the widely used \emph{Procrustes
analysis} \cite{oro2736,Goodall:1991}. For unknown correspondences, other popular alternatives to this method include correspondence free manifold alignment procedure \cite{Wang2009ManifoldAW}, soft assignment coupled with a Procrustes matching \cite{Rangarajan:1997} or Iterative closest point and its variants for 3D shapes \cite{Besl_ICP:1992,yang2020teaser}. 
{When one models} the considered sets of samples as empirical probability distributions{,} the optimal transport framework {provides a solution} to find, without supervision, a soft-correspondence map between them given by an \emph{optimal coupling}. 
{OT-based approaches} have been used with success in numerous applications such as embeddings'
alignments \cite{alavarez:2019,Grave2018UnsupervisedAO} and Domain Adaptation (DA)
\cite{courty2017optimal} to name a few. {However, one important limit of using OT} for such tasks is that the two sets are assumed to lie in the same space so that the cost between samples across them can be computed. This
major drawback {does not allow OT to} handle correspondences' estimation across heterogeneous
spaces, preventing its application in problems such as, for instance, heterogeneous DA (HDA). To circumvent this restriction, one
may rely on the Gromov-Wasserstein distance \cite{memoli_gw}{: a} non-convex quadratic OT problem {that} finds the correspondences between two sets of samples based on their pairwise intra-domain similarity (or distance) matrices. Such an approach was successfully applied {to} sets of samples {that} do not lie in the same Euclidean space, \textit{e.g} for shapes \cite{solomon_entropic_2016}, word embeddings
\cite{alvarez-melis_gromov-wasserstein_2018} {and HDA \cite{ijcai2018-412} mentioned previously}. One important limit of GW is that it finds the samples' correspondences but discards the relations between the features by considering pairwise similarities only. Another line of works \cite{alavarez:2019,Grave2018UnsupervisedAO} considers the problem of matching sets of points with respect to a global transformations of the features, usually modeled by a linear transformation such as a rotation. These approches differ from the work proposed here where we consider instead a probabilistic coupling of the features as described below.

In this Chapter, we propose a novel OT approach called CO-Optimal transport (\COOT) that simultaneously infers the correspondences between the samples \emph{and} the features of two arbitrary sets. Our new formulation includes GW as a special case, and has an extra-advantage of working with raw data directly without needing to compute, store and choose computationally demanding similarity measures required for the latter. Moreover, \COOT\ provides a meaningful mapping between both instances and features across the two datasets thus having the virtue of being interpretable. We thoroughly analyze the proposed problem, derive an optimization procedure for it and highlight several insightful links to other approaches. On the practical side, we provide evidence of its versatility in machine learning by putting forward two applications in HDA and co-clustering where our approach achieves state-of-the-art results. 

The rest of this chapter is organized as follows. We introduce the \COOT\ problem
in Section \ref{sec:coot_problem}, states its mathematical properties in Section \ref{sec:prop_of_coot} and give an optimization routine for solving it efficiently in Section \ref{sec:optimsec}. In Section \ref{sec:relation_with_other_ot_distances}, we show how \COOT\ is related to other OT-based distances and recover efficient solvers for some of them in particular cases. Finally, in Section \ref{sec:hda} and Section \ref{sec:co_clustering}, we present an experimental study providing highly competitive results in {HDA} and co-clustering compared to several baselines. 

\section{CO-Optimal transport optimization problem}
\label{sec:coot_problem}
We consider two datasets represented by matrices $\X=[\x_1,\dots,\x_n]^T\in\mathbb{R}^{n\times d}$ and $\X'=[\xbf'_1,\dots,\xbf'_{n'}]^T\in\mathbb{R}^{n'\times d'}${, where in general we assume that $n\neq n'$ and $d\neq d'$.} In what follows, the rows of the datasets are denoted as \emph{samples} and its columns as \emph{features}. {We endow the samples $(\xbf_i)_{i \in [\![n]\!]}$ and $(\xbf'_i)_{i \in [\![n']\!]}$ with weights $\w=[w_1,\dots,w_n]^\top \in\Sigma_n$ and $\w'=[w_1',\dots,w_{n'}']^\top \in\Sigma_{n'}$ that both lie in the simplex so as to define empirical distributions supported on $(\xbf_i)_{i \in [\![n]\!]}$ and $(\xbf'_i)_{i \in [\![n']\!]}$. In addition to these distributions, we similarly associate weights given by vectors $\v\in\Sigma_d$ and $\v'\in\Sigma_{d'}$ with features. Note that
when no additional information is available about the data, all the weights' vectors
can be set as uniform.} 

We define the CO-Optimal Transport problem as follows:
\begin{equation}
\tag{COOT}
 \label{eq:co-optimal-transport}
\min_{\begin{smallmatrix}\GGs \in\Pi(\w,\w') \\ \GGv\in\Pi(\v,\v')\end{smallmatrix}} \quad \sum_{i,j,k,l} L(X_{i,k},X'_{j,l})\GGs_{i,j}\GGv_{k,l} =\min_{\begin{smallmatrix}\GGs \in\Pi(\w,\w') \\ \GGv\in\Pi(\v,\v')\end{smallmatrix}} \froeb{\L(\X,\X') \otimes \GGs}{\GGv}\\
\end{equation}
where $L:\mathbb{R}\times \mathbb{R} \rightarrow \mathbb{R}_+$ is a divergence measure between 1D variables, $\L(\X,\X')$ is the $d\times d'\times n\times n'$ tensor of all pairwise divergences between the elements of $\X$ and $\X'$, and $\couplingset(\cdot,\cdot)$ is the set of linear transport constraints. 

Note that problem \eqref{eq:co-optimal-transport} seeks for a simultaneous transport $\GGs$ between samples and a transport
$\GGv$ between features across distributions. In the following, we write $\COOT(\X,\X',\w,\w',\v,\v')$ (or $\COOT(\X,\X')$ when it is
clear from the context) to denote the objective value of the optimization problem \eqref{eq:co-optimal-transport}. 

\paragraph{Entropic regularization} Equation \eqref{eq:co-optimal-transport} can be also extended to the entropic regularized case favoured in the OT community for remedying the heavy computation burden of OT and reducing its sample complexity \cite{cuturi2013sinkhorn,altschuler2017near,genevay:2019}. This leads to the following problem:
\begin{equation}
 \label{eq:co-optimal-transport-reg}
      \min_{\begin{smallmatrix}\GGs \in\Pi(\w,\w') \\ \GGv\in\Pi(\v,\v')\end{smallmatrix}} \froeb{\L(\X,\X') \otimes \GGs}{\GGv} 
        + \Omega(\GGs,\GGv)
   \end{equation}
where for $\epsilon_{1},\epsilon_{2}>0$, the regularization term writes as $\Omega(\GGs,\GGv)=  \epsilon_{1} H(\GGs|\w
\w'^{T})+\epsilon_{2}H(\GGv|\v\v'^{T})$ with
$H(\GGs|\w\w'^{T})=\sum_{i,j} \log(\frac{\pi^s_{i,j}}{w_{i}w'_{j}})\pi^s_{i,j}$
being the relative entropy. Note that similarly to OT \cite{cuturi2013sinkhorn} and GW
\cite{peyre2016gromov}, adding the regularization term can lead to a more robust
estimation of the transport matrices but prevents them from being sparse.

\begin{figure*}[!t]
    \centering
    \includegraphics[width=1\linewidth]{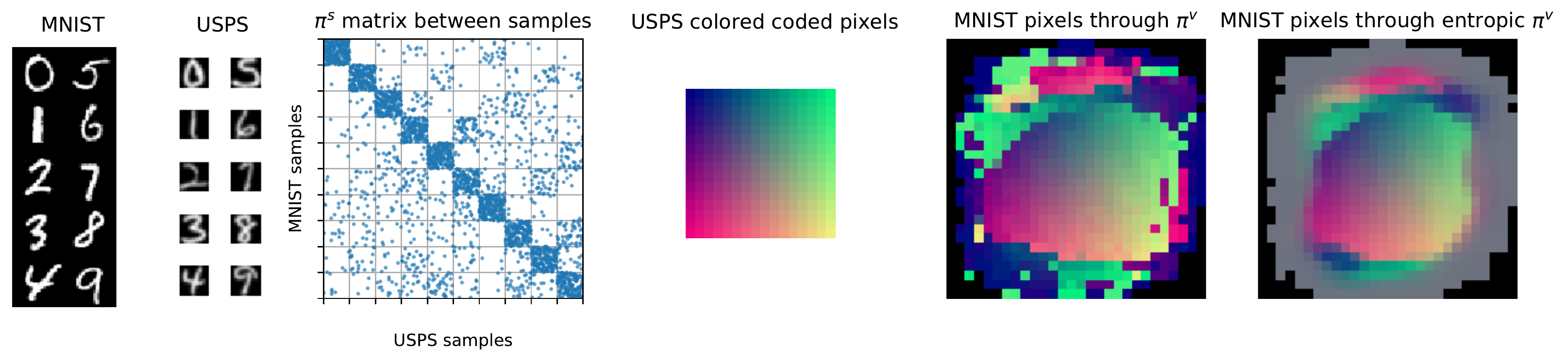}
    \caption{Illustration of \COOT\ between MNIST and USPS datasets. \textbf{(left)} samples from MNIST and USPS data sets; \textbf{(center left)} Transport matrix $\GGs$ between samples sorted by class; \textbf{(center)} USPS image with pixel{s} colored \emph{w.r.t..} their 2D position; \textbf{(center right)} transported colors on MNIST image using $\GGv$, black pixels correspond to non-informative MNIST pixels always at 0; \textbf{(right)} transported colors on MNIST image using $\GGv$ with entropic regularization.}
    \label{fig:mnist_usps}
\end{figure*}

\subsection{Illustration of \COOT \label{sec:illus_of_coot}} 

In order to illustrate {our proposed} \COOT\ method and to explain the intuition behind it, we solve the optimization problem \eqref{eq:co-optimal-transport} {using the algorithm described in Section \ref{sec:optimsec}}
between two classical digit recognition datasets: MNIST and USPS. We choose these particular datasets for our illustration as they contain images of different
resolutions (USPS is 16$\times$16 and MNIST is 28$\times$28) that belong to the
same classes (digits between 0 and 9). Additionally, the digits are also
slightly differently centered as illustrated on the examples in the left part of
Figure \ref{fig:mnist_usps}. {Altogether,} this means that without specific
pre-processing, the images do not lie in the same {topological} space and thus
cannot be compared directly using conventional distances. We randomly select 300
images per class in each dataset, normalize magnitudes of pixels to $[0,1]$ and consider digit images as \emph{samples} while each pixel acts as a \emph{feature} leading to 256 and 784 features for USPS and MNIST respectively. We use
uniform weights for $\w,\w'$ and normalize average values of each pixel for $\v,\v'$ in order to discard non-informative ones that are always equal to 0.

The result of solving problem \eqref{eq:co-optimal-transport} is reported in
Figure \ref{fig:mnist_usps}. In the center-left part, we provide the coupling
$\GGs$ between the samples, {\textit{i.e} the different images}, sorted by class and observe that 67\% of mappings occur between the samples from the same class as indicated by block diagonal structure of the coupling matrix.
The coupling $\GGv$, in its turn, describes the relations between {the features, \textit{i.e} the pixels,} in both domains. 
To visualize it, we color-code the pixels of the source USPS image and use $\GGv$ to transport the colors on a target MNIST image so that its pixels are defined as  convex combinations of colors from the former with coefficients given by $\GGv$. 
The corresponding results are shown in the right part of
Figure~\ref{fig:mnist_usps} for both the original \COOT\ and its entropic
regularized counterpart. From these two images, we can observe that colored
pixels appear only in the central areas and exhibit a strong spatial coherency despite the fact that 
the geometric structure of the image is totally unknown to the optimization problem, as each pixel is treated as an independent feature. \COOT\ has recovered a meaningful spatial transformation between the two datasets in a completely unsupervised way, different from trivial rescaling of images that one may expect when aligning USPS digits occupying the full image space and MNIST digits lying in the middle of it. 

For further evidence Figures \ref{fig:usps_to_mnist_piv} and \ref{fig:mnist_to_usps_piv} illustrate different images of both datasets obtained by transporting pixels from USPS (\textit{resp.} MNIST) to MNIST (\textit{resp.} USPS) using the optimal coupling $\GGv$. Notably the case USPS $\rightarrow$ MNIST show that transporting the pixel through $\GGv$ leads to a better spatial coherency than a simple rescaling of the image.

\begin{figure*}[t]
    \centering
    \includegraphics[width=1\linewidth]{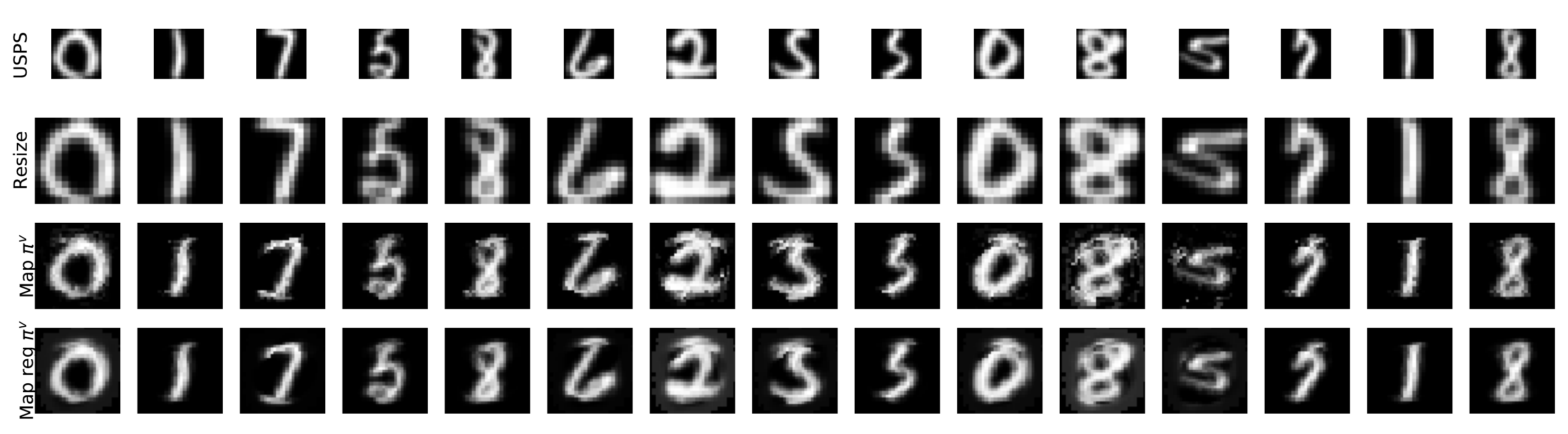}
    \caption{Linear mapping from USPS to MNIST using $\GGv$.  \textbf{(First row)} Original USPS samples,
    \textbf{(Second row)} Samples resized to target resolution, \textbf{(Third row)} Samples mapped using $\GGv$, \textbf{(Fourth row)} Samples mapped using $\GGv$ with entropic regularization.}
    \label{fig:usps_to_mnist_piv}
\end{figure*}
\begin{figure*}[t]
    \centering
    \includegraphics[width=1\linewidth]{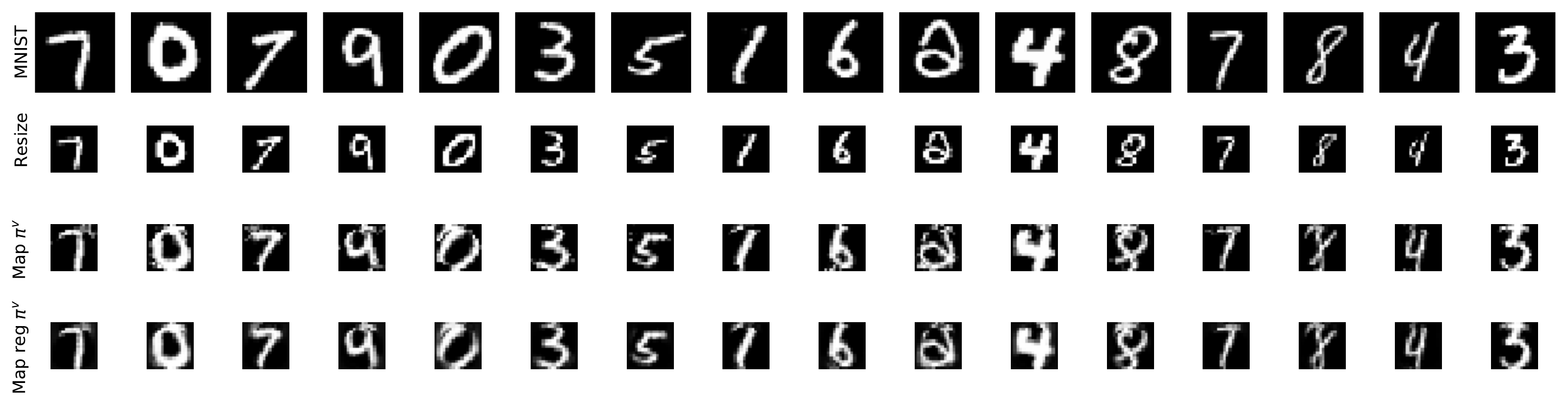}
    \caption{Linear mapping from MNIST to USPS using $\GGv$.  \textbf{(First row)} Original MNIST samples,
    \textbf{(Second row)} Samples resized to target resolution, \textbf{(Third row)} Samples mapped using $\GGv$, \textbf{(Fourth row)} Samples mapped using $\GGv$ with entropic regularization.}
    \label{fig:mnist_to_usps_piv}
\end{figure*}

\section{{Properties of \COOT}}
\label{sec:prop_of_coot}

\paragraph{\COOT\ as a billinear program} \COOT\ is a special case of a Quadratic Program (QP) with linear constraints called a Bilinear Program (BP). More precisely, it is an indefinite BP
problem \cite{gallo:1977}. It was proved (\textit{e.g} in in \cite{pardalos:1987,horst1996global}) that there exists an optimal solution lying on extremal points of the polytopes $\Pi(\w,\w')$ and $\Pi(\v,\v')$. When $n=n',d=d'$ and weights $\w=\w'=\frac{\bm{1}_n}{n}, \v=\v'=\frac{\bm{1}_d}{d}$ are uniform, Birkhoff’s theorem \cite{birkhoff:1946} states that the set of extremal points of $\Pi(\frac{\one_{n}}{n},\frac{\one_{n}}{n})$ and $\Pi(\frac{\one_{d}}{d},\frac{\one_{d}}{d})$ are the set of permutation matrices so that there exists an optimal solution $(\GGs_{*},\GGv_{*})$ which transport maps are supported on two permutations $\sigma_{*}^{s},\sigma_{*}^{v} \in \mathbb{S}_{n} \times \mathbb{S}_{d}$. 

The BP problem is also related to the Bilinear Assignment Problem (BAP) where $\GGs$ and $\GGv$ are searched in the set of permutation matrices. The latter was shown to be NP-hard if $d=O(\sqrt[\leftroot{-2}\uproot{2}r]{n})$ for fixed $r$ and solvable in polynomial time if $d=O(\sqrt{\log(n)})$
\cite{Custic:2016}. In this case, we look for the best permutations of the rows and columns of our datasets that lead to the smallest cost. \COOT\ provides a tight convex relaxation of the BAP by 1) relaxing the constraint set of permutations into the convex set of doubly stochastic matrices and 2) ensuring that two problems are equivalent, \ie, one can always find a pair of permutations that minimizes \eqref{eq:co-optimal-transport}, as explained in the paragraph above.

Finding a meaningful similarity measure between datasets is useful in many machine learning tasks as pointed out, \textit{e.g} in \cite{alvarezmelis2020geometric}. Interestingly enough, \COOT\ induces a {notion of} distance between datasets $\X$ and $\X'$. More precisely it vanishes $\textit{iff}$ they are the same up to a permutation of rows and columns as established below:
\begin{prop}[\COOT\ is a distance]
\label{sec:metric_properties}
Suppose $L=|\cdot|^{p}, p \geq 1$, $n=n',d=d'$ and that the weights $\w,\w',\v,\v'$ are uniform. Then $\COOT(\X,\X')=0$ \textit{iff} there exists a permutation of the samples $\sigma_{1} \in S_{n}$ and of the features $\sigma_{2} \in S_d$, \textit{s.t}, $\forall i,k \ \X_{i,k}=\X'_{\sigma_{1}(i),\sigma_{2}(k)}$. Moreover, it is symmetric and satisfies the triangular inequality as long as $L$ satisfies the triangle inequality, \ie, $\COOT(\X,\X'')\leq \COOT(\X,\X')+\COOT(\X',\X'').$
\end{prop}
Note that in the general case when $n\neq n', d\neq d'$, positivity and triangle inequality still hold but $\COOT(\X,\X')>0$. The proof can be found in Section \ref{sec:properties}. Interestingly, our result generalizes the metric property proved in \cite{faliszewski19} for the election isomorphism problem with this latter result being {valid} only for the BAP case (for a discussion on the connection between \COOT\ and the work of \cite{faliszewski19}, see Section \ref{sec:election_iso}). Finally, we note that this metric property means that \COOT\ can be used as a divergence in a large
number of potential applications as, for instance, in generative learning \cite{bunne_gan}.

\section{Optimization algorithm and complexity \label{sec:optimsec}}

        \begin{algorithm}[t]
        \caption{\label{alg:bcd}
		   BCD for \COOT}
        \begin{algorithmic}[1]
                \State $\pi^{s}_{(0)}\leftarrow \w\w'^{T},\pi^{v}_{(0)}\leftarrow \v\v'^{T}, k \leftarrow 0$
          \While {$k < $ maxIt {\bf and} $err >$ 0} 
          \State $\GGs_{(k)} \leftarrow \argmin_{\GGs \in \couplingset(\w,\w')} \froeb{\L(\X,\X')\otimes \GGv_{(k-1)}}{\GGs}$ // \text{ linear OT problem on the samples}
          \State $\GGv_{(k)} \leftarrow \argmin_{\GGv \in \couplingset(\v,\v')} \froeb{\L(\X,\X')\otimes \GGs_{(k-1)}}{\GGv}$ // \text{ linear OT problem on the features}
                   \State $err \leftarrow ||\GGv_{(k-1)} - \GGv_{(k)}||_F$
          \State $k\leftarrow k+1$  
          \EndWhile        
        \end{algorithmic}
        \end{algorithm}

Even though solving  \COOT\ exactly is NP-hard, in practice computing a solution can be done rather efficiently. To this end, we propose to use
Block Coordinate Descent (BCD) that consists in iteratively solving the problem
for $\GGs$ or $\GGv$ with the other kept fixed. 
Interestingly, this boils down to solving at each step a linear OT problem that requires $O(n^3\log(n))$ operations with a network simplex algorithm as detailed in the pseudo-code given in Algorithm \ref{alg:bcd}. This approach, also known as the ``mountain climbing procedure''
\cite{konno:1976} in the BP literature, was proved to decrease the loss at each iteration and so to
converge within a finite number of iterations \cite{horst1996global}. We also note that at
each iteration one needs to compute the equivalent cost matrix $L(\X,\X')\otimes \GG^{(\cdot)}$ which
has a complexity of $O(ndn'd')$. However, one can reduce it using Proposition 1 from
\cite{peyre2016gromov} for the case when $L$ is the squared Euclidean distance $|\cdot|^{2}$ or the Kullback-Leibler divergence. In this case, the overall computational complexity becomes $O(\min\{(n+n')dd'+n'^{2}n;(d+d')nn'+d'^{2}d\})$. We refer the interested reader to Section \ref{sec:complexity} for further details. 

Finally, we can use the same BCD procedure for the entropic regularized version of \COOT\
\eqref{eq:co-optimal-transport-reg} where each iteration
an entropic regularized OT problem can be solved efficiently
using Sinkhorn's algorithm \cite{cuturi2013sinkhorn} with several possible improvements
\cite{altschuler2017near,Altschuler:2019,screekhorn:2019}. Note that this procedure can be easily adapted in the same way to include unbalanced OT problems \cite{chizat_unbalanced} as well.

\section{Relation with other OT distances}
\label{sec:relation_with_other_ot_distances}

\subsection{Gromov-Wasserstein}
The \COOT\ problem is defined for arbitrary matrices $\X\in \R^{n\times d}, \X'\in
\R^{n'\times d'}$ and so can be readily used to compare pairwise similarity matrices between the samples $\mathbf{C}=\left(c(\xbf_i,\xbf_j)_{i,j}\right)\in \R^{n\times n}, \mathbf{C}'=\left(c'(\xbf_k',\xbf_l')\right)_{k,l} \in \R^{n'\times n'}$ for some $c,c'$. To avoid redundancy, we use the term ``similarity'' for both similarity and distance functions in what follows. This situation arises in applications dealing with relational data, \textit{e.g}, in a graph context (see Chapter \ref{cha:fgw}) or \textit{e.g.} in deep metric alignment \cite{gwcnn}. These problems have been
successfully tackled using the Gromov-Wasserstein distance (see Chapter \ref{cha:ot_general}). We recall that given $\mathbf{C}\in \R^{n\times n}$ and $\mathbf{C}'\in \R^{n'\times n'}$, the GW distance is defined by: 
\begin{equation}
\label{eq:gromov}
\gw(\mathbf{C},\mathbf{C}',\w,\w')=\min_{\GGs \in\Pi(\w,\w')} \froeb{\L(\mathbf{C},\mathbf{C}') \otimes \GGs}{\GGs}
\end{equation}
As suggested by the similar objective functions and constraints, GW and \COOT\ are linked in multiple ways. Below, we explicit the link between GW and \COOT\ using a reduction of a concave QP to an associated BP problem established in \cite{Konno1976} and show that they are equivalent when working with squared Euclidean distance matrices $\mathbf{C}\in \R^{n\times n}, \mathbf{C}' \in \R^{n'\times n'}$ or with inner product similarities (see Chapter \ref{cha:gw_euclidean}). More precisely, this latter equivalence follows from \cite{Konno1976} where it was shown that a concave QP can be solved by a reduction to an associated BP problem:

\begin{theo}[Adapted from \cite{Konno1976}]
\label{equivalence_theo}
If $\mathbf{Q}$ is a negative definite matrix then problems: 
\begin{equation}
\label{qap_konno}
\begin{array}{cl}{\min _{\xbf} f(\xbf)} & {=\mathbf{c}^{T} x+\frac{1}{2} \xbf^{T} \mathbf{Q} \xbf} \\ {\text {s.t.}} & {\mathbf{A} \xbf = \mathbf{b}},\;  {\xbf} {\geq 0}\end{array}
\end{equation}
\begin{equation}
\label{bilinearqap}
\begin{array}{cl}{\min _{\xbf, \ybf} g(\xbf, \ybf)} & {=\frac{1}{2}\mathbf{c}^{T} \xbf+\frac{1}{2} \mathbf{c}^{T}\ybf+\frac{1}{2} \xbf^{T} \mathbf{Q} \ybf} \\ {\text {s.t.}} & {\mathbf{A} \xbf = \mathbf{b}, \mathbf{A} \ybf =\mathbf{b}},\;   {\xbf, \ybf}  {\geq 0}\end{array}
\end{equation}
are equivalent. More precisely, if $\xbf^{*}$ is an optimal solution for \eqref{qap_konno}, then $(\xbf^{*},\xbf^{*})$ is a solution for \eqref{bilinearqap} and if $(\xbf^{*},\ybf^{*})$ is optimal for \eqref{bilinearqap}, then both $\xbf^{*}$ and $\ybf^{*}$ are optimal for \eqref{qap_konno}.
\end{theo}
Using this principle one can link GW with the \COOT\ problem when working on intra domain similarity matrices $\mathbf{C}\in \R^{n\times n}, \mathbf{C}' \in \R^{n'\times n'}$ thanks to the next proposition:

\begin{prop}
\label{concavity_gw_theo2}
Let $L=|\cdot|^{2}$ and suppose that $\mathbf{C} \in \R^{n\times n},\mathbf{C}' \in \R^{n'\times n'}$ are squared Euclidean distance matrices such that $\mathbf{C}=\xbf \mathbf{1}_{n}^{T}+\mathbf{1}_{n}\xbf^{T}-2\X\X^{T}, \mathbf{C}'=\xbf' \mathbf{1}_{n'}^{T}+\mathbf{1}_{n'}\xbf'^{T}-2\X'\X'^{T}$ with $\xbf=\text{diag}(\X\X^T),\xbf'=\text{diag}(\X'\X'^T)$. Then, the GW problem can be written as {a concave quadratic program (QP) which Hessian reads} $\mathbf{Q}=-4*\X\X^T \otimes_{K} \X'\X'^T$.

If $\mathbf{C} \in \R^{n\times n},\mathbf{C}' \in \R^{n'\times n'}$ are inner products similarities, \ie\ such that $\mathbf{C}=\X\X^{T},\mathbf{C}'=\X'\X'^{T}$ then the GW is also a concave quadratic program (QP) which Hessian reads $\mathbf{Q}=-2*\X\X^T \otimes_{K} \X'\X'^T$.
\end{prop}

When working with arbitrary similarity matrices, \COOT\ provides a lower-bound for GW and using Proposition \ref{concavity_gw_theo2} we can prove that both problems become equivalent for the cases of squared Euclidean distances and inner product similarities.
\begin{prop}
\label{prop:cot_equal_gw}
Let $\mathbf{C} \in \R^{n\times n},\mathbf{C}' \in \R^{n'\times n'}$ be any symmetric matrices, then: 
$$\COOT(\mathbf{C},\mathbf{C}',\w,\w',\w,\w')\leq GW(\mathbf{C},\mathbf{C}',\w,\w').$$
The converse is also true {under the hypothesis of Proposition \ref{concavity_gw_theo2}}. In this case, if $(\GGs_{*},\GGv_{*})$ is an optimal solution of
\eqref{eq:co-optimal-transport}, then both $\GGs_{*},\GGv_{*}$ are solutions of
\eqref{eq:gromov}. Conversely, if $\GGs_{*}$ is an optimal solution of
\eqref{eq:gromov}, then $(\GGs_{*},\GGs_{*})$ is an optimal solution for
\eqref{eq:co-optimal-transport} .
\end{prop}

\paragraph{Equivalence of algorithms} Under the hypothesis of Proposition \ref{concavity_gw_theo2} we know that there exists an optimal solution for the
\COOT\ problem of the form $(\GG_{*},\GG_{*})$, where $\GG_{*}$ is an optimal
solution of the GW problem. This gives a conceptually very simple fixed-point
procedure where one only iterates over one coupling in order to compute a
optimal solution of GW as described in Algorithm \ref{alg:bcd2}.
Interestingly enough, the iterations of the fixed point method are exactly equivalent to
the Frank Wolfe procedure described in Chapter \ref{cha:fgw}, since, in the concave
setting, the line search step can be fixed to $1$ \cite{NIPS2018_7323} (see
Section \ref{sec:equivalence} for more details). Also note that the steps of
Algorithm \ref{alg:bcd2} are iterations of Difference of Convex Algorithm (DCA) \cite{tao2005dc,yuille2003concave} where
the concave function is approximated at each iteration by its linear
majorization. When applying the same procedure for entropic
regularized \COOT, the resulting DCA also recovers exactly the projected gradients
iterations proposed in \cite{peyre2016gromov} for solving the entropic regularized version of GW.

\begin{algorithm}[t]
    \caption{\label{alg:bcd2}
     DC Algorithm for solving \COOT\ and GW with squared Euclidean matrices or inner product similarities}
        \begin{algorithmic}[1]
                  \State $\GG_{(0)}\leftarrow \w\w'^{T}, k \leftarrow 0$
          \While {$k < $ maxIt {\bf and} $err >$ 0} 
          \State $\GG_{(k)} \leftarrow \argmin_{\GG \in \couplingset(\w,\w')} \froeb{\L(\C,\C')\otimes \GG_{(k-1)}}{\GG}$ // \text{ linear OT problem}
          \State $err \leftarrow ||\GG_{(k-1)} - \GG_{(k)}||_F$
          \State $k\leftarrow k+1$  
          \EndWhile \\
          \Return {$\GG_{(k)}$ for GW and $(\GG_{(k)},\GG_{(k)})$ for \COOT\ }
        \end{algorithmic}
\end{algorithm}

We would like to stress out that \COOT\ is much more than a generalization of GW and that is for multiple reasons. First, it can be used on raw data without requiring to choose or compute the similarity matrices, that can be costly, for instance, when dealing with shortest path distances in graphs, and to store them ($O(n^2+n'^2)$ overhead). Second, it can take into account additional information given by feature weights $\v,\v'$ and provides an interpretable mapping between them across two heterogeneous datasets. Finally, contrary to GW, \COOT\ is not invariant neither to feature rotations nor to the change of signs leading to a more informative samples' coupling when compared to GW in some applications. One such example is given in the previous MNIST-USPS transfer task (Figure~\ref{fig:mnist_usps}) for which the coupling matrix obtained via GW (given in Figure~\ref{fig:cooot_vs_gw}) exhibits important flaws in respecting class memberships when aligning samples.

\begin{figure*}[t]
  \centering
  \includegraphics[width=0.9\linewidth]{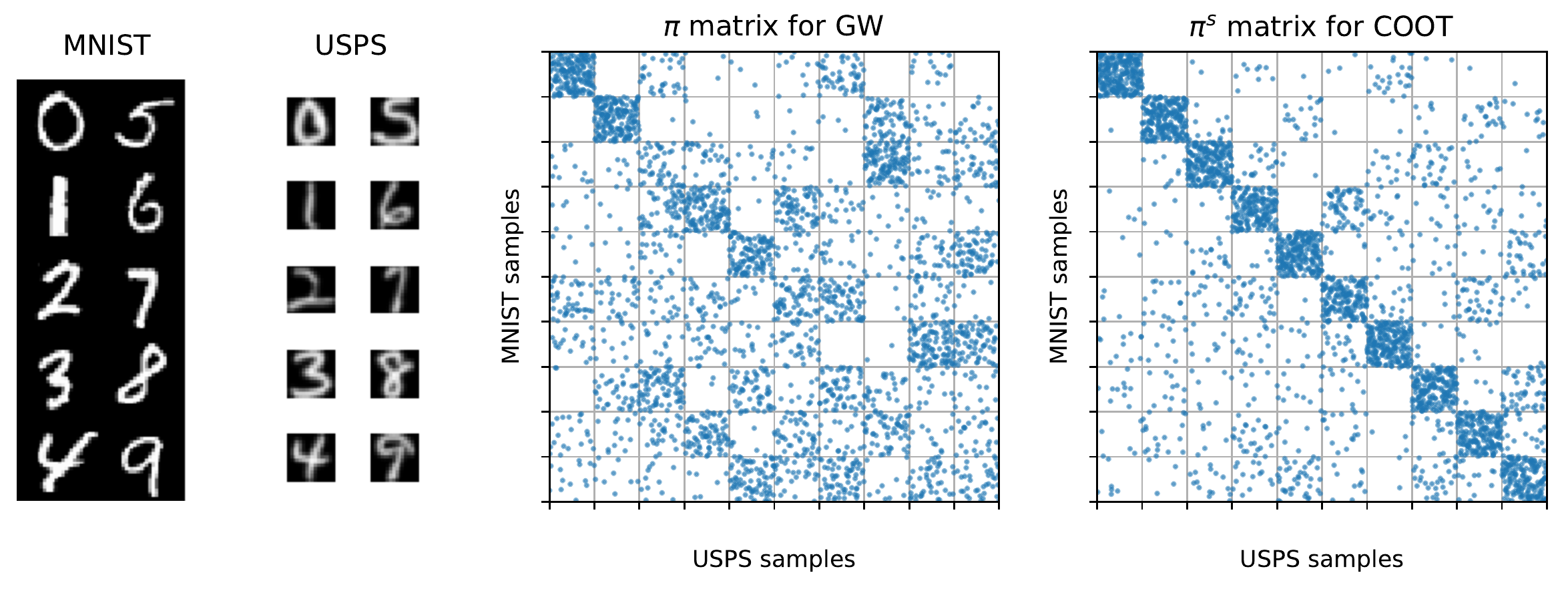}
  \caption{Comparison between the coupling matrices obtained via GW and COOT on MNIST-USPS.}
  \label{fig:cooot_vs_gw}
\end{figure*}

\subsection{Relation with Invariant OT and hierarchical approaches} 

In \cite{alavarez:2019}, authors consider a scenario where the OT problem is used to align measures supported on sets of points for which meaningful pairwise distances are hard or impossible to calculate. This may happen, for instance, due to some latent transformation that have been applied to the features. The main underlying idea of their approach is to find an assignment of the points and to calculate a transformation to match the features. More precisely, for two datasets $\X$, $\X'$, with same feature space $\R^{d}$, the corresponding objective function is:
\begin{equation}
    InvOT(\X,\X') = \min_{\GG \in \Pi(\w,\w')} \min_{\Pbf \in F_{d}} \ \froeb{\mathbf{M}_\Pbf}{\GG} 
    \label{eq:OTalvarez}
\end{equation}
where $M_\Pbf(i,j) = \|\x_i-\Pbf \x'_j\|^{2}_2$ and $F_d$ is a space of matrices $F_d = \{\mathbf{P} \in \R^{d\times d} | \ ||\Pbf||_\F=\sqrt{d}\}$.
As noted by the authors in Lemma 4.3, equation \eqref{eq:OTalvarez} can be related to the GW problem when $\mathbf{C}, \mathbf{C}'$ are calculated using inner product similarities and when $\X'$ is $\w'$-whitened, \textit{i.e} $\X'^{T}\text{diag}(\w')\X'=\mathbf{I_{d'}}$. In this case, author show that GW and InvOT are equivalent, namely a solution of GW is a solution of $InvOT$ and conversely. Since the GW problem with cosine similarities is actually concave we have proven \COOT\ and GW are also equivalent in this case which proves the following proposition:

\begin{prop}
\label{prop:cot_equal_gw_equal_invariantot}
Using previous notations, $L = |\cdot|^{2}$, $d=d'$, and inner product similarities $\mathbf{C}=\X\X^T,\mathbf{C}'=\X'\X'^T$. Suppose that $\X'$ is $\w'$-whitened, \textit{i.e} $\X'^{T}\text{diag}(\w')\X'=\mathbf{I_{d'}}$. Then, $InvOT(\X, \X')$, $\COOT(\mathbf{C},\mathbf{C}')$ and $GW(\mathbf{C},\mathbf{C}')$ are equivalent, namely any optimal coupling of one of this problem is a solution to others problems.
\end{prop}

Another way of proving this result is to consider the Theorem \ref{maintheorem} of Chapter \ref{cha:gw_euclidean} where we proved that the Gromov-Wasserstein distance, when considering inner product similarities, is equivalent to the problem:
\begin{equation}
\tag{MaxOT}
\label{maxot2}
\min_{\GG \in \Pi(\w,\w')} \min_{\Pbf \in F_{d}} \ \sum_{i=1,j=1}^{n,n'} \langle \x_i,\Pbf \x'_j \rangle \pi_{ij}
\end{equation} 
When $\X'$ is $\w'$-whitened we can check easily by developping the terms in $M_\Pbf(i,j) = \|\x_i-\Pbf \x'_j\|^{2}_2$ that $InvOT$ is equivalent to \eqref{maxot2} (see Chapter \ref{cha:gw_euclidean}). 
InvOT was further used as a building block for aligning clustered datasets in \cite{Hierarchical_ot_2019} where the authors applied it as a divergence measure between the clusters, thus leading to an approach different from ours. Finally, in \cite{YurochkinCCMS19} the authors proposed a hierarchical OT distance as an OT problem with costs defined based on precomputed Wasserstein distances but with no global features' mapping, contrary to \COOT\ that optimises two couplings of the features and the samples simultaneously. 

\subsection{Election isomorphism}

{The election isomorphism problem mentioned earlier has recently been introduced in
\cite{faliszewski19} to compare two elections given by preference orders for candidates of voters. The authors express their problem quite
similarly to \COOT\ and also seek for correspondences between voters and
candidates across two elections where the preferences of each voter are known
(which is unrealistic in modern democracies). They focus on the setting
where both elections have exactly the same number of voters $n=n'$ and candidates
$d=d'$ and search for an optimal permutation \textit{via} a
Linear Integer Program. It is interesting to see that their problem consisting in aligning voters using the
Spearman distance is actually equivalent to solving $\COOT$ with $L=|\cdot|$ on voters
preferences. To that extent, \COOT\ is a more
general approach as it is applicable for general loss functions $L$, contrary to
the Spearman distance used in \cite{faliszewski19}, and generalizes to the cases
where $n\neq n'$ and $m\neq m'$. A more detailed comparison is given in
Section \ref{sec:election_iso}.

\section{Experiments}

In the next section, we {highlight} two possible applications of $\COOT$ in a machine learning context: HDA and co-clustering. {We consider these two particular tasks because 1) OT-based methods are considered as a strong baseline in DA; 2) COOT is a natural match for co-clustering as it allows for soft assignments of data samples and features to co-clusters.

\subsection{Heterogeneous domain adaptation}
\label{sec:hda}
In a classification context, the problem of domain adaptation (DA) arises whenever  one has to perform classification on 
a set of data $\X_t = \{\x_i^t\}_{i=1}^{N_t}$ (usually called the target domain) but has only few or no labelled data associated. Given a source
domain $\X_s = \{\x_i^s\}_{i=1}^{N_s}$ with associated labels $\Y_s =
\{\y_i^s\}_{i=1}^{N_s}$, one would like to leverage on this knowledge to train
a classifier
in the target domain. Unfortunately, direct use of the source information usually leads to poor results because of the discrepancy 
between source and target distributions. Among others, several works, e.g.~\cite{courty2017optimal}, use OT to perform this 
adaptation. However, in the case where the data do not belong to the same metric space ($\X_s\in \mathbb{R}^{N_s \times d}$ and $\X_t \in \mathbb{R}^{N_t \times d'}$ with $d\neq d'$), the problem is getting harder as 
domains probability distributions can not  be anymore compared or aligned in a straightforward way. This instance of the DA
problem, known as {\em Heterogeneous Domain Adaptation} (HDA), has
received less attention in the literature, partly due to the lack of appropriate divergence measures that can be used in such context.
State-of-the-art HDA methods include Canonical
Correlation Analysis~\cite{yeh2014heterogeneous} {and its kernelized version}
and a more recent { approach based on the} Gromov-Wasserstein
discrepancy~\cite{ijcai2018-412}. Usually, one considers a {\em semi-supervised}
variant of the problem,
where one has access to a small number $n_t$ of labelled samples per class in the target domain, because the {\em unsupervised} problem ($n_t=0$) 
is much more difficult. We investigate here the use of  \COOT\ for  both {\em semi-supervised} HDA, where one has access to a small number $n_t$ of labelled samples per class in the target domain and {\em unsupervised} HDA with $n_t=0$.

\paragraph{Solving HDA with \COOT} In order to solve the HDA problem, we compute $\COOT(\X_s,\X_t)$ between the two domains and use the $\GGs$
matrix providing a transport/correspondence between samples {(as illustrated in Figure \ref{fig:mnist_usps})} to estimate the
labels in the target domain via label propagation \cite{redko2018optimal}.
Assuming uniform sample weights and one-hot encoded labels, a class prediction $\hat{\Y}_t$ in the target domain samples can be obtained by {computing} $\hat{\Y}_t = \GGs \Y_s$. {When labelled target samples are available, we further prevent source samples to be mapped to target samples from a different class by adding a high cost in the cost matrix for every such source sample as suggested in \cite[Section 4.2]{courty2017optimal}. }

\paragraph{Datasets} We choose to test our method on the classical Caltech-Office dataset~\cite{saenko10}, which is dedicated to object recognition in images from several domains. Those domains exhibit  variability in term of presence/absence of background, lightning conditions, image quality, that as such induce distribution shifts between the domains. Among the available domains,  we select the following three:  {\em Amazon} (A), the {\em Caltech-256} image collection (C) and {\em Webcam} (W). Ten overlapping classes between the domains are used and two different deep feature representations of image in each domain are obtained using two different neural networks, namely: the Decaf ~\cite{donahue14} and GoogleNet~\cite{szegedy2015} neural network architectures. In both cases, we extract the image representations as  the activations of the last fully-connected layer, yielding respectively sparse 4096 and 1024 dimensional vectors. The heterogeneity comes from these two very different representations.

\begin{table}[t]
	\begin{center}

	\resizebox{0.9\columnwidth}{!}{
		\begin{tabular}{ccccccc}
		\toprule
			{Domains} & {No-adaptation baseline} & {CCA} & {KCCA} & {EGW} & {SGW} & {\COOT}\\
			\midrule
			C$\rightarrow$W & $69.12$$\pm 4.82$ & $11.47$$\pm 3.78$ & $66.76$$\pm 4.40$ & $11.35$$\pm 1.93$ & $\underline{ 78.88}$$\pm 3.90$ & $\bf 83.47$$\pm 2.60$\\
			W$\rightarrow$C & $83.00$$\pm 3.95$ & $19.59$$\pm 7.71$ & $76.76$$\pm 4.70$ & $11.00$$\pm 1.05$ & $\underline{ 92.41}$$\pm 2.18$ & $\bf 93.65$$\pm 1.80$\\
			W$\rightarrow$W & $82.18$$\pm 3.63$ & $14.76$$\pm 3.15$ & $78.94$$\pm 3.94$ & $10.18$$\pm 1.64$ & $\underline{ 93.12}$$\pm 3.14$ & $\bf 93.94$$\pm 1.84$\\
			W$\rightarrow$A & $84.29$$\pm 3.35$ & $17.00$$\pm 12.41$ & $78.94$$\pm 6.13$ & $7.24$$\pm 2.78$ & $\underline{ 93.41}$$\pm 2.18$ & $\bf 94.71$$\pm 1.49$\\
			A$\rightarrow$C & $\underline{ 83.71}$$\pm 1.82$ & $15.29$$\pm 3.88$ & $76.35$$\pm 4.07$ & $9.82$$\pm 1.37$ & $80.53$$\pm 6.80$ & $\bf 89.53$$\pm 2.34$\\
			A$\rightarrow$W & $81.88$$\pm 3.69$ & $12.59$$\pm 2.92$ & $81.41$$\pm 3.93$ & $12.65$$\pm 1.21$ & $\underline{ 87.18}$$\pm 5.23$ & $\bf 92.06$$\pm 1.73$\\
			A$\rightarrow$A & $\underline{ 84.18}$$\pm 3.45$ & $13.88$$\pm 2.88$ & $80.65$$\pm 3.03$ & $14.29$$\pm 4.23$ & $82.76$$\pm 6.63$ & $\bf 92.12$$\pm 1.79$\\
			C$\rightarrow$C & $67.47$$\pm 3.72$ & $13.59$$\pm 4.33$ & $60.76$$\pm 4.38$ & $11.71$$\pm 1.91$ & $\underline{ 77.59}$$\pm 4.90$ & $\bf 83.35$$\pm 2.31$\\
			C$\rightarrow$A & $66.18$$\pm 4.47$ & $13.71$$\pm 6.15$ & $63.35$$\pm 4.32$ & $11.82$$\pm 2.58$ & $\underline{ 75.94}$$\pm 5.58$ & $\bf 82.41$$\pm 2.79$\\\midrule
			\bf Mean & $78.00$$\pm 7.43$ & $14.65$$\pm 2.29$ & $73.77$$\pm 7.47$ & $11.12$$\pm 1.86$ & $\underline{ 84.65}$$\pm 6.62$ & $\bf 89.47$$\pm 4.74$\\
			\bf p-value & $<$.001 & $<$.001 & $<$.001 & $<$.001 & $<$.001 & -\\
		\bottomrule
		\end{tabular}
		}
	\end{center}
	\caption{{\bf Semi-supervised HDA} for $n_t=3$ from Decaf to GoogleNet task.}

	\label{tab:table_HDA_D_to_G_ns=3}
\end{table}

\paragraph{Competing methods and experimental settings} 
We evaluate COOT on {\em Amazon} (A), {\em Caltech-256} (C) and {\em Webcam} (W) domains from Caltech-Office dataset~\cite{saenko10} with 10 overlapping classes between the domains and two different deep feature representations obtained for images from each domain using the Decaf ~\cite{donahue14} and GoogleNet~\cite{szegedy2015} neural network architectures. In both cases, we extract the image representations as the activations of the last fully-connected layer, yielding respectively sparse 4096 and 1024 dimensional vectors. The heterogeneity comes from these two very
different representations. We consider 4 baselines: CCA, its kernalized version KCCA~\cite{yeh2014heterogeneous} with a Gaussian kernel which
width parameter is set to the inverse of the dimension of the input vector, EGW representing the entropic version of GW and
SGW~\cite{ijcai2018-412} that incorporates labelled target data into two regularization terms. For EGW and SGW, the entropic regularization term was set to $0.1$, and the two other
regularization hyperparameters for the semi-supervised case to $\lambda=10^{-5}$ and $\gamma=10^{-2}$
as done in \cite{ijcai2018-412}. We use \COOT\ with entropic regularization
on the feature mapping, with parameter $\epsilon_2=1$ in all experiments.  For
all OT methods, we use label propagation to obtain target labels as the maximum entry of $\hat{\Y}_t$ in each row. For all non-OT methods, classification
was conducted with a k-nn classifier with $k=3$. 
We run the experiment in a semi-supervised setting with $n_t=3$, \ie, $3$ samples per class were labelled in the
target domain. The baseline score is the result of classification by only
considering labelled samples in the target domain as the training set.  For each
pair of domains, we selected $20$ samples per class to form the learning sets.
We run this random selection process 10 times and consider the mean accuracy
of the different runs as a performance measure.

\paragraph{Results}
We first provide in Table~\ref{tab:table_HDA_D_to_G_ns=3} the results for the semi-supervised case where we perform adaptation from
Decaf to GoogleNet features. Note that we report the results in the opposite direction in Section \ref{sec:hda_supp} for $n_t \in \{0,1,3,5\}$. From it, we see that \COOT\ surpasses all the other state-of-the-art methods in terms of mean accuracy. This result is confirmed by a $p$-value lower than $0.001$ on a pairwise method comparison with \COOT\ in a Wilcoxon signed rank test. SGW provides the second best result, while CCA and EGW have a less than average performance. Finally, KCCA performs better than the two latter methods, but still fails most of the time to surpass the {no-adaptation baseline score given by a classifier learned on the available labelled target data}. Results for the unsupervised case can be found in
Table~\ref{tab:table_HDA_D_to_G_ns=0}. This setting is rarely considered in the
literature as unsupervised HDA is regarded as a very difficult problem. In this table, we do not provide scores for the no-adaptation baseline and SGW, as they require labelled data. 

\begin{table}
	\begin{center}
	\resizebox{.7\columnwidth}{!}{
		\begin{tabular}{ccccc}
		\toprule
			{Domains} & {CCA} & {KCCA} & {EGW} & {\COOT}\\
			\midrule
			C$\rightarrow$W & $14.20$$\pm 8.60$ & $\underline{ 21.30}$$\pm 15.64$ & $10.55$$\pm 1.97$ & $\bf 25.50$$\pm 11.76$\\
			W$\rightarrow$C & $13.35$$\pm 3.70$ & $\underline{ 18.60}$$\pm 9.44$ & $10.60$$\pm 0.94$ & $\bf 35.40$$\pm 14.61$\\
			W$\rightarrow$W & $10.95$$\pm 2.36$ & $\underline{ 13.25}$$\pm 6.34$ & $10.25$$\pm 2.26$ & $\bf 37.10$$\pm 14.57$\\
			W$\rightarrow$A & $14.25$$\pm 8.14$ & $\underline{ 23.00}$$\pm 22.95$ & $9.50$$\pm 2.47$ & $\bf 34.25$$\pm 13.03$\\
			A$\rightarrow$C & $11.40$$\pm 3.23$ & $\underline{ 11.50}$$\pm 9.23$ & $11.35$$\pm 1.38$ & $\bf 17.40$$\pm 8.86$\\
			A$\rightarrow$W & $19.65$$\pm 17.85$ & $\underline{ 28.35}$$\pm 26.13$ & $11.60$$\pm 1.30$ & $\bf 30.95$$\pm 18.19$\\
			A$\rightarrow$A & $11.75$$\pm 1.82$ & $\underline{ 14.20}$$\pm 4.78$ & $13.10$$\pm 2.35$ & $\bf 42.85$$\pm 17.65$\\
			C$\rightarrow$C & $12.00$$\pm 4.69$ & $\underline{ 14.95}$$\pm 6.79$ & $12.90$$\pm 1.46$ & $\bf 42.85$$\pm 18.44$\\
			C$\rightarrow$A & $15.35$$\pm 6.30$ & $\underline{ 23.35}$$\pm 17.61$ & $12.95$$\pm 2.63$ & $\bf 33.25$$\pm 15.93$\\\midrule
			\bf Mean & $13.66$$\pm 2.55$ & $\underline{ 18.72}$$\pm 5.33$ & $11.42$$\pm 1.24$ & $\bf 33.28$$\pm 7.61$\\
			\bf p-value & $<$.001 & $<$.001 & $<$.001 & -\\
		\bottomrule
		\end{tabular}
		}
	\end{center}
		\caption{{\bf Unsupervised HDA} for $n_t=0$ from Decaf to GoogleNet task.}
		\label{tab:table_HDA_D_to_G_ns=0}
\end{table}
As one can expect, most of the methods fail in obtaining good
classification accuracies in this setting, despite having access to discriminant 
feature representations. Yet, \COOT\ succeeds in providing a meaningful mapping
in some cases. {The overall superior performance of \COOT\ highlights its strengths and underlines the limits of other HDA methods. First, \COOT\ does not depend on approximating empirical quantities from the data, contrary to CCA and KCCA that rely on the estimation of the cross-covariance matrix that is known to be flawed for high-dimensional data with few samples \cite{SongSRH16}. Second, \COOT\ takes into account the features of the raw data that are more informative than the pairwise distances used in EGW. Finally, \COOT\ avoids the sign invariance issue discussed previously that hinders GW's capability to recover classes without supervision as illustrated for the MNIST-USPS problem before.}

\subsection{Co-clustering and data summarization}
\label{sec:co_clustering}
While clustering methods present an important discovery tool for data analysis, one of their main limations is to completely discard the potential relationships that may exist between the features that describe the data samples. For instance, in recommendation systems, where each user is described in terms of his or her preferences for some product, clustering algorithms may benefit from the knowledge about the correlation between different products revealing their probability of being recommended to the same users. This idea is the cornerstone of \textit{co-clustering} \cite{hartigan-direct-clustering-data-1972} where the goal is to perform clustering of both samples and features simultaneously. More precisely given a data matrix $\X \in \mathbb{R}^{n\times d}$ and  the number of samples (rows) and features (columns) clusters denoted by $g\leq n$ and $m\leq d$, respectively, we seek to find $\X_c \in \mathbb{R}^{g\times m}$ that summarizes $\X$ in the best way possible. 

\paragraph{\COOT-clustering} 
We look for $\X_c$ which is as close as possible to the original $\X$ \textit{w.r.t.} \COOT\ by solving: 
\begin{equation}
\min_{\X_{c}} \COOT(\X,\X_{c}) =\min_{\GGs,\GGv, \X_{c}} \froeb{\L(\X,\X_{c}) \otimes \GGs}{\GGv}
\end{equation} 
with potentially entropic regularization.
More precisely, we set $\w,\w',\v,\v'$ as uniform, initialize $\X_{c}$ with random values and apply the BCD algorithm over ($\GGs,\GGv,\X_{c}$) by alternating between the following steps: 
\begin{enumerate}
	\item Obtain $\GGs$ and $\GGv$ by solving $\COOT(\X,\X_{c})$  
	\item Set $\X_{c}$ to  $gm\GG^{s\top} \X \GGv$. 
\end{enumerate}
{This second step of the procedure is a least-square
estimation when $L=|\cdot|^2$ and corresponds to minimizing the \COOT\ objective
\emph{w.r.t..} $\X_c$. In practice, we observed
that few iterations of this procedure are enough to ensure the convergence. Once solved, we use the soft assignments provided by coupling matrices $\GGs \in \mathbb{R}^{n\times g},\GGv \in \mathbb{R}^{d\times m}$ to partition data points and features to clusters by taking the index of the maximum element in each row of $\GGs$ and $\GGv$, respectively.

\paragraph{Simulated data}
We follow \cite{LaclauRMBB17} where four scenarios with different number of co-clusters, degrees of separation and sizes were considered (for details, see the supplementary materials). {We choose to evaluate \COOT\ on simulated data as it provides us with the ground-truth for feature clusters that are often unavailable for real-world data sets.} As in \cite{LaclauRMBB17}, we use the same co-clustering baselines including
ITCC \cite{Dhillon:2003:IC:956750.956764}, Double K-Means (DKM)~\cite{rocci_08},
Orthogonal Nonnegative Matrix Tri-Factorizations (ONTMF) \cite{ding_06}, the
Gaussian Latent Block Models (GLBM) \cite{NADI08CI} and Residual Bayesian
Co-Clustering (RBC) \cite{shan_10} as well as the K-means and NMF run on both
modes of the data matrix, as clustering baseline. The performance of all methods
is measured using the co-clustering error (CCE) defined as
follows \cite{Patrikainen06}:
\begin{equation}
\text{CCE}((\bz,\bw),(\hat{\bz},\hat{\bw}))=e(\bz,\hat{\bz})+
e(\bw,\hat{\bw})-e(\bz,\hat{\bz})\times e(\bw,\hat{\bw})
\end{equation}
where $\hat{\bz}$ and $\hat{\bw}$ are the partitions of samples and features
estimated by the algorithm; $\bz$ and $\bw$ are the true partitions and
$e(\bz,\hat{\bz})$ (resp. $e(\bw,\hat{\bw})$) denotes the error rate, i.e., the
proportion of misclassified instances (resp. features). 

\begin{table}[t]
\centering
\resizebox{0.55\textwidth}{!}{
\begin{tabular}{lccccl}
\hline
Data set&$n\times d$&$g\times m$&Overlapping&Proportions\\
\hline
D1&$600\times300$&$3\times 3$&[+]&Equal\\
D2&$600\times300$&$3\times 3$&[+]&Unequal\\
D3&$300\times200$&$2\times 4$&[++]&Equal\\
D4&$300\times300$&$5\times4$&[++]&Unequal\\
\hline
\end{tabular}
}
\caption{\label{tab:data_description}Size ($n\times d$), number of co-clusters ($g\times m$), degree of overlapping ([+] for well-separated and [++] for ill-separated co-clusters) and the proportions of co-clusters for simulated data sets.}
\end{table} 

For all configurations, we generate 100 data sets and present the mean and standard deviation of the CCE
over all sets for all baselines in Table \ref{tab:result}. Table \ref{tab:data_description} below summarizes the characteristics of the simulated data sets used in our experiment.

Based on these results, we see that our algorithm outperforms all the other baselines on D1, D2
and D4 data sets, while being behind (CCOT-GW) proposed by \cite{LaclauRMBB17} on
D3. This result is rather strong as our method relies on the original data
matrix, while (CCOT-GW) relies on its kernel
representation and thus benefits from the non-linear information captured by
it. Finally, we note that while both competing methods rely on OT, they remain very different as (CCOT-GW) approach is based on detecting the positions and the number of jumps in the scaling vectors of GW entropic regularized solution, while our method relies on coupling matrices to obtain the partitions.

\begin{table*}[t]

\begin{center}
\resizebox{1\linewidth}{!}{
\begin{tabular}{lcccccccccc}
\hline
\multirow{2}{*}{Data set} &\multicolumn{10}{c}{Algorithms}\\
\cline{2-11}
  & K-means&NMF&DKM&Tri-NMF&GLBM&ITCC&RBC&CCOT&CCOT-GW & COOT\\
\hline
D1&$.018\pm{.003}$&$.042\pm{.037}$&$.025\pm{.048}$&$.082\pm{.063}$&$.021\pm{.011}$&$.021\pm{.001}$&$.017\pm{.045}$&$.018\pm{.013}$&$.004\pm{.002}$ & $\mathbf{0}$ \\
D2&$.072\pm{.044}$&$.083\pm{.063}$&$.038\pm{.000}$&$.052\pm{.065}$&$.032\pm{.041}$&$.047\pm{.042}$&$.039\pm{.052}$&$.023\pm{.036}$&$.011\pm{.056}$ & $\mathbf{.009\pm{0.04}}$\\
D3&--&--&$.310\pm{.000}$&--&$.262\pm{.022}$&$.241\pm{.031}$&--&$.031\pm{.027}$&$\mathbf{.008\pm{.001}}$ & $.04\pm{.05}$\\
D4&$.126\pm{.038}$&--&$.145\pm{.082}$&--&$.115\pm{.047}$&$.121\pm{.075}$&$.102\pm{.071}$&$.093\pm{.032}$&$.079\pm{.031}$ & $\mathbf{0.068\pm{0.04}}$\\
\hline
\end{tabular}
}
\caption{\label{tab:result} Mean ($\pm$ standard-deviation) of the co-clustering error (CCE) obtained for all configurations. ``-" indicates that the algorithm cannot find a partition with the requested number of co-clusters. All the baselines results (first 9 columns) are from \cite{LaclauRMBB17}.
}
\end{center}
\end{table*}

\paragraph{Olivetti Face dataset} 
As a first application of \COOT\ for the
co-clustering problem on real data, we propose to run the algorithm on the well
known Olivetti faces dataset \cite{samaria1994parameterisation}. 

\begin{figure}
  \centering
  \vspace{-4mm}
  \includegraphics[width=\linewidth]{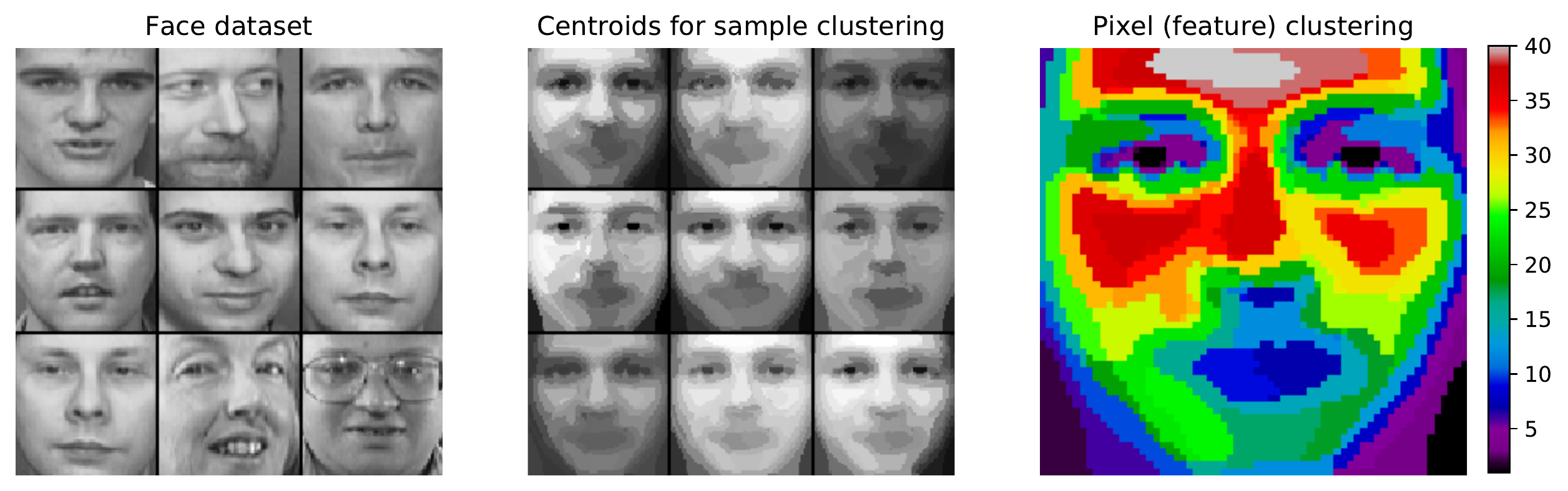}
          \caption{Co-clustering with \COOT\ on the Olivetti faces dataset. \textbf{(left)} Example images from the dataset, \textbf{(center)} centroids estimated by \COOT\, \textbf{(right)} clustering of the pixels estimated by \COOT\ where each color represents a cluster.  }
  \label{fig:coclust_faces}
\end{figure}
We take 400
images normalized between 0 and 1 and run our algorithm with $g=9$ image
clusters and $m=40$ feature (pixel) clusters. {As before, we consider the empirical distributions supported on images and features, respectively.}  The resulting reconstructed image's
clusters are given in Figure
\ref{fig:coclust_faces} and the pixel clusters are illustrated in its rightmost part. We can see that despite the high variability in the data set, we still
manage to recover detailed centroids, {whereas} L2-based clustering {such as standard NMF or k-means based on $\ell_2$ norm cost function are known to
provide blurry estimates in this case. Finally, as in the MNIST-USPS example, \COOT\ recovers spatially
localized pixel clusters with no prior information about the pixel relations.

\paragraph{MovieLens} We now evaluate our approach on the benchmark \MovieL-100K\footnote{https://grouplens.org/datasets/movielens/100k/} data set that provides 100,000 user-movie ratings, on a scale of one to five, collected from 943 users on 1682 movies. The main goal of our algorithm here is to summarize the initial data matrix so that $\X_c$ reveals the blocks (co-clusters) of movies and users that share similar tastes. We set the number of user and film clusters to $g=10$ and $m=20$, respectively as in \cite{Banerjee:2007:GME:1314498.1314563}. 

\begin{table}[t!]
\centering
\resizebox{0.7\linewidth}{!}{
\begin{tabular}{cc}
\hline
M1 &M20\\
\hline
Shawshank Redemption (1994)& Police Story 4: Project S (Chao ji ji hua) (1993)\\
Schindler's List (1993) & Eye of Vichy, The (Oeil de Vichy, L') (1993) \\
Casablanca (1942) & Promise, The (Versprechen, Das) (1994)\\
Rear Window (1954) & To Cross the Rubicon (1991)\\
Usual Suspects, The (1995) & Daens (1992)\\
\hline
\end{tabular}
}
\caption{\label{tab:top5} Top 5 of movies in clusters M1 and M20. Average rating of the top 5 rated movies in M1 is 4.42, while for the M20 it is 1.}
\end{table}
The obtained results provide the first movie cluster consisting of
films with high ratings (3.92 on average), while the last movie cluster includes
movies with very low ratings (1.92 on average). Among those, we show the 5
best/worst rated movies in those two clusters in Table \ref{tab:top5}. Overall, our algorithm manages to find a coherent co-clustering structure in \MovieL-100K and obtains results similar to those provided in \cite{LaclauRMBB17,Banerjee:2007:GME:1314498.1314563}.

\section{Conclusion}

In this chapter, we presented a novel variant of the optimal transport problem which aims at comparing 
distributions supported in different spaces. To this end, two optimal transport maps, one acting on the sample space, 
and the other on the feature space, are optimized to connect the two heterogenous distributions. We show that this 
novel problem has connections with bilinear assignment and provide algorithms to solve it. We demonstrate its
usefulness and versatility on two difficult machine learning problems: heterogeneous domain adaptation and 
co-clustering/data summarization, where promising results were obtained. 

Numerous follow-up of this work are expected. Beyond the potential applications of the method in various contexts, such 
as {\em e.g.} statistical matching, data analysis or even losses in deep learning settings, one immediate and intriguing question 
lies into the generalization of this framework in the continuous setting, and  the potential connections to duality theory. This might lead 
to stochastic optimization schemes  enabling large scale solvers for this problem.


\chapter{Proofs of claims and additional results}
\minitoc
\label{cha:proofs}

\section{Proofs and additional results of Chapter \ref{cha:fgw}}
\label{sec:proofs_chap_fgw}

This section contains all the proofs of the claims and additional results of the Chapter \ref{cha:fgw}. In the following we will denote by $H_{q}$ the Wasserstein loss and by $J_{q}$ the Gromov-Wasserstein loss. More precisely, with notation of Chapter \ref{cha:fgw}:
\begin{equation}
H_{q}(\pi)=\int d(a,b)^{q} \dr \pi(a,b)
\end{equation}

\begin{equation}
J_{q}(d_\Xcal,d_\Ycal,\pi)=\int \int L(x,y,x',y)^{q} \dr \pi(x,y) \dr \pi(x',y')=\int \int |d_{\Xcal}(x,x')-d_{\Ycal}(y,y')|^{q} \dr \pi(x,y) \dr \pi(x',y')
\end{equation}

\begin{equation}
\lossfgw(\pi)=\int \int \big((1-\alpha) d(a,b)^{q} +\alpha |d_{\Xcal}(x,x')-d_{\Ycal}(y,y')|^{q} \big)^{p} \dr\pi((x,a),(y,b))\dr\pi((x',a'),(y',b'))
\end{equation}
We note $P_{i}\#\pi$ the projection on the marginal $i$ of $\pi$.

\subsection{Additional results -- Comparison with W and GW \label{sec:compa}}
\paragraph{Cross validation results}

During the nested cross validation, we divided the dataset into 10 and use 9 folds for training, where $\alpha$ is chosen within $[0,1]$ \textit{via} a 10-CV cross-validation, 1 fold for testing, with the best value of $\alpha$ (with the best average accuracy on the 10-CV) previously selected.  The experiment is repeated 10 times for each dataset except for MUTAG and PTC where it is repeated 50 times. Table \ref{discretetable} and \ref{continuoustable} report the average number of time $\alpha$ was chose within $]0,...1[$ without $0$ and $1$ corresponding to the Wasserstein and Gromov-Wasserstein distances respectively. Results suggests that both structure and feature pieces of
information are necessary as $\alpha$ is consistently selected inside $]0,...1[$ except for PTC and COX2.

\begin{table}[h]
    \caption{Percentage of $\alpha$ chosen in $]0,...,1[$ compared to $\{0,1\}$ for discrete labeled graphs \label{discretetable}}
    \vspace{1.5mm}
\begin{center}
    \resizebox{0.5\columnwidth}{!}{
\begin{sc}
    \setlength{\tabcolsep}{4pt}
\begin{tabular}{llll}
\toprule
{Discrete attr.} &          MUTAG &          NCI1 &           PTC \\
\midrule
\midrule
FGW raw sp               &  100\% &     100\% &  98\% \\
FGW wl h=2 sp       &   100\% &          100\% & 88\% \\
FGW wl  h=4 sp       &   100\% &          100\% &  88\% \\
\midrule
\end{tabular}
\quad
\end{sc}}
\end{center}
\end{table}

\begin{table}[h]
    \caption{Percentage of $\alpha$ chosen in $]0,...,1[$ compared to $\{0,1\}$ for vector attributed graphs \label{continuoustable}}
    \vspace{1.5mm}
    \begin{center}
    \resizebox{0.7\linewidth}{!}{
\begin{sc}
    \setlength{\tabcolsep}{4pt}
\begin{tabular}{lllllll}
\toprule
{Vector attributes} &                   BZR &          COX2 &     CUNEIFORM &       ENZYMES &        PROTEIN &      SYNTHETIC \\
\midrule
\midrule
\small{FGW sp}             &  100 \% &  90\% &  100\% &  100\% &  100\% &  100\% \\
\midrule
\end{tabular}
\end{sc}}
\end{center}

\end{table}

\paragraph{Nested CV results} We report in tables \ref{tab:wgw:vec} and \ref{tab:wgw:disc} the average classification accuracies of the nested classification procedure by taking $W$ and $GW$ instead of $FGW$ (\textit{i.e} by taking $\alpha=0,1$). Best result for each dataset is in bold. A (*) is added when best score does not yield to a significative improvement compared to the second best score. The significance is based on a Wilcoxon signed rank test between the best method and the second one.

Results illustrates that $FGW$ encompasses the two cases of $W$ and $GW$, as scores of $FGW$ are usually greater or equal on every dataset than scores of both $W$ and $GW$ and when it is not the case the difference is not statistically significant. 

\begin{table}[H]
    \caption{Average classification accuracy on the graph datasets with discrete attributes.\label{tab:wgw:disc}}
    \vspace{1.5mm}
\begin{center}
    \resizebox{0.6\linewidth}{!}{
\begin{sc}
    \setlength{\tabcolsep}{4pt}
\begin{tabular}{llll}
\toprule
{Discrete attr.} &          MUTAG &          NCI1 &           PTC-MR \\
\midrule
\midrule
FGW raw sp               &  83.26$\pm$10.30 &     72.82$\pm$1.46 &  55.71$\pm$6.74 \\
FGW wl h=2 sp       &   86.42$\pm$7.81 &          85.82$\pm$1.16 &  {63.20$\pm$7.68} \\
FGW wl  h=4 sp       &   \textbf{88.42$\pm$5.67} &          \textbf{86.42$\pm$1.63}* & 65.31$\pm$7.90 \\
\midrule
W raw sp       &  79.36$\pm$3.49  &   70.5$\pm$4.63        &  54.79$\pm$5.76 \\
W wl h=2 sp       &  87.78$\pm$8.64  &     85.83$\pm$1.75      & 63.90$\pm$7.66  \\
W wl  h=4 sp       & 87.15$\pm$8.23   &     86.42$\pm$1.64     & \textbf{66.28$\pm$6.95}* \\
\midrule
GW sp &82.73$\pm$9.59 & 73.40$\pm2.80$ &  54.45$\pm$ 6.89\\
\bottomrule
\end{tabular}
\quad
\end{sc}}
\end{center}
\end{table}

\begin{table}[H]
    \caption{Average classification accuracy on the graph datasets with vector attributes.\label{tab:wgw:vec}}
     \vspace{1.5mm}
\begin{center}
    \begin{center}
    \resizebox{.9\textwidth}{!}{
\begin{sc}
    \setlength{\tabcolsep}{4pt}
\begin{tabular}{lllllll}
\toprule
{Vector attributes} &                   BZR &          COX2 &     CUNEIFORM &       ENZYMES &        PROTEIN &      SYNTHETIC \\
\midrule
\midrule
\small{FGW sp}             &  85.12$\pm$4.15 &  \textbf{77.23$\pm$4.86}* &  \textbf{76.67$\pm$7.04} &  71.00$\pm$6.76 &   74.55$\pm$2.74 &  \textbf{100.00$\pm$0.00} \\
\midrule
\small{W}               &  \textbf{85.36$\pm$4.87}* & 77.23$\pm$3.16  &  61.48$\pm$10.23 &\textbf{71.16$\pm$6.32}*  &  \textbf{75.98$\pm$ 1.97}*  &  34.07$\pm$11.33 \\
\midrule
\small{GW sp}                 &  82.92$\pm$6.72  & 77.65$\pm$5.88  & 50.66$\pm$8.91  & 23.66$\pm$3.63  &  71.96$\pm$ 2.40  &  41.66$\pm$4.28  \\
\bottomrule
\end{tabular}
\end{sc}}
\end{center}
\end{center}
\end{table}

\paragraph{Timings} In this paragraph we provide some timings for the discrete attributed datasets. Table \ref{tab:wgw:timings} displays the average timing for computing $FGW$ between two pair of graphs. 

\begin{table}[H]
    \caption{Average timings for the computation of $FGW$ between two pairs of graph \label{tab:wgw:timings}}
        \vspace{1.5mm}
\begin{center}
    \resizebox{0.5\linewidth}{!}{
\begin{sc}
    \setlength{\tabcolsep}{4pt}
\begin{tabular}{llll}
\toprule
{Discrete attr.} &          MUTAG &          NCI1 &           PTC-MR \\
\midrule
\midrule
FGW               &  2.5 ms  &  7.3 ms  & 3.7 ms \\
\bottomrule
\end{tabular}
\quad
\end{sc}}
\end{center}
\end{table}

\subsection{Proof of Proposition \ref{fgwcomparetogwandw} -- Comparison between FGW, GW and W}
\label{proof:prop1}

We recall the proposition:
\begin{prop*}{Comparison between $FGW$, $GW$ and $W$.}
\begin{itemize}
\item The following inequalities hold:
\begin{equation}
\label{wassinequality}
\fgwdistance_{\alpha,p,q}(\mu,\nu) \geq (1-\alpha)\wass_{pq}(\mu_{A},\nu_{B})^{q}
\end{equation}
\begin{equation}
\label{gromovinequality}
\fgwdistance_{\alpha,p,q}(\mu,\nu) \geq \alpha \gw_{pq}(\mu_{X},\nu_{Y})^{q}
\end{equation}
\item Let us suppose that the structure spaces $(\Xcal,d_{\Xcal})$,$(\Ycal,d_{\Ycal})$ are part of a single ground space $(\mathcal{Z},d_{\mathcal{Z}})$ (\textit{i.e.} $\Xcal,\Ycal \subset \mathcal{Z}$ and $d_{\Xcal}=d_{\Ycal}=d_{\mathcal{Z}}$). We consider the Wasserstein distance between $\mu$ and $\nu$ for the distance on $\mathcal{Z} \times \Omega$ : $\tilde{d}((x,a),(y,b))=(1-\alpha)d(a,b)+\alpha d_{\mathcal{Z}}(x,y)$. Then:
\begin{equation}
\label{samespace}
\fgwdistance_{\alpha,p,1}(\mu,\nu)(\mu,\nu) \leq 2\wass_{p}(\mu,\nu).
\end{equation}
\end{itemize}
\end{prop*}

\begin{proof}
For the two inequalities \eqref{wassinequality} and \eqref{gromovinequality} let $\pi$ be an optimal coupling for the Fused Gromov-Wasserstein distance between $\mu$ and $\nu$. Clearly: 

\begin{equation*}
\begin{split}
\fgwdistance_{\alpha,p,q}(\mu,\nu) &=\Big(\int\limits_{(\Xcal \times \Omega \times \Ycal \times \Omega)^{2}}\big((1-\alpha) d(a,b)^{q} +\alpha L(x,y,x',y')^{q}\big)^{p} \,\dr\pi((x,a),(y,b))\dr\pi((x',a'),(y',b'))\Big)^{\frac{1}{p}} \\
&\geq \Big(\int\limits_{\Xcal \times \Omega \times \Ycal \times \Omega} (1-\alpha)^{p} d(a,b)^{pq} \,\dr\pi((x,a),(y,b))\Big)^{\frac{1}{p}}= (1-\alpha) \Big( \int\limits_{\Omega \times \Omega}d(a,b)^{pq}\dr P_{2,4}\#\pi(a,b)\Big)^{\frac{1}{p}}
\end{split}
\end{equation*}
Since $\pi \in \couplingset(\mu,\nu)$ the coupling $P_{2,4}\#\pi$ is in $\couplingset(\mu_{A},\nu_{B})$. So by suboptimality:
$$\fgwdistance_{\alpha,p,q}(\mu,\nu){\geq} (1-\alpha) (\wass_{pq}(\mu_{A},\nu_{B}))^{q}$$
which proves equation (\ref{wassinequality}). Same reasoning is used for equation (\ref{gromovinequality}).

For the last inequality \eqref{samespace} let $\pi \in \couplingset(\mu,\nu)$ be any admissible coupling. By suboptimality:
{\small
\begin{equation*}
\begin{split}
&\fgwdistance_{\alpha,p,1}(\mu,\nu) \leq \Big( \int\limits_{(\Xcal \times \Omega \times \Ycal \times \Omega)^{2}}\big((1{-}\alpha) d(a,b) +\alpha |d_{\mathcal{Z}}(x,x')-d_{\mathcal{Z}}(y,y')|\big)^{p}\dr\pi((x,a),(y,b))\dr\pi((x',a'),(y',b'))\Big)^{\frac{1}{p}} \\
&\stackrel{(*)}{\leq}\Big(\int\limits_{(\Xcal \times \Omega \times \Ycal \times \Omega)^{2}}\big((1{-}\alpha) d(a,b) +\alpha d_{\mathcal{Z}}(x,y)+\alpha d_{\mathcal{Z}}(x',y')\big)^{p}\dr\pi((x,a),(y,b))\dr\pi((x',a'),(y',b'))\Big)^{\frac{1}{p}} \\
&\leq\Big(\int\limits_{(\Xcal \times \Omega \times \Ycal \times \Omega)^{2}}\big((1{-}\alpha) d(a,b) +\alpha d_{\mathcal{Z}}(x,y)+(1{-}\alpha) d(a',b')+\alpha d_{\mathcal{Z}}(x',y')\big)^{p}\dr\pi((x,a),(y,b))\dr\pi((x',a'),(y',b'))\Big)^{\frac{1}{p}} \\ 
&\stackrel{(**)}{\leq}2\Big( \int\limits_{\Xcal \times \Omega \times \Ycal \times \Omega}\big((1{-}\alpha) d(a,b) +\alpha d_{\mathcal{Z}}(x,y)\big)^{p}\dr\pi((x,a),(y,b))\Big)^{\frac{1}{p}} 
\end{split}
\end{equation*}
}

(*) is the triangle inequality of $d_{\mathcal{Z}}$ and (**) Minkowski inequality. Since this inequality is true for any admissible coupling $\pi$ we can apply it with the optimal coupling for the Wasserstein distance defined in the proposition and the claim follows.

\end{proof}

\subsection{Proof of Theorem \ref{metrictheo} -- Metric properties of FGW}
\label{proof:theo1}

We recall the theorem:
\begin{theo*}[\tv{Metric properties}]
Let $p,q\geq 1$, $\alpha \in ]0,1[$ and $(\mu,\nu) \in \mathbb{S}_{pq}(\Omega) \times \mathbb{S}_{pq}(\Omega)$. The functional $\pi \rightarrow \lossfgw(\pi)$ always achieves an infimum $\pi^{*}$ in $\Pi(\mu,\nu)$ \emph{s.t.} $\fgwdistance_{\alpha,p,q}(\mu,\nu)=\lossfgw(\pi^{*})<+\infty$. Moreover:
\begin{itemize}
\item[$\bullet$] $\fgwdistance_{\alpha,p,q}$ is symmetric and, for $q=1$, satisfies the triangle inequality. For $q\geq 2$, the triangular inequality is relaxed by a factor $2^{q-1}$.
\item[$\bullet$] For $\alpha \in ]0,1[$, $\fgwdistance_{\alpha,p,q}(\mu,\nu)=0$ if an only if there exists a bijective function $ \phi=(\phi_1,\phi_2): \tv{\supp(\mu)} \rightarrow \tv{\supp(\nu)}$ such that:
\begin{equation*}
\phi\#\mu=\nu
\end{equation*}
\begin{equation*}
\forall (x,a) \in \tv{\text{supp}(\mu)} \ , \phi_{2}(x,a)=a
\end{equation*}
\begin{equation*}
\forall (x,a),(x',a') \in \tv{\text{supp}(\mu)^{2}}, \ d_{\Xcal}(x,x')=d_{\Ycal}(\phi_{1}(x,a),\phi_{1}(x',a'))
\end{equation*}
\item[$\bullet$] If $(\mu,\nu)$ are generalized labeled graphs then $\fgwdistance_{\alpha,p,q}(\mu,\nu)=0$ if and only if $(\Xcal \times \Omega,d_\Xcal,\mu)$ and $(\Ycal \times \Omega,d_\Ycal,\nu)$ are (II)-strongly isomorphic.
\end{itemize}
\end{theo*}

We propose to prove the theorem point by point: first the existence, then the the triangle inequality statement and finally the equality relation. 

\begin{prop}[Existence of an optimal coupling for the $FGW$ distance.]

For $p,q \geq 1$, $\pi \rightarrow \lossfgw(\pi)$ always achieves a infimum $\pi^{*}$ in $\couplingset(\mu,\nu)$ such that $\fgwdistance_{\alpha,p,q}(\mu,\nu)=\lossfgw(\pi^{*})<+\infty$.

\end{prop}

\begin{proof}
Since $\Xcal \times \Omega$ and $\Ycal \times \Omega$ are Polish spaces we known that $\Pi(\mu,\nu)\subset \mathcal{P}(\Xcal \times \Omega \times \Ycal \times \Omega)$ is compact (Theorem 1.7 in \cite{San15a}), so by applying Weierstrass theorem we can conclude that the infimum is attained at some $\pi^{*}\in \Pi(\mu,\nu)$ if $\pi \rightarrow \lossfgw(\pi)$ is l.s.c.

We will use Lemma \ref{lsc_on_measure} to prove that the functionnal is l.s.c. on $\Pi(\mu,\nu)$. If we consider $W=\Xcal \times \Omega \times \Ycal \times \Omega$ which is a a metric space endowed with the distance $d_\Xcal\otimes d \otimes d_\Ycal \otimes d$ and $f((w=(x,a,y,b),w'=(x',a',y',b'))=((1-\alpha)d(a,b)^{q}+\alpha L(x,y,x',y')^{q})^{p}$ then $f$ is l.s.c. by continuity of $d$, $d_\Xcal$ and $d_\Ycal$. With the previous reasoning we can conclude that the infimum is attained.

Finally finiteness come from: 
\begin{equation}
\begin{split}
&\int\limits_{(\tv{\Xcal \times \Omega} \times \tv{\Ycal \times \Omega})^{2}} \big((1-\alpha) d(a,b)^{q} +\alpha L(x,y,x',y')^{q} \big)^{p} \,d\pi((x,a),(y,b))\,d\pi((x',a'),(y',b')) \\
&\stackrel{*}{\leq} \int 2^{p-1} (1-\alpha)  d(a,b)^{qp} d\mu_A(a) d\nu(b)+ \int 2^{p-1} \alpha L(x,y,x',y')^{qp} d\mu_X(x)d\mu_X(x') d\nu_Y(y')d\nu_Y(y') \\
&\stackrel{**}{<} +\infty
\end{split}
\end{equation}
where in (*) we used equation \eqref{holder} in Memo \ref{memo:holder} and in (**) that $\mu,\nu$ are in $\mathbb{S}_{pq}(\Omega)$.
\end{proof}

\begin{prop}[Symmetry and triangle inequality.]
\label{triangleprop}
$\fgwdistance_{\alpha,p,q}$ is symmetric and for $q=1$ satisfies the triangle inequality. For $q\geq 2$ the triangular inequality is relaxed by a factor $2^{q-1}$ 
\end{prop}
To prove this result we will use the following lemma:
\begin{lemma}
Let $(\Xcal \times \Omega,d_{\Xcal},\mu),(\Ycal \times \Omega,d_{\Ycal},\beta),(\mathcal{Z} \times \Omega,d_{\mathcal{Z}},\nu) \in \mathbb{S}(\Omega)^{3}$.
For $(x,a),(x',a') \in (\Xcal \times \Omega)^{2}$, $(y,b),(y',b') \in (\Ycal \times \Omega)^{2}$ and $(z,c),(z',c') \in (\mathcal{Z} \times \Omega)^{2}$ we have:
\begin{equation}
\label{ineqL}
L(x,z,x',z')^{q} \leq 2^{q-1}(L(x,y,x',y')^{q}+L(y,z,y',z')^{q})
\end{equation}
\begin{equation}
\label{ineqD}
d(a,c)^{q} \leq 2^{q-1}(d(a,b)^{q}+d(b,c)^{q})
\end{equation}

\end{lemma}

\begin{proof}
Direct consequence of equation \eqref{holder} in Memo \ref{memo:holder} and triangle inequalities of $d,d_{\Xcal},d_{\Ycal},d_{\mathcal{Z}}$.
\end{proof}

\begin{proof}[Proof of Proposition \ref{triangleprop}]
To prove the triangle inequality of $\fgwdistance_{\alpha,p,q}$ distance for arbitrary measures we will use the Gluing lemma which stresses the existence of couplings with a prescribed structure. Let $(\tv{\Xcal \times \Omega},d_{\Xcal},\mu),(\tv{\Ycal \times \Omega},d_{\Ycal},\beta),(\mathcal{Z} \times \Omega,d_{\mathcal{Z}},\nu) \in \mathbb{S}(\Omega)^{3}$.

Let $\pi_{1} \in \couplingset(\mu,\beta)$ and $\pi_{2} \in \couplingset(\beta,\nu)$ be optimal transportation plans for the Fused Gromov-Wasserstein distance between $\mu$, $\beta$ and $\beta$, $\nu$ respectively. By the Gluing Lemma (see \cite{Villani} and Lemma 5.3.2 in \cite{ambrosio2005gradient}) there exists a probability measure $\pi \in \Pcal\big((\tv{\Xcal \times \Omega}) \times (\tv{\Ycal \times \Omega}) \times (\mathcal{Z} \times \Omega)\big)$ with marginals $\pi_{1}$ on $(\tv{\Xcal \times \Omega}) \times (\tv{\Ycal \times \Omega})$ and $\pi_{2}$ on $(\tv{\Ycal \times \Omega}) \times (\mathcal{Z} \times \Omega)$. Let $\pi_{3}$ be the marginal of $\pi$ on $(\tv{\Xcal \times \Omega})\times (\mathcal{Z} \times \Omega)$. By construction $\pi_{3} \in \Pi(\mu,\nu)$. So by suboptimality of $\pi_{3}$:
{\small
\begin{equation*}
\begin{split}
&\fgwdistance_{\alpha,p,q}(d_{\Xcal},d_{\mathcal{Z}},\mu,\nu) \leq\Big(\int\limits_{(\tv{\Xcal \times \Omega} \times \mathcal{Z} \times \Omega)^{2}}\big((1-\alpha) d(a,c)^{q} +\alpha L(x,z,x',z')^{q} \big)^{p} d\pi_{3}((x,a),(z,c))d\pi_{3}((x',a'),(z',c'))\Big)^{\frac{1}{p}} \\ 
&=\Big(\int\limits_{(\tv{\Xcal \times \Omega} \times \tv{\Ycal \times \Omega} \times \mathcal{Z} \times \Omega)^{2}}\big((1-\alpha) d(a,c)^{q} +\alpha L(x,z,x',z')^{q} \big)^{p} d\pi((x,a),(y,b),(z,c))d\pi((x',a'),(y',b'),(z',c'))\Big)^{\frac{1}{p}} \\ 
&\stackrel{(*)}{\leq}2^{q-1} \Big(\int\limits_{(\tv{\Xcal \times \Omega} \times \tv{\Ycal \times \Omega} \times \mathcal{Z} \times \Omega)^{2}}\big((1-\alpha) d(a,b)^{q}+(1-\alpha)d(b,c)^{q} +\alpha L(x,y,x',y')^{q}+\alpha L(y,z,y',z')^{q} \big)^{p} \\
& d\pi((x,a),(y,b),(z,c))d\pi((x',a'),(y',b'),(z',c'))\Big)^{\frac{1}{p}} \\
&\stackrel{(**)}{\leq}2^{q-1} \bigg(\bigg(\int\limits_{(\tv{\Xcal \times \Omega} \times \tv{\Ycal \times \Omega} \times \mathcal{Z} \times \Omega)^{2}}\big((1-\alpha) d(a,b)^{q} +\alpha L(x,y,x',y')^{q}\big)^{p} d\pi((x,a),(y,b),(z,c))d\pi((x',a'),(y',b'),(z',c'))\Big)^{\frac{1}{p}} \\ 
&+ \Big(\int\limits_{(\tv{\Xcal \times \Omega} \times \tv{\Ycal \times \Omega} \times \mathcal{Z} \times \Omega)^{2}}\big((1-\alpha) d(b,c)^{q} +\alpha L(y,z,y',z')^{q}\big)^{p} d\pi((x,a),(y,b),(z,c))d\pi((x',a'),(y',b'),(z',c'))\Big)^{\frac{1}{p}} \Big) \\ 
&=2^{q-1} \bigg(\bigg(\int\limits_{(\tv{\Xcal \times \Omega} \times \tv{\Ycal \times \Omega} )^{2}}\big((1-\alpha) d(a,b)^{q} +\alpha L(x,y,x',y')^{q}\big)^{p} d\pi_{1}((x,a),(y,b)) d\pi_{1}((x',a'),(y',b'))\bigg)^{\frac{1}{p}} \\
&+ \bigg(\!\!\!\!\!\!\!\!\!\!\!\!\!\! \int\limits_{(\tv{\Ycal \times \Omega} \times \mathcal{Z} \times \Omega)^{2}}\big((1-\alpha) d(b,c)^{q} +\alpha L(y,z,y',z')^{q}\big)^{p} d\pi_{2}((y,b),(z,c))\,d\pi_{2}((y',b'),(z',c'))\bigg)^{\frac{1}{p}} \bigg) \\ 
&= 2^{q-1}(\fgwdistance_{\alpha,p,q}(\mu,\beta){+}\fgwdistance_{\alpha,p,q}(\beta,\nu))
\end{split}
\end{equation*}
}
with (*) comes from (\ref{ineqL}) and (\ref{ineqD}) and (**) is Minkowski inequality. So when $q=1$,  $\fgwdistance_{\alpha,p,q}$ satisfies the triangle inequality and when $q>1$,  $\fgwdistance_{\alpha,p,q}$ satisfies a relaxed triangle inequality so that it defines a semi-metric as described previously.
\end{proof}

\begin{prop}[\tv{Equality relation.}]
For $\alpha \in ]0,1[$, $\fgwdistance_{\alpha,p,q}(\mu,\nu)=0$ if an only if there exists a bijective function $ \phi=(\phi_1,\phi_2): \tv{\supp(\mu)} \rightarrow \tv{\supp(\nu)}$ such that:
\begin{equation}
\label{p1}
\phi\#\mu=\nu
\end{equation}
\begin{equation}
\label{p2}
\forall (x,a) \in \tv{\supp(\mu)} \ , \phi_{2}(x,a)=a
\end{equation}
\begin{equation}
\label{p3}
\forall (x,a),(x',a') \in \tv{\supp(\mu)^{2}}, \ d_{\Xcal}(x,x')=d_{\Ycal}(\phi_{1}(x,a),\phi_{1}(x',a'))
\end{equation}
Moreover if $(\mu,\nu)$ are generalized labeled graphs then $\fgwdistance_{\alpha,p,q}(\mu,\nu)=0$ if and only if $(\Xcal \times \Omega,d_\Xcal,\mu)$ and $(\Ycal \times \Omega,d_\Ycal,\nu)$ are (II)-strongly isomorphic.
\end{prop}

\begin{proof}

For the first point, let us assume that there exists a function $\phi$ verifying \eqref{p1}, \eqref{p2} and \eqref{p3}. We consider the map $\pi=( I_{d} \times \phi ) \# \mu \in \couplingset(\mu,\nu)$. We note $\phi=(\phi_{1},\phi_{2})$. Then:
\begin{equation}
\label{eqeqeqeq}
\begin{split}
\lossfgw(\pi)&=\int\limits_{(\tv{\Xcal \times \Omega} \times \tv{\Ycal \times \Omega})^{2}}\Big((1-\alpha) d(a,b)^{q} +\alpha L((x,y,x',y')^{q} \Big)^{p}\dr\pi((x,a),(y,b))\dr\pi((x',a'),(y',b')) \\
&=\int\limits_{(\tv{\Xcal \times \Omega} )^{2}} \Big((1-\alpha) d(a,\phi_{2}(x,a))^{q} +\alpha L\big((x,\phi_{1}(x,a),x',\phi_{1}(x',a')\big)^{q} \Big)^{p}d\mu(x,a)d\mu(x',a') \\ 
&=\int\limits_{(\tv{\Xcal \times \Omega} )^{2}} \Big((1-\alpha) d(a,\phi_{2}(x,a))^{q} +\alpha |d_{\Xcal}(x,x')-d_{\Ycal}(\phi_{1}(x,a),\phi_{1}(x',a'))|^{q} \Big)^{p}d\mu(x,a)d\mu(x',a') \\
&=0
\end{split} 
\end{equation}

Conversely suppose that $\fgwdistance_{\alpha,p,q}(\mu,\nu) =0$. To prove the existence of a map $\phi: \tv{\supp(\mu)} \rightarrow \tv{\supp(\nu)}$ verifying \eqref{p1}, \eqref{p2} and \eqref{p3} we will use the Gromov-Wasserstein properties. We are looking for a vanishing Gromov-Wassersein distance between the spaces $\tv{\Xcal \times \Omega}$ and $\tv{\Ycal \times \Omega}$ equipped with our two measures $\mu$ and $\nu$.

More precisely, we define for $\big((x,a),(y,b),(x',a'),(y',b')\big) \in (\tv{\Xcal \times \Omega} \times \tv{\Ycal \times \Omega})^{2}$ and $\beta \in ]0,1[$:

$$ d_{\tv{\Xcal \times \Omega}} \big((x,a),(x',a')\big)=(1-\beta) d_{\Xcal}(x,x')+ \beta d(a,a')$$
and
$$ d_{\tv{\Ycal \times \Omega}} \big((y,b),(y',b')\big)=(1-\beta) d_{\Ycal}(y,y')+\beta d(b,b')$$
We will prove that $d_{GW,p}(d_{\tv{\Xcal \times \Omega}},d_{\tv{\Ycal \times \Omega}},\mu,\nu)=0$. To show that we will bound the Gromov cost with the metrics $d_{\tv{\Xcal \times \Omega}},d_{\tv{\Ycal \times \Omega}}$ by the Gromov cost with the metrics $d_{\Xcal},d_{\Ycal}$ and a Wasserstein cost.

Let $\pi \in \couplingset(\mu,\nu)$ be any admissible transportation plan. Then for $n\geq 1$: 
{\small
\begin{equation*}
\begin{split}
&J_{n}(d_{\tv{\Xcal \times \Omega}},d_{\tv{\Ycal \times \Omega}},\pi)\stackrel{def}{=}\int\limits_{(\tv{\Xcal \times \Omega} \times \tv{\Ycal \times \Omega})^{2}} L(x,y,x',y')^{n} d\pi((x,a),(y,b))d\pi((x',a'),(y',b')) \\
&=\int\limits_{(\tv{\Xcal \times \Omega} \times \tv{\Ycal \times \Omega})^{2}} |(1-\beta) (d_{\Xcal}(x,x')-d_{\Ycal}(y,y'))+ \beta (d(a,a')-d(b,b')) |^{n} d\pi((x,a),(y,b)) d\pi((x',a'),(y',b')) \\ 
&\leq \int\limits_{(\tv{\Xcal \times \Omega} \times \tv{\Ycal \times \Omega})^{2}}(1-\beta) | d_{\Xcal}(x,x')-d_{\Ycal}(y,y')|^{n} d\pi((x,a),(y,b))\,d\pi((x',a'),(y',b'))\\ 
&+\int\limits_{(\tv{\Xcal \times \Omega} \times \tv{\Ycal \times \Omega})^{2}}  \beta |d(a,a')-d(b,b') |^{n} d\pi((x,a),(y,b)) d\pi((x',a'),(y',b'))
\end{split}
\end{equation*}
}
using Jensen inequality with convexity of $t\rightarrow t^{n}$ and subadditivity of |.| . We note $(*)$ the first term above and $(**)$ the second term above. By the triangle inequality property of $d$ we have: \\ \\
(**)$\leq \beta \int\limits_{\big(\tv{\Xcal \times \Omega} \times \tv{\Ycal \times \Omega})^{2}}\!\!\!\!\!\!\!\!\!\!\!\!\!\! (d(a,b)+d(a',b') \big)^{n} d\pi((x,a),(y,b))\,d\pi((x',a'),(y',b')\!) \stackrel{def}{=} \beta M_{n}(\pi)$ such that we have shown:

\begin{equation}
\label{relationgwgw}
\forall \pi \in \couplingset(\mu,\nu), \forall n \geq 1, J_{n}(d_{\tv{\Xcal \times \Omega}},d_{\tv{\Ycal \times \Omega}},\pi) \leq (1-\beta) J_{n}(d_{\Xcal},d_{\Ycal},\pi) + \beta M_{n}(\pi)
\end{equation}

Now let $\pi_{*}$ be an optimal coupling for $\fgwdistance_{\alpha,p,q}$ between $\mu$ and $\nu$. By hypothesis $\fgwdistance_{\alpha,p,q}(\mu,\nu)=0$ so that:

\begin{equation}
J_{qp}(d_{\Xcal},d_{\Ycal},\pi{*})=0
\end{equation}

and:

\begin{equation}
\label{heq}
H_{qp}(\pi^{*})=0.
\end{equation}

Then $\int\limits_{(\tv{\Xcal \times \Omega}\times \tv{\Ycal \times \Omega})} d(a,b)^{qp} d\pi^{*}((x,a),(y,b)){=}0$ which implies that $d$ is zero $\pi^{*}$ \textit{a.e.} so that $\int\limits_{(\tv{\Xcal \times \Omega}\times \tv{\Ycal \times \Omega})} d(a,b)^{m} d\pi^{*}((x,a),(y,b))=0$ for any $m\in \mathbb{N}^{*}$. In this way: 
 
\begin{equation*}
\begin{split}
&M_{qp}(\pi^{*})= \beta \int\limits_{(\tv{\Xcal \times \Omega}\times \tv{\Ycal \times \Omega})^{2}} \sum_{h} {qp\choose h} d(a,b)^{h} d(a',b')^{qp-h} d\pi^{*}((x,a),(y,b))d\pi^{*}((x',a'),(y',b')) \\ 
&= \beta\sum_{h} {qp\choose h} \Big(\int\limits_{(\tv{\Xcal \times \Omega}\times \tv{\Ycal \times \Omega})} d(a,b)^{h} d\pi^{*}((x,a),(y,b))\Big) \Big(\int\limits_{(\tv{\Xcal \times \Omega}\times \tv{\Ycal \times \Omega})} d(a',b')^{qp-h} d\pi^{*}((x',a'),(y',b'))\Big) =0
\end{split}
\end{equation*}

Using equation (\ref{relationgwgw}) we have shown $$J_{qp}(d_{\tv{\Xcal \times \Omega}},d_{\tv{\Ycal \times \Omega}},\pi_{*})=0$$
which implies that $d_{GW,p}(d_{\tv{\Xcal \times \Omega}},d_{\tv{\Ycal \times \Omega}},\mu,\nu)=0$ for the coupling $\pi^{*}$.

Thanks to the Gromov-Wasserstein properties (see Chapter \ref{cha:ot_general}) this states the existence of an isometry between $\text{supp}(\mu)$ and $\text{supp}(\nu)$. So there exists a surjective function $\phi=(\phi_{1},\phi_{2}): \tv{\supp(\mu)} \rightarrow \tv{\supp(\nu)}$ which verifies \ref{I} and:

\begin{equation}
\forall ((x,a),(x',a')) \in (\tv{\supp(\mu)})^{2}, \ d_{\tv{\Xcal \times \Omega}} ((x,a),(x',a'))=d_{\tv{\Ycal \times \Omega}} (\phi(x,a),\phi(x',a'))
\end{equation}
or equivalently:
{\small
\begin{equation}
\label{hat}
\forall ((x,a),(x',a')) \in (\tv{\supp(\mu)})^{2}, \ (1-\beta) d_{\Xcal}(x,x')+ \beta d(a,a')=(1-\beta) d_{\Ycal}(\phi_{1}(x,a),\phi_{1}(x',a'))+ \beta d(\phi_{2}(x,a),\phi_{2}(x',a'))
\end{equation}
}
In particular $\pi^{*}$ is concentrated on $\{(x,y)=\phi(x,y)\}$ or equivalently $\pi^{*}=(I_{d} \times \phi)\#\mu$. Injecting $\pi^{*}$ in \eqref{heq} leads to:
\begin{equation}
\begin{split}
H_{qp}(\pi^{*})&=\int\limits_{(\tv{\Xcal \times \Omega}\times \tv{\Ycal \times \Omega})} d(a,b)^{qp} \dr\pi^{*}((x,a),(y,b)) =\int\limits_{\tv{\Xcal \times \Omega}}d(a,\phi_{2}(x,a))^{qp} \dr\mu(x,a)=0
\end{split}
\end{equation}
Which implies:
\begin{equation}
\label{p22}
\forall (x,a) \in \tv{\text{supp}(\mu)} \ , \phi_{2}(x,a)=a
\end{equation}
Moreover, using the equality (\ref{hat}) we can conclude that: 
\begin{equation}
\label{endgame}
\forall (x,a)(x',a') \in \tv{\text{supp}(\mu)^{2}}, \ d_{\Xcal}(x,x')=d_{\Ycal}(\phi_{1}(x,a),\phi_{1}(x',a'))
\end{equation}

In this way $f$ verifies all the properties \eqref{p1},\eqref{p2},\eqref{p3}. 

Moreover suppose that $\mu$ and $\nu$ are generalized labeled graphs. In this case there exists $\ell_f: \Xcal \rightarrow \Omega$ surjective such that $\mu=(id\times \ell_f)\# \mu_X$. Then \eqref{endgame} implies that:
\begin{equation}
\label{endgamemille}
\forall (x,x') \in \tv{\text{supp}(\mu_X)^{2}}, \ d_{\Xcal}(x,x')=d_{\Ycal}(\phi_{1}(x,\ell_f(x)),\phi_{1}(x',\ell_f(x')))
\end{equation} 
We define $I: \supp(\mu_X) \rightarrow \supp(\mu_Y)$ such that $I(x)=\phi_{1}(x,\ell_f(x))$. Then we have by \eqref{endgamemille} $d_{\Xcal}(x,x')=d_{\Ycal}(I(x),I(x'))$ for $(x,x') \in \supp(\mu_X)^{2}$. Overall we have $\phi(x,a)=(I(x),a)$ for all $(x,a) \in \supp(\mu)$. Also since $\phi\#\mu=\nu$ we have $I\#\mu_X=\nu_Y$.

Moreover $I$ is a surjective function. Indeed let $y\in \supp(\nu_Y)$. Let $b \in \supp(\nu_B)$ such that $(y,b) \in \supp(\nu)$. By surjectivity of $\phi$ there exists $(x,a) \in \supp(\mu)$ such that $(y,b)=\phi(x,a)=(I(x),a)$ so that $y=I(x)$. 

Overall $\phi$ satisfies all \ref{I}, \ref{II} and \ref{III} if $\mu$ and $\nu$ are generalized labeled graphs. The converse is also true using the reasoning in \eqref{eqeqeqeq}.
\end{proof}

\subsection{Proof of Theorem \ref{concentration} -- Convergence and concentration inequality.}
\label{proof:prop2}

We recall the theorem:

\begin{theo*}{Convergence of finite samples and a concentration inequality}
Let $p \geq 1$. We have:
\begin{equation*}
\lim\limits_{n \rightarrow \infty}\fgwdistance_{\alpha,p,1}(\mu_{n},\mu)=0
\end{equation*}
Moreover, suppose that $s > d_{p}^{*}(\mu)$. Then there exists a constant $C$ that does not depend on $n$ such that:
\begin{equation*}
\mathbb{E}[\fgwdistance_{\alpha,p,1}(\mu_{n},\mu)] \leq C n^{-\frac{1}{s}}.
\end{equation*}
The expectation is taken over the \textit{i.i.d} samples $(x_i,a_i)$. A particular case of this inequality is when $\alpha=1$ so that we can use the result above to derive a concentration result for the Gromov-Wasserstein distance. More precisely, if $\nu_{n}=\frac{1}{n} \sum_{i} \delta_{x_{i}}$ denotes the empirical measure of $\nu \in \Pcal(\Xcal)$ and if $s' > d_{p}^{*}(\nu)$ we have:
\begin{equation*}
\label{concentrationgromov}
\mathbb{E}[\gw_{p}(\nu_{n},\nu)] \leq C' n^{-\frac{1}{s'}}.
\end{equation*}
\end{theo*}
\begin{proof}
The proof of the convergence in $FGW$ derives directly from the weak convergence of the empirical measure and Lemma \ref{lsc_on_measure}. Moreover, since $\mu_{n}$ and $\mu$ are both in the same ground space, we have:
$$\fgwdistance_{\alpha,p,1}(\mu_{n},\mu) \leq 2\wass_{p}(\mu_{n},\mu) \implies \mathbb{E}[\fgwdistance_{\alpha,p,1}(\mu_{n},\mu)] \leq 2 \mathbb{E}[\wass_{p}(\mu_{n},\mu)].$$
We can directly apply theorem 1 in \cite{weedbach2017} to state the inequality.
\end{proof}

\subsection{Proof of proposition \ref{interpolationtheorem} -- Interpolation properties between GW and W}
\label{proof:theo2}
We recall the proposition:
\begin{prop*}[Interpolation properties.]
As $\alpha$ tends to zero, one recovers the Wasserstein distance between the features information and as $\alpha$ goes to one, one recovers the Gromov-Wasserstein distance between the structure information:

\begin{equation*}
\lim\limits_{\alpha \rightarrow 0}\fgwdistance_{\alpha,p,q}(\mu,\nu)=(\wass_{pq}(\mu_{A},\nu_{B}))^{q}
\end{equation*}

\begin{equation*}
\lim\limits_{\alpha \rightarrow 1}\fgwdistance_{\alpha,p,q}(\mu,\nu)=(\gw_{pq}(\mu_{X},\nu_{Y}))^{q}
\end{equation*}
\end{prop*}

\begin{proof}

Let $\pi_{OT} \in \couplingset(\mu_{A},\nu_{B})$ be an optimal coupling for the $pq$-Wasserstein distance between $\mu_{A}$ and $\nu_{B}$. We can use the same Gluing lemma (lemma 5.3.2 in \cite{ambrosio2005gradient}) to construct:
\begin{equation*}
  \rho \in \Pcal( 
      \mathrlap{\overbrace{\phantom{(c - 2)}}^{\text{$\mu$}}}
      \Xcal \times 
      \mathrlap{\underbrace{\phantom{2d) + (}}_{\text{$\pi_{OT}$}}}
      \Omega \times
      \mathrlap{\overbrace{\phantom{2d) + (}}^{\text{$\nu$}}}
    \Omega \times \Ycal
      )
\end{equation*}
such that $\rho \in \couplingset(\mu,\nu)$ and $P_{2,3} \# \rho =\pi_{OT}$.

Moreover we have:
\begin{equation}
\label{rhorho}
\underset{\Omega \times \Omega}{\int}d(a,b)^{pq} d\pi_{OT}(a,b)=\underset{\tv{\Xcal \times \Omega}\times \Omega \times \Ycal}{\int} d(a,b)^{pq} d\rho(x,a,b,y)
\end{equation}
Let $\alpha \geq 0$ and $\pi_{\alpha}$ optimal plan for the Fused Gromov-Wasserstein distance between $\mu$ and $\nu$. 

We can deduce that:
{\small
\begin{equation*}
\begin{split}
&\fgwdistance_{\alpha,p,q}(\mu,\nu)^{p}- (1-\alpha)^{p}\wass_{pq}(\mu_{A},\nu_{B})^{pq} \\
&= \int\limits_{(\tv{\Xcal \times \Omega} \times \tv{\Ycal \times \Omega})^{2}}\bigg((1-\alpha) d(a,b)^{q} +\alpha L(x,y,x',y')^{q}\bigg)^{p}d\pi_{\alpha}((x,a),(y,b))d\pi_{\alpha}((x',a'),(y',b'))-\underset{\Omega \times \Omega}{\int} (1-\alpha)^{p}d(a,b)^{pq}d\pi_{OT}(a,b) \\
&{\stackrel{(*)}{\leq}} \int\limits_{(\tv{\Xcal \times \Omega} \times \tv{\Ycal \times \Omega})^{2}}\bigg((1-\alpha) d(a,b)^{q} +\alpha L(x,y,x',y')^{q} \bigg)^{p} d\rho(x,a,b,y)d\rho(x',a',b',y')-\underset{\tv{\Xcal \times \Omega}\times \tv{\Ycal \times \Omega}}{\int}(1-\alpha)^{p} d(a,b)^{pq} d\rho(x,a,b,y) \\ 
&=(1-\alpha)^{p}\int\limits_{(\tv{\Xcal \times \Omega} \times \tv{\Ycal \times \Omega})^{2}}d(a,b)^{pq} d\rho(x,a,b,y)d\rho(x',a',b',y') - (1-\alpha)^{p} \underset{\tv{\Xcal \times \Omega}\times \tv{\Ycal \times \Omega}}{\int}d(a,b)^{pq} d\rho(x,a,b,y) \\ 
&+\sum_{k=0}^{p-1} {p\choose k}(1-\alpha)^{k}\alpha^{p-k} \int\limits_{(\tv{\Xcal \times \Omega} \times \tv{\Ycal \times \Omega})^{2}} d(a,b)^{qk}L(x,y,x',y')^{q(p-k)} d\rho(x,a,b,y)d\rho(x',a',b',y')\\ 
&= \sum_{k=0}^{p-1} {p\choose k}(1-\alpha)^{k}\alpha^{p-k} \int\limits_{(\tv{\Xcal \times \Omega} \times \tv{\Ycal \times \Omega})^{2}} d(a,b)^{qk}L(x,y,x',y')^{q(p-k)} d\rho(x,a,b,y)d\rho(x',a',b',y').
\end{split}
\end{equation*}
}

We note $H_{k}=\int\limits_{(\tv{\Xcal \times \Omega} \times \tv{\Ycal \times \Omega})^{2}} d(a,b)^{qk}L(x,y,x',y')^{q(p-k)} d\rho(x,a,b,y)d\rho(x',a',b',y')$.

Using (\ref{wassinequality})  we have shown that:

\begin{equation*}
(1-\alpha)(\wass_{pq}(\mu_{A},\nu_{B}))^{q} \leq \fgwdistance_{\alpha,p,q}(\mu,\nu) \leq  \bigg((1-\alpha)^{p} (\wass_{pq}(\mu_{A},\nu_{B}))^{pq}+\sum_{k=0}^{p-1} {p\choose k}(1-\alpha)^{k}\alpha^{p-k}H_{k}\bigg)^{\frac{1}{p}}
\end{equation*}

So $\lim\limits_{\alpha \rightarrow 0} \fgwdistance_{\alpha,p,q}(\mu,\nu)=(\wass_{pq}(\mu_{A},\nu_{B}))^{q}$.
 \\ \\
 
For the case $\alpha \rightarrow 1$ we rather consider $\pi_{GW} \in \couplingset(\mu_{X},\nu_{Y})$ an optimal coupling for the $pq$-Gromov-Wasserstein distance between $\mu_{X}$ and $\nu_{Y}$ and we construct

\begin{equation*}
  \gamma \in \Pcal( 
      \mathrlap{\overbrace{\phantom{(c - 2)}}^{\text{$\mu$}}}
      \Omega \times 
      \mathrlap{\underbrace{\phantom{2d) + (}}_{\text{$\pi_{GW}$}}}
      \Xcal \times
      \mathrlap{\overbrace{\phantom{2d) + (}}^{\text{$\nu$}}}
    \tv{\Ycal \times \Omega}
      )
\end{equation*}

such that $\gamma \in \couplingset(\mu,\nu)$ and $P_{2,3} \# \rho =\pi_{GW}$. In the same way as previous reasoning we can derive:
\begin{equation}
\alpha(\gw_{pq}(\mu_{X},\nu_{Y}))^{q} \leq \fgwdistance_{\alpha,p,q}(\mu,\nu) \leq  \big(\alpha^{p} (\gw_{pq}(\mu_{X},\nu_{Y}))^{pq}+\sum_{k=0}^{p-1} {p\choose k}(1-\alpha)^{p-k}\alpha^{k} J_{k})^{\frac{1}{p}}
\end{equation}
with $J_{k}= \int d(a,b)^{q(p-k)}L(x,y,x',y')^{qk} d\rho(x,a,b,y)d\rho(x',a',b',y')$. In this way $\lim\limits_{\alpha \rightarrow 1}\fgwdistance_{\alpha,p,q}(\mu,\nu)=(\gw_{pq}(\mu_{X},\nu_{Y}))^{q}$.
\end{proof}

\subsection{Proof of Theorem \ref{cstespeedtheo} -- Constant speed geodesic.}
\label{proof:theo3}
We recall the theorem:
\begin{theo*}[\tv{Constant speed geodesic.}]
Let $p\geq 1$ and $(\Xcal \times \Omega ,d_{\Xcal},\mu_{0})$ and $(\Ycal \times \Omega ,d_{\Ycal},\mu_{1})$ in $\mathbb{S}_p(\mathbb{R}^{d})$. Let $\pi^{*}$ be an optimal coupling for the Fused Gromov-Wasserstein distance between $\mu_{0},\mu_{1}$ and $t \in [0,1]$. We equip $\mathbb{R}^{d}$ with  $\ell_{m}$ norm for $m \geq 1$.

We define $\eta_{t} : \Xcal \times \Omega \times \Ycal \times \Omega \rightarrow \Xcal \times \Ycal \times \Omega$ such that:
\begin{equation*}
\forall (x,a),(y,b) \in \Xcal \times \Omega \times \Ycal \times \Omega, \ \eta_{t}(x,\a,y,\b)=(x,y,(1-t)\a+t\b)
\end{equation*}
Then:
\begin{equation*}
\left(\Xcal\times \Ycal \times \Omega,(1-t)d_{\Xcal} \oplus t d_{\Ycal},\mu_{t}=\eta_{t} \# \pi^{*} \right)_{t \in [0,1]}
\end{equation*}
is a constant speed geodesic connecting $(\Xcal \times \Omega ,d_{\Xcal},\mu_{0})$ and $(\Ycal \times \Omega ,d_{\Ycal},\mu_{1})$ in the metric space $\left(\mathbb{S}_p(\mathbb{R}^{d}), \fgwdistance_{\alpha,p,1}\right)$.

\end{theo*}

\begin{proof}
We note $\mathbb{S}_{t}=\left(\Xcal\times \Ycal \times \Omega,d_t,\mu_{t}=\eta_{t} \# \pi^{*} \right)_{t \in [0,1]}$ where $d_{t}=(1-t)d_{\Xcal} \oplus t d_{\Ycal}$. Let $\|.\|$ be any $\ell_m$ norm for $m \geq 1$. It suffices to prove:
\begin{equation}
\fgwdistance_{\alpha,p,1}(\mu_{t},\mu_{s}) \leq |t-s| \fgwdistance_{\alpha,p,1}(\mu_{0},\mu_{1})
\end{equation}

To do so we consider $\Delta^{t}_{s} \in \Pcal(\Xcal \times \Ycal \times \Omega \times \Xcal \times \Ycal \times \Omega)$ defined by  $\Delta^{t}_{s}=(\eta_{t} \times \eta_{s} )\# \pi^{*} \in \couplingset(\mu_{t},\mu_{s})$ and the following ``diagonal'' coupling:
\begin{equation}
\dr\gamma_{s}^{t}( (x,y), \a , (x'',y''),\b) = \dr \Delta^{t}_{s} ((x,y), \a , (x'',y''),\b) \dr\delta_{(x_{0},x_{1})}(x_{0}'',x_{1}'')
\end{equation}
Then $\gamma_{s}^{t} \in \Pcal(\Xcal \times \Ycal \times \Omega \times \Xcal \times \Ycal \times \Omega)$  and since $\Delta^{t}_{s} \in \couplingset(\mu_{t},\mu_{s})$ then $\gamma_{s}^{t} \in \couplingset(\mu_{t},\mu_{s})$ So by suboptimality:
{\small
\begin{equation*}
\begin{split}
&\fgwdistance_{\alpha,p,1}(\mu_{t},\mu_{s})^{p} \leq \int\limits_{(\Xcal \times \Ycal \times \Omega \times \Xcal \times \Ycal \times \Omega)^{2}} \bigg((1-\alpha)d(a,b)+\alpha |d_{t}[(x,y),(x',y')]{-}d_{s}[(x'',y''),(x''',y''')] |\bigg)^{p} \\
& \dr\gamma_{s}^{t}(x,y, \a ,x'',y'',\b)\dr\gamma_{s}^{t}(x',y', \a',x''',y''',\b') \\
&=\int\limits_{(\Xcal \times \Ycal \times \Omega \times \Xcal \times \Ycal \times \Omega)^{2}}\bigg((1-\alpha)d(\a,\b)+\alpha |d_{t}[(x,y),(x',y')]{-}d_{s}[(x,y),(x',y')] |\bigg)^{p} \\
& \dr\Delta_{s}^{t}(x,y, \a ,x,y,\b)d\Delta_{s}^{t}(x',y', \a',x',y',\b') \\ 
&=\int\limits_{(\Xcal \times\Omega \times \Ycal \times\Omega)^{2}}\big((1-\alpha) \|(1-t)\a+t\b{-}(1{-}s)\a{-}s\b \| +\alpha |(1-t)d_{\Xcal}(x,x'){+}td_{\Ycal}(y,y')-(1-s)d_{\Xcal}(x,x')+sd_{\Ycal}(y,y') |\big)^{p} \\
&\dr\pi^{*}(x,a,y,b)d\pi^{*}(x',\a',y',\b') \\ 
&= |t-s|^{p} \int\limits_{(\Xcal \times \Omega \times \Ycal \times \Omega)^{2}}\bigg( (1-\alpha) \|\a-\b\| + \alpha |d_{\Xcal}(x,x')-d_{\Ycal}(y,y')| \bigg)^{p} d\pi^{*}(x,\a,y,\b)\dr\pi^{*}(x',\a',y',\b')
\end{split}
\end{equation*}
}

So $\fgwdistance_{\alpha,p,1}(\mu_{t},\mu_{s}) \leq |t-s| \fgwdistance_{\alpha,p,1}(d_{0},d_{1},\mu_{0},\mu_{1})$. 
\end{proof}

\section{Proofs and additional results of Chapter \ref{cha:gw_euclidean}}
\label{sec:proofs_chap_gw}

This Chapter contains all the proofs of the Chapter \ref{cha:gw_euclidean}. We recall the following notations:
\begin{equation*}
\gwloss_2(c_{\Xcal},c_{\Ycal},\pi)\stackrel{def}{=} \int_{\Xcal \times \Xcal} \int_{\Ycal \times \Ycal} |c_{\Xcal}(\xbf,\xbf')-c_{\Ycal}(\ybf,\ybf')|^2 \dr \pi(\xbf,\ybf) \dr \pi(\xbf',\ybf')
\end{equation*}
\begin{equation*}
J(T)\stackrel{def}{=} \int \left(\|\xbf-\xbf'\|_{2}^{2}-\|T(\xbf)-T(\xbf')\|_{2}^{2}\right)^{2} \dr \mu(\xbf) \dr \mu(\xbf')
\end{equation*}
\begin{equation*}
L\gm^{2}_{2}(\mu,\nu)\stackrel{def}{=} \underset{\begin{smallmatrix} T\#\mu=\nu \\ \text{T is linear} \end{smallmatrix}}{\inf} J(T) = \underset{\begin{smallmatrix} T\#\mu=\nu \\ \text{T is linear} \end{smallmatrix}}{\inf} \int \left(\|\xbf-\xbf'\|_{2}^{2}-\|T(\xbf)-T(\xbf')\|_{2}^{2}\right)^{2} \dr \mu(\xbf) \dr \mu(\xbf')
\end{equation*}

\subsection{Proof of Theorem \ref{qap} \label{proof:maintheo} -- New special of the QAP}
We recall the theorem:

\begin{theo*}[A new special case for the Quadratic Assignment Problem]
For real numbers $x_{1} < \dots < x_{n}$ and $y_{1} < \dots < y_{n}$, 
\begin{equation}
\label{eq:qap}
\underset{\sigma \in \Sn}{\min} \sum_{i,j}  - (x_{i}-x_{j})^{2}(y_{\sigma(i)}-y_{\sigma(j)})^{2}
\end{equation}
is achieved either by the identity permutation $\sigma(i)=i$ ($Id$) or the anti-identity permutation $\sigma(i)=n+1-i$ ($anti-Id$). In other words:
\begin{equation}
\exists \sigma \in \{Id,anti-Id\}, \ \sigma \in \underset{\sigma \in \Sn}{\argmin} \sum_{i,j}  - (x_{i}-x_{j})^{2}(y_{\sigma(i)}-y_{\sigma(j)})^{2} 
\end{equation}
\end{theo*}
Let us note $\mathcal{I}=\{\xbf,\ybf \in \mathbb{R}^{n}\times \mathbb{R}^{n} | x_{1} < x_{2} < \dots < x_{n} \ , y_{1} < y_{2} < \dots < y_{n} \}$. We consider for $\xbf,\ybf \in \mathcal{I}$:
\begin{equation}
\label{qpproblem}
\underset{\sigma \in \Sn}{\text{max}} \ Z(\xbf,\ybf,\sigma)=\underset{\sigma \in \Sn}{\text{max}} \sum_{i,j} (x_{i}-x_{j})^{2} (y_{\sigma(i)}-y_{\sigma(j)})^{2}
\end{equation}
The original problem is equivalent to maximizing $Z(\xbf,\ybf,\sigma)$ over $\Sn$. Given $\xbf,\ybf \in \mathcal{I}$, we define $X \stackrel{\text{def}}{=}\sum_{i} x_{i}$ and $Y \stackrel{\text{def}}{=}\sum_{i} y_{i}$. Then:
\begin{equation*}
\begin{split}
&\underset{\sigma \in \Sn}{\text{max}} \ Z(\xbf,\ybf,\sigma) =\underset{\sigma \in \Sn}{\text{max}} \sum_{i,j} (x_{i}-x_{j})^{2} (y_{\sigma(i)}-y_{\sigma(j)})^{2} \\
&=\underset{\sigma \in \Sn}{\text{max}} \sum_{i,j} (x_{i}^{2}+x_{j}^{2})(y_{\sigma(i)}^{2}+y_{\sigma(j)}^{2}) -2 \sum_{i,j} x_{i} x_{j}(y_{\sigma(i)}^{2}+y_{\sigma(j)}^{2}) -2 \sum_{i,j} y_{\sigma(i)} y_{\sigma(j)}(x_{i}^{2}+x_{j}^{2}) \\ 
&+4 \sum_{i,j} x_{i} x_{j} y_{\sigma(i)} y_{\sigma(j)} \\
&= \underset{\sigma \in \Sn}{\text{max}} \ 2 n \sum_{i} x_{i}^{2}y_{\sigma(i)}^{2}-2 \sum_{i,j} x_{i} x_{j}(y_{\sigma(i)}^{2}+y_{\sigma(j)}^{2}) -2 \sum_{i,j} y_{\sigma(i)} y_{\sigma(j)}(x_{i}^{2}+x_{j}^{2}) \\ 
&+4 \sum_{i,j} x_{i} x_{j} y_{\sigma(i)} y_{\sigma(j)} + 2 (\sum_{i} x_{i}^{2})(\sum_{i} y_{i}^{2}) \\
&= \underset{\sigma \in \Sn}{\text{max}} \ 2 n \sum_{i} x_{i}^{2}y_{\sigma(i)}^{2}-4 X \sum_{i} x_{i} y_{\sigma(i)}^{2} -4 Y \sum_{i} x_{i}^{2} y_{\sigma(i)} +4 \sum_{i,j} x_{i} x_{j} y_{\sigma(i)} y_{\sigma(j)} + 2 (\sum_{i} x_{i}^{2})(\sum_{i} y_{i}^{2}) \\
&\stackrel{(*)}{=} Cte+2 \big(\underset{\sigma \in \Sn}{\text{max}} \  \sum_{i} n x_{i}^{2}y_{\sigma(i)}^{2}-2 \sum_{i} (X x_{i} y_{\sigma(i)}^{2} + Y x_{i}^{2} y_{\sigma(i)}) + 2(\sum_{i} x_{i} y_{\sigma(i)})^{2}\big) \\
\end{split}
\end{equation*}
where in (*) we defined $Cte\stackrel{\text{def}}{=}2 (\sum_{i} x_{i}^{2})(\sum_{i} y_{i}^{2})$ the term that does not depend on $\sigma$. Overall we have:
\begin{equation}
\label{equivone}
\forall \xbf,\ybf \in \mathcal{I}, \ \underset{\sigma \in \Sn}{\text{argmax}} \ Z(\xbf,\ybf,\sigma) = \underset{\sigma \in \Sn}{\text{argmax}} \sum_{i} n x_{i}^{2}y_{\sigma(i)}^{2}-2 \sum_{i} (X x_{i} y_{\sigma(i)}^{2} + Y x_{i}^{2} y_{\sigma(i)}) + 2(\sum_{i} x_{i} y_{\sigma(i)})^{2}
\end{equation}

Since $Z$ is invariant by translation of $\xbf,\ybf$ we can suppose without loss of generality that $X=Y=0$. We consider the set $\mathcal{D}=\{\xbf,\ybf \in \mathbb{R}^{n}\times \mathbb{R}^{n} | x_{1} < x_{2} < \dots < x_{n} \ , y_{1} < y_{2} < \dots < y_{n}, \ \sum_i x_i= \sum_j y_j=0 \}$. We want to find for $\xbf,\ybf \in \mathcal D$:
\begin{equation}
\tag{QAP}
\label{eq:qap_to_prove}
\max_{\sigma \in \Sn} n \sum_{i} x_i^{2} y_{\sigma(i)}^{2} +2 \left(\sum_i x_i y_{\sigma(i)}\right)^{2} \stackrel{def}{=} \max_{\sigma \in \Sn} g(\xbf,\ybf,\sigma)
\end{equation}

We have the following result:

\begin{lemma}
\label{lemma:equivalence_both_probs}
Let $\xbf,\ybf \in \mathcal{D}$ and consider the problem: 
\begin{equation}
\tag{QP}
\label{eq:qp_equiv}
\max_{\GG \in DS} \sum_{ijkl} (x_i^2 y_j^2+ 2 x_i y_j x_k y_l)\pi_{ij} \pi_{kl}\\
\end{equation}
where $DS$ is the set of doubly stochastic matrices. Then \eqref{eq:qp_equiv} and \eqref{eq:qap_to_prove} are equivalent. More precisely if $\sigma^{*}$ is an optimal solution of \eqref{eq:qap_to_prove} then $\GG_{\sigma^{*}}$ defined by $\pi_{\sigma^{*}}(i,j)=1$ if $j=\sigma^{*}(i)$ else $0$ for all $(i,j) \in \integ{n}^{2}$ is an optimal solution of \eqref{eq:qp_equiv} and if $\GG^{*}$ is an optimal solution of \eqref{eq:qp_equiv} then it is supported on a permutation $\sigma^{*}$ which is an optimal solution of \eqref{eq:qap_to_prove}.
\end{lemma}

\begin{proof}
The problem \eqref{eq:qap_to_prove} can be rewritten as:
{\small
\begin{equation}
\label{eq:QAP_with_P}
\begin{split}
&\max_{\begin{smallmatrix}P_{ij}\in\{0,1\} \\ \forall j \sum_i P_{ij}=1 \\ \forall i \sum_j P_{ij}=1 \end{smallmatrix}} n \sum_{ij} x_i^{2} y_{j}^{2} P_{ij} +2 \left(\sum_{i,j} x_i y_{j}P_{ij}\right)^{2} =\max_{\begin{smallmatrix}P_{ij}\in\{0,1\} \\ \forall j \sum_i P_{ij}=1 \\ \forall i \sum_j P_{ij}=1 \end{smallmatrix}} n \sum_{ij} x_i^{2} y_{j}^{2} P_{ij} +2 \sum_{ijkl}x_i x_k y_j y_l P_{ij} P_{kl} \\
&\stackrel{*}{=}\max_{\begin{smallmatrix}P_{ij}\in\{0,1\} \\ \forall j \sum_i P_{ij}=1 \\ \forall i \sum_j P_{ij}=1 \end{smallmatrix}} \sum_{ijkl} x_i^2 y_j^2 P_{ij} P_{kl} + 2 \sum_{ijkl}x_i y_j x_k y_l P_{ij} P_{kl}= \max_{\begin{smallmatrix}P_{ij}\in\{0,1\} \\ \forall j \sum_i P_{ij}=1 \\ \forall i \sum_j P_{ij}=1 \end{smallmatrix}} \sum_{ijkl} (x_i^2 y_j^2+ 2 x_i y_j x_k y_l)P_{ij} P_{kl}\\
\end{split}
\end{equation}
}
In (*) we used $\sum_{k,l} P_{k,l}=n$. We consider the following relaxation of \eqref{eq:QAP_with_P} as:
\begin{equation}
\max_{\GG \in DS} \sum_{ijkl} (x_i^2 y_j^2+ 2 x_i y_j x_k y_l)\pi_{ij} \pi_{kl}\\
\end{equation}
which is a maximization of a convex function. More precislely it is quadratic programming problem which Hessian is positive semi-definite $\xbf \xbf^{T}\kron \ybf \ybf^{T}$. Since the problem is a maximization of a convex function an optimal solution $\GG^{*}$ of \eqref{eq:qp_equiv} lies necassarily in the extremal points of $DS$ \cite{rockafellar-1970a} such that both \eqref{eq:qp_equiv} and \eqref{eq:qap_to_prove} are equivalent: if $\GG^{*}$ is an optimal solution it is necessarily supported on a $\sigma^{*} \in \Sn$ such that $\sigma^{*}$ is an optimal solution of \eqref{eq:qap_to_prove} and if $\sigma^{*} \in \Sn$ is an optimal solution of \eqref{eq:qap_to_prove} then $\GG^{*}$ defined by $\pi^{*}_{ij}=1$ if $j=\sigma^{*}(i)$ else $0$ for all $(i,j) \in \integ{n}^{2}$ is an optimal solution of \eqref{eq:qp_equiv}.

\end{proof}

\begin{lemma}
\label{lemma:gg_two_cases}
Let $\xbf,\ybf \in \mathcal D$. For $\sigma \in \Sn$ we note $C(\xbf,\ybf,\sigma)=\sum_{i}x_i y_{\sigma(i)}$. Let $\GG^{*}$ an optimal solution of \eqref{eq:qp_equiv} with $\sigma^{*}$ the permutation associated to $\GG^{*}$.

If $C(\xbf,\ybf,\sigma^{*})>0$ then $\GG^{*}=\mathbf{I_n}$ is the identiy and if $C(\xbf,\ybf,\sigma^{*})<0$ then $\GG^{*}=\mathbf{J_n}$ is the anti-identity. 
\end{lemma}

To prove this result we will rely on the following theorem which gives necessary conditions for being an optimal solution of \eqref{eq:qp_equiv}:

\begin{theo*}[Theorem 1.12 in \cite{murty_linear_1988}]
\label{murty_theo}
Consider the following (QP):
\begin{equation}
\label{eq:qp_general}
\begin{array}{cl}{\min _{\xbf} f(\xbf)} & {=\mathbf{c} \xbf+\xbf^{T} \mathbf{Q} \xbf} \\ {\text {s.t.}} & {\mathbf{A} \xbf = \mathbf{b}},\;  {\xbf \geq 0}\end{array}
\end{equation}
Then if $\xbf_{*}$ is an optimal solution of \eqref{eq:qp_general} it is an optimal solution of the following (LP):
\begin{equation}
\label{eq:lp_general}
\begin{array}{cl}{\min _{\xbf} f(\xbf)} & {=(\mathbf{c} + \xbf_{*}^{T}\mathbf{Q} )  \xbf} \\ {\text {s.t.}} & {\mathbf{A} \xbf = \mathbf{b}},\;  {\xbf \geq 0}\end{array}
\end{equation}
\end{theo*}

\begin{proof}{Of Lemma \ref{lemma:gg_two_cases}.}
Applying Theorem \ref{murty_theo} in our case gives that if $\GG^{*}$ is a solution of \eqref{eq:qp_equiv} it necessarily a solution of the following (LP):
\begin{equation}
\max_{\GG \in DS} \sum_{ijkl} (x_i^2 y_j^2+ 2 x_i y_j x_k y_l)\pi^{*}_{ij} \pi_{kl}=n \sum_{ij} x_i^{2} y_{j}^{2} \pi^{*}_{ij}+  \max_{\GG \in DS} 2(\sum_{ij} x_i y_j \pi^{*}_{ij}) (\sum_{kl} x_k y_l \pi_{kl})
\end{equation}

Since $\GG^{*}$ is supported on a permutation $\sigma^{*}$ this gives:
\begin{equation}
\tag{LP}
n \sum_{i} x_i^{2} y_{\sigma^{*}(i)}^{2}+ \max_{\GG \in DS} C(\xbf,\ybf,\sigma^{*})\sum_{kl} x_k y_l \pi_{kl}
\end{equation}
where $C(\xbf,\ybf,\sigma^{*})=2\left(\sum_{i}x_i y_{\sigma^{*}(i)}\right)$. 
\begin{itemize}
\item If $C(\xbf,\ybf,\sigma^{*})>0$ then this (LP) has a unique solution which is the identity $\GG^{*}=\mathbf{I_n}$. This is a consequence of the Rearrangement Inequality (see Memo \ref{memo:rearr}) which states that for all permutations $\sum_{i} x_{i} y_{\sigma(i)} < \sum_{i} x_{i} y_{i}$ (since $x_i$ and $y_j$ are distinct). Using the fact that an optimal solution of (LP) is supported on a permutation concludes.
\item If $C(\xbf,\ybf,\sigma^{*})<0$ then the anti-identity is the unique optimum with the same reasoning since $\sum_i x_i y_{n+1-i} < \sum_{i} x_{i} y_{\sigma(i)}$ for all permutation because of Rearrangement Inequality.
\end{itemize}
\end{proof}

\begin{figure*}[!b]
\begin{memo}[Rearrangement Inequality]
\label{memo:rearr}
Let $x_1 \leq \cdots \leq x_n$, $y_1 \leq \cdots \leq y_n$ then we have:
\begin{equation}
\forall \sigma \in \Sn, \ \sum_{i} x_{i}y_{n+1-i} \leq \sum_{i} x_{i}y_{\sigma(i)} \leq \sum_{i}x_{i}y_{i}
\end{equation}
If the numbers are different then the lower bound (resp upper bound) is attained only for the permutation which reverses the order (resp for the identiy permutation)
\end{memo}
\end{figure*}

Using both results we can prove the following proposition which is the main ingredient to prove Theorem \ref{qap}:

\begin{prop}
\label{prop:twocases}
Let $\xbf,\ybf \in \mathcal{D}$ and $\sigma^{*}$ a solution of \eqref{eq:qap_to_prove} \ie\ $\sigma^{*} \in \argmax_{\sigma \in \Sn} g(\xbf,\ybf,\sigma)$. For $\sigma \in \Sn$ we note $C(\xbf,\ybf,\sigma)=\sum_{i}x_i y_{\sigma(i)}$. 

If $C(\xbf,\ybf,\sigma^{*})>0$ then $\sigma^{*}$ is the identiy permutation $\sigma^{*}(i)=i$ and if $C(\xbf,\ybf,\sigma^{*})<0$ then $\sigma^{*}$ is the anti-identity permutation $\sigma^{*}(i)=n+1-i$ for all $i \in \integ{n}$.
\end{prop}

\begin{proof}
Let $\sigma^{*}$ be an optimal solution of \eqref{eq:qap_to_prove} and $\GG^{*}$ defined by $\pi^{*}_{ij}=1$ if $j=\sigma^{*}(i)$ else $0$. By Lemma \ref{lemma:equivalence_both_probs} we know that $\GG^{*}$ is an optimal solution of \eqref{eq:qp_equiv}. Consider the case $C(\xbf,\ybf,\sigma^{*})>0$. Suppose that $\sigma^{*}$ is not the identity, then $\GG^{*} \neq \mathbf{I_n}$ which is not possible by Lemma \ref{lemma:gg_two_cases} since $\GG^{*}$ is an optimal solution of \eqref{eq:qp_equiv}. Same applies for $C(\xbf,\ybf,\sigma^{*})<0$ and the anti-identity.

\end{proof}

To state that \eqref{eq:qap_to_prove} admits the identity or the anti-identity as optimal permutations we will rely on the previous proposition and on the continuity of the loss $g$:

\begin{lemma}[Continuity of $g$]
\label{lemma:continuity_Z}
Let $\xbf,\ybf \in \mathcal{D}$ fixed. There exists $\epsilon_{x,y}>0$ such that for all $\|\mathbf{h}\|<\epsilon_{x,y}$ we have:
\begin{equation} 
\argmax_{\sigma \in \Sn} g(\xbf+\mathbf{h},\ybf,\sigma) \subset \argmax_{\sigma \in \Sn} g(\xbf,\ybf,\sigma)
\end{equation}
\end{lemma}
\begin{proof}
Let $\xbf,\ybf \in \mathcal D$, $\sigma^{*} \in \argmax_{\sigma \in \Sn} g(\xbf,\ybf,\sigma)$ and $\tau$ any permutation in $\Sn$ such that $\tau \notin \argmax_{\sigma \in \Sn} g(\xbf,\ybf,\sigma)$. Then we have $g(\xbf,\ybf,\sigma^{*})>g(\xbf,\ybf,\tau)$. Let $\eta=g(\xbf,\ybf,\sigma^{*})-g(\xbf,\ybf,\tau)>0$. For all permutation $\beta$ we have that $g(.,\ybf,\beta)$ is continuous. In this way:
\begin{equation}
\label{eq:continuityh}
\begin{split}
&\forall \beta \in \Sn, \exists \epsilon_{\xbf,\ybf}(\beta,\sigma^{*},\tau)>0, \ \forall \|\mathbf{h}\|<\epsilon_{\xbf,\ybf}(\beta,\sigma^{*},\tau), \ |g(\xbf+\mathbf{h},\ybf,\beta)-g(\xbf,\ybf,\beta)|< \frac{\eta}{4} \\
\end{split}
\end{equation}
Let $\mathbf{h}\in \R^{n}$ such that $\|\mathbf{h}\|<\underset{(\beta,\sigma,\tau') \in (\Sn)^{3}}{\min}\epsilon_{\xbf,\ybf}(\beta,\sigma,\tau')$. By \eqref{eq:continuityh} applied to $\sigma^{*}$ and $\tau$:
\begin{equation}
\begin{split}
g(\xbf+\mathbf{h},\ybf,\sigma^{*})-g(\xbf+\mathbf{h},\ybf,\tau)&=g(\xbf+\mathbf{h},\ybf,\sigma^{*})-g(\xbf,\ybf,\sigma^{*})\\
&+g(\xbf,\ybf,\sigma^{*})-g(\xbf,\ybf,\tau)+g(\xbf,\ybf,\tau)-g(\xbf+\mathbf{h},\ybf,\tau) \\
&>-\frac{\eta}{4}+\eta-\frac{\eta}{4} \\
&= \frac{\eta}{2}>0
\end{split}
\end{equation}
So that $g(\xbf+\mathbf{h},\ybf,\sigma^{*})>g(\xbf+\mathbf{h},\ybf,\tau)$ and in this way $\tau \notin \argmax_{\sigma \in \Sn} g(\xbf+\mathbf{h},\ybf,\sigma)$ because $\sigma^{*}$ leads to a striclty better cost. Overall we have proven that for any permutation $\tau$, if $\tau \notin \argmax_{\sigma \in \Sn} g(\xbf,\ybf,\sigma)$ and $\|\mathbf{h}\|<\underset{(\beta,\sigma,\tau') \in (\Sn)^{3}}{\min}\epsilon_{\xbf,\ybf}(\beta,\sigma,\tau')$ then $\tau \notin \argmax_{\sigma \in \Sn} g(\xbf+\mathbf{h},\ybf,\sigma)$ which proves that $\argmax_{\sigma \in \Sn} g(\xbf+\mathbf{h},\ybf,\sigma) \subset \argmax_{\sigma \in \Sn} g(\xbf,\ybf,\sigma)$.
\end{proof}

Using the previous lemma we can now prove the following result:

\begin{lemma}
\label{lemma:endlemma}
Let $\xbf,\ybf \in \mathcal{D}$ fixed. There exists $\epsilonb_{0} \in \R^{n}$ such that:
\begin{equation}
\begin{split}
&\argmax_{\sigma \in \Sn} g(\xbf+\epsilonb_{0},\ybf,\sigma) \subset \argmax_{\sigma \in \Sn} g(\xbf,\ybf,\sigma) \\
&\argmax_{\sigma \in \Sn} g(\xbf+\epsilonb_{0},\ybf,\sigma) \subset \{Id,anti-Id\}
\end{split}
\end{equation}
\end{lemma}

\begin{proof}
Let $\xbf,\ybf \in \mathcal{D}$. We consider $\epsilonb_{0}=(\zeta,-\zeta,0,\dots,0)$ with $\zeta >0$ small enough to ensure $\zeta<\frac{x_2-x_1}{2}$ and $\|\epsilonb_{0}\|< \epsilon_{x,y}$ defined in Lemma \ref{lemma:continuity_Z}. We have $\xbf+\epsilonb_0,\ybf \in \mathcal{D}$ since $\sum_{i} (x_i +\epsilon_0(i))=\sum_{i} x_i+ \zeta-\zeta = 0$ and $x_1+\epsilon_0(1)<\dots<x_N+\epsilon_0(N)$ since $x_1+\zeta < x_2 - \zeta$. 

Let $\sigma_{\epsilonb_{0}}^{*} \in \argmax_{\sigma \in \Sn} g(\xbf+\epsilonb_{0},\ybf,\sigma)$. By Lemma \ref{lemma:continuity_Z} we have $\sigma_{\epsilonb_{0}}^{*} \in \argmax_{\sigma \in \Sn} g(\xbf,\ybf,\sigma)$. 

Moreover we have $C(\xbf+\epsilonb_0,\ybf,\sigma_{\epsilonb_{0}}^{*})=\sum_{i}x_{i}y_{\sigma_{\epsilonb_{0}}^{*}(i)}+\zeta(y_{\sigma_{\epsilonb_{0}}^{*}(0)}-y_{\sigma_{\epsilonb_{0}}^{*}(1)})$.

\begin{itemize}
\item If $\sum_{i}x_{i}y_{\sigma_{\epsilonb_{0}}^{*}(i)}=0$ then $C(\xbf+\epsilonb_0,\ybf,\sigma_{\epsilonb_{0}}^{*})=\zeta(y_{\sigma_{\epsilonb_{0}}^{*}(0)}-y_{\sigma_{\epsilonb_{0}}^{*}(1)}) \neq 0$ since all $y_i$ are distinct. We can apply Proposition \ref{prop:twocases} with $\xbf+\epsilonb_0,\ybf \in \mathcal{D}$ to conclude that $\sigma_{\epsilonb_{0}}^{*}$ is wether the identity or the anti-identity.
\item If $\sum_{i}x_{i}y_{\sigma_{\epsilonb_{0}}^{*}(i)} \neq 0$ then $\sigma_{\epsilonb_{0}}^{*} \in \argmax_{\sigma \in \Sn} g(\xbf,\ybf,\sigma)$ and $C(\xbf,\ybf,\sigma_{\epsilonb_{0}}^{*}) \neq 0$ so by Proposition \ref{prop:twocases} with $\xbf,\ybf \in \mathcal{D}$ we can conclude that $\sigma_{\epsilonb_{0}}^{*}$ is wether the identity or the anti-identity.
\end{itemize}

\end{proof}

\begin{corr}[Theorem \ref{qap} is valid]
\label{prop:zerocase}
Let $\xbf,\ybf \in \mathcal{D}$. The identity or the anti-identity is an optimal solution of \eqref{eq:qap_to_prove}
\end{corr}

\begin{proof}
Let $\xbf,\ybf \in \mathcal{D}$. We consider $\epsilonb_{0}$ defined in Lemma \ref{lemma:endlemma} and $\sigma_{\epsilonb_{0}}^{*} \in \argmax_{\sigma \in \Sn} g(\xbf+\epsilonb_{0},\ybf,\sigma)$. Then by Lemma \ref{lemma:endlemma} $\sigma_{\epsilonb_{0}}^{*}$ is wether the identity or the anti-identity. Moreover by Lemma \ref{lemma:endlemma} $\sigma_{\epsilonb_{0}}^{*} \in \argmax_{\sigma \in \Sn} g(\xbf,\ybf,\sigma)$ so it is an optimal solution of \eqref{eq:qap_to_prove}. This concludes that the identity or the anti-identity is an optimal solution of \eqref{eq:qap_to_prove} which proves Theorem \ref{qap}.
\end{proof}

\subsection{Proof of Theorem \ref{sovable_gw} -- Equivalence between $GM$ and $GW$ in the discrete case }
\label{proof:eq_GM_GW}
This paragraph aims at proving the equivalence between $GM$ and $GW$. We recall the theorem: 

\begin{theo*}[Equivalence between $\gw$ and $\gm$ for discrete measures]
Let $\mu \in \P(\R^{p})$, $\nu \in \P(\R^{q})$ be discrete probability measures with same number of atoms and uniform weights, \ie\ $\mu=\frac{1}{n}\sum_{i=1}^{n}\delta_{\xbf_i},\nu=\frac{1}{n}\sum_{i=1}^{n}\delta_{\ybf_i}$ with $\xbf_i \in \R^{p},\ybf_i \in \R^{q}$.  For $\xbf \in \R^{p}$ we note $\|\xbf\|_{2,p}=\sqrt{\sum_{i=1}^{p} |x_{i}|^{2}}$ the $\ell_{2}$ norm on $\R^{p}$ (same for $\R^{q}$). Let $c_{\Xcal}(\xbf,\xbf')=\|\xbf-\xbf'\|_{2,p}^{2}$ , $c_{\Ycal}(\ybf,\ybf')=\|\ybf-\ybf'\|_{2,q}^{2}$. Then: 
\begin{equation}
\gw_{2}(c_{\Xcal},c_{\Ycal},\mu,\nu)=\gm_{2}(c_{\Xcal},c_{\Ycal},\mu,\nu)
\end{equation}

Moreover if $p=q=1$, \ie\ $c_{\Xcal}(x,x')=c_{\Ycal}(x,x')=|x-x'|^{2}$ for $x,x' \in \R$, and if $x_{1} < \dots < x_{n}$ and $y_{1} < \dots < y_{n}$ the optimal values are achieved by considering either the identity or the anti-identity permutation.
\end{theo*} 

\begin{proof}
The proof is essentially based on theoretical results from \cite{NIPS2018_7323} and on Theorem \ref{qap}. In \cite{NIPS2018_7323} authors consider the minimizing energy problem $\underset{\X \in \Pi_n}{\min} -\tr(\mathbf{B}\X^{T}\mathbf{A}\X)$ where $\Pi_n$ the set of permutation matrices. In fact, the $GM$ problem defined in this chapter is equivalent to $\underset{\X \in \Pi_n}{\min} -\tr(\mathbf{B}\X^{T}\mathbf{A}\X)$ by considering $\mathbf{A}=(\|\xbf_{i}-\xbf_{j}\|_{2,p}^{2})_{i,j}$ and $\mathbf{B}=(\|\ybf_{i}-\ybf_{j}\|_{2,q}^{2})_{i,j}$.

To tackle this problem authors propose to minimize $-\tr(\mathbf{B}\X^{T}\mathbf{A}\X)$ over the set of doubly stochastic matrices (which is the convex-hull of $\Pi_n$): 
\begin{equation*}
DS=\{\X\in \mathbb{R}^{n\times n} \ \text{s.t.} \ \X \one_n=\X^{T}\one_n=\one_n \ , \X \geq 0\}
\end{equation*}
Minimizing $-\tr(\mathbf{B}\X^{T}\mathbf{A}\X)$ over $DS$ is equivalent to solving the $GW$ distance when $a_{i}=b_{j}=\frac{1}{n}$. The paper proves that when both $\mathbf{A}$ and $\mathbf{B}$ are conditionally positive (or negative) definite of order 1 then the relaxation leads to the same optimum so that the minimum over $DS$ is the same as the minimum over $\Pi_n$ \cite[Theorem 1]{NIPS2018_7323}. Yet $\mathbf{A}$ and $\mathbf{B}$ defined previously satisfy this property (see examples under Definition 2 in \cite{NIPS2018_7323}) and so $GW$ and $GM$ coincide.

Moreover when $p=q=1$ and when the sample are sorted we can apply Theorem \ref{qap} to prove that an optimal permutation of the $GM$ problem is found whether at the identity or the anti-identity permutation which concludes the proof.

\end{proof}

\subsection{Computing $GW$ in the 1d case \label{proof:computing_gw}}

We recall the result:
\begin{lemma*}
The $\gm_2$ and $\gw_2$ costs in 1D with same numbers of atoms and uniform weights can be computed in $O(n)$.
\end{lemma*}
\begin{proof}
As seen in Theorem \ref{sovable_gw} finding the optimal permutation $\sigma^{*}$ can be found in $O(n\log(n))$. Moreover the final costs can be written using binomial expansion:
\begin{equation}
\begin{split}
\sum_{i,j} \big((x_{i}-x_{j})^{2} - (y_{\sigma^{*}(i)}-y_{\sigma^{*}(j)})^{2}\big)^{2}&=2n\sum_{i} x_{i}^{4} - 8 \sum_{i} x_{i}^{3} \sum_{k} x_{k} + 6 (\sum_{i} x_{i}^{2})^{2} \\
&+2n\sum_{i} y_{i}^{4}- 8 \sum_{i} y_{i}^{3} \sum_{k} y_{k} + 6 (\sum_{i} y_{i}^{2})^{2}\\
&-4(\sum_{i}x_{i})^{2}(\sum_{k}y_{k})^{2} \\
&-4n\sum_{i}  x_{i}^{2}y_{\sigma^{*}(i)}^{2}+ 8 \sum_{i}((\sum_{k} x_{k}) x_{i} y_{\sigma^{*}(i)}^{2} + (\sum_{k} y_{k}) x_{i}^{2} y_{\sigma^{*}(i)}) \\
&- 8(\sum_{i} x_{i} y_{\sigma^{*}(i)})^{2}
\end{split}
\label{eq:linear:computation}
\end{equation}

which can be computed in $O(n)$ operations.
\end{proof}

\subsection{Proof of Theorem \ref{propertiessgw} -- Properties of \sgw \label{proof:prop_sgw}}
We recall the theorem:
\begin{theo*}[Properties of $\sgw$]
\begin{itemize}
\item For all $\D$, $\sgw_{\D}$ and $\risgw$ are translation invariant. $\risgw$ is also rotational invariant when $p=q$, more precisely if $\Qbf \in \mathcal{O}(p)$ is an orthogonal matrix, $\risgw(\Qbf\#\mu,\nu)=\risgw(\mu,\nu)$ (same for any $\Qbf' \in \mathcal{O}(q)$ applied on $\nu$).
\item $\sgw$ and $\risgw$ are pseudo-distances on $\Pm(\R^{p})$, \textit{i.e.} they are symmetric, satisfy the triangle inequality and $\sgw(\mu,\mu)=\risgw(\mu,\mu)=0$ .
\item Let $\mu,\nu \in \Pm(\R^{p})\times \Pm(\R^{p})$ be probability distributions with \emph{compact supports}. If $\sgw(\mu,\nu)=0$ then $\mu$ and $\nu$ are isomorphic for the distance induced by the $\ell_{1}$ norm on $\R^{p}$, \ie\ $d(x,x')=\sum_{i=1}^{p} |x_{i}-x_{i}'|$ for $(x,x') \in \R^{p} \times \R^{p}$. In particular this implies:
\begin{equation}
\sgw(\mu,\nu)=0 \implies \gw_{2}(d,d,\mu,\nu)=0
\end{equation} 
\end{itemize}
\end{theo*}

The invariance by translation is clear since the costs are invariant by translation of the support of the measures. The pseudo-distances properties are straightforward thanks to the properties of $\gw$. For the invariance by rotation if  $p=q$ then $\mathbb{V}_{p}(\R^{p})$ is bijective with $\mathcal{O}(p)$ so for $Q \in \mathcal{O}(p)$:

\begin{equation}
\begin{split}
\risgw(Q\#\mu,\nu)&=\underset{\D \in \mathbb{V}_{p}(\R^{p})}{\min}\sgw_{\D}(Q\#\mu,\nu) \\
&=\underset{\D \in \mathcal{O}(p)}{\min}\sgw_{\D}(Q\#\mu,\nu) \\
&= \underset{\D \in \mathcal{O}(p)}{\min} \underset{\theta \sim \lambda_{q-1}}{\E}[\gw(d^{2},P_{\theta}\#(\D Q\#\mu),P_{\theta}\#\nu)] \\
&= \underset{\D' \in \mathcal{O}(p)}{\min} \underset{\theta \sim \lambda_{q-1}}{\E}[\gw(d^{2},P_{\theta}\#\D'\#\mu,P_{\theta}\#\nu)] \\
&= \risgw(\mu,\nu)
\end{split}
\end{equation}

On the other side for $\nu$ a change of formula on theta gives the result.

For the case $SGW=0 \implies GW=0$ it will be a consequence of the following theorem:

\begin{theo}[]
\label{cramer_gene}
Let $\mu,\nu \in \Pm(\R^{p})\times \Pm(\R^{p})$ be probability distributions such that $\mu,\nu$ have compact supports. If for almost all $\thetab \in \Sp^{p-1}$, $P_{\thetab}\#\mu$ and $P_{\thetab}\#\nu$ are isomoprhic then $\mu$ and $\nu$ are isomoprhic. In other words if for almost all $\thetab \in \Sp^{p-1}$ we have:
\begin{equation}
\begin{split}
&\exists T_{\thetab} : \supp(P_{\thetab}\#\mu) \subset \R \mapsto \supp(P_{\thetab}\#\nu) \subset \R,  \ \text{surjective} \ \text{s.t.} \ T_{\thetab} \# (P_{\thetab}\#\mu)= P_{\thetab}\#\nu \\
&\forall x,x' \in \supp(P_{\thetab}\#\mu), |T_{\thetab}(x)-T_{\thetab}(x')|=|x-x'|
\end{split}
\end{equation}
Then there exists a measure preserving isometry $f$ between $\supp(\mu)$ and $\supp(\nu)$. More precisely we have $f\#\mu=\nu$ and: 
\begin{equation}
\forall \xbf,\xbf' \in \text{supp}(\mu), \|f(\xbf)-f(\xbf')\|_1=\|\xbf-\xbf'\|_1
\end{equation}
\end{theo}

To prove this theorem we will exhibit the isometry. This result can be put in light of Cramer–Wold theorem \cite{cramer} which states that a probability measure is uniquely determined by the totality of its one-dimensional projections. Equivalently, if we consider two probability measures so that the one-dimensional measures resulting from the projections over all the hypersphere are equal then the measures are equal. The equality relation is replaced in our theorem by the isomoprhism relation.

The proof is divided into four parts. In the first one, we construct an "almost orthogonal" basis on which measures are isomorphic. Building upon this result we define a sequence of functions from $\text{supp}(\mu)$ to $\text{supp}(\nu)$ and show that it has a convergent subsequence. We conclude by proving that the limit of the subsequence is actually a good candidate for being the isometry we are looking for. In the following $\|.\|_{1}$ denotes the $\ell_{1}$ norm, $\|.\|_{2}$ denotes the $\ell_{2}$ norm and $p\geq 2$. We recall that $\mathcal{F}_{\mu}$ is the Fourier transform of $\mu$. 

\smallskip
We consider the following $\mathcal{Q}_{\thetab}$ property for $\thetab \in \Sp^{p-1}$:
\begin{equation}\tag{$\mathcal{Q}_{\thetab}$}
\label{conserv}
\begin{split}
&\exists T_{\thetab} : \supp(P_{\thetab}\#\mu) \subset \R \mapsto \supp(P_{\thetab}\#\nu) \subset \R,  \ \text{surjective} \ \text{s.t.} \ T_{\thetab} \# (P_{\thetab}\#\mu)= P_{\thetab}\#\nu \\
&\forall x,x' \in \text{supp}(P_{\thetab}\#\mu), |T_{\thetab}(x)-T_{\thetab}(x')|=|x-x'|
\end{split}
\end{equation}
Informally if we have the $\mathcal{Q}_{\thetab}$ property for $\thetab \in \Sp^{p-1}$ it implies that $\mu$ and $\nu$ are isomorphic on the 1D line given by the projection \textit{w.r.t.} $\thetab$. We have the following lemma:
\begin{lemma}
\label{allmost_orthogonal_basis}
Let $\mu,\nu \in \Pm(\R^{p})\times \Pm(\R^{p})$ and suppose that $\mathcal{Q}_{\thetab}$ holds for almost all $\thetab \in \Sp^{p-1}$. Let $n>p-1$. There exists a basis $(\ebf_{1}(n),...,\ebf_{p}(n))$ of $\R^{p}$ part of the following spaces:

\begin{equation}
\mathcal{B}^{n}_{p}\stackrel{def}{=}\{(\thetab_{1},...,\thetab_{p}) \in (\Sp^{p-1})^{p} \ \text{s.t.} \ |\langle \thetab_{i},\thetab_{j} \rangle | < \frac{1}{n}\}
\end{equation}
and 
\begin{equation}
Q\stackrel{def}{=}\{(\thetab_{1},...,\thetab_{p}) \in (\Sp^{p-1})^{p} \ \text{s.t.} \ \forall i \in \{1,...,p\}, \mathcal{Q}_{\thetab_{i}} \}
\end{equation}

\end{lemma}

\begin{proof}
We want to construct a basis $(\ebf_{1},...,\ebf_{p})$ as orthogonal as possible such that for all $i$ we have $\mathcal{Q}_{\ebf_{i}}$.

We note $\lambda^{\otimes p}_{p-1}$ the product measure $\lambda_{p-1}\otimes ... \otimes \lambda_{p-1}$ where $\lambda_{p-1}$ is the uniform measure on the sphere. $\mathcal{B}^{n}_{p}$ is an open set as inverse image by a continuous function of an open set. Then $\lambda^{\otimes p}_{p-1}(\mathcal{B}^{n}_{p})>0$. Moreover, since for almost all 
$\thetab \in \Sp^{p-1}$ we have $\mathcal{Q}_{\thetab}$ then $\lambda^{\otimes p}_{p-1}(Q)>0$ and so $\lambda^{\otimes p}_{p-1}(\mathcal{B}^{n}_{p} \cap Q) >0$. 

In this way we can consider $(\ebf_{1}(n),...,\ebf_{p}(n)) \in \mathcal{B}^{n}_{p} \cap Q$. If $n>p-1$ the Gram matrix of $(\ebf_{1}(n),...,\ebf_{p}(n))$ is strictly diagonal dominant, thus invertible, such that $(\ebf_{1}(n),...,\ebf_{p}(n))$ is a basis. Note that we can not consider directly an orthogonal basis since the set of all orthogonal basis has measure zero. 
\end{proof}

We now express all the vectors and inner products in this new almost orthogonal basis as expressed in the following lemma:

\begin{lemma}
\label{write_in_new}
Let $n>p-1$ and a basis $(\ebf_{1}(n),...,\ebf_{p}(n))$ as defined in Lemma \ref{allmost_orthogonal_basis}. Then all $\xbf \in \R^{p}$ can be written as:
\begin{equation}
\xbf= \sum_{i=1}^{p} [\langle \xbf,\ebf_{i}(n) \rangle+ R(\xbf,\ebf_{i}(n))] \ebf_{i}(n)
\end{equation}
where $|R(\xbf,\ebf_{i}(n))|=o(\frac{1}{n})$. Moreover for all $(\xbf,\ybf) \in \R^{p} \times \R^{p}$:
\begin{equation}
\langle \xbf,\ybf\rangle=\sum_{i=1}^{p} \langle \xbf,\ebf_{i}(n)\rangle \langle \ybf,\ebf_{i}(n)\rangle+\tilde{R}(\xbf,\ybf)
\end{equation} 
where $|\tilde{R}(\xbf,\ybf)|=o(\frac{1}{n})$.
\end{lemma}

\begin{proof}
In the following $x_i$ denotes the $i$-th coordinate of a vector $\xbf$ in the standard basis, \ie\ a vector writes $\xbf=(x_1,\dots,x_p)$. For $\xbf\in \R^{p}$, we can write in the new basis $\xbf= \sum_{i=1}^{p} [\langle \xbf,\ebf_{i}(n) \rangle+ R(\xbf,\ebf_{i}(n))] \ebf_{i}$ with $R(\xbf,\ebf_{i}(n))\stackrel{def}{=}x_{i}-\langle \xbf,\ebf_{i}(n) \rangle$. We have also $|R(\xbf,\ebf_{i}(n))|=o(\frac{1}{n})$. Indeed, 
\begin{equation*}
\begin{split}
\xbf=\sum_{i=1}^{p} x_i\ebf_{i} &\implies \forall j, \langle \xbf, \ebf_{j} \rangle= \sum_{i=1}^{p} x_{i} \langle \ebf_{i},\ebf_{j}\rangle \implies x_{j} - \langle \xbf, \ebf_{j} \rangle =\sum_{i \neq j} x_{i} \langle \ebf_{i},\ebf_{j}\rangle \\
&\implies |R(\xbf,\ebf_{j}(n))| = | \sum_{i \neq j} x_{i} \langle \ebf_{i},\ebf_{j}\rangle |  \implies |R(\xbf,\ebf_{j}(n))| \leq \frac{1}{n} \sum_{i \neq j} |x_{i}| 
\end{split}
\end{equation*}
Also in the same way for $\xbf,\ybf \in \R^{p} \times \R^{p}$ we can rewrite their inner product:
\begin{equation}
\label{doublescalr}
\langle \xbf,\ybf\rangle=\sum_{i=1}^{p} \langle \xbf,\ebf_{i}(n)\rangle \langle \ybf,\ebf_{i}(n)\rangle+\tilde{R}(\xbf,\ybf)
\end{equation} 
with:
\begin{equation*}
\begin{split}
\tilde{R}(\xbf,\ybf)&\stackrel{def}{=}\langle \xbf,\ybf\rangle - \sum_{i=1}^{p} \langle \xbf,\ebf_{i}(n)\rangle \langle \ybf,\ebf_{i}(n)\rangle \\
&= \sum_{i\neq j} \langle \xbf,\ebf_{i}(n)\rangle \langle \ybf,\ebf_{i}(n)\rangle \langle \ebf_{j}(n),\ebf_{i}(n)\rangle + \sum_{i,j}  \langle \xbf,\ebf_{i}(n)\rangle R(\ybf,\ebf_{j}(n)) \langle \ebf_{j}(n),\ebf_{i}(n)\rangle \\
&+  \sum_{i,j}  \langle \ybf,\ebf_{j}(n)\rangle R(\xbf,\ebf_{i}(n)) \langle \ebf_{j}(n),\ebf_{i}(n)\rangle + \sum_{i,j} R(\xbf,\ebf_{j}(n)) R(\ybf,\ebf_{i}(n))\langle \ebf_{j}(n),\ebf_{i}(n)\rangle 
\end{split}
\end{equation*}
and with the same calculus than for $R$ we have $|\tilde{R}(\xbf,\ybf)|=o(\frac{1}{n})$.
\end{proof}

\begin{prop}
\label{thebig_prop}
Let $\mu,\nu \in \Pm(\R^{p})\times \Pm(\R^{p})$ and suppose that $\mathcal{Q}_{\thetab}$ holds for almost all $\thetab \in \Sp^{p-1}$ and that $\nu$ has compact support. There exists a sequence $(f_{n})_{n\in \mathbb{N}}$ from $\supp(\mu)$ to $\supp(\nu)$ uniformly bounded which satisfies:

\begin{equation}
\label{close}
\forall n \in \mathbb{N}, \forall \xbf,\xbf' \in \supp(\mu)^{2}, \ \big| \|f_{n}(\xbf)-f_{n}(\xbf')\|_1-\|\xbf-\xbf'\|_1 \big| =  o(\frac{1}{n})
\end{equation}

\begin{equation}
\label{fourr}
\forall n \in \mathbb{N}, \forall \sbf \in \R^{p}, \ |\mathcal{F}_{f_{n}\#\mu}(\sbf)-\mathcal{F}_{\nu}(\sbf)|=o(\frac{1}{n})
\end{equation}

\end{prop}

\begin{proof}
In the following $x_i$ denotes the $i$-th coordinate of a vector $\xbf$ in the standard basis, \ie\ a vector writes $\xbf=(x_1,\dots,x_p)$. We define:
\begin{equation}
\label{fn}
\forall n >p-1, \ \forall \xbf \in \text{supp}(\mu), \ f_{n}(\xbf)=(T_{\ebf_{1}(n)}(\langle \xbf,\ebf_{1}(n)\rangle),...,T_{\ebf_{p}(n)}(\langle \xbf,\ebf_{p}(n)\rangle))
\end{equation}
where $(\ebf_k(n))_{k\in \integ{p}}$ is the almost orthogonal basis define in Lemma \ref{allmost_orthogonal_basis}, and $T_{\ebf_{k}(n)}$ is defined from \eqref{conserv} since we have $\mathcal{Q}_{\ebf_{k}(n)}$ for all $k$. It is clear from the definition that $f_{n}(\xbf) \in \supp(\nu)$. Moreover for $\xbf,\xbf' \in \supp(\mu)$:
\begin{equation*}
\begin{split}
\|f_{n}(\xbf)-f_{n}(\xbf')\|_1&= \sum_{k=1}^{p} |T_{\ebf_{k}(n)}(\langle \xbf,\ebf_{k}(n)\rangle)-T_{\ebf_{k}(n)}(\langle \xbf',\ebf_{k}(n)\rangle) | \stackrel{(*)}{=} \sum_{k=1}^{p} |\langle \xbf,\ebf_{k}(n)\rangle-\langle \xbf',\ebf_{k}(n)\rangle| \\
&= \sum_{k=1}^{p} |\langle \xbf-\xbf',\ebf_{k}(n)\rangle| \\
\end{split}
\end{equation*}
where in (*) we used that $T_{\ebf_{k}(n)}$ is an isometry since we have $\mathcal{Q}_{\ebf_{k}(n)}$ and $\langle \xbf,\ebf_{k}(n)\rangle \in \supp(P_{\ebf_{k}(n)} \# \mu)$ (idem for $\xbf'$). 
In this way: 
\begin{equation*}
\begin{split}
\big| \|f_{n}(\xbf)-f_{n}(\xbf')\|_1-\|\xbf-\xbf'\|_1 \big|&= \big| \sum_{k=1}^{p} |\langle \xbf-\xbf',\ebf_{k}(n)\rangle| -|x_{k}-x_{k}'|  \big| \leq \sum_{k=1}^{p}\big| |\langle \xbf-\xbf',\ebf_{k}(n)\rangle| -|x_{k}-x_{k}'| \big| \\
&\stackrel{*}{\leq} \sum_{k=1}^{p}|\langle \xbf-\xbf',\ebf_{k}(n)\rangle -(x_{k}-x_{k}')| = \sum_{k=1}^{p}|R(\xbf-\xbf',\ebf_{k}(n))|=o(\frac{1}{n}) \\
\end{split}
\end{equation*}
where in (*) the second triangular inequality $| |x|- |y| |\leq |x-y|$. Hence: 
\begin{equation}
\big| \|f_{n}(\xbf)-f_{n}(\xbf')\|_1-\|\xbf-\xbf'\|_1 \big| =o(\frac{1}{n})
\end{equation}
Moreover we have by definition of the Fourier transform, for $s \in \R^{P}$, 
\begin{equation*}
\label{batmanbegin}
\begin{split}
\mathcal{F}_{f_{n}\#\mu}(\sbf)&= \int e^{-2i\pi\langle \sbf,f_{n}(\xbf)\rangle} d\mu(\xbf) = \int e^{-2i\pi \sum_{k=1}^{p} s_{k} T_{\ebf_{k}(n)}(\langle \xbf,\ebf_{k}(n)\rangle)} d\mu(\xbf) 
\end{split}
\end{equation*}
Moreover using \eqref{conserv} we have $\mathcal{F}_{T_{\ebf_{k}(n)}\#(P_{\ebf_{k}(n)}\#\mu)}(t)=\mathcal{F}_{P_{\ebf_{k}(n)}\#\nu}(t)$ for all $k\in \{1,...,p\}$, and any real $t\in \R$. This implies $\int e^{-2i\pi t.T_{\ebf_{k}(n)}(\langle \ebf_{k}(n) ,\xbf\rangle)} \dr\mu(\xbf) =  \int e^{-2i\pi t \langle \ebf_{k}(n) ,\ybf\rangle} \dr\nu(\ybf)$. So by applying this results for $t=s_{k}$ we have: 
\begin{equation}
\label{equivfourr2}
\int e^{-2i\pi  s_{k}  T_{\ebf_{k}(n)}(\langle \xbf,\ebf_{k}(n)\rangle)} d\mu(\xbf) =  \int e^{-2i\pi s_{k}  \langle \ebf_{k}(n),y\rangle} d\nu(\ybf)
\end{equation}
Combining both results: 
\begin{equation}
\label{fourr_}
\mathcal{F}_{f_{n}\#\mu}(\sbf)=  \int e^{ -2i\pi \sum_{k=1}^{p} s_{k}  \langle \ebf_{k}(n),\ybf\rangle} \dr\nu(\ybf)
\end{equation}
We can now bound $|\mathcal{F}_{f_{n}\#\mu}(\sbf)-\mathcal{F}_{\nu}(\sbf)|$ as:
\begin{equation*}
\begin{split}
|\mathcal{F}_{f_{n}\#\mu}(\sbf)-\mathcal{F}_{\nu}(\sbf)|&= | \mathcal{F}_{f_{n}\#\mu}(\sbf)- \int e^{-2i\pi\langle \sbf,\ybf\rangle} \dr\nu(\ybf) |\\
&\stackrel{*}{=} | \mathcal{F}_{f_{n}\#\mu}(\sbf) - \int e^{-2i\pi [\sum_{k=1}^{p} \langle \sbf,\ebf_{k}(n) \rangle \langle \ebf_{k}(n),\ybf\rangle +\tilde{R}(\sbf,\ybf)]} \dr\nu(\ybf) | \\
&\stackrel{**}{=} |  \int e^{-2i\pi \sum_{k=1}^{p} s_{k}  \langle \ebf_{k}(n),\ybf\rangle} \dr\nu(\ybf) - \int e^{-2i\pi\tilde{R}(\sbf,\ybf)}e^{-2i\pi \sum_{k=1}^{p} \langle \sbf,\ebf_{k}(n) \rangle \langle \ebf_{k}(n),\ybf\rangle } \dr\nu(\ybf) | \\
\end{split}
\end{equation*}
where in (*) we used the expression in the new base of the inner product $\langle \sbf,\ybf\rangle$ seen in Lemma \ref{write_in_new}, in (**) we used \eqref{fourr_}. By injecting the expression of $s_{k}$ \textit{w.r.t.} the new base we have:
\begin{equation}
\begin{split}
|\mathcal{F}_{f_{n}\#\mu}(\sbf)-\mathcal{F}_{\nu}(\sbf)|&\leq |  \int e^{-2i\pi \sum_{k=1}^{p} (\langle \sbf,\ebf_{k}(n) \rangle +R(\sbf,\ebf_{k}(n)))  \langle \ebf_{k}(n),\ybf\rangle} \dr\nu(\ybf) - \int e^{-2i\pi\tilde{R}(\sbf,\ybf)}e^{-2i\pi \sum_{k=1}^{p} \langle s,\ebf_{k}(n) \rangle \langle \ebf_{k}(n),y\rangle } \dr\nu(\ybf) | \\
&=\big| \int e^{-2i\pi \sum_{k=1}^{p} \langle \sbf,\ebf_{k}(n) \rangle \langle \ebf_{k}(n),\ybf\rangle} (e^{-2i\pi \sum_{k=1}^{p} R(\sbf,\ebf_{k}(n)) \langle \ebf_{k}(n),\ybf\rangle}-e^{-2i\pi\tilde{R}(\sbf,\ybf)}) \dr\nu(\ybf) \big| \\
&\leq \int |e^{-2i\pi \sum_{k=1}^{p} R(\sbf,\ebf_{k}(n)) \langle \ebf_{k}(n),\ybf\rangle}-e^{-2i\pi\tilde{R}(\sbf,\ybf)}| \dr\nu(\ybf) \\
&= \int |e^{-2i\pi\tilde{R}(\sbf,\ybf)} (e^{-2i\pi (\sum_{k=1}^{p} R(\sbf,\ebf_{k}(n)) \langle \ebf_{k}(n),y\rangle-\tilde{R}(\sbf,\ybf))}-1)|\dr\nu(\ybf) \\
&\leq \int |e^{-2i\pi (\sum_{k=1}^{p} R(\sbf,\ebf_{k}(n)) \langle \ebf_{k}(n),\ybf\rangle-\tilde{R}(\sbf,\ybf))}-1| \dr\nu(\ybf) \\
&= \int | 2ie^{-i\pi (\sum_{k=1}^{p} R(\sbf,\ebf_{k}(n)) \langle \ebf_{k}(n),\ybf\rangle-\tilde{R}(\sbf,\ybf))} \sin(\pi (\sum_{k=1}^{p} R(\sbf,\ebf_{k}(n)) \langle \ebf_{k}(n),\ybf\rangle-\tilde{R}(\sbf,\ybf)) | \dr\nu(\ybf) \\
&\leq  \int | \sin(\pi (\sum_{k=1}^{p} R(\sbf,\ebf_{k}(n)) \langle \ebf_{k}(n),\ybf\rangle-\tilde{R}(\sbf,\ybf)) | \dr\nu(\ybf) \\
&\leq \pi \int ( \sum_{k=1}^{p} |R(\sbf,\ebf_{k}(n)) \langle \ebf_{k}(n),\ybf\rangle|+|\tilde{R}(\sbf,\ybf)| )\dr\nu(\ybf) \\
&\stackrel{*}{=}o(\frac{1}{n})
\end{split}
\end{equation}
in (*) the fact that each term is $o(\frac{1}{n})$. In this way: 
\begin{equation}
|\mathcal{F}_{f_{n}\#\mu}(s)-\mathcal{F}_{\nu}(s)| =o(\frac{1}{n})
\end{equation}

Moreover $(f_{n})_{n > p-1}$ is also uniformly bounded. To see that we consider $\xbf \in \text{supp}(\mu)$. We have that for all $k \in \integ{p}$ $T_{\ebf_{k}(n)}(\langle \xbf,e_{k}(n)\rangle) \in \supp(P_{\ebf_{k}(n)}\#\nu)$ by definition of $T_{\ebf_{k}(n)}$. So there exists a $\ybf_{0}(\xbf,n,k) \in \text{supp}(\nu)$ such that $T_{\ebf_{k}(n)}(\langle \xbf,\ebf_{k}(n)\rangle)=\langle \ybf_{0}(x,n,k),\ebf_{k}(n)\rangle$. In this way $|T_{\ebf_{k}(n)}(\langle \xbf,\ebf_{k}(n)\rangle)|=|\langle \ybf_{0}(x,n,k),\ebf_{k}(n)\rangle| \leq \|\ybf_{0}(x,n,k)\|_{2} \|\ebf_{k}(n)\|_2$ by Cauchy-Swartz. 

Moreover $\|\ebf_{k}(n)\|_2 < \sqrt{\frac{1}{n}} \leq \sqrt{\frac{1}{p-1}} \leq 1$ and since $\nu$ has compact support then there is a constant $M_{\nu}$ we have $\|\ybf_{0}(x,n,k)\|_{2}\leq M_{\nu}$

So we have for $n\in \mathbb{N}$, $\xbf \in \text{supp}(\mu)$, 
\begin{equation*}
 \begin{split}
\|f_{n}(\xbf)\|_{2}^{2}&=\sum_{k=1}^{p} |T_{\ebf_{k}(n)}(\langle \xbf,\ebf_{k}(n)\rangle)|^{2} \leq p M_{\nu}
\end{split}
\end{equation*}
Since on $\mathbb{R}^{p}$ all norms are equivalent this suffices to state the existence of a constant $C$ such that $\forall \xbf \in \mathbb{R}^{p}, n \in \mathbb{N}, \|f_{n}(\xbf)\|_1\leq C$ so that $(f_n)_{n\in \mathbb{N}}$ is uniformly bounded. Reindexing $(f_n)_{n> p-1}$ gives the desired result.

\end{proof}

We can now prove Theorem \ref{cramer_gene}.

\begin{proof}[Proof of Theorem \ref{cramer_gene}]
We consider the sequence $(f_{n})_{n \in \mathbb{N}}$ defined in Proposition \ref{thebig_prop}. We will show that $(f_{n})_{n \in \mathbb{N}}$ is equicontinuous. Let $\epsilon >0$, using \eqref{close} there exists a $N\in \mathbb{N}$ such that  we have for all $\xbf,\xbf' \in \text{supp}(\mu)$:
\begin{equation} 
\|f_{n}(\xbf)-f_{n}(\xbf')\|_1 \leq  \epsilon +\|\xbf-\xbf'\|_1 \ \text{ for all } n \geq N
\end{equation}

Now let $\delta < \epsilon$. Suppose that $\|\xbf-\xbf'\|_1 < \delta$ then 

\begin{equation} 
\|f_{n}(x)-f_{n}(x')\|_1 <  \epsilon +\delta < 2 \epsilon \ \text{ for all } n \geq N
\end{equation}

Without loss of generality we can reindex $(f_{n})_{n \in \mathbb{N}}$ for $n$ large enough ($n\geq N$) so that $(f_{n})_{n \in \mathbb{N}}$ is equicontinuous with the previous argument.

Since $(f_{n})_{n \in \mathbb{N}}$ is a uniformly bounded and equicontinuous sequence from the support of $\mu$ which is compact to $\R^{p}$ we can apply Arzela-Ascoli theorem (see Memo \ref{memo:Arzela}) which states that $(f_{n})_{n \in \mathbb{N}}$ has a uniformly convergent subsequence. We denote by $(f_{\phi(n)})_n$ this sequence. We have $f_{\phi(n)} \underset{n \to \infty}{\underset{\rightarrow}{u}} f$ this sequence.

Moreover equation \eqref{fourr} states that for all $\sbf \in \R^{p}$, $\mathcal{F}_{f_{n}\#\mu}(\sbf) \underset{n \to \infty}{\rightarrow} \mathcal{F}_{\nu}(\sbf)$. In this way $(\mathcal{F}_{f_{n}\#\mu}(\sbf))_{n \in \mathbb{N}}$ is a convergent real valued sequence, so every adherence value goes to the same limit, hence $\mathcal{F}_{f_{\phi(n)}\#\mu}(\sbf) \underset{n \to \infty}{\rightarrow} \mathcal{F}_{\nu}(\sbf)$.

Moreover the function $f$ is a measure preserving isometry from $\supp(\mu)$ to $\supp(\nu)$. Indeed let $\epsilon_{1} >0, \sbf\in \R^{p}$, there exists from previous statements $N_{0},N_{1} \in \mathbb{N}$ such that for $n\geq N_{0}$, $|\mathcal{F}_{f_{\phi(n)}\#\mu}(\sbf)-\mathcal{F}_{\nu}(\sbf)| < \epsilon_{1}$ and $n\geq N_{1}$, $|\mathcal{F}_{f_{\phi(n)}\#\mu}(\sbf)-\mathcal{F}_{f\#\mu}(\sbf)| < \epsilon_{1}$. Let $n \geq \text{max}(N_{0},N_{1})$
\begin{equation*}
\begin{split}
|\mathcal{F}_{f\#\mu}(\sbf)-\mathcal{F}_{\nu}(\sbf)| &\leq |\mathcal{F}_{f_{\phi(n)}\#\mu}(\sbf)-\mathcal{F}_{\nu}(\sbf)| + |\mathcal{F}_{f_{\phi(n)}\#\mu}(s)-\mathcal{F}_{f\#\mu}(\sbf)|\\
&< 2 \epsilon_{1}
\end{split}
\end{equation*}
As this result holds for any $\epsilon_{1} >0$ we have $\mathcal{F}_{f\#\mu}(\sbf)=\mathcal{F}_{\nu}(\sbf)$ and by injectivity of the Fourrier transform $f\#\mu=\nu$ such that $f$ is measure preserving. 

In the same way for any $\xbf,\xbf' \in \text{supp}(\mu), \epsilon_{2}>0$ and $n$ large enough:  
\begin{equation*}
\begin{split}
\big| \|f(\xbf)-f(\xbf')\|_1-\|\xbf-\xbf'\|_1 \big| &\leq \big| \|f_{\phi(n)}(\xbf)-f_{\phi(n)}(\xbf')\|_1-\|f(\xbf)-f(\xbf')\|_1 \big| \\
&+ \big|\|f_{\phi(n)}(\xbf)-f_{\phi(n)}(\xbf')\|_1-\|\xbf-\xbf'\|_1 \big|< 2 \epsilon_{2}
\end{split}
\end{equation*}
using $f_{\phi(n)} \underset{n \to \infty}{\underset{\rightarrow}{u}} f$ and \eqref{close}. As this result holds true for any $\epsilon_{2} >0$ we have $\|f(\xbf)-f(\xbf')\|=\|\xbf-\xbf'\|$ for any $\xbf,\xbf' \in \text{supp}(\mu)$ which concludes.

\end{proof}

\begin{figure*}[!b]
\begin{memo}
\label{memo:Arzela}
Let $(\Xcal,d)$ be a compact metric space and $\|.\|$ a norm on $\mathbb{R}^{p}$. We say taht:
\begin{itemize}
\item A family $\mathcal{F} \subset C(\Xcal,\mathbb{R}^{p})$  is \emph{bounded} means if there exists a positive constant $M<\infty$ such that $\|f(x)\| \leq M$ for all $x \in \Xcal$ and  $f \in \mathcal{F}$
\item A family $\mathcal{F} \subset C(X,\mathbb{R}^{p})$  is \emph{equicontinuous} means if for every  $\epsilon>0$ there exists $\delta>0$ (which depends only on  $\epsilon$) such that for  $x, y \in \Xcal$: 
\begin{equation}
d(x, y)<\delta \Rightarrow\|f(x)-f(y)\|<\epsilon \quad \forall f \in \mathcal{F}
\end{equation}
\end{itemize}
If $(f_{n})_{n\in \mathbb{N}}$ is a sequence in $C(X,\mathbb{R}^{p})$ that is bounded and equicontinuous then the Arzela-Ascoli states that it has a uniformly convergent subsequence (see Theorem 7.25 \cite{rudin_principles_1976}) .
\end{memo}
\end{figure*}

\begin{corr}
Let $\mu,\nu \in \Pm(\R^{p})\times \Pm(\R^{p})$ with compact support. If $\sgw(\mu,\nu)=0$ then $\mu$ and $\nu$ are isomorphic for the distance induced by the $\ell_{1}$ norm on $\R^{p}$, \ie\ $d(\xbf,\xbf')=\sum_{i=1}^{p} |x_{i}-x_{i}'|$ for $(\xbf,\xbf') \in \R^{p} \times \R^{p}$. In particular this implies:
\begin{equation}
\sgw(\mu,\nu)=0 \implies \gw_{2}(d,d,\mu,\nu)=0
\end{equation} 

\end{corr}
\begin{proof}
If $\sgw(\mu,\nu)=0$ then using the Gromov-Wasserstein properties it implies that for almost all $\thetab \in \Sp^{p-1}$ the projected measures are isomorphic. Moreover since $\mu,\nu$ have compact support, it is bounded and we can directly apply Theorem \ref{cramer_gene} to state the existence of a measure preserving application $f$ as defined in Theorem \ref{cramer_gene}. We consider the coupling $\pi=(id\times f)\#\mu \in \couplingset(\mu,\nu)$ since $f\#\mu=\nu$. Then we have:
\begin{equation*}
\begin{split}
\int \int |d(\xbf,\xbf')-d(\ybf,\ybf')|^{2} \dr \pi(\xbf,\ybf) \dr \pi(\xbf',\ybf') &= \int \int |d(\xbf,\xbf')-d(f(\xbf),f(\xbf'))|^{2} \dr \mu(\xbf)\dr\mu(\xbf') \\
&=\int \int |\|\xbf-\xbf'\|_1-\|f(\xbf)-f(\xbf')\|_1|^{2} \dr \mu(\xbf)\dr\mu(\xbf') =0
\end{split}
\end{equation*}
Since $f$ is an isometry. This directly implies that $\gw_{2}(d,d,\mu,\nu)=0$.
\end{proof}

\subsection{Additional results -- $SW_{\D}$ and $RISW$ \label{proof:robust_sw}}

\begin{figure}[t]
  \centering
  \includegraphics[width=.7\linewidth]{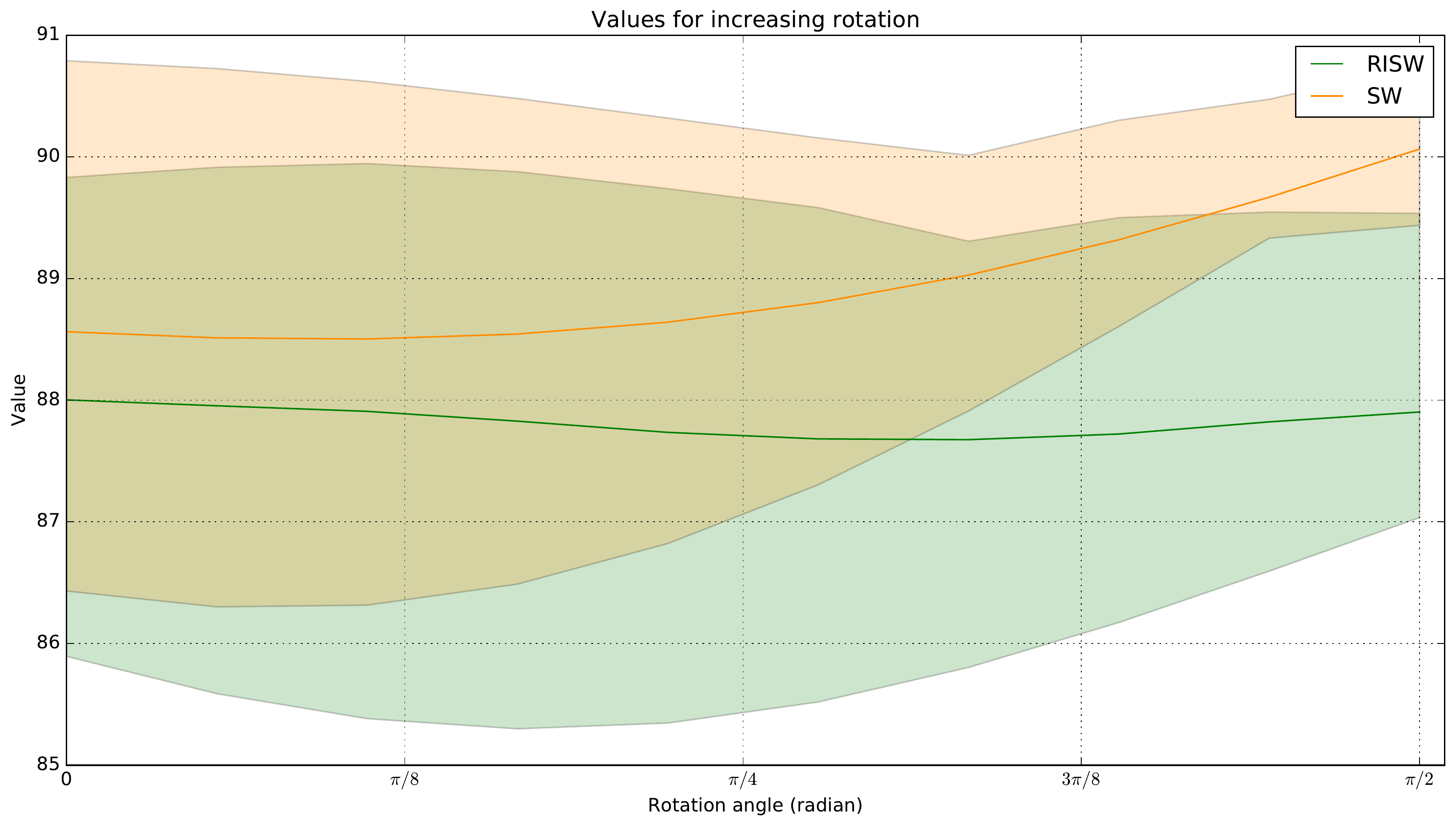}
  \caption{Illustration of $SW$, $RISW$ on spiral datasets for varying rotations on discrete 2D spiral datasets. (left) Examples of spiral distributions for source and target with different rotations. (right) Average value of $SW$ and $RISW$ with $L=20$ as a function of rotation angle of the target. Colored areas correspond to the 20\% and 80\% percentiles. }
  \label{fig:spiral_example_sw}
\end{figure}

Analogously to $SGW$ we can define for the Sliced-Wasserstein distance $SW_{\D}(\mu,\nu)$ for $\mu,\nu \in \Pm(\R^{p})\times \Pm(\R^{q})$ with $p\neq q$ and its rotational invariant counterpart as:

\begin{equation}{}
\label{sw_here}
\begin{split}
&SW_{\D}(\mu,\nu)= \int_{\mathbf{S}^{q-1}}SW(P_{\theta}\#\mu_{\D},P_{\theta}\#\nu) \dr\lambda_{q-1}(\theta) \\
&RISW(\mu,\nu)= \underset{\D \in \mathbb{V}_{q}(\R^{p})}{\min}SW_{\D}(\mu,\nu)
\end{split}
\end{equation}

where $SW$ is the Sliced-Wasserstein distance. The complexity for computing $SW_{\D}$ between two discrete probability measures with $n$ atoms and uniform weights is $O(Ln(p+q+\log(n)))$ which is exactly the same complexity as $SGW_{\D}$. With these formulations, we can perform the same experiment as for RISGW on the spiral dataset. The optimisation on the Stiefel manifold is performed using Pymanopt as for $SGW$. Results are reported in Figure \ref{fig:spiral_example_sw}. As one can see, $RISW$ is rotational invariant on average whereas $SW$ is not. One can also note that, due to the sampling process of the spiral dataset, the variance is quite large. This can be explained by the fact that, unlike $SGW$, the Sliced-Wasserstein may realign the distributions without taking the rotation into account.

\subsection{Proof of Lemma \ref{calculation_gw} -- Reductions of the GW costs for inner products and squared Euclidean distance matrices}
\label{sec:proof_reduc}
We recall the lemma:
\begin{lemma*}
Suppose that there exist scalars $a,b,c$ such that $c_{\Xcal}(\xbf,\xbf')=a\|\xbf\|_{2}^{2}+b\|\xbf'\|_{2}^{2}+c\scalar{\xbf}{\xbf'}{p}$ and $c_{\Ycal}(\ybf,\ybf')=a\|\ybf\|_{2}^{2}+b\|\ybf'\|_{2}^{2}+c\scalar{\ybf}{\ybf'}{q}$. Then:
\begin{equation}
\gwloss_2(c_{\Xcal},c_{\Ycal},\pi) = C_{\mu,\nu} -2 Z(\pi)
\end{equation}
where $C_{\mu,\nu}=\int c_{\Xcal}^{2} \dr \mu \dr \mu + \int c_{\Ycal}^{2} \dr \nu \dr \nu-4ab \int \|\xbf\|_{2}^{2} \|\ybf\|_{2}^{2} d\mu(\xbf) d\nu(\ybf)$ and:
\begin{equation}
\begin{split}
Z(\pi) &= (a^{2}+b^{2}) \int \|\xbf\|_{2}^{2} \|\ybf\|_{2}^{2} \dr\pi(\xbf,\ybf)+ c^{2} \|\int \ybf \xbf^{T} \dr\pi(\xbf,\ybf)\|_{\F}^{2} \\
&+(a+b)c\int \big[\|\xbf\|_{2}^{2} \scalar{\E_{Y\sim\nu}[Y]}{\ybf}{q} + \|\ybf\|_{2}^{2}\scalar{\E_{X\sim \mu}[X]}{\xbf}{p} \dr\pi(\xbf,\ybf)\big] \\
\end{split}
\end{equation}
\end{lemma*}

\begin{proof}
Let $\pi \in \Pi(\mu,\nu)$. We have $J_{2}(c_\Xcal,c_\Ycal,\pi)=\int c_{\Xcal}^{2}\dr \mu \dr \mu+ \int c_{\Ycal}^{2}\dr \nu \dr \nu -2\int c_{\Xcal} c_\Ycal \dr \pi \dr \pi$. In this way:
\begin{equation*}
\begin{split}
&\int c_{\Xcal} c_{\Ycal} \dr\pi\dr\pi=\int (a\|\xbf\|_{2}^{2}+b\|\xbf'\|_{2}^{2}+c\scalar{\xbf}{\xbf'}{p})(a\|\ybf\|_{2}^{2}+b\|\ybf'\|_{2}^{2}+c\scalar{\ybf}{\ybf'}{q})\dr\pi(\xbf,\ybf)\dr\pi(\xbf',\ybf')\\
&=\int \big[(a^{2}\|\xbf\|_{2}^{2}\|\ybf\|_{2}^{2}+ab\|\xbf\|_{2}^{2}\|\ybf'\|_{2}^{2}+ac\|\xbf\|_{2}^{2}\scalar{\ybf}{\ybf'}{q}) + (ab\|\xbf'\|_{2}^{2}\|\ybf\|_{2}^{2}+b^{2}\|\xbf'\|_{2}^{2}\|\ybf'\|_{2}^{2}+bc\|\xbf'\|_{2}^{2}\scalar{\ybf}{\ybf'}{q}) \\
&+ (ca\scalar{\xbf}{\xbf'}{p} \|\ybf\|_{2}^{2} +cb\scalar{\xbf}{\xbf'}{p} \|\ybf'\|_{2}^{2}+c^{2} \scalar{\xbf}{\xbf'}{p} \scalar{\ybf}{\ybf'}{q}) \big]\dr\pi(\xbf,\ybf)\dr\pi(\xbf',\ybf') \\
&= (a^{2}+b^{2}) \int \|\xbf\|_{2}^{2}\|\ybf\|_{2}^{2} \dr\pi(\xbf,\ybf) + 2ab \int \|\xbf\|_{2}^{2} \|\ybf\|_{2}^{2} d\mu(\xbf) d\nu(\ybf) +c^{2} \int \scalar{\xbf}{\xbf'}{p} \scalar{\ybf}{\ybf'}{q} \dr\pi(\xbf,\ybf)\dr\pi(\xbf',\ybf')\\
&+(a+b)c \int \|\xbf\|_{2}^{2}\scalar{\int \ybf \dr\pi(\xbf,\ybf)}{\ybf}{q} \dr\pi(\xbf,\ybf)+(a+b)c \int \|\ybf\|_{2}^{2}\scalar{\int \xbf \dr\pi(\xbf,\ybf)}{\xbf}{p} \dr\pi(\xbf,\ybf)
\end{split}
\end{equation*}
Moreover:
\begin{equation*}
\begin{split}
&\int \scalar{\xbf}{\xbf'}{p} \scalar{\ybf}{\ybf'}{q} \dr\pi(\xbf,\ybf)\dr\pi(\xbf',\ybf')=\int \xbf^{T}\xbf' \ybf^{T}\ybf' \dr\pi(\xbf,\ybf)\dr\pi(\xbf',\ybf') = \int \tr(\xbf^{T}\xbf' \ybf^{T}\ybf') \dr\pi(\xbf,\ybf)\dr\pi(\xbf',\ybf')\\
&=\int \tr(\ybf' \xbf^{T}\xbf' \ybf^{T}) \dr\pi(\xbf,\ybf)\dr\pi(\xbf',\ybf')\stackrel{*}{=} \int \tr(\ybf'\xbf'^{T}\xbf\ybf^{T}) \dr\pi(\xbf,\ybf)\dr\pi(\xbf',\ybf')\\
&=\tr((\int \ybf'\xbf'^{T}\dr\pi(\xbf',\ybf')(\int \xbf\ybf^{T} \dr\pi(\xbf,\ybf)))) =\|\int \xbf\ybf^{T} \dr\pi(\xbf,\ybf)\|_{\F}^{2}=\|\int \ybf\xbf^{T} \dr\pi(\xbf,\ybf)\|_{\F}^{2}
\end{split}
\end{equation*}
where in (*) we used that $\xbf'^{T}\xbf \in \R$ so it is equal to its transpose. 
\end{proof}

\subsection{Proof of the existence and finiteness \eqref{GWprob} and \eqref{InvOT} }
\label{sec:finitesness}
In this section we prove that \eqref{GWprob} and \eqref{InvOT} are finite and that \eqref{InvOT} always admits a maximizer. More precisely:
\begin{lemma}
\label{lemma:existence_compacity}
Let $\mu \in \P(\R^{p}),\nu \in \P(\R^{q})$ with $\int \|\xbf\|_{2}^{4}\dr\mu(\xbf)<+\infty,\int \|\ybf\|_{2}^{4}\dr\nu(\ybf)<+\infty$. Both \eqref{GWprob} and \eqref{InvOT} are finite.

Moreover the set $F_{p,q}$ is a compact subset of $\R^{q\times p}$. The functional $\pi \rightarrow \underset{\Pbf \in F_{p,q}}{\sup} \int \scalar{\Pbf \xbf}{\ybf}{q} \dr\pi(\xbf,\ybf)$ is continuous for the weak convergence of measure. In particular problem \eqref{InvOT} admits an optimal solution $\pi^{*} \in \Pi(\mu,\nu)$
\end{lemma}

\begin{figure*}[!b]
\begin{memo}
\label{continuity_inf}
Let $\Xcal,\Ycal$ be topological spaces with $\Ycal$ compact. If $f: \Xcal \times \Ycal \rightarrow \R$ is continuous then $g: x\rightarrow \inf_{y} f(x,y)$ is well defined and continuous. 
\begin{proof}
Note first that $g(x)>-\infty$, since for every $x \in \Xcal$, $f(x,.): \Ycal\rightarrow \R$ is continuous on a compact space so it is bounded. To prove the continuity it suffices to show that  $g^{-1}(]-\infty,a[)$ and $g^{-1}(]b,+\infty[)$ are open. For the former, if we note $\pi_\Xcal: \Xcal \times \Ycal \rightarrow \Xcal$ the canonical projection then $g^{-1}(]-\infty,a[)=\pi_\Xcal \circ f^{-1}(]-\infty,a[)$. By continuity of $f$ we can conclude that $g^{-1}(]-\infty,a[)$ is open.
For the latter first observe that $g(x)>b \implies \forall y, \ f(x,y)>b$ which means that $g(x)>b \implies \forall y, (x,y) \in f^{-1}(]b,+\infty[)$. In particular for $x \in g^{-1}(]b,+\infty[)$ and $y \in \Ycal$ there exists a neighborhood $U_{(x,y)} \times V_{(x,y)}$ contained in $f^{-1}(]b,+\infty[)$. Since $\Ycal$ is compact there exists a finite subset $\{(x,y_i)\}$ of all the neighborhoods $U_{(x,y)} \times V_{(x,y)}$ which cover all of $\{x\}\times \Ycal$. Overall:
\begin{equation}
\{x\}\times \Ycal \subset (\cap_{i=1}^{k} U_{(x,y_i)}) \times \Ycal \subset f^{-1}(]b,+\infty[)
\end{equation}
Hence $g^{-1}(]b,+\infty[)=\cup_{x\in g^{-1}(]b,+\infty[)} \cap_{i=1}^{k} U_{(x,y_i)})$ which is open so $g$ is continuous.
\end{proof}
\end{memo}
\end{figure*}

\begin{proof}

In this case $\int \|\xbf\|_{2}^{2}\dr\mu(\xbf)<+\infty,\int \|\ybf\|_{2}^{2}\dr\nu(\ybf)<+\infty$ by Hölder's inequality and for all $(\xbf,\xbf',\ybf,\ybf') \in \Xcal^{2} \times \Ycal^{2}$ $\big( \scalar{\xbf}{\xbf'}{p} - \scalar{\ybf}{\ybf'}{q} \big)^{2}\leq 2(\scalar{\xbf}{\xbf'}{p}^{2}+\scalar{\ybf}{\ybf'}{q}^{2}) \leq 2(\|\xbf\|_{2}^{2}\|\xbf'\|_{2}^{2}+\|\ybf\|_{2}^{2}\|\ybf'\|_{2}^{2})$ by Cauchy-Swartz. In particular this implies: 
\begin{equation*}
\begin{split}
&\underset{\pi \in \Pi(\mu,\nu)}{\inf} \int \int \big(\scalar{\xbf}{\xbf'}{p} - \scalar{\ybf}{\ybf'}{q}\big)^{2} \dr\pi(\xbf,\ybf)\dr\pi(\xbf',\ybf') \leq \int \int \big(\scalar{\xbf}{\xbf'}{p} - \scalar{\ybf}{\ybf'}{q}\big)^{2} \dr\mu(\xbf) \dr\mu(\xbf') \dr\nu(\ybf) \dr\nu(\ybf') \\
&\leq 2 (\int \|\xbf\|_{2}^{2}\dr\mu(\xbf) \int \|\xbf'\|_{2}^{2}\dr\mu(\xbf')+\int \|\ybf\|_{2}^{2}\dr\nu(\xbf) \int \|\ybf'\|_{2}^{2}\dr\nu(\ybf')) <+\infty
\end{split}
\end{equation*}
Moreover for any $\pi \in \couplingset(\mu,\nu), \Pbf \in F_{p,q}$: 
\begin{equation*}
\begin{split}
&\int \scalar{\Pbf \xbf}{\ybf}{q} \ \dr\pi(\xbf,\ybf) = \int \scalar{\Pbf }{\ybf \xbf^{T}}{\F} \ \dr\pi(\xbf,\ybf)\leq \int \|\Pbf\|_{\F}^{2}\|\xbf\ybf^{T}\|^{2}_{\F}\dr\pi(\xbf,\ybf)\\
&\leq p \int \|\xbf\ybf^{T}\|^{2}_{\F} \dr \pi(\xbf,\ybf)= p \int \tr(\ybf \xbf^{T} \xbf\ybf^{T}) \dr \pi(\xbf,\ybf)= p \int \tr(\xbf^{T} \xbf\ybf^{T}\ybf ) \dr \pi(\xbf,\ybf) \\
&= p \int \tr(\|\xbf\|_2^{2}\|\ybf\|_{2}^{2} ) \dr \pi(\xbf,\ybf)=p \int \|\xbf\|^{2}_{2}\ybf\|_2^{2} \dr \pi(\xbf,\ybf) \stackrel{*}{\leq} \frac{p}{2} (\int \|\xbf\|_{2}^{4} \dr \mu(\xbf) + \int \|\ybf\|_{2}^{4} \dr \nu(\ybf)) \\
&<+\infty
\end{split}
\end{equation*}
where in (*) we used Young's inequality (recalled in Memo \ref{memo:young}).

For the compacity, using Borel Lebesgue theorem it suffices to show that $F_{p,q}$ is closed and bounded. It is clearly bounded by $\sqrt{p}$ and closed as the pre-image of the closed set $\{0\}$ by the continuous application $\Pbf \rightarrow \sqrt{p}-\|\Pbf\|_{\F}$. 

We note $f: \couplingset(\mu,\nu) \times F_{p,q} \rightarrow \R$ the function $f(\pi,\Pbf)=-\int \scalar{\Pbf \xbf}{\ybf}{q}\dr\pi(\xbf,\ybf)$. For any $\pi \in \couplingset(\mu,\nu)$, $f(\pi,.)$ is continuous. Indeed suppose that $\Pbf_m \underset{\F}{\rightarrow} \Pbf$ where $\underset{\F}{\rightarrow}$ denotes the convergence in Frobenius norm, then: 
\begin{equation*}
\begin{split}
&|\int \scalar{\Pbf_{m} \xbf}{ \ybf}{q} \dr \pi(\xbf, \ybf)-\int \scalar{\Pbf \xbf}{\ybf}{q} \dr \pi(x, y)| \leq \int |\scalar{(\Pbf_{m}-\Pbf)\xbf}{\ybf}{q}|\dr \pi(\xbf,\ybf) \\
&\leq \|\Pbf_{m}-\Pbf\|_{\F}^{2}(\int \|\xbf\|_{2}^{2} \|\ybf\|_{2}^{2} \dr\pi(\xbf,\ybf)) \leq \frac{1}{2}\|\Pbf_{m}-\Pbf\|_{\F}^{2}(\int \|x\|_{2}^{4} \dr \mu(x) + \int \|y\|_{2}^{4} \dr \nu(\ybf)) \underset{m\rightarrow +\infty}{\rightarrow} 0
\end{split}
\end{equation*}
In this way, since $F_{p,q}$ is compact, $g: \pi \rightarrow \inf_{\Pbf \in F} f(\pi,\Pbf)$ is well defined and continuous for the weak convergence of measure (see Memo \ref{continuity_inf}). Since it is continuous and $\couplingset(\mu,\nu)$ is compact we can applied the Weirstrass' theorem to state that $\inf_{\pi \in \couplingset(\mu,\nu)}g(\pi)$ exists which concludes the proof. 

\end{proof}

\subsection{Proof of Lemma \ref{lemma:reduce} --  Reduction of the GW cost}
\label{sec:proof_reduce}
We recall the lemma:
\begin{lemma*}
Let $\Xcal$ and $\Ycal$ be compact subset of respectively $\R^{p}$ and $\R^{q}$. Let $\mu \in \P(\Xcal),\nu \in \P(\Ycal)$. We can assume without loss of generality that $\E_{X\sim \mu}[X]=0$ and $\E_{Y\sim\nu}[Y]=0$. In this case \eqref{eq:sqGW} is equivalent to:
\begin{equation*}
\sup_{\pi \in \Pi(\mu,\nu)} \int \|\xbf\|_{2}^{2} \|\ybf\|_{2}^{2} \dr\pi(\xbf,\ybf)+2 \|\int \ybf\xbf^{T}\dr\pi(\xbf,\ybf)\|^{2}_{\F}
\end{equation*}
\end{lemma*}
Indeed equation \eqref{eq:sqGW} corresponds to the case $a=1,b=1,c=-2$ of Lemma \ref{calculation_gw} and so is equivalent to:
\begin{equation}
\begin{split}
&\sup_{\pi \in \Pi(\mu,\nu)} 2 \int \|\xbf\|_{2}^{2} \|\ybf\|_{2}^{2} \dr\pi(x,y)+4 \|\int \xbf\ybf^{T}\dr \pi(\xbf,\ybf)\|^{2}_{\F}\\
&-4\int \big[\|\xbf\|_{2}^{2} \scalar{\E_{Y\sim\nu}[Y]}{y}{d'} + \|\ybf\|_{2}^{2}\scalar{\E_{X\sim \mu}[X]}{x}{d} \dr \pi(\xbf,\ybf)\big]\\
\end{split}
\end{equation}
We can reduce this problem using the translation invariance of the $\gw$ cost since $\|\xbf+\mathbf{t}-(\xbf'+\mathbf{t})\|_{2}^{2}=\|\xbf-\xbf'\|_{2}^{2}$ for all $\xbf,\xbf'$ (same for $\ybf$). More precisely we can rely on the following lemmas:
\begin{lemma}
\label{couplinglemma}
Let $f: \mathbb{R}^{p} \rightarrow \mathbb{R}^{p}$ and $f': \mathbb{R}^{q} \rightarrow \mathbb{R}^{q}$ Borel and $\mu \in \P(\R^{p}),\nu \in \P(\R^{q})$
Then $$\Pi\left(f\# \mu, f'\# \nu\right)=\left\{(f \times f')\# \pi | \pi \in \Pi(\mu, \nu)\right\}$$
\end{lemma}
\begin{proof}
This is a straightforward extension of the Lemma 6 in \cite{subspace_robust_wass_patty_2019}.
\end{proof}
Based on previous Lemma \ref{couplinglemma} we have:
\begin{lemma}[Translation Invariance]
\label{translation_invariance}
Let $\mu \in \P(\R^{p}),\nu \in \P(\R^{q})$ and $c_\Xcal(\xbf,\xbf')=\|\xbf-\xbf'\|_{2}^{2},c_\Ycal(\ybf,\ybf')=\|\ybf-\ybf'\|_{2}^{2}$. Let $\mathbf{t},\mathbf{t}' \in \R^{p} \times \R^{q}$ and $f_{\mathbf{t}}(\xbf)=\xbf+\mathbf{t}, f'_{\mathbf{t}'}=\ybf+\mathbf{t}'$ be translations. Then:
\begin{equation}
\underset{\pi \in \Pi(\mu,\nu)}{\sup} J_2(c_\Xcal,c_\Ycal,\pi)= \underset{\pi \in \Pi(f_{\mathbf{t}}\#\mu,f'_{\mathbf{t}'}\#\nu)}{\sup} J_2(c_\Xcal,c_\Ycal,\pi)
\end{equation}
\end{lemma}

\begin{proof}
Using Lemma \ref{couplinglemma} we have: 
\begin{equation*}
\begin{split}
\underset{\pi \in \Pi(f_{\mathbf{t}}\#\mu,f'_{\mathbf{t}'}\#\nu)}{\sup} J_2(c_\Xcal,c_\Ycal,\pi)&=\underset{\pi \in \Pi(\mu,\nu)}{\sup} J_2(c_\Xcal,c_\Ycal,(f_{\mathbf{t}}\times f'_{\mathbf{t}'})\#\pi) \stackrel{*}{=}\underset{\pi \in \Pi(\mu,\nu)}{\sup} J_2(c_\Xcal,c_\Ycal,\pi)
\end{split}
\end{equation*}
where in (*) we used the invariance of the cost $J_2(c_\Xcal,c_\Ycal,\pi)$ with respect to translations.
\end{proof}
Using this property we can assume without loss of generality $\mu$ and $\nu$ centered, \textit{i.e}, $\E_{X\sim \mu}[X]=0$ and $\E_{Y\sim\nu}[Y]=0$. By plugging this condition into $Z(\pi)$ of Lemma \ref{calculation_gw} we have the desired result.

\subsection{Proof of Lemma \ref{hcanbewritten}}
\label{sec:lemma_calculus}
We recall the lemma:
\begin{lemma*} 
If $\pi^{*}$ is a solution of the primal problem $\sup_{\pi \in \couplingset(\mu,\nu)} F(\pi)$ then there exists $\Pbf \in \R^{q\times p}$ and $h^{*}\in C(\Xcal \times \Ycal)$ of the form $h(\xbf,\ybf)=\scalar{\Pbf \xbf}{\ybf}{q}+\|\xbf\|_{2}^{2}\|\ybf\|_{2}^{2}$ such that $(\pi^{*},h^{*})$ is a solution of the dual problem \eqref{dual_problem}. 
Moreover when $h$ is in such form we have $F^{*}(h)=\frac{1}{8}\|\Pbf\|^{2}_{\F}$.
\end{lemma*}

To prove this result we will need the following calculus:
\begin{lemma}
\label{frechetF}
With previous notations, the Fréchet derivative of $F$ reads:
\begin{equation}
\nabla F(\pi)= (\xbf,\ybf)\rightarrow 4\scalar{\mathbf{V_{\pi}}\xbf}{\ybf}{q}+\|\xbf\|_{2}^{2}\|\ybf\|_{2}^{2}
\end{equation}
\end{lemma}
\begin{proof}
\begin{equation}
\begin{split}
F(\pi +t \epsilon)&=2\|\mathbf{V_{\pi}}\|^{2}_{\F}+\int \|\xbf\|_{2}^{2}\|\ybf\|_{2}^{2}\dr\pi(x,y)+ \|t \mathbf{V_{\epsilon}}\|_{\F}^{2}+t \int (4\scalar{\mathbf{V_{\pi}}}{\ybf\xbf^{T}}{\F}+\|\xbf\|_{2}^{2}\|\ybf\|_{2}^{2})d\epsilon(\xbf,\ybf)\\
&=F(\pi)+t\int (4\scalar{\mathbf{V_{\pi}}}{\ybf\xbf^{T}}{\F}+\|\xbf\|_{2}^{2}\|\ybf\|_{2}^{2})d\epsilon(x,y)+o_{t\rightarrow 0}(t)
\end{split}
\end{equation}
for $\epsilon \in \mathcal{M}(\Xcal\times \Ycal)$. Hence $\nabla F(\pi): (x,y)\rightarrow 4\scalar{\mathbf{V_{\pi}}}{\ybf\xbf^{T}}{\F}+\|\xbf\|_{2}^{2}\|\ybf\|_{2}^{2}=4\scalar{\mathbf{V_{\pi}} \xbf}{\ybf}{q}+\|\xbf\|_{2}^{2}\|\ybf\|_{2}^{2}$
\end{proof}

\begin{proof}[Proof of Lemma \ref{hcanbewritten} --  Parametrization of the dual problem]
If  $\pi^{*}$ is a maximizer of the primal problem then we know that $h^{*}=\nabla F(\pi^{*})$ is a maximizer of $\sup_{h \in \mathcal{C}(\Xcal \times \Ycal)} \sup_{\pi \in \couplingset(\mu,\nu)} \int h(\xbf,\ybf)\dr\pi(\xbf,\ybf) -F^{*}(h)$. The calulus in Lemma \ref{frechetF}) implies that $\nabla F(\pi^{*})=(\xbf,\ybf)\rightarrow 4\scalar{\mathbf{V_{\pi^{*}}} \xbf}{\ybf}{q}+\|\xbf\|_{2}^{2}\|\ybf\|_{2}^{2}$. Setting $\Pbf=4\mathbf{V_{\pi^{*}}} \in \R^{q\times p}$ concludes for the first point.

For the second point we have by definition $F^{*}(h)=\sup_{\pi \in \mathcal{M}(\Xcal\times \Ycal)} \int h(\xbf,\ybf)\dr\pi(\xbf,\ybf) -F(\pi)$. We note $G(\pi)=\int h(\xbf,\ybf)\dr\pi(\xbf,\ybf) -F(\pi)$. Then $\nabla G(\pi)= h-\nabla F(\pi)= (\xbf,\ybf) \rightarrow h(\xbf,\ybf)-4\scalar{\mathbf{V_{\pi}} \xbf}{\ybf}{q}-\|\xbf\|_{2}^{2}\|\ybf\|_{2}^{2}$ by Lemma \ref{frechetF}. Using the fact that $h$ is parametrized by a linear application we have $\nabla G(\pi)=(\xbf,\ybf) \rightarrow \scalar{(\Pbf-4 \mathbf{V_{\pi}})\xbf}{\ybf}{q}$. Then for $(\xbf,\ybf)\in \Xcal \times \Ycal$:
\begin{equation}
\begin{split}
&\nabla G(\pi)(\xbf,\ybf)=0 \iff \scalar{(\Pbf-4 \mathbf{V_{\pi}})\xbf}{\ybf}{q}=0 \\
\end{split}
\end{equation}
We write $\Pbf=\sum_{i} \lambda_i \vbf_i \ubf_i^{T}$ the SVD of $\Pbf$. We note $\gamma_{\Pbf}\in \mathcal{M}(\Xcal \times \Ycal)$ the measure $\gamma_{\Pbf}=\frac{1}{4} \sum_i \lambda_i \delta_{(\ubf_i,\vbf_i)}$ (note that we do not need that $\gamma_{\Pbf}$ be a probability measure).
Then $\int \ybf\xbf^{T}\dr \gamma_{\Pbf}(\xbf,\ybf)=\frac{1}{4} \sum_{i} \lambda_i \vbf_i \ubf_i^{T}=\frac{1}{4}\Pbf$. By previous calculus $\nabla G(\gamma_{\Pbf})=0$ so that $\gamma_{\Pbf}$ satisfies the first order condition and is a solution to $\sup_{\pi \in \mathcal{M}(\Xcal\times \Ycal)} \int h(\xbf,\ybf)\dr\pi(\xbf,\ybf) -F(\pi)$. Overall:
\begin{equation}
\begin{split}
&F^{*}(h)=\int h(\xbf,\ybf)\dr\gamma_{\Pbf}(\xbf,\ybf)-F(\gamma_{\Pbf}) \\
&=\int \scalar{\Pbf \xbf}{\ybf}{q}+\|\xbf\|_{2}^{2} \|\ybf\|_{2}^{2} \dr \gamma_{\Pbf}(\xbf,\ybf)-\int \|\xbf\|_{2}^{2} \|\ybf\|_{2}^{2} \dr\gamma_{\Pbf}(\xbf,\ybf)-2 \|\int \ybf \xbf^{T} \dr \gamma_{\Pbf}(\xbf,\ybf)\|^{2}_{\F} \\
&=\froeb{\Pbf}{\int \ybf \xbf^{T} \dr \gamma_{\Pbf}(x,y)} -2 \|\int \ybf \xbf^{T} \dr \gamma_{\Pbf}(\xbf,\ybf)\|^{2}_{\F} =\frac{1}{4} \|\Pbf\|^{2}_{\F}-\frac{1}{8}\|\Pbf\|^{2}_{\F}=\frac{1}{8}\|\Pbf\|^{2}_{\F}
\end{split}
\end{equation}
\end{proof}

\subsection{Characterization of $\mathcal{H}(\mu,\nu)$ for linear push-forward}
We have the following result:
\begin{lemma}
\label{lemma:push_lin}
Let $\mu,\nu \in \P(\R^{p}) \times \P(\R^{p})$. If $T$ is a linear symmetric push-forward of $\mu$ to $\nu$ then: 
\begin{equation}
T\in \mathcal{H}(\mu,\nu) \iff T=\lambda\Obf
\end{equation} 
where $\Obf \in \mathcal{O}(p)$ is an orthogonal matrix and $\lambda>0$.
\end{lemma}

\begin{proof}
We note $T(\xbf)=\Abf \xbf$. Then we have $\|\Abf \xbf\|_{2}^{2}=\xbf^{T}\Abf^{T}\Abf \xbf$, we can write $\Abf^{T}\Abf=\Pbf \mathbf{D} \Pbf^{T}$ where $\mathbf{D}$ diagonal and $\Pbf$ orthogonal, then it is equivalent to find a $g=f'$ such that: 

\begin{equation}
\|\mathbf{D}^{1/2}\v\|_{2}^2=\sum_{i} |\beta_{i} v_{i}|^{2} = g(\|\v\|_{2}^2)
\end{equation}
(with $\beta_{i}\geq 0$ and we note $\v=\Pbf^{T}\xbf$). If there is $i$ such that $\beta_{i}=0$ (the first one without loss of generality) then for all $t\geq 0$ we can take $\v=(\sqrt{t},0,..,0)$ then $g(\|\v\|_{2}^2)=g(t)=0$ which is not possible unless $\Abf=0$. In this way $\beta_{i}>0$ for all $i$. Let $\e_{i}$ be one eigenvector of $\Abf^{T}\Abf$ then $\|\Abf \e_{i}\|_{2}^{2}=\e_{i}^{T}\Abf^{T}\Abf \e_{i}=\lambda_{i}=g(\|\e_{i}\|_{2}^{2})=g(1)$. Overall $\beta_{i}=g(1)=\lambda$ for all $i$. So $\Abf^T\Abf=\lambda_{0}\mathbf{I_p}$. Overall it implies that $\Abf$ can be written as $\Abf=\lambda\Obf$ with $\Obf$ an orthogonal matrix. 
Conversely if $T=\lambda\Obf$ then $\|T(\xbf)\|_{2}^{2}=c^{2} \|\xbf\|_{2}^{2}$. We can take the function $f(t)=\frac{1}{2}c^{2}t^{2}$ which is convex and continuous from $\R$ to $\R$.  
\end{proof}

\subsection{Proof of Theorem \ref{maintheo3} -- Closed form expression of the linear Gromov-Monge problem between Gaussian measures}
\label{sec:prof_maintheo3}
We recall the theorem:

\begin{theo*}
Let $\mu=\mathcal{N}(0,\Sigmab_{\nu}) \in \P(\R^{p}),\nu=\mathcal{N}(0,\Sigmab_{\mu})\in \P(\R^{q})$ centered without loss of generality. Let $\Sigmab_{\mu}=\V_{\mu} \Dbf_{\mu} \V_{\mu}^\top,\Sigmab_{\nu}=\V_{\nu} \Dbf_{\nu} \V_{\nu}^\top$ be the diagonalizations of the covariance matrices such that eigenvalues of $\Dbf_{\mu}$ and $\Dbf_{\nu}$ are ordered nondecreasing.
When $p\neq q$ we have:
\begin{equation*}
\begin{split}
L\gm^{2}_2(\mu,\nu)= 4(\tr(\Sigmab_\mu)-\tr(\Sigmab_\nu))^{2}+8(\tr(\Sigmab_\mu \Sigmab_\mu)+\tr(\Sigmab_\nu \Sigmab_\nu)) +16\min_{\Bbf \in \Stief}-\tr(\Dbf_\mu\Bbf^\top  \Dbf_\nu  \Bbf)
\end{split}
\end{equation*}

When $p=q$, an optimal linear Monge map is given by $T(\xbf)=\Abf \xbf$ where:
\begin{equation*}
\Abf=\V_\nu  \Dbf^{1/2}_\nu \Dbf^{-1/2}_\mu \V_\mu^\top=\Sigmab_{\nu}^{1/2}\V_{\nu}\V_{\mu}^{\top}\Sigmab_{\mu}^{-1/2}
\end{equation*}
so that:
\begin{equation*}
L\gm^{2}_2(\mu,\nu)= 4(\tr(\Sigmab_\mu)-\tr(\Sigmab_\nu))^{2}+8(\tr(\Sigmab_\mu \Sigmab_\mu)+\tr(\Sigmab_\nu \Sigmab_\nu))-16\tr(\Dbf_{\mu}\Dbf_{\nu})
\end{equation*}
\end{theo*}

In order to prove Theorem \ref{maintheo3} we will rely on the the following result:  

\begin{prop}[Proposition 3.1 in \cite{qap_anstreicher_1998}]
\label{prop:anstreicher}
Consider $\Sigmab_1,\Sigmab_2 \in \R^{p\times p} \times \R^{p\times p}$ two symmetric matrices and the following (QQP):
\begin{equation}
\label{qqp}
\min_{\begin{smallmatrix} \Bbf\Bbf^{T}=\mathbf{I} \end{smallmatrix}} \ \tr(\Sigmab_1 \Bbf \Sigmab_2 \Bbf^{T})
\end{equation}
Let $\Sigmab_1=\V_{1} \Dbf_{1} \V_{1}^{T}$ and $\Sigmab_2=\V_{2} \Dbf_{2} \V_{2}^{T}$ be the orthogonal diagonalizations of $\Sigmab_1,\Sigmab_2$ where the eigenvalues of $\Dbf_{1}$ are ordered nonincreasing the eigenvalues in $\Dbf_{2}$ are ordered nondecreasing.
An optimal solution of \eqref{qqp} is found at $\Bbf^{*}=\V_{1} \V_{2}^{T}$ and the optimal value of \eqref{qqp} is $\tr(\Dbf_{1} \Dbf_{2})$.
\end{prop} 

We will prove that when considering only linear push-forward we can recast the Gromov-Monge problem in the form of Proposition \ref{prop:anstreicher}. We have the following result:

\begin{lemma}
Let $\mu=\mathcal{N}(0,\Sigmab_{\nu}) \in \P(\R^{p}),\nu=\mathcal{N}(0,\Sigmab_{\mu})\in \P(\R^{q})$ centered without loss of generality. For a linear push-forward $T\#\mu=\nu$ we note $T(\xbf)=\Abf \xbf$. The Gromov-Monge problem is equivalent to:
\begin{equation}
\label{eq:equiv1}
\min_{\begin{smallmatrix}\Abf \in \R^{q \times p} \\ \Abf\Sigmab_{\mu}\Abf^\top=\Sigmab_{\nu}\end{smallmatrix}}\ \froeb{\Mbf}{\Abf^\top \Abf}
\end{equation}
where:
\begin{equation}
    \Mbf=\underset{\xbf,\xbf' \sim \mu}{\E}[-2 (\xbf^\top \xbf) . \xbf \xbf^\top -2(\xbf'^\top \xbf'). \xbf \xbf^\top -4(\xbf^\top \xbf'). \xbf' \xbf^\top]
\end{equation}

\end{lemma}

\begin{proof}
We have:
\begin{equation*}
\begin{split}
\underset{\begin{smallmatrix}T\#\mu=\nu \\ \text{T linear} \end{smallmatrix}}{\min} \ J(T) &= \underset{\begin{smallmatrix}T\#\mu=\nu \\ \text{T linear} \end{smallmatrix}}{\min} \ \underset{\xbf,\xbf' \sim \mu}{\E}[\big(\|\xbf-\xbf'\|_{2}^{2}-\|T(\xbf)-T(\xbf')\|_{2}^{2}\big)^{2}] \\
&=\underset{\begin{smallmatrix}T\#\mu=\nu \\ \text{T linear} \end{smallmatrix}}{\min} \ \underset{\xbf,\xbf' \sim \mu}{\E}[\|\xbf-\xbf'\|^{4}]+\underset{\xbf,\xbf' \sim \mu}{\E}[\|T(\xbf)-T(\xbf')\|_{2}^{4}]-2\underset{\xbf,\xbf' \sim \mu}{\E}[\|\xbf-\xbf'\|_{2}^{2}\|T(\xbf)-T(\xbf')\|_{2}^{2}] \\
&= \underset{\xbf,\xbf' \sim \mu}{\E}[\|\xbf-\xbf'\|_{2}^{4}] + \underset{\ybf,\ybf' \sim \nu}{\E}[\|\ybf-\ybf'\|_{2}^{4}] + 2\underset{\begin{smallmatrix}T\#\mu=\nu \\ \text{T linear} \end{smallmatrix}}{\min} \ J_{1}(T) \\
\end{split}
\end{equation*}
With $J_{1}(T)\stackrel{def}{=}-\underset{\xbf,\xbf' \sim \mu}{\E}[\|\xbf-\xbf'\|_{2}^{2}\|T(\xbf)-T(\xbf')\|_{2}^{2}]$. In this way the problem is equivalent to minimizing $J_1(T)$. Since $T\#\mu=\nu$:
\begin{equation*}
\begin{split}
J_{1}(T) &= -\underset{\xbf,\xbf' \sim \mu}{\E}[\|\xbf-\xbf'\|_{2}^{2}\|T(\xbf)-T(\xbf')\|_{2}^{2}] \\
&= -\underset{\xbf,\xbf' \sim \mu}{\E}[(\|\xbf\|_{2}^{2}+\|\xbf'\|_{2}^{2}-2\scalar{\xbf}{\xbf'}{p})(\|T(\xbf)\|_{2}^{2}+\|T(\xbf')\|^{2}-2\scalar{T(\xbf)}{T(\xbf')}{q} \rangle)] \\
&= \red{-\underset{\xbf,\xbf' \sim \mu}{\E}[\|\xbf\|_{2}^{2}\|T(\xbf)\|_{2}^{2}+\|\xbf'\|_{2}^{2}\|T(\xbf')\|_{2}^{2}]} \blue{-\underset{\xbf,\xbf' \sim \mu}{\E}[\|\xbf'\|_{2}^{2}\|T(\xbf)\|_{2}^{2}+\|\xbf\|_{2}^{2}\|T(\xbf')\|_{2}^{2}]} \\
&\orange{+2\underset{\xbf,\xbf' \sim \mu}{\E}[\langle \xbf,\xbf'\rangle \|T(\xbf)\|_{2}^{2}] +2\underset{\xbf,\xbf' \sim \mu}{\E}[\scalar{\xbf}{\xbf'}{p}\|T(\xbf')\|_{2}^{2}]}-4\underset{\xbf,\xbf' \sim \mu}{\E}[\scalar{\xbf}{\xbf'}{p} \scalar{T(\xbf)}{T(\xbf')}{q}]\\
&\purple{+2\underset{\xbf,\xbf' \sim \mu}{\E}[\scalar{T(\xbf)}{T(\xbf')}{q} \|\xbf\|_{2}^{2}] +2\underset{\xbf,\xbf' \sim \mu}{\E}[\scalar{T(\xbf)}{T(\xbf')}{q}\|\xbf'\|_{2}^{2}]}\\
&=\red{-2\underset{\xbf \sim \mu}{\E}[\|\xbf\|_{2}^{2}\|T(\xbf)\|_{2}^{2}]}\blue{-2\underset{\xbf,\xbf' \sim \mu}{\E}[\|\xbf'\|_{2}^{2}\|T(\xbf)\|_{2}^{2}]}\orange{+4\underset{\xbf,\xbf' \sim \mu}{\E}[\scalar{\xbf}{\xbf'}{p} \|T(\xbf)\|_{2}^{2}]}\\
&-4\underset{\xbf,\xbf' \sim \mu}{\E}[\scalar{\xbf}{\xbf'}{p} \scalar{T(\xbf)}{T(\xbf')}{q}]\purple{+4\underset{\xbf,\xbf' \sim \mu}{\E}[\scalar{T(\xbf)}{T(\xbf')}{q}\|\xbf\|_{2}^{2}]}\\
\end{split}
\end{equation*}

Since $T$ is linear it can be written in the form $T(\xbf)=\Abf \xbf$. The push-forward constraint in this case reads \cite{flamary2019concentration}:
\begin{equation}
    \Abf\Sigmab_{\mu}\Abf^\top=\Sigmab_{\nu}
\end{equation}
When $A$ is symmetric positive definite this equation admits a unique solution which is the optimal linear map for the Wasserstein problem. However here we have no result about the regularity of $\Abf$. Plugging $T(\xbf)=\Abf \xbf$ into $J_1(T)$ and writing the push-forward condition gives the equivalent problem:
\begin{equation}
    \min_{\Abf \in \R^{q \times p}, \Abf\Sigmab_{\mu}\Abf^\top=\Sigmab_{\nu}}\ J_2(\Abf)
\end{equation}
where 
\begin{align}
    J_2(\Abf)&=\red{-2\underset{\xbf \sim \mu}{\E}[\xbf^\top \xbf \xbf^\top \Abf^\top \Abf \xbf]}\blue{-2\underset{\xbf,\xbf' \sim \mu}{\E}[\xbf'^\top \xbf' \xbf^\top \Abf^\top \Abf \xbf]}\orange{+4\underset{\xbf,\xbf' \sim \mu}{\E}[\xbf^\top \xbf' \xbf^\top \Abf^\top \Abf \xbf]}\\
&-4\underset{\xbf,\xbf' \sim \mu}{\E}[\xbf^\top \xbf' \xbf^\top \Abf^\top \Abf \xbf']\purple{+4\underset{\xbf,\xbf' \sim \mu}{\E}[\xbf^\top \Abf^\top \Abf \xbf' \xbf^\top \xbf]}
\end{align}
Using the property of the trace of a matrix and linearity of the inner product, one can reformulate $J_2$ as 
\begin{align}
    J_2(\Abf)&=\froeb{\Mbf}{\Abf^\top \Abf}
\end{align}
where 
\begin{align}
    \Mbf&=\underset{\xbf,\xbf' \sim \mu}{\E}[\red{-2 (\xbf^\top\xbf) . \xbf \xbf^\top} \blue{-2(\xbf'^\top \xbf'). \xbf \xbf^\top} \orange{+4(\xbf^\top \xbf'). \xbf \xbf^\top }-4(\xbf^\top \xbf'). \xbf' \xbf^\top \purple{+4( \xbf^\top \xbf). \xbf' \xbf^\top}]
\end{align}
Indeed for all terms we used the following reasoning:
\begin{equation}
    \begin{split}
 \red{\underset{\xbf \sim \mu}{\E}[\xbf^\top \xbf \xbf^\top \Abf^\top \Abf \xbf]}&\stackrel{*}{=}\red{\underset{\xbf \sim \mu}{\E}[\tr(\xbf^\top \xbf \xbf^\top \Abf^\top \Abf \xbf)]} \stackrel{**}{=}\red{\underset{\xbf \sim \mu}{\E}[\xbf^\top \xbf\tr( \xbf^\top \Abf^\top \Abf \xbf)]} \stackrel{***}{=}\red{\underset{\xbf \sim \mu}{\E}[\xbf^\top \xbf\tr( \xbf \xbf^\top \Abf^\top \Abf )]} \\
&=\red{\underset{\xbf \sim \mu}{\E}[\tr( \xbf^\top \xbf \xbf \xbf^\top \Abf^\top \Abf )]} =\red{\tr(\underset{\xbf \sim \mu}{\E}[ \xbf^\top \xbf \xbf \xbf^\top \Abf^\top \Abf ])} =\red{\tr(\underset{\xbf \sim \mu}{\E}[ \xbf^\top \xbf \xbf \xbf^\top ]\Abf^\top \Abf )}\\
&=\red{\froeb{\underset{\xbf \sim \mu}{\E}[ \xbf^\top \xbf \xbf \xbf^\top ]}{\Abf^\top \Abf}} \\
    \end{split}
\end{equation}
where in (*) we used that $\xbf^\top \xbf \xbf^\top \Abf^\top \Abf \xbf \in \R$, in (**) that $\xbf^\top \xbf \in \R$ and in (***) the cyclical permutation invariance of $\tr$. 
Moreover $\purple{\underset{\xbf,\xbf' \sim \mu}{\E}[(\xbf^\top \xbf). \xbf' \xbf^\top]}, \orange{\underset{\xbf,\xbf' \sim \mu}{\E}(\xbf^\top \xbf'). \xbf \xbf^\top]}=0$ since Gaussians are centered. Hence:
\begin{equation}
    \Mbf=\underset{\xbf,\xbf' \sim \mu}{\E}[\red{-2 (\xbf^\top \xbf) . \xbf \xbf^\top} \blue{-2(\xbf'^\top \xbf'). \xbf \xbf^\top} -4(\xbf^\top \xbf'). \xbf' \xbf^\top]
\end{equation}
\end{proof}

We can go further on this calculus by using Isserlis' theorem recalled here:
\begin{theo*}[Isserlis \cite{isserlis}]
If $(X_1,X_2,X_3,X_4)$ is a zero-mean multivariate normal random vector then:
\begin{equation}
\E[X_1 X_2 X_3 X_4]=\E[X_1 X_2] \E[X_3 X_4]+\E[X_1 X_3] \E[X_2 X_4]+\E[X_1 X_4] \E[X_2 X_3]
\end{equation}
\end{theo*}

Applying this theorem to $\Mbf$ in \eqref{eq:equiv1} gives the following lemma:
\begin{lemma}
Let $\mu=\mathcal{N}(0,\Sigmab_{\nu}) \in \P(\R^{p}),\nu=\mathcal{N}(0,\Sigmab_{\mu})\in \P(\R^{q})$ centered without loss of generality. Then:
\begin{equation}
\min_{\begin{smallmatrix}\Abf \in \R^{q \times p} \\ \Abf\Sigmab_{\mu}\Abf^\top=\Sigmab_{\nu}\end{smallmatrix}}\ \froeb{\Mbf}{\Abf^\top \Abf}=\min_{\begin{smallmatrix}\Abf \in \R^{q \times p} \\ \Abf\Sigmab_{\mu}\Abf^\top=\Sigmab_{\nu}\end{smallmatrix}}\ \froeb{-4\tr(\Sigmab_{\mu}).\Sigmab_{\mu}-8\Sigmab_{\mu}^\top\Sigmab_{\mu}}{\Abf^\top \Abf}
\end{equation}

\end{lemma}

\begin{proof}
We will use Isserlis' theorem to calculate $\Mbf$ in \eqref{eq:equiv1}. We have for $i,j$:
\begin{equation}
    \begin{split}
(\underset{\xbf\sim \mu}{\E}[(\xbf^\top \xbf) . \xbf \xbf^\top])_{i,j}&=\underset{\xbf}{\E}[\sum_{k} x_{k} x_{k}x_{i} x_{j}] =\sum_{k} \underset{x}{\E}[x_{k} x_{k}x_{i} x_{j}] \\
&\stackrel{*}{=}\sum_{k} \underset{x}{\E}[x_{k}x_{k}]\underset{\xbf}{\E}[x_{i}x_{j}]+\underset{\xbf}{\E}[x_{k}x_{i}]\underset{\xbf}{\E}[x_{k}x_{j}]+\underset{\xbf}{\E}[x_{k}x_{j}]\underset{\xbf}{\E}[x_{k}x_{i}] \\
&=\sum_{k} \Sigmab_{\mu}^{k,k}\Sigmab_{\mu}^{i,j}+\Sigmab_{\mu}^{k,i}\Sigmab_{\mu}^{k,j}+\Sigmab_{\mu}^{k,j}\Sigmab_{\mu}^{k,i} \\
&=\sum_{k} \Sigmab_{\mu}^{k,k}\Sigmab_{\mu}^{i,j}+2\Sigmab_{\mu}^{k,i}\Sigmab_{\mu}^{k,j}\\
&= \Sigmab_{\mu}^{i,j} \tr(\Sigmab_{\mu})+2(\Sigmab_{\mu}^\top \Sigmab_{\mu})_{i,j}
    \end{split}
\end{equation}
In (*) we used Isserlis' theorem. Hence $\underset{\xbf\sim \mu}{\E}[(\xbf^\top \xbf) . \xbf \xbf^\top]=\tr(\Sigmab_{\mu}).\Sigmab_{\mu} +2\Sigmab_{\mu}^\top \Sigmab_{\mu}$. In the same way: 
\begin{equation}
    \begin{split}
(\underset{\xbf,\xbf'}{\E}[(\xbf'^\top \xbf') . \xbf \xbf^\top])_{i,j}&=\underset{\xbf,\xbf'}{\E}[\sum_{k} x'_{k} x'_{k}x_{i} x_{j}] =\sum_{k} \underset{\xbf,\xbf'}{\E}[x'_{k} x'_{k}x_{i} x_{j}] \\
&=\sum_{k} \underset{\xbf,\xbf'}{\E}[x'_{k}x'_{k}]\underset{\xbf,\xbf'}{\E}[x_{i}x_{j}]+\underset{\xbf,\xbf'}{\E}[x'_{k}x_{i}]\underset{\xbf,\xbf'}{\E}[x'_{k}x_{j}]+\underset{\xbf,\xbf'}{\E}[x'_{k}x_{j}]\underset{\xbf,\xbf'}{\E}[x'_{k}x_{i}] \\
&\stackrel{*}{=}\sum_{k} \Sigmab_{\mu}^{k,k}\Sigmab_{\mu}^{i,j}= \tr(\Sigmab_{\mu})\Sigmab_{\mu}^{i,j}
    \end{split}
\end{equation}
In (*) we used the independence of $\xbf$ and $\xbf'$ and the fact that Gaussians are centered. Finally,
\begin{equation}
    \begin{split}
(\underset{\xbf,\xbf'}{\E}[(\xbf^\top \xbf') . \xbf' \xbf^\top])_{i,j}&=\underset{\xbf,\xbf'}{\E}[\sum_{k} x_{k} x'_{k}x'_{i} x_{j}] =\sum_{k} \underset{x,x'}{\E}[x'_{k} x_{k}x'_{i} x_{j}] \\
&=\sum_{k} \underset{\xbf,\xbf'}{\E}[x'_{k}x_{k}]\underset{\xbf,\xbf'}{\E}[x'_{i}x_{j}]+\underset{\xbf,\xbf'}{\E}[x'_{k}x'_{i}]\underset{\xbf,\xbf'}{\E}[x_{k}x_{j}]+\underset{\xbf,\xbf'}{\E}[x'_{k}x_{j}]\underset{x}{\E}[x_{k}x'_{i}] \\
&=\sum_{k} \Sigmab_{\mu}^{k,i}\Sigmab_{\mu}^{k,j}=(\Sigmab_{\mu}^\top\Sigmab_{\mu})_{i,j}
    \end{split}
\end{equation}
Overall we want to solve the following optimization problem:
\begin{equation}
 \min_{\Abf, \Abf\Sigmab_{\mu}\Abf^\top=\Sigmab_{\nu}}\ \froeb{\Mbf}{\Abf^\top \Abf}
\end{equation}
where:
\begin{equation}
\Mbf=-4\tr(\Sigmab_{\mu}).\Sigmab_{\mu}-8\Sigmab_{\mu}^\top\Sigmab_{\mu}
\end{equation}
\end{proof}

We can now prove the main result:
\begin{proof}[Of Theorem \ref{maintheo3}]
With previous notations and calculus, we consider the following change of variable:
\begin{equation}
    \Abf= \V_\nu \Dbf^{1/2}_\nu \Bbf  \Dbf^{-1/2}_\mu \V_\mu^\top
\end{equation}
then the pushforward equality becomes:
\begin{equation}
    \Bbf \Bbf^\top= \mathbf{I}_{q} \text{ \ie\ } \Bbf \in \Stief
\end{equation}
Indeed: 
\begin{equation}
\begin{split}
\Abf \Sigmab_{\mu}\Abf^\top&= ( \V_\nu  \Dbf^{1/2}_\nu \Bbf  \Dbf^{-1/2}_\mu \V_\mu^\top) (\V_\mu \Dbf_\mu \V_\mu^\top) \Abf^\top \\
&=  \V_\nu \Dbf^{1/2}_\nu \Bbf \Dbf^{1/2}_\mu \V_\mu^\top \Abf^\top \\
&=  \V_\nu \Dbf^{1/2}_\nu \Bbf \Dbf^{1/2}_\mu \V_\mu^\top (\V_\mu \Dbf^{-1/2}_\mu \Bbf^\top \Dbf^{1/2}_\nu \V_\nu^\top ) \\
&=  \V_\nu \Dbf^{1/2}_\nu \Bbf \Bbf^\top \Dbf^{1/2}_\nu \V_\nu^\top  \\
&= \Sigmab_\nu \iff \V_\nu \Dbf^{1/2}_\nu \Bbf \Bbf^\top \Dbf^{1/2}_\nu \V_\nu^\top =  \V_\nu  \Dbf_\nu \V_\nu^\top \\
& \iff \Bbf\Bbf^\top= \mathbf{I}_{q} \iff \Bbf \in \Stief
\end{split}
\end{equation}
With this change of variable the criterion becomes:
\begin{equation}
\begin{split}    
J_3(\Bbf)&=\froeb{\Mbf}{(\V_\mu \Dbf^{-1/2}_\mu  \Bbf^\top \Dbf^{1/2}_\nu \V_\nu^\top ) ( \V_\nu \Dbf^{1/2}_\nu \Bbf  \Dbf^{-1/2}_\mu \V_\mu^\top)}\\
    &= \froeb{\Mbf}{\V_\mu \Dbf^{-1/2}_\mu  \Bbf^\top \Dbf_\nu \Bbf  \Dbf^{-1/2}_\mu \V_\mu^\top}\\
&= \froeb{\Dbf^{-1/2}_\mu \V_\mu^\top \Mbf  \V_\mu \Dbf^{-1/2}_\mu}{\Bbf^\top  \Dbf_\nu  \Bbf}\\
         &= \froeb{\tilde{\Mbf}}{\Bbf^\top   \Dbf_\nu   \Bbf}\end{split}
\end{equation}
with 
\begin{equation}
\begin{split}
   \tilde{\Mbf}&= \Dbf^{-1/2}_\mu \V_\mu^\top \Mbf  \V_\mu \Dbf^{-1/2}_\mu\\
   &=\Dbf^{-1/2}_\mu \V_\mu^\top \left(-4\tr(\Sigmab_{\mu}).\Sigmab_{\mu}-8\Sigmab_{\mu}^\top\Sigmab_{\mu}\right)  \V_\mu \Dbf^{-1/2}_\mu \\
   &=\Dbf^{-1/2}_\mu\left(-4\tr(\Sigmab_{\mu})\Dbf_\mu\right)\Dbf_\mu^{-1/2}+\Dbf^{-1/2}_\mu(-8\Dbf^{2}_\mu)\Dbf^{-1/2}_\mu\\
   &=-4\tr(\Sigmab_{\mu}).\mathbf{I}_{p}-8\Dbf_\mu
\end{split}
\end{equation}
Overall we have the following optimization problem:
\begin{equation}
\begin{split}
    \min_{\Bbf,\Bbf \Bbf^\top=\mathbf{I}_{q}} \ \froeb{\tilde{\Mbf}}{\Bbf^\top\Dbf_\nu\Bbf}&= \min_{\Bbf,\Bbf \Bbf^\top=\mathbf{I}} \tr\big((-4\tr(\Sigmab_{\mu}).\mathbf{I}_{p}-8\Dbf_\mu)\Bbf^\top  \Dbf_\nu  \Bbf\big) \\
    &=\min_{\Bbf,\Bbf \Bbf^\top=\mathbf{I}_{q}} -4\tr(\Sigmab_{\mu}).\tr\big(\Bbf^\top  \Dbf_\nu  \Bbf\big)-8\tr(\Dbf_\mu\Bbf^\top  \Dbf_\nu  \Bbf) \\
    &=-4\tr(\Sigmab_{\mu}).\tr\big(\Sigmab_{\nu})+8\min_{\Bbf,\Bbf \Bbf^\top=\mathbf{I}_{q}}-\tr(\Dbf_\mu\Bbf^\top  \Dbf_\nu  \Bbf)
\end{split}
\end{equation}

Note that for the final cost we need also to compute:
\begin{equation}
    \begin{split}
        \underset{\xbf,\xbf' \sim \mu}{\E}[\|\xbf-\xbf'\|_{2}^{4}]&=\underset{\xbf,\xbf' \sim \mu}{\E}[(\|\xbf\|_{2}^{2}-2\scalar{\xbf}{\xbf'}{p}+\|\xbf'\|_{2}^{2})(\|\xbf\|_{2}^{2}-2\scalar{\xbf}{\xbf'}{p}+\|\xbf'\|_{2}^{2})] \\
        &=\tr(\Sigmab_\mu)^{2}+2\tr(\Sigmab_\mu^\top\Sigmab_\mu)+\tr(\Sigmab_\mu)^{2}+4\tr(\Sigmab_\mu^\top\Sigmab_\mu)+\tr(\Sigmab_\mu)^{2}+\tr(\Sigmab_\mu)^{2}+2\tr(\Sigmab_\mu^\top\Sigmab_\mu)\\
        &=4\tr(\Sigmab_\mu)^{2}+8\tr(\Sigmab_\mu^\top\Sigmab_\mu) 
    \end{split}
\end{equation}
Overall we have:
\begin{equation*}
    \begin{split}
\underset{\begin{smallmatrix}T\#\mu=\nu \\ \text{T linear} \end{smallmatrix}}{\min} \ J(T)&=4\tr(\Sigmab_\mu)^{2}+8\tr(\Sigmab_\mu^\top\Sigmab_\mu)+4\tr(\Sigmab_\nu)^{2}+8\tr(\Sigmab_\nu^\top\Sigmab_\nu) \\
&+2\left( -4\tr(\Sigmab_{\mu}).\tr\big(\Sigmab_{\nu})+8\min_{\Bbf,\Bbf \Bbf^\top=\mathbf{I}_{q}}-\tr(\Dbf_\mu\Bbf^\top  \Dbf_\nu  \Bbf)  \right)\\
&=4\tr(\Sigmab_\mu)^{2}+8\tr(\Sigmab_\mu^\top\Sigmab_\mu)+4\tr(\Sigmab_\nu)^{2}+8\tr(\Sigmab_\nu^\top\Sigmab_\nu)-8\tr(\Sigmab_{\mu}).\tr\big(\Sigmab_{\nu}) \\
&+16\min_{\Bbf,\Bbf \Bbf^\top=\mathbf{I}}-\tr(\Dbf_\mu\Bbf^\top  \Dbf_\nu  \Bbf) \\
&= 4(\tr(\Sigmab_\mu)-\tr(\Sigmab_\nu))^{2}+8(\tr(\Sigmab_\mu^{T} \Sigmab_\mu)+\tr(\Sigmab_\nu^{T} \Sigmab_\nu)) +16\min_{\Bbf,\Bbf \Bbf^\top=\mathbf{I}_{q}}-\tr(\Dbf_\mu\Bbf^\top  \Dbf_\nu  \Bbf)
    \end{split}
\end{equation*}
Since the covariances are symmetric it gives \eqref{eq:eqiv2}.

We can use Proposition \ref{prop:anstreicher} to solve $\min_{\Bbf,\Bbf \Bbf^\top=\mathbf{I}_{q}}-\tr(\Dbf_\mu\Bbf^\top  \Dbf_\nu  \Bbf)$. Indeed $-\Dbf_\mu$ is already diagonal which values are nonincreasing and $\Dbf_\nu$ is already diagonal which values are nondecreasing. In this way when $p=q$ $\min_{\Bbf,\Bbf \Bbf^\top=\mathbf{I}_{q}}-\tr(\Dbf_\mu\Bbf^\top  \Dbf_\nu  \Bbf)$ can be solved in close form with $\Bbf=\mathbf{I}_{q}$ which corresponds to $\Abf=\V_\nu  \Dbf^{1/2}_\nu \Dbf^{-1/2}_\mu \V_\mu^\top=\Sigmab_{\nu}^{1/2}\V_{\nu}\V_{\mu}^{\top}\Sigmab_{\mu}^{-1/2}$.
\end{proof}

\section{Proofs and additional results of Chapter \ref{cha:coot}}
\label{sec:properties}

This section contains all the proofs of the claims and additional results of the Chapter \ref{cha:coot}. We recall the notations of the Chapter. Two datasets are represented by matrices $\X=[\x_1,\dots,\x_n]^T\in\mathbb{R}^{n\times d}$ and $\X'=[\xbf'_1,\dots,\xbf'_{n'}]^T\in\mathbb{R}^{n'\times d'}$. The rows of the datasets are denoted as \emph{samples} and their columns as features. Let $\mu=\sum_{i=1}^{n} w_i \delta_{\xbf_i}$ and $\mu'=\sum_{i=1}^{n'} w_i' \delta_{\xbf_i'}$ be two empirical distributions related to the samples, where $\xbf_i\in\mathbb{R}^d$ and
$\xbf_i'\in\mathbb{R}^{d'}$. We refer in the following to $\w=[w_1,\dots,w_n]^\top$ and $\w'=[w_1',\dots,w_{n'}']^\top$ as to sample weights
vectors that both lie in the simplex ($\w\in\Delta_n$ and $\w'\in\Delta_{n'}$). In addition to them, we also introduce weights for the features that
are stored on vectors $\v\in\Delta_d$ and $\v'\in\Delta_{d'}$. Finally, we let $\vec$ denote the column-stacking operator.

\subsection{Proof of Proposition \ref{sec:metric_properties} -- COOT is a distance}
We recall the proposition:

\begin{prop*}[\COOT\ is a distance]
Suppose $L=|\cdot|^{p}, p \geq 1$, $n=n',d=d'$ and that the weights $\w,\w',\v,\v'$ are uniform. Then $\COOT(\X,\X')=0$ \textit{iff} there exists a permutation of the samples $\sigma_{1} \in S_{n}$ and of the features $\sigma_{2} \in S_d$, \textit{s.t}, $\forall i,k \ \X_{i,k}=\X'_{\sigma_{1}(i),\sigma_{2}(k)}$. Moreover, it is symmetric and satisfies the triangular inequality as long as $L$ satisfies the triangle inequality, \ie, $\COOT(\X,\X'')\leq \COOT(\X,\X')+\COOT(\X',\X'').$
\end{prop*}
\begin{proof}
The symmetry follows from the definition of \COOT. To prove the triangle inequality of \COOT\ for arbitrary measures, we will use the gluing lemma (see \cite{Villani}) which states the existence of couplings with a prescribed structure. 
Let $\X \in\mathbb{R}^{n\times d},\X' \in\mathbb{R}^{n'\times d'},\X'' \in\mathbb{R}^{n''\times d''}$ associated with $\w \in \Sigma_n,\v \in \Sigma_d,\w' \in \Sigma_n',\v' \in \Sigma_d',\w'' \in \Sigma_n'',\v'' \in \Sigma_d''$. Without loss of generality, we can suppose in the proof that all weights are different from zeros (otherwise we can consider $\tilde{w}_{i}=w_{i}$ if $w_{i}>0$ and $\tilde{w}_{i}=1$ if $w_{i}=0$ see proof of Proposition 2.2 in \cite{cot_peyre_cutu})

Let $(\GGs_{1},\GGv_{1})$ and $(\GGs_{2},\GGv_{2})$ be two couples of optimal solutions for the $\COOT$ problems associated with $\COOT(\X,\X',\w,\w',\v,\v')$ and $\COOT(\X',\X'',\w',\w'',\v',\v'')$ respectively.

We define:
\begin{equation*}
S_{1}=\GGs_{1}\text{diag}\left(\frac{1}{\w'}\right)\GGs_{2}, \quad
S_{2}=\GGv_{1}\text{diag}\left(\frac{1}{\v'}\right)\GGv_{2}
\end{equation*}

Then, it is easy to check that $S_{1} \in \Pi(\w,\w'')$ and $S_{2} \in \Pi(\v,\v'')$ (see \textit{e.g} Proposition 2.2 in \cite{cot_peyre_cutu}). We now show the following:
\begin{equation*}
\begin{split}
&\COOT(\X,\X'',\w,\w'',\v,\v'') \stackrel{*}{\leq} \langle \mathbf{L}(\X,\X'')\otimes S_{1}, S_{2} \rangle = \langle \mathbf{L}(\X,\X'')\otimes [\GGs_{1}\text{diag}(\frac{1}{\w'})\GGs_{2}], [\GGv_{1}\text{diag}(\frac{1}{\v'})\GGv_{2}] \rangle \\
&\stackrel{**}{\leq} \langle [\mathbf{L}(\X,\X')+\mathbf{L}(\X',\X'')]\otimes [\GGs_{1}\text{diag}(\frac{1}{\w'})\GGs_{2}], [\GGv_{1}\text{diag}(\frac{1}{\v'})\GGv_{2}] \rangle \\
&=\langle \mathbf{L}(\X,\X')\otimes [\GGs_{1}\text{diag}(\frac{1}{\w'})\GGs_{2}], [\GGv_{1}\text{diag}(\frac{1}{\v'})\GGv_{2}] \rangle +\langle \mathbf{L}(\X',\X'')\otimes [\GGs_{1}\text{diag}(\frac{1}{\w'})\GGs_{2}], [\GGv_{1}\text{diag}(\frac{1}{\v'})\GGv_{2}] \rangle,
\end{split}
\end{equation*}
\noindent where in (*) we used the suboptimality of $S_{1},S_{2}$ and in (**) the fact that $L$ satisfies the triangle inequality.

Now note that:
\begin{equation*}
\begin{split}
&\langle \mathbf{L}(\X,\X')\otimes [\GGs_{1}\text{diag}(\frac{1}{\w'})\GGs_{2}], [\GGv_{1}\text{diag}(\frac{1}{\v'})\GGv_{2}] \rangle +\langle \mathbf{L}(\X',\X'')\otimes [\GGs_{1}\text{diag}(\frac{1}{\w'})\GGs_{2}], [\GGv_{1}\text{diag}(\frac{1}{\v'})\GGv_{2}] \rangle\\
&= \sum_{i,j,k,l,e,o} L(X_{i,k},X'_{e,o}) \frac{\GGs_{1}{_{i,e}} \GGs_{2}{_{e,j}}}{w'_{e}} \frac{\GGv_{1}{_{k,o}} \GGv_{2}{_{o,l}}}{v'_{o}} +\sum_{i,j,k,l,e,o} L(X'_{e,o},X''_{j,l}) \frac{\GGs_{1}{_{i,e}} \GGs_{2}{_{e,j}}}{w'_{e}} \frac{\GGv_{1}{_{k,o}} \GGv_{2}{_{o,l}}}{v'_{o}} \\
&\stackrel{*}{=} \sum_{i,k,e,o} L(X_{i,k},X'_{e,o})\GGs_{1}{_{i,e}} \GGv_{1}{_{k,o}} +\sum_{l,j,e,o} L(X'_{e,o},X''_{j,l}) \GGs_{2}{_{e,j}} \GGv_{2}{_{o,l}}
\end{split}
\end{equation*}
where in (*) we used:
\begin{equation*}
\sum_{j} \frac{\GGs_{2}{_{e,j}}}{w'_{e}}=1, \ \sum_{l} \frac{\GGv_{2}{_{o,l}}}{v'_{o}}=1, \ \sum_{i} \frac{\GGs_{1}{_{i,e}}}{w'_{e}}=1, \ \sum_{k} \frac{\GGv_{1}{_{k,o}}}{v'_{o}}=1
\end{equation*}
Overall, from the definition of $\GGs_{1},\GGv_{1}$ and $\GGs_{2},\GGv_{2}$ we have:
\begin{equation*}
\begin{split}
&\COOT(\X,\X'',\w,\w'',\v,\v'')\leq  \COOT(\X,\X',\w,\w',\v,\v')+\COOT(\X',\X'',\w',\w'',\v',\v'').
\end{split}
\end{equation*}

For the identity of indiscernibles, suppose that $n=n',d=d'$ and that the weights $\w,\w',\v,\v'$ are uniform. Suppose that there exists a permutation of the samples $\sigma_{1} \in S_{n}$ and of the features $\sigma_{2} \in S_{d}$, \textit{s.t} $\forall i,k \in [\![n]\!]\times[\![d]\!], \ \X_{i,k}=\X'_{\sigma_{1}(i),\sigma_{2}(k)}$. We define the couplings $\pi^{s},\pi^{v}$ supported on the graphs of the permutations $\sigma_1,\sigma_2$ respectively, \textit{i.e} $\pi^{s}=(id \times \sigma_{1})$ and $\pi^{v}=(id \times \sigma_{2})$. These couplings have the prescribed marginals and lead to a zero cost hence are optimal.

Conversely, as described in the chapter, there always exists an optimal solution of \eqref{eq:co-optimal-transport} which lies on extremal points of the polytopes $\Pi(\w,\w')$ and $\Pi(\v,\v')$. When $n=n',d=d'$ and uniform weights are used, Birkhoff’s theorem \cite{birkhoff:1946} states that the set of extremal points of $\Pi(\frac{\one_{n}}{n},\frac{\one_{n}}{n})$ and $\Pi(\frac{\one_{d}}{d},\frac{\one_{d}}{d})$ are the set of permutation matrices so there exists an optimal solution $(\GGs_{*},\GGv_{*})$ supported on $\sigma_{*}^{s},\sigma_{*}^{v}$ respectively with $\sigma_{*}^{s},\sigma_{*}^{v} \in \mathbb{S}_{n} \times \mathbb{S}_{d}$. 
Then, if $\COOT(\X,\X')=0$, it implies that $\sum_{i,k} L(X_{i,k},X'_{\sigma_{*}^{s}(i),\sigma_{*}^{v}(k)})=0$. If $L=|\cdot|^{p}$ then $X_{i,k}=X'_{\sigma_{*}^{s}(i),\sigma_{*}^{v}(k)}$ which gives the desired result. If $n\neq n',d\neq d'$ the \COOT\ cost is always strictly positive as there exists a strictly positive element outside the diagonal.
\end{proof}

\subsection{Complexity of computing the value of \COOT}
\label{sec:complexity}
We recall the result:
\begin{lemma*}
The overall computational complexity of computing the value of $\COOT$ when $L=|.|^{2}$ is $O(\min\{(n+n')dd'+n'^{2}n;(d+d')nn'+d'^{2}d\})$.
\end{lemma*}
\begin{proof}
As mentionned in \cite{peyre2016gromov}, if $L$ can be written as $L(a,b)=f(a)+f(b)-h_{1}(a)h_{2}(b)$ then we have that
$$\mathbf{L}(\X,\X')\otimes \GGs=\mathbf{C}_{\X,\X'}-h_{1}(\X) \GGs h_{2}(\X')^{T},$$
where $\mathbf{C}_{\X,\X'}=\X \w \mathbbm{1}_{n'}^{T}+\mathbbm{1}_{n} \w'^{T}\X'^{T}$ so that the latter can be computed in $O(ndd'+n'dd')=O((n+n')dd')$. To compute the final cost, we must also calculate the inner product with $\GGv$ that can be done in $O(n'^{2}n)$ making the complexity of $\langle \mathbf{L}(\X,\X')\otimes \GGs, \GGv \rangle$ equal to $O((n+n')dd'+n'^{2}n)$. 

Finally, as the cost is symmetric \textit{w.r.t } $\GGs,\GGv$, we obtain the overall complexity of $O(\min\{(n+n')dd'+n'^{2}n;(d+d')nn'+d'^{2}d\})$.
\end{proof}

\subsection{Proofs of theorem \ref{equivalence_theo} -- Equivalence between QAP and BAP}

As pointed in \cite{Konno1976}, we can relate the solutions of a QAP and a BAP. We will prove the equivalent following result (maximization version of theorem \ref{equivalence_theo}):
\begin{theo*}
\label{equivalence_theo_2}
If $\mathbf{Q}$ is a positive semi-definite matrix, then problems: 
\begin{equation}
\label{qap2}
\begin{array}{cl}{\max _{\xbf} f(\xbf)} & {=\mathbf{c}^{T} \xbf+\frac{1}{2} \xbf^{T} \mathbf{Q} \xbf} \\ {\text {s.t.}} & {\mathbf{A} \xbf = \mathbf{b}},\;  {\xbf \geq 0}\end{array}
\end{equation}
\begin{equation}
\label{bilinearqap2}
\begin{array}{cl}{\max _{\xbf, \ybf} g(\xbf, \ybf)} & {=\frac{1}{2}\mathbf{c}^{T} \xbf+\frac{1}{2} \mathbf{c}^{T}\ybf+\frac{1}{2} \xbf^{T} \mathbf{Q} \ybf} \\ {\text {s.t.}} & {\mathbf{A} \xbf = \mathbf{b}, \mathbf{A} \ybf =\mathbf{b}},\;   {\xbf, \ybf \geq 0}\end{array}
\end{equation}
are equivalent. More precisely, if $\xbf^{*}$ is an optimal solution for \eqref{qap2}, then $(\xbf^{*},\xbf^{*})$ is a solution for \eqref{bilinearqap2} and if $(\xbf^{*},\ybf^{*})$ is optimal for \eqref{bilinearqap2}, then both $\xbf^{*}$ and $\ybf^{*}$ are optimal for \eqref{qap2}.
\end{theo*}

\begin{proof}
This proof follows the proof of Theorem 2.2 in \cite{Konno1976}. Let $\zbf^{*}$ be optimal for \eqref{qap2} and $(\xbf^{*},\ybf^{*})$ be optimal for \eqref{bilinearqap2}. Then, by definition, for all $\xbf$ satisfying the constraints of \eqref{qap2}, $f(\zbf^{*})\geq f(\xbf)$. In particular, $f(\zbf^{*})\geq f(\xbf^{*})=g(\xbf^{*},\xbf^{*})$ and $f(\zbf^{*})\geq f(\ybf^{*})=g(\ybf^{*},\ybf^{*})$. Also, $g(\xbf^{*},\ybf^{*})\geq \max_{\xbf,\xbf \ \text{s.t} \ \mathbf{A} \xbf = \mathbf{b}, \xbf\geq 0} g(\xbf,\xbf)=f(\zbf^{*})$.

To prove the theorem, it suffices to prove that 
\begin{equation}
\label{toprove}
f(\ybf^{*})=f(\xbf^{*})=g(\xbf^{*},\ybf^{*}) 
\end{equation}
since, in this case, $g(\xbf^{*},\ybf^{*})=f(\xbf^{*})\geq f(\zbf^{*})$ and $g(\xbf^{*},\ybf^{*})=f(\ybf^{*})\geq f(\zbf^{*})$.

Let us prove \eqref{toprove}. Since $(\xbf^{*},\ybf^{*})$ is optimal, we have:
\begin{equation*}
\begin{split}
 0\leq g(\xbf^{*},\ybf^{*})-g(\xbf^{*},\xbf^{*})&= \frac{1}{2} \mathbf{c}^{T}(\ybf^{*}-\xbf^{*}) + \frac{1}{2} {\xbf^{*}}^{T} \mathbf{Q} (\ybf^{*}-\xbf^{*})\\
0\leq g(\xbf^{*},\ybf^{*})-g(\ybf^{*},\ybf^{*})&= \frac{1}{2} \mathbf{c}^{T}(\xbf^{*}-\ybf^{*}) + \frac{1}{2} {\ybf^{*}}^{T} \mathbf{Q} (\xbf^{*}-\ybf^{*}).
\end{split}
\end{equation*}

By adding these inequalities we obtain:
\begin{equation*}
(\xbf^{*}-\ybf^{*})^{T} \mathbf{Q} (\xbf^{*}-\ybf^{*})\leq 0.
\end{equation*}

Since $\mathbf{Q}$ is positive semi-definite, this implies that $\mathbf{Q} (\xbf^{*}-\ybf^{*})=0$. So, using previous inequalities, we have $\mathbf{c}^{T}(\xbf^{*}-\ybf^{*})=0$, hence $g(\xbf^{*},\ybf^{*})=g(\xbf^{*},\xbf^{*})=g(\ybf^{*},\ybf^{*})$ as required. 

Note also that this result holds when we add a constant term to the cost function. 
\end{proof}

\subsection{Proofs of Proposition \ref{concavity_gw_theo2} -- Concavity of GW}
We recall the proposition:
\begin{prop*}
Let $L=|\cdot|^{2}$ and suppose that $\mathbf{C} \in \R^{n\times n},\mathbf{C}' \in \R^{n'\times n'}$ are squared Euclidean distance matrices such that $\mathbf{C}=\xbf \mathbf{1}_{n}^{T}+\mathbf{1}_{n}\xbf^{T}-2\X\X^{T}, \mathbf{C}'=\xbf' \mathbf{1}_{n'}^{T}+\mathbf{1}_{n'}\xbf'^{T}-2\X'\X'^{T}$ with $\xbf=\text{diag}(\X\X^T),\xbf'=\text{diag}(\X'\X'^T)$. Then, the GW problem can be written as {a concave quadratic program (QP) which Hessian reads} $\mathbf{Q}=-4*\X\X^T \otimes_{K} \X'\X'^T$.

If $\mathbf{C} \in \R^{n\times n},\mathbf{C}' \in \R^{n'\times n'}$ are inner products similarities, \ie\ such that $\mathbf{C}=\X\X^{T},\mathbf{C}'=\X'\X'^{T}$ then the GW is also a concave quadratic program (QP) which Hessian reads $\mathbf{Q}=-2*\X\X^T \otimes_{K} \X'\X'^T$.
\end{prop*}

This result is a consequence of the following lemma.
\begin{lemma}
\label{concavity_gw_theo_sup}
When $\C,\C'$ are squared Euclidean distance matrices as defined previously, the GW problem can be formulated as:
\begin{align*}
GW(\mathbf{C},\mathbf{C}',\w,\w')&=\min_{\GGs \in\Pi(\w,\w')} -4\vec(\mathbf{M})^{T}\vec(\GGs) -8\vec(\GGs)^{T}\mathbf{Q}\vec(\GGs) +Cte  
\end{align*}
with
\begin{equation*}
\begin{split}
&\mathbf{M}=\xbf\xbf'^T-2\xbf\w'^{T}\X'\X'^T-2\X\X^T\w \xbf'^{T} \text{ and } \mathbf{Q}=\X\X^T \otimes_{K} \X'\X'^T,\\
&Cte=\sum_{i}\|\xbf_i-\xbf_j\|_{2}^{4} \w_i \w_j + \sum_{i}\|\xbf'_i-\xbf'_j\|_{2}^{4} \w'_i \w'_j -4 \w^{T}\xbf\w'^{T}\xbf'
\end{split}
\end{equation*}
When $\C,\C'$ are inner product similarities as defined in Proposition \ref{concavity_gw_theo2}, the GW problem can be formulated as:
\begin{align*}
GW(\mathbf{C},\mathbf{C}',\w,\w')&=\min_{\GGs \in\Pi(\w,\w')} -4\vec(\GGs)^{T}\mathbf{Q}\vec(\GGs) +Cte'  
\end{align*}
\end{lemma}

\begin{proof}
Using the results in \cite{peyre2016gromov} for $L=|\cdot|^{2}$, we have $\mathbf{L}(\mathbf{C},\mathbf{C}') \otimes \GGs=c_{\mathbf{C},\mathbf{C}'}-2\mathbf{C}\GGs \mathbf{C}'$ with $c_{\mathbf{C},\mathbf{C}'}=(\mathbf{C})^{2}\w\one_{n'}^{T}+\one_{n}\w'^{T}(\mathbf{C}')^{2}$, where $(\mathbf{C})^{2}=(\mathbf{C}_{i,j}^{2})$ is applied element-wise.

We now have that:
\begin{equation*}
\begin{split}
 &\langle \mathbf{C}\GGs \mathbf{C}', \GGs \rangle =\tr\big[{\GGs}^{T}(\xbf \one_{n}^{T}+\one_{n}\xbf^{T}-2\X\X^T) \GGs (\xbf' \one_{n'}^{T}+\one_{n'}\xbf'^{T}-2\X'\X'^T) \big] \\
 &=\tr\big[({\GGs}^{T}\xbf\one_{n}^{T}+\w'\xbf^{T}-2{\GGs}^{T}\X\X^T)(\GGs \xbf'\one_{n'}^{T}+\w \xbf'^{T}-2\GGs \mathbf{X}'\mathbf{X}'^T) \big]\\
 &=\tr\big[{\GGs}^{T}\xbf\w'^{T}\xbf'\one_{n'}^{T}+{\GGs}^{T}\xbf\xbf'^{T}-2{\GGs}^{T}\xbf\w'^{T}\mathbf{X}'\mathbf{X}'^T +\w'\xbf^{T}\GGs \xbf' \one_{n'}^{T}+\w'\xbf^{T}\w \xbf'^{T} - 2 \w'\xbf^{T}\GGs \mathbf{X}'\mathbf{X}'^T \\
 &-2 {\GGs}^{T} \X\X^T\GGs \xbf' \one_{n'}^{T} -2 {\GGs}^{T}\X\X^T\w \xbf'^{T} +4 {\GGs}^{T} \X\X^T\GGs \mathbf{X}'\mathbf{X}'^T\big] \\
 &\stackrel{*}{=}\tr\big[{\GGs}^{T}\xbf\w'^{T}(\xbf'\one_{n'}^{T}+\one_{n'}\xbf'^{T})+{\GGs}^{T}\xbf\xbf'^{T}+\w'\xbf^{T}\w \xbf'^{T}-2{\GGs}^{T}\xbf\w'^{T}\mathbf{X}'\mathbf{X}'^T-2\w'\xbf^{T}\GGs \mathbf{X}'\mathbf{X}'^T\\
 &-2 {\GGs}^{T} \X\X^T\GGs \xbf' \one_{n'}^{T} -2 {\GGs}^{T}\X\X^T\w \xbf'^{T} +4 {\GGs}^{T} \X\X^T\GGs \mathbf{X}'\mathbf{X}'^T\big],
 \end{split}
\end{equation*}
where in (*) we used:
\begin{equation*}
\begin{split}
\tr(\w'\xbf^{T}\GGs \xbf' \one_{n'}^{T})=\tr(\xbf'\one_{n'}^{T}\w'\xbf^{T}\GGs)=\tr({\GGs}^{T}\xbf\w'^{T}\one_{n'}\xbf'^{T}).
 \end{split}
\end{equation*}

Moreover, since:
\begin{equation*}
\begin{split}
&\tr({\GGs}^{T}\X\X^T\GGs \xbf' \one_{n'}^{T})=\tr(\one_{n'}^{T}{\GGs}^{T}\X\X^T\GGs \xbf')=\tr(\w^{T}\X\X^T\GGs \xbf')=\tr({\GGs}^{T}\X\X^T\w \xbf'^{T})
 \end{split}
\end{equation*}
and $\tr(\w'\xbf^{T}\GGs \mathbf{X}'\mathbf{X}'^T)=\tr({\GGs}^{T}\xbf\w'^{T}\mathbf{X}'\mathbf{X}'^T)$, we can simplify the last expression to obtain:
\begin{equation*}
\begin{split}
 &\langle \mathbf{C}\GGs \mathbf{C}', \GGs \rangle =\tr\big[{\GGs}^{T}\xbf\w'^{T}(\xbf'\one_{n'}^{T}+\one_{n'}\xbf'^{T})+{\GGs}^{T}\xbf\xbf'^{T}+\w'\xbf^{T}\w \xbf'^{T} \\
 &-4{\GGs}^{T}\xbf\w'^{T}\mathbf{X}'\mathbf{X}'^T-4{\GGs}^{T}\X\X^T\w \xbf'^{T}+4 {\GGs}^{T} \X\X^T\GGs \mathbf{X}'\mathbf{X}'^T \big].
 \end{split}
\end{equation*}

Finally, we have that
\begin{equation*}
\begin{split}
 &\langle \mathbf{C}\GGs \mathbf{C}', \GGs \rangle =\tr\big[{\GGs}^{T}\xbf\w'^{T}\xbf'\one_{n'}^{T}+{\GGs}^{T}\xbf\w'^{T}\one_{n'}\xbf'^{T}+{\GGs}^{T}\xbf\xbf'^{T}\\
 &+\w'\xbf^{T}\w \xbf'^{T} -4{\GGs}^{T}\xbf\w'^{T}\mathbf{X}'\mathbf{X}'^T-4{\GGs}^{T}\X\X^T\w \xbf'^{T}+4 {\GGs}^{T} \X\X^T\GGs \mathbf{X}'\mathbf{X}'^T \big] \\
 &=\tr\big[ 2 \w' \xbf^{T} \w \xbf'^{T}+2{\GGs}^{T}\xbf\xbf'^{T} -4{\GGs}^{T}\xbf\w'^{T}\mathbf{X}'\mathbf{X}'^T-4{\GGs}^{T}\X\X^T\w \xbf'^{T}+4 {\GGs}^{T} \X\X^T\GGs \mathbf{X}'\mathbf{X}'^T \big] \\
 &=2\w^{T}\xbf \w'^{T}\xbf'+ 2 \langle \xbf\xbf'^T-2\xbf\w^{T}\mathbf{X}'\mathbf{X}'^T-2\X\X^T\w\xbf'^{T},\GGs\rangle +4\tr({\GGs}^{T} \X\X^T\GGs \mathbf{X}'\mathbf{X}'^T).
 \end{split}
\end{equation*}

The term $2\w^{T}\xbf \w'^{T}\xbf'$ is constant since it does not depend on the coupling. Also, we can verify that $c_{\mathbf{C},\mathbf{C}'}$ does not depend on $\GGs$ as follows:
\begin{equation*}
\begin{split}
\langle c_{\mathbf{C},\mathbf{C}'}, \GGs\rangle&=\sum_{i}\|\xbf_i-\xbf_j\|_{2}^{4} \w_i \w_j + \sum_{i}\|\xbf'_i-\xbf'_j\|_{2}^{4} \w'_i \w'_j
 \end{split}
\end{equation*}
implying that:
\begin{equation*}
\begin{split}
 &\langle c_{\mathbf{C},\mathbf{C}'}-2\mathbf{C}\GGs \mathbf{C}', \GGs \rangle = Cte-4 \langle \xbf\xbf'^T-2\xbf\w^{T}\mathbf{X}'\mathbf{X}'^T-2\X^{T}\X\w\xbf'^{T}, \GGs \rangle -8 \tr({\GGs}^{T} \X\X^T\GGs \mathbf{X}'\mathbf{X}'^T).
 \end{split}
\end{equation*}
We can rewrite this equation as stated in the proposition using the $\vec$ operator. 

Using a standard QP form $\mathbf{c}^T \xbf +\frac{1}{2}\xbf \mathbf{Q}' \xbf^{T}$ with $\mathbf{c}=-4\vec(\mathbf{M})$ and $\mathbf{Q}'=-4\X\X^T \otimes_{K} \X'\X'^T$ we see that the Hessian is negative semi-definite as the opposite of a Kronecker product of positive semi-definite matrices $\X\X^T$ and $\X'\X'^T$.

The inner product case is the same calculus, only the constant term changes and $\Mbf=0$. Note that both calculus were also made in Chapter \ref{cha:gw_euclidean}, and precisely in lemma \ref{calculation_gw}.
\end{proof}

\subsection{Proof of Proposition \ref{prop:cot_equal_gw} -- Equality between COOT and GW in the concave regime}
We recall the proposition:
\begin{prop*}
Let $\mathbf{C} \in \R^{n\times n},\mathbf{C}' \in \R^{n'\times n'}$ be any symmetric matrices, then: 
$$\COOT(\mathbf{C},\mathbf{C}',\w,\w',\w,\w')\leq GW(\mathbf{C},\mathbf{C}',\w,\w').$$
The converse is also true {under the hypothesis of Proposition \ref{concavity_gw_theo2}}. In this case, if $(\GGs_{*},\GGv_{*})$ is an optimal solution of $\COOT$, then both $\GGs_{*},\GGv_{*}$ are solutions of $\gw$. Conversely, if $\GGs_{*}$ is an optimal solution of $\gw$, then $(\GGs_{*},\GGs_{*})$ is an optimal solution for $\COOT$.
\end{prop*}
\begin{proof}
The inequality follows from the fact that any optimal solution of the GW problem is an admissible solution for the \COOT\ problem, hence the inequality is true by suboptimality of this optimal solution.

For the equality part, by following the same calculus as in the proof of lemma \ref{concavity_gw_theo_sup}, we can verify that:
\begin{equation*}
\begin{split}
\COOT(\mathbf{C},\mathbf{C}',\w,\w',\w,\w')&=\min_{\GGs \in\Pi(\w,\w')} -2\vec(\mathbf{M})^{T}\vec(\GGs)\\
&-2\vec(\mathbf{M})^{T}\vec(\GGv) -8\vec(\GGs)^{T}\mathbf{Q}\vec(\GGv) +Cte,
\end{split}
\end{equation*}
with $\mathbf{M},\mathbf{Q}$ as defined in lemma \ref{concavity_gw_theo_sup}. Since $\mathbf{Q}$ is negative semi-definite, we can apply Theorem \ref{equivalence_theo} to prove that both problems are equivalent and lead to the same cost and that every optimal solution of GW is an optimal solution of \COOT\ and vice versa. Same applies for the inner product case.
\end{proof}

\subsection{Equivalence of DC algorithm and Frank-Wolfe algorithm for GW}
\label{sec:equivalence}
We recall the result:
\begin{prop*}
When $\mathbf{X}=\mathbf{C}$, $\mathbf{X}'=\mathbf{C}'$ are squared Euclidean distance matrices or inner product similarities the iterations of Algorithm \ref{alg:bcd2} are the same as the iteration of the FW procedure defined in Chapter \ref{cha:fgw} for solving $\gw$ (provided that the initialization is the same). 
\end{prop*}

\begin{proof}
Using Proposition \ref{prop:cot_equal_gw}, we know that when $\mathbf{X}=\mathbf{C}$, $\mathbf{X}'=\mathbf{C}'$ are squared Euclidean distance matrices or inner product similarities, then there is an optimal solution of the form $(\GG^{*},\GG^{*})$. In this case, we can set $\GGs_{(k)}=\GGv_{(k)}$ during the iterations of Algorithm \ref{alg:bcd} to obtain an optimal solution for both \COOT\ and GW. This reduces to Algorithm \ref{alg:bcd2} that corresponds to a DC algorithm where the quadratic form is replaced by its linear upper bound. 

Below, we prove that this DC algorithm for solving GW problems is equivalent to the Frank-Wolfe (FW) based algorithm presented in Chapter \ref{cha:fgw} and recalled in Algorithm \ref{alg:cg_gw} when $L=|\cdot|^2$ and for squared Euclidean distance matrices $\mathbf{C}', \mathbf{C}''$. 

\begin{algorithm}[H]
    \caption{\label{alg:cg_gw}
     FW Algorithm for GW (see Chapter \ref{cha:fgw})}
    \begin{algorithmic}[1]
    \State \textbf{Input:} maxIt, thd
        \State $\pi^{(0)}\leftarrow \w\w'^\top$
        \While {$k < $ maxIt {\bf and} $err >$ thd} 
        \State $\mathbf{G}\leftarrow$ Gradient from equation \eqref{eq:gromov} \emph{w.r.t.} $\GGs_{(k-1)}$
        \State $\tilde\GGs_{(k)}\leftarrow OT(\w,\w', \mathbf{G})$
        \State $\mathbf{z}_{k}(\tau) \leftarrow \GGs_{(k-1)}+\tau(\tilde\GGs_{(k)}-\GGs_{(k-1)})$ for $\tau\in(0,1)$
        \State $\tau^{(k)}\leftarrow \underset{\tau\in(0,1)}{\text{argmin}} \langle \L(\mathbf{C},\mathbf{C}') \otimes \mathbf{z}_{k}(\tau), \mathbf{z}_{k}(\tau) \rangle$
        \State $\GGs_{(k)}\leftarrow (1-\tau^{(k)})\GGs_{(k-1)}+\tau^{(k)}\tilde\GGs_{(k)} $
         \State $err \leftarrow ||\GGs_{(k-1)} - \GGs_{(k)}||_F$
          \State $k\leftarrow k+1$  
          \EndWhile

    \end{algorithmic}
\end{algorithm}

The cases when $L=|\cdot|^{2}$ and $\mathbf{C},\mathbf{C}'$ are squared Euclidean distance matrices or inner product similarities have interesting implications in practice, since in this case the resulting GW problem is a concave QP (as explained in this Chapter and shown in Lemma \ref{concavity_gw_theo_sup}). In \cite{NIPS2018_7323}, the authors investigated the solution to QP with \emph{conditionally concave energies} using a FW algorithm and showed that in this case the line-search step of the FW is always $1$. Moreover, as shown in Proposition \ref{concavity_gw_theo_sup}, the GW problem can be written as a concave QP with concave energy and is minimizing \textit{a fortiori} a conditionally concave energy. Consequently, the line-search step of the FW algorithm proposed in Chapter \ref{cha:fgw} and described in Algorithm.\ref{alg:cg_gw} always leads to an optimal line-search step of $1$. In this case, the Algorithm.\ref{alg:cg_gw} is equivalent to Algorithm.\ref{alg:cg_gw_squared} goven below, since $\tau^{(k)}=1$ for all $k$.
\begin{algorithm}[H]
    \caption{\label{alg:cg_gw_squared}
     FW Algorithm for GW with squared Euclidean distance matrices or inner product similarities}
    \begin{algorithmic}[1]
    \State \textbf{Input:} maxIt, thd
        \State $\pi^{(0)}\leftarrow \w\w'^\top$
        \While {$k < $ maxIt {\bf and} $err >$ thd} 
        \State $\mathbf{G}\leftarrow$ Gradient from equation \eqref{eq:gromov} \emph{w.r.t.} $\GGs_{(k-1)}$
        \State $\GGs_{(k)}\leftarrow OT(\w,\w', \mathbf{G})$
         \State $err \leftarrow ||\GGs_{(k-1)} - \GGs_{(k)}||_F$
          \State $k\leftarrow k+1$  
          \EndWhile 
    \end{algorithmic}
\end{algorithm}

Finally, by noticing that in the step 3 of Algorithm \ref{alg:cg_gw_squared} the gradient of \eqref{eq:gromov} \textit{w.r.t } $\GGs_{(k-1)}$ is $2L(\mathbf{C},\mathbf{C}')\otimes \GGs_{(k-1)}$, which gives the same OT solution as for the OT problem in step 3 of Algorithm \ref{alg:bcd2}, we can conclude that the iterations of both algorithms are equivalent.
\end{proof}

\subsection{Additional results -- Relation with election isomorphism problem}
\label{sec:election_iso}
This section shows that \COOT\ approach can be used to solve the election isomorphism problem defined in \cite{faliszewski19} as follows: let $E=(C,V)$ and $E' = (C',V')$ be two elections, where $C = \{c_1,\dots,c_m\}$ (resp. $C'$) denotes a set of candidates and $V = (v_1, \dots,v_n)$ (resp. $V'$) denotes a set of voters, where each voter $v_i$ has a preference order, also denoted by $v_i$. The two elections $E=(C,V)$ and $E' = (C',V')$, where $\abs{C} = \abs{C'}$, $V = (v_1,\dots,v_n)$, and $V' = (v_1',\dots,v_n')$, are said to be isomorphic if there exists a bijection $\sigma: C \rightarrow C'$ and a permutation $\nu \in S_n$ such that $\sigma(v_i) = v'_{\nu(i)}$ for all $i \in [n]$. The authors further propose a distance underlying this problem defined as follows:
\begin{align*}
        \text{d-ID}(E,E') = \min_{\nu \in S_n} \min_{\sigma \in \Pi(C,C')} \sum_{i=1}^n d\left(\sigma(v_i),v'_{\nu(i)}\right),
\end{align*}
where $S_n$ denotes the set of all permutations over $\{1, \dots, n\}$,
$\Pi(C,C')$ is a set of bijections and $d$ is an arbitrary distance between
preference orders. The authors of \cite{faliszewski19} compute
$\text{d-ID}(E,E')$ in practice by expressing it as the following Integer Linear
Programming problem over the tensor $\mathbf{P}_{ijkl} = M_{ij}N_{kl}$ where
$\mathbf{M} \in \mathbb{R}^{m\times m}$, $\mathbf{N} \in \mathbb{R}^{n\times n}$
\begin{align}
\min_{\mathbf{P}, \mathbf{N}, \mathbf{M}} &\sum_{i,j,k,l} P_{k,l,i,j} \vert \text{pos}_{v_i}(c_k) - \text{pos}_{v'_j}(c'_l) \vert \notag\\
\text{s.t.} \quad & (\mathbf{N}\bm{1}_n)_k = 1,\ \forall k, (\mathbf{N}^\top\bm{1}_n)_l = 1,\ \forall l \label{ref:margelect}\\
							 & (\mathbf{M}\bm{1}_m)_i = 1,\ \forall i, (\mathbf{M}^\top\bm{1}_m)_j = 1,\ \forall j\notag\\
							 & P_{kl} \leq N_{k,l}, P_{i,j,k,l} \leq M_{i,j}, \ \forall i,j,k,l\notag\\
							 & \sum_{i,k} P_{i,j,k,l} = 1, \ \forall j,l
\label{ref:electionot}
\end{align}
where $\text{pos}_{v_i}(c_k)$ denotes the position of candidate $c_k$ in the
preference order of voter $v_i$. Let us now define two matrices $\X$ and $\X'$
such that $\X_{i,k} = \text{pos}_{v_i}(c_k)$ and $\X'_{j,l} =
\text{pos}_{v'_j}(c'_l)$ and denote by $\GGs_{*},\GGv_{*}$ a minimizer of
$\COOT(\X,\X', \bm{1}_n/n, \bm{1}_n/n, \bm{1}_m/m, \bm{1}_m/m)$ with $L=|\cdot|$ and by
$\mathbf{N}^*, \mathbf{M}^*$ the minimizers of problem \eqref{ref:margelect},
respectively. 

As shown in the chapter, there exists an optimal solution for $\COOT(\X,\X')$ given by permutation matrices as solutions of the Monge-Kantorovich problems for uniform distributions supported on the same number of elements. Then, one may show that the solution of the two problems coincide modulo a multiplicative factor, \ie, $\GGs_{*} = \frac{1}{n} \mathbf{N}^*$ and  $\GGv_{*} = \frac{1}{m} \mathbf{M}^*$ are optimal since $\abs{C} = \abs{C'}$ and $\abs{V} = \abs{V'}$. For $\GGs_{*}$ (the same reasoning holds for $\GGv_{*}$ as well), we have that
\[
  (\GGs_{*})_{ij} = \left\{\begin{array}{l}
    \frac{1}{n}, \quad j = \nu^*_i \\
    0, \quad \text{otherwise}.
  \end{array}\right.
\]
where $\nu^*_i$ is a permutation of voters in the two sets. The only difference
between the two solutions $\GGs_{*}$ and $\mathbf{N}^*$ thus stems from marginal
constraints \eqref{ref:margelect}. To conclude, we note that \COOT\ is a more
general approach as it is applicable for general loss functions $L$, contrary to
the Spearman distance used in \cite{faliszewski19}, and generalizes to the cases
where $n\neq n'$ and $m\neq m'$.

\subsection{Additional results -- Complementary results for the HDA experiment}
\label{sec:hda_supp}
Here, we present the results for the heterogeneous domain adaptation experiment not included in section \ref{sec:hda}. Table~\ref{tab:D_2_G} follows the same experimental protocol as in the chapter but 
shows the two cases where $n_t=1$ and $n_t=5$. Table~\ref{tab:G_2_D_sup} and Table~\ref{tab:G_2_D_unsup} contain the results for the adaptation from GoogleNet to Decaf features,
in a semi-supervised and unsupervised scenarios, respectively Overall, the results are coherent with those from the chapter: in both settings, when $n_t=5$, one can see that the  performance differences between {SGW} and {COOT} is rather significant.
\begin{table}[t]
	\begin{center}	
	\resizebox{\columnwidth}{!}{
	\begin{tabular}{ccccccc}
		\toprule
			\multicolumn{7}{c}{Decaf  $\rightarrow$ GoogleNet }\\
			\midrule
			{Domains} & {Baseline} & {CCA} & {KCCA} & {EGW} & {SGW} & {COOT}\\
						\midrule
			\multicolumn{7}{c}{$n_t=1$}\\
			\midrule
			C$\rightarrow$W & $30.47$$\pm 6.90$ & $13.37$$\pm 7.23$ & $29.21$$\pm 13.14$ & $10.21$$\pm 1.31$ & $\underline{ 66.95}$$\pm 7.61$ & $\bf 77.74$$\pm 4.80$\\
			W$\rightarrow$C & $26.53$$\pm 7.75$ & $16.26$$\pm 5.18$ & $40.68$$\pm 12.02$ & $10.11$$\pm 0.84$ & $\underline{ 80.16}$$\pm 4.78$ & $\bf 87.89$$\pm 2.65$\\
			W$\rightarrow$W & $30.63$$\pm 7.78$ & $13.42$$\pm 1.38$ & $36.74$$\pm 8.38$ & $8.68$$\pm 2.36$ & $\underline{ 78.32}$$\pm 5.86$ & $\bf 89.11$$\pm 2.78$\\
			W$\rightarrow$A & $30.21$$\pm 7.51$ & $12.47$$\pm 2.99$ & $39.11$$\pm 6.85$ & $9.42$$\pm 2.90$ & $\underline{ 80.00}$$\pm 3.24$ & $\bf 89.05$$\pm 2.84$\\
			A$\rightarrow$C & $41.89$$\pm 6.59$ & $12.79$$\pm 2.95$ & $28.84$$\pm 6.24$ & $9.89$$\pm 1.17$ & $\underline{ 72.00}$$\pm 8.91$ & $\bf 84.21$$\pm 3.92$\\
			A$\rightarrow$W & $39.84$$\pm 4.27$ & $19.95$$\pm 23.40$ & $38.16$$\pm 19.30$ & $12.32$$\pm 1.56$ & $\underline{ 75.84}$$\pm 7.37$ & $\bf 89.42$$\pm 4.24$\\
			A$\rightarrow$A & $42.68$$\pm 8.36$ & $15.21$$\pm 7.36$ & $38.26$$\pm 16.99$ & $13.63$$\pm 2.93$ & $\underline{ 75.53}$$\pm 6.25$ & $\bf 91.84$$\pm 2.48$\\
			C$\rightarrow$C & $28.58$$\pm 7.40$ & $18.37$$\pm 17.81$ & $35.11$$\pm 17.96$ & $11.05$$\pm 1.63$ & $\underline{ 61.21}$$\pm 8.43$ & $\bf 78.11$$\pm 5.77$\\
			C$\rightarrow$A & $31.63$$\pm 4.25$ & $15.11$$\pm 5.10$ & $33.84$$\pm 9.10$ & $11.84$$\pm 1.67$ & $\underline{ 66.26}$$\pm 7.95$ & $\bf 82.11$$\pm 2.58$\\
						\midrule
			\bf Mean & $33.61$$\pm 5.77$ & $15.22$$\pm 2.44$ & $35.55$$\pm 3.98$ & $10.80$$\pm 1.47$ & $\underline{ 72.92}$$\pm 6.37$ & $\bf 85.50$$\pm 4.89$\\
			\midrule
			\multicolumn{7}{c}{$n_t=5$}\\
			\midrule
			C$\rightarrow$W & $74.27$$\pm 5.53$ & $14.53$$\pm 7.37$ & $73.27$$\pm 4.99$ & $11.40$$\pm 1.13$ & $\underline{ 84.00}$$\pm 3.99$ & $\bf 85.53$$\pm 2.67$\\
			W$\rightarrow$C & $90.27$$\pm 2.67$ & $21.13$$\pm 6.85$ & $85.00$$\pm 3.44$ & $10.60$$\pm 1.05$ & $\bf 95.20$$\pm 2.84$ & $\underline{ 94.53}$$\pm 1.83$\\
			W$\rightarrow$W & $90.93$$\pm 2.50$ & $15.80$$\pm 3.27$ & $90.67$$\pm 2.95$ & $9.80$$\pm 2.60$ & $\bf 95.40$$\pm 2.47$ & $\underline{ 94.93}$$\pm 2.70$\\
			W$\rightarrow$A & $90.47$$\pm 2.92$ & $16.67$$\pm 4.85$ & $87.93$$\pm 2.47$ & $9.80$$\pm 2.68$ & $\underline{ 95.40}$$\pm 1.53$ & $\bf 95.80$$\pm 2.15$\\
			A$\rightarrow$C & $\underline{ 88.33}$$\pm 2.33$ & $15.73$$\pm 4.64$ & $83.13$$\pm 2.84$ & $10.40$$\pm 1.89$ & $84.47$$\pm 5.81$ & $\bf 91.47$$\pm 1.45$\\
			A$\rightarrow$W & $\underline{ 88.40}$$\pm 3.17$ & $13.60$$\pm 6.25$ & $87.27$$\pm 2.82$ & $11.87$$\pm 2.40$ & $87.87$$\pm 4.66$ & $\bf 93.00$$\pm 1.96$\\
			A$\rightarrow$A & $86.20$$\pm 3.08$ & $14.07$$\pm 2.93$ & $87.00$$\pm 3.48$ & $14.07$$\pm 1.65$ & $\underline{ 89.80}$$\pm 2.58$ & $\bf 92.20$$\pm 1.69$\\
			C$\rightarrow$C & $75.93$$\pm 4.83$ & $13.13$$\pm 2.98$ & $70.47$$\pm 3.45$ & $11.13$$\pm 1.52$ & $\bf 85.73$$\pm 3.54$ & $\underline{ 84.60}$$\pm 2.32$\\
			C$\rightarrow$A & $73.47$$\pm 3.62$ & $15.47$$\pm 6.50$ & $74.13$$\pm 5.42$ & $11.20$$\pm 2.47$ & $\underline{ 85.07}$$\pm 3.26$ & $\bf 87.20$$\pm 1.78$\\
			\midrule
			\bf Mean & $84.25$$\pm 7.01$ & $15.57$$\pm 2.25$ & $82.10$$\pm 7.03$ & $11.14$$\pm 1.23$ & $\underline{ 89.21}$$\pm 4.64$ & $\bf 91.03$$\pm 3.97$\\
			\bottomrule
			
		\end{tabular}
		}
	\end{center}
	\caption{\label{tab:D_2_G}{\bf Semi-supervised Heterogeneous Domain Adaptation} results for adaptation from Decaf  to GoogleNet  representations with different values of $n_t$.}
\end{table}

\begin{table}[t]
	\begin{center}
	\resizebox{.9\columnwidth}{!}{
	\begin{tabular}{ccccccc}
		\toprule
			\multicolumn{7}{c}{GoogleNet  $\rightarrow$ Decaf }\\
			\midrule
			{Domains} & {Baseline} & {CCA} & {KCCA} & {EGW} & {SGW} & {COOT}\\
						\midrule
			\multicolumn{7}{c}{$n_t=1$}\\
			\midrule
			C$\rightarrow$A & $31.16$$\pm 6.87$ & $12.16$$\pm 2.78$ & $33.32$$\pm 2.47$ & $7.00$$\pm 2.11$ & $\underline{ 77.16}$$\pm 8.00$ & $\bf 83.26$$\pm 5.00$\\
			C$\rightarrow$C & $30.42$$\pm 3.73$ & $13.74$$\pm 5.29$ & $32.58$$\pm 9.98$ & $12.47$$\pm 2.81$ & $\underline{ 76.63}$$\pm 8.31$ & $\bf 86.21$$\pm 3.26$\\
			W$\rightarrow$A & $37.68$$\pm 4.04$ & $15.79$$\pm 3.71$ & $34.58$$\pm 5.71$ & $14.32$$\pm 1.77$ & $\underline{ 86.68}$$\pm 1.90$ & $\bf 89.95$$\pm 3.43$\\
			A$\rightarrow$C & $35.95$$\pm 3.89$ & $15.32$$\pm 8.18$ & $40.16$$\pm 17.54$ & $13.21$$\pm 3.49$ & $\underline{ 87.89}$$\pm 4.03$ & $\bf 90.68$$\pm 7.54$\\
			A$\rightarrow$A & $36.89$$\pm 4.73$ & $13.84$$\pm 2.47$ & $34.84$$\pm 10.44$ & $13.16$$\pm 1.56$ & $\underline{ 89.79}$$\pm 3.93$ & $\bf 94.68$$\pm 2.21$\\
			W$\rightarrow$W & $32.05$$\pm 4.63$ & $19.89$$\pm 11.82$ & $36.26$$\pm 21.98$ & $10.00$$\pm 2.59$ & $\underline{ 84.21}$$\pm 4.55$ & $\bf 90.42$$\pm 2.66$\\
			W$\rightarrow$C & $32.68$$\pm 5.56$ & $21.53$$\pm 21.01$ & $33.79$$\pm 22.72$ & $11.47$$\pm 3.03$ & $\underline{ 86.26}$$\pm 3.41$ & $\bf 89.53$$\pm 1.92$\\
			A$\rightarrow$W & $33.84$$\pm 4.75$ & $16.00$$\pm 7.74$ & $39.32$$\pm 18.94$ & $11.00$$\pm 4.01$ & $\underline{ 87.21}$$\pm 3.67$ & $\bf 91.53$$\pm 5.85$\\
			C$\rightarrow$W & $32.32$$\pm 7.76$ & $15.58$$\pm 7.72$ & $34.05$$\pm 15.96$ & $12.89$$\pm 2.52$ & $\underline{ 81.84}$$\pm 3.51$ & $\bf 84.84$$\pm 5.71$\\
			\midrule
			\bf Mean & $33.67$$\pm 2.45$ & $15.98$$\pm 2.81$ & $35.43$$\pm 2.50$ & $11.73$$\pm 2.08$ & $\underline{ 84.19}$$\pm 4.43$ & $\bf 89.01$$\pm 3.38$\\
			\midrule
			\multicolumn{7}{c}{$n_t=3$}\\
			\midrule
			C$\rightarrow$A & $76.35$$\pm 4.15$ & $17.47$$\pm 3.45$ & $73.94$$\pm 4.53$ & $7.41$$\pm 2.27$ & $\underline{ 88.24}$$\pm 2.23$ & $\bf 89.88$$\pm 0.94$\\
			C$\rightarrow$C & $78.94$$\pm 3.61$ & $18.18$$\pm 3.44$ & $69.94$$\pm 3.51$ & $14.18$$\pm 3.16$ & $\underline{ 89.71}$$\pm 2.25$ & $\bf 91.06$$\pm 1.91$\\
			W$\rightarrow$A & $85.41$$\pm 3.25$ & $19.29$$\pm 3.10$ & $80.59$$\pm 3.82$ & $14.24$$\pm 2.72$ & $\underline{ 94.76}$$\pm 1.45$ & $\bf 95.29$$\pm 2.35$\\
			A$\rightarrow$C & $89.53$$\pm 4.05$ & $23.18$$\pm 7.17$ & $80.59$$\pm 6.30$ & $13.88$$\pm 2.69$ & $\underline{ 93.76}$$\pm 2.72$ & $\bf 94.76$$\pm 1.83$\\
			A$\rightarrow$A & $89.76$$\pm 1.92$ & $17.00$$\pm 3.11$ & $83.71$$\pm 3.30$ & $14.41$$\pm 2.28$ & $\underline{ 93.29}$$\pm 2.09$ & $\bf 95.53$$\pm 1.45$\\
			W$\rightarrow$W & $86.65$$\pm 5.07$ & $21.88$$\pm 4.78$ & $84.65$$\pm 3.67$ & $9.94$$\pm 2.37$ & $\bf 94.88$$\pm 1.79$ & $\underline{ 94.53}$$\pm 1.66$\\
			W$\rightarrow$C & $88.94$$\pm 5.02$ & $22.59$$\pm 9.23$ & $80.06$$\pm 5.65$ & $13.65$$\pm 3.15$ & $\bf 96.18$$\pm 1.15$ & $\underline{ 95.29}$$\pm 2.91$\\
			A$\rightarrow$W & $90.29$$\pm 1.35$ & $22.35$$\pm 7.00$ & $87.88$$\pm 2.53$ & $13.88$$\pm 3.60$ & $\underline{ 94.53}$$\pm 1.54$ & $\bf 95.35$$\pm 1.59$\\
			C$\rightarrow$W & $78.59$$\pm 3.44$ & $22.53$$\pm 13.42$ & $80.12$$\pm 2.95$ & $11.59$$\pm 3.25$ & $\underline{ 89.29}$$\pm 1.86$ & $\bf 89.59$$\pm 2.22$\\
			\midrule
			\bf Mean & $84.94$$\pm 5.19$ & $20.50$$\pm 2.34$ & $80.16$$\pm 5.12$ & $12.58$$\pm 2.31$ & $\underline{ 92.74}$$\pm 2.72$ & $\bf 93.48$$\pm 2.38$\\
			\midrule
			\multicolumn{7}{c}{$n_t=5$}\\
			\midrule
			C$\rightarrow$A & $84.20$$\pm 2.65$ & $18.60$$\pm 3.75$ & $84.33$$\pm 2.33$ & $6.40$$\pm 1.27$ & $\bf 92.13$$\pm 2.61$ & $\underline{ 91.93}$$\pm 2.05$\\
			C$\rightarrow$C & $85.33$$\pm 2.76$ & $21.80$$\pm 5.91$ & $78.60$$\pm 2.74$ & $13.47$$\pm 2.00$ & $\underline{ 91.33}$$\pm 2.48$ & $\bf 92.27$$\pm 2.67$\\
			W$\rightarrow$A & $95.13$$\pm 2.29$ & $31.00$$\pm 9.67$ & $91.93$$\pm 2.82$ & $14.67$$\pm 1.40$ & $\underline{ 96.13}$$\pm 2.04$ & $\bf 96.40$$\pm 1.84$\\
			A$\rightarrow$C & $91.67$$\pm 2.60$ & $21.80$$\pm 4.35$ & $85.33$$\pm 3.27$ & $13.40$$\pm 3.63$ & $\bf 95.47$$\pm 1.51$ & $\underline{ 94.87}$$\pm 1.27$\\
			A$\rightarrow$A & $93.20$$\pm 1.57$ & $23.33$$\pm 4.66$ & $89.67$$\pm 1.98$ & $13.27$$\pm 2.10$ & $\bf 95.33$$\pm 1.07$ & $\underline{ 95.00}$$\pm 1.37$\\
			W$\rightarrow$W & $95.00$$\pm 2.33$ & $23.80$$\pm 5.48$ & $92.13$$\pm 1.78$ & $11.20$$\pm 2.58$ & $\underline{ 96.47}$$\pm 1.93$ & $\bf 96.67$$\pm 1.37$\\
			W$\rightarrow$C & $95.67$$\pm 1.50$ & $28.27$$\pm 9.71$ & $87.67$$\pm 3.79$ & $14.27$$\pm 3.19$ & $\bf 97.67$$\pm 1.31$ & $\underline{ 96.93}$$\pm 2.25$\\
			A$\rightarrow$W & $92.13$$\pm 2.36$ & $22.67$$\pm 3.94$ & $89.20$$\pm 3.14$ & $11.67$$\pm 2.50$ & $\underline{ 93.60}$$\pm 1.40$ & $\bf 94.27$$\pm 2.11$\\
			C$\rightarrow$W & $84.00$$\pm 3.45$ & $20.40$$\pm 4.31$ & $82.53$$\pm 3.56$ & $11.07$$\pm 3.70$ & $\underline{ 90.20}$$\pm 2.23$ & $\bf 92.40$$\pm 1.69$\\
			\midrule
			\bf Mean & $90.70$$\pm 4.57$ & $23.52$$\pm 3.64$ & $86.82$$\pm 4.26$ & $12.16$$\pm 2.37$ & $\underline{ 94.26}$$\pm 2.42$ & $\bf 94.53$$\pm 1.85$\\
			\bottomrule
				
		\end{tabular}
		}
	\end{center}
	\caption{\label{tab:G_2_D_sup}{\bf Semi-supervised Heterogeneous Domain Adaptation} results for adaptation  from GoogleNet  to Decaf  representations with different values of  $n_t$.}
\end{table}

\begin{table}[t]
	\begin{center}
	\resizebox{0.65\columnwidth}{!}{
		\begin{tabular}{ccccc}
		\toprule
			\multicolumn{5}{c}{GoogleNet  $\rightarrow$ Decaf }\\
			\midrule
			{Domains} & {CCA} & {KCCA} & {EGW} & {COOT}\\
			\midrule
			C$\rightarrow$A & $11.30$$\pm 4.04$ & $\underline{ 14.60}$$\pm 8.12$ & $8.20$$\pm 2.69$ & $\bf 25.10$$\pm 11.52$\\
			C$\rightarrow$C & $13.35$$\pm 4.32$ & $\underline{ 17.75}$$\pm 10.16$ & $11.90$$\pm 2.99$ & $\bf 37.20$$\pm 14.07$\\
			W$\rightarrow$A & $14.55$$\pm 10.68$ & $\underline{ 25.05}$$\pm 24.73$ & $14.55$$\pm 2.05$ & $\bf 39.75$$\pm 17.29$\\
			A$\rightarrow$C & $13.80$$\pm 6.51$ & $\underline{ 20.70}$$\pm 17.94$ & $16.00$$\pm 2.44$ & $\bf 30.25$$\pm 18.71$\\
			A$\rightarrow$A & $16.90$$\pm 10.45$ & $\underline{ 28.95}$$\pm 30.62$ & $12.70$$\pm 1.79$ & $\bf 41.65$$\pm 16.66$\\
			W$\rightarrow$W & $14.50$$\pm 6.72$ & $\underline{ 24.05}$$\pm 19.35$ & $9.55$$\pm 1.77$ & $\bf 36.85$$\pm 9.20$\\
			W$\rightarrow$C & $13.15$$\pm 4.98$ & $\underline{ 14.80}$$\pm 8.79$ & $11.40$$\pm 2.65$ & $\bf 30.95$$\pm 17.18$\\
			A$\rightarrow$W & $10.85$$\pm 4.62$ & $\underline{ 14.40}$$\pm 12.36$ & $12.70$$\pm 2.99$ & $\bf 40.85$$\pm 16.21$\\
			C$\rightarrow$W & $18.25$$\pm 14.02$ & $\underline{ 25.90}$$\pm 25.40$ & $11.30$$\pm 3.87$ & $\bf 34.05$$\pm 13.82$\\
			\bf Mean & $14.07$$\pm 2.25$ & $\underline{ 20.69}$$\pm 5.22$ & $12.03$$\pm 2.23$ & $\bf 35.18$$\pm 5.24$\\
		\bottomrule
		\end{tabular}
		}
	\end{center}
		\caption{	\label{tab:G_2_D_unsup} {\bf Unsupervised Heterogeneous Domain Adaptation} results for adaptation from GoogleNet to Decaf representations.}
\end{table}}{}

\chapter{Conclusion}
\epigraph{Le soleil est noyé. - C'est le soir - dans le port \\
Le navire bercé sur ses câbles, s'endort}{-- Tristan Corbière, \textit{Les Amours jaunes}}
\minitoc
\label{cha:conclusion}

\section{Overview of the contributions}

This thesis presents a set of optimal transport tools for dealing with probability distributions on \emph{incomparable spaces}, or equivalently probability distributions whose supports do not lie in a common metric space. We explained how this problem occurs \textit{e.g.} when one needs to consider some structural knowledge about the data or when the data come from heterogeneous sources. As a first instance of probability distributions on incomparable spaces we studied the setting of \emph{structured data} such as labeled graphs, times series or any ``relational'' data whose structure can be modelled through a notion of \emph{cost} or \emph{similarity}. We showed how to describe them as a probability distributions and how to compare them using the so-called Fused Gromov-Wasserstein distance which builds upon Wasserstein and the Gromov-Wasserstein distances. This new optimal transport distance was successfully applied in a graph context where it finds applications for the classification, clustering or summarization of labeled graphs. 

As the main building block of $FGW$, the Gromov-Wasserstein distance is a central notion of this thesis. We attempted to bridge the gap between the understanding of the Wasserstein distance and $\gw$ where we consider the special case of Euclidean spaces. This setting allows us to derive a \emph{sliced} approach, based on the first closed-form expression for $\gw$ between 1D probability distributions. We called it Sliced Gromov-Wasserstein, \textit{akin} to the Sliced Wasserstein distance which has recently found many applications in machine learning. The Euclidean setting reveals to be also a good starting point for analysing the regularity of $\gw$ optimal transport plans and, as such, we partially answered the following question: can we find guarantees on the probability measures so that an optimal transport plan for $\gw$ can be expressed through a deterministic function? This question, central in linear optimal transport theory, can be tackled with the celebrated Brenier's theorem in the context of the Wasserstein distance but was still quite under-addressed when dealing with $\gw$. 

Although $\gw$ is a powerful tool for comparing probability distributions on incomparable spaces it is limited in its ability to find correspondences between the \emph{features} of the samples of these distributions. This limitation originates from the method \emph{itself} which discards the feature information by focusing only on pair-to-pair distance matrices. To circumvent this constraint we proposed a novel optimal transport distance which both finds the correspondences between the samples \emph{and} the features of the distributions. This work is based on the CO-Optimal transport framework which computes two optimal transport plans directly on the raw data unlike $\gw$ which requires pre-computed pair-to-pair distance or similarity matrices. We showed that it is particularly suited for problems such as Heterogeneous Domain Adaptation and Co-clustering. In the light of the previous results we drew interesting connections between COOT and the $\gw$ distance the latter being a special case of the former in the \emph{concave regime} and with data described by distance or similarity matrices. 

\paragraph{Preprints \& papers published during this thesis}

\begin{itemize}
\item \cite{Vayer_2020} \textbf{T. Vayer}, L. Chapel, R. Flamary, R. Tavenard, and N. Courty. \emph{Fused Gromov-Wasserstein distance for structured objects}. Journal. In: Algorithms. (2020)
\item \cite{vay_struc} \textbf{T. Vayer}, L. Chapel, R. Flamary, R. Tavenard, and N. Courty. \emph{Optimal Transport for structured
data with application on graphs}. International conference. In: International Conference on Machine Learning (ICML). (2019)
\item \cite{vay_sliced_gromov_2019} \textbf{T. Vayer}, R. Flamary, R. Tavenard, L. Chapel and N. Courty. \emph{Sliced Gromov-Wasserstein}. International conference. In: Advances in Neural Information Processing Systems (NeurIPS). (2019)
\item \textbf{T. Vayer}, L. Chapel, R. Flamary, R. Tavenard, and N. Courty. \emph{Fused Gromov Wasserstein distance}. National conference. In: Conférence sur l'Apprentissage automatique (CAp). (2018) 
\item \textbf{T. Vayer}, L. Chapel, R. Flamary, R. Tavenard, and N. Courty. \emph{Transport Optimal pour les Signaux sur Graphes}. National conference. In: GRETSI. (2019)
\item \cite{redko2020cooptimal} I. Redko, \textbf{T. Vayer}, R. Flamary and N. Courty. \emph{CO-Optimal Transport}. In: arXiv:2002.03731. Submitted to NeurIPS (2020)
\item \cite{vayer2020time} \textbf{T. Vayer}, L. Chapel, N. Courty, R. Flamary, Y. Soullard and R. Tavenard. \emph{Time Series Alignment with Global Invariances
}. In: arXiv:2002.03848. (2020) 
\end{itemize}

\section{Perspective for further works}

There are many possible extensions and improvements of the works developed here. To conclude we discuss potential limitations and further works of the methods proposed in this thesis from the OT perspective to the machine learning point of view.

\paragraph{What are the improvements and limitations of the different results of this thesis from the OT perspective?} The work about $\gw$ on Euclidean spaces suggests, in our opinion, interesting further works. The slicing approach for $\gw$ finds \textit{e.g.} natural extensions inspired by the Wasserstein distance. An immediate one is the max-sliced approach \cite{Deshpande_2019_CVPR} where only the projection which maximizes the loss is drawn. This setting is directly applicable in our case without changing the theoretical properties of $SGW$ stated in Theorem \ref{propertiessgw}. Moreover, one limitation of $SGW$ in the formulation that we proposed is the map $\Delta$ which has to be chosen or optimized when one wants to compute $SGW$ when dimensions differ or in order to retrieve the invariants of $GW$. The max-sliced approach could be a interesting remedy for this situation. We can for example draw one line in each space without requiring any map $\Delta$ but with the price of losing the divergence property (\ie\ $max-SGW(\mu,\mu)$ could be different from zero). Another potential direction would be to consider many lines in each space and then find their correspondences by using \textit{e.g.} Wasserstein which could define a proper metric \textit{w.r.t.} isomorphisms but which will come with a cubic complexity. Interestingly enough this approach would exactly be the \COOT\ distance but with replacing the matching of the feature by a sliced Wasserstein. In any case there are room for improvements in the way of defining $SGW$ when dimensions of the support of the probability measures differ. 

Regarding the theoretical aspects of $SGW$ another interesting line of research, in our opinion, would be to draw connections with the statistical properties of the Sliced Wasserstein distance which is  known to behave well in terms of convergence of finite samples. We believe that similar type of studies for the Sliced Gromov-Wasserstein could be promising. 

From a computational OT point of view the CO-Optimal Transport framework seems also to be quite suited for large scale datasets. Indeed since the BCD procedure used for solving COOT is based on alternating linear optimal transport problems another line of works could be to rely on the dual or semi dual formulations of linear OT, used for large scale OT, in order to compute COOT for large scale datasets whose samples are in incomparable spaces. Finally from a theoretical point of view a continuous formulation of COOT could also be of interest for studying the properties of $\gw$, and especially the regularity of its optimal plans. Indeed we have proven that an optimal solution for COOT can be found \textit{via} permutation matrices, whose continuous counterparts are deterministic push-forwards. As such, can we find a natural extension for COOT which preserves this property in the continuous setting? Also if the problem is concave, such as when $\gw$ is considered with squared Euclidean distances, can we preserve the result stating that COOT and $\gw$ are equivalent so that an optimal solution of the first is optimal solution for the second? In this way can we conclude that there exists an optimal map for $\gw$ that is supported on a Monge map when the problem is concave? We believe, enthusiastically, that these questions worth investigating.

\paragraph{How other fields can contribute to the frameworks developed in this thesis?} One of the most challenging improvement of the work developed in this thesis is maybe to lighten the computational complexity of calculating $FGW$ which is driven by the complexity of the GW distance. In this form, the computational complexity $FGW$ renders this framework limited to small to medium scale \textit{scenarii} and, in particular, is inapplicable for very large graphs. Moreover one can not hope to directly rely on the works about $\gw$ on Euclidean spaces in order to derive tractable formulations since the structures of labeled graphs are usually highly non-Euclidean. However it worth pointing out that approximating $\C_1,\C_2$ with Euclidean matrices can be done in several ways \cite{BorgGroenen2005,glunt,alfakih_solving_1999,liberti}. We believe that a kind of \emph{no free lunch theorem} applies here: if we faithfully model the structure of the data it is likely to result in a more precise notion of distance between structured data but also in a more challenging optimization problem than if we approximate it for better scalability. Needless to say, since $GW$ is at the backbone of our method, any further improvements in computational efficiency for non convex Quadratic programs, and by way of consequence for $GW$, would directly benefit to $FGW$. In this way progress in graph matching could definitely contribute to derive efficient algorithms for solving $FGW$ as both are intrinsically related. Another interesting study would be regarding the design of the matrices $\C_1,\C_2$ representative of the structures for $FGW$. We considered in this thesis that these matrices are given, such as shortest-path matrices, yet, one could argue that this is a strong prior on the structures of the graphs and that no \emph{perfect} choice arises in practice. The problem of finding ``good'' structure matrices is very related with the metric learning field \cite{bellet2013survey} and further appealing works for $FGW$ could be to go into this direction and to add $\C_1,\C_2$ into the learning process. Related to this problem we believe that very insightful connections between $FGW$ and the field of signal graph processing \cite{gsp} can be drawn and that both can benefit from each other. As one example: can we build upon the spectral tools from signal graph processing in order to find adequate measures of structure $\C_1,\C_2$ of the graphs?

\paragraph{How do the above-mentioned optimal transport tools can, or can not, be used \emph{for} machine learning? }As mentioned throughout this thesis, having a adequate measure for comparing data lying on incomparable spaces can be useful for a wide range of machine learning tasks. For applications involving structured data we proposed the $FGW$ distance based on optimal transport. Note, however, that many ways of comparing structured data have been proposed in the literature: from the design of dedicated kernels (see \cite{kriege_survey_2020} for a comprehensive survey on this topic) to other choices of distances between graphs \cite{bento_family_2019,metric_guide} including the popular graph-edit distance \cite{Willett98chemicalsimilarity,raymond_edit}. More recently, end-to-end approaches \cite{graph_similarity_learning,bai_gnn,riba_2018,gnn_graph_search} attempt to \emph{learn} a similarity measure function between graphs based on graph neural networks, \ie\ to learn a neural network based function that takes two graphs as input and outputs the desired similarity. The question of finding a ``good'' similarity measure for structured data is far from being closed and is naturally dependent of the application. The choice of $FGW$ can be motivated by its metric properties which allow detecting whether two graphs are isomorphic and, as such, could be used as an alternative to the graph-edit distance. However, despite its appealing properties, one could question the optimal transport framework on which this method is based. As described in this thesis the coupling matrices are used to find a probabilistic matching between \emph{all} the nodes of two graphs. Conversely, and depending on the application, one might want to match only a \emph{small} portion of the nodes of two graphs which is not possible using the $FGW$ framework which considers \emph{global} matchings. In this way, in a scenario where only the \emph{local} structure is important, the $FGW$ machinery appears to be quite disproportionate, and some kernel methods which are built on local structures may be more suited. The use of the $FGW$ distance for structured data such as time series is also questionable. Indeed the set of couplings \emph{itself} is constrained so that a matching ``back in time'' is made possible. On the contrary, it may be more interesting to force that the points which are matched to the target sequence at time $t$ can only depend on the source sequence up to time $t$ as done \textit{e.g.} using Dynamic Time Warping distances \cite{pmlr-v70-cuturi17a,sakoe1978dynamic}. To remedy to this problem the recent casual optimal transport framework \cite{lassale_casual,veraguas2016causal,acciaio2020cournotnash,ZalashkoAnastasiia2017Cot} may be more suited for these type of problems. Apart from these limitations the $FGW$ framework could be useful for other applications than these mentioned in this thesis. As a first instance it could be used as a way for learning graph neural networks (GNN). Given a set of graphs we could for example minimize a triplet loss \cite{triplet_loss} or a graph similarity score \cite{simgnn} which impose that the distances in the embedding of the GNN are close to the $FGW$ distances so as to force the GNN to produce similar embeddings for similar graphs. Moreover the properties of the Fréchet barycenter with $FGW$ could be elaborated, especially the projection on a smaller graph using $FGW$. A perspective on this topic would be to draw connections with graph signal processing and more standard coarsening procedures (see \cite{loukas_coarsening} and references therein) by questioning the ability of the $FGW$ projection to reduce the graph without altering too much its spectral properties. 

Finally the results of this thesis suggests that COOT may be suited for Heterogeneous Domain Adaptation. An interesting further work on this topic would be to see if we can confirm these empirical results with theoretical guarantees such as bounds for Heterogeneous Domain Adaptation by relying on the COOT framework. For example can we prove that the adaptation is easier when the datasets are close from each other \textit{w.r.t.} the COOT distance than if they are far from each other? To the best of our knowledge the question of finding bounds for HDA problems is still quite under-addressed \cite{boundHDA} and COOT may lead to promising further works on this topic.
\par
\bigskip
All these considered we hope that the works proposed in this thesis will pave the path for positive and interesting studies on various and interdisciplinary topics in machine learning and that it will also contribute to the richness of the optimal transport theory, which is far from being extinguished.


 \bibliographystyle{StyleThese}

\AtEndDocument{\includepdf[pages=2,scale=1]{./main_style.pdf}}

\end{document}